\documentclass[a4paper, 11pt, twoside,openright]{book}

\usepackage[T1]{fontenc}
\usepackage{graphicx}
\usepackage[hidelinks]{hyperref}
\usepackage{setspace} 
\usepackage[cmex10]{amsmath}
\usepackage{amssymb}
\usepackage{amsthm,commath}
\usepackage{textcomp}
\usepackage{gensymb}
\usepackage{todonotes}
\usepackage[pagestyles]{titlesec}
\usepackage{setspace}
\usepackage{pgfplotstable}
\usepackage{blkarray, bigstrut} %

\usepackage{subcaption}
\usepackage{caption}
\usepackage{algorithm,tabularx}
\usepackage[noend]{algpseudocode}

\usepackage{lscape}
\usepackage{color, colortbl}
\usepackage{pgfplots} 
\usepackage{geometry}
\usepackage{multirow}
\usepackage{pbox}

\usepackage[thicklines]{cancel}
\usepackage{bbm}
\newtheorem{lemma}{Lemma}

\newtheorem*{lemma**}{Lemma\theoremnum}
\newenvironment{lemma*}[1][]{%
  \edef\theoremnum{\if\relax\detokenize{#1}\relax\else~#1\fi}
  \begin{lemma**}
}{%
  \end{lemma**}
}

\numberwithin{lemma}{subsection} 
\newtheorem{theorem}{Theorem}

\newtheorem*{theorem**}{Theorem\theoremnum}
\newenvironment{theorem*}[1][]{%
  \edef\theoremnum{\if\relax\detokenize{#1}\relax\else~#1\fi}
  \begin{theorem**}
}{%
  \end{theorem**}
} 

\numberwithin{theorem}{subsection} 
\newtheorem{defn}{Definition}
\newtheorem{corollary}{Corollary}[lemma]

\newtheorem*{corollary**}{Corollary\theoremnum}
\newenvironment{corollary*}[1][]{%
  \edef\theoremnum{\if\relax\detokenize{#1}\relax\else~#1\fi}
  \begin{corollary**}
}{%
  \end{corollary**}
}

\numberwithin{corollary}{subsection} 

\usepackage{pifont}
\newcommand{\cmark}{\ding{51}}%
\newcommand{\xmark}{\ding{55}}%

\usepackage{bm}
\usepackage{comment}
\usepackage{booktabs}
\usepackage{enumitem} 

\usepackage{multicol}
\definecolor{Gray}{gray}{0.9}

\graphicspath{{images/}}

\DeclareMathOperator{\micro}{$\micro$}

\DeclareMathOperator*{\argmax}{arg\,max}
\DeclareMathOperator*{\argmin}{arg\,min}
\DeclareMathOperator{\Tr}{Tr}

\makeatletter
\def\thickhrulefill{\leavevmode \leaders \hrule height 1ex \hfill \kern \z@}
\def\@makechapterhead#1{%
  \vspace*{120\p@}
  {\parindent \z@ \centering \reset@font
        \thickhrulefill\quad
        \scshape \@chapapp{} \thechapter
        \quad \thickhrulefill
        \par\nobreak
				\vspace*{10\p@}%
        \interlinepenalty\@M
        \hrule
        \vspace*{10\p@}%
        \Huge \bfseries #1\par\nobreak
        \par
        \vspace*{10\p@}%
        \hrule
    \vskip 100\p@
  }}
\def\@makeschapterhead#1{%
  \vspace*{10\p@}%
  {\parindent \z@ \centering \reset@font
        \thickhrulefill
        \par\nobreak
        \vspace*{10\p@}%
        \interlinepenalty\@M
        \hrule
				\scshape
        \vspace*{10\p@}%
        \Huge \bfseries #1\par\nobreak
        \par
        \vspace*{10\p@}%
        \hrule
    \vskip 100\p@
  }}

\newcommand{\quotes}[1]{``#1''}
\newcommand*\diff{\mathop{}\!\mathrm{d}}

\algnewcommand{\Inputs}[1]{%
  \State \textbf{Inputs:}
  \Statex \hspace*{\algorithmicindent}\parbox[t]{.8\linewidth}{\raggedright #1}
}
\algnewcommand{\Initialize}[1]{%
  \State \textbf{Initialize:}
  \Statex \hspace*{\algorithmicindent}\parbox[t]{.8\linewidth}{\raggedright #1}
}

\makeatletter
\def\BState{\State\hskip-\ALG@thistlm}
\makeatother

\makeatletter
\newcommand{\multiline}[1]{%
  \begin{tabularx}{\dimexpr\linewidth-\ALG@thistlm}[t]{@{}X@{}}
    #1
  \end{tabularx}
}

\newpagestyle{main}{
	\sethead[\itshape\MakeUppercase{\chaptername\ \thechapter. \chaptertitle}][][] 
	{}{}{\itshape\MakeUppercase{\thesection\ \sectiontitle}}
	\headrule	\setfoot{}{\usepage}{}
}

\pgfplotsset{compat=1.17}

\makeindex
\begin{document}

\frontmatter

\begin{titlepage}

\begin{center}

	{\LARGE Nunzio Alexandro Letizia}\\
	\vspace{3.5mm}
	
	\begin{spacing}{2}
	{ \Huge \textbf{\textsc{Deep Learning Models for}}}\\
	{ \Huge \textbf{\textsc{Physical Layer Communications}}}\\
	\end{spacing}
	\vspace{3.5mm}
	
	{\LARGE \textsc{DOCTORAL THESIS}} \\

submitted in fulfilment of the requirements for the degree of \\
Doktor der Technischen Wissenschaften\\
	-------------------------------------- \\
Alpen-Adria-Universit{\"a}t Klagenfurt\\
Fakult{\"a}t f{\"u}r Technische Wissenschaften\\
  
\vspace{3mm}
\end{center}
	\begin{center}
			\includegraphics[width=0.37\textwidth]{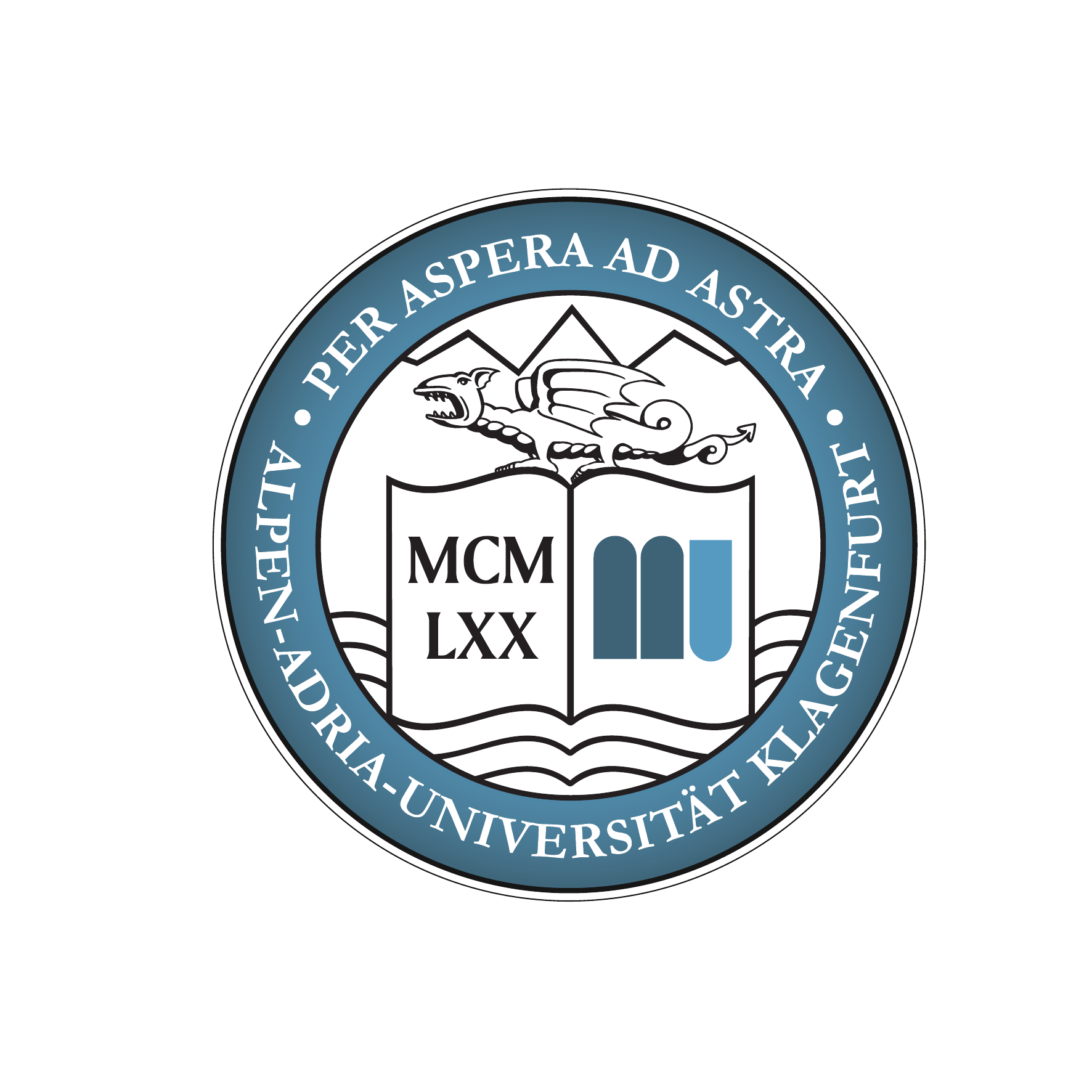}			
	\end{center}

\begin{multicols}{2}
    \noindent \large{\textbf{Supervisor}}\\
    \noindent Univ.--Prof. Dr. Andrea M. Tonello\\
    \noindent \normalsize{Alpen-Adria-Universit{\"a}t Klagenfurt}\\
    \noindent \normalsize{Institut f{\"u}r Vernetzte und Eingebettete Systeme}\\

    \noindent \large{\textbf{Co-supervisor}}\\
    \noindent Univ.--Prof. Dr. Bernhard Rinner\\
    \noindent \normalsize{Alpen-Adria-Universit{\"a}t Klagenfurt}\\
    \noindent \normalsize{Institut f{\"u}r Vernetzte und Eingebettete Systeme}\\
\end{multicols}
\vspace{1mm}

\begin{multicols}{2}
    \noindent \large{\textbf{Evaluator}}\\
    \noindent Prof. Dr.-Ing. Stephan ten Brink\\
    \noindent \normalsize{Universit{\"a}t Stuttgart}\\
    \noindent \normalsize{Institute of Telecommunications}\\
    
    \noindent \large{\textbf{Evaluator}}\\
    \noindent Dr. Alex Dytso\\
    \noindent \normalsize{Qualcomm Engineering}\\
    \noindent \normalsize{Princeton University}\\

\end{multicols}

\vspace{6mm}
\centering\large{Klagenfurt, October 2024}

\end{titlepage}

\pagestyle{plain} 

\newpage\null\thispagestyle{empty}\newpage
\newpage

\section*{Affidavit}
\label{sec:affidavit}

I hereby declare in lieu of an oath that
\begin{itemize}
	\item the submitted academic paper is entirely my own work and that no auxiliary materials have been used other than those indicated,
	\item I have fully disclosed all assistance received from third parties during the process of writing the thesis, including any significant advice from supervisors,
	\item any contents taken from the works of third parties or my own works that have been included either literally or in spirit have been appropriately marked and the respective source of the information has been clearly identified with precise bibliographical references (e.g. in footnotes),
     \item I have fully and truthfully declared the use of generative models (Artificial Intelligence, e.g. ChatGPT, Grammarly Go, Midjourney) including the product version,
	\item to date, I have not submitted this paper to an examining authority either in Austria or abroad and that
	\item when passing on copies of the academic thesis (e.g. in printed or digital form), I will ensure that each copy is fully consistent with the submitted digital version.
\end{itemize}
I am aware that a declaration contrary to the facts will have legal consequences.

\vspace{2cm}

Nunzio Alexandro Letizia m.p.

\vspace{2cm}

Klagenfurt, July 2024

\newpage
\section*{Dedication}
\label{sec:dedication}

This thesis is dedicated to:

\begin{itemize}
    \item Anna M., I am profoundly grateful for her endless patience, support, and belief in me;
    \item my family, whose unwavering encouragement has always been unconditional;
    \item friends and colleagues, being part of a group of wonderful people has given me extra motivation and lots of smiles. To future PhD students, try to exchange and share ideas, do not isolate yourself.
\end{itemize}

\newpage\null\thispagestyle{empty}\newpage
\newpage

\section*{Abstract}
\label{sec:abstract}

The recent surge in data availability and computing resources has empowered researchers to harness the potential of machine learning (ML) techniques, leading to significant advancements across various engineering fields. ML, particularly deep learning (DL) algorithms, excel in tackling problems where traditional physical modeling proves inadequate or mathematically intractable. In fact, DL leverages real-world observations of phenomena to automatically acquire knowledge and uncover inherent relationships within the data.

For the past decades, communication engineering has thrived on a foundation of meticulously crafted physical and mathematical models, which have played a pivotal role in groundbreaking achievements, forming the very backbone of our modern technological landscape. From the efficient transmission of information across the globe to the development of reliable communication protocols, these models have been instrumental in shaping the way we connect and share information. 
However, a recent paradigm shift is underway, with a growing focus on top-down, data-driven learning models. This shift is particularly evident in areas like channel modeling and physical layer design, where finding universally optimal strategies remains an open challenge.

This thesis aims to address some of the most crucial unsolved challenges within physical layer communications by leveraging innovative DL paradigms. We mathematically reformulate classic problems, such as channel capacity and optimal coding-decoding schemes, using the ML language, and we attempt to solve them for any arbitrary communication medium.
The envisioned methodology involves designing and developing the architecture, algorithms, and necessary code to train the corresponding DL model. By leveraging powerful tools, we propose novel solutions to long-standing problems in the field. For instance, one specific challenge we tackle is the construction of capacity-achieving input distributions in noisy channels. Our approach has the potential to significantly improve channel coding compared to traditional methods, effectively leading to a reduction in bit error rate for data transmission.

We believe this research holds significant promise for the development of next-generation communication systems and can potentially influence the design principles in future 6G standards and beyond.


\tableofcontents
\listoffigures
\listoftables
\section*{Abbreviations}
\label{sec:abbreviations}
List of abbreviations used in the thesis:

\begin{multicols}{2}

	\fontsize{7}{7}\selectfont
	\begin{itemize}
		\setlength\itemsep{0em}
  		\item ACG - Average Channel Gain
		\item AE - Autoencoder
		\item APIN - APeriodic Impulsive Noise
            \item AWGN - Additive White Gaussian Noise
		\item BER - Bit Error Rate
            \item BLER - Block Error Rate
		\item BPSK - Binary Phase Shift Keying
            \item CAE - Contractive Autoencoder
            \item CB - Coherence Bandwidth
            \item CBGN - Colored BackGround Noise
		\item CDF - Cumulative Distribution Function
            \item CNN - Convolutional Neural Network
            \item CODINE - Copula Density Neural Estimation
            \item CORTICAL - Cooperative Channel Capacity Learning
		\item CTF - Channel Transfer Function
            \item DAE - Denoising Autoencoder
            \item DCGAN - Deep Convolutional Generative Adversarial Network
		\item DFT - Discrete Fourier Transform
            \item DIME - Discriminative Mutual Information Estimation
            \item DL - Deep Learning
            \item E-FMS - Embedded Flight Management System
            \item ELBO - Evidence Lower Bound
            \item FID - Frèchet Inception Distance
            \item FVBN - Fully Visible Belief Network
		\item GAN - Generative Adversarial Network
            \item GLOW - Generative Flow
            \item GP - Gaussian Process
            \item GPT - Generative Pre-trained Transformer
            \item HD - Hellinger Distance
            \item HIF - High Impedance Fault
            \item HVAE - Hierarchical Variational Autoencoder
            \item IC - Individual Channel
            \item IE - Impedance Entanglement
            \item IM - Impedance Modulation
            \item IS - Inception Score
            \item JS - Jensen-Shannon
            \item KID - Kernel Inception Distance
            \item KL - Kullback-Leibler
            \item KURT - Kurtosis
            \item LIF - Low Impedance Fault
            \item LTI - Linear Time-Invariant
            \item LVM - Latent Variable Model
            \item MAP - Maximum A-Posteriori
  		\item MaxL - Maximum Likelihood
            \item MCMC - Markov Chain Monte Carlo
            \item MDP - Markov Decision Process
            \item MHVAE - Markovian Hierarchical Variational Autoencoder
		\item MI - Mutual Information
		\item MIMO - Multiple Input Multiple Output
            \item MIND - Mutual Information Neural Decoder
  		\item ML - Machine Learning
            \item MLP - Multi Layer Perceptron
            \item MMD - Maximum Mean Discrepancy
            \item NBN - Narrow-Band Noise
            \item NICE - Non-linear Independent Components Estimation
		\item NN - Neural Network
            \item NVP - Non-Volume Preserving
		\item OFDM - Orthogonal Frequency Division Multiplexing
		\item PCA - Principal Component Analysis
		\item PDF - Probability Density Function
		\item PINS - Periodic Impulsive Noise Synchronous to mainS
		\item PINAS - Periodic Impulsive Noise Asynchronous to mainS
		\item PLC - Power Line Communication
		\item PLN - Power Line Network	
		\item PSD - Power Spectral Density
            \item PSK - Phase-Shift Keying
            \item QAM - Quadrature Amplitude Modulation
            \item ReLU - Rectified Linear Unit
            \item RGB - Red Green Blue
            \item RKHS - Reproducing Kernel Hilbert Space
            \item RL - Reinforcement Learning
            \item RMS-DS - Root Mean Square Delay Spread
            \item RST - Recursive Smooth Trajectory
		\item RX - Receiver
            \item SGD - Stochastic Gradient Descent
            \item SGN - Segmented Generative Network
		\item SISO - Single Input Single Output
            \item SKW - Skewness
		\item SNR - Signal to Noise Ratio
		\item SOM - Self-Organizing Map
            \item STD - Standard Deviation
            \item STFT - Short Time Fourier Transform
            \item TV - Total Variation
		\item TX - Transmitter
            \item UAV - Unmanned Aerial Vehicle
            \item VAE - Variational Autoencoder
            \item VM - Voltage Modulation
	\end{itemize}
\end{multicols}
\normalsize

\mainmatter

\chapter{Introduction}
\label{sec:intro}



\section{Motivation}
Communications engineering is the field studying the design, implementation, and optimization of systems that transmit information from one point to another. 
The exchange of information can occur over wired connections (e.g., twisted pairs, fiber optics), wireless channels (e.g., cellular networks), or a combination of both. 

Communication systems have revolutionized the way we interact and access information. This field has played a pivotal role in driving economic growth, fostering global collaboration and shaping social landscapes. However, the ever-increasing demand for data transmission poses significant challenges. The relentless growth of data traffic requires constant innovation to overcome limitations in signal degradation, bandwidth, and interference.

The physical layer is the first and lowest layer of the Open Systems Interconnection (OSI) model. It represents the heart of a communication system as it is responsible for the raw transmission of bits over a physical medium. This layer handles the encoding and decoding of data, the modulation and demodulation of signals, and it deals with several challenges posed by the medium such as attenuation, noise, and interference.  
A robust physical layer is essential for ensuring reliable data transmission, maximizing bandwidth, and supporting the increasing demands of modern communication systems.

Innovation in the physical layer design necessitates to find solutions to complex problems such as channel impairments, where real-world channels are far from ideal models and the transmitted signals suffer from fading, path loss, multi-path propagation, and various sources of noise. Another relevant problem is represented by bandwidth limitations, since the available spectrum for wireless communication is a finite resource. As demand grows, engineers must find alternative and creative ways to maximize data rates within constrained bandwidth allocations.
Interference also constitutes a serious challenge in crowded wireless environments, as signals from various sources can overlap and corrupt each other, making it difficult to isolate the desired transmission.
Lastly, power and complexity constraints set strict hardware requirements, especially in mobile and battery-powered devices. Therefore, achieving high-performance communication under multiple limitations requires careful optimization and innovative techniques.

Traditional approaches to physical layer design rely on complex mathematical models and hand-crafted signal processing algorithms. While effective, these methods are often sub-optimal and can struggle to fully adapt to the dynamic and unpredictable nature of communication channels. Deep learning (DL) offers a valid and compelling alternative. Its ability to learn complex patterns and relationships directly from data makes it a powerful tool for tackling physical layer challenges.

The traditional design approach was firstly proposed by Shannon \cite{Shannon1948}. He suggested to represent any communication system as a chain of mathematically modeled blocks, namely, the transmitter, the channel, and the receiver. 
The transmitter comprises several interconnected blocks, a source coder, a channel coder and a signal modulator. The channel is usually described by a transfer function and an additive noise term. The receiver, instead, serves as the counterpart to the transmitter and is often broken down into stages mirroring those of the transmitter, i.e., signal demodulation, channel decoding, and source decoding.

\begin{figure}
	\centering
	\includegraphics[scale=0.6]{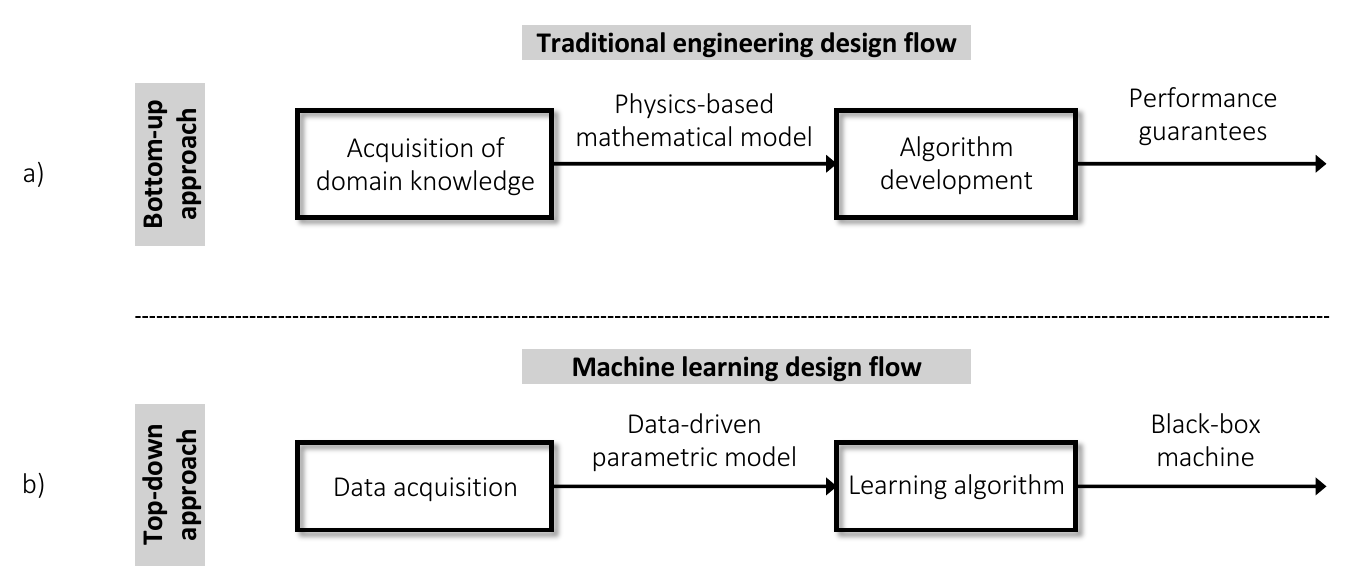}
	\caption{Different engineering design approaches: a) bottom-up and b) top-down.}
	\label{fig:intro_approach}
\end{figure}

Three generations of scientists and engineers grew up with this mathematical mindset which offered excellent tools to acquire domain knowledge and use it to create models for each constituent block, thereby rendering the overall behavior tractable. Such a framework intrinsically has the advantage that each component can be individually studied and optimized. We would refer to this approach as physical and \textit{bottom-up} and it is represented in Fig.~\ref{fig:intro_approach}a.

From an epistemological point of view, the mathematical theory of communications is based on knowledge coming from \textit{a-priori} hypotheses and intuitions. On the contrary, \textit{a-posteriori} knowledge is created by what is known and observed from experience, therefore generated via an empirical analysis and a \textit{top-down} approach. In his \textit{Critique of Pure Reason} \cite{Kant}, Kant differentiates between two types of knowledge: knowledge of the structure of time and space and their relationships, is a-priori knowledge; knowledge acquired from experience is a-posteriori knowledge; most of our knowledge comes from the process of learning and observing phenomena, and without the a-priori one it would be impossible to reach the true knowledge. 
In this respect, the broad field of machine learning (ML) \cite{Bishop2006,DLBook} can be considered an implementation by humans of techniques in machines to acquire knowledge from a-posteriori observations of realizations of natural phenomena, essentially by translating the complexity of real-world data into the complexity of the model, see Fig.~\ref{fig:intro_approach}b. 

ML is bringing new lymph in the domain of communication systems modeling, design, optimization, and management. It provides a paradigm shift: rather than only concentrating on a physical bottom-up description of the communication scheme, ML aims to learn and capture information from a collection of data, to derive the input-output relations of the observed system. Such learning strategy can be effectively adopted to cover multitude applications across all three fundamental protocol stacks: the physical layer as we introduced, but also the MAC layer and the network layer \cite{Simeone2018}.

\begin{figure}
	\centering
	\includegraphics[scale=0.65]{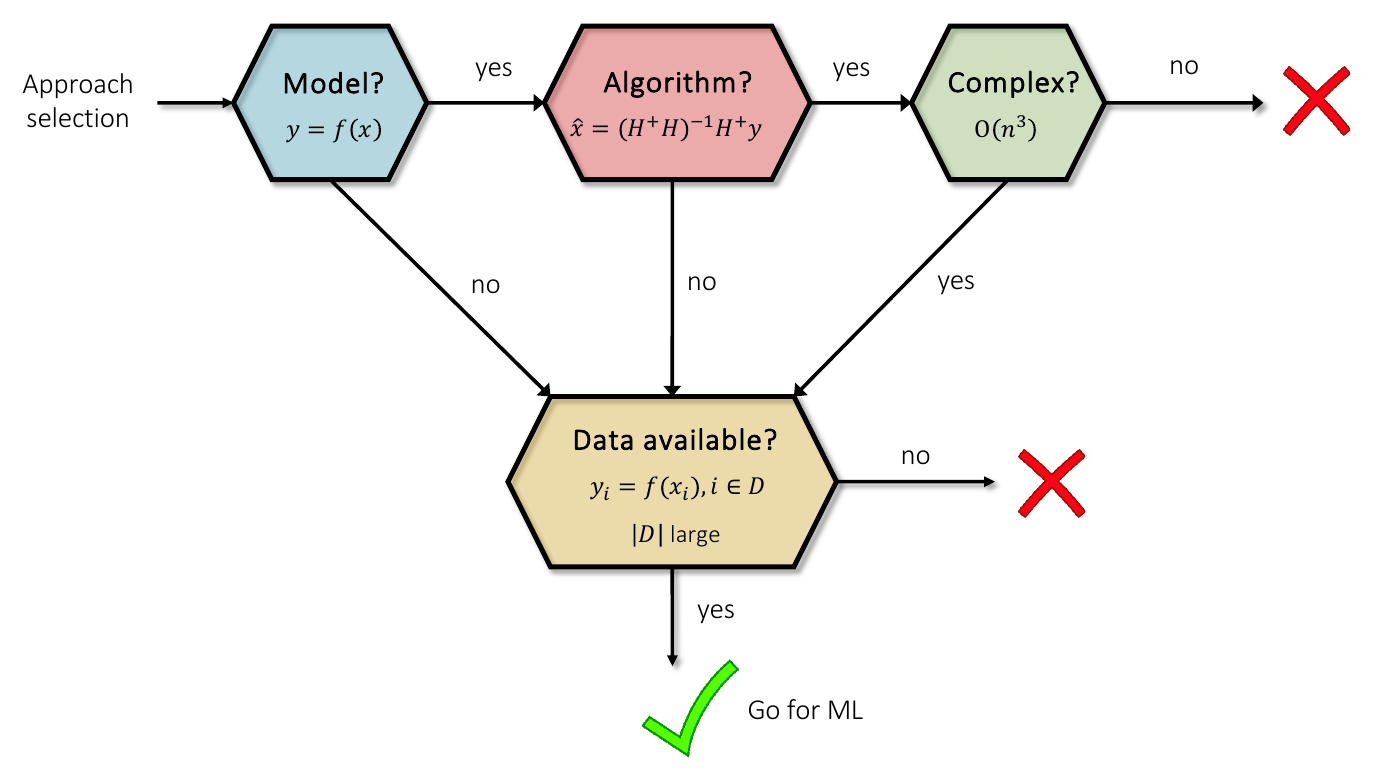}
	\caption{Flowchart illustrating when to use ML for communication engineering design.}
	\label{fig:intro_when_ml}
\end{figure}

Given the current shift but also hype around ML, it is important to remark that a top-down data-driven approach should not always be pursued regardless the specifications of the problem. In fact, common sense guidelines \cite{Simeone2018} explaining \textit{when to use ML} tools for communications, and more in general, engineering tasks (see Fig.~\ref{fig:intro_when_ml}), suggest to adopt ML when a physics-based model (e.g., channel model) is lacking or the existing algorithms are too complex to be implemented (e.g., decoding strategy). Moreover, a sufficiently large training dataset shall be collected while the phenomenon being learned (e.g., noise) is stationary during the observation period. 

Finally, ML techniques can complement traditional approaches. For instance, ML can be integrated with analytical models or simulation-based methods to enhance performance and study limitations. In fact, the integration of ML and DL techniques with standard approaches in communication engineering offers a promising avenue for addressing the complexities and challenges of modern communication systems. By leveraging the complementary strengths of both methodologies, we could enhance system performance, adaptability, and reliability across various applications. 

\section{Scientific contributions}
Physical layer design encompasses multiple tasks where a data-driven approach can be pursued. 
At the transmitter side, DL has recently been applied to tackle tasks such as source and channel coding \cite{8461983}, signal encoding \cite{Dorner2018}, analog-to-digital (ADC) conversion and digital-to-analog conversion (DAC) \cite{8807322}. For more details, we refer to Sec.~\ref{sec:autoencoders_related} and Sec.~\ref{sec:cortical_intro}.

At the receiver side, channel detection and decoding represents, perhaps, the most immediate DL application since it can be treated as a standard classification task \cite{Nachmani2018}. Other tasks which benefit from data-driven models are channel estimation and equalization \cite{8353153}, modulation classification \cite{Oshea2017}, localization \cite{8482358}, interference and noise mitigation \cite{9802083}. For more details, we refer to Sec.~\ref{sec:mind_introduction} and Sec.~\ref{sec:autoencoders_related}.

Characterization and modeling of the communication medium are fundamental prerequisites for effective physical layer design. In fact, modulation schemes, coding techniques, and equalization algorithms are derived from a deep understanding of the medium, which can be aided by DL techniques \cite{8644250}. Equally important is the ability to synthesize close-to-real channel samples \cite{8792076} to effectively simulate, test and optimize the communication chain. In this respect, autoencoders have been proposed as fundamental tool for the end-to-end design of the whole chain \cite{Oshea2017}. We refer to Sec.~\ref{sec:gan_ch_synthesis} and Sec.~\ref{sec:autoencoders_related} for an in-depth analysis. 

While several interesting initial results have been proposed, there remains a vast landscape of unexplored possibilities, presenting a fertile ground for future research and innovation in the field.

This dissertation describes a journey towards optimal channel coding and decoding schemes via novel DL formulations. During this journey, we also borrow concepts and tools from probability and statistical learning theory, variational inference, signal processing, and information theory.

Concretely, we present contributions in five main research areas highly interconnected, which naturally cluster the chapters of this thesis into parts.
\begin{enumerate}
    \item \textbf{Density Estimation}. The first part utilizes the concept of copula to analyze data dependence and learn the underlying probability density function (PDF).
In this context, we will introduce:
\begin{itemize}
    \item a segmented methodology to implicitly model any copula and sample from it \cite{Letizia2020};
    \item a variational approach to explicitly model any copula density function and subsequently any PDF \cite{letizia2022copula}.
\end{itemize}
    \item \textbf{Physical Layer Design}. The second part focuses on analyzing the building blocks of a communication system and combining them into optimal autoencoding techniques.  
In this regard, our contributions are:
\begin{itemize}
    \item deep generative models able to synthesize samples from any real-world medium for which a dataset is available \cite{ML_PLC, Letizia2019a};
    \item a neural decoder that uses the mutual information (MI) as decoding criterion and sets apart from the standard classification approach \cite{tonello2022mind};
    \item the first autoencoder that considers the channel capacity as the optimal coding criterion \cite{Letizia2021}.
\end{itemize}
    \item \textbf{Mutual Information and Channel Capacity}. The third part addresses the complex and well-known channel capacity problem from a learning point of view. In particular, we propose:
\begin{itemize}
    \item a novel family of MI neural estimators based on the variational representation of the $f$-divergence \cite{f-DIME};
    \item the first cooperative framework that acts as a neural capacity estimator, referred to as CORTICAL \cite{CORTICAL};
    \item the code for training and validating CORTICAL on a set of non-Shannon representative scenarios \cite{CORTICAL_github}.
\end{itemize}
\item \textbf{Power Line Communication (PLC)}. The fourth part validates the proposed theoretical contributions in the context of PLC. In summary, we design and train:
\begin{itemize}
    \item the first deep generative networks capable of modeling the PLC channel and noise \cite{RighiniLetizia2019, Letizia2019a};
    \item a CORTICAL framework for a channel corrupted by additive Nakagami-$m$ noise under an average power constraint \cite{LetiziaIsplc2021}.
\end{itemize}
\item \textbf{Interpolation}. 
While the core of the thesis explores the theoretical foundations and applications of DL within communications systems, the final part shifts focus and studies the problem of determining valid fitting functions given a set of data points. Interpolation techniques are indeed crucial not only for data analysis and visualization but they also play a direct role in tasks such as channel modeling and signal processing. In this regard, our contribution is:
\begin{itemize}
    \item a novel recursive approach which uses polynomial and rational functions as interpolants \cite{LetiziaRobotics, 9525383}.
    \end{itemize}
\end{enumerate}

\section{Thesis outline}
The objective of this thesis is to reformulate and solve classical problems within physical layer communications, using techniques that arise in the emerging field of DL. 

In Ch.~\ref{sec:fundamentals}, we provide the basic notation and terminology used in this thesis and briefly review the current state-of-the-art statistical learning models. In particular, we distinguish between supervised and unsupervised learning approaches, as they are essential to properly categorize the type of problem in hand. We discuss about different estimation techniques and statistical quantities to measure the distances between probability distributions. We also briefly summarize over DL architectures exploited in this research and introduce generative models concepts, focusing on generative adversarial networks. 

Ch.~\ref{sec:copulas} studies the fundamental properties that dependent random variables share in terms of their joint distribution. For this purpose, the concept of copula is introduced and utilized inside a DL framework. We offer a methodology to generate new data via copula sampling, referred to as segmented generative networks, and we also propose a novel approach to estimate the PDF thanks to a copula-based discriminative formulation, referred to as CODINE.

Motivated by a lack of models and the existence of annotated datasets, In Ch.~\ref{sec:medium}, we present a full DL solution to the long-standing problem of medium modeling. In particular, we utilize generative models to learn the channel and noise distributions and generate new samples with the same statistical characteristics. The envisioned methodology is extremely general as it can be applied to any stationary communication channel. 

With both the generative and discriminative approaches at hand, Ch.~\ref{sec:decoder} reformulates the decoding problem in terms of discriminative learning and proposes a neural architecture for optimal decoding in an unknown channel. The learning process is capable of providing at convergence an estimate of the a-posteriori PDF, and thus of the MI of the input-output channel pair. We utilize the latter as decoding metric. Our top-down approach, referred to as MIND, is also motivated in this case due to either a lack of model for the classical maximum a-posteriori estimation or due to a complex decoding algorithm.

In Ch.~\ref{sec:autoencoders}, we extend the DL formulation to include also the transmitter, hence, we discuss autoencoders for optimal end-to-end communication. We review the main concepts and comment on the limitations of the existing literature, especially highlighting the role of the MI during training. We address the problem of designing capacity-achieving codes with autoencoders, showing evidence of the importance of incorporating a capacity term in the loss function. 
The autoencoder represents a promising framework to study optimal coding/decoding schemes even in the presence of a physical channel model, as for most channels, no optimal strategies are known.

The increasing relevance and presence of the MI in this research suggested us to invest resources in studying its estimation. Ch.~\ref{sec:mi_estimators} firstly summarizes the state-of-the-art on the MI estimation problem, focusing on the latest neural approaches. We then proposed a new family of estimators based on the $f$-divergence, referred to as $f$-DIME, with desired properties and excellent performance. We also developed a novel sampling strategy to correctly train unbiased neural estimators. 

Ch.~\ref{sec:cortical} represents the main contribution of this thesis. It presents a novel DL cooperative framework, referred to as CORTICAL, to estimate the channel capacity and build the capacity-achieving distribution of any discrete-time continuous memoryless vector channel. CORTICAL consists of two cooperative networks: a generator with the objective of learning to sample from the capacity-achieving input distribution, and a discriminator with the objective to learn to distinguish between paired and unpaired channel input-output samples. The latter utilizes $f$-DIME to estimate the MI. 

In Ch.~\ref{sec:plc}, the theoretical contributions of this research are applied to the power line communication medium, whose complex and variable nature offer interesting challenges and benefit from a DL perspective. We train generative models with real measured data, we apply the neural decoding strategy for the impedance modulation problem, we estimate the MI of channel input-output pairs affected by a specific type of noise present in power lines, for which we also study optimal coding schemes. Lastly, an application to anomaly detection is offered.

While the whole thesis can be related to neural distribution interpolation techniques for communication data and signals, we dedicate Ch.~\ref{sec:data interpolation} to provide insights about data and signals interpolation techniques via polynomials. 

In Ch.~\ref{sec:conclusion} we conclude the thesis and we summarize our findings, offering possible future research directions.

\section{Publications}
\label{sec:relatedpub}
We report below the list of publications where the thesis results have been documented. 

\noindent \textbf{Journals}:
\begin{itemize}
	\item A. M. Tonello, N. A. Letizia, D. Righini and F. Marcuzzi, "Machine Learning Tips and Tricks for Power Line Communications," \textit{IEEE Access}, vol. 7, pp. 82434-82452, June 2019.
     \item N. A. Letizia and A. M. Tonello, "Capacity-Driven Autoencoders for Communications," \textit{IEEE Open Journal of the Communications Society}, vol. 2, pp. 1366-1378, June 2021.
    \item N. A. Letizia, B. Salamat and A. M. Tonello, "A Novel Recursive Smooth Trajectory Generation Method for Unmanned Vehicles," \textit{IEEE Transactions on Robotics}, vol. 37, no. 5, pp. 1792-1805, Oct. 2021.
    \item N. A. Letizia and A. M. Tonello, "Segmented Generative Networks: Data Generation in the Uniform Probability Space," \textit{IEEE Transactions on Neural Networks and Learning Systems}, vol. 33, no. 3, pp. 1338-1347, March 2022.
    \item A. M. Tonello and N. A. Letizia, "MIND: Maximum Mutual Information Based Neural Decoder," \textit{IEEE Communications Letters}, vol. 26, no. 12, pp. 2954-2958, Dec. 2022.
    \item N. A. Letizia, A. M. Tonello and H. Vincent Poor, "Cooperative Channel Capacity Learning," \textit{IEEE Communications Letters}, vol. 27, no. 8, pp. 1984-1988, Aug. 2023.
    \item B. Salamat, N. A. Letizia and A. M. Tonello, "Control Based Motion Planning Exploiting Calculus of Variations and Rational Functions: A Formal Approach," \textit{IEEE Access}, vol. 9, pp. 121716-121727, 2021.
    \item N. A. Letizia and A. M. Tonello, "Copula Density Neural Estimation," \textit{ArXiV} 2211.15353.
    \item N. A. Letizia and A. M. Tonello, "Discriminative Mutual Information Estimators for Channel Capacity Learning", \textit{ArXiV} 2107.03084.
\end{itemize}
\textbf{Conferences}:
\begin{itemize}
\item D. Righini, N. A. Letizia and A. M. Tonello, "Synthetic Power Line Communications Channel Generation with Autoencoders and GANs," \textit{in Proceedings of the IEEE International Conference on Communications, Control, and Computing Technologies for Smart Grids (SmartGridComm)}, Beijing, China, 2019, pp. 1-6.
\item N. A. Letizia, B. Salamat and A. M. Tonello, "A New Recursive Framework for Trajectory Generation of UAVs," \textit{in Proceedings of the IEEE Aerospace Conference}, Big Sky, MT, USA, 2020, pp. 1-8.
\item N. A. Letizia, A. M. Tonello and D. Righini, "Learning to Synthesize Noise: The Multiple Conductor Power Line Case," \textit{in Proceedings of the IEEE International Symposium on Power Line Communications and its Applications (ISPLC)}, Malaga, Spain, 2020, pp. 1-6.
\item N. A. Letizia and A. M. Tonello, "Supervised Fault Detection in Energy Grids Measuring Electrical Quantities in the PLC Band," \textit{in Proceedings of the IEEE International Symposium on Power Line Communications and its Applications (ISPLC)}, Malaga, Spain, 2020, pp. 1-5.
\item F. Marcuzzi, A. M. Tonello and N. A. Letizia, "Discovering Routing Anomalies in Large PLC Metering Deployments from Field Data," \textit{in Proceedings of the IEEE International Symposium on Power Line Communications and its Applications (ISPLC)}, Malaga, Spain, 2020, pp. 1-6.
\item A. M. Tonello, N. A. Letizia and M. De Piante, "Learning the Impedance Entanglement for Wireline Data Communication," \textit{in Proceedings of the International Balkan Conference on Communications and Networking (BalkanCom)}, Novi Sad, Serbia, 2021, pp. 96-100. 
\item N. A. Letizia and A. M. Tonello, "Capacity Learning for Communication Systems over Power Lines," \textit{in Proceedings of the IEEE International Symposium on Power Line Communications and its Applications (ISPLC)}, Aachen, Germany, 2021, pp. 55-60.
\item N. A. Letizia and A. M. Tonello, "Discriminative Mutual Information Estimation for the Design of Channel Capacity Driven Autoencoders," \textit{in Proceedings of the International Balkan Conference on Communications and Networking (BalkanCom)}, Sarajevo, Bosnia and Herzegovina, 2022, pp. 41-45.
\item N. A. Letizia, N. Novello and A. M. Tonello, "Mutual Information Estimation via $f$-Divergence and Data Derangements", \textit{in Advances in Neural Information Processing Systems (NeurIPS)}, Vancouver, Canada, 2024, pp. 105114-105150.
\end{itemize}

\thispagestyle{empty}
\pagestyle{main}
\part{Part 1}

\chapter{Deep Learning Fundamentals} %
\chaptermark{Deep Learning Fundamentals}
\label{sec:fundamentals}
Tom Mitchell's ML preface \cite{Mitchell1997} defines ML as:
\quotes{A computer program is said to learn from experience $E$ with respect to some class of tasks $T$ and performance measure $P$, if its performance at tasks in $T$, as measured by $P$, improves with experience $E$}.
The experience is primarily described by the amount and quality of data used for the learning process. 
According to different interpretations
of the experience, it is possible to divide the learning approach into supervised, unsupervised and reinforcement learning. 

In the following, we recall different learning strategies and statistical tools we use throughout the thesis, keeping in mind that, in communication theory, signals are studied as stochastic processes. 

\section{Learning theory concepts}
\sectionmark{Learning theory concepts}
\label{sec:distances}
This section introduces key learning concepts that are essential for understanding subsequent chapters.
\subsection{Supervised learning}
Let $(\mathbf{x}_i,\mathbf{y}_i)\sim p_{XY}(\mathbf{x},\mathbf{y})$, $i=1,\dots, N$, be samples collected into a training set $\mathcal{D}$ belonging to the joint probability density function  (PDF) $p_{XY}(\mathbf{x},\mathbf{y})$. Probabilistic supervised learning predicts $\mathbf{y}$ from $\mathbf{x}$ by estimating $p_{Y|X}(\mathbf{y}|\mathbf{x})$ under a \textit{discriminative model} or by estimating the joint distribution $p_{XY}(\mathbf{x},\mathbf{y})$ under a \textit{generative model}. 

A \textit{regression} problem comprises a continuous output $\mathbf{y}$, meanwhile a discrete target is associated to a \textit{classification} problem.
The standard way to formulate the learning process is to define a cost function $C$, namely a performance measure that evaluates the quality of the prediction $\mathbf{\hat{y}}$. In most applications, we can rely only on the observed dataset $\mathcal{D}$ and derive an empirical sample distribution since we do not have knowledge of the true joint distribution $p_{XY}(\mathbf{x},\mathbf{y})$. In particular, the training objective minimizes
\begin{equation}
\label{eq:general_cost}
C(\mathbf{\hat{y}}) = \mathbb{E}_{(\mathbf{x},\mathbf{y}) \sim \mathcal{D}}[\delta(\mathbf{y},\mathbf{\hat{y}})]
\end{equation}
where $\delta$ is a measure of distance between the desired target $\mathbf{y}$ and the prediction $\mathbf{\hat{y}}$.

\subsection{Unsupervised learning}
Let $\mathbf{x}_i\sim p_X(\mathbf{x})$, $i=1,\dots, N$, be samples collected into a training set $\mathcal{D}$ belonging to the PDF $p_X(\mathbf{x})$.
Unsupervised learning aims at finding useful properties of the structure of a dataset $\mathcal{D}$, ideally inferring the true unknown distribution $p_X(\mathbf{x})$.

Several different tasks are solved using unsupervised learning, for instance: clustering, which divides the data into cluster of similar samples; feature extraction, which transforms data in a different latent space easier to handle and interpret; density estimation and generation/synthesis of new samples. The latter objective consists of learning, from data in $\mathcal{D}$, the distribution $p_X(\mathbf{x})$ and producing new unseen samples from it.

Unsupervised learning tasks require the introduction of an hidden variable $\mathbf{z}_i$ for each sample $\mathbf{x}_i$, leading to the selection of different models under a probabilistic approach. In the \textit{discriminative models}, the latent code $\mathbf{z}_i$ is extracted from $\mathbf{x}_i$ by defining a probabilistic mapping $p_{Z|X}(\mathbf{z|x};\theta)$ parameterized by $\theta$. \textit{Autoencoders} encode $\mathbf{x}_i$ into a latent variable $\mathbf{z}_i$ so that recovering $\mathbf{x}_i$ from $\mathbf{z}_i$ is possible through a decoder. The encoder models the posterior distribution $p_{Z|X}(\mathbf{z|x};\theta)$, while the decoder models the likelihood $p_{X|Z}(\mathbf{x|z};\theta)$.
Lastly, in the \textit{generative models}, an hidden variable $\mathbf{z}_i$ generates the observation $\mathbf{x}_i$. After a specification of a parameterized family $p_Z(\mathbf{z}|\theta)$, the distribution of the observation can be rewritten as $p_X(\mathbf{x|\theta})=\sum_z p_Z(\mathbf{z|\theta})p_{X|Z}(\mathbf{x|z;\theta})$.
Fig. \ref{img:taxonomy} schematically
summarizes the discussed learning models.

\begin{figure}
\centering
	\includegraphics[width=\textwidth]{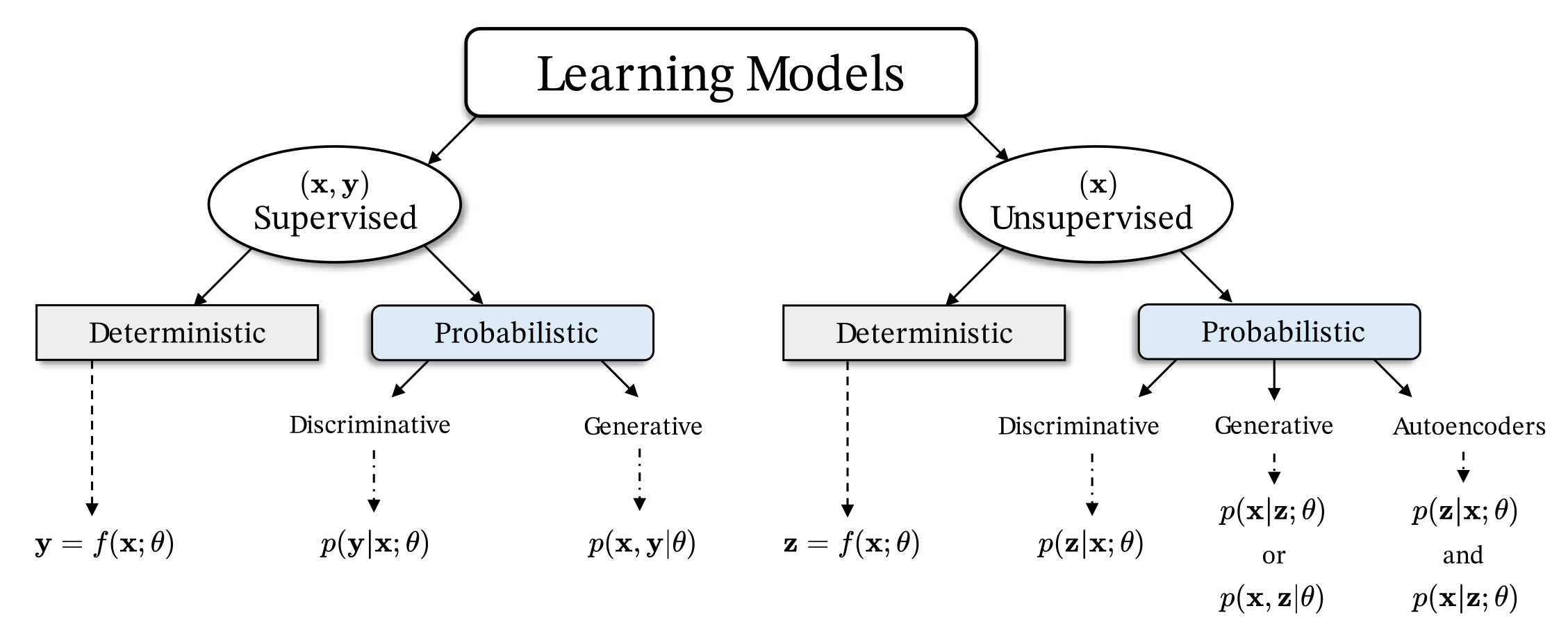}
	\caption{Taxonomy of learning models. Deterministic models extract either fixed relationships between input and output (supervised) or patterns of the input (unsupervised). Probabilistic discriminative models use the input to predict either the output (supervised) or the hidden variable causing the input (unsupervised). Probabilistic generative models learn the statistical relationship between either input and output (unsupervised) or input and the hidden variable (supervised). Autoencoders model how to encode the input into the hidden variable, as well as how to decode from the hidden variable to the input.}
	\label{img:taxonomy}
\end{figure}

\subsection{Reinforcement learning}
Reinforcement learning (RL) addresses the problem of an \textit{agent} learning to act in a dynamic \textit{environment} by finding the best sequence of actions that maximizes a \textit{reward} function. The basic idea is that the agent explores the interactive environment. According to the observation experience it gets, it changes his actions in order to receive higher rewards. 

Basic RL can be modeled as a Markov decision process (MDP). Let $S_t$ be the observation (or state) provided to the agent at time $t$. The agent reacts by selecting an action $A_t$ to obtain from the environment the updated reward $R_{t+1}$, the discount $\gamma_{t+1}$, and the next state $S_{t+1}$.
In particular, the agent-environment interaction is formalized by a tuple $\langle \mathcal{S}, \mathcal{A}, T,r,\gamma\rangle$ where $\mathcal{S}$ is a finite set of states, $\mathcal{A}$ is a finite set of actions, $T(s,a,s')=P[S_{t+1}=s'|S_t = s, A_t = a]$ is the transition probability from state $s$ to state $s'$ under the action $a$, $r(s,a) = \mathbb{E}[R_{t+1}|S_t = s, A_t = a]$ is the reward function, and $\gamma \in [0,1]$ is a discount factor.
To find out which actions are good, the agent builds a \textit{policy}, i.e, a map $\pi : \mathcal{S} \times \mathcal{A} \to [0,1]$ that defines the probability of taking an action $a$ when the state is $s$. If we denote with $G_t = \sum_{k=0}^{\infty}{\biggl(\prod_{i=1}^{k}{\gamma_{t+i}}\biggr) R_{t+k+1}}$ the discount return, then the goal of the agent is to maximize the expected discount return, i.e., \textit{value} $q^{\pi}(s,a) = \mathbb{E}_{\pi}[G_t|S_t = s, A_t = a]$, by finding a good policy $\pi(s,a)$. 

RL algorithms can be categorized as \cite{SuttonRL,ArulkumaranDRL}: a) policy based methods, when the agent, given the observation as input, optimizes the policy $\pi$ without using a value function $q$; b) value based methods, when the agent, given the observation and the action as inputs, learns a value function $q$; c) actor critic methods, where a \textit{critic} measures how good the action taken is (value-based), and an \textit{actor} controls the behaviour of the agent (policy-based).
Despite several applications of reinforcement learning for physical layer communications (see Ch. 9 of \cite{Eldar2022}), the thesis mostly focuses on the first two learning approaches.

\subsection{Maximum likelihood estimation}
Maximum likelihood (MaxL) estimation is a statistical method commonly used in DL for estimating the parameters of a PDF $p_{\text{model}}(\mathbf{x};\theta)$ that best explains the observed data $p_{X}(\mathbf{x})$, assuming the parameters are fixed but unknown.
Most of generative models work with the MaxL principle \cite{Goodfellow2016}; given a probability distribution parameterized by $\theta$, the estimator for $p_X(\mathbf{x})$ is defined as
\begin{equation}
\label{eq:MaxL}
    \theta_{\text{MaxL}} = \argmax_{\theta}p_{\text{model}}(\mathbf{x};\theta), \; \mathbf{x}\sim p_{X}(\mathbf{x}).
\end{equation}
Given $N$ data points $\mathbf{x}_i\sim p_X(\mathbf{x})$, $i=1,\dots, N$, we seek to maximize
\begin{equation}
    \theta_{\text{MaxL}} = \argmax_{\theta}\prod_{i=1}^{N}{p_{\text{model}}(\mathbf{x}_i;\theta)},
\end{equation}
which can be conveniently rewritten as
\begin{equation}
\label{eq:LogL}
    \theta_{\text{MaxL}} = \argmax_{\theta}\sum_{i=1}^{N}{\log p_{\text{model}}(\mathbf{x}_i;\theta)},
\end{equation}
or alternatively
\begin{equation}
\label{eq:CE_MaxL}
    \theta_{\text{MaxL}} = \argmin_{\theta} \mathbb{E}_{\mathbf{x}\sim p_X(\mathbf{x})}\bigl[-\log p_{\text{model}}(\mathbf{x};\theta)\bigr].
\end{equation}
The last expression is equivalent to a cross-entropy minimization over the parameter $\theta$. Notice that \eqref{eq:CE_MaxL} often appears in terms of Kullback-Leibler (KL) divergence
\begin{equation}
\theta_{\text{MaxL}} = \argmin_{\theta} D_{\text{KL}}(p_{X}(\mathbf{x})||p_{\text{model}}(\mathbf{x};\theta)) = \argmin_{\theta} \mathbb{E}_{\mathbf{x}\sim p_X(\mathbf{x})}\biggl[\log \biggl(\frac{p_{X}(\mathbf{x})}{p_{\text{model}}(\mathbf{x};\theta)}\biggr)\biggr].
\label{MaxLKL}
\end{equation}
Practically, we are interested in a parameterization of $p_{\text{model}}$ which is expressive enough to fully capture data patterns and which allows iterative optimization. Artificial neural networks (NNs) represent a viable solution as explain in Sec. \ref{sec:tips-tricks}.

Similarly, MaxL can be applied to estimate the conditional probability $p_{Y|X}(\mathbf{y|x})$
\begin{equation}
\label{eq:CE_CMaxL}
    \theta_{\text{MaxL}} = \argmin_{\theta} \mathbb{E}_{(\mathbf{x},\mathbf{y})\sim p_{Y|X}(\mathbf{y|x})}\bigl[-\log p_{\text{model}}(\mathbf{y|x};\theta)\bigr].
\end{equation}
In communications, the estimation of the transition probability or likelihood $p_{Y|X}$ is extremely relevant as it is used in the maximum a-posteriori (MAP) decoding strategy. Indeed, we are typically interested in estimating the transmitted input $\mathbf{x}$ given the observation of the output $\mathbf{y}$. Formally, we need to solve
\begin{equation}
    \mathbf{\hat{x}} = \argmax_{\mathbf{x}} p_{X|Y}(\mathbf{x|y}),
\end{equation}
which using Bayes' rule and the logarithmic trick reads as
\begin{equation}
    \mathbf{\hat{x}} = \argmin_{\mathbf{x}} -\log p_{Y|X}(\mathbf{y|x}) - \log p_{X}(\mathbf{x}).
\end{equation}
A famous example is described by the scalar additive white Gaussian noise (AWGN) channel of variance $\sigma_N^2$
\begin{equation}
    p_{Y|X}(y|x)= \frac{1}{\sqrt{2\pi}\sigma_N}\exp\biggl[{-\frac{1}{2}\biggl(\frac{y-x}{\sigma_N}\biggr)^2\biggr]}
\end{equation}
and uniform discrete input source with alphabet dimension $M$, for which it is known that the minimal Euclidean distance represents the optimal decoding criterion \cite{Proakis2001}
\begin{equation}
    \hat{x} = \argmin_{x_i \in \{x_1, \dots, x_{M}\}} |y-x_i|^2.
\end{equation}
It is thus prerogative of any decoding algorithm to estimate the channel model. We leave further details to Ch.~\ref{sec:medium} and Ch.~\ref{sec:decoder}.

\subsection{Statistical distances}
Statistical distances quantify the difference or similarity between two probability distributions. 
They are widely used in various fields such as statistics, machine learning, data analysis, and information theory. Different statistical distances capture different aspects of the distributions and are suitable for different tasks. Arguably the most common statistical distance is the KL divergence.

Let $P$ and $Q$ be absolutely continuous measures w.r.t. $\diff x $ and assume they possess
densities $p$ and $q$, then the KL-divergence is defined as
\begin{equation}
D_{\text{KL}}(P||Q) = \int_{\mathcal{X}}{p(x)\log\biggl(\frac{p(x)}{q(x)}\biggr)\diff x},
\end{equation}
where $\mathcal{X}$ is a compact domain.
The KL divergence measures how much the distribution $P$ differs from a reference distribution $Q$ \cite{KL1951}. The divergence equals zero if and only if $P=Q$ as measures, it is not symmetric and it is generally not upper bounded.
A fundamental particular case of KL divergence is represented by the mutual information (MI) between two random variables $X$ and $Y$. It quantifies the statistical dependence between $X$ and $Y$ by measuring the amount of information obtained about one variable via the observation of the other. The MI is symmetric and it is defined as
\begin{equation}
\label{eq:fundamentals_MI}
    I(X;Y) = \mathbb{E}_{(\mathbf{x},\mathbf{y})\sim p_{XY}(\mathbf{x},\mathbf{y})}\biggl[\log \frac{p_{XY}(\mathbf{x},\mathbf{y})}{p_{X}(\mathbf{x})p_{Y}(\mathbf{y})}\biggr],
\end{equation}
which rewrites in terms of KL divergence as
\begin{equation}
    I(X;Y) = D_{\text{KL}}(p_{XY}(\mathbf{x},\mathbf{y})||p_{X}(\mathbf{x})p_{Y}(\mathbf{y})).
\end{equation}

The KL divergence can be extended to a more general class of divergences referred to as $f$-divergence, where $f$ is a convex lower semicontinuous function $f:\mathbb{R}_+ \to \mathbb{R}$ satisfying $f(1)=0$. The $f$-divergence between $P$ and $Q$ is defined as
\begin{equation}
\label{eq:f-divergence}
D_f(P||Q) = \int_{\mathcal{X}}{q(\mathbf{x})f\biggl(\frac{p(\mathbf{x})}{q(\mathbf{x})}\biggr)\diff \mathbf{x}}.
\end{equation}
It is immediate to notice that the KL divergence directly follows from the $f$-divergence using the generator function $f(u)=u\log u$. Other popular statistical measures based on \eqref{eq:f-divergence} are:
\begin{itemize}
    \item the Jensen-Shannon (JS) divergence, defined in terms of KL divergence as
\begin{equation}
    D_{\text{JS}}(P||Q) = \frac{1}{2}D_{\text{KL}}\biggl(P\biggl|\biggl|\frac{P+Q}{2}\biggr)+\frac{1}{2}D_{\text{KL}}\biggl(Q\biggl|\biggl|\frac{P+Q}{2}\biggr);
\end{equation}
\item the squared Hellinger distance (HD), defined as
\begin{equation}
\label{eq:HD-distance}
H^2(P,Q) :=  D_{\text{HD}}(P||Q) = \frac{1}{2} \int_{\mathcal{X}}{\biggl(\sqrt{p(\mathbf{x})}-\sqrt{q(\mathbf{x})}\biggr)^2\diff \mathbf{x}},
\end{equation}
with $0\leq H(P,Q) \leq 1$, and generator $f(u) = \frac{1}{2}(\sqrt{u}-1)^2$;
\item the total variation (TV) distance, which can be rewritten in terms of $f$-divergence as
\begin{equation}
\label{eq:TV-distance}
V(P,Q) =  D_{\text{TV}}(P||Q) = \frac{1}{2} \int_{\mathcal{X}}{|p(\mathbf{x})-q(\mathbf{x})|\diff \mathbf{x}},
\end{equation}
when $f(u)=\frac{1}{2}|u-1|$. 
\end{itemize}
It is also known that $V(P,Q)\leq \sqrt{1-\exp(-D_{\text{KL}}(P||Q))}$ and $H^2(P,Q)\leq V(P,Q) \leq \sqrt{2}H(P,Q)$.
The challenge that unites statistical distances consists in their explicit calculation. One of the objective of Ch.~\ref{sec:mi_estimators} is to shed some light on how to exploit NNs to estimate such distances.

\section{Machine learning tools}
\label{sec:tips-tricks}
In this section, we briefly introduce ML and DL tools utilized to demonstrate the main results of the thesis. In particular, we focus our attention to NNs since they can handle diverse type of data, including numerical data, text, images, and sequences. NNs are designed and trained for specific tasks by adjusting their architecture, activation functions, and other parameters, making them versatile for solving various problems.

\subsection{Neural networks}
\label{subsec:NN}
Neural networks are among the most popular tools in the ML community as they are known being universal function approximators \cite{Hornik1989}, they can be implemented in parallel on concurrent architectures and most importantly, they can be trained by backpropagation \cite{Rumelhart1985}.

A feedforward NN with $L$ layers maps a given input $\mathbf{x}_0 \in \mathbb{R}^{D_0}$ to an output $\mathbf{x}_L \in \mathbb{R}^{D_L}$ by implementing a function $F(\mathbf{x}_0;\mathbf{\theta})$ where $\mathbf{\theta}$ represents the parameters of the NN. To do so, the input is processed through $L$ iterative steps
\begin{equation}
\mathbf{x}_l = f_l(\mathbf{x}_{l-1};\theta_l), \; l=1,\dots, L
\end{equation}
where $f_l(\mathbf{x}_{l-1};\theta_l)$ maps the input of the $l$-th layer to its output. The most used layer is the fully-connected one, whose mapping is expressed as
\begin{equation}
f_l(\mathbf{x}_{l-1};\theta_l) = \sigma(\mathbf{W}_l\cdot \mathbf{x}_{l-1}+\mathbf{b}_l)
\end{equation}
where $\sigma(\cdot)$ is the activation function while $\mathbf{W}_l$ and $\mathbf{b}_l$ are the parameters, weights and the biases, respectively.
According to the specific application, several different types of layers and activation functions can be defined. Fig. \ref{img:fundamentals_NN} shows a general fully-connected architecture.

\begin{figure}
\centering
\includegraphics[scale = 0.33]{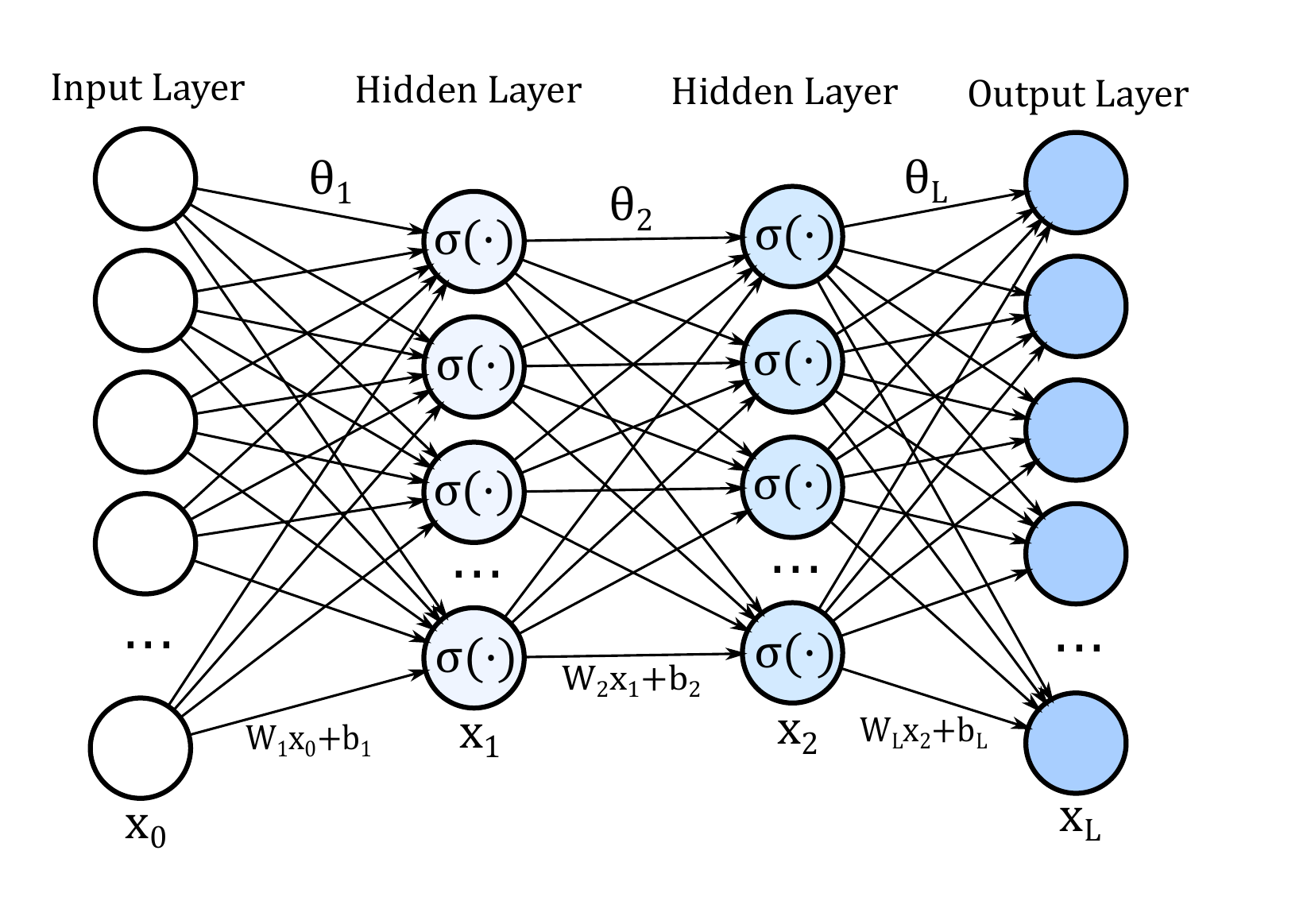}
\caption{Architecture of a fully connected neural network with two hidden layers.}
\label{img:fundamentals_NN}
\end{figure}

Defined a metric $\delta$ and a cost function $C$, the easiest and most classical algorithm to find the feasible set of parameters $\mathbf{\theta}$ is the gradient descent method which iteratively updates $\mathbf{\theta}$ as $\mathbf{\theta}_t = \mathbf{\theta}_{t-1}-\eta\nabla C(\mathbf{\theta}_{t-1})$ where $\eta$ is the learning rate. Its popular variants are stochastic gradient descent (SGD) and adaptive learning rates (Adam) \cite{AdamKingma}.
Common choices for the cost function are the mean squared error and categorical cross-entropy, for which $\delta$ in \eqref{eq:general_cost} takes the form $\delta(\mathbf{y},\mathbf{\hat{y}}) = ||\mathbf{y}-\mathbf{\hat{y}}||^2$ and $\delta(\mathbf{y},\mathbf{\hat{y}}) = -\mathbf{y}\cdot \log{\mathbf{\hat{y}}}$, respectively.

\subsection{Convolutional neural networks}
Convolutional neural networks (CNNs) (Fig. \ref{img:fundamentals_CNN}) are able to capture the spatial and temporal dependencies of data, and for this reason they find application in image and document recognition \cite{LeCun1998}, medical image analysis \cite{Li2014}, natural language processing \cite{Collobert2008}, and more in general pattern recognition. CNNs are multi layer perceptrons (MLP)  with a regularization approach since they consist of multiple convolutional layers to ensure the translation invariance characteristics. In particular, given an input data matrix $\mathbf{I}_i$, the feature map $\mathbf{F}_j$ is obtained as
\begin{equation}
\mathbf{F}_j = \sigma\biggl(\sum_{i=1}^{C}{\mathbf{I}_i \ast \mathbf{K}_{i,j}+\mathbf{B}_{j}}\biggr)
\end{equation}
namely, through the superposition of $C$ layers, e.g., $C=3$ for RGB images, each comprising a convolution between the input matrix $\mathbf{I}_{i}$ and a kernel matrix $\mathbf{K}_{i,j}$, plus an additive bias term $\mathbf{B}_j$, and a final application of non-linear activation function $\sigma(\cdot)$, typically, a \textit{sigmoid}, \textit{tanh}, or \textit{ReLU}. Each set of kernel matrices represents a filter that extracts local features. To control the problem of overfitting, the dimension of data and features to be extracted is reduced by pooling layers. Finally, fully-connected layers are used to extract semantic information from features.

\begin{figure}
\centering
\includegraphics[scale = 0.40]{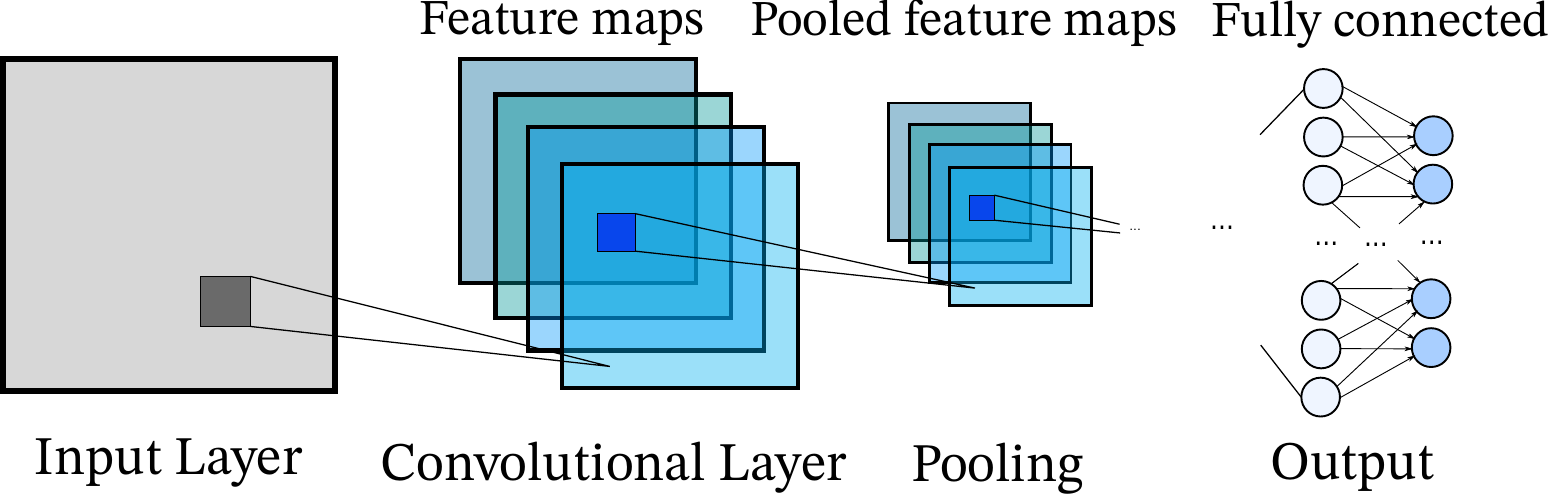}
\caption{Structure of a convolutional neural network with one convolutional layer.}
\label{img:fundamentals_CNN}
\end{figure}

\subsection{Autoencoders}
An autoencoder (AE) is a particular type of NN consisting of an encoding block which tries to learn a latent representation $\mathbf{z}$, typically in a lower-dimensional space, of the input variable $\mathbf{x}$ and a decoding block which reconstructs $\mathbf{x}$ at the output using the information inside the code $\mathbf{z}$ (Fig. \ref{img:fundamentals_autoencoder-architecture}).

A classical learning formulation for \textit{deterministic} AEs requires to solve the following optimization problem
\begin{equation}
\label{eq:AutoCost}
\theta_{\text{opt}} = \argmin_{\theta} \delta(\mathbf{x},G(F(\mathbf{x};\theta_1);\theta_2)),
\end{equation}
where $\theta = (\theta_1,\theta_2)$ are the parameters of the NN, $\delta$ is a measure of distance, while $F$ and $G$ stand for the encoder and decoder function, respectively.
When $F$ and $G$ are linear functions, \eqref{eq:AutoCost} collapse to the principal component analysis (PCA) algorithm. Given a $N \times D$ matrix $\mathbf{x}$ where each row is a sample $\mathbf{x}_i \in \mathbb{R}^D$, PCA represents the encoder as $F(\mathbf{x};\theta) = \mathbf{W}^T\mathbf{x}$ and the decoder as $G(\mathbf{z};\theta) = \mathbf{W}\mathbf{z}$ where $\mathbf{W}$ is the unknown parameter, a $D\times M$ matrix where $M$ is the dimension of the latent space. If $\delta$ is the classic quadratic loss function, then \eqref{eq:AutoCost} can be rewritten as
\begin{equation}
\label{eq:PCA}
\mathbf{W}_{\text{opt}} = \argmin_{\mathbf{W}} \sum_{i=1}^{N}{||\mathbf{x}_i-\mathbf{WW}^T \mathbf{x}_i||_2^2}.
\end{equation}
Let $\Sigma$ be the sample covariance matrix of $\mathbf{x}$, then $\mathbf{W}$ is given by the $M$ principal eigenvectors of $\Sigma$. The resulting transformation $\mathbf{z}=\mathbf{W}^T\mathbf{x}$ is a matrix whose row elements are mutually uncorrelated. PCA is broadly used as a dimensionality reduction and feature extractor tool.

\begin{figure}[t]
\centering
	\includegraphics[scale=0.3]{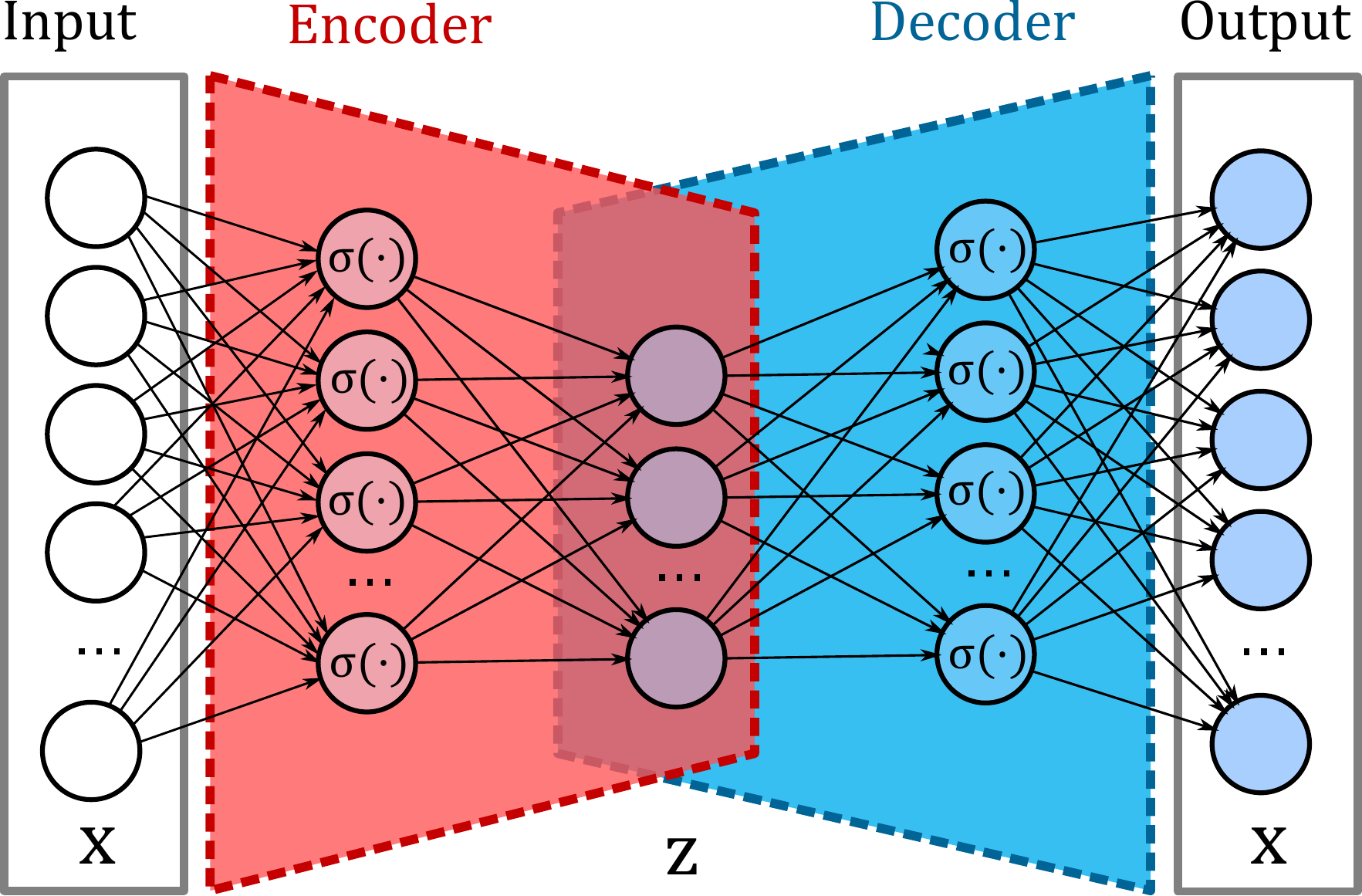}
	\caption{Architecture of an AE.}
	\label{img:fundamentals_autoencoder-architecture}
\end{figure}

Correlation is an indicator of linear dependence, but in most of cases we are interested in representations where features have a different form of dependence and the ability to identify and extract non-linear dependencies is the main reason for adopting AEs. 
An interesting type of AE for communications purposes is the denoising autoencoder (DAE) \cite{Vincent2008}. The idea is to train a network in order to minimize the following denoising criterion:
\begin{equation}
\label{DAE}
\mathcal{L}_{\text{DAE}} = \mathbb{E}_{\mathbf{x}\sim \mathcal{D}}[\delta(\mathbf{x},G(F(\mathbf{\tilde{x}})))],
\end{equation}
where $\mathbf{\tilde{x}}$ is a stochastic corruption of $\mathbf{x}$. When we train a DAE using the quadratic loss and a corruption Gaussian noise $\mathbf{\tilde{x}} = \mathbf{x}+\epsilon$ with $\epsilon \sim \mathcal{N}(0,\sigma I)$, the work in \cite{Alain2014} proved that the AE recovers properties of the training density $p_X(\mathbf{x})$.
This is somehow remarkable in a communication framework because it asserts that corrupting the transmitted signal with some form of noise can be beneficial in the reconstruction's phase. 

They also showed that the DAE with a small corruption of variance $\sigma^2$ is similar to a contractive autoencoder (CAE) with penalty coefficient $\lambda=\sigma^2$. The CAE \cite{Rifai2011} is a particular form of regularized AE which is trained in order to minimize the following reconstruction criterion:
\begin{equation}
\label{eq:CAE}
\mathcal{L}_{\text{CAE}} = \mathbb{E}_{\mathbf{x}\sim \mathcal{D}}\biggl[\delta(\mathbf{x},G(F(\mathbf{x})))+\lambda \biggl| \frac{\partial F(\mathbf{x})}{\partial \mathbf{x}} \biggr|_F^2 \biggr],
\end{equation}
where $|\mathbf{A}|_F$ is the Frobenius norm. The idea behind CAE is that the regularization term attempts to make $F(\cdot)$ or $G(F(\cdot))$ as simple as possible, but at the same time the reconstruction error must be small.

Despite the fact that typically AEs find a low-dimensional representation of the input vector $\mathbf{x}$ at some intermediate level, when redundancy is desired or a design property, sparse AEs learn an hidden representation of the data in a way such data $\mathbf{z}$ is a $k$-sparse vector. Sparsity is mathematically described in terms of pseudo-norm $l_0$ which, due to its non-differentiability and non-convexity, results intractable. For this reason sparsity is often relaxed and described in terms of norm $l_1$. In this way, sparse AEs are trained in order to minimize the following reconstruction criterion:
\begin{equation}
\label{eq:SAE}
\mathcal{L}_{\text{SAE}} = \mathbb{E}_{\mathbf{x}\sim \mathcal{D}}[\delta(\mathbf{x},G(F(\mathbf{x})))+\lambda | F(\mathbf{x})|_1].
\end{equation}
AEs can be used inside a \textit{probabilistic} framework that aims at modeling the underlying distributions. In this context, variational autoencoders (VAEs) have been introduced in \cite{Kingma2013} as generative probabilistic models based on variational inference.

\subsection{Generative models}
\sectionmark{Generative models}
\label{sec:generative_networks}
Variational autoencoders \cite{Kingma2013} are a particular class of generative models based on variational inference. Let $\mathbf{z}$ be the latent variable of the observed value $\mathbf{x}$ for a parameter $\theta$, then $p_{\theta}(\mathbf{z|x})$ represents the intractable true posterior which can be approximated by a tractable one, $q_{\phi}(\mathbf{z|x})$, for a parameter $\phi$. A probabilistic encoder produces $q_{\phi}(\mathbf{z|x})$ while a probabilistic decoder produces $p_{\theta}(\mathbf{x|z})$. The idea is to maximize a variational lower bound $\mathcal{L}$, often referred to as evidence lower bound (ELBO), on the marginal log-likelihood (evidence)
\begin{equation}
\log p_{\theta}(\mathbf{x}_i) = D_{\text{KL}}(q_{\phi}(\mathbf{z|x}_i)||p_{\theta}(\mathbf{z|x}_i))+\mathcal{L}(\theta,\phi;\mathbf{x}_i)
\label{eq:VALower1}
\end{equation}
where
\begin{equation}
\mathcal{L}(\theta,\phi;\mathbf{x}_i) =  -D_{\text{KL}}(q_{\phi}(\mathbf{z|x}_i)||p_{\theta}(\mathbf{z}))  +\mathbb{E}_{q_{\phi}(\mathbf{z|x}_i)}\biggl[\log p_{\theta}(\mathbf{x}_i|\mathbf{z})\biggr].
\label{eq:VALower2}
\end{equation}
Rather than outputting the code $\mathbf{z}$, the encoder outputs parameters describing a distribution for each dimension of the latent space. In the case where the prior is assumed to be Gaussian, $\mathbf{z}$ will consist of mean and variance. Tuning in the latent space and processing the new latent samples through the decoder is a way to generate new data. 
An hierarchical variational autoencoder (HVAE) is an extended version of a VAE, encompassing multiple hierarchies of latent variables. Therein, latent variables are considered to be generated from higher-level, more abstract latent variables themselves, forming a hierarchical structure. In a Markovian HVAE (MHVAE) with $T$ hierarchical latents, the generative process follows a Markov chain, thus it can be interpreted as stacking VAEs on top of each other (see Fig. \ref{img:fundamentals_MHVAE}). The joint and posterior distributions rewrite as
\begin{align}
    \label{eq:MHVAE}
    p_{\theta}(\mathbf{x},\mathbf{z}_{1:T}) & = p_{Z_T}(\mathbf{z}_T)p_{\theta}(\mathbf{x|z}_{1})\prod_{t=2}^{T}{p_{\theta}(\mathbf{z}_{t-1}|\mathbf{z}_{t})} \\ 
    q_{\phi}(\mathbf{z}_{1:T}|\mathbf{x}) & = q_{\phi}(\mathbf{z}_1|\mathbf{x})\prod_{t=2}^{T}{q_{\phi}(\mathbf{z}_{t}|\mathbf{z}_{t-1})}.
\end{align}

\begin{figure}
\centering
	\includegraphics[scale=0.4]{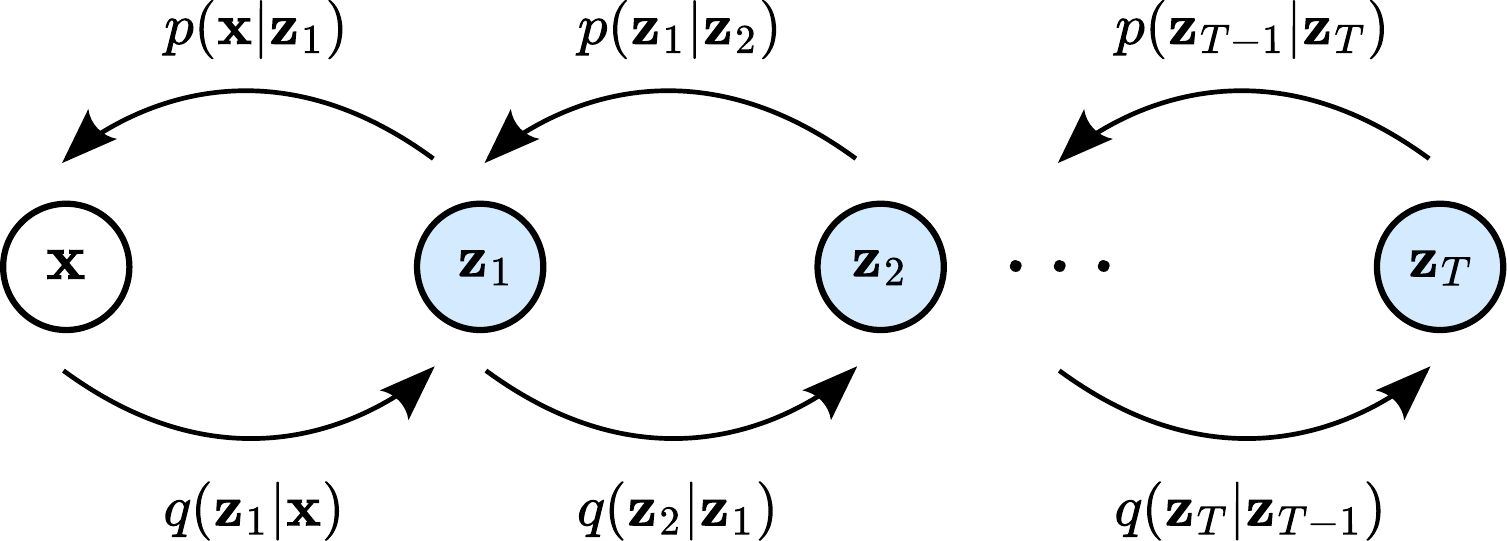}
	\caption{Graph of a Markovian hierarchical variational autoencoder with $T$ hierarchical latents. Each latent $\mathbf{z}_t$ is generated only from the previous latent $\mathbf{z}_{t+1}$.}
	\label{img:fundamentals_MHVAE}
\end{figure}

Probabilistic diffusion models \cite{Ho2020} have been gaining increasing attention and importance in the field of generative modeling. They can be thought as a particular case of a MHVAE when:
\begin{itemize}
    \item the dimension of the latent $\mathbf{z}_t$, $\forall t=1,\dots,T$, equals the dimension of the data $\mathbf{x}$;
    \item the encoder transition is a simple Gaussian model $q(\mathbf{z}_{t}|\mathbf{z}_{t-1}) = \mathcal{N}(\sqrt{\alpha_t}\mathbf{z}_{t-1}, (1-\alpha_t)\mathbf{I})$, where $\alpha_t$ is a learnable coefficient; 
    \item the distribution of the latent $p_{Z_T}(\mathbf{z}_{T})$ at the timestep $T$ is a multivariate normal distribution. 
\end{itemize}
Under such restrictions, it is possible to show that the task of diffusion models consists of learning how to predict the original ground truth sample from an arbitrarily noisy version of it \cite{CalvinLuo2022}.

Flow-based generative models map the data via a non-linear deterministic transformation into a latent space of independent variables where the probability density results tractable \cite{Dinh2014}. The framework lies behind the change of variable rule
\begin{equation}
p_{X}(\mathbf{x}) = p_{Z}(g^{-1}(\mathbf{x}))\cdot \biggl| \det \frac{\partial g^{-1}(\mathbf{x})}{\partial \mathbf{x}} \biggr|
\label{NICE}
\end{equation}
when both the determinant of the Jacobian and $g^{-1}$ are easy to compute; in that case it is straightforward to directly sample from $p_{X}(\mathbf{x})$ since $\mathbf{x} = g(\mathbf{z})$.
NICE, Real NVP and GLOW \cite{Kingma2018} belong to flow-based generative models which are able to provide a good latent-variable inference and excellent log-likelihood evaluation.

The Gaussian Process Latent Variable Model (GP-LVM) \cite{GPLVM2005} aims at inferring both the latent code $\mathbf{z}$ and the mapping function $\mathbf{f}$ that lead to the dataset $\mathbf{x}$. The prior distribution over $\mathbf{z}$ is set as Gaussian, while $\mathbf{f}(\cdot)$ is described as a Gaussian Process (GP) $\mathbf{f} \sim \mathcal{GP}(\mathbf{0},\mathbf{K})$ where $K(\cdot,\cdot)$ is the covariance function, commonly referred to as squared exponential kernel.
This approach ensures a smooth mapping from the latent to the sample space while providing a closed form expression to approximate the true posterior distribution $p_{Z|X}(\mathbf{z}|\mathbf{x})$ and to find a variational lower bound for a robust training procedure \cite{GPLVM2010}.

The methods presented so far have been mostly and successfully applied to images and they all share the ability to generate new samples in parallel. 
The synthesis of fully visible belief networks (FVBNs) \cite{Frey1995,Frey1998} and autoregressive models \cite{Gregor2014} is difficult to parallelize, therefore, due to their sequential nature, they are relatively slow. Indeed, their core idea is to factorize the joint probability distribution of $D$ dimensional inputs $\mathbf{x}$ into products of one-dimensional conditional distributions:
\begin{equation}
 p_{\text{model}}(\mathbf{x}) = \prod_{j=1}^{D}{p_{\text{model}}(x_{j}|x_{1},\dots, x_{j-1}}).
\label{eq:ChainRule}
\end{equation}
Generation is done by generating one dimension at a time leading to good quality of the samples (like WaveNet for human speech \cite{Oord2016}) since they directly optimize the likelihood. 
However, the recent success of the \textit{transformer} architecture \cite{Vaswani2017} shows that parallelization is possible, and together with a self-attention mechanism to effectively model long-range dependencies, transformers have demonstrated remarkable performance in generating human-like text. The pre-training and fine-tuning paradigm used in transformer-based generative models, often referred to as generative pre-trained transformer (GPT), enables them to be highly adaptable to a variety of tasks. 

Boltzmann machines \cite{Hinton06afast} and generative stochastic network \cite{AlainB2015} rely on estimating the transition operator of a Markov Chain $p(x_{j}|x_{j-1})$ but are now less used since Markov chains fail to scale in high dimensional spaces.

One of the most successful data generation techniques is represented by generative adversarial networks (GANs) \cite{Goodfellow2014}.
Considering the importance of GANs to this thesis, we have dedicated a separate section to delve into their functioning and principles.

\section{Generative adversarial networks}
\label{sec:gans}
The main idea behind GANs is to train a pair of networks in competition with each other: a generator network $G$ that captures the observed data distribution $p_{X}(\mathbf{x})$ and a discriminator network $D$ that distinguishes if a sample is an original coming from real data rather than a fake coming from data generated by the generator $G$ (see Fig. \ref{img:fundamentals_GANs}). 
The training procedure for $G$ is to maximize the probability of $D$ making a mistake. GANs can be thought as a minimax two-player game which will end when a Nash equilibrium point is reached. 

In detail, the generator takes as input a noise vector $\mathbf{z}$ with distribution $p_{\text{noise}}(\mathbf{z})$ and maps it to the data space via a non-linear transformation $\hat{\mathbf{x}} = G(\mathbf{z})$. 
The vanilla value function $V(G,D)$ used to train GANs reads as follows
\begin{equation}
V(G,D) = \mathbb{E}_{\mathbf{x} \sim p_{X}(\mathbf{x})}[\log D(\mathbf{x})] + \mathbb{E}_{\mathbf{z} \sim p_{\text{noise}}(\mathbf{z})}[\log(1-D(G(\mathbf{z})))],
\label{eq:fundamentals_GANs}
\end{equation}
where the objective for the generator is to minimize $V$ while the discriminator maximizes $V$.
It can be proved that such optimization strategy, i.e.,
\begin{equation}
\label{eq:min_max_GAN}
    G^* = \argmin_G \max_D V(G,D),
\end{equation}
enables the generator to implicitly learn the true data distribution $p_{X}(\mathbf{x})$. Indeed, if $p_{\hat{X}}(\mathbf{x})$ is the distribution of the output $\hat{\mathbf{x}}$, it is true that $p_{\hat{X}} = p_{X}$ at the equilibrium.

\begin{figure}
\centering
\includegraphics[scale = 0.4]{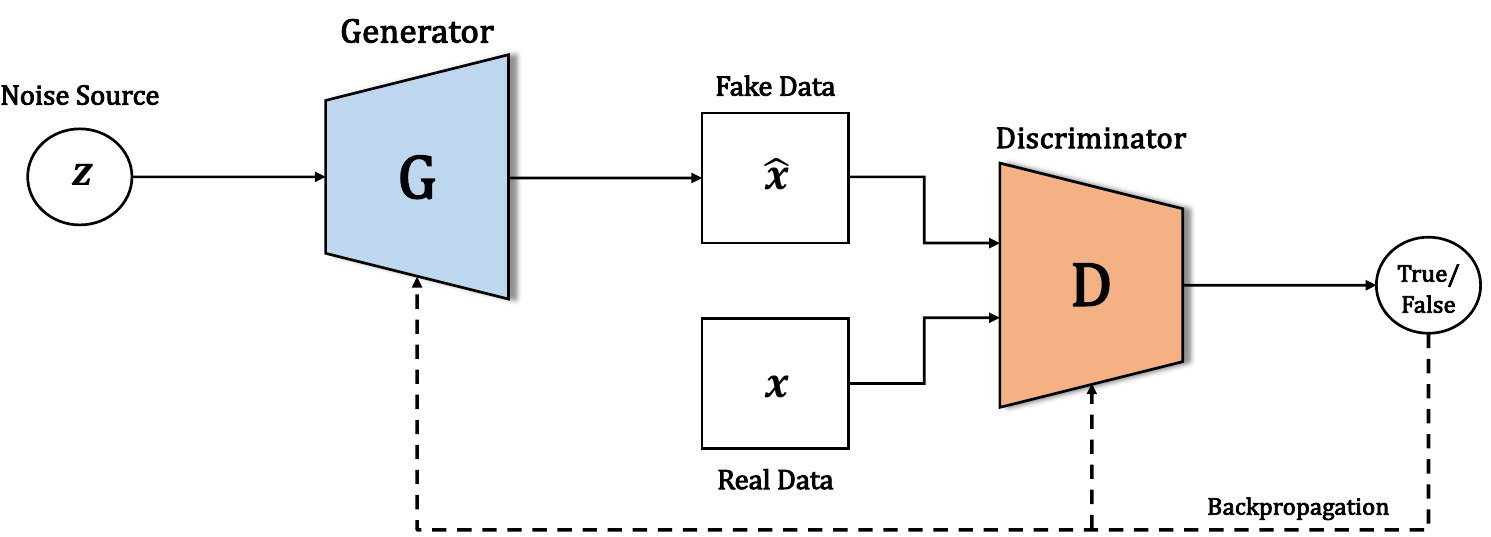}
\caption{GAN framework in which generator and discriminator are learned during the training process.}
\label{img:fundamentals_GANs}
\end{figure}

Several architectures such as the deep convolutional generative adversarial network (DCGAN) \cite{Radford2016} and the StyleGAN \cite{Karras2019} have been proposed in the literature to stabilize the learning process as it is well-known that GANs suffer from problems such as mode collapse and training instability. The latter architecture is amongst the most recent evolution of GANs as it produces state-of-the-art results in synthesizing high-resolution images. The authors achieve improved performance by modifying the generator's architecture (a mapping network for the latent representation and a synthesis one with different resolution levels) for a better understanding of the output. 
However, it is also true that minor efforts have been carried out in generating structured non-images data with GANs.  In Sec.~\ref{sec:medium}, we attempted to study the GAN-based synthesis of time-domain signals.

An alternative method for enhancing GANs involves gaining a deeper comprehension of their training process. Indeed, GANs represent a unique paradigm in ML, where the traditional notion of a cost function is replaced by a value function and a game-theoretic framework. This distinctive characteristic sets GANs apart from many other optimization problems in the field. An extremely useful interpretation of GANs is offered in \cite{Nowozin2016}, where it was shown that the discriminator in GANs acts as an auxiliary NN responsible for the estimation of the variational $f$-divergence, essentially the $\max$ operator in \eqref{eq:min_max_GAN}. The generator, instead, is a neural sampler trained to minimize such divergence, thus, any $f$-divergence can be used for training generative neural samplers, allowing for flexible divergence choices to better match the desired distribution. 

The upcoming chapters will revisit and build upon the latest concept along with the explanations provided thus far.
\chapter{Copula for Statistical Analysis and Data Generation} %
\chaptermark{Copula}
\label{sec:copulas}

To comprehend how generative models operate and why they are so successful in estimating and reproducing data relationships, it is crucial to grasp the fundamental concept of data dependence. Understanding how variables in a dataset are interrelated is paramount to appreciating the power of generative models. 

In this context, the concept of copula emerges as a key foundation. Copulas provide a mathematical framework to characterize the nuanced dependencies between variables, enabling us to model and generate complex data distributions with precision and insight. 

This chapter is divided in two parts. The first one formulates  the data generation problem via copulas and describes how to implicitly sample from them. We achieve our objective in two well-defined separate steps.
In the second part, we finally exploit the envisioned procedure to explicitly estimate the copula density function, and consequently, the PDF of the collected data.

The results presented in this chapter are documented in \cite{Letizia2020,letizia2022copula}.

\subsection*{\textbf{Definitions and notation}}
\label{subsec:notation}
$\mathbf{X}$ denotes a multivariate random variable of dimension $d$ whose components are $X_{i}$ with $i=1,\dots, d$, while $\mathbf{x}^{(j)}$ for $j=1,\dots, n$ denotes the $j$-th realization of $\mathbf{X}$ among $n$ independent observations. $\mathbf{x}_i$ denotes a column vector of $n$ realizations of $X_{i}$. Furthermore, $x_{i}^{(j)}$ denotes the $i$-th entry (out of $d$) of the $j$-th collected sample (out of $n$ observations). In a compact notation, $\mathbf{x}=[\mathbf{x}_1,\mathbf{x}_2,\dots,\mathbf{x}_d]$ denotes a $n\times d$ matrix, sometimes referred to as the training data. 
$\Sigma_x$ and $p_{\mathbf{X}}(\mathbf{x})$ denote the sample covariance matrix and probability density function of $\mathbf{X}$, respectively. $F_{\mathbf{X}}(\mathbf{x}) = P(X_1\leq x_1, \dots, X_d \leq x_d)$ and $F^{-1}_{\mathbf{X}}(\mathbf{x})$ denote the cumulative distribution function and quantile function, respectively. 
The expected value of $\mathbf{X}$ is denoted with the expectation operator $\mathbb{E}_{\mathbf{x}\sim p_{\mathbf{X}}(\mathbf{x})}[\mathbf{X}]$. 

\section{Segmented generative networks}
\sectionmark{SGN}
\label{sec:sgn}
Recent advancements in generative networks (see Ch. \ref{sec:generative_networks}) have shown that it is possible to produce real, world-like data using deep neural networks. 

Some implicit probabilistic models that follow a stochastic procedure to directly generate data, such as GANs, have been introduced to overcome the intractability of the posterior distribution. However, the ability to model data requires a deep knowledge and understanding of its statistical dependence — which can be preserved and studied in appropriate latent spaces. 

In this section, we present a segmented generation process through linear and non-linear manipulations in a same-dimension latent space where data is projected to. Inspired by the known stochastic method to generate correlated data, we develop a segmented approach for the generation of dependent data, exploiting the concept of copula. The generation process is split into two frames, one embedding the covariance or copula information in the uniform probability space, and
the other embedding the marginal distribution information in the
sample domain. 
The proposed network structure, referred to as a
segmented generative network (SGN), also provides an empirical method to sample directly from implicit copulas.
To show its generality, we evaluate the presented approach in three application scenarios: a toy example, handwritten digits and face image generation.

\subsection{Proposed approach}
\label{subsec:sgn_proposal_approach}
The statistical dependence between the components of a multivariate random variable $\mathbf{X}$ is described by the joint probability distribution $p_{\mathbf{X}}(x_1, x_2, \dots, x_d)$. The correlation, instead, measures how the components are related on average and it is expressed in terms of the expectation $\mathbb{E}[\mathbf{X}\cdot \mathbf{X}^T]$ or the covariance $\mathbb{E}[(\mathbf{X}-m_X)\cdot (\mathbf{X}-m_X)^T]$, where $m_X$ denotes the expected value of $\mathbf{X}$. The correlation is often associated with the idea of \textit{linear} dependence.

Let $\mathbf{x} = [\mathbf{x}_1, \mathbf{x}_2, \dots, \mathbf{x}_d]$ be a set of realizations of $\mathbf{X}$, also referred to as the collected multivariate data. We would like to generate new unseen samples $\mathbf{\hat{x}} = [\mathbf{\hat{x}}_1, \mathbf{\hat{x}}_2, \dots, \mathbf{\hat{x}}_d]$ similar, in some way as we will discuss, to $\mathbf{x}$. In the following, we propose two different approaches: 
\begin{itemize}
\item The first one revisits the known stochastic generation process of correlated data (data that exhibits the same correlation as the collected one) and highlights the need for a segmentation and domain adaptation step. Such method will be referred to as segmented generative network targeting and modeling the correlation of data (SGN-C). 
\item The second one, instead, studies the stochastic generation process of dependent data (data that exhibits the same joint distribution as the collected one) by applying the same domain adaptation of SGN-C but integrating the concept of copula in order to segment the generation process. Moreover, it provides a direct method to sample from copulas. Such approach will be referred to as segmented generative network targeting and modeling the statistical dependence of data (SGN-D). 
\end{itemize}

\subsubsection{Transform sampling}
\label{subsec:sgn_transform_sampling}
The training data has been generated by some fixed unknown or difficult to construct probability distribution $p_{\mathbf{X}}(\mathbf{x}) = p_{\mathbf{X}}(x_1, x_2, \dots, x_d)$ with cumulative distribution function $F_{\mathbf{X}}(\mathbf{x}) = P(X_1\leq x_1, \dots, X_d \leq x_d)$. However, it is plausible to assume that the marginal density of each $X_i$ is known or can be easily derived as $p_{X_i}(x_i)$ with cumulative $F_{X_i}(x_i)$. Hence, the data can be mapped into a latent space with the same dimension using the inverse transform sampling method. In particular, let $U_i$ be a uniform random variable, then
\begin{equation}
X_i = F^{-1}_{X_i}(U_i)
\label{eq:sgn_InverseTransform}
\end{equation}
is a random variable with cumulative distribution $F_{X_i}$.
To project the data $\mathbf{x}$ into the latent space, it is enough to compute the transformation $u_i = F_{X_i}(x_i) \text{  }\forall i=1,\dots,d$.
This first step can be interpreted as a simple encoder which tries to represent the data in a domain space where linear manipulations are easier to be implemented.
One way to go back is to train a neural network which, given $\mathbf{u}$ as input and $\mathbf{x}$ as output, finds the inverse mapping between the two spaces, the latent and the sample ones. This inverse transformation back to the sample space can be easily interpreted as a decoder. Fig. \ref{fig:sgn_OurFramework} describes the SGN framework.

The autoencoder described so far has no generative properties since it only replicates the dataset $\mathbf{x}$. The way to generate new samples is to build a new encoded set $\mathbf{\hat{u}}$ and feed it into the already trained inverse network. 

In general, $X_i$ and $X_j$ with $i\neq j \leq d$ are dependent random variables (let us consider for example the intensity of two consecutive pixels in an image or the amplitude of a waveform like the sound). A first order of approximation is the linear dependence: the idea is to choose the new encoded set $\mathbf{\hat{u}}$ as a set of $d$ correlated uniform variables — correlation quantified by the sample covariance matrix $\Sigma_u$ of the encoded initial set $\mathbf{u}$.

\subsubsection{Correlated uniforms (SGN-C)}
\label{subsec:sgn_correlated_uniforms}

Let $\Sigma_x$ and $\Sigma_u$ be the sample covariance matrices of the dataset $\mathbf{x}$ and the latent uniform code $\mathbf{u}$, respectively. If $\mathbf{\hat{x}}$ has covariance matrix $\Sigma_{\hat{x}}$ equal to $\Sigma_x$, then we define $\mathbf{\hat{x}}$ \textit{similar} to $\mathbf{x}$. 
Denoting with $\mathbf{F^{-1}}$ the non-linear inverse cumulative distribution function, then if $\mathbf{\hat{u}}$ is similar to $\mathbf{u}$, it follows that $\mathbf{\hat{x}} = \mathbf{F^{-1}(\hat{u})}$ is similar to $\mathbf{x}$. 

To generate $d$ correlated uniform distributed random variables, we use the NORTA (Normal to anything) method \cite{NORTA}, in the following denoted as covariance method. 
The first step requires the generation of $d$ correlated Gaussian distributed random variables. In particular, given a mean vector $\mu$ and a covariance matrix $\Sigma$, to generate a sample $Y\sim \mathcal{N}(\mu,\Sigma)$ from the multivariate normal distribution, we need to consider first a vector $\mathbf{z}$ of uncorrelated Gaussian random variables, then find a matrix $\mathbf{C}$, square root of $\Sigma$ such as $\mathbf{C}\cdot \mathbf{C}^T = \Sigma$ (for example using the Cholesky decomposition). It follows that $\mathbf{y}=\mu + \mathbf{C}\cdot \mathbf{z}$ is a vector of $d$ Gaussian random variables $\{Y_i\}_{i=1}^d$ with the desired properties. Applying the probability integral transform to each entry, leads to $d$ correlated uniform random variables 
\begin{equation}
\hat{U}_i = \Phi(Y_i)
\label{eq:sgn_ProbabilityIntegralTransform}
\end{equation}
where $\Phi$ is the cumulative distribution function of the standard normal distribution.
If we wanted to impose $\Sigma_u$ as the covariance matrix of $\mathbf{\hat{u}}$, we could transform the latent code $\mathbf{u}$ into a new latent code $\mathbf{n}$ in the Gaussian space, compute the covariance matrix $\Sigma_n$ and use it to sample from the multivariate normal distribution $\mathcal{N}(\mu,\Sigma_n)$. Instead, if we only wanted to impose the correlation matrix $\mathbf{R_u}$, we would rather just sample from the multivariate normal distribution $\mathcal{N}(\mu,\mathbf{R_u})$ to get a good practical approximation as described in \cite{Falk1999}.

The generation of the encoded set can be easily implemented in parallel, which means that we do not need to wait for any information from previous samples. Once the encoded set $\mathbf{\hat{u}}$ is built, the transformation $\mathbf{F^{-1}}(\mathbf{\hat{u}})$ gives $\mathbf{\hat{x}}$ as a new generated sample, exploiting only the linear information inside the dataset. Notice that exploiting an artificial neural network for the inverse transform is not mandatory: one could use the inverse cumulative distribution function (quantile function, see \eqref{eq:sgn_InverseTransform}) implemented by step functions or kernel smoothing functions.
Sec. \ref{subsec:sgn_evaluation_of_results} will show some visual results.

\begin{figure}
\centering
\includegraphics[scale = 0.4]{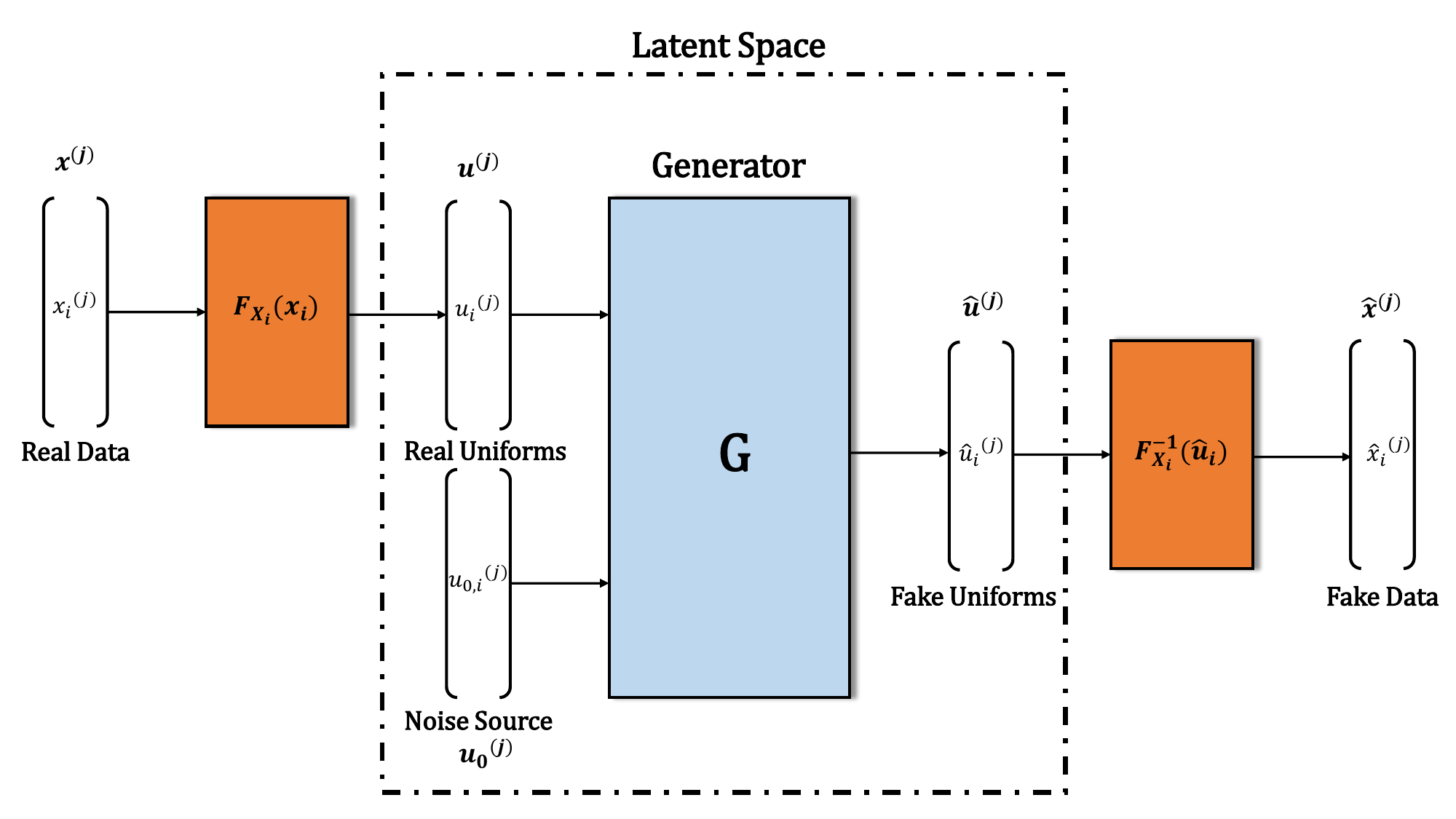}
\caption{Transform sampling and SGN approach where original data are projected into a latent space, the new codes are generated and back-projected again.}
\label{fig:sgn_OurFramework}
\end{figure}

\subsubsection{Dependent uniforms (SGN-D)}
\label{subsec:sgn_dependent_uniforms}

Working only with linear dependence (correlation) is not sufficient to fully reproduce the relationship between data. Thus, it is necessary to build a new encoded set $\mathbf{\hat{u}}$ of $d$ dependent uniform variables.  

When we discussed the procedure to build correlated uniforms, we were looking for a vector $\mathbf{\hat{u}}$ whose covariance matrix $\Sigma_{\hat{u}}$ was equal to the prescribed one $\Sigma_u$ and we built it by using the SGN-C algorithm described in Sec. \ref{subsec:sgn_correlated_uniforms}. Another approach could have been the following: given a family of functions $S_{\theta}$ which takes uncorrelated uniform variables $\mathbf{u_0}$ as input and transforms them into correlated ones $\mathbf{\hat{u}}$, one could solve the optimization problem
\begin{equation}
\theta_{\min} = \argmin_{\theta} \delta(\Sigma_u, \Sigma_{S_{\theta}(\mathbf{u_0})}),
\label{eq:sgn_DeltaSigma}
\end{equation}
where $\delta$ is a measure of distance between the sample covariance matrices of $\mathbf{u}$ and $\mathbf{\hat{u}}$, respectively.

In the same way, given the latent codes $\mathbf{u}$ with probability density function $p_{\mathbf{u}}(\mathbf{u})$, the main idea to generate dependent uniform random variables is to train a neural network $G_{\theta}$ which takes independent uniform variables $\mathbf{u_0}$ as input and maps them into new dependent ones, $\mathbf{\hat{u}}$, with distribution $q_{\mathbf{\hat{u}}}(\mathbf{\hat{u}})$. This is again a generative model whose objective function is
\begin{equation}
\theta_{\min} = \argmin_{\theta} \delta(p_{\mathbf{u}}(\mathbf{u}), q_{\mathbf{\hat{u}}}(G_{\theta}(\mathbf{u_0}))
\label{eq:sgn_DeltaP}
\end{equation}
where now $\delta$ is a measure of \textit{discrepancy} between the real distribution $p_{\mathbf{u}}$ and the generated one $q_{\mathbf{\hat{u}}}$.

Before introducing the approach chosen for solving problem \eqref{eq:sgn_DeltaP}, we focus on the particular properties that $p_{\mathbf{u}}$ and $q_{\mathbf{\hat{u}}}$ have, recalling the concept of \textit{copula}.

Let $(U_1,U_2,\dots,U_d)$ be uniform random variables, then their joint cumulative distribution function $F_{\mathbf{U}}(\mathbf{u}) = P(U_1\leq u_1, \dots, U_d \leq u_d)$ is a copula $C:[0,1]^d \rightarrow [0,1]$ (see \cite{Nelsen2006} for an analytic description). Copulas are a useful tool to construct multivariate distributions and analyze data dependence. Indeed, Sklar's theorem \cite{Sklar} states that if $F_{\mathbf{X}}$ is a $d$-dimensional cumulative distribution function with continuous marginals $F_{X_1},\dots,F_{X_d}$, then $F_{\mathbf{X}}$ has a unique copula representation
\begin{equation}
F_{\mathbf{X}}(x_1,\dots,x_d) = C(F_{X_1}(x_1),\dots,F_{X_d}(x_d)).
\label{eq:sgn_SklarF}
\end{equation}
Moreover, when the multivariate distribution has a probability density function $f_{\mathbf{X}}$, it holds that
\begin{equation}
f_{\mathbf{X}}(x_1,\dots,x_d) = c(F_{X_1}(x_1),\dots,F_{X_d}(x_d))\cdot \prod_{i=1}^{d}{f_{X_i}(x_i)},
\label{eq:sgn_Sklarf}
\end{equation}
where $c$ is the density of the copula. This last relationship is rather interesting because it affirms that the dependence internal structure of $f_{\mathbf{X}}$ can be recovered using the density of the marginals $f_{X_i}$ and the density $c$ of the copula. Under this perspective, problem \eqref{eq:sgn_DeltaP} can be reformulated as
\begin{equation}
\theta_{\min} = \argmin_{\theta} \delta(c_{\mathbf{u}}(\mathbf{u}), c_{\mathbf{\hat{u}}}(G_{\theta}(\mathbf{u_0})),
\label{eq:sgn_DeltaC}
\end{equation}
where $c_{\mathbf{u}}$ and $c_{\mathbf{\hat{u}}}$ are the densities of the copulas related to $\mathbf{u}$ and $\mathbf{\hat{u}}$, respectively. The objective is to reach the equality $c_{\mathbf{u}}=c_{\mathbf{\hat{u}}}$ and to sample a new encoded set $\mathbf{\hat{u}}$ from $c_{\mathbf{u}}$, preserving the entire hidden dependence in $\mathbf{u}$ by construction.
Due to the fact that we are working with a finite set of samples, we should build the empirical copula of the encoded given set $\mathbf{u}$, with expression
\begin{equation}
C_n(\mathbf{u}) = \frac{1}{n}\sum_{j=1}^{n}{\mathbbm{1}_{\{U_{1}^{(j)} < u_1, \dots, U_{d}^{(j)} < u_d\}}}
\end{equation}
where $n$ is the number of observations, $U_{i}^{(j)}$ denotes the $j$-th realization of the $i$-th random variable, with $i=1,\dots,d$, and $\mathbbm{1}_A$ is the indicator function. When its dimensionality increases, this is not feasible anymore and one way to proceed is to choose a parametric family of multivariate copulas, like the multivariate Gaussian copula with correlation matrix $\Sigma$ (equivalent to NORTA \cite{Bedford2016}) or the multivariate Student's t-copula with $\nu$ degrees of freedom and correlation matrix $\Sigma$, which is more suitable for data containing phenomena of extreme value dependence \cite{t-copula}. 
Archimedean copulas are a particular class of copulas that admit a closed formula and allow modeling dependence varying one parameter, they have the following representation
\begin{equation}
C(u_1,\dots, u_d; \theta) = \psi^{[-1]}(\psi(u_1,\theta)+\dots+\psi(u_d,\theta); \theta)
\end{equation}
where $\psi$ is a continuous, strictly decreasing and convex generator function with pseudo-inverse $\psi^{[-1]}$ \cite{Nelsen2006} and $\theta$ is a parameter. 
Sec. \ref{subsec:sgn_evaluation_of_results} compares some results using different copulas in specific applications.
An interesting way to overcome the lack of parametric multivariate copulas is to take advantage of the huge number of parametric families of bivariate copulas through the concept of vine copulas \cite{VineCopula, Bedford2016}. The idea is to model the copula density as the product of pairs of conditional copula bivariate densities, under a tree or vine decomposition which leads to tractable and flexible probabilistic models. A recent attempt to generate data using vine copulas \cite{VCAE} exploited autoencoders and their ability to find lower dimensional representation.

We have argued about the parameters of the function $\delta$ that we are trying to minimize, but we never mentioned so far the type of distance/discrepancy that $\delta$ has to mime. Since we are elevating our approach to a general one which does not focalize on low dimension of data or specific type of distributions, we propose two different approaches. The first one is the maximum mean discrepancy (MMD) metric.

Let $\mathbf{x}=\{\mathbf{x}^{(1)},\dots,\mathbf{x}^{(l)}\}$ and $\mathbf{y}=\{\mathbf{y}^{(1)},\dots,\mathbf{y}^{(m)}\}$ be observations taken independently from $p = c_{\mathbf{u}}$ and $q = c_{\mathbf{\hat{u}}}$, respectively. Let $(\chi,d)$ be a nonempty compact metric space in which $p$ and $q$ are defined. Then, the maximum mean discrepancy is defined as
\begin{equation}
\text{MMD}(\mathcal{F},p,q) := \sup_{f\in \mathcal{F}}(\mathbb{E}_{\mathbf{x}\sim p}[f(\mathbf{x})]-\mathbb{E}_{\mathbf{y}\sim q}[f(\mathbf{y})]),
\label{eq:sgn_MMD}
\end{equation}
where $\mathcal{F}$ is a class of functions $f:\chi \rightarrow \mathbb{R}$. 
Since $p=q$ if and only if $\mathbb{E}_{\mathbf{x}\sim p}[f(\mathbf{x})]=\mathbb{E}_{\mathbf{y}\sim q}[f(\mathbf{y})]$ $\forall f\in \mathcal{F}$, MMD is a metric that measures the disparity between $p$ and $q$ (see \cite{FortetMMD}).

When $\mathcal{F}$ is a \textit{reproducing kernel Hilbert space} (RKHS), $f$ can be replaced by a kernel $k\in \mathcal{H}$ (i.e. Gaussian or Laplace kernels). In this case, Gretton et al. \cite{Gretton2012} showed that 
\begin{equation}
\text{MMD}^2(\mathcal{H},p,q) = \mathbb{E}_{\mathbf{x,x'}\sim p}[k(\mathbf{x,x'})]-2\mathbb{E}_{\mathbf{x}\sim p,\mathbf{y}\sim q}[k(\mathbf{x,y})]  +\mathbb{E}_{\mathbf{y,y'}\sim q}[k(\mathbf{y,y'})],
\label{eq:sgn_MMDK}
\end{equation}
where $\mathbf{x'}$ is an independent copy of $\mathbf{x}$ with the same distribution, and $\mathbf{y'}$ is an independent copy of $\mathbf{y}$.
For practical implementation, an unbiased empirical estimate is given by
\begin{align}
\text{MMD}_u^2(\mathcal{H},\mathbf{x},\mathbf{y}) = & \frac{1}{l(l-1)}\sum_{i\neq i'}{k(\mathbf{x}^{(i)},\mathbf{x}^{(i')})} \nonumber \\ 
& + \frac{1}{m(m-1)}\sum_{j\neq j'}{k(\mathbf{y}^{(j)},\mathbf{y}^{(j')})} \nonumber \\ 
& - \frac{2}{lm}\sum_{i=1}^{l}{\sum_{j=1}^{m}{k(\mathbf{x}^{(i)},\mathbf{y}^{(j)})}}.
\label{eq:sgn_MMDKe}
\end{align}

Finally, we can define the type of discrepancy $\delta$ as the $\text{MMD}_u^2$ estimator, in particular
\begin{equation}
\delta(c_{\mathbf{u}}(\mathbf{u}), c_{\mathbf{\hat{u}}}(G_{\theta}(\mathbf{u_0})) = \text{MMD}_u^2(\mathcal{H},\mathbf{u},G_{\theta}(\mathbf{u_0}))
\label{eq:sgn_DeltaMMD}
\end{equation}
and proceed with its minimization by the exploitation of the chain rule and the gradient descent method as described in \cite{MMDnets}. We found this methodology not effective for embedding the copula dependence structure. Therefore, we decided to use it during the evaluation phase (see. Sec. \ref{subsec:sgn_evaluation_of_results}) and, instead, adopt a GAN framework to reproduce the copula dependence.

The core idea is to identify the copula with a first generator $G_u$. The generator $G_u$ receives an equal-dimensional independent uniform noise source $\mathbf{z}$ as input and internally creates the copula dependence structure by opposing a discriminator $D_u$ which tries to distinguish between real and fake dependent uniform samples, $\mathbf{u}$ and $\mathbf{\hat{u}}$, respectively. 
The corresponding value function reads as follows
\begin{equation}
V(G_u,D_u) =  \mathbb{E}_{\mathbf{u} \sim c_{\mathbf{u}}(\mathbf{u})}[\log D_u(\mathbf{u})]  + \mathbb{E}_{\mathbf{z} \sim \mathcal{U}(0,1)}[\log(1-D_u(G_u(\mathbf{z})))].
\label{eq:sgn_V1}
\end{equation}

At the same time, in order to strengthen the copula generation process,  another GAN $(G_x, D_x)$ mimes the relationship (the quantile function $F_{X_i}^{-1}$ for $i=1,\dots,d$) between the latent code $\mathbf{u}$ and the sample space $\mathbf{x}$. To do so, it takes the generated uniforms $\mathbf{\hat{u}}$ and statistically transforms them into new samples $\mathbf{\hat{x}}$, checking again the statistical significance with another discriminator $D_x$. The second value function is defined as
\begin{equation}
V(G_x,D_x) = \mathbb{E}_{\mathbf{x} \sim p_{\mathbf{x}}(\mathbf{x})}[\log D_x(\mathbf{x})]
 + \mathbb{E}_{\mathbf{\hat{u}} \sim c_{\mathbf{\hat{u}}}}[\log(1-D_x(G_x(\mathbf{\hat{u}})))].
\label{eq:sgn_V2}
\end{equation}
We denote the full generation process as segmented generative networks modeling the dependence (SGN-D). Succinctly, the first generator $G_u$ embeds the copula density structure $c_{\mathbf{u}}(\mathbf{u})$ and produces dependent uniform samples $\mathbf{\hat{u}}$. The second generator $G_x$ embeds the marginals structure (the quantile $F_{X_i}^{-1}$ for $i=1,\dots,d$) and statistically maps $\mathbf{\hat{u}}$ into new samples $\mathbf{\hat{x}}$.
Fig. \ref{fig:sgn_UXGAN} graphically summarizes the entire followed methodology. 

\begin{figure}
\centering
\includegraphics[scale = 0.5]{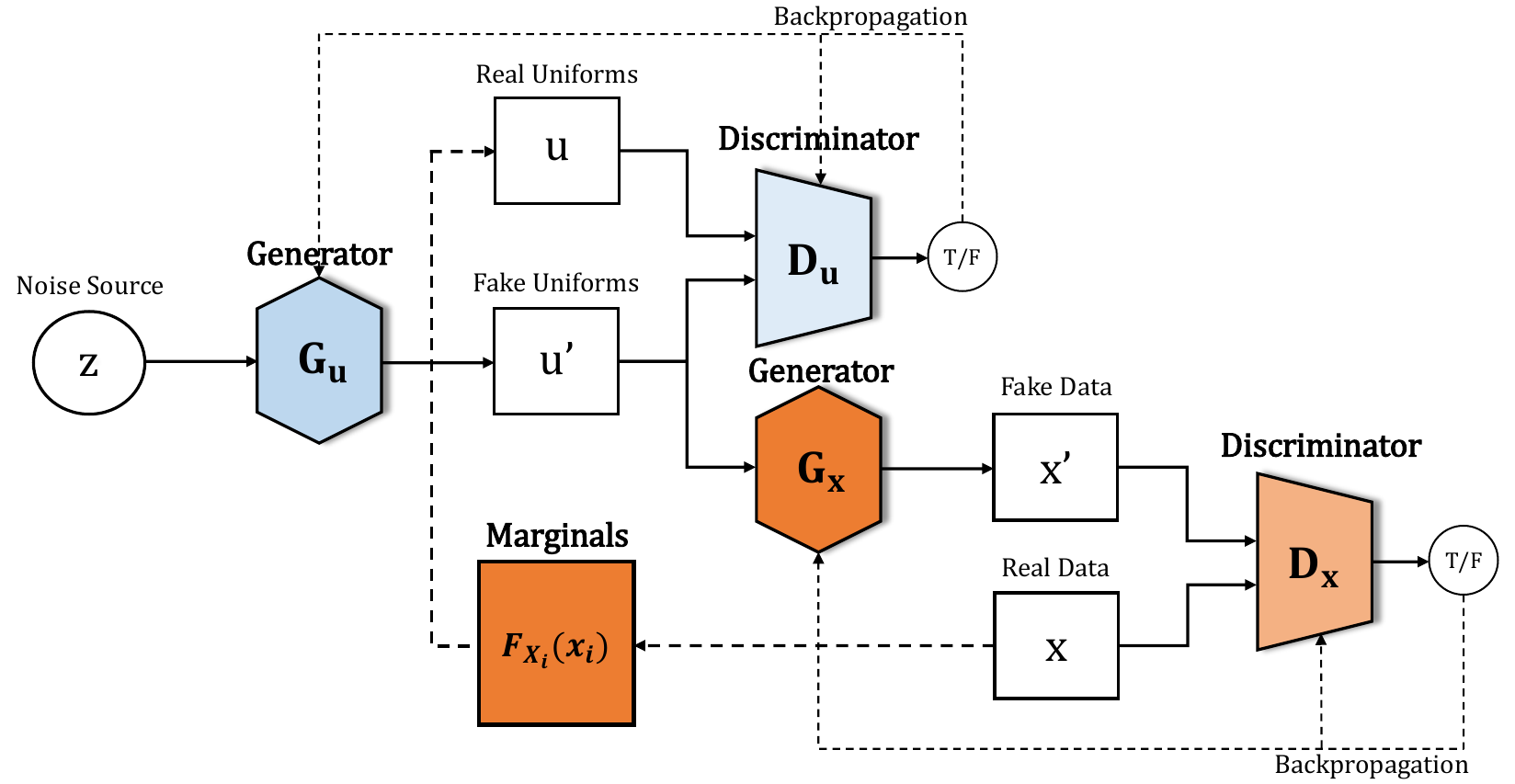}
\caption{Segmented generative network modeling the dependence: a first network builds the dependent uniforms and a second network converts them into new data.}
\label{fig:sgn_UXGAN}
\end{figure}

To identify the marginal cumulative distribution function $F_{X_i}$, for $i=1,\dots, d$, it is sufficient to \textit{cyclically} repeat the aforementioned segmentation process by concatenating another GAN which takes as input the last generated samples $\mathbf{\hat{x}}$ and projects them into a new encoded set $\mathbf{\tilde{u}}$.  

Next section presents some graphical and numerical results, comparing the different methodologies.

\subsection{Evaluation of results}
\label{subsec:sgn_evaluation_of_results}
This section discusses and compares the SGN-C and SGN-D approaches to generate new samples in three different case studies. We consider a $2D$ toy dataset and two higher dimensional datasets, MNIST \cite{MNIST} and CelebA \cite{CelebA}. For each of them we qualitatively evaluate the generation performance of the SGN-C approach (e.g. covariance, Gaussian and $t$-copula) and of the SGN-D approach (with GANs). Finally, for the last dataset scenario, we compare some quantitative results using three different metrics. 

\subsubsection{Qualitative evaluation}
We used Keras with TensorFlow \cite{TensorFlow} as backend to implement the proposed model. The code has been tested on a Windows-based operating system provided with Python 3.6, TensorFlow 1.13.1, Intel core i7-3820 CPU and one GPU GTX1080. 
Two different types of neural networks architectures have been deployed according to the specific dataset under analysis. To allow the reproducibility of the work presented, Tab. \ref{tab:combined_toy} and \ref{tab:combined_images} report all the implementation aspects. The details of the network and the chosen parameters for the first toy example are reported in Tab. \ref{tab:combined_toy}. For the two dataset containing images, the established DCGAN \cite{Radford2016} architecture has been used as foundation for the proposed SGN-D structure with all details reported in Tab. \ref{tab:combined_images}. To overcome numerical issues in the cases of the images, the definition interval of the uniform distribution is transformed from $[0,1]$ to $[-1,1]$.

\begin{table}
	\scriptsize 
	\centering
	\caption{SGN-D based on GAN architecture for synthetic toy model.}
	\begin{tabular}{ p{5cm}|p{2cm}|p{1.5cm}} 
		\toprule
		\textbf{Operation} & \textbf{Feature maps}  		& \textbf{Activation}  \\
		\midrule
		\textbf{Generator $G_u$} & &  \\ 
		$G_u(\mathbf{z}):\mathbf{z} \sim \mathcal{U}(0,1)$ & 2\\ 
		Fully connected & $500$ & ReLU \\
		Dropout & $0.5$ &  \\
		Fully connected & $2$ & Sigmoid \\  \hline
		\textbf{Generator $G_x$} &&   \\ 
		$G_x(\mathbf{u}):\mathbf{u} \sim c_u$ & 2\\ 
		Fully connected & $500$ & ReLU \\
		Dropout & $0.5$ &  \\
		Fully connected & $2$ & Sigmoid \\  \hline
		
		\textbf{Discriminator $D_u$} & &  \\ 
		$D_u(\mathbf{u}):\mathbf{u} \sim c_u$ & $2$\\ 
		Fully connected & $500$ & LeakyReLU \\ 
		Dropout & $0.4$ &  \\ 
		Fully connected & $10$ & LeakyReLU \\ 
		Fully connected & $1$ & Sigmoid \\  \hline
		\textbf{Discriminator $D_x$} & &  \\ 
		$D_x(\mathbf{x}):\mathbf{x} \sim p_x(\mathbf{x})$ & $2$ \\ 
		Fully connected & $500$ & LeakyReLU \\ 
		Dropout & $0.4$ &  \\ 
		Fully connected & $10$ & LeakyReLU \\ 
		Fully connected & $1$ & Sigmoid \\  \hline   \hline

		Number of generators & \multicolumn{2}{c}{2} \\
		Batch size & \multicolumn{2}{c}{64} \\
		Number of iterations & \multicolumn{2}{c}{50000} \\ 
		Leaky ReLU slope &  \multicolumn{2}{c}{0.2} \\ 
		Learning rate &  \multicolumn{2}{c}{0.0002}  \\ 
		Optimizer &  \multicolumn{2}{c}{Adam ($\beta_1$ = 0.5, $\beta_2$ = 0.9999)}  \\ \hline
	\end{tabular}
	
	\label{tab:combined_toy}
\end{table}

\begin{table}
	\scriptsize 
	\centering
	\caption{SGN-D based on DCGAN architecture for images.}
	\begin{tabular}{ p{5cm}|p{2cm}|p{1.5cm}} 
		\toprule
		\textbf{Operation} & \textbf{Feature maps}  		& \textbf{Activation}  \\
		\midrule
		\textbf{Generator $G_u$} & &  \\ 
		$G_u(\mathbf{z}):\mathbf{z} \sim \mathcal{U}(-1,1)$ & 1024\\ 
		Fully connected & $4096$ & ReLU \\ 
		Reshape and upsampling & $(8,8,64)$ &  \\ 
		Convolution and BatchNorm & $64$ & ReLU \\ 
		Upsampling & &  \\ 
		Convolution and BatchNorm & $32$ & ReLU \\ 
		Convolution & $3$ & Tanh \\  \hline
		\textbf{Generator $G_x$} &&   \\ 
		$G_x(\mathbf{u}):\mathbf{u} \sim c_u $ & $(32,32,3)$ \\ 
		Convolution and BatchNorm & $128$ & ReLU \\ 
		Upsampling & &  \\ 
		Convolution and BatchNorm & $64$ & ReLU \\ 
		Convolution & $3$ & Tanh \\  \hline
		
		\textbf{Discriminator $D_u$} & &  \\ 
		$D_u(\mathbf{u}):\mathbf{u} \sim c_u$ & $(32,32,3)$\\ 
		Convolution & $32$ & LeakyReLU \\ 
		Dropout & $0.25$ &  \\ 
		Convolution and BatchNorm & $64$ & LeakyReLU \\ 
		Dropout & $0.25$ &  \\ 
		Convolution and BatchNorm & $128$ & LeakyReLU \\ 
		Dropout & $0.25$ &  \\ 
		Convolution and BatchNorm & $256$ & LeakyReLU \\ 
		Dropout & $0.25$ &  \\ 
		Flatten and dense & $1$ & Sigmoid \\   \hline
		\textbf{Discriminator $D_x$} & &  \\ 
		$D_x(\mathbf{x}):\mathbf{x} \sim p_x(\mathbf{x})$ & $(32,32,3)$ \\ 
		Convolution & $32$ & LeakyReLU \\ 
		Dropout & $0.25$ &  \\ 
		Convolution and BatchNorm & $64$ & LeakyReLU \\ 
		Dropout & $0.25$ &  \\ 
		Convolution and BatchNorm & $128$ & LeakyReLU \\ 
		Dropout & $0.25$ &  \\ 
		Convolution and BatchNorm & $256$ & LeakyReLU \\ 
		Dropout & $0.25$ &  \\ 
		Flatten and dense & $1$ & Sigmoid \\   \hline  \hline

		Number of generators & \multicolumn{2}{c}{2} \\
		Batch size & \multicolumn{2}{c}{64} \\
		Number of iterations & \multicolumn{2}{c}{50000} \\ 
		Leaky ReLU slope &  \multicolumn{2}{c}{0.2} \\ 
		Learning rate &  \multicolumn{2}{c}{0.0002}  \\ 
		Optimizer &  \multicolumn{2}{c}{Adam ($\beta_1$ = 0.5, $\beta_2$ = 0.9999)}  \\ \hline
	\end{tabular}
	
	\label{tab:combined_images}
\end{table}

\subsubsection{1) 2D toy database}
Consider a set of $2$ random variables whose statistics is computed from a collection of $2000$ observations. Thus, given $\mathbf{x} = [\mathbf{x}_1, \mathbf{x}_2]$, we wish to generate a new sample $\mathbf{\hat{x}} = [\mathbf{\hat{x}}_1, \mathbf{\hat{x}}_2]$. In order to impose a non-linear statistical dependence structure, we build $\mathbf{x}$ as follows
\begin{equation}
\label{eq:sgn_toy2D}
\mathbf{x} = [\sin(t), t\cos(t)] + \mathbf{n},
\end{equation}
where $t\sim \mathcal{N}(0,1)$ and $\mathbf{n}\sim \mathcal{N}(\mathbf{0},\sigma^2 \mathbb{I})$ with $\sigma = 0.01$. 

Proceeding as explained in Sec. \ref{subsec:sgn_correlated_uniforms} leads to a set of correlated uniforms and correlated samples, $\mathbf{\hat{u}}$ and $ \mathbf{\hat{x}}$, respectively. Despite having the same covariance matrix and the same marginals, the linear transformation is not capable to account for the full dependence structure as shown in the top row of Fig. \ref{fig:sgn_Case1}.
What is missing is indeed the copula density component in \eqref{eq:sgn_Sklarf}. Since in general there is no closed-form expression for the copula density  $c(\mathbf{u})$, the multivariate Gaussian copula and the Student's $t$-copula are possible choices that contain information regarding correlation coefficients, thus linear properties. Nevertheless, both of them are not enough to cover or at least approximate the dependence in $\mathbf{x}$.
Therefore, we considered both Clayton \cite{ClaytonCop} and Frank Archimedean copulas. From the central row of Fig. \ref{fig:sgn_Case1}, we can immediately understand that there exists a set of parameters $\theta$ which approximates the dependence around the mean values but is not able to do the same around the tails.

In the low-dimensional space, i.d. $2D$, it is still feasible to calculate the empirical copula for a certain bins resolution. In such a case, sampling from copula results in a trivial task and provides good quality of the generated samples (bottom left corner of Fig. \ref{fig:sgn_Case1}). Whenever the space dimension increases, such approach cannot be easily followed anymore due to numerical problems, reason why a SGN-D based methodology, which implicitly estimates the distribution, can be exploited.
Bottom sub-plots with orange samples of Fig. \ref{fig:sgn_Case1} show the generated samples under the SGN-D framework for both $G_u$ (back-projected with quantiles in sample domain) and $G_x$ output samples. 

\begin{figure}
\centering
\includegraphics[scale = 0.6]{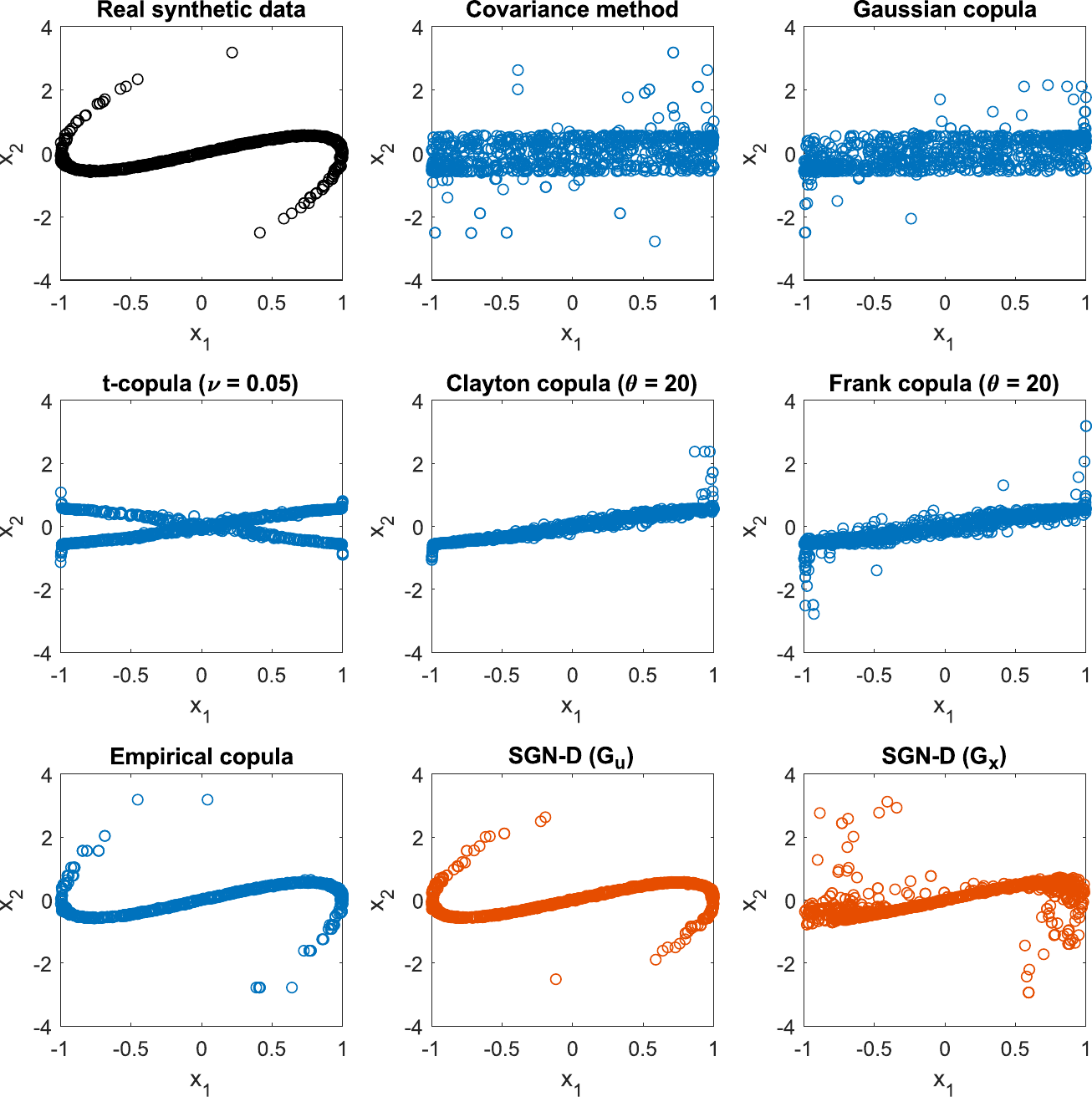}
\caption{Comparison of $2D$ samples generated using SGN-C, Archimedean and empirical copulas (blue samples), and SGN-D (orange samples) approaches.}
\label{fig:sgn_Case1}
\end{figure}

\subsubsection{2) Handwritten digit database}
Consider now a set $28\times 28$ pixels representing images of the digit $4$. This results in $784$ dependent random variables whose statistics is computed from a collection of $2000$ observations. Thus, given $\mathbf{x} = [\mathbf{x}_1, \mathbf{x}_2, \dots, \mathbf{x}_{784}]$, we wish to generate a new sample $\mathbf{\hat{x}} = [\mathbf{\hat{x}}_1, \mathbf{\hat{x}}_2,\dots, \mathbf{\hat{x}}_{784}]$. The non-linear dependence structure is an intrinsic property of images so this example fits the purpose.

Again, following the steps presented in Sec. \ref{subsec:sgn_dependent_uniforms} for both SGN-C and SGN-D methods leads to a set of correlated uniforms and correlated samples, $\mathbf{\hat{u}}$ and $ \mathbf{\hat{x}}$, respectively. The covariance method is not enough to obtain a smooth detailed picture (see the generated digits in Fig. \ref{fig:sgn_Case2}b), nevertheless it is able to capture and reproduce in most of the cases the essence of the picture, i.e., the digit $4$. 
On the other hand, it is interesting to notice that the digits generated by the generator $G_x$ (Fig. \ref{fig:sgn_Case2}f) cannot be distinguished from the real, while the digits marginally back-projected from dependent uniforms coming from generator $G_u$ (Fig. \ref{fig:sgn_Case2}e) are rather blurry. Such effect is due to the discrete nature of the handwritten digit dataset distribution. Indeed, the joint distribution is mostly non zero around small spheres centered in $0$ and $1$, therefore the approximation of the marginals $F_{X_i}(x_i)$ results poor in the intermediate values (due to the steepness of the cumulative function). At the same time, the output of $G_u$ is a set of dependent uniforms with values uniformly distributed from $0$ to $1$ which have to be transformed into the sample space, using the poorly estimated $1D$ inverse cumulative distribution function. The second GAN $G_x$ solves this numerical issue by mapping the data from the uniform space into the sample one through highly non-linear smooth transformations. 

\begin{figure}
\centering
\includegraphics[scale = 0.6]{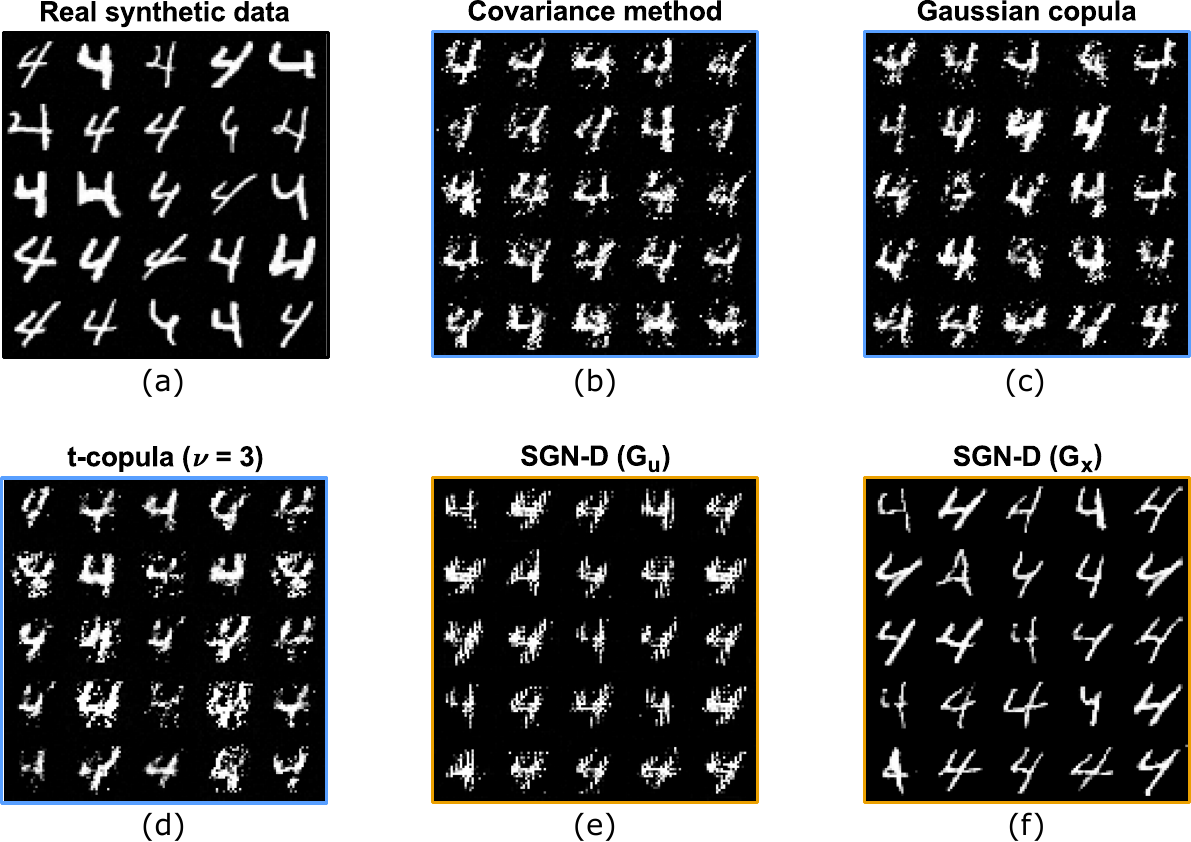}
\caption{Comparison of high-dimensional samples (digits) generated using SGN-C (b)-(d) and SGN-D (e)-(f) approaches.}
\label{fig:sgn_Case2}
\end{figure}

\subsubsection{3) Celebrity faces database}
As last example, we propose a set of (cropped) $32\times 32$ color images representing faces of celebrities covering some pose variations and including different backgrounds. The motivation resides in the intrinsic continuous property of the distribution since all the colors are feasible, yielding to robust cumulative marginals. 
The CelebA dataset \cite{CelebA} contains more than $200K$ faces. To be coherent with previous examples and to focus more on the correct identification of the block components charged to generate both correlated and dependent uniforms rather than to their quality, we considered only the first $20000$ samples of the dataset.
Fig. \ref{fig:sgn_Case3} illustrates the results.

The methods involving covariance and parametric copulas perform poorly in terms of quality of the details but capture the relevant information of the image and replicate it in new blurry faces. On the contrary, GANs introduce higher non-linear dependence, thus harmonic details, the more the network trains itself. 

The most interesting part is that, conversely to the MNIST case where the dependent uniforms were correctly generated but erroneously back-projected to sample domain due to poor quantiles, this time the estimated marginals do not have steep gradients. Therefore, the projected generated samples are smooth as depicted in Fig. \ref{fig:sgn_Case3}e. Moreover, these samples are extremely similar to the output of the generator $G_x$ (Fig. \ref{fig:sgn_Case3}f) accordingly to the interpretation that $G_x$ mimes the quantile functions $F_{X_i}^{-1}(x_i)\; \forall i \in \{1,\dots,1024\}$. Fig. \ref{fig:sgn_Case3_u} illustrates an example of data representation (a) and data generation (b) in the uniform space.

\begin{figure}
\centering
\includegraphics[scale = 0.6]{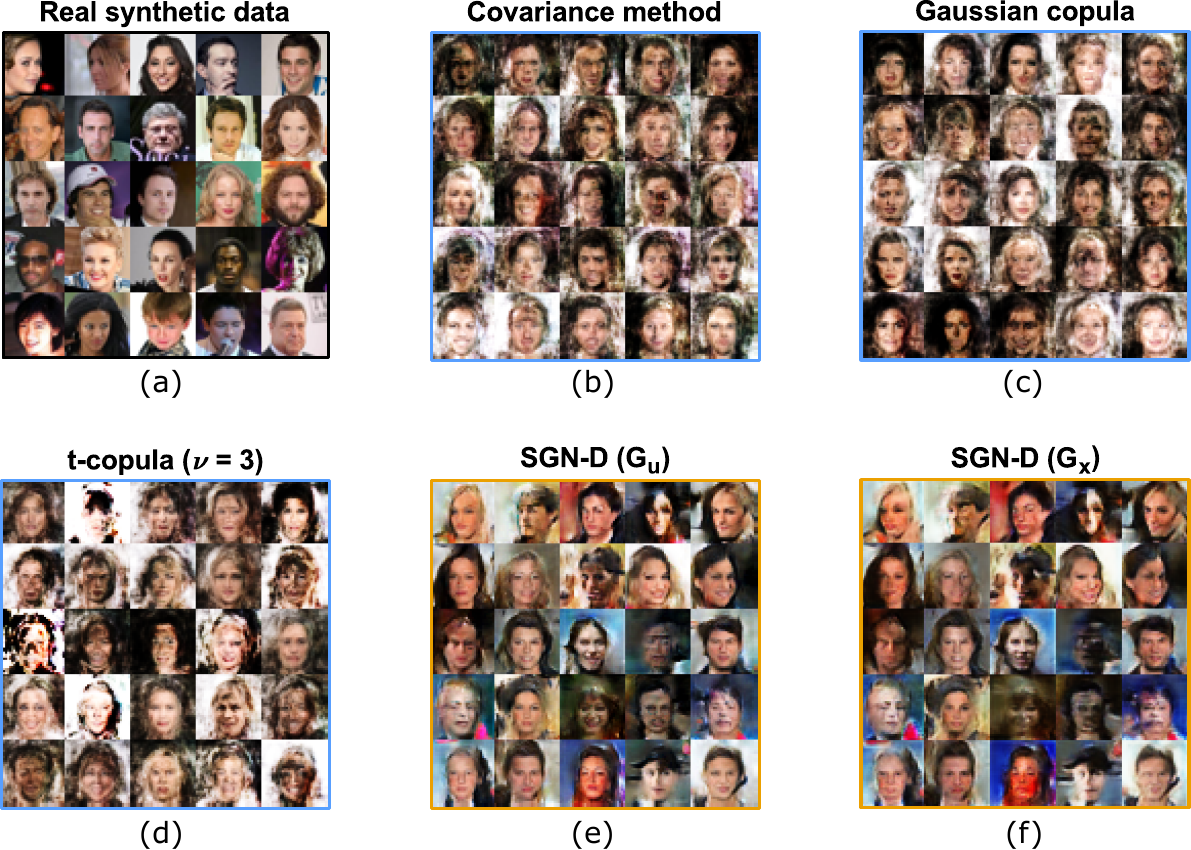}
\caption{Comparison of high-dimensional samples (faces) generated using SGN-C (b)-(d) and SGN-D (e)-(f) approaches.}
\label{fig:sgn_Case3}
\end{figure}

\begin{figure}
\centering
\includegraphics[scale = 0.8]{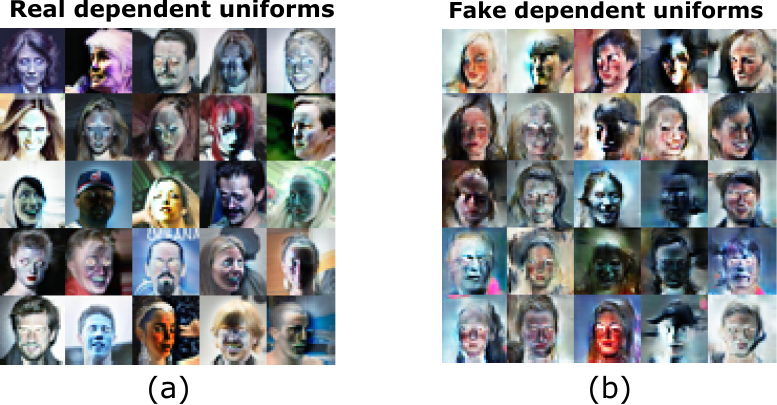}
\caption{Dependent uniforms obtained from transform sampling of original data (a) and obtained as output of $G_u$ (b).}
\label{fig:sgn_Case3_u}
\end{figure}

\subsubsection{Quantitative evaluation}
While several measures have been introduced so far to assess the performance of a specific generative model, there is no consensus as to which measure is the most appropriate \cite{ProsCons}. Nevertheless, lately, measures that deal with embedding layers and feature space have found extensive use. For the purpose of this section, we decided to adopt three of them, in particular, the Inception Score (IS) \cite{InceptionScore}, the Frèchet Inception Distance (FID) \cite{FID}, and the Kernel Inception Distance (KID) \cite{KID} metrics.

The IS computes the average KL divergence between the conditional distribution of the images label $p(y|\mathbf{x})$ and the marginal distribution $p(y)=\mathbb{E}_{\mathbf{x}}[p(y|\mathbf{x})]$, on the pre-trained Inception Net \cite{InceptionV3}. It is defined as
$\exp(\mathbb{E}_{\mathbf{x}}[D_{\text{KL}}(p(y|\mathbf{x})||p(y))])$ and assumes high values for a low entropy of $p(y|\mathbf{x})$, achieved when samples are easily classifiable, and for a high entropy of $p(y)$, to favor diversity.

FID compares the statistics of generated samples to real ones by computing the Frèchet distance between two multivariate Gaussian distributions. Indeed, both data are projected into a feature space (Inception representations) wherein a Gaussian distribution fits them.  

\begin{equation}
\text{FID}(r,g) = ||\mu_r - \mu_g||_2^2 + \text{Tr}\biggl(\Sigma_r + \Sigma_g - 2(\Sigma_r \Sigma_g)^{\frac{1}{2}}\biggr),
\end{equation}
where $X_r \sim \mathcal{N}(\mu_r, \Sigma_r)$ and $X_g \sim \mathcal{N}(\mu_g, \Sigma_g)$ are outputs of a pool layer in the Inception Net \cite{InceptionV3} for real and generated samples, respectively. Low FID values correspond to better similarity in distribution.

However, the Gaussianity of the Inception representations is often not guaranteed. As a result of the ReLU activations, the representations are not negative with some components equal to zero \cite{KID}. To overcome such limitation, we decided to use the KID metric \cite{KID}. 
KID is the squared MMD (Sec.\ref{subsec:sgn_dependent_uniforms}) between the Inception representations. In particular, let $\phi(\cdot)$ be the function mapping the real ($\mathbf{x}_r$) and generated ($\mathbf{x}_g$) samples into the Inception representation, then
\begin{equation}
\text{KID}(r,g) = \text{MMD}^2(\mathcal{H},\phi(\mathbf{x}_r),\phi(\mathbf{x}_g)).
\end{equation}
We used the polynomial kernel $k(\mathbf{x},\mathbf{y})=(\frac{1}{d} \mathbf{x}\cdot \mathbf{y}^T +1)^3$, where $d$ is the representation dimension. Compared to FID, KID has the advantage of being independent from the distribution of the latent representation.

We evaluated the scores only for the generated samples in the CelebA dataset scenario. The Inception network is a deep CNN  pre-trained on the ImageNet dataset. Hence, datasets that are too semantically different from ImageNet would lead to poor inception scores.
Since we are not interested in getting the best performance out of our architecture, but rather in a fair comparison between the different methodologies, we look at relative results. In particular, Tab. \ref{tab:InceptionScores} reports the scores for the different methodologies adopted. 

\begin{table}[]
\centering
\begin{tabular}{l|c|c|c|c|c|c|}
\cline{2-7}
                                          & \multicolumn{6}{c|}{\textbf{CelebA Dataset ($32\times 32$)}} \\ \hline

\multicolumn{1}{|l|}{\textbf{Space}}     & \multicolumn{3}{c|}{\textbf{Uniform}}       & \multicolumn{3}{c|}{\textbf{Sample}}      \\ \hline \hline                                      
                                          
\multicolumn{1}{|l|}{\textbf{Method}}     & \textbf{IS}  & \textbf{FID} & \textbf{KID} & \textbf{IS} & \textbf{FID} & \textbf{KID}     \\ \hline \hline

\multicolumn{1}{|l|}{Real synthetic data} & $3.71$ & $25.6$     & $0.00$ &  $2.77$                    & $18.0$ & $0.00$ \\ \hline \hline
\multicolumn{1}{|l|}{Covariance}         & $3.43$ & $210$     & $0.19$ & $2.13$                    & $136$ & $0.11$                        \\ \hline
\multicolumn{1}{|l|}{Gaussian copula}    & $3.30$ & $208$     & $0.19$ & $2.16$                    & $136$ & $0.11$                        \\ \hline
\multicolumn{1}{|l|}{t-copula}            & $3.50$ & $195$     & $0.16$ & $2.13$                    & $125$ & $0.10$ \\ \hline \hline
\multicolumn{1}{|l|}{SGN-D ($G_u$)}         & $3.27$ & $76.6$     & $0.03$ &  &  &  \\ \hline
\multicolumn{1}{|l|}{SGN-D ($G_x$)}         &  &      &  & $2.61$                    & $70.9$ & $0.03$ \\ \hline 
\end{tabular}
\caption{Performance of SGN-C and SGN-D using the, IS, FID and KID measures on the CelebA dataset.}
\label{tab:InceptionScores}
\end{table}

The IS, FID and KID scores are consistent with human visual perceptions. Indeed, as depicted from visual intuition, there is almost no difference between the scores obtained from covariance and parametric copulas methods, while there is a significant gap between the achieved scores with the linear (SGN-C) and non-linear dependence (SGN-D) approaches.

\subsection{Summary}
This section of the chapter has firstly discussed some known procedures to generate correlated variables from a sample set, highlighting the need of a domain transformation (from the sample to the uniform variables domain). The same domain adaptation has been exploited for the generation of statistically dependent variables, recalling the concept of copula. This mathematical tool enables the partitioning and segmentation of the dependence structure generation into two well-defined steps, an initial step which creates the data dependence between uniform random variables (copula) followed by a second step which projects the uniform random variables back into the sample domain (inverse transform sampling), leading to new data. 
The former case has been analyzed through the aid of different copula structures while the latter case through an estimation of the marginal cumulative distribution (and its inverse), where the more samples are available, the better the approximation is. Such segmentation totally disregards the semantics of the input data and, in principle, can be applied to any type of data/signal. 

This procedure has led to the design of a segmented generative network architecture, based on GANs, successfully implemented as proved by several qualitative and quantitative results. This goes in the direction of explainable machine learning, i.e., a full comprehension of the design of a neural network through a mathematical segmentation of the problem.

In the following section, we describe how to apply such segmentation steps to explicitly estimate the PDF of the collected data.

\section{Copula density neural estimation}
\sectionmark{CODINE}
\label{sec:codine}
A natural way to discover data properties is to study the underlying PDF. Parametric and non-parametric models \cite{Silverman86} are viable solutions for density estimation problems that deal with low-dimensional data. The former are typically used when a prior knowledge on the data structure (e.g. distribution family) is available. The latter, instead, are more flexible since they do not require any specification of the distribution's parameters. Practically, the majority of methods from both classes fail in estimating high-dimensional densities. Hence, some recent works leveraged deep NNs as density estimators \cite{PixelRNN,dinh2017density}.
Although significant efforts have been made to scale NN architectures in order to improve their modeling capabilities, most of tasks translate into conditional distribution estimations. 
Instead, generative models attempt to learn the a-priori distribution to synthesize new data out of it. Deep generative models such as GANs \cite{Goodfellow2014}, VAEs \cite{Kingma2013} and diffusion models \cite{Ho2020}, tend to either implicitly estimate the underlying PDF or explicitly estimate a variational lower bound, providing the designer with no simple access to the investigated PDF. 

In this section, we propose to work with pseudo-observations, a projection of the collected observations into the uniform probability space via transform sampling, as described in Sec. \ref{subsec:sgn_proposal_approach}. The probability density estimation becomes a copula density estimation problem, which is formulated and solved using DL techniques. The envisioned copula density neural estimation method is referred to as CODINE. We further present self-consistency tests and metrics that can be used to assess the quality of the estimator. We prove and exploit the fact that the MI can be rewritten in terms of copula PDFs. Finally, we apply CODINE in the context of data generation.

\subsection{Variational approach for the copula density estimation}
\label{subsec:codine}
Let us assume that the collected $n$ data observations $\{\mathbf{x}^{(1)},\dots,\mathbf{x}^{(n)}\}$ are sampled from $p_{X}(\mathbf{x}) = p_{X}(x_1, x_2, \dots, x_d)$ with CDF $F_{X}(\mathbf{x}) = P(X_1\leq x_1, \dots, X_d \leq x_d)$.
Under similar assumptions as in Sec. \ref{subsec:sgn_transform_sampling}, 
the inverse transform sampling method can be used to map the data into the uniform probability space. In fact, if $U_i$ is a uniform random variable, then
$X_i = F^{-1}_{X_i}(U_i)$ is a random variable with CDF $F_{X_i}$.
Therefore, if the CDF is invertible, the transformation $u_i = F_{X_i}(x_i) \; \forall i=1,\dots,d$ projects the data $\mathbf{x}$ into the uniform probability space with finite distribution's support $u_i \in [0,1]$. The obtained transformed observations are typically called pseudo-observations.
In principle, the transform sampling method is extremely beneficial: it offers a statistical normalization, thus a pre-processing operation that constitutes the first step of any DL pipeline.

To characterize the nature of the transformed data in the uniform probability space, we use the concept of copula introduced in Sec. \ref{subsec:sgn_dependent_uniforms}.
We here recall that, when the multivariate distribution is described in terms of the PDF $p_{X}$, it holds that
\begin{equation}
\label{eq:Sklarf2}
p_{X}(x_1,\dots,x_d) = c_{U}(F_{X_1}(x_1),\dots,F_{X_d}(x_d))\cdot \prod_{i=1}^{d}{p_{X_i}(x_i)},
\end{equation}
where $c_{U}$ is the density of the copula.

The relation in \eqref{eq:Sklarf2} is the fundamental building block of this section. It separates the dependence internal structure of $p_{X}$ into two distinct components: the product of all the marginals $p_{X_i}$ and the density of the copula $c_{U}$. By nature, the former accounts only for the marginal information, thus, the statistics of each univariate variable. The latter, instead, accounts only for the joint dependence of data.

Considering the fact that building the marginals is usually a straightforward task, the estimation of the empirical joint density $\hat{p}_{X}(\mathbf{x})$ of the observations $\{\mathbf{x}^{(1)},\dots,\mathbf{x}^{(n)}\}$ passes through the estimation of the empirical copula density $\hat{c}_{U}(\mathbf{u})$ of the pseudo-observations $\{\mathbf{u}^{(1)},\dots,\mathbf{u}^{(n)}\}$. 

We propose to use deep NNs to model dependencies in high-dimensional data, and in particular to estimate the copula PDF. The proposed framework relies on the following simple idea: we can measure the statistical distance  between the pseudo-observations and uniform i.i.d. realizations using NN  parameterization. Surprisingly, by maximizing a variational lower bound on a divergence measure, we get for free the copula density neural estimator.

\subsubsection{Variational formulation}
\label{subsec:codine_f-div}
Given the definition of $f$-divergence introduced in \eqref{eq:f-divergence}, one might be interested in studying the particular case when the two densities $p$ and $q$ correspond to $c_U$ and $\pi_U$, respectively, where $\pi_U$ describes a multivariate uniform distribution on $[0,1]^d$. In such situation, it is possible to express the copula density function via the variational representation of the $f$-divergence. The following Theorem formulates an optimization problem whose solution yields to the desired copula density.
 
\begin{theorem}
\label{theorem:codine_theorem1}
Let $\mathbf{u} \sim c_U(\mathbf{u})$ be $d$-dimensional samples drawn from the copula density $c_U$. Let $f^*$ be the Fenchel conjugate of $f:\mathbb{R}_+ \to \mathbb{R}$, a convex lower semicontinuous function that satisfies $f(1)=0$ and has derivative $f^{\prime}$. If $\pi_U(\mathbf{u})$ is a multivariate uniform distribution on the unit cube $[0,1]^d$ and $\mathcal{J}_{f}(T)$ is a value function defined as 
\begin{equation}
\mathcal{J}_{f}(T) = \mathbb{E}_{\mathbf{u} \sim c_{U}(\mathbf{u})}\biggl[T\bigl(\mathbf{u}\bigr)\biggr] -\mathbb{E}_{\mathbf{u} \sim \pi_{U}(\mathbf{u})}\biggl[f^*\biggl(T\bigl(\mathbf{u}\bigr)\biggr)\biggr],
\label{eq:codine_discriminator_function_f}
\end{equation}
then
\begin{equation}
\label{eq:codine_optimal_ratio_T}
c_U(\mathbf{u}) = \bigl(f^{*}\bigr)^{\prime} \bigl(\hat{T}(\mathbf{u})\bigr), 
\end{equation}
where
\begin{equation}
\hat{T}(\mathbf{u}) = \arg \max_T \mathcal{J}_f(T),
\end{equation}
\end{theorem}
 
\begin{proof}
From the hypothesis, the $f$-divergence between $c_U$ and $\pi_U$ reads as follows
\begin{equation}
\small
D_f(c_U||\pi_U) = \int_{\mathbb{R}^d}{\pi_U(\mathbf{u})f\biggl(\frac{c_U(\mathbf{u})}{\pi_U(\mathbf{u})}\biggr)\diff \mathbf{u}} = \int_{[0,1]^d}{f\bigl(c_U(\mathbf{u})\bigr)\diff \mathbf{u}}.
\end{equation}
Moreover, from Lemma 1 of \cite{Nguyen2010}, $D_f$ can be expressed in terms of its lower bound via Fenchel convex duality
\begin{equation}
\small
\label{eq:codine_f_bound}
D_f(c_U||\pi_U) \geq \sup_{T\in \mathbb{R}} \biggl\{ \mathbb{E}_{\mathbf{u} \sim c_U(\mathbf{u})} \bigl[T(\mathbf{u})\bigr]-\mathbb{E}_{\mathbf{u}\sim \pi_U(\mathbf{u})}\bigl[f^*\bigl(T(\mathbf{u})\bigr)\bigr]\biggr\},
\end{equation}
where $T: [0,1]^d \to \mathbb{R}$ and $f^*$ is the Fenchel conjugate of $f$.
Since the equality in \eqref{eq:codine_f_bound} is attained for $T(\mathbf{u})$ as
\begin{equation}
\hat{T}(\mathbf{u}) = f^{\prime} \bigl(c_U(\mathbf{u})\bigr),
\end{equation}
it is sufficient to find the function $\hat{T}(\mathbf{u})$ that maximizes the variational lower bound $\mathcal{J}_{f}(T)$.
Finally, by Fenchel duality it is also true that
$c_U(\mathbf{u}) = \bigl(f^{*}\bigr)^{\prime} \bigl(\hat{T}(\mathbf{u})\bigr)$.
\end{proof}

Notice that the density of the copula can be derived with the same approach also by working in the sample domain. Indeed, when $p$ and $q$ correspond to the joint and the product of the marginals, respectively, the following corollary holds.

\begin{corollary}
\label{cor:codine_cor1}
Let $\mathbf{x} \sim p_X(\mathbf{x})$ be $d$-dimensional samples drawn from the joint density $p_X$. Let $f^*$ be the Fenchel conjugate of $f:\mathbb{R}_+ \to \mathbb{R}$, a convex lower semicontinuous function that satisfies $f(1)=0$ and has derivative $f^{\prime}$. If $\pi_X(\mathbf{x})$ is the product of the marginals $p_{X_i}(x_i)$ and $\mathcal{J}_{f}(T)$ is a value function defined as 
\begin{equation}
\label{eq:codine_discriminative_joint_marg}
\mathcal{J}_{f,x}(T) = \mathbb{E}_{\mathbf{x} \sim p_{X}(\mathbf{x})}\biggl[T\bigl(\mathbf{x}\bigr)\biggr] -\mathbb{E}_{\mathbf{x} \sim \pi_{X}(\mathbf{x})}\biggl[f^*\biggl(T\bigl(\mathbf{x})\bigr)\biggr)\biggr],
\end{equation}
then
\begin{equation}
c_U(\mathbf{u}) = \bigl(f^{*}\bigr)^{\prime} \bigl(\hat{T}(\mathbf{F}_X^{-1}(\mathbf{u}))\bigr)
\end{equation}
is the copula density, where
\begin{equation}
\mathbf{F}_X^{-1}(\mathbf{u}) := [F_{X_i}^{-1}(u_i),\dots,F_{X_d}^{-1}(u_d)]
\end{equation}
and
\begin{equation}
\hat{T}(\mathbf{x}) = \arg \max_T \mathcal{J}_{f,x}(T),
\end{equation}
\end{corollary}
 
\begin{proof}
From the hypothesis, the $f$-divergence between $p_X$ and $\pi_X$ reads as follows
\begin{align}
D_f(p_X||\pi_X) & = \int_{\mathbb{R}^d}{\prod_{i}{p_{X_i}(x_i)}f\biggl(\frac{p_X(\mathbf{x})}{\prod_{i}{p_{X_i}(x_i)}}\biggr)\diff \mathbf{x}} \nonumber \\
& = \int_{[0,1]^d}{f\bigl(c_U(\mathbf{u})\bigr)\diff \mathbf{u}},
\end{align}
where $c$ is the density of the copula obtained as in \eqref{eq:Sklarf2}. The thesis then follows immediately from Theorem \ref{theorem:codine_theorem1} 1.
\end{proof}

A great advantage of the formulation in \eqref{eq:codine_discriminator_function_f} comes from the second expectation term. Conversely to the variational discriminative formulation in \eqref{eq:codine_discriminative_joint_marg} that tests jointly with marginals samples, the comparison in \eqref{eq:codine_discriminator_function_f} is made between samples from the joint copula structure and independent uniforms. The latter can be easily generated without the need of any scrambler that factorizes $p_X$ into the product of the marginal PDFs. On the other hand, \eqref{eq:codine_discriminator_function_f} needs samples from the copula, thus, needs an estimate of the marginals of $X$ to apply transform sampling. In the following, we use the formulation in \eqref{eq:codine_discriminator_function_f} which inherently possesses another desired property when using NNs.

\subsubsection{Parametric implementation}
\label{subsec:codine_implementation}
To proceed, we propose to parametrize $T(\mathbf{u})$ with a deep NN $T_{\theta}$ of parameters $\theta$ and maximize $\mathcal{J}_{f}(T)$ with gradient ascent and back-propagation $\hat{\theta} = \arg \max_{\theta} \mathcal{J}_f(T_{\theta})$.
Since at convergence the network outputs a transformation of the copula density evaluated at the input $\mathbf{u}$, the final layer possesses a unique neuron with activation function that depends on the generator $f$ (see the code \cite{CODINE_github} for more details).
The resulting estimator of the copula density reads as follows
\begin{equation}
\hat{c}_U(\mathbf{u}) = \bigl(f^{*}\bigr)^{\prime} \bigl(T_{\hat{\theta}}(\mathbf{u})\bigr),
\end{equation}
and its training procedure enjoys two normalization properties. 

The former consists in a natural normalization of the input data in the interval $[0,1)$ via transform sampling that facilitates the training convergence and helps producing improved dependence measures \cite{Poczos2012}. The latter normalization property is perhaps at the core of the proposed methodology. The typical problem in creating neural density estimators is to enforce the network to return densities that integrate to one
\begin{equation}
\label{eq:codine_density_test}
\int_{\mathbb{R}^d}{p_X(\mathbf{x};\theta)\diff \mathbf{x}} = 1
\end{equation}
Energy-based models have been proposed to tackle such constraint, but they often produce intractable densities (due to the normalization factor, see \cite{Papamakarios2015}). Normalizing flows \cite{Rezende2015} provide exact likelihoods but they are limited in representation. In contrast, the discriminative formulation of \eqref{eq:codine_discriminator_function_f} produces a copula density neural estimator that naturally favors a solution of \eqref{eq:codine_density_test}, without any architectural modification or regularization term.

\subsubsection{Evaluation measures}
\label{subsec:codine_metrics}
\begin{figure}[t]
	\centering
	\includegraphics[scale=0.3]{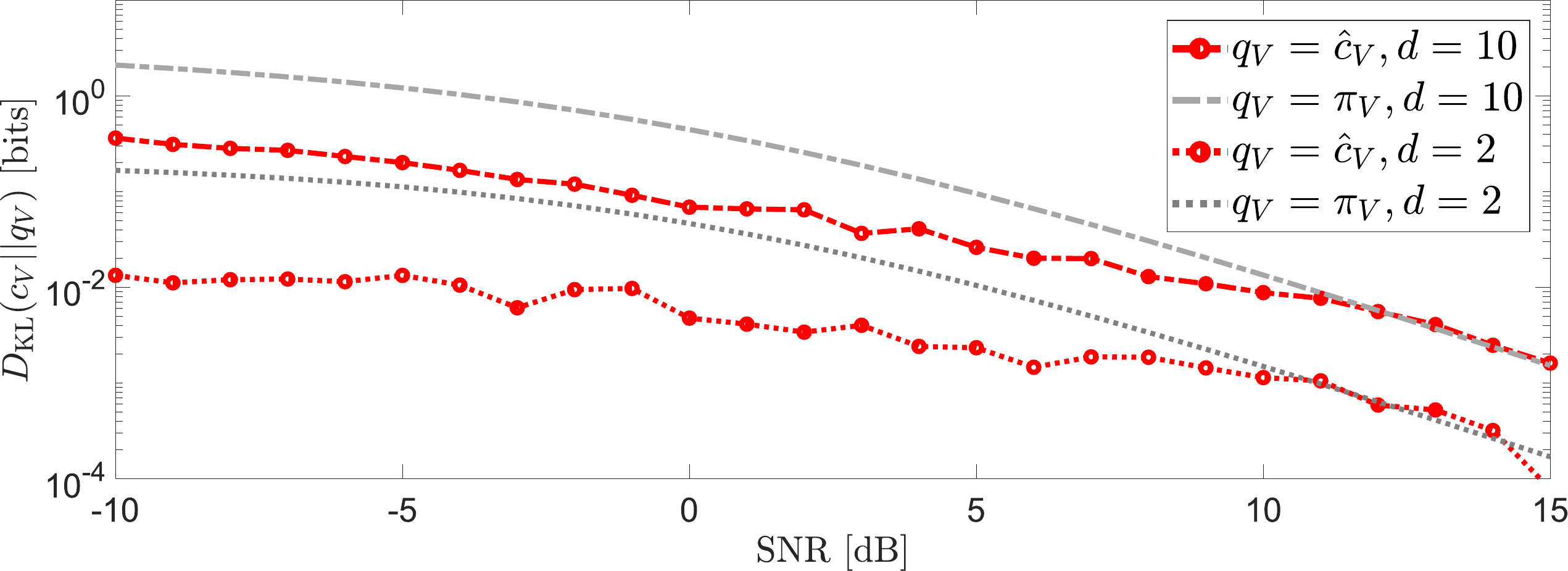}
	\caption{Approximation quality of $c_V$. Comparison between CODINE method $\hat{c}_V$ and the flat copula density $\pi_V$ for different values of the signal-to-noise ratio (SNR) and for different dimensionality $d$ of the input.}
	\label{fig:codine_q_c}
\end{figure}
To assess the quality of the copula density estimator $\hat{c}_U(\mathbf{u})$, we propose the following set of self-consistency tests over the basic property illustrated in \eqref{eq:codine_density_test}. In particular,
\begin{enumerate}
\item if $\hat{c}_U(\mathbf{u})$ is a well-defined density and $\hat{c}_U(\mathbf{u})=c_U(\mathbf{u})$, then the following relation must hold
\begin{equation}
\mathbb{E}_{\mathbf{u} \sim \pi_{U}(\mathbf{u})}\bigl[\hat{c}_U(\mathbf{u})\bigr]=1,
\end{equation}
\item in general, for any $n$-th order moment, if $\hat{c}_U(\mathbf{u})$ is a well-defined density and $\hat{c}_U(\mathbf{u})=c_U(\mathbf{u})$, then
\begin{equation}
\mathbb{E}_{\mathbf{u} \sim \pi_{U}(\mathbf{u})}\bigl[\mathbf{u}^n\cdot \hat{c}_U(\mathbf{u})\bigr]=\mathbb{E}_{\mathbf{u} \sim c_{U}(\mathbf{u})}\bigl[\mathbf{u}^n\bigr].
\end{equation}
\end{enumerate}
The first test verifies that the copula density integrates to one while the second set of tests extends the first test to the moments of any order. Similarly, joint consistency tests can be defined, e.g., the Spearman rank correlation $\rho_{X,Y}$ between pairs of variables can be rewritten in terms of their joint copula density $\hat{c}_{UV}$ and it reads as follows
\begin{equation}
\rho_{X,Y} = 12 \cdot\mathbb{E}_{(\mathbf{u},\mathbf{v}) \sim \pi_{U}(\mathbf{u})\pi_{V}(\mathbf{v})}\bigl[\mathbf{u} \mathbf{v}\cdot \hat{c}_{UV}(\mathbf{u},\mathbf{v})\bigr]-3.
\end{equation}
\begin{figure}[t]
	\centering
	\includegraphics[scale=0.37]{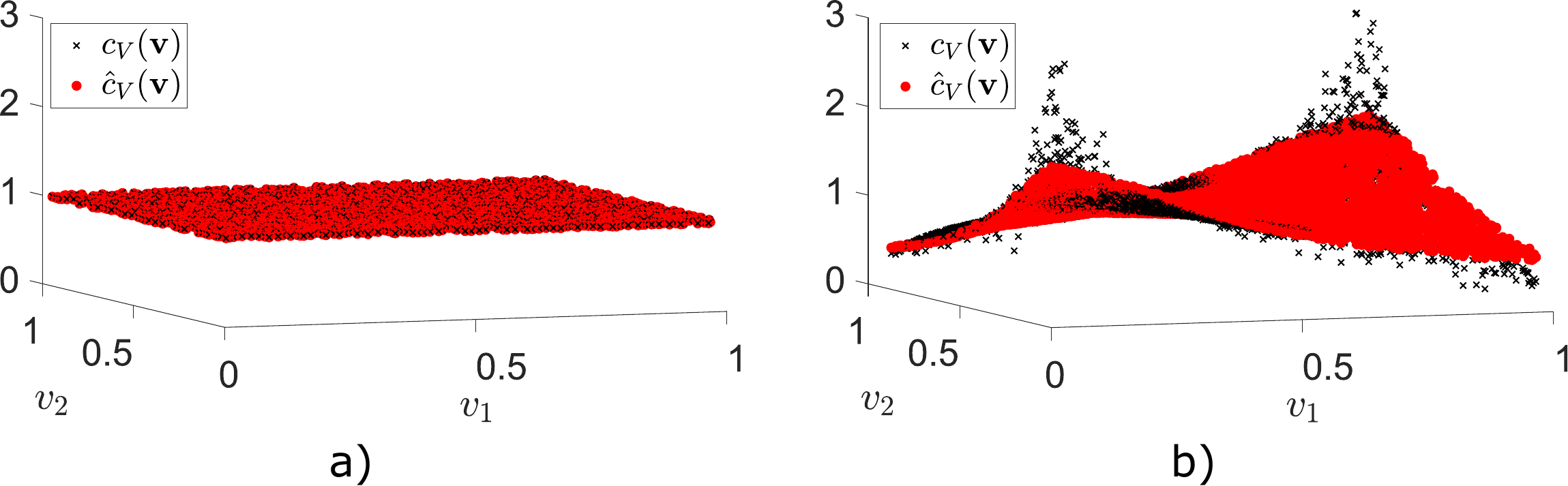}
	\caption{Ground-truth and estimated copula density (SNR=$0$ dB) at the channel output ($d=2$) using the GAN $f$ generator for: a) uncorrelated noise $\rho=0$. b) correlated noise with coefficient $\rho=0.5$.}
	\label{fig:codine_copula_approx}
\end{figure}
When the copula density is known, it is possible to assess the quality of the copula density neural estimator $\hat{c}_{U}(\mathbf{u})$ by computing the KL divergence between the true and the estimated copulas
\begin{equation}
Q_c = D_{\text{KL}}(c_{U}||\hat{c}_{U}) = \mathbb{E}_{\mathbf{u}\sim c_{U}(\mathbf{u})}\biggl[\log\frac{c_{U}(\mathbf{u})}{\hat{c}_{U}(\mathbf{u})}\biggr].
\end{equation}

Once the dependence structure is characterized via a valid copula density $\hat{c}_{U}(\mathbf{u})$, a multiplication with the estimated marginal components $\hat{p}_{X_i}(x_i)$, $\forall i=\{1,\dots,d\}$ yields the estimate of the joint PDF $\hat{p}_{X}(\mathbf{x})$. In general, it is rather simple to build one-dimensional marginal density estimates $\hat{p}_{X_i}(x_i)$, e.g., using histograms or kernel functions. 

Now, as a first example to validate the density estimator, we consider the transmission of $d$-dimensional Gaussian samples over an additive white Gaussian channel (AWGN). Given the  AWGN model $Y=X+N$, where $X \sim \mathcal{N}(0,\mathbb{I})$ and $N \sim \mathcal{N}(0,\Sigma_N)$, it is simple to obtain closed-form expressions for the probability densities involved. 
Indeed, the copula joint density function reads as follows
\begin{equation}
\small
c_{UV}(\mathbf{u},\mathbf{v}) = \sqrt{\frac{\det(\tilde{\Sigma}_N+\mathbb{I})}{\det(\Sigma_N)}}\frac{\exp\bigl(-\frac{1}{2}\bigl(\mathbf{F}_Y^{-1}(\mathbf{v})-\mathbf{F}_X^{-1}(\mathbf{u})\bigr)^T\Sigma_N^{-1}\bigl(\mathbf{F}_Y^{-1}(\mathbf{v})-\mathbf{F}_X^{-1}(\mathbf{u})\bigr)\bigr)}{\exp\bigl(-\frac{1}{2}\bigl(\mathbf{F}_Y^{-1}(\mathbf{v})\bigr)^T\bigl(\tilde{\Sigma}_N+\mathbb{I}\bigr)^{-1}\bigl(\mathbf{F}_Y^{-1}(\mathbf{v})\bigr)\bigr)},
\end{equation}
whereas, the copula density of the output $Y$ assumes form as
\begin{equation}
\small
\label{eq:codine_copula_gaussian_y}
c_{V}(\mathbf{v}) = \sqrt{\frac{\det(\tilde{\Sigma}_N+\mathbb{I})}{\det(\Sigma_N+\mathbb{I})}}{\exp\biggl(-\frac{1}{2}\bigl(\mathbf{F}_Y^{-1}(\mathbf{v})\bigr)^T\bigl(\bigl(\Sigma_N+\mathbb{I}\bigr)^{-1}-\bigl(\tilde{\Sigma}_N+\mathbb{I}\bigr)^{-1}\bigr)\bigl(\mathbf{F}_Y^{-1}(\mathbf{v})\bigr)\biggr)},
\end{equation}
 where $\mathbf{F}_X(\mathbf{x})$ is an operator that element-wise applies transform sampling (via Gaussian cumulative distributions) to the components of $\mathbf{x}$ such that $(\mathbf{u},\mathbf{v}) = (\mathbf{F}_X(\mathbf{x}),\mathbf{F}_Y(\mathbf{y}))$ and $\tilde{\Sigma}_N = \Sigma_N \odot \mathbb{I}$, where $\mathbf{A} \odot \mathbf{B}$ denotes the Hadamard product between matrices $\mathbf{A}$ and $\mathbf{B}$. 

In Fig. \ref{fig:codine_q_c} we illustrate the KL divergence (in bits) between the ground-truth and the neural estimator obtained using the GAN generator $f(u)$ reported in Tab. \ref{tab:codine_generators}. To work with non-uniform copula structures, we study the case of non-diagonal noise covariance matrix $\Sigma_N$. In particular, we impose a tridiagonal covariance matrix such that $\Sigma_N=\sigma_N^2 R$ where $R_{i,i} = 1$ with $i=1,\dots,d$, and $R_{i,i+1}=\rho$, with $i=1,\dots,d-1$ and $\rho=0.5$. Moreover, Fig. \ref{fig:codine_q_c} depicts the quality of the approximation for different values of the signal-to-noise ratio (SNR), defined as the reciprocal of the noise power $\sigma_N^2$, and for different dimensions $d$. To provide a numerical comparison, we also report the KL divergence $D_{\text{KL}}(c_{V}||\pi_{V})$ between the ground-truth and the flat copula density $\pi_V = 1$. It can be shown that when $c_V$ is Gaussian, we obtain
\begin{equation}
D_{\text{KL}}(c_{V}||\pi_{V}) = \frac{1}{2}\log\biggl(\frac{\det(\tilde{\Sigma}_N+\mathbb{I})}{\det(\Sigma_N+\mathbb{I})}\biggr),
\end{equation}
where the proof uses the theorem on the expectation of quadratic forms, which states that
\begin{equation}
    \mathbb{E}_{\mathbf{x} \sim p_X(\mathbf{x})}[X^T\mathbf{A}X] = \mu^T\mathbf{A}\mu + \Tr(\mathbf{A}\Sigma_X).
\end{equation}

Notice that in Fig. \ref{fig:codine_q_c} we use the same simple NN architecture for both $d=2$ and $d=10$. Nonetheless, CODINE can approximate multidimensional densities even without any further hyper-parameter search. 
Fig. \ref{fig:codine_copula_approx}a reports a comparison between ground-truth and estimated copula densities at $0$ dB in the case of independent components ($\rho=0$) and correlated components ($\rho=0.5$). It is worth mentioning that when there is independence between components, the copula density is everywhere unitary $c_V(\mathbf{v})$$=$$1$. Hence, independence tests can be derived based on the structure of the estimated copula via CODINE, but we leave it for future discussions.

\begin{table}[t!]
\centering
\begin{tabular}{l|l|l}
Name & Generator $f(u)$                   & Conjugate $f^*(t)$  \\ \hline
GAN  & $u\log u -(u+1)\log (u+1)+\log(4)$ & $-\log (1-\exp(t))$ \\ \hline
KL   & $u\log u$                          & $\exp(t-1)$         \\ \hline
HD   & $(\sqrt{u}-1)^2$                   & $t/(1-t)$     \\ \hline
\end{tabular}
\caption{List of generator and conjugate functions used in the experiments.}
\label{tab:codine_generators}
\end{table} 

\subsection{Applications}
\label{subsec:codine_applications}

\subsubsection{Mutual information estimation}
Given two random variables, $X$ and $Y$, the MI $I(X;Y)$ quantifies the statistical dependence between $X$ and $Y$. It measures the amount of information obtained about one variable via the observation of the other and it can be rewritten also in terms of KL divergence as
${I}(X;Y) = D_{\text{KL}}(p_{XY}||p_Xp_Y)$.
From Sklar's theorem, it is simple to show that the MI can be computed using only copula densities as follows
\begin{equation}
\label{eq:codine_mutual_information_copula}
{I}(X;Y) = \mathbb{E}_{(\mathbf{u},\mathbf{v})\sim c_{UV}(\mathbf{u},\mathbf{v})}\biggl[\log\frac{c_{UV}(\mathbf{u},\mathbf{v})}{c_U(\mathbf{u}) c_V(\mathbf{v})}\biggr].
\end{equation}
Therefore, \eqref{eq:codine_mutual_information_copula} requires three separate copula densities estimators, each of which obtained as explained in Sec. \ref{subsec:codine}. Alternatively, one could learn the copulas density ratio via maximization of the variational lower bound on the MI (see Ch. \ref{sec:mi_estimators}). Using again Fenchel duality, the KL divergence
\begin{equation}
D_{\text{KL}}(c_{UV}||c_Uc_V) = \int_{[0,1]^{2d}}{c_{UV}(\mathbf{u},\mathbf{v})\log\biggl(\frac{c_{UV}(\mathbf{u},\mathbf{v})}{c_U(\mathbf{u})c_V(\mathbf{v})}\biggr)\diff \mathbf{u} \diff \mathbf{v}}
\end{equation}
corresponds to the supremum over $T$ of
\begin{equation}
\label{eq:codine_DIME_copula}
\mathcal{J}_{\text{KL}}(T) = \mathbb{E}_{c_{UV}}\biggl[T\bigl(\mathbf{u},\mathbf{v}\bigr)\biggr] - \mathbb{E}_{c_{U}c_{V}}\biggl[\exp{\bigl(T \bigl(\mathbf{u},\mathbf{v }\bigr)-1 \bigr)}\biggr].
\end{equation}
When $X$ is a univariate random variable, its copula density $c_U$ is unitary. Notice that \eqref{eq:codine_DIME_copula} can be seen as a special case of the more general \eqref{eq:codine_discriminator_function_f} when $f$ is the generator of the KL divergence and the second expectation is not computed over independent uniforms with distribution $\pi_U$ but over samples from the product of copula densities $c_U\cdot c_V$. 
We estimate the MI between $X$ and $Y$ in the AWGN model using \eqref{eq:codine_DIME_copula} and the generators described in Tab. \ref{tab:codine_generators}. Fig. \ref{fig:codine_i_x_y}a and Fig. \ref{fig:codine_i_x_y}b show the estimated MI for $d=1$ and $d=5$, respectively, and compare it with the closed-form capacity formula $I(X;Y) = d/2\log_2(1+\text{SNR})$. 

\begin{figure}
	\centering
	\includegraphics[scale=0.35]{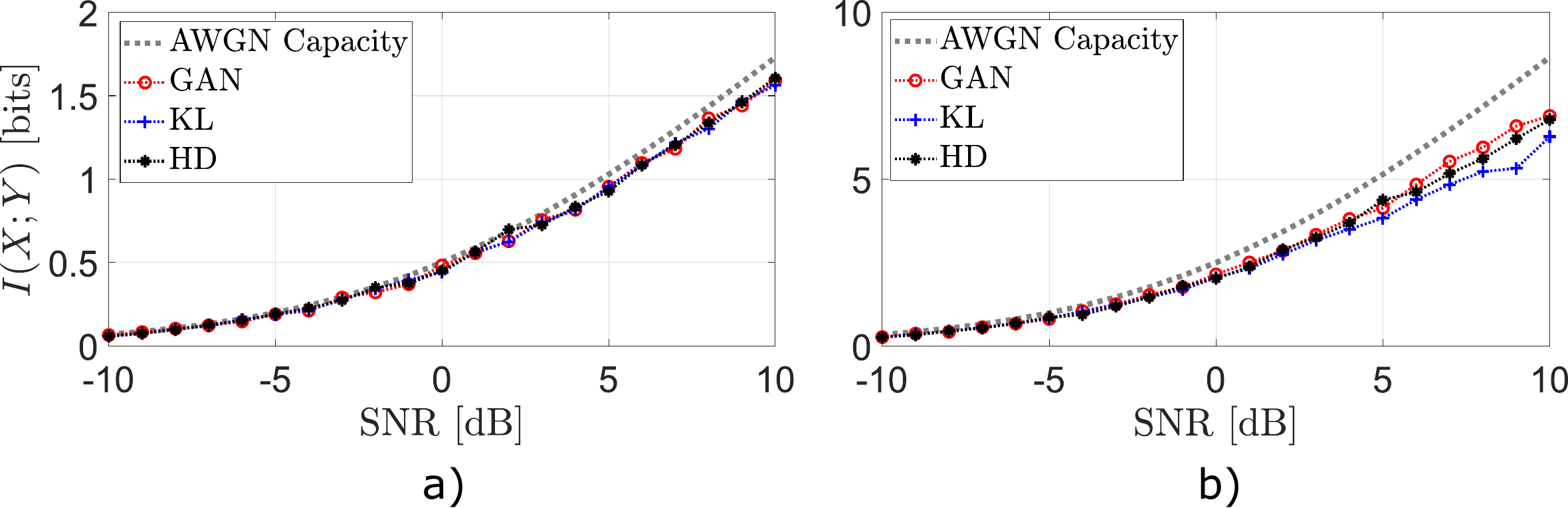}
	\caption{Estimated mutual information $I(X;Y)$ via joint copula $c_{UV}$ with different generators $f$ for: a) $d=1$. b) $d=5$.}
	\label{fig:codine_i_x_y}
\end{figure}

\subsubsection{Data generation}
\begin{figure}[b]
	\centering
	\includegraphics[scale=0.38]{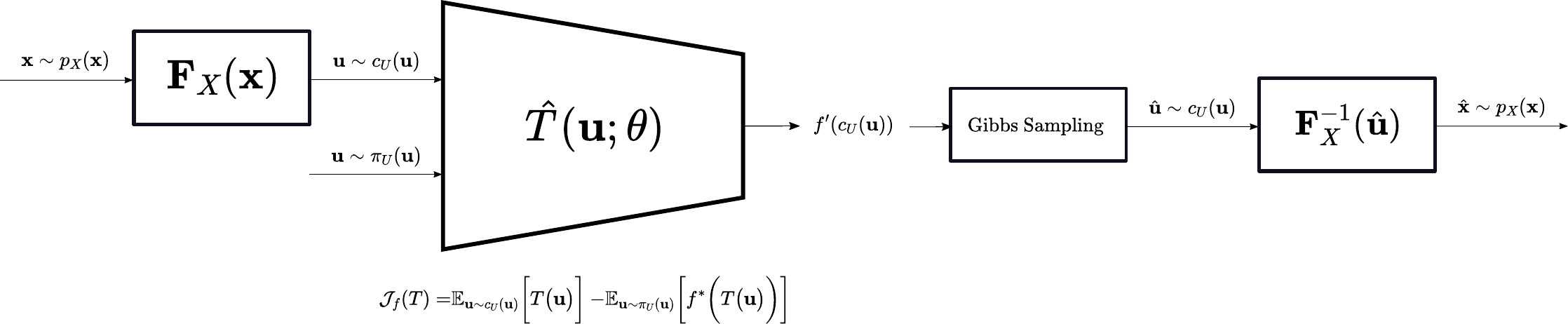}
	\caption{CODINE generation strategy.}
	\label{fig:codine_generation}
\end{figure}
As a second application example, we generate $n$ new pseudo-observations $\{\hat{\mathbf{u}}^{(1)},\dots,\hat{\mathbf{u}}^{(n)}\}$ from $\hat{c}_U$ by deploying a Markov chain Monte Carlo (MCMC) algorithm. Validating the quality of the generated data provides an alternative path for assessing the copula estimate itself.
We propose to use Gibbs sampling to extract valid uniform realizations of the copula estimate. In particular, we start with an initial guess $\hat{\mathbf{u}}^{(0)}$ and produce next samples $\hat{\mathbf{u}}^{(i+1)}$ by sampling each component from univariate conditional densities $\hat{c}_U(u_j^{(i+1)}|u_1^{(i+1)},\dots,u_{j-1}^{(i+1)},u_{j+1}^{(i)},\dots,u_{d}^{(i)})$ for $j=1,\dots,d$.
It is clear that the generated data in the sample domain is obtained via inverse transform sampling through the estimated quantile functions $\hat{F}_{X_i}^{-1}$.
The proposed generation scheme is illustrated in Fig. \ref{fig:codine_generation}.

Consider a bi-dimensional random variable whose realizations have form $\mathbf{x} = [\mathbf{x}_1, \mathbf{x}_2]$ and for which we want to generate new samples $\mathbf{\hat{x}} = [\mathbf{\hat{x}}_1, \mathbf{\hat{x}}_2]$. To force a non-linear statistical dependence structure, we define the same toy example $\mathbf{x}$ as in \eqref{eq:sgn_toy2D}.
We use CODINE to estimate its copula density and sample from it via Gibbs sampling. Fig. \ref{fig:codine_toy} compares the copula density estimate obtained via kernel density estimation (Fig. \ref{fig:codine_toy}a) with the estimate obtained using CODINE (Fig. \ref{fig:codine_toy}b). It also shows the generated samples in the uniform (Fig. \ref{fig:codine_toy}d) and in the sample domain (Fig. \ref{fig:codine_toy}f). 
It is plausible that the Gibbs sampling mechanism produced some discrepancies between $\mathbf{x}$ and $\hat{\mathbf{x}}$.

\begin{figure}[t]
	\centering
	\includegraphics[scale=0.35]{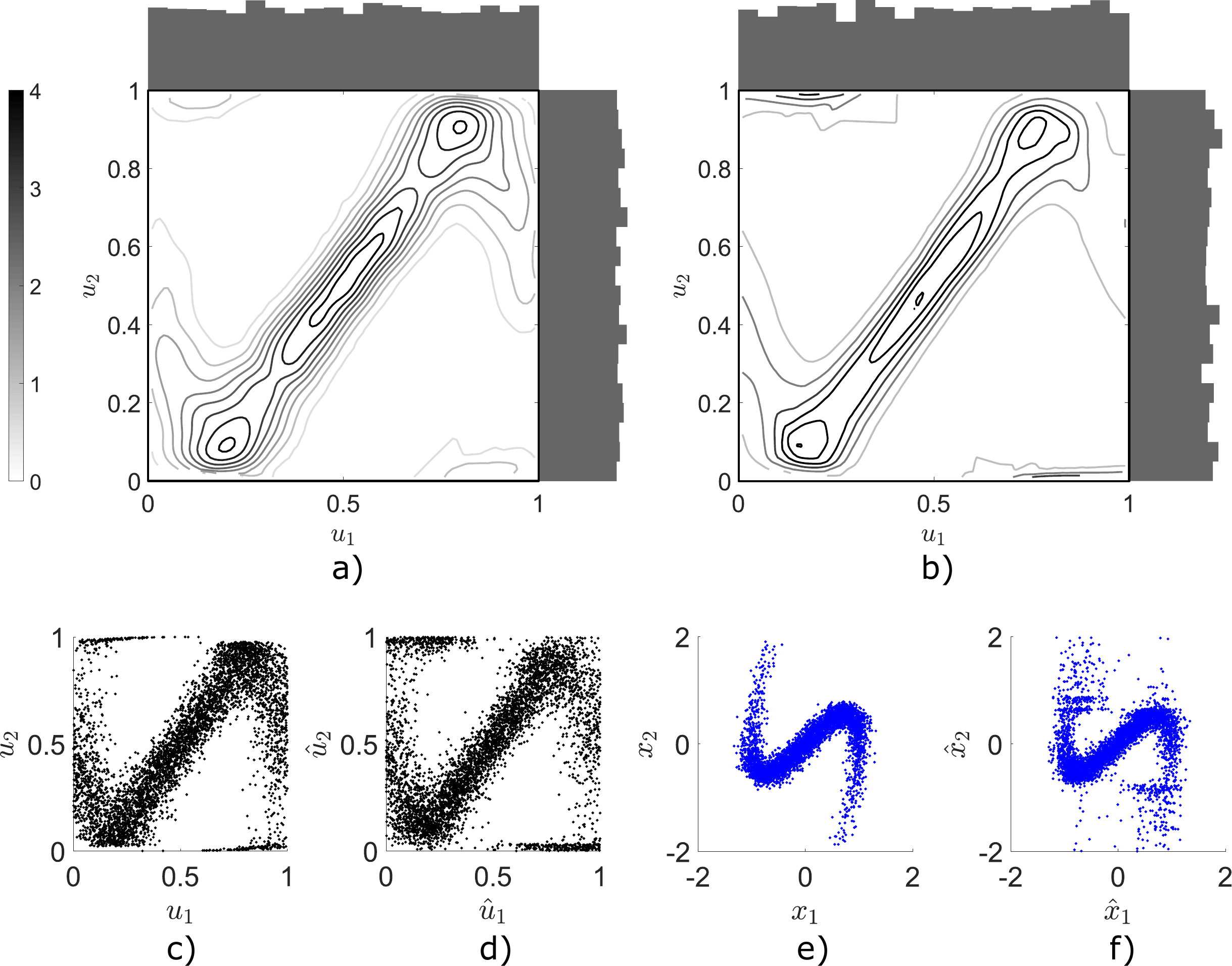}
	\caption{Toy example contour plot and marginal densities of a) ground-truth copula density $c_u$ obtained using kernel density estimation and b) copula density neural estimate $\hat{c}_u$. c) Pseudo-observations. d) Data generated in the uniform probability space via Gibbs sampling. e) Observations. f) Data generated in the sample domain via inverse transform sampling.}
	\label{fig:codine_toy}
\end{figure}

As a more complex example, we report the generation of the MNIST handwritten digits of size $28 \times 28$. In particular, we study the copula density of the latent space obtained from an AE. 
With such experiment, we firstly prove that the approach works even for data of dimension $d = 25$ and we secondly want to emphasize the fact that the approach is potentially scalable to higher dimensions. 
The idea is to train a CNN AE (see Fig. \ref{fig:codine_architecture}) to reconstruct digits. During training, the AE learns latent representations (via an encoder mapping) that can be analyzed and synthesized with CODINE. 

\begin{figure}[t]
	\centering
	\includegraphics[scale=0.4]{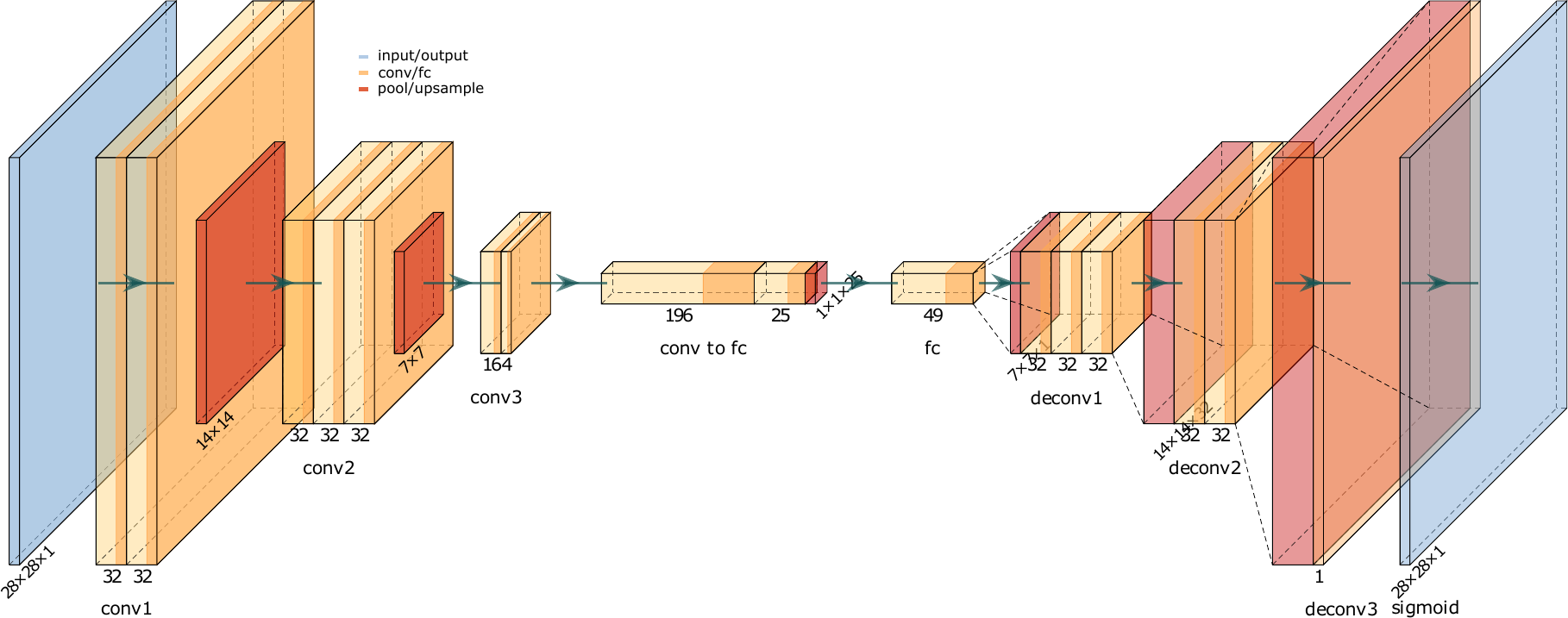}
	\caption{Autoencoder architecture used to learn latent vectors ($d=25$). The autoencoder is trained with binary cross-entropy}
	\label{fig:codine_architecture}
\end{figure}

The analysis block (Fig. \ref{fig:codine_generation} comprises a first marginal block to project the latent representations into the uniform probability space. Then, a CODINE block learns the copula density which is used by Gibbs sampling to generate new uniform latent representations. An inverse transform sampling block maps data from the uniform to the sample space.
Once new latent samples are generated, it is possible to feed them into the pre-trained decoder and obtain new digits, as illustrated in Fig. \ref{fig:codine_digits}.

\begin{figure}[b]
	\centering
	\includegraphics[scale=0.15]{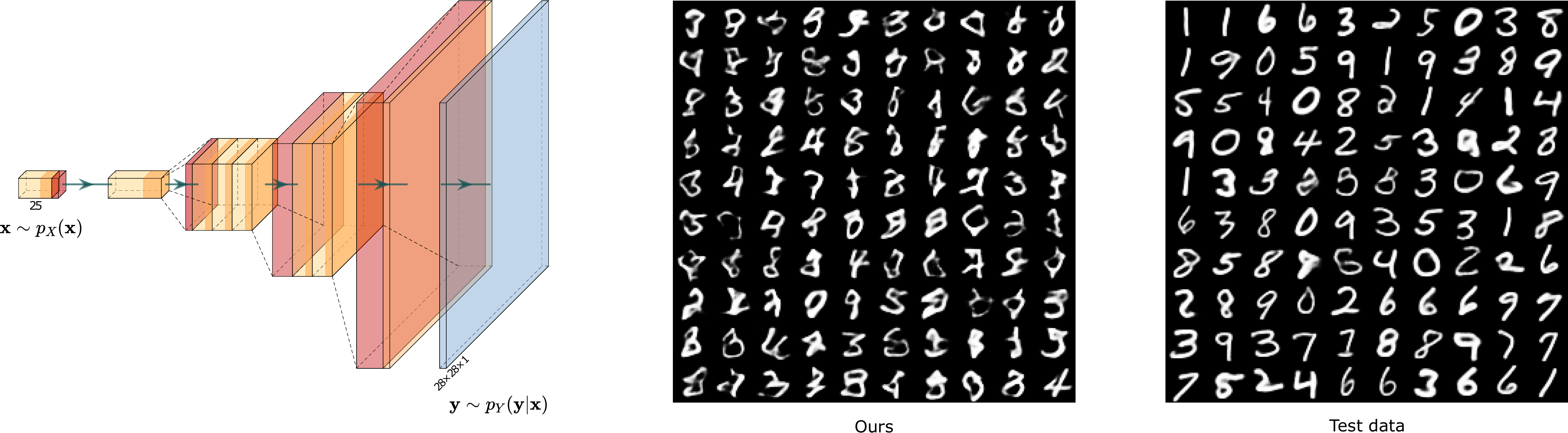}
	\caption{100 randomly selected digits obtained using CODINE to generate the latent vector fed in the decoder. Comparison with test data.}
	\label{fig:codine_digits}
\end{figure}

Although the CODINE block $T(\mathbf{u})$ is a simple shallow NN, the generated digits visually resemble the original ones, meaning that our approach managed to, at least partially, estimate the density of the uniform latent representations.

\subsubsection{An idea to drive the generation with the copula estimate}
MCMC methods require an increasing amount of time to sample from high-dimensional distributions. Thus, we study if it is possible to devise a neural sampling mechanism which exploit the copula density guidance. In particular, we exploit the MMD measure defined in \eqref{eq:sgn_MMD}. Indeed, let $(\chi,d)$ be a nonempty compact metric space in which two copula densities, $c_U(\mathbf{u})$ and $c_V(\mathbf{v})$, are defined. Then, the MMD reads as
\begin{equation}
\text{MMD}(\mathcal{G},c_U,c_V) := \sup_{g\in \mathcal{G}}\bigl\{\mathbb{E}_{\mathbf{u}\sim c_U(\mathbf{u})}[g(\mathbf{u})]-\mathbb{E}_{\mathbf{v}\sim c_V(\mathbf{v})}[g(\mathbf{v})]\bigr\},
\label{eq:codine_MMD}
\end{equation}
where $\mathcal{G}$ is a class of functions $g:\chi \rightarrow \mathbb{R}$. 
Since $c_U=c_V$ if and only if $\mathbb{E}_{\mathbf{u}\sim c_U(\mathbf{u})}[g(\mathbf{u})]=\mathbb{E}_{\mathbf{v}\sim c_V(\mathbf{v})}[g(\mathbf{v})]$ $\forall g\in \mathcal{G}$, the MMD measures the disparity between $c_U$ and $c_V$. If $g(\mathbf{u})=c_U(\mathbf{u})$ is a valid function, we can define a plausible loss function based on the MMD metric as follows
\begin{equation}
\min_{\theta_G} \mathbb{E}_{\mathbf{u}\sim c_U(\mathbf{u})}[c_U(\mathbf{u})]-\mathbb{E}_{\mathbf{v}\sim \pi_V(\mathbf{v})}[c_U(G(\mathbf{v};\theta_G))].
\label{eq:codine_MMD-metric}
\end{equation}
Thus, given $n$ pseudo-observations $\{\mathbf{u}^{(1)},\dots,\mathbf{u}^{(n)}\}$ for which we have built and estimated the underlined $c_U(\mathbf{u})$, it is possible to design a NN architecture, the generator $G$, which maps independent uniforms with distribution $\pi_V$ into uniforms with distribution $c_{V,\theta_G}$. The guidance provided by $c_U(\mathbf{u})$ helps minimizing the discrepancy between the two copulas when the optimization is performed over $\theta_G$. The optimal generator resulting from the solution of \eqref{eq:codine_MMD-metric} synthesizes new pseudo-observations $\hat{\mathbf{u}} = G(\mathbf{v};\theta_G^*)$.

To verify if $c_U$ is a properly defined function, it is useful to notice that the Wasserstein metric, and in particular the Kantorovich Rubinstein duality, links with the MMD in \eqref{eq:codine_MMD} for a class of functions $\mathcal{G}$ that are $K$-Lipschitz continuous
\begin{equation}
\label{eq:codine_WGAN}
W(c_U,c_V) = \frac{1}{K}\sup_{||h||_L\leq K}\biggl[ \mathbb{E}_{\mathbf{u}\sim c_U(\mathbf{u})}[h(\mathbf{u})]-\mathbb{E}_{\mathbf{v}\sim c_V(\mathbf{v})}[h(\mathbf{v}))\biggr].
\end{equation}
Under such conditions, \eqref{eq:codine_MMD-metric} can be interpreted as the generator loss function of a Wasserstein-GAN \cite{WGAN} where the optimum discriminator $h$ is supposed to be known and corresponds to the learnt copula density $c_U$. The proposed idea lies in between two established approaches. The first one, from generative moment matching networks \cite{GMMs}, assumes $\mathcal{G}$ as the reproducing kernel Hilbert space where $g$ is a kernel $k \in \mathcal{H}$ and the supremum in \eqref{eq:codine_MMD} is thus attained. Such MMD-based approach only optimizes over the generator's parameters but does not produce expressive generators, mainly because of the restriction imposed by the kernel structure. The second, instead, requires to learn both the generator and the discriminator, the latter in order to reach the supremum in \eqref{eq:codine_WGAN}. However, enforcing the Lipschitz constraint is not trivial and the alternation between generator and discriminator training suffers from the usual instability and slow convergence problems of GANs. Even if the copula-based approach does not claim optimality, it possesses two desirable properties: compared to kernel-based methods, it uses a more powerful and appropriate discriminator, the copula density itself. Moreover, the fact that $c_U$ is obtained from a prior analysis renders the generator learning process uncoupled from the discriminator's one.

The last concept serves as a foundational seed idea, inviting future exploration and offering ample room for subsequent considerations and applications.

\subsection{Summary}
\label{subsec:codine_conclusions}
This section presented CODINE, a copula density neural estimator. It works by maximizing a variational lower bound on the $f$-divergence between two distributions defined on the uniform probability space, namely, the distribution of the pseudo-observations, i.e., the copula, and the distribution of independent uniforms. CODINE also provides alternative approaches for measuring statistical dependence such as the MI, and for data generation.

\part{Part 2}

\chapter{Medium Modeling via Generative Adversarial Networks} 
\chaptermark{Medium Modeling with GANs}
\label{sec:medium}
Channel modeling plays a crucial role in the development of reliable communication technologies. The performance of any communication system is strongly dependent on the physical and statistical properties of its communication medium. However, different infrastructures and application scenarios pose challenges such as extreme channel variability, including noise, interference, fading, and distortion, rendering the development of comprehensive mathematical bottom-up channel models impossible. 

In this chapter, we present a general purpose data-driven solution for modeling any communication medium. In particular, we use GANs to implicitly model the empirical channel statistical distribution obtained from real measurement campaigns. We propose a time-frequency strategy to model the noise and extend it to noise affecting multiple conductors. We also provide tips and tricks on how to properly train such complex NNs.

The results presented in this chapter are documented in \cite{RighiniLetizia2019,Letizia2019a}.

\section{Generative adversarial network based channel synthesis}
\sectionmark{GAN-based channel synthesis}
\label{sec:gan_ch_synthesis}
We define as synthetic channel model a phenomenological method that emulates the statistics of the communication channel. Such method does not include the medium physical knowledge to define the model but only the channel transfer functions (CTFs) distribution. Specifically, we shortly describe a tool that is able to generate synthetic data that follow the observed statistics, totally abstracting from the physical interpretation of the medium.

GANs (see Sec. \ref{sec:gans}) have already found some use and applications in communication theory for problems such as stochastic channel modeling and information encoding \cite{OsheaGAN} and for end-to-end system design where the channel effects on the input signals are modeled via conditional GANs \cite{Ye2018}. Nonetheless,
GANs have been mostly successfully applied on image synthesis. Images have two fundamental properties that make them suitable for GAN implementation: 
\begin{itemize}
\item CNNs enable an ad hoc learning scheme, reducing overfitting and mode-collapsing problems;
\item A qualitative assessment based on samples could be in principle enough since visual examination of samples by humans is one of the most common and intuitive ways to evaluate GANs.
\end{itemize} 
Less work has been carried out when generating structured non-images data. Indeed, problems like global convergence of the training and modeling high-dimensional complex distributions are still open questions. 
Among these problems, mode collapse is perhaps the most critical one. This is a well know issue of GAN architectures and it is characterized by a partial sampling of the data distribution. MGANs \cite{MGAN} are
a possible architecture to reduce the probability to run into this problem employing multiple generators instead of using a single one as in the original GAN (see Fig. \ref{fig:medium_MGAN}).

\begin{figure}[t]
  \centering
	\includegraphics[scale=0.3]{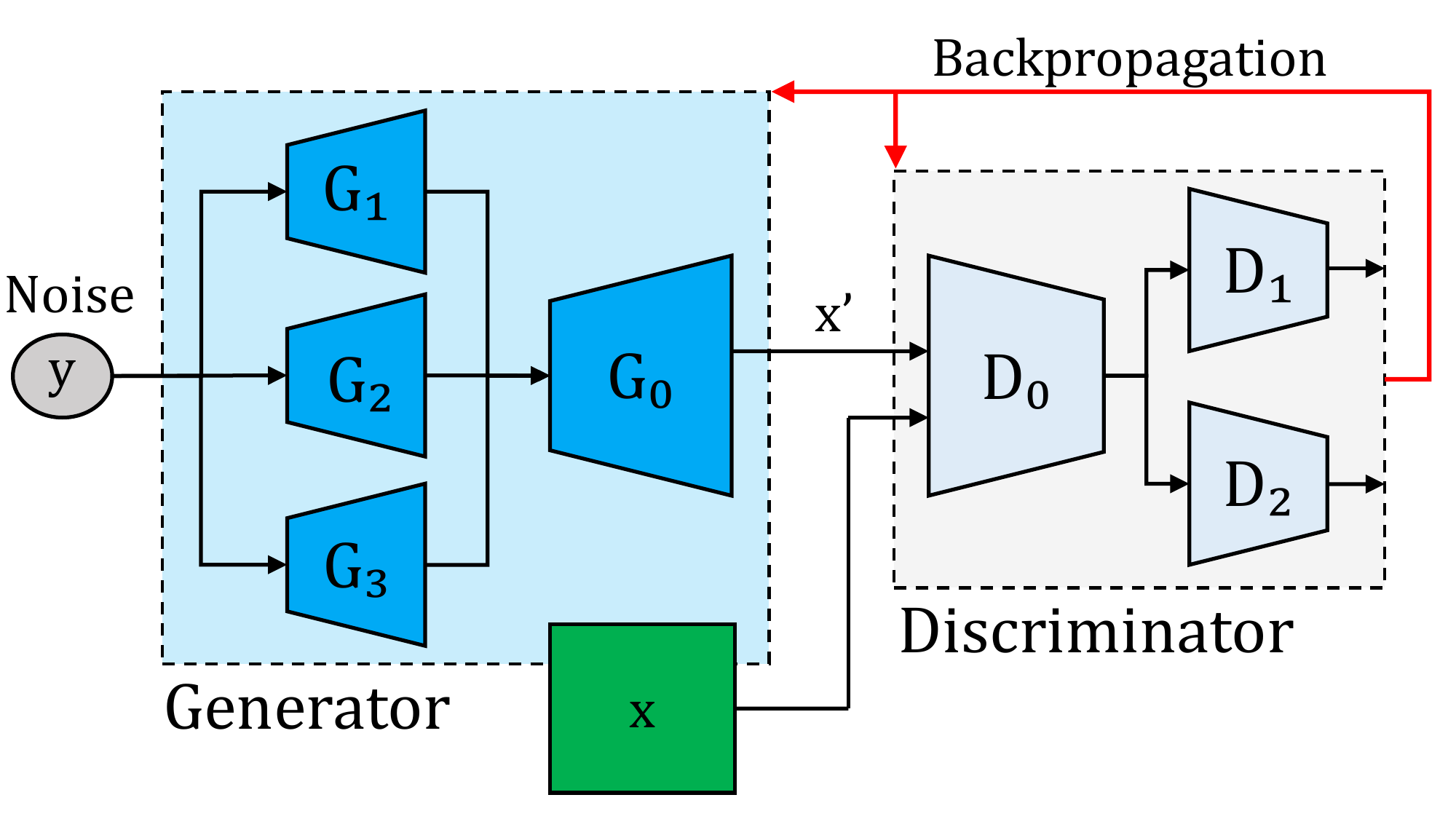}
	\caption{MGAN architecture to generate the new channel latent realizations.}
	\label{fig:medium_MGAN}
\end{figure}%

Complex channels such as the optical fiber or the power line ones demand precise characterization over different frequency bands, making vanilla GANs practically unusable for training synthetic models. In fact, CTFs are realizations of high-dimensional random processes for which no ad hoc architectures have been developed.

We thus propose a more general approach for generating synthetic channels, based on feature characterization and generation.

Let $\mathbf{x} \sim \mathcal{D}$ be the input data which we would like to learn the statistics $p_{X}(\mathbf{x})$. If an a-priori knowledge of the complexity of the distribution $p_{X}(\mathbf{x})$ is provided, the idea is to target a lower-dimensional distribution $p_{Z}(\mathbf{z})$. Suppose there exists a function $F: X\to Z$ which maps $\mathbf{x}$ into $\mathbf{z}$, then a deterministic relation between the two distributions exists, so that with the change of variable rule we obtain $p_{X}$ from $p_{Z}$ as follows
\begin{equation}
p_{X}(\mathbf{x}) = p_{Z}(F(\mathbf{x}))\cdot \biggl| \det \frac{\partial F(\mathbf{x})}{\partial \mathbf{x}}\biggr|.
\label{eq:medium_NICE}
\end{equation}
We propose to model only the distribution $p_{Z}(\mathbf{z})$ of the features $\mathbf{z}$ by identifying the transform function $F(\cdot)$ as the encoder block of an already trained AE (e.g. a probabilistic VAE). Then, after generating new features $\mathbf{z}_g$ with MGANs, the decoder block $G(\cdot)$ takes $\mathbf{z}_g$ as input and produces new samples $\mathbf{x}_g$. 
We forced the features space to be energy-constrained. To keep trace of this limit, a good approach is to add a regularization term in the cost function. In detail, the proposed cost function for the AE training is
\begin{equation}
\label{PowerAE}
\mathcal{L}_{\text{AE}} = \mathbb{E}_{\mathbf{x}\sim \mathcal{D}}[\delta(\mathbf{x},G(F(\mathbf{x})))+\lambda | \mathbf{A}\cdot F(\mathbf{x})+\mathbf{b}|^2_2],
\end{equation}
where $\delta$ is the cross entropy, $\mathbf{A}$ and $\mathbf{b}$ are constant parameters of a linear transformation, a weight matrix and a bias vector, respectively. A comparison between real and generated latent samples is offered in Fig. \ref{fig:medium_Features}. Architectural details are reported in Tab. \ref{tab:autoencoder_nn} and \ref{tab:gan_nn}.

\begin{figure}
  \centering
	\includegraphics[scale=0.6]{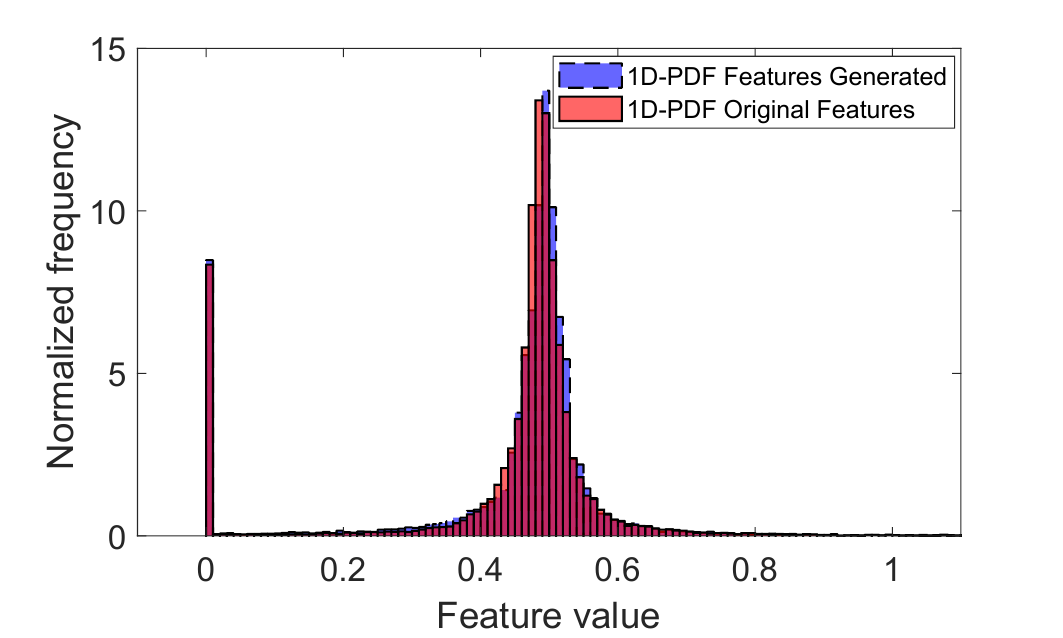}
	\caption{Cumulative 1D-PDF of the extracted features from the AE and of the generated ones using MGANs.}
	\label{fig:medium_Features}
\end{figure}

\begin{table}
	\scriptsize 
	\centering
	\caption{AE architecture for channel feature generation.}
	\begin{tabular}{ p{5cm}|p{3cm}|p{3cm}} 
		\toprule
		Operation & Feature maps  		& Activation  \\
		\midrule
		\textbf{Encoder} & &\\
		Fully connected & 100 & eLU \\ 
		Dropout &0.5&  \\ 
		Fully connected &80& ReLU  \\ 
		Dropout &0.5&  \\ 
		Fully connected & 50 & ReLU  \\  \hline
		\textbf{Features} & &\\
		Fully connected & 25 & ReLU \\  \hline
		\textbf{Decoder} & 50	&   \\
		Dropout &0.5&   \\
		Fully connected & 80 & ReLU  \\  
		Dropout &0.5&   \\
		Fully connected & 100 & Tanh  \\ \hline
		Batch size &  \multicolumn{2}{c}{128}  \\ 
		Number of iterations &  \multicolumn{2}{c}{100000}  \\ 
		Learning rate &  \multicolumn{2}{c}{0.0002}   \\ 
		Regularization constant &  \multicolumn{2}{c}{$\lambda$ = 0.0002}  \\ 
		Optimizer &  \multicolumn{2}{c}{Adam ($\beta_1$ = 0.9, $\beta_2$ = 0.999)}  \\
		Weight matrix &  \multicolumn{2}{c}{$\mathbf{A}=\mathbb{I}$} \\
		Bias vector &  \multicolumn{2}{c}{$\mathbf{b}$ = -0.5} \\ \hline
	\end{tabular}
	
	\label{tab:autoencoder_nn}
\end{table}

\begin{table}
	\scriptsize 
	\centering
	\caption{MGAN architecture.}
	\begin{tabular}{ p{3cm}|p{1.5cm}|p{1.5cm}} 
		\toprule
		Operation & Feature maps  		& Activation  \\
		\midrule
		\textbf{Generators} & &  \\ 
		$G(\mathbf{y}):\mathbf{y} \sim \mathcal{U}(-1,1)$ & 16\\ 
		Dropout &0.5& \\ 
		Fully connected &32& ReLU \\ 
		Dropout &0.5&  \\ 
		Fully connected &45& ReLU \\ 
		Dropout &0.5&  \\ 
		Fully connected & 60 & ReLU \\  \hline
		\textbf{Common Generator $G_0$} & &  \\ 
		Fully connected & 70 & ReLU \\ 
		Fully connected & 30 & ReLU \\ 
		Fully connected & 25 & Sigmoid \\  \hline
		\textbf{Common Discriminator $D_0$} &&   \\ 
		Fully connected & 25 & Leaky ReLU  \\
		Dropout &0.5& \\
		Fully connected & 18 & Leaky ReLU \\  
		Dropout &0.5&  \\
		Fully connected & 16 & Leaky ReLU \\  \hline
		\textbf{Classifier} & &\\ 
		Fully connected & 3 & Softmax \\  \hline
		\textbf{Discriminator} & &  \\ 
		Fully connected & 1 & Sigmoid \\ \hline
		Number of generators & \multicolumn{2}{c}{3} \\
		Batch size & \multicolumn{2}{c}{128} \\
		Number of iterations & \multicolumn{2}{c}{150000} \\ 
		Leaky ReLU slope &  \multicolumn{2}{c}{0.2} \\ 
		Learning rate &  \multicolumn{2}{c}{0.0002}  \\ 
		Regularization constant &  \multicolumn{2}{c}{$\beta$ = 0.5} \\ 
		Optimizer &  \multicolumn{2}{c}{Adam ($\beta_1$ = 0.5, $\beta_2$ = 0.9999)}  \\ \hline
	\end{tabular}
	
	\label{tab:gan_nn}
\end{table}

To verify the goodness of the proposed methodology (schematically described in Algorithm \ref{AE&GAN}), consolidated metrics for real and generated data are compared in Ch. \ref{sec:plc}.

\begin{algorithm}[b]
\caption{AE \& GAN}\label{AE&GAN}
\begin{algorithmic}[1]
\BState Training VAE with input/output $\mathbf{x}$
\State \hspace{3mm}$\mathbf{z} \gets \text{encoder}(\mathbf{x})$
\State \hspace{3mm}Save decoder block
\BState Training MGAN with the original features $\mathbf{z}$ 
\State \hspace{3mm}Sample uniform $\mathbf{y}\sim \mathcal{U}(-1,1)$
\State \hspace{3mm}$\mathbf{z}_g \gets \text{generator}(\mathbf{y})$
\BState $\mathbf{x}_g \gets \text{decoder}(\mathbf{z}_g)$
\end{algorithmic}
\end{algorithm}

Notice that the proposed strategy is extremely similar with the most recent and successful text-to-image architectures such as the latent diffusion model Stable Diffusion \cite{SD2021}. In fact, instead of working directly in the complicated pixel domain, it is convenient to compresses the image $\mathbf{x}$ into a significantly smaller latent variable $\mathbf{z}$, rendering in this way the model faster and more efficient.

\section{Spectrograms generation for noise modeling}
\sectionmark{GAN-based noise synthesis}
\label{sec:medium_channelsynthesis}
Modeling and reproducing noise patterns play an important role in the development of enhanced communication algorithms. 
ML techniques can also be exploited to model complex noise distributions and synthetically reproduce unseen traces. 
Traditional methods, however, do not provide an ensemble characterization of the noise and its
time-variant nature, resulting in difficult parametrization. 
To overcome these limitations, a top-down modeling approach that exclusively depends on the measurements can be followed. Inspired by the SpecGAN method illustrated
in \cite{donahue2018adversarial}, we propose a versatile approach to generate noise in communications, which will be denoted also as synthetic noise. 

The core idea consists of transforming noise measurements into spectrograms which are used to train a DCGAN to generate new spectrograms with the same statistical distribution. Then, the Griffin-Lim algorithm \cite{GriffinLim} converts the synthesized spectrograms into new noise traces.

The scalability of the approach we illustrate in the following allows to incorporate the mutual dependence of multi-conductor noise traces and replicate them.
Finally, the presented method is evaluated through qualitative and quantitative metrics on power line noise measurements: the generated noise traces are perceived indistinguishable from the measured ones, and at the same time, their statistical properties are preserved as proven by the numerical results discussed in Ch. \ref{sec:plc}. 

\subsection{Spectrogram representation}
Time-variant signals are often represented through spectrograms. When dealing with the Fourier transform, time localization gets lost, therefore transforms such as the Short Time Fourier Transform (STFT) or wavelet transform can incorporate information about time localization, at expenses of a less precise frequency localization. The magnitude squared of the STFT is defined as the spectrogram and it gives information on how energy varies over time and frequency.
To analyze and generate noise traces, we divided the measured noise into several traces of length $N$ samples and for each vector we found its spectrogram representation. In particular, let $x(n)$ and $X_w(m,f)$ be the noise sequence and its STFT, respectively. Let $w(n)$ be the analysis window of length $L$, then from the definition of STFT
\begin{equation}
X_w(m,f)=\mathcal{F}[x_w(m,n)]=\sum_{n=0}^{N-1}{x_w(m,n)e^{-j2\pi f n}}
\end{equation}
with
\begin{equation}
x_w(m,n) = w(n-mH)\cdot x(n),
\end{equation}
the spectrogram has expression
\begin{equation}
S(m,f) = |X_w(m,f)|^2,
\end{equation}
where $m$ represents the index for the frame in time, $H$ the stride, and $f$ represents the index for the frequency bin. 

\begin{figure}
	\centering
	\includegraphics[scale=0.35]{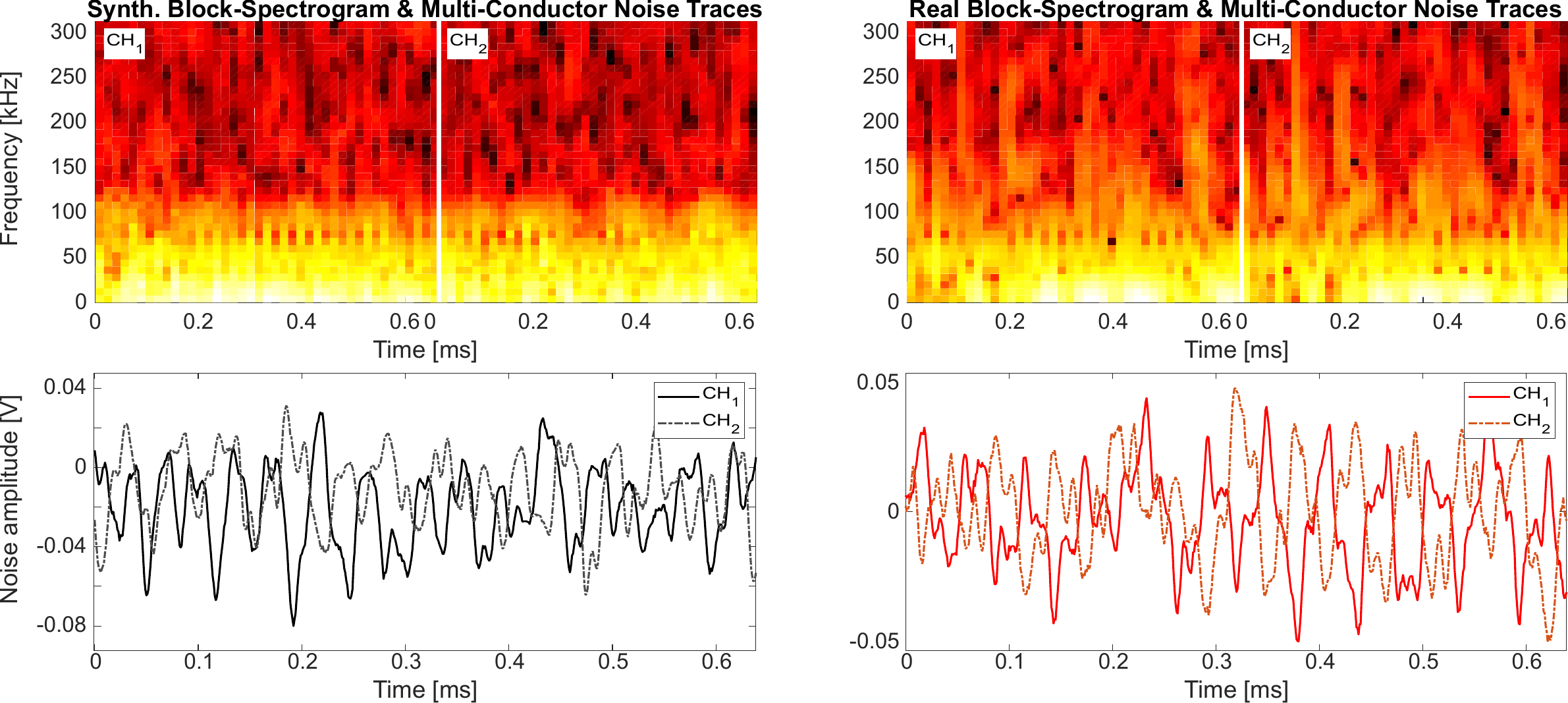}
	\caption{Block-spectrogram of randomly picked  multi-conductor generated (synthetic) and measured (real) noise traces.}
	\label{fig:medium_Block-Spectrogram}
\end{figure}

\subsection{DCGAN}
The statistics of the noise, i.e. spatial and time dependence, resides in the collected dataset. However, the spectrogram $S(m,f)$ offers a visual representation of the measured noise, consequentially  the statistical information is nicely transferred into images. In order to generate new noise with the same distribution of the observed one, new spectrograms $\hat{S}(m,f)$ need to be synthesized. For such purpose, 
we use the DCGAN architecture for image generation \cite{Radford2016}.

\subsection{Griffin-Lim algorithm}
The spectrogram offers an immediate visual representation of time-variant signals. Nevertheless, it only carries magnitude information losing the phase information of the original signal. Therefore, it is not possible to perfectly reconstruct the original samples $x(n)$ from the spectrogram $S(m,f)$ without using phase information. A way to approximate $x(n)$ is the Griffin-Lim algorithm (LSEE-MSTFT) proposed in \cite{GriffinLim}. The idea of LSEE-MSTFT is to iteratively build an estimate of $x(n)$ whose spectrogram converges (in norm $l_2$) to the given one $S(m,f)$. To do so, starting from an initial estimate $x^0(n)$, at each $i$-th step the STFT of the estimation has expression
\begin{equation}
\hat{X}_w^i(m,f)=\sqrt{S(m,f)}\cdot e^{j\angle X_w^i(m,f)}
\end{equation}
and the next estimate $x^{i+1}(n)$ reads as follows
\begin{equation}
x^{i+1}(n)=\frac{\sum_{m=0}^{L-1}{w(n-mH)\cdot \hat{x}_w^i(m,n)}}{\sum_{m=0}^{L-1}{w^2(n-mH)}}.
\end{equation}
By proceeding in such way, $x^{i+1}(n)$ converges to a discrete-time real signal whose spectrogram is $S(m,f)$. When the LSEE-MSTFT algorithm receives as input the generated spectrograms $\hat{S}(m,f)$, it outputs a vector $\hat{x}(n)=x^{i+1}(n)$ corresponding to a generated noise measurement.

\begin{algorithm}
	\caption{Generation of Multi-Conductor Noise}
	\label{alg:noise_gen}
	\begin{algorithmic}[1]
		\Inputs{$p$ multi-conductor noise measurements of $M$ conductors;}
		\Initialize{Parameters for the STFT, DCGAN, LSEE-MSTFT algorithm.}
		\For{$i=1$ to $p$}
		\For{$j=1$ to $M$}
		\State Compute the spectrogram $S_j^{(i)}(m,f)$;
		\EndFor

		\State Create the block-spectrogram $BS^{(i)}(m,f)$;
		\EndFor
		\State Train the DCGAN with the block-spectrogram dataset;
		\State Generate a set of $k$ new block-spectrograms $\hat{BS}(m,f)$;
		\For{$i=1$ to $k$}
		\State Extract the single generated spectrograms $\hat{S}_j^{(i)}(m,f)$;
		\For{$j=1$ to $M$}
		\State Generate the multi-conductor noise vector $\hat{x}^{(i)}_j(n)$ 
		\State using the LSEE-MSTFT algorithm;
		\EndFor
		\EndFor
	\end{algorithmic}
\end{algorithm}

\subsection{Multi-conductor noise generation}
Herein, we extend the methodology presented in the previous paragraph to learn and synthesize the multiple conductor noise statistics.

Let $x_j(n)$ be the synchronized noise measurements collected for $M$ conductors with $j=1,\dots,M$, and let $S_j(m,f)$ be the respective spectrogram. An interesting approach to generate the cross-statistical dependence between the noise in multi-conductors is to concatenate the $M$ spectrograms into a block-spectrogram image $BS(m,f)$, which contains the mutual time-frequency information. To generate new multi-conductor noise measurements, it is enough to generate new block-spectrogram images $\hat{BS}(m,f)$, using for example the DCGAN architecture when $M$ is small, or the Progressive and StyleGANs \cite{karras2018progressive, Karras2019} to deal with higher dimensions in number of samples and number of conductors. Each generated spectrogram $\hat{S}_j(m,f)$ inside $\hat{BS}(m,f)$ takes into account its own statistics but also the cross-statistics, thus the mutual dependence with the other spectrograms. The LSEE-MSTFT algorithm applied for each $\hat{S}_j(m,f)$ returns $M$ generated noise vectors $\hat{x}_j(n)$ with the desired statistics.
The steps aforementioned are summarized in the Alg.~\ref{alg:noise_gen} for multi-conductor noise generation. 

\subsection{Parameter details}
We consider and generate sequences of noise vectors (see Ch. \ref{sec:plc}) of length $N=1024$ samples, corresponding to $1.024$ ms of data when sampled at $f_s=1$ MHz. 
To compute the STFT we filtered the discrete-time signal $x(n)$ with the Blackman-Harris window $w(n)$ of length $L=65$. We choose a stride (hop) $H$ of $15$, resulting in around $75\%$ frame overlap, with $64$ frames and $65$ frequency bins, according to the formulas
\begin{align*}
\text{freq. bins} &= L \nonumber \\
\text{frames} &= 1+\left\lfloor\frac{N-L}{H}\right \rfloor 
\end{align*}
where $\lfloor \cdot \rfloor$ denotes the floor function. We trim the top Nyquist frequency from each STFT $X_w(m,f)$ in order to deal with spectrograms $S(m,f)$ of square dimension $64\times 64$. We replace the Nyquist bin during resynthesis with the mean of the dataset.
After taking the base $10$ logarithm of the spectrogram $S(m,f)$ to better align with human perception, we scaled it to be between $-1$ and $1$ to match the $\text{tanh}$ output non-linearity of the generator network and we kept trace of the de-normalization parameters.
Once the noise dataset has been converted into its image representation, we generated new spectrograms using the DCGAN. The training details for the DCGAN model are reported in Tab.~\ref{tab:DCGAN_parameters}. 
We stopped the training procedure when the statistical metrics of the generated noise, obtained with LSEE-MSTFT, were significative. The number of iteration steps in the estimation process of the LSEE-MSTFT was set to $16$.

We implemented the multi-conductor generation approach in the case of noise measured in $2$ conductors. For each synchronized noise trace the respective spectrogram was built and placed together in the block-spectrogram image of dimension $64\times 128$. We exploited the same DCGAN architecture with the only exception of a first layer with $4096\times 2$ neurons and the first feature maps of dimension $(8,8\times 2,64)$.

\begin{table}
	\centering
	\caption{DCGAN architecture for spectrogram generation.}
	\begin{tabular}{ p{5cm}|p{3cm}|p{3cm}} 
		\toprule
		Operation & Feature maps  		& Activation  \\
		\midrule
		\textbf{Generator $G$} &&   \\ 
		$G(\mathbf{y}):\mathbf{y} \sim \mathcal{U}(-1,1)$ & 100\\ 
		Fully connected & $4096$ & ReLU \\ 
		Reshape and upsampling & $(8,8,64)$ &  \\ 
		Convolution and BatchNorm & $64$ & ReLU \\ 
		Upsampling & &  \\ 
		Convolution and BatchNorm & $32$ & ReLU \\ 
		Upsampling & &  \\ 
		Convolution and BatchNorm & $16$ & ReLU \\ 
		Convolution & $3$ & Tanh \\  \hline
		
		\textbf{Discriminator $D$} & &  \\ 
		$D(\mathbf{x}):\mathbf{x} \sim p_x(\mathbf{x})$ & $(64,64,1)$ \\ 
		Convolution & $32$ & LeakyReLU \\ 
		Dropout & $0.25$ &  \\ 
		Convolution and BatchNorm & $64$ & LeakyReLU \\ 
		Dropout & $0.25$ &  \\ 
		Convolution and BatchNorm & $128$ & LeakyReLU \\ 
		Dropout & $0.25$ &  \\ 
		Convolution and BatchNorm & $256$ & LeakyReLU \\ 
		Dropout & $0.25$ &  \\ 
		Flatten and dense & $1$ & Sigmoid \\   \hline
		
		Batch size & \multicolumn{2}{c}{32} \\
		Number of iterations & \multicolumn{2}{c}{50000} \\ 
		Leaky ReLU slope &  \multicolumn{2}{c}{0.2} \\ 
		Learning rate &  \multicolumn{2}{c}{0.0002}  \\ 
		Optimizer &  \multicolumn{2}{c}{Adam ($\beta_1$ = 0.5, $\beta_2$ = 0.9999)}  \\ \hline
	\end{tabular}
	\label{tab:DCGAN_parameters}
\end{table}

\section{Further ideas and summary}
\sectionmark{StyleGAN-based medium synthesis}
\label{sec:stylegan_synthesis}
In principle, as we proposed to translate the set of measured channels into images, it is possible to adopt the StyleGAN architecture \cite{Karras2019} to significantly improve the channel generation process. Indeed, any CTF can be represented via a time-frequency decomposition: static channels will possess the same values across time, thus images divided into horizontal stripes. Time-variant channels, instead, are easily mapped into images where vertical portions represent a specific spectrum for a given time window (see Fig. \ref{fig:medium_timeinv} and \ref{fig:medium_timevar}).

\begin{figure}[t]
  \centering
	\includegraphics[scale=0.45]{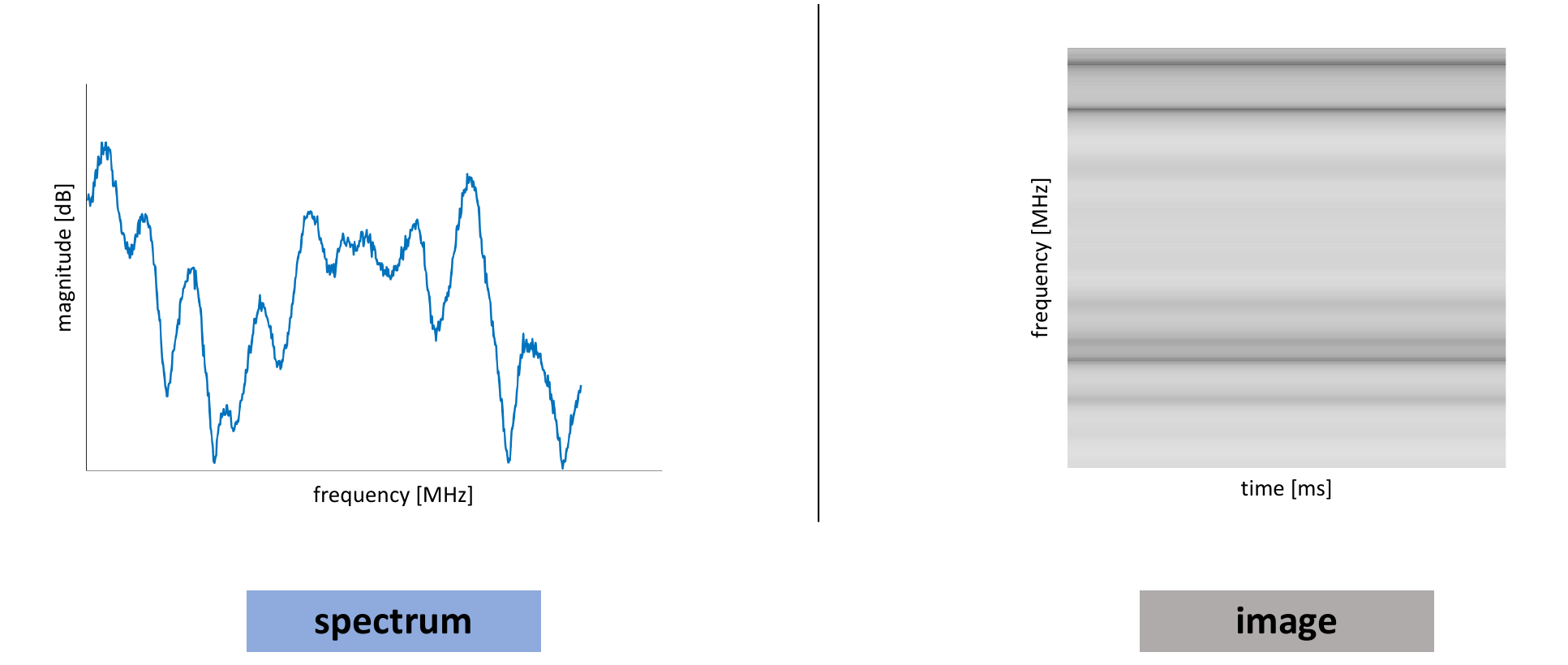}
 	\caption{Time-invariant transformation}
	\label{fig:medium_timeinv}
\end{figure}%
\begin{figure}[t]
  \centering
	\includegraphics[scale=0.45]{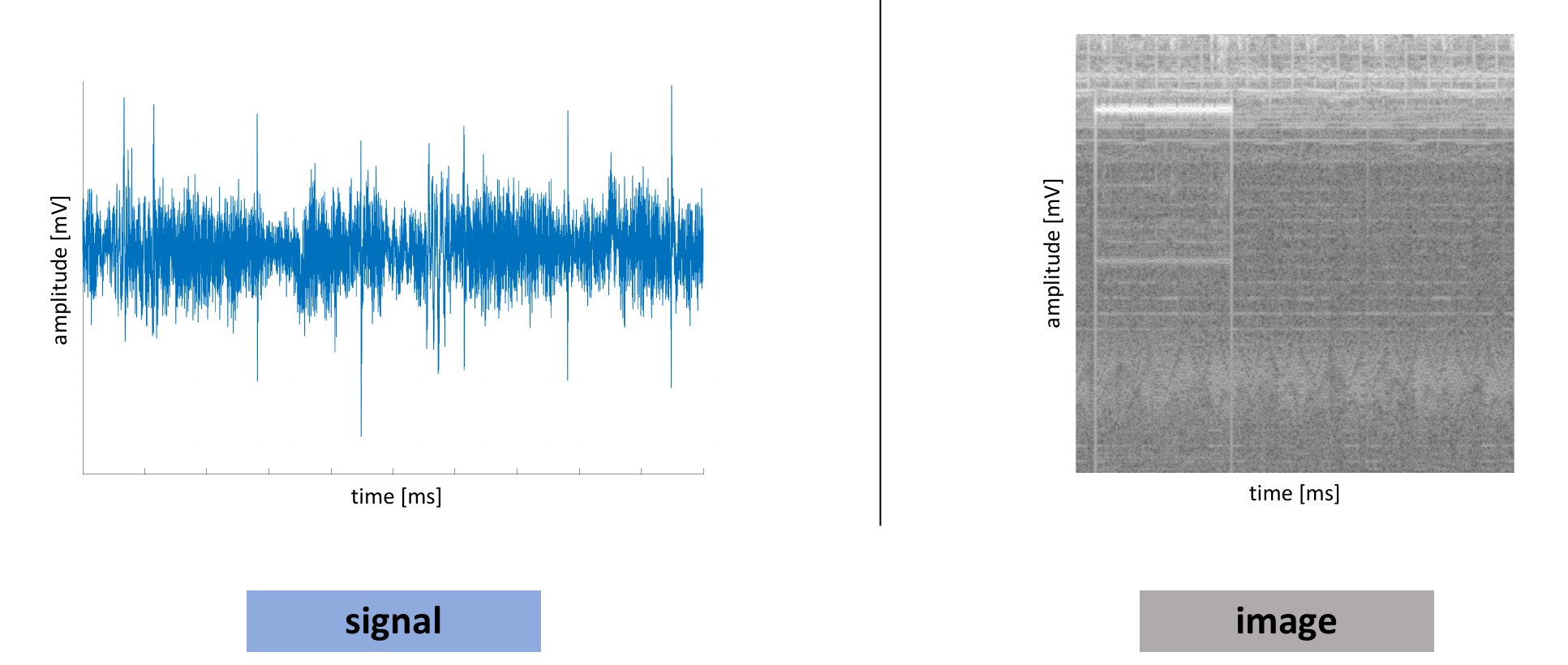}
 	\caption{Time-variant transformation}
	\label{fig:medium_timevar}
\end{figure}%

This chapter has proposed a methodology to model and generate the channel and noise in communication systems via DL techniques. The methodology has been segmented in three phases: a) a pre-processing part where the collected channel or noise traces are represented either as latent features or as in the time-frequency domain as spectrograms, respectively; b) a second generative phase in which GANs produce either new latent features or spectrograms with the same statistics of the original ones; c) a third last phase where either the generated features are decoded into channel samples or the spectrograms are converted to time-domain noise traces using the Griffin-Lim algorithm. 
The generation of multi-conductor noise has been presented as an extension of the single-conductor one using the concept of block-spectrogram, which takes into account the noise spatial cross-dependence.

\chapter{Learning Approaches and Architectures for Neural Decoding} 
\chaptermark{Neural decoders}
\label{sec:decoder}

Physical layer design encompasses the tasks
of designing the channel coder, the modulator and the
associated algorithms at the receiver side, which includes
synchronization, channel estimation, detection/equalization,
interference/noise mitigation, channel decoding.

In this chapter, we consider the detection/decoding problem, where we aim at developing an optimal neural architecture.
The definition of the optimal criterion is a fundamental step. We propose to use the instantaneous mutual information (MI) of the channel input-output signal pair, which yields to the minimization of the a-posteriori information of the transmitted codeword given the communication channel output observation. 
Since the computation of the a-posteriori information is a formidable task, and for the majority of channels it is unknown, we propose a novel neural estimator based on a discriminative formulation. This leads to the derivation of the mutual information neural decoder (MIND). The developed neural architecture is capable not only to solve the decoding problem in unknown channels, but also to return an estimate of the average MI achieved with the coding scheme, as well as the decoding error probability. Several numerical results are reported and compared with maximum a-posteriori and maximum likelihood decoding strategies. 

The results presented in this chapter are documented in \cite{tonello2022mind}.

\section{Introduction}
\label{sec:mind_introduction}
In digital communication systems, data detection and decoding are fundamental tasks implemented at the receiver. The maximum a-posteriori (MAP) decoding approach is the optimal one to minimize the symbol or sequence error probability \cite{Proakis2001,Bahl1974}. If the channel model is known, the MAP decoder can be realized according to known algorithms, among which the BCJR algorithm \cite{Bahl1974}. For instance, for a linear time invariant (LTI) channel with additive Gaussian noise, the decoder comprises a first stage where the channel impulse response is estimated. Then, a max-log-MAP sequence estimator algorithm is implemented. The decoding metric is essentially related to the Euclidean distance between the received and hypothesized transmitted message filtered by the channel impulse response \cite{Proakis2001}. The task becomes more intriguing for channels that cannot be easily modeled or that are even unknown and for which only data samples are available. In such a case, a learning strategy turns out to be an attractive solution and it can be enabled by recent advances in machine learning for communications \cite{Oshea2017, Nachmani2018, Dorner2018}.

In the following, we derive a neural architecture for optimal decoding in an unknown channel. We propose to exploit the MI of the input-output channel pair as decoding metric. The computation of the MI, unknown in many instances, is per-se a challenging task. Therefore, the MI has to be learned, which can be done in principle with neural architectures \cite{Mine2018, Wunder2019, Song2020, f-DIME}. 

However, they are not directly deployable since in the problem at hand the decoder has a more specific task: it has to learn and minimize the a-posteriori information
\begin{equation}
    i(\mathbf{x|y}) = -\log_2{(p_{X|Y}(\mathbf{x}|\mathbf{y}))}
\end{equation}
of the transmitted codeword $\mathbf{x}$ given the observed/received signal vector $\mathbf{y}$. The direct estimation of the a-posteriori information becomes then the objective. Such an estimation is possible by directly learning the conditional PDF $p_{X|Y}(\mathbf{x}|\mathbf{y})$ with a new estimator that exploits a discriminative model referred to as mutual information neural decoder (MIND). MIND allows not only to perform the decoding task, but also to return an estimate of the achieved rate, the MI, and of the decoding error probability. Finally, in general, the same principle can be used to solve supervised \cite{pmlr-v235-novello24a} and weakly-supervised \cite{novello2024label} classification problems.

\section{Maximal mutual information neural decoder}
\sectionmark{Maximal mutual information neural decoder}
\label{sec:mind}

\subsection{Mutual information decoding principle}
\label{subsec:mind_mid}

The considered communication system comprises an encoder, a channel, and a decoder. The encoder maps a message of $k$ bits into one of the $M=2^k$ encoded messages of $n$ (not necessarily binary) symbols. The uncoded message is denoted with the vector $\mathbf{b}$, while the coded message with the vector $\mathbf{x}$ belonging to the set $\mathcal{A}_x$. Thus, the code rate is $R=k/n$ bits per channel use. The channel conveys the coded message to the receiver. Without loss of generality, we can write that the received message is $\mathbf{y}=H(\mathbf{x},\mathbf{h},\mathbf{n})$, where the CTF $H$ implicitly models the channel internal state $\mathbf{h}$ including stochastic variables, e.g. noise, $\mathbf{n}$. For instance, for a linear time-invariant (LTI) channel with additive noise, we obtain $\mathbf{y}=\mathbf{h}*\mathbf{x}+\mathbf{n}$, where $*$ denotes convolution.

The aim of the decoder is to retrieve the transmitted message so that to minimize the decoding error, or equivalently to maximize the received information for each transmitted message. Indeed, if we consider the MI of the pair $(\mathbf{x},\mathbf{y})$, with joint PDF $p_{XY}(\mathbf{x},\mathbf{y})$, and marginals $p_X(\mathbf{x})$ and $p_Y(\mathbf{y})$, this can be written as

\begin{equation}
i(\mathbf{x};\mathbf{y})=i(\mathbf{x})-i(\mathbf{x}|\mathbf{y})
\end{equation}
where $i(\mathbf{x}) = -\log_2 p_{X}(\mathbf{x})$ is the information of the transmitted message and $i(\mathbf{x}|\mathbf{y})$ is the a-posteriori information (residual uncertainty observing $\mathbf{y}$) \cite{Gallager1968}. It follows that for each transmitted message, the instantaneous mutual information is maximized when the a-posteriori information is minimized. The a-posteriori information is given by
\begin{equation}
\label{eq:MIND_PI}
i(\mathbf{x}|\mathbf{y})=-\log_2 p_{X|Y}(\mathbf{x}|\mathbf{y}) = -\log_2{\frac{p_{XY}(\mathbf{x},\mathbf{y})}{p_{Y}(\mathbf{y})}}
\end{equation}
so that the decoding criterion becomes
\begin{equation}
\mathbf{\hat{x}}=\argmin_{\mathbf{x}\in \mathcal{A}_x}{i(\mathbf{x}|\mathbf{y})} = \argmax_{\mathbf{x}\in \mathcal{A}_x}{\log_2 p_{X|Y}(\mathbf{x}|\mathbf{y})}
\end{equation}
which corresponds to the log-MAP decoding principle \cite{Proakis2001,Bahl1974}.

To compute the a-posteriori information, we need to know the a-posteriori probability $p_{X|Y}(\mathbf{x}|\mathbf{y})$. If the CTF can be modeled, the decoder is realized according to known approaches. In fact, the a-posteriori probability is explicitly calculated as $p_{X|Y}(\mathbf{x}|\mathbf{y})=p_{Y|X}(\mathbf{y}|\mathbf{x})p_{X}(\mathbf{x})/p_Y(\mathbf{y})$, i.e., from the conditional channel PDF, the coded message a-priori probability and the channel output PDF. However, for unknown channels and sources, we have to resort to a learning strategy as explained in the following subsection.

\subsection{Discriminative formulation of MIND}
\label{subsec:mind_ndec}

In the following, we show that it is possible to design a parametric discriminator whose output leads to the estimation of the a-posteriori information. The optimal discriminator is found by maximizing an appropriately defined value function.

Before describing the proposed solution, it should be noted that the adversarial training of GANs (see Sec. \ref{sec:gans}) pushes the discriminator output $D(\mathbf{a})$ towards the optimum value $D^*(\mathbf{a}) = p_{data}(\mathbf{a})/(p_{data}(\mathbf{a})+p_{gen}(\mathbf{a}))$, 
where $p_{data}(\mathbf{a})$ is the real data PDF and $p_{gen}(\mathbf{a})$ is the PDF of the data generated by $G$.
Thus, the output of the optimal discriminator $D^*(\mathbf{a})$ can be used to estimate the PDF ratio 
$p_{gen}(\mathbf{a})/p_{data}(\mathbf{a}) = (1-D^*(\mathbf{a}))/D^*(\mathbf{a}).$
 
Now, the idea is to design a discriminator that estimates the probability density ratio $p_{XY}(\mathbf{x},\mathbf{y})/p_{Y}(\mathbf{y})$ and consequently the a-posteriori information as shown in \eqref{eq:MIND_PI}. 
The following Lemma provides an a-posteriori information estimator that exploits a discriminator trained to maximize a certain value function. At the equilibrium, a transformation of the discriminator output yields the searched density ratio. 

\begin{lemma}
\label{lemma:MIND_Lemma1}
Let $\mathbf{x}\sim p_X(\mathbf{x})$ and $\mathbf{y}\sim p_Y(\mathbf{y})$ be the channel input and output vectors, respectively. Let $H(\cdot)$ be the channel transfer function, in general stochastic, such that $\mathbf{y} = H(\mathbf{x}, \mathbf{h}, \mathbf{n} )$, with $\mathbf{h}$ and $\mathbf{n}$ being internal state and noise variables, respectively. 
Let the discriminator $D(\mathbf{x},\mathbf{y})$ be a scalar function of $(\mathbf{x},\mathbf{y})$.
If $\mathcal{J}_{MIND}(D)$ is the value function defined as
\begin{align}
\label{eq:MIND_discriminator_function}
\mathcal{J}_{MIND}(D) &= \; \mathbb{E}_{(\mathbf{x},\mathbf{y}) \sim p_{U}(\mathbf{x})p_{Y}(\mathbf{y})}\biggl[|\mathcal{T}_x|\log \biggl(D\bigl(\mathbf{x},\mathbf{y}\bigr)\biggr)\biggr] \nonumber \\ 
& + \mathbb{E}_{(\mathbf{x},\mathbf{y}) \sim p_{XY}(\mathbf{x},\mathbf{y})}\biggl[\log \biggl(1-D\bigl(\mathbf{x},\mathbf{y}\bigr)\biggr)\biggr],
\end{align}
where $p_U(\mathbf{x})=1/|\mathcal{T}_x|$ describes a multivariate uniform distribution over the support $\mathcal{T}_x$ of $p_X(\mathbf{x})$ having Lebesgue measure $|\mathcal{T}_x|$, then the optimal discriminator output is
\begin{equation}
\label{eq:MIND_optimal_discriminator_1}
D^*(\mathbf{x},\mathbf{y}) = \arg \max_D \mathcal{J}_{MIND}(D) = \frac{p_{Y}(\mathbf{y})}{p_{Y}(\mathbf{y}) + p_{XY}(\mathbf{x},\mathbf{y})},
\end{equation}
and the a-posteriori information is computed as
\begin{equation}
\label{eq:MIND_i-DMIE}
i(\mathbf{x}|\mathbf{y}) = -\log_2 \biggl(\frac{1-D^*(\mathbf{x},\mathbf{y})}{D^*(\mathbf{x},\mathbf{y})}\biggr).
\end{equation}
\end{lemma}

\begin{proof}
From the hypothesis of the Lemma, the value function can be rewritten as
\begin{align}
\label{eq:MIND_Lebesgue1}
\mathcal{J}_{MIND}(D) &= \int_{\mathcal{T}_x} \int_{\mathcal{T}_y}\biggl[|\mathcal{T}_x|p_{U}(\mathbf{x})p_Y(\mathbf{y}) \log \biggl(D(\mathbf{x},\mathbf{y})\biggr) \nonumber \\ 
&+ p_{XY}(\mathbf{x},\mathbf{y}) \log \biggl(1-D(\mathbf{x},\mathbf{y})\biggr)\biggr] \diff \mathbf{x} \diff \mathbf{y}
\end{align}
where $\mathcal{T}_y$ is the support of $p_Y(\mathbf{y})$.
To maximize $\mathcal{J}_{MIND}(D)$, a necessary and sufficient condition requires to take the derivative of the integrand with respect to $D$ and equating it to $0$, yielding the following equation in $D$
\begin{equation}
\frac{|\mathcal{T}_x|p_{U}(\mathbf{x})p_Y(\mathbf{y})}{D(\mathbf{x},\mathbf{y})} -\frac{p_{XY}(\mathbf{x},\mathbf{y})}{1-D(\mathbf{x},\mathbf{y})} =0.
\end{equation}
The solution of the above equation yields the optimum discriminator $D^*(\mathbf{x},\mathbf{y})$ since: $p_U(\mathbf{x})=\frac{1}{|\mathcal{T}_x|}$ and $\mathcal{J}_{MIND}(D^*)$ is a maximum being the second derivative w.r.t. $D$ a non-positive function.

Finally, at the equilibrium, 
\begin{equation}
p_{X|Y}(\mathbf{x}|\mathbf{y})=\frac{p_{XY}(\mathbf{x},\mathbf{y})}{p_{Y}(\mathbf{y})} = \frac{1-D^*(\mathbf{x},\mathbf{y})}{D^*(\mathbf{x},\mathbf{y})},
\end{equation}
therefore the thesis follows.
\end{proof}

\subsection{Parametric implementation}
\label{subsec:mind_implementation}
The practical implementation of MIND is realized through a neural network-based learning approach. That is, the discriminator $D(\mathbf{x},\mathbf{y})$ with input $(\mathbf{x},\mathbf{y})$ is parameterized by a neural network. In a first phase, training of the discriminator is performed via optimization of \eqref{eq:MIND_discriminator_function}.
Training consists of transmitting repeatedly all coded messages such that at equilibrium the a-posteriori information $i(\mathbf{x}|\mathbf{y})$ is estimated. Then, in the testing phase, decoding can take place for each received new message $\mathbf{y}$ by minimizing the a-posteriori information. 
While the testing phase can be of moderate complexity depending on the neural network dimension, the training phase is more complex since no assumption on the channel and source statistics is made. In the following, we propose two different architectures and implementations of Lemma \ref{lemma:MIND_Lemma1}.

\subsubsection{Unsupervised approach}
The value function in \eqref{eq:MIND_discriminator_function}, as it stands, requires the discriminator to be a scalar parameterized function which takes as input both the transmitted and received vector of samples $(\mathbf{x},\mathbf{y})$. The resulting architecture is presented in Fig. \ref{fig:MIND_unsupervised}. Such formulation is general and can be applied to any coding scheme. In fact, to learn the parameters of $D$, it is sufficient during training to alternate the input joint samples $(\mathbf{x},\mathbf{y})$ with marginal samples $(\mathbf{u},\mathbf{y})$, where $\mathbf{u}$ are realizations of a multivariate uniform distribution with independent components, defined over the support $\mathcal{T}_x$ of $p_X(\mathbf{x})$. Then, during the testing phase, decoding is accomplished by finding $\mathbf{x}$ that minimizes \eqref{eq:MIND_i-DMIE}, for all possible coded messages $\mathbf{x}$. This  method can be thought as an unsupervised learning approach since no known labels are exploited in the loss function for the training process. Hence, the objective is to implicitly estimate the density ratio in \eqref{eq:MIND_optimal_discriminator_1}.
Since in practical cases, the coded message $\mathbf{x}$ has discrete alphabet, another architecture is proposed below.

\begin{figure}
	\centering
	\includegraphics[scale=0.63]{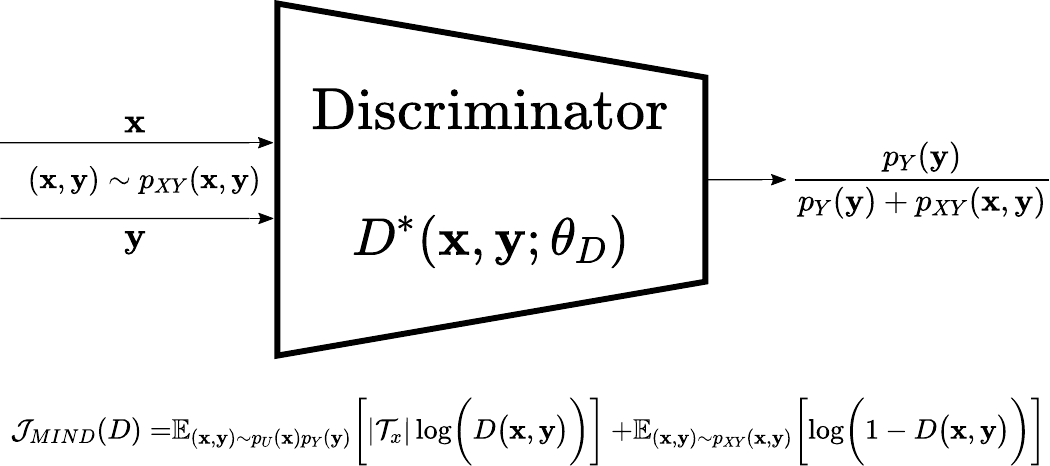}
	\caption{Unsupervised discriminator architecture and relative value function. Input is fed with paired samples from the joint distribution while the output consists of a single probability ratio value.}
	\label{fig:MIND_unsupervised}
\end{figure} 

the single output network, which has to represent a transformation of a probability mass function, tends to collapse the information from previous layers in a unique node and empirically can be verified that it can lead to training problems; lastly, the knowledge of the true label (the transmitted sample) does not explicitly appear in the loss function, thus, it does not guide the training towards optimality. In summary, learning the joint distribution $p_{XY}(\mathbf{x},\mathbf{y})$ may be not convenient for our purpose. All these empirical and heuristic considerations made us rethink about the architecture to use, as explained next. 

\subsubsection{Supervised approach}
It is possible to deploy a discriminator architecture that receives as input only the channel output $\mathbf{y}$. Indeed, if the channel input has discrete nature, e.g., $\mathbf{x}$ belongs to the alphabet $\mathcal{A}_x = \{\mathbf{x}_1, \mathbf{x}_2, \dots, \mathbf{x}_{M}\}$ with PDF 
\begin{equation}
\label{eq:MIND_px}
        p_X(\mathbf{x})=\sum_{\mathbf{x}_i \in \mathcal{A}_x}{P_X(\mathbf{x}_i)\delta(\mathbf{x}-\mathbf{x}_i)},
\end{equation}
then also the multivariate uniform PDF in Lemma \ref{lemma:MIND_Lemma1} reads as
\begin{equation}
\label{eq:MIND_pu}
       p_U(\mathbf{x})=\sum_{\mathbf{x}_i \in \mathcal{A}_x}{P_U(\mathbf{x}_i)\delta(\mathbf{x}-\mathbf{x}_i)}, 
\end{equation}
with 
\begin{equation}
    \label{eq:MIND_1/M}
    P_U(\mathbf{x}_i)=1/M.
\end{equation}
Thus, the value function in \eqref{eq:MIND_discriminator_function}, or its Lebesgue integral expression in \eqref{eq:MIND_Lebesgue1}, can be decomposed as a sum of independent elements for every input codeword $\mathbf{x}_i$, for $i\in \{1,\dots,M\}$, as follows
\begin{equation}
\label{eq:MIND_j_i}
\mathcal{J}_{MIND}(D) = \sum_{i = 1}^{M}{\mathcal{J}_i(D)},
\end{equation}
where 
\begin{align}
\label{eq:MIND_discriminator_function_2}
\mathcal{J}_{i}(D) &= \; \mathbb{E}_{\mathbf{y} \sim p_{Y}(\mathbf{y})}\biggl[\log \biggl(D\bigl(\mathbf{x}_i,\mathbf{y}\bigr)\biggr)\biggr] \\ \nonumber
&+P_{X}(\mathbf{x}_i)\mathbb{E}_{\mathbf{y} \sim p_{Y|X}(\mathbf{y}|\mathbf{x}_i)}\biggl[\log \biggl(1-D\bigl(\mathbf{x}_i,\mathbf{y}\bigr)\biggr)\biggr].
\end{align}
By writing \eqref{eq:MIND_discriminator_function_2} with Lebesgue integrals and following the proof of Lemma \ref{lemma:MIND_Lemma1}, it is simple to verify that the discriminator which maximizes \eqref{eq:MIND_j_i} has to maximize all terms $\mathcal{J}_{i}(D)$, and this happens for 
\begin{equation}
\label{eq:MIND_metric}
D^*(\mathbf{x}_i,\mathbf{y}) = \frac{p_{Y}(\mathbf{y})}{p_{Y}(\mathbf{y}) + P_{X}(\mathbf{x}_i)p_{Y|X}(\mathbf{y}|\mathbf{x}_i)} = \frac{1}{1 + P_{X|Y}(\mathbf{x}_i|\mathbf{y})},
\end{equation}
where $P_{X|Y}(\mathbf{x}_i|\mathbf{y})$ is the PMF of $X|Y$,
having PDF
\begin{equation}
    p_{X|Y}(\mathbf{x}|\mathbf{y})=\sum_{\mathbf{x}_i \in \mathcal{A}_x}{P_{X|Y}(\mathbf{x}_i|\mathbf{y})\delta(\mathbf{x}-\mathbf{x}_i)}.
\end{equation}
More details are offered in Sec. \ref{sec:mind_appendix}.

Now, it is more convenient from an implementation perspective to define the $M$-dimensional vectors
\begin{equation}
\bar{\mathbf{D}}(\mathbf{y}) := 
\begin{bmatrix}
   D(\mathbf{x}_1,\mathbf{y})  \\
   D(\mathbf{x}_2,\mathbf{y}) \\
   \vdots \\
    D(\mathbf{x}_M,\mathbf{y})
\end{bmatrix}
, 
\bar{\mathbf{1}} := 
\begin{bmatrix}
   1  \\
   1 \\
   \vdots \\
    1
\end{bmatrix}
, 
\bar{\mathbf{1}}({\mathbf{x}_i}) := 
\begin{bmatrix}
   0  \\
   \vdots \\
    0 \\
    \underbrace{1}_\text{i-th position} \\ 
    0 \\
    \vdots \\
    0
\end{bmatrix}.
\end{equation}
In fact, the expectations in \eqref{eq:MIND_discriminator_function_2} can be rewritten as
\begin{align}
\label{eq:MIND_disc_vect_1}
    &\mathcal{J}_{MIND}(D) =   \mathbb{E}_{\mathbf{y} \sim p_{Y}(\mathbf{y})}\biggl[\log \bigl(\bar{\mathbf{D}}(\mathbf{y}) \bigr)^T\cdot \bar{\mathbf{1}}_M\biggr] \\ \nonumber 
    &+\sum_{\mathbf{x}_i\in \mathcal{A}_x}{P_X(\mathbf{x}_i) \mathbb{E}_{\mathbf{y} \sim p_{Y|X}(\mathbf{y}|\mathbf{x}_i)}\biggl[\log \bigl(1-\bar{\mathbf{D}}(\mathbf{y})\bigr)^T\cdot \bar{\mathbf{1}}_M({\mathbf{x}_i})\biggr]},
\end{align}
where $\bar{\mathbf{1}}_M$ is a vector of $M$ unitary elements and $\bar{\mathbf{1}}_M({\mathbf{x}_i})$ is the one-hot code of $\mathbf{x}$.
Finally, the value function of the supervised architecture (see Fig. \ref{fig:MIND_supervised}) assumes the following vector form 
\begin{align}
\label{eq:MIND_new_discriminator_function}
 &\mathcal{J}_{MIND}(D) =   \mathbb{E}_{\mathbf{y} \sim p_{Y}(\mathbf{y})}\biggl[\log \bigl(\bar{\mathbf{D}}(\mathbf{y}) \bigr)^T\cdot \bar{\mathbf{1}}_M\biggr] \\ \nonumber
 &+\mathbb{E}_{\mathbf{x} \sim p_{X}(\mathbf{x})} \mathbb{E}_{\mathbf{y} \sim p_{Y|X}(\mathbf{y}|\mathbf{x})}\biggl[\log \bigl(1-\bar{\mathbf{D}}(\mathbf{y})\bigr)^T\cdot \bar{\mathbf{1}}_M(\mathbf{x})\biggr].
\end{align}
The new expression \eqref{eq:MIND_new_discriminator_function} possesses the label information in the scalar product with the one-hot positional code and therefore can be treated as a supervised learning problem. 
\begin{figure}
	\centering
	\includegraphics[scale=0.63]{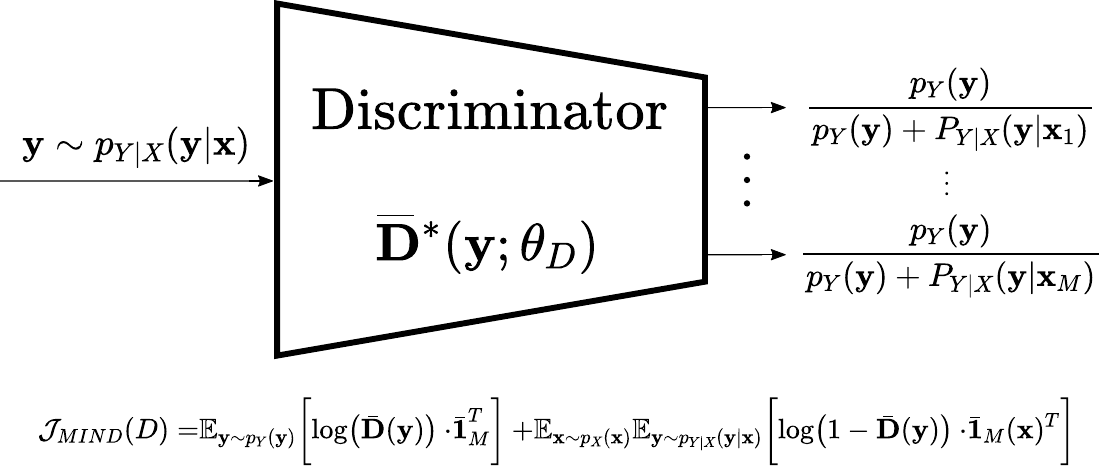}
	\caption{Supervised discriminator architecture and relative value function. Input is fed with the received channel samples while the output consists of multiple probability ratio values.}
	\label{fig:MIND_supervised}
\end{figure} 

By comparing the $M$ discriminator outputs $D^*_i = D^*(\mathbf{x}_i,\mathbf{y})$ in \eqref{eq:MIND_metric}, the final decoding stage implements
\begin{equation}
\mathbf{\hat{x}}_i = \argmax_{i \in \{1,...M\}} \log_2 \frac{1-D^*_i}{D^*_i} = \argmin_{\mathbf{x}_i \in \mathcal{A}_x} i(\mathbf{x}_i|\mathbf{y}),
\end{equation}
thus, MIND aims at finding the codeword $\mathbf{x}_i$ which minimizes the residual uncertainty $i(\mathbf{x}_i|\mathbf{y})$ after the observation of $\mathbf{y}$.

\subsection{Estimation of achieved information rate and decoding error probability}
\label{subsec:mind_estimation}
Dealing with an unknown channel, a relevant question is the estimation of the information rate achieved with the used coding scheme. MIND can be exploited for such a goal. In fact, the normalized average mutual information (in bits per channel use) is given by \cite{Gallager1968}
\begin{equation}
    I_n(X;Y) = \frac{1}{n}(H(X)-H(X|Y))
\end{equation}
whose computation requires the entropy of the source $H(X)$ and the conditional entropy $H(X|Y)$:
\begin{equation}
\label{eq:MIND_source_entropy}
H(X) = \mathbb{E}_{\mathbf{x} \sim p_X(\mathbf{x})}[i(\mathbf{x})] = -\sum_{\mathbf{x}_i \in \mathcal{A}_x}{P_X(\mathbf{x}_i)\log_2 P_X(\mathbf{x}_i)}
\end{equation}
\begin{equation}
\label{eq:MIND_conditional_entropy}
H(X|Y) = -\mathbb{E}_{\mathbf{y} \sim p_Y(y)}\biggl[\sum_{\mathbf{x}_i \in \mathcal{A}_x}{P_{X|Y}(\mathbf{x}_i|\mathbf{y})\log_2 P_{X|Y}(\mathbf{x}_i|\mathbf{y})}\biggr].
\end{equation}
The discriminator in MIND, for a given coding scheme, returns an estimate of the a-posteriori probability mass function values $P_{X|Y}(\mathbf{x}_i|\mathbf{y})=(1-D_i)/D_i \; \forall i=\{1,\dots,M\}$, therefore \eqref{eq:MIND_conditional_entropy} can be directly computed. About \eqref{eq:MIND_source_entropy}, it is worth mentioning that the source distribution can be obtained using Monte Carlo integration from the a-posteriori probability obtained with MIND by averaging over $N$ realizations $\mathbf{y}_j$ of $Y$ as follows 
\begin{equation}
\label{eq:MIND_monte-carlo_source}
P_X(\mathbf{x}_i) \stackrel{N\to \infty}{=} \frac{1}{N}\sum_{j=1}^{N}{P_{X|Y}(\mathbf{x}_i|\mathbf{y}_j)}.
\end{equation}
Similarly,
\begin{equation}
H(X|Y) \stackrel{N\to \infty}{=} -\frac{1}{N} \sum_{j=1}^{N}{\sum_{\mathbf{x}_i \in \mathcal{A}_x}{P_{X|Y}(\mathbf{x}_i|\mathbf{y}_j)\log_2 P_{X|Y}(\mathbf{x}_i|\mathbf{y}_j)}}.
\end{equation}

Finally, the average mutual information (in bits) estimator $I_N$ reads as
\begin{align}
I_N(X;Y) & = \;  \frac{1}{N} \sum_{j=1}^{N}{\sum_{\mathbf{x}_i \in \mathcal{A}_x}{p_{X|Y}(\mathbf{x}_i|\mathbf{y}_j)\log_2 p_{X|Y}(\mathbf{x}_i|\mathbf{y}_j)}}  \nonumber \\ 
    & -\frac{1}{N}\sum_{\mathbf{x}_i \in \mathcal{A}_x}{\sum_{j=1}^{N}{p_{X|Y}(x|\mathbf{y}_j)}\log_2 \biggl(\frac{1}{N} \sum_{j=1}^{N}{p_{X|Y}(x|\mathbf{y}_j)}\biggr)},
\end{align}
and it depends only on the a-posteriori probability estimate provided by the output of the discriminator.

MIND can also be used to estimate the decoding error probability $P_e$. In fact, 
\begin{equation}
    P_e = 1-P[\mathbf{x}=\hat{\mathbf{x}}] = 1-\mathbb{E}_{\mathbf{y} \sim p_Y(y)}[P_{X|Y}(\mathbf{x}=\mathbf{\hat{x}}|\mathbf{y})],
\end{equation}
where the probability $P_{X|Y}(\mathbf{x}=\mathbf{\hat{x}}|\mathbf{y})$ is the instantaneous probability of correct decoding obtained from the decision criterion $\max_{\mathbf{x}_i} P_{X|Y}(\mathbf{x}_i|\mathbf{y})$. Hence,
\begin{equation}
    P_e \stackrel{N\to \infty}{=} 1-\frac{1}{N} \sum_{j=1}^{N}{\max_{\mathbf{x}_i \in \mathcal{A}_x} P_{X|Y}(\mathbf{x}_i|\mathbf{y}_j)}.
\end{equation}

We now compare MIND to other well-known decoding criteria in scenarios for which an analytic solution is available. Furthermore, thanks to the network ability to estimate the a-posteriori probability, the average input-output mutual information is also reported.

\section{Numerical results}
\label{sec:mind_results}
The assessment of the proposed decoding approach is done by considering the following three representative scenarios: a) Uncoded transmission of symbols that are produced by a non-uniform source over an additive Gaussian noise channel; b) Uncoded transmission in a non-linear channel with additive Gaussian noise; c) Short block coded transmission in an additive Middleton noise channel.

Details about the architecture and hyper parameters of the neural networks are reported in the GitHub repository \cite{MIND_github}.

\subsection{Non-uniform source}
To show the ability of the decoder to learn and exploit data dependence at the source, we firstly consider 4-PAM transmission with symbols $\mathbf{x}_i \in \{-3,-1,1,3\}$ and mass function generated by a non-uniform source $p(\mathbf{x}) = [(1-P)/2, P/2, (1-P)/2, P/2]$, with $P=0.05$. It should be noted that data dependence at the source can be either created with a coding stage from a source that emits i.i.d. symbols, or intrinsically by the source of data traffic having certain statistics, as in this case. In Fig. \ref{fig:MIND_non-uniform_source}a the symbol error rate (SER) is shown. MIND is compared with the optimal MAP Genie decoder that knows the a-priori distribution of the transmitted data symbols $p_X(\mathbf{x})$ and with the MaxL decoder that only knows the Gaussian nature of the channel and assumes a source with uniform distribution. Moreover, the estimated $P_e$ is also shown. MIND and the MAP Genie decoder perform identically. In Fig. \ref{fig:MIND_non-uniform_source}b, the source entropy, the conditional entropy and the average achieved MI, as estimated by MIND, are reported and compared with the ground truth numerically obtained. It is worth noticing that the average MI increases with the SNR, or alternatively, the residual uncertainty introduced by the channel decreases. In addition, the estimated source entropy tends to stabilize around its real value given by $-P\log_2(P/2)-(1-P)\log_2((1-P)/2)$.

\begin{figure}
\centering
  \includegraphics[scale=0.25]{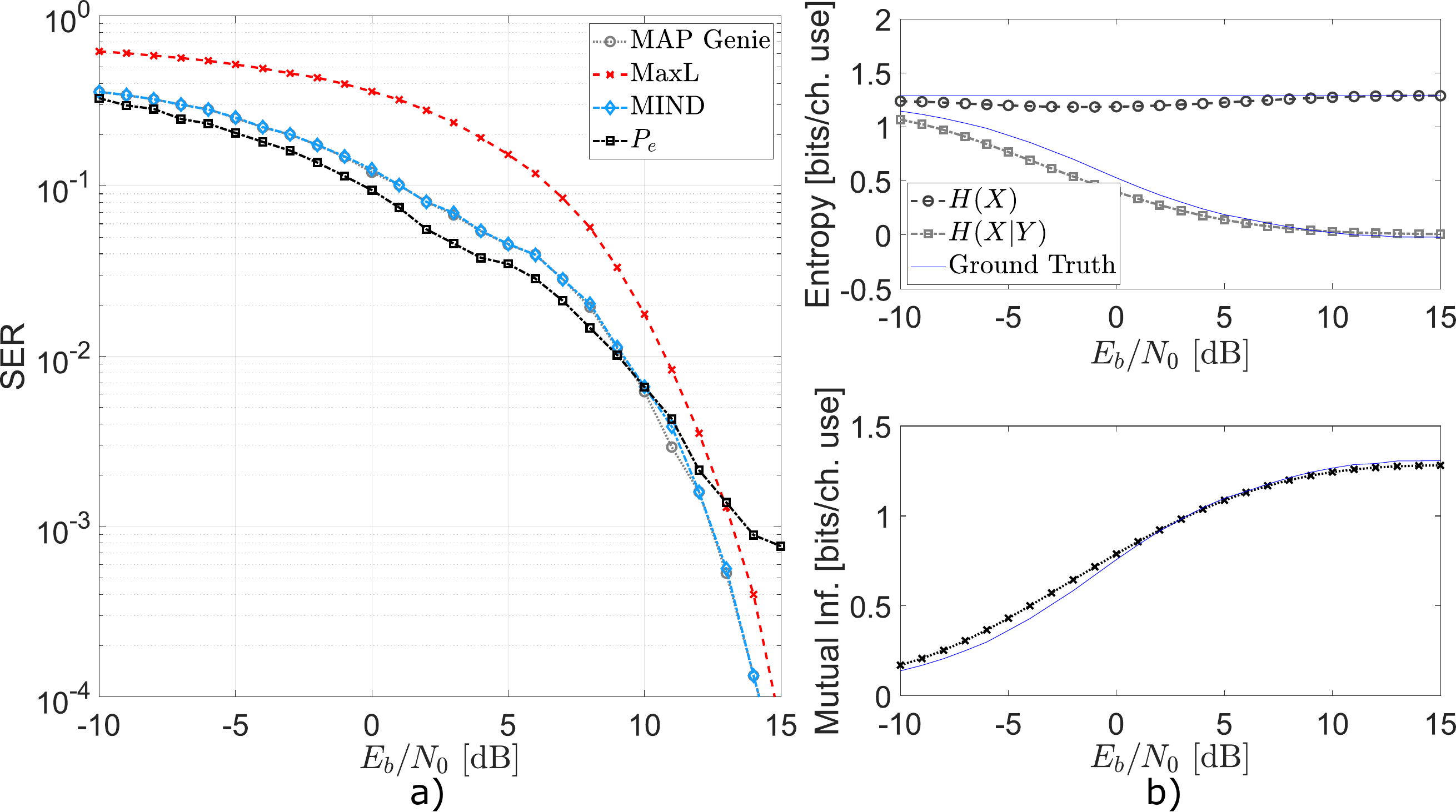}
  \caption{a) Symbol error rate for a 4-PAM modulation with non-uniform source distribution over an AWGN channel. Comparison among the optimal MAP decoder, MaxL decoder, MIND decoder and the estimated probability of error provided by MIND. b) Estimated source and conditional entropy (top right) and estimated average mutual information (bottom right) using MIND.}
	\label{fig:MIND_non-uniform_source}
 \end{figure}

\subsection{Non-linear channel}
As a second example, we consider 4-PAM uncoded transmission with uniform source distribution over a non-linear channel with additive Gaussian noise. The objective of this experiment is to show the ability of MIND to discover such a channel non-linearity. In particular, the channel introduces a non-linearity (for instance because of the presence of non-linear amplifiers) modeled as $y_k=\text{sign}(x_k)\sqrt{|x_k|}+n_k$, where $k$ denotes the $k$-th time instant.
Fig.\ref{fig:MIND_non-linear_channel}a demonstrates how the MIND decoder manages to implicitly learn the non-linear channel model during the training phase and effectively use such information during decoding. A comparison with the MaxL decoder with and without perfect channel state information (CSI) knowledge is also conducted. Results show that MIND exhibits performance identical to the MAP Genie. Fig. \ref{fig:MIND_non-linear_channel}b illustrates both the behaviour of the entropies and the average mutual information as function of the SNR.

\begin{figure}
\centering
   \includegraphics[scale=0.25]{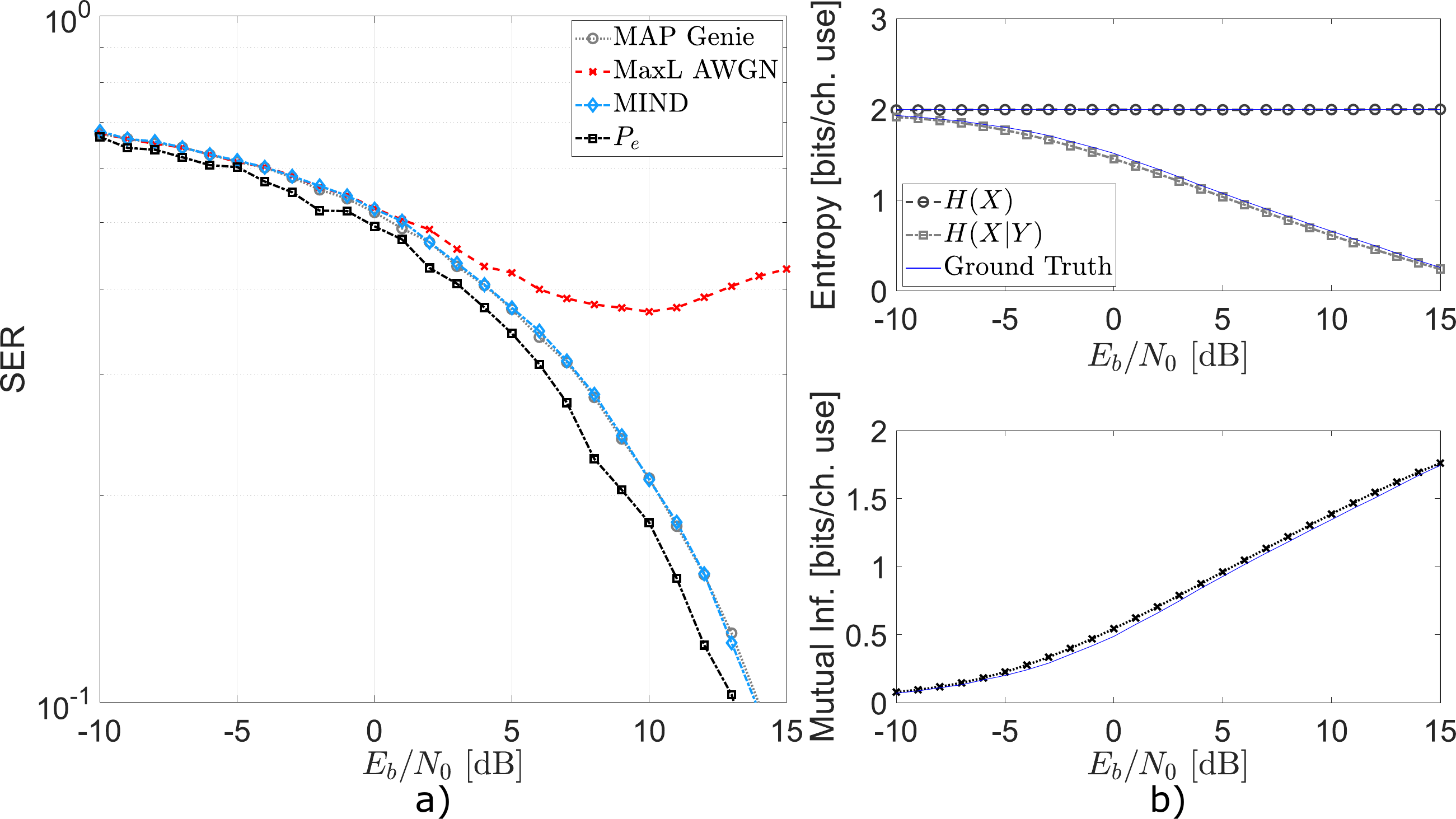}
  \caption{a) Symbol error rate for a 4-PAM modulation with uniform source distribution over a non-linear channel affected by AWGN. Comparison among MaxL decoder with no CSI, MaxL decoder with perfect CSI, MIND decoder and the estimated probability of error. b) Estimated source and conditional entropy (top right) and estimated average mutual information (bottom right) using MIND.}
  \label{fig:MIND_non-linear_channel}
\end{figure}

\subsection{Additive non-gaussian noise channel}
Lastly, and perhaps as most interesting example, we consider a short block coded transmission over an additive non-gaussian channel. The aim is to assess the ability of MIND to learn and exploit the presence of non-gaussian noise. In particular, we suppose that the noise follows a truncated Middleton distribution \cite{Middleton1977}, also called Bernoulli-Gaussian noise model, so that at any given time instant $k$ it is obtained as $n_k = (1-\epsilon_k)n_{1,k}+\epsilon_k n_{2,k}$ where $n_{1,k} \sim \mathcal{N}(0,\sigma_b^2)$ is a zero-mean Gaussian random variable with variance $\sigma_b^2$ and $n_{2,k} \sim \mathcal{N}(0,B\sigma_b^2)$ is also a zero-mean Gaussian random variable but with variance $B$ times larger. Instead, $\epsilon_k$ is a Bernoulli random variable with probability of success $P$. The PDF of the noise samples is then given by
$p_{N}(\mathbf{n}_k) = (1-P)\mathcal{N}(0,\sigma_b^2)+P\mathcal{N}(0,B\sigma_b^2).$
Assuming that the noise model is known and that BPSK symbols are transmitted, two decoding strategies can be devised. The first, denoted with MaxL Middleton, uses maximum likelihood decoding with the known conditional PDF $p(\mathbf{y}|\mathbf{x})=p_{N}(\mathbf{y}-\mathbf{x})$. The second strategy is a genie decoder that knows the outcome of the Bernoulli event for every time instant. That is, it knows whether a received sample is hit by Gaussian noise with variance $\sigma_b^2$ or $B\sigma_b^2$. A third decoding strategy is offered by MIND that learns the channel statistics.
We distinguish among three different types of codes: a) a binary repetition code with length $5$; b) a $(7,4)$ Hamming code; c) a rate $1/2$ convolutional code with memory $2$ and block-length $18$. For each of them, the bit error rate (BER) obtained with the genie, the MaxL Middleton and the MIND decoders is reported. The parameter $B$ was set to $5$ in all the experiments involving Middleton noise, wherein we also set $P=0.05$. 

\begin{figure*}[b]
	\centering
	\includegraphics[scale=0.2]{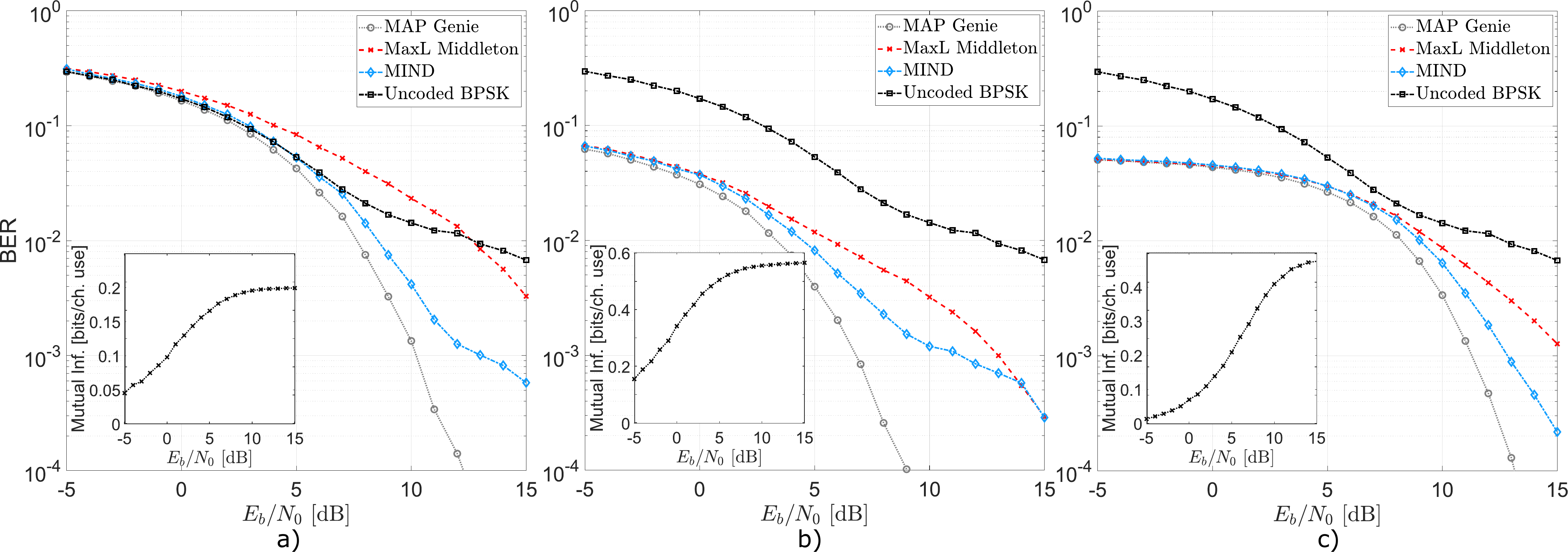}
	\caption{Bit error rate and mutual information of a short block coded transmission in an additive Middleton noise channel: a) Repetition code; b) Hamming code; c) Convolutional code.}
	\label{fig:MIND_middleton}
\end{figure*} 

Fig.~\ref{fig:MIND_middleton} shows the gain provided by the neural-based decoder scheme over the classical maximum likelihood one that exploits the Middleton distribution. With MIND, the BER performance gets closer to the genie decoder. Fig.\ref{fig:MIND_middleton} also reports an estimate of the average mutual information provided by MIND.

\section{Summary}
\label{sec:MIND_conclusions}
In this chapter, MIND, a neural decoder that uses the mutual information as decoding criterion, has been proposed. Two specific architectures have been described and they are capable of learning the a-posteriori information $i(\mathbf{x}|\mathbf{y})$ of the codeword $\mathbf{x}$ given the channel output observation $\mathbf{y}$ in unknown channels. Additionally, MIND allows the estimation of the achieved information rate with the used coding scheme as well as the decoding error probability. Several numerical results obtained in illustrative channels show that MIND can achieve the performance of the genie MAP decoder that perfectly knows the channel model and source distribution. It outperforms the conventional MaxL decoder that assumes the presence of Gaussian noise. 

\newpage
\section{Mathematical derivations}
\label{sec:mind_appendix}
\subsection{From unsupervised to supervised}
Here we provide extra details on the derivation of the supervised loss function from the supervised one. In particular, given the loss
\begin{equation}
\small
\mathcal{J}_{MIND}(D) =  \mathbb{E}_{(\mathbf{x},\mathbf{y}) \sim p_{U}(\mathbf{x})p_{Y}(\mathbf{y})}\biggl[|\mathcal{T}_x|\log \biggl(D\bigl(\mathbf{x},\mathbf{y}\bigr)\biggr)\biggr] + \mathbb{E}_{(\mathbf{x},\mathbf{y}) \sim p_{XY}(\mathbf{x},\mathbf{y})}\biggl[\log \biggl(1-D\bigl(\mathbf{x},\mathbf{y}\bigr)\biggr)\biggr],
\end{equation}
its Lebesgue integral form has expression
\begin{equation}
    \small
    \mathcal{J}_{MIND}(D) = \int_{\mathcal{T}_x} \int_{\mathcal{T}_y}\biggl[p_{U}(\mathbf{x})p_Y(\mathbf{y})|\mathcal{T}_x| \log \biggl(D(\mathbf{x},\mathbf{y})\biggr) + p_{XY}(\mathbf{x},\mathbf{y}) \log \biggl(1-D(\mathbf{x},\mathbf{y})\biggr)\biggr] \diff \mathbf{x} \diff \mathbf{y}.
\end{equation}

Assuming the channel input has discrete nature, 
then, from \eqref{eq:MIND_px}, the Lebesgue integral can be rewritten as
\begin{align}
    \mathcal{J}_{MIND}(D) &= \int_{\mathcal{T}_x} \int_{\mathcal{T}_y}{\sum_{x_i\in \mathcal{A}_x}{P_U(\mathbf{x}_i)\delta(\mathbf{x}-\mathbf{x}_i)}|\mathcal{T}_x|p_Y(\mathbf{y}) \log \biggl(D(\mathbf{x},\mathbf{y})\biggr)} \diff \mathbf{x} \diff \mathbf{y} \nonumber \\ 
    & +\int_{\mathcal{T}_x} \int_{\mathcal{T}_y}{\sum_{\mathbf{x}_i\in \mathcal{A}_x}{P_X(\mathbf{x}_i)\delta(\mathbf{x}-\mathbf{x}_i)} p_{Y|X}(\mathbf{y}|\mathbf{x}) \log \biggl(1-D(\mathbf{x},\mathbf{y})\biggr)} \diff \mathbf{x} \diff \mathbf{y}
\end{align}
where we used \eqref{eq:MIND_pu}.
By exploiting the sifting property of the delta function
\begin{align}
    \mathcal{J}_{MIND}(D) &=  {\sum_{\mathbf{x}_i\in \mathcal{A}_x}{P_U(\mathbf{x}_i)}{\int_{\mathcal{T}_y}|\mathcal{T}_x|p_Y(\mathbf{y}) \log \biggl(D(\mathbf{x}_i,\mathbf{y})\biggr)} \diff \mathbf{y}}\nonumber \\
    & + {\sum_{\mathbf{x}_i\in \mathcal{A}_x}{P_X(\mathbf{x}_i)\int_{\mathcal{T}_y} p_{Y|X}(\mathbf{y}|\mathbf{x}_i) \log \biggl(1-D(\mathbf{x}_i,\mathbf{y})\biggr)} \diff \mathbf{y}}.
\end{align}
From the hypothesis of the MIND loss function in \eqref{eq:MIND_1/M}, 
\begin{align}
    \mathcal{J}_{MIND}(D) &=  {\sum_{\mathbf{x}_i\in \mathcal{A}_x}{\frac{1}{M}\int_{\mathcal{T}_y}|\mathcal{T}_x|p_Y(\mathbf{y}) \log \biggl(D(\mathbf{x}_i,\mathbf{y})\biggr)} \diff \mathbf{y}}\nonumber \\ 
    & + {\sum_{\mathbf{x}_i\in \mathcal{A}_x}{P_X(\mathbf{x}_i)\int_{\mathcal{T}_y} p_{Y|X}(\mathbf{y}|\mathbf{x}_i) \log \biggl(1-D(\mathbf{x}_i,\mathbf{y})\biggr)} \diff \mathbf{y}},
\end{align}
 and since $|\mathcal{T}_x|=M$, it can be simplified into
\begin{align}
    \mathcal{J}_{MIND}(D) =  & \sum_{\mathbf{x}_i\in \mathcal{A}_x}{\biggl[\int_{\mathcal{T}_y}p_Y(\mathbf{y}) \log \biggl(D(\mathbf{x}_i,\mathbf{y})\biggr)\diff \mathbf{y}} \nonumber \\
     & + P_X(\mathbf{x}_i)\int_{\mathcal{T}_y} p_{Y|X}(\mathbf{y}|\mathbf{x}_i) \log \biggl(1-D(\mathbf{x}_i,\mathbf{y})\biggr) \diff \mathbf{y}\biggr],
\end{align}
or alternatively
\begin{equation}
\small
    \mathcal{J}_{MIND}(D) =  \sum_{\mathbf{x}_i\in \mathcal{A}_x}{\biggl[\int_{\mathcal{T}_y}p_Y(\mathbf{y}) \log \biggl(D(\mathbf{x}_i,\mathbf{y})\biggr)\diff \mathbf{y} + \int_{\mathcal{T}_y} P_{XY}(\mathbf{x}_i,\mathbf{y}) \log \biggl(1-D(\mathbf{x}_i,\mathbf{y})\biggr) \diff \mathbf{y}\biggr]},
\end{equation}
where $P_{XY}(\mathbf{x}_i,\mathbf{y}):=P_X(\mathbf{x}_i)\cdot p_{Y|X}(\mathbf{y}|\mathbf{x}_i)$. The above relation is equivalent to
\begin{equation}
\small
    \mathcal{J}_{MIND}(D) =  \sum_{\mathbf{x}_i\in \mathcal{A}_x}{\biggl[\int_{\mathcal{T}_y}p_Y(\mathbf{y}) \log \biggl(D(\mathbf{x}_i,\mathbf{y})\biggr) + P_{XY}(\mathbf{x}_i,\mathbf{y}) \log \biggl(1-D(\mathbf{x}_i,\mathbf{y})\biggr) \diff \mathbf{y}\biggr]}.
\end{equation}

Finally, the MIND loss function can be rewritten with expectations as
\begin{align}
    \mathcal{J}_{MIND}(D) = \sum_{\mathbf{x}_i\in \mathcal{A}_x}\mathcal{J}_{i}(D) =  & \sum_{\mathbf{x}_i\in \mathcal{A}_x}{\mathbb{E}_{\mathbf{y} \sim p_{Y}(\mathbf{y})}\biggl[\log \biggl(D\bigl(\mathbf{x}_i,\mathbf{y}\bigr)\biggr)\biggr]} \nonumber \\ 
    & + P_{X}(\mathbf{x}_i)\mathbb{E}_{\mathbf{y} \sim p_{Y|X}(\mathbf{y}|\mathbf{x}_i)}\biggl[\log \biggl(1-D\bigl(\mathbf{x}_i,\mathbf{y}\bigr)\biggr)\biggr],
\end{align}
or more conveniently as in \eqref{eq:MIND_new_discriminator_function}.

\subsection{Considerations on the expectation theorem}
In the following we present some other useful results that have either been used in the mathematical derivation or can be used to provide alternative proofs to the main results.

For simplicity, we now consider only the one-dimensional case. Let $p_U(u)$ be a uniform probability density function in $[0,1]$, and let $K(\cdot)$ be a function that maps $x$ into $u$. Then
\begin{equation}
   \int_{\mathcal{T}_x}{p_U(K(a))  \log \bigl(D(a)\bigr)} \diff a = \int_{[0,1]}{p_U(b)  \log \bigl(D(K^{-1}(b))\bigr)} \frac{1}{k(K^{-1}(b))} \diff b
\end{equation}
where $k(\cdot)$ is the derivative of $K(\cdot)$. Notice that $p_U(K(x))=1$, but for notation convenience we leave the entire expression. In particular, a valid choice consists of $K(x)=F_X(x)$ with $F_X(\cdot)$ being the cumulative distribution function of $x$. Hence
\begin{equation}
   \int_{\mathcal{T}_x}{p_U(F_X(a))  \log \bigl(D(a)\bigr)} \diff a = \int_{[0,1]}{p_U(b)  \log \bigl(D(F_{X}^{-1}(b))\bigr)} \frac{1}{p_X(F_{X}^{-1}(b))} \diff b
\end{equation}
which rewritten with an expectation reads as
\begin{align}
\label{eq:MIND_w_den}
   \int_{\mathcal{T}_x}{p_U(F_X(a))  \log \bigl(D(a)\bigr)} \diff a = & \; \mathbb{E}_{u \sim p_{U}(u)}\biggl[\log \biggl( \frac{D(F_{X}^{-1}(u))}{\exp{(p_X(F_{X}^{-1}(u))})} \biggr)\biggr] \nonumber \\
   = & \; \mathbb{E}_{x \sim p_{X}(x)}\biggl[\log \biggl( \frac{D(x)}{\exp{(p_X(x)})} \biggr)\biggr].
\end{align}
Notice that the LHS of \eqref{eq:MIND_w_den} can be interpreted (as done in Lemma \ref{lemma:MIND_Lemma1}) as an integral over a uniform distribution with support $\mathcal{T}_x$, i.e.,
\begin{equation}
    \int_{\mathcal{T}_x}{p_U(F_X(a))  \log \bigl(D(a)\bigr)} \diff a = \int_{\mathcal{T}_x}{\log \bigl(D(a)\bigr)} \diff a = \int_{\mathcal{T}_x}{p_{U_{\mathcal{T}}}(a)  |\mathcal{T}_x|\log \bigl(D(a)\bigr)} \diff a. 
\end{equation}
Without considering the term at the denominator, instead,
\begin{equation}
   \int_{[0,1]}{p_U(b)  \log \bigl(D(F_{X}^{-1}(b))\bigr)} \diff b = \int_{\mathcal{T}_x}{p_U(F_X(a))  \log \bigl(D(a)\bigr)p_X(a)} \diff a,
\end{equation}
and using again expectations
\begin{equation}
\small
\label{eq:MIND_wo_den}
   \int_{\mathcal{T}_x}{p_U(F_X(a))  \log \bigl(D(a)\bigr)p_X(a)} \diff a = \mathbb{E}_{u \sim p_{U}(u)}\biggl[\log \biggl( D(F_{X}^{-1}(u)) \biggr)\biggr] = \mathbb{E}_{x \sim p_{X}(x)}\biggl[\log \biggl( D(x) \biggr)\biggr] .
\end{equation}
For the purposes of MIND, we would like to work with the expression in \eqref{eq:MIND_wo_den} in order to really extract the a-posteriori $p_{X|Y}(x|y)$ from $D$. However, from an implementation perspective, the denominator term inside the expectation operator in \eqref{eq:MIND_w_den} is needed, thus, an estimate of the density $p_X(x)$ is also required, which renders the sampling mechanism challenging. Instead, it is simpler to estimate (with Monte Carlo) the expectation in \eqref{eq:MIND_wo_den} but the term $p_X(x)$ will appear inside the integral, leading to a discriminator having as output the term $p_{XY}(x,y)/p(x)p(y)$.
\chapter{End-to-end Optimal Communications Schemes with Autoencoders} 
\chaptermark{Autoencoders for communications}
\label{sec:autoencoders}
The autoencoder (AE) concept has fostered the reinterpretation and the design of modern communication systems. It consists of an encoder, a channel, and a decoder block which modify their internal neural structure in an end-to-end learning fashion.

However, the current approach to train an AE relies on the use of the cross-entropy loss function. This approach can be prone to overfitting issues and often fails to learn an optimal system and signal representation (code). In addition, less is known about the AE ability to design channel capacity-approaching codes, i.e., codes that maximize the input-output MI under a certain power constraint. The task being even more formidable for an unknown channel for which the capacity is unknown and therefore it has to be learnt. 

In this chapter, we address the challenge of designing capacity-approaching codes by incorporating the presence of the communication channel into a novel loss function for the AE training. In particular, we exploit the MI between the transmitted and received signals as a regularization term in the cross-entropy loss function, with the aim of controlling the amount of information stored. By jointly maximizing the MI and minimizing the cross-entropy, we propose a theoretical approach that a) computes an estimate of the channel capacity and b) constructs an optimal coded signal approaching it. 
Theoretical considerations are made on the choice of the cost function and the ability of the proposed architecture to mitigate the overfitting problem.
Simulation results offer an initial evidence of the potentiality of the proposed method.

The results presented in this chapter are documented in \cite{Letizia2021}.

\section{Introduction}
\sectionmark{Introduction}
\label{sec:autoencoders_related}

Communication systems have reached a high degree of performance, meeting demanding requirements in numerous application fields due to the ability to cope with real-world effects exploiting various accomplished physical system models. Reliable transmission in a communication medium has been investigated in the milestone work of C. Shannon \cite{Shannon1948} who suggested to represent the communication system as a chain of fundamental blocks, i.d., the transmitter, the channel and the receiver. Each block is mathematically modeled in a bottom-up fashion, so that the full system results mathematically tractable. 

On the contrary, ML algorithms take advantage of the ability to work with and develop top-down models. In particular, the communication chain has been reinterpreted as an AE-based system \cite{Oshea2017}.
The AE can be trained end-to-end such that the block-error rate (BLER) of the full system is minimized. 
This idea pioneered a number of related works aimed at showing the potentiality of deep learning methods applied to wireless communications \cite{Dorner2018, TurboAE, Alberge2019, OFDM_AE}. In \cite{Dorner2018}, communication over-the-air has been proved possible without the need of any conventional signal processing block, achieving competitive bit error rates w.r.t. to classical approaches. Turbo AEs \cite{TurboAE} reached state-of-the-art performance with capacity-approaching codes at a low SNR. These methods rely on the a-priori channel knowledge (most of the time a Gaussian noise intermediate layer is assumed) and they fail to scale when the channel is unknown. To model the intermediate layers representing the channel, one approach is GANs as explained in Sec. \ref{sec:gan_ch_synthesis}. In this way, the generator implicitly learns the channel distribution $p_Y(y|x)$, resulting in a differentiable network which can be jointly trained in the AE model \cite{OsheaGAN, RighiniLetizia2019, Letizia2019a}. A recent work in \cite{Alberge2019} considered an AWGN channel with additive radar interference and demonstrated the AEs ability to produce optimal constellations in regions where no optimal solutions were available, outperforming standard configurations. 

None of the aforementioned methods explicitly considered the information rate in the cost function. In this direction, the work in \cite{Hoydis2019} included the information rate in the formulation and leveraged AEs to jointly perform geometric and probabilistic signal constellation shaping. Labeling schemes for the learned constellations have been discussed in \cite{Cammerer2020}, where the authors introduced the bit-wise AE. If the channel model is not available, the encoder can be independently trained to learn and maximize the MI between the input and output channel samples, as presented in \cite{Wunder2019}. But therein the decoder is independently designed from the encoder and it does not necessarily grant error-free decoding. Indeed, the decoding stage may not have enough capacity to learn the demapping scheme nor converge during training, especially for large networks. Therefore, the encoder and decoder learning process shall be done \textit{jointly}. In addition, the cost function used to train the AE shall be appropriately chosen. With this goal in mind, let us firstly look into the historical developments and progresses made in the ML field, strictly related to the challenge considered in this section.

The AE was firstly introduced as a non-linear principle component analysis method, exploiting neural networks \cite{Kramer1991}. Indeed, the original network contained an internal bottleneck layer which forced the AE to develop a compact and efficient representation of the input data, in an unsupervised manner.
Several extensions have been further investigated, such as the denoising autoencoder (DAE) \cite{Vincent2008}, trained to reconstruct corrupted input data, the contractive autoencoder (CAE) \cite{Rifai2011}, which attempts to find a simple encoding and decoding function, and the $k$-sparse AE \cite{Makhzani2013} in which only $k$ intermediate nodes are kept active.

However, all of them are particular forms of \textit{regularized} AEs. Regularization is often introduced as a penalty term in the cost function and it discourages complex and extremely detailed models that would poorly generalize on unseen data. In this context, the information bottleneck Lagrangian \cite{Tishby1999} was used as a regularization term to study the sufficiency (fidelity) and minimality (complexity) of the internal representation \cite{Alemi2017,Soatto2018}. 

In the context of communication systems design, regularization in the loss function has not been introduced yet. In addition, the decoding task is usually performed as a classification task. In ML applications, classification is usually carried out by exploiting a final softmax layer together with the categorical cross-entropy loss function. The softmax layer provides a probabilistic interpretation of each possible bits sequence so that the cross-entropy measures the dissimilarity between the reference and the predicted sequence of bits distribution, $p(s)$ and $q(\hat{s})$, respectively.
Nevertheless, training a classifier via cross-entropy suffers from the following problems: firstly, it does not guarantee any optimal latent representation. Secondly, it is prone to overfitting issues, especially in the case of large networks, thus, long codes design. 
Lastly, in the particular case of AEs for communications, the fundamental trade-off between the rate of transmission and reliability, namely, the channel capacity $C$, is not explicitly considered in the learning phase. 

These observations made us rethinking the problem by formulating the two following questions:
\begin{itemize}
\item[a)] Given a power constraint, is it possible to design capacity-approaching codes exploiting the principle of AEs?
\item[b)] Given a power constraint, is it possible to estimate channel capacity with the use of an AE?
\end{itemize}
The two questions are inter-related and the answer of the first one provides an answer to the second one in a constructive way, since if such a code is obtained, then the distribution of the input signal that maximizes the mutual information is also determined, and consequentially the channel capacity can also be obtained.

Inspired by the information bottleneck method \cite{Tishby1999} and by the notion of channel capacity, a novel loss function for AEs in communications is proposed in this chapter. The amount of information stored in the latent representation is controlled by a regularization term estimated using the recently introduced mutual information neural estimator (MINE) \cite{Mine2018}, enabling the theoretical design of nearly optimal codes. 
The influence of the channel appears in the \textit{end-to-end learning} phase in terms of mutual information. 

More specifically, the contributions to this topic are the following:
\begin{itemize}
\item A new loss function is proposed. It enables a new signal constellation shaping method.
\item Channel coding is obtained by \textit{jointly} minimizing the cross-entropy between the input and decoded message, and maximizing the mutual information between the transmitted and received signals.
\item A regularization term $\beta$ controls the amount of information stored in the symbols for a fixed message alphabet dimension $M$ and a fixed rate $R<C$, playing as a trade-off parameter between error-free decoding ability and maximal information transfer via coding. The NN architecture is referred to as \textit{rate-driven autoencoder}.
\item In addition, the label smoothing regularization technique is used during the AE learning process. An entropy description of the predicted messages is discussed and illustrated. 
\item By including the mutual information, we propose a new theoretical iterative scheme to built capacity-approaching codes of length $n$ and rate $R$ and consequently estimate channel capacity. This yields a scheme referred to as \textit{capacity-driven autoencoder}.
\item With the notion of explainable ML in mind, the rationale for the proposed metric and methodology is discussed in more fundamental terms following a) the concept of confidence penalty, b) the information bottleneck method \cite{Tishby1999} and c) by discussing the cross-entropy decomposition.  
\end{itemize} 


\section{Autoencoders for physical layer communications}
\sectionmark{Autoencoders base}
\label{sec:autoencoders_base}
The communication chain can be divided into three fundamental blocks: the transmitter, the channel, and the receiver. The transmitter attempts to communicate a message $s\in \mathcal{M} = \{1,2,\dots,M\}$. To do so, it transmits $n$ complex baseband symbols $\mathbf{x}\in \mathbb{C}^{n}$ at a rate $R=(\log_2 M)/n$ (bits per channel use) over the channel, under a power constraint. In general, the channel modifies $\mathbf{x}$ into a distorted and noisy version $\mathbf{y}$. The receiver takes as input $\mathbf{y}$ and produces an estimate $\hat{s}$ of the original message $s$. From an analytic point of view, the transmitter applies a transformation $f:\mathcal{M} \to \mathbb{C}^{n}$, $\mathbf{x}=f(s)$ where $f$ is referred to as the \textit{encoder}. The channel is described in probabilistic terms by the conditional transition probability density function $p_Y(y|x)$. The receiver, instead, applies an inverse transformation $g: \mathbb{C}^{n} \to \mathcal{M}$, $\hat{s}= g(\mathbf{y})$ where $g$ is referred to as the \textit{decoder}. Such communication scheme can be interpreted as an AE which learns internal robust representations $\mathbf{x}$ of the messages $s$ in order to reconstruct $s$ from the perturbed channel output samples $\mathbf{y}$ \cite{Oshea2017}. 

The AE is a deep NN trained end-to-end using stochastic gradient descent (SGD). The encoder block $f(s;\theta_E)$ maps $s$ into $\mathbf{x}$ and consists of an embedding layer followed by a feedforward NN with parameters $\theta_E$ and a normalization layer to fulfill a given power constraint. The channel is identified with a set of layers; a canonical example is the AWGN channel, a Gaussian noise layer which generates $y_i = x_i+w_i$ with $w_i\sim \mathcal{CN}(0,\sigma^2), i=1,\dots,n$. The decoder block $g(\mathbf{y};\theta_D)$ maps the received channel samples $\mathbf{y}$ into the estimate $\hat{s}$ by building the empirical probability mass function $p_{\hat{S}|Y}(\hat{s}|y;\theta_D)$. It consists of a feedforward NN, with parameters $\theta_D$, followed by a softmax layer which outputs a probability vector of dimension $M$ that assigns a probability to each of the possible $M$ messages. The encoder and decoder parameters $(\theta_E, \theta_D)$ are jointly optimized during the training process with the objective to minimize the categorical cross-entropy loss function
\begin{equation}
\label{eq:AE_CE}
\mathcal{L}(\theta_E, \theta_D) = \mathbb{E}_{(s,y)\sim p_{\hat{S}Y}(\hat{s},y)}[-\log(p_{\hat{S}|Y}(\hat{s}|y;\theta_D))],
\end{equation}
where $\mathbf{y}$ explicitly depends on the encoding block $\mathbf{x} = f(s;\theta_E)$, and thus, on the parameters $\theta_E$, while the decoder $g(\mathbf{y};\theta_D)$ calculates the probability mass function $p_{\hat{S}|Y}(\hat{s}|y;\theta_D)$.
The performance of the AE-based system is typically measured in terms of BER or BLER
\begin{equation}
\label{eq:AE_BLER}
P_e =  {P[\hat{s}\neq s]}.
\end{equation}

\section{Rate-driven autoencoders}
\sectionmark{Rate-driven autoencoders}
\label{sec:cacao}
The cross-entropy loss function does not guarantee any optimality in the code design and it is often prone to overfitting issues \cite{Tishby2017, Soatto2018}. In addition and most importantly, optimal system performance is measured in terms of achievable rates, thus, in terms of MI $I(X;Y)$ between the transmitted $\mathbf{x}$ and the received signals $\mathbf{y}$, as defined in \eqref{eq:fundamentals_MI}.
In communications, the trade-off between the rate of transmission and reliability is expressed in terms of channel capacity. For a memory-less channel, the capacity is defined as
\begin{equation}
\label{eq:AE_capacity}
C = \max_{p_X(x)} I(X;Y),
\end{equation}
where $p_X(x)$ is the input signal probability density function. Finding the channel capacity $C$ is at least as complicated as evaluating the MI. As a direct consequence, building capacity-approaching codes is a formidable task. 

Given a certain power constraint and rate $R$, the AE-based system, that is trained to minimize the cross-entropy loss function, is able, if large enough, to automatically build nearly zero-error codes. Nevertheless, there exists a code at the same rate that exhibits better performance and may even exist a higher rate error-free code (asymptotically). Therefore, the AE does not provide a capacity-achieving code. In other words, conventional autoencoding approaches, through cross-entropy minimization, allow to obtain excellent decoding schemes. Nevertheless, no guarantee to find an optimal encoding scheme is given, especially in deep NNs where problems such as vanishing and exploding gradients occur \cite{279181}. 
Hence, the starting point to design capacity-approaching codes is to redefine the loss function used by the AE.
In detail, we propose to include the MI quantity as a regularization term. The proposed loss function reads as follows

\begin{equation}
\label{eq:AE_CE_MI}
\hat{\mathcal{L}}(\theta_E, \theta_D) = \mathbb{E}_{(s,y)\sim p_{\hat{S}Y}(\hat{s},y)}[-\log(p_{\hat{S}|Y}(\hat{s}|y;\theta_D))]-\beta I(X;Y).
\end{equation}

The loss function in \eqref{eq:AE_CE_MI} forces the AE to jointly modify the network parameters $(\theta_E,\theta_D)$. The decoder reconstructs the original message $s$ with lowest possible error probability $P_e$, while the encoder finds the optimal input signal distribution $p_X(x)$ which maximizes $I(X;Y)$, for a given rate $R$ and code length $n$ and for a certain power constraint. We will denote such type of trained AE as \textit{rate-driven} AE.
It should be noted that such a NN architecture does not necessarily provide an optimal code capacity-wise, since we set a target rate which does not correspond to channel capacity. To solve this second objective, in Sec. \ref{sec:autoencoders_capacity} we will describe a theoretical methodology leading to a new scheme that we name \textit{capacity-driven} autoencoder.

To compute the MI $I(X;Y)$, we can exploit recent results such as MINE \cite{Mine2018}, as discussed below.

\subsection{Mutual information estimation}
\label{subsec:autoencoders_mutual_information_estimation}
The difficulty in computing $I(X;Y)$ resides in its dependence on the joint PDF $p_{XY}(x,y)$, which is usually unknown. Common approaches to estimate the MI rely on binning, density and kernel estimation \cite{Moon1995}, $k$-nearest neighbours \cite{Kraskov2004}, $f$-divergence functionals \cite{Nguyen2010}, and variational lower bounds (VLBs) \cite{Poole2019}. 

Recently, the MINE estimator \cite{Mine2018} proposed a NN-based method to estimate the MI based on the Donsker-Varadhan dual representation \cite{Donsker1983} of the KL divergence, in particular 
\begin{equation}
D_{\text{KL}}(p||q) = \sup_{T:\Omega \to \mathbb{R}} \mathbb{E}_{x\sim p(x)}[T(x)]-\log(\mathbb{E}_{x\sim q(x)}[e^{T(x)}]),
\end{equation}
where the supremum is taken over all functions $T$ such that the expectations are finite. Indeed, by parametrizing a family of functions $T_{\theta} : \mathcal{X}\times \mathcal{Y} \to \mathbb{R}$ with a deep NN with parameters $\theta \in \Theta$, the following bound \cite{Mine2018} holds
\begin{equation}
\label{eq:AE_lower_bound}
I(X;Y)\geq I_\theta(X;Y),
\end{equation}
where $I_{\theta}(X;Y)$ is the neural information measure  defined as
\begin{equation}
\label{eq:AE_MINE}
I_\theta(X;Y) = \sup_{\theta \in \Theta} \mathbb{E}_{(x,y)\sim p_{XY}(x,y)}[T_{\theta}(x,y)] -\log(\mathbb{E}_{(x,y)\sim p_X(x) p_Y(y)}[e^{T_{\theta}(x,y)}]).
\end{equation}
The neural information $I_\theta(X;Y)$ in \eqref{eq:AE_MINE} can be maximized using back-propagation and gradient ascent, leading to a tighter bound in \eqref{eq:AE_lower_bound}. To avoid biased gradients, the authors in \cite{Mine2018} suggested to replace the expectation in the denominator (coming from the derivative of the logarithm) with an exponential moving average. The consistency property of MINE guarantees the convergence of the estimator to the true MI value.

Estimating the MI $I(X;Y)$ is not enough to build capacity-approaching codes for a generic channel. A maximization over all possible input distribution $p_X(x)$ is also required. Therefore, to learn an optimal scheme, at each iteration the encoder needs both the cross-entropy gradient for the decoding phase and the MI gradient, from MINE, for the optimal input signal distribution. 
The proposed loss function in \eqref{eq:AE_CE_MI} (see also Fig. \ref{fig:AE_CE_MI}) shows such double role. In this way, the AE trained with the new loss function intrinsically designs codes for which the MI $I(X;Y)$ is known and maximal by construction, under the aforementioned constraints of power, rate $R$ and code-length $n$.

\begin{figure}
	\centering
	\includegraphics[scale=0.5]{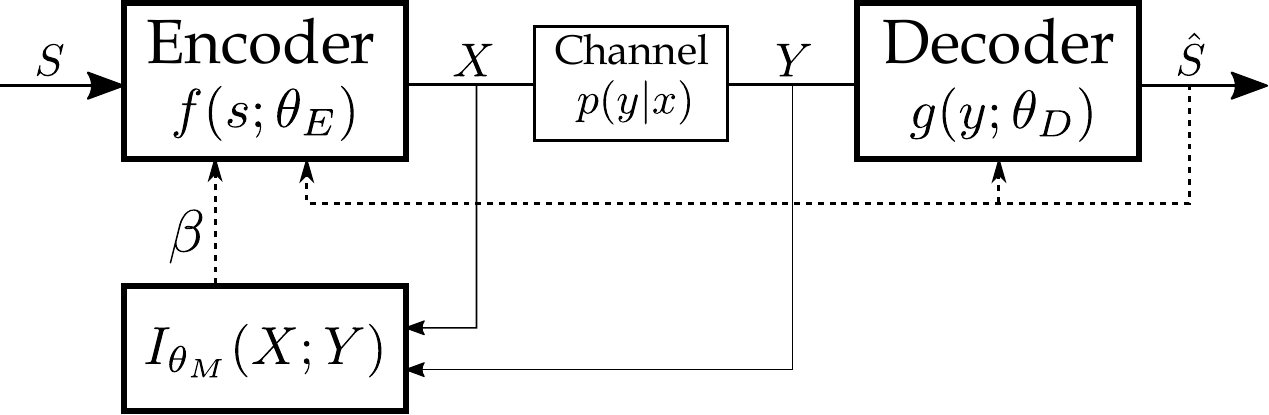}
	\caption{Rate-driven autoencoder with the mutual information estimator block. The channel samples and the decoded message are used during the training process. The former allows to built an optimal encoding scheme exploiting the mutual information block, the latter, instead, allows to measure the decoding error through the cross-entropy loss function.}
	\label{fig:AE_CE_MI}
\end{figure} 

The rationale behind the proposed method is formally discussed in the next section.

\subsection{Mutual information regularization}
\label{sec:autoencoders_MI_reg}
The AE is a classifier network since the decoding block performs a classification task while attempting to recover the sequence of transmitted bits. The training of large AE networks may suffer from known NN degradation issues such as overfitting and vanishing gradients. In the following, we discuss in detail both problems motivating why the addition of the MI regularization term to the cross-entropy loss function, as in \eqref{eq:AE_CE_MI}, and the adoption of the label smoothing technique, offer better performance. Indeed, the cross-entropy can be minimized even for random labels as shown in \cite{Zhang2017}, leading to several overfitting issues. 
To combat overconfident predictions, the MI term plays the role of both an entropy penalty and an information bottleneck regularizer. Furthermore, its gradient directly influences the encoder parameter, mitigating the vanishing gradient problem.
  
\subsubsection{Overfitting mitigation}
\label{subsec:autoencoders_overfitting}
The objective of the AE is twofold: the encoder searches for a latent representation of the transmitted signal that is robust to channel and noise distortions while the decoder learns how to successfully recover the distorted transmitted signal. The latter, in particular, attempts to find a signal representation that provides robust performance in a more general domain than the one given by the observed/training dataset. This is rendered possible by the stochastic description of the channel. Nevertheless, fitting the empirical distribution often leads to overfitting, an undesired effect that appears when the network places most of the probability mass on a subset of the possible output classes. In other words, the decoder block produces a peaky leptokurtic conditional output distribution $p_{\hat{S}|Y}(\hat{s}|y;\theta_D)$ which may improve the training loss at the expenses of the generalization ability of the network. To avoid leptokurtic distributions, a strategy relies on entropy regularization, which penalizes confident output distributions by encouraging more entropic ones. In \cite{Pereyra2017}, the authors introduced the confidence penalty regularization term as an entropy term in the supervised learning setting:
\begin{equation}
H(p_{\hat{S}|Y}(\hat{s}|y;\theta_D)) = \sum_{i}{p_{\hat{S}|Y}(\hat{s_i}|y;\theta_D)\log(p_{\hat{S}|Y}(\hat{s_i}|y;\theta_D))},
\label{eq:AE_confidence}
\end{equation}
where $y$ and $\hat{s}$ are the input and output of the classifier, respectively. However, in the AE set up, $y$ is the output of the channel layer and is dependent itself to the training process. Indeed, the overfitting issue may occur while producing a peaky latent space distribution that does not maximize the entropy of $X$ and consequently, taking in consideration the channel layers, of $Y$. To combat such leptokurtic behavior in the encoding distribution while considering the channel dependence in the output, we propose to penalize it by exploiting the MI between the transmitted and received signals:
\begin{equation}
I(X;Y) = H(X) + H(Y) - H(X,Y),
\end{equation}
where $H(X,Y)$ is the joint entropy measure.
The objective of this penalty regularization term is to promote more entropic signal input-output distributions while reducing the joint entropy, thus, placing the joint probability mass in peaky clusters.  

The MI penalty regularizer directly influences the encoder network parameters during the learning phase. However, as mentioned before while discussing the confidence penalty in \eqref{eq:AE_confidence}, peakiness may appear also in the target output distribution $p_{\hat{S}|Y}(\hat{s}|y;\theta_D)$. To further improve the decoder generalization, we propose also to adopt the label smoothing technique \cite{Szegedy2015b}, a type of entropy regularizer particularly effective when dealing with hard target distributions such as the one-hot vectors of the AE scheme. The idea of label smoothing is to replace the target output distribution (also referred to as label distribution) $p_{\hat{S}|Y}(\hat{s}|y)=\delta_{\hat{S},S}$ with a mixture of the original ground-truth and the chosen distribution $u(\hat{s})$:
\begin{equation}
\hat{p}_{\hat{S}|Y}(\hat{s}|y)=(1-\epsilon)\delta_{\hat{S},S}+\epsilon u(\hat{s}),
\end{equation}
where $u(\hat{s})$ is chosen to be as the uniform distribution $u(\hat{s}) = 1/M$, and $\epsilon$ is a positive number. Label smoothing forces more entropic output messages $\hat{S}$ distributions, improving the decoder generalization ability. The cross-entropy loss function using label smoothing reads as follows
\begin{equation}
\hat{\mathcal{L}}(\theta_E, \theta_D) = \mathbb{E}_{(s,y)\sim \hat{p}_{\hat{S}|Y}(\hat{s}|y)\cdot p_{Y}(y|x)}[-\log(p_{\hat{S}|Y}(\hat{s}|y;\theta_D))].
\end{equation}

In Sec. \ref{subsec:autoencoders_results_overfitting}, we demonstrate how the combination of both the MI regularizer and the label smoothing technique leads to more entropic predicted output distributions.  

Another important benefit of the proposed MI penalty term consists in the mitigation of the vanishing gradient problem.

\subsubsection{Vanishing gradient mitigation}
\label{subsec:autoencoders_vanishing}
It is well known that adding layers with certain activation functions to the NN may result in small gradients of the loss function, inhibiting an effective update of the parameters of the first layers. In the case of large AE networks, the encoder block may suffer from such vanishing gradient problem, especially if the decoder block is itself a deep NN and the channel block is modeled by a multi-layer channel generator obtained via a GAN based training scheme \cite{OsheaGAN}.
Indeed, the generator in a GAN framework is typically a deep network that implicitly estimates the conditional channel distribution $p_Y(y|x)$. For this reason, the back-propagated gradient responsible for the update of the encoder parameters $\theta_E$ may be negligible when only cross-entropy is used as training cost function. Hence, a regularizer whose gradient influences directly the parameters of the encoder block can better guide the network in identifying the latent space. In particular, if the regularization term is the MI, then the encoder also attempts to find the optimal channel input distribution, tackling the achievement of channel capacity, which is exactly the aim of optimal communication schemes.

In detail, given the loss function in \eqref{eq:AE_CE_MI}
\begin{equation}
\hat{\mathcal{L}}(\theta_E, \theta_D) = \mathcal{L}(\theta_E, \theta_D) - \beta I_{\theta_E}(X;Y),
\end{equation}
and given a probabilistic channel generative model $y=h(x; \theta_h)$,
at each training iteration the gradient back-propagated from the decoder network $g(y;\theta_D)$ to the encoder network $f(s;\theta_E)$ can be computed as
\begin{equation}
\nabla_{\theta_E}\hat{\mathcal{L}}(\theta_E, \theta_D) = \nabla_{\theta_E} \mathcal{L}(\theta_E, \theta_D) - \beta \nabla_{\theta_E}I_{\theta_E}(X;Y)
\end{equation}
and using the chain rule
\begin{align}
\label{eq:AE_vanished}
\nabla_{\theta_E}\hat{\mathcal{L}}(\theta_E, \theta_D) = \; & \frac{\partial \mathcal{L}}{\partial g} \cdot \frac{\partial g}{\partial h} \cdot \frac{\partial h}{\partial f} \cdot \nabla_{\theta_E} f(s;\theta_E) \nonumber \\
& - \beta \frac{\partial I}{\partial h} \cdot \frac{\partial h}{\partial f} \cdot \nabla_{\theta_E} f(s;\theta_E) \nonumber \\ 
& - \beta \frac{\partial I}{\partial f} \cdot  \nabla_{\theta_E} f(s;\theta_E).
\end{align}
From the relations in \eqref{eq:AE_vanished}, it is clear that the regularization term with strength $\beta$ can in principle generate a more energetic gradient. Indeed, the first term in the RHS comes from the cross-entropy gradient and it depends on the decoder ($\partial g / \partial h$), while the second and third terms in the RHS contain the gradient of the MI regularizer and do not depend on the decoder.

\subsubsection{Information bottleneck method}
\label{subsec:autoencoders_IB}
In \cite{Soatto2018}, the authors proved how a deep NN can just memorize the dataset (in its weights) to minimize the cross-entropy, yielding to poor generalization. Hence, the authors proposed an information bottleneck (IB) regularization term to prevent overfitting, similarly to the IB Lagrangian, originally presented in \cite{Tishby1999}. Indeed, the IB method optimally compresses the input random variable by eliminating the irrelevant features which do not contribute to the prediction of the output random variable. 

From an AE-based communication systems point of view, let $S \rightarrow X \rightarrow Y$ be a prediction Markov chain, where $S$ represents the message to be sent, $X$ the compressed symbols and $Y$ the received symbols. The IB method solves
\begin{equation}
\label{eq:AE_IB}
\mathcal{L}(p(x|s)) = I(S;X)-\beta I(X;Y),
\end{equation}
where the positive Lagrange multiplier $\beta$ plays as a trade-off parameter between the complexity of the encoding scheme (rate) and the amount of relevant information preserved in it. 

The communication chain adds an extra Markov chain constraint $Y\rightarrow \hat{S}$, where $\hat{S}$ represents the decoded message. Therefore, in order to deal with the full AE chain, we decide to substitute the first term of the RHS in \eqref{eq:AE_IB} with the cross-entropy loss function, as presented in \eqref{eq:AE_CE_MI}. However, the Lagrange multiplier (or regularization parameter in ML terms) operates now as a trade-off parameter between the complexity to reconstruct the original message and the amount of information preserved in its compressed version.

\subsubsection{Cross-entropy decomposition}
\label{subsec:autoencoders_decomposition}
To further motivate the choice for the new loss function with the MI as the regularization term, let us consider the following decomposition of the cross-entropy loss function.

\begin{lemma}\emph{(See \cite{Hoydis2019})}
\label{lemma:autoencoders_lemma1}
Let $s \in \mathcal{M}$ be the transmitted message and let $(x,y)$ be samples drawn from the joint distribution $p_{XY}(x,y)$. If $x = f(s;\theta_E)$ is an invertible function representing the encoder and if $p_{\hat{S}|Y}(\hat{s}|y;\theta_D) = g(y;\theta_D) $ is the decoder block, then the cross-entropy function $\mathcal{L}(\theta_E,\theta_D)$ admits the following decomposition
\begin{equation}
\mathcal{L}(\theta_E,\theta_D) = H(S)-I_{\theta_E}(X;Y)+\mathbb{E}_{y\sim p_Y(y)}[D_{\text{KL}}(p_{X|Y}(x|y)||p_{\hat{S}|Y}(x|y;\theta_D))].
\end{equation}
\end{lemma}
We report a proof of the Lemma \ref{lemma:autoencoders_lemma1}, which was stated in \cite{Hoydis2019}, but herein the proof is complete and slightly different.
\begin{proof}
The cross-entropy loss function can be rewritten as follows
\begin{align}
& \mathcal{L}(\theta_E,\theta_D) = \mathbb{E}_{(x,y)\sim p_{XY}(x,y)}[-\log(p_{\hat{S}|Y}(\hat{s}|y;\theta_D))] \\ \nonumber
&= -\sum_{x,y}{p_{XY}(x,y)\log(p_{\hat{S}|Y}(\hat{s}|y;\theta_D))} \\ \nonumber
&= -\sum_{x,y}{p_{XY}(x,y)\log(p_X(x))} + \\
&+ \sum_{x,y}{p_{XY}(x,y)\log\biggl(\frac{p_X(x)}{p_{\hat{S}|Y}(\hat{s}|y;\theta_D)}\biggr)} \nonumber
\end{align}
Using the encoder hypothesis, the first term in the last expression corresponds to the source entropy $H(S)$. Therefore,
\begin{align*}
&\mathcal{L}(\theta_E,\theta_D) = H(S) + \sum_{x,y}{p_{XY}(x,y)\log\biggl(\frac{p_X(x)\cdot p_Y(y)}{p_{XY}(x,y)}\biggr)}+\\ \nonumber
&+ \sum_{x,y}{p_{XY}(x,y)\log\biggl(\frac{p_{XY}(x,y)}{p_{\hat{S}|Y}(\hat{s}|y;\theta_D)\cdot p_Y(y)}\biggr)} \\ \nonumber
&= H(S)-I(X;Y)+\sum_{x,y}{p_{XY}(x,y)\log\biggl(\frac{p_{X|Y}(x|y)}{p_{\hat{S}|Y}(\hat{s}|y;\theta_D)}\biggr)} \\ \nonumber
&= H(S)-I_{\theta_E}(X;Y)+\mathbb{E}_{y}[D_{\text{KL}}(p_{X|Y}(x|y)|| p_{\hat{S}|Y}(x|y;\theta_D))] \qedhere \nonumber
\end{align*} 
\end{proof}

The cross-entropy decomposition in Lemma \ref{lemma:autoencoders_lemma1} can be read in the following way: the first two terms are responsible for the conditional entropy of the received symbols. In the particular case of a uniform source, only the MI between the transmitted and received symbols is controlled by the encoding function during the training process. On the contrary, the last term measures the error in computing the divergence between the true posterior distribution and the decoder-approximated one. As discussed before, the network could minimize the cross-entropy just by minimizing the KL-divergence, concentrating itself only on the label information (decoding) rather than on the symbol distribution (coding). To avoid this, we propose \eqref{eq:AE_CE_MI} where the parameter $\beta$ helps in balancing the two different contributions as follows
\begin{align}
\label{eq:AE_CE_MI_decomp1}
\mathcal{L}(\theta_E,\theta_M, \theta_D) &= H(S)-I_{\theta_E}(X;Y)-\beta I_{\theta_E, \theta_M}(X;Y)+ \nonumber \\ 
&+\mathbb{E}_{y\sim p_Y(y)}[D_{\text{KL}}(p_{X|Y}(x|y)||p_{\hat{S}|Y}(x|y;\theta_D))].
\end{align}
Moreover, if the MI estimator is consistent, \eqref{eq:AE_CE_MI_decomp1} is equal to
\begin{align}
\label{eq:AE_CE_MI_decomp2}
\mathcal{L}(\theta_E, \theta_D) &= H(S)-(\beta+1)I_{\theta_E}(X;Y)+ \nonumber \\ 
&+\mathbb{E}_{y\sim p_Y(y)}[D_{\text{KL}}(p_{X|Y}(x|y)||p_{\hat{S}|Y}(x|y;\theta_D))].
\end{align}
It is immediate to notice that for $\beta<-1$, the network gets in conflict since it would try to minimize both the mutual information and the KL-divergence. Therefore, optimal values for $\beta$ lie on the semi-line $\beta>-1$.

\section{Capacity-driven autoencoders}
\label{sec:autoencoders_capacity}
Interestingly, the MI block can be exploited to obtain an estimate of the channel capacity. The AE-based system is subject to a power constraint coming from the transmitter hardware and it generally works at a fixed rate $R$ and channel uses $n$. For $R$ and $n$ fixed, the scheme discussed in Fig. \ref{fig:AE_CE_MI} optimally designs the coded signal distribution and provides an estimate $I_{\theta_M}(X;Y)$ of the MI $I(X;Y)$ which approaches $R$. However, a question remains still open: \textit{is the achieved rate with the designed code actually channel capacity}?

To find the channel capacity $C$ and determine the optimal signal distribution $p_X(x)$, a broader search on the coding rate needs to be conducted, relaxing both the constraints on $R$ and $n$. The flexibility on $R$ and $n$ requires to use different AEs. In the following, we denote with $AE(k,n)$ a rate-driven AE-based system that transmits $n$ complex symbols at a rate $R=k/n$, where $k=\log_2(M)$ and $M$ is the number of possible messages. The proposed methodology can be segmented in two phases: 
\begin{enumerate}
\item Training of a rate-driven AE $AE(k,n)$ for a fixed coding rate $R$ and channel uses $n$, enabled via the loss function in \eqref{eq:AE_CE_MI};
\item Adaptation of the coding rate $R$ to build capacity approaching codes $\mathbf{x}\sim p_X(x)$ and consequently find the channel capacity $C$.
\end{enumerate}
We remark that the capacity $C$ is the maximum data rate $R$ that can be conveyed through the channel at an arbitrarily small error probability. Therefore, the proposed algorithm makes an initial guess rate $R_0$ and smoothly increases it by playing on both $k$ and $n$. 

The basic idea is to iteratively train at the $i$-th iteration a pair of rate-driven AEs $AE^{(i)}(k_i,n_j), AE^{(i+1)}(k_{i+1},n_j)$ and evaluate both the MIs $I^{i}_{\theta_M}(X;Y), I^{(i+1)}_{\theta_M}(X;Y)$, at a fixed power constraint. The first AE works at a rate $R_i=k_i/n_j$, while the second one at $R_{i+1}=k_{i+1}/n_j$, with $R_{i+1}>R_i$. If the ratio 
\begin{equation}
\biggl|\frac{I^{(i+1)}_{\theta_M}(X;Y)-I^{(i)}_{\theta_M}(X;Y)}{I^{(i)}_{\theta_M}(X;Y)}\biggr|<\epsilon,
\end{equation}
where $\epsilon$ is an input positive parameter, the code is reaching the capacity limit for the fixed power. If the rate $R_{i+1}$ is not achievable (it does not exist a nearly error-free decoding scheme), a longer code $n_{j+1}$ is required. The algorithm in Tab.\ref{alg:autoencoders_1} describes the pseudocode that implements the channel capacity estimation and capacity-approaching code using as a building block the rate-driven AE.

\begin{algorithm}
\caption{Capacity Learning with Capacity-driven Autoencoders}
\label{alg:autoencoders_1}
\begin{algorithmic}[1]
\Inputs{$L$ SNR increasing values, $\epsilon$ threshold.}
\Initialize{$R_0 = k_0/n_0$ initial rate, $i=0, j = 0$.}
\For{$l=1$ to $L$}
	\State Train $AE^{(0)}(k_0,n_0)$;
	\State Compute $I^{(0)}_{\theta_M}(X;Y)$;
	\While{$\Delta>\epsilon$}
		\State $k_{i+1} = (R_{i}\cdot n_j) +1$;
		\State $R_{i+1}= k_{i+1}/n_j$;
		\State Train $AE^{(i+1)}(k_{i+1},n_j)$;
		\If{$R_{i+1}$ is not achievable}
			\State $n_{j+1}=n_j+1$;
			\State $j= j +1$;
		\Else
			\State Compute $I^{(i+1)}_{\theta_M}(X;Y)$;
			\State Evaluate $\Delta = \biggl|\frac{I^{(i+1)}_{\theta_M}(X;Y)-I^{(i)}_{\theta_M}(X;Y)}{I^{(i)}_{\theta_M}(X;Y)}\biggr|$;
		\EndIf
		\State $i=i+1$;
	\EndWhile
	\State $C_l = I^{(i)}_{\theta_M}(X;Y)$ estimated capacity.
\EndFor
\end{algorithmic}
\end{algorithm}

\subsubsection{Important remarks}
\label{subsec:autoencoders_remarks}
The proposed capacity-driven AE offers a constructive learning methodology to design a coding scheme that approaches capacity and to know what such a capacity is, even for channels that are unknown or for which a closed form expression for capacity does not exist. Indeed, training involves numerical procedures which may introduce some challenges. Firstly, the AE is a NN and it is well known that its performance depends on the training procedure, architecture design and hyper-parameters tuning. Secondly, the MINE block converges to the true MI mostly for a low number of samples. In practice, when $n$ is larger than $4$, the estimation often produces unreliable results, therefore, a further investigation on stable numerical estimators via NNs needs to be conducted. Lastly, the AE fails to scale with high code dimension. Indeed, for large values of $n$, the network could get stuck in local minima or, in the worst scenario, could not provide nearly zero-error codes due to limited resources or not large enough networks. The proposed approach transcends such limitations, although they have to be taken into account in the implementation phase.
In addition, the work follows the direction of explainable ML, in which the learning process is motivated by an information-theoretic approach.
Possible improvements are in defining an even tighter bound in \eqref{eq:AE_lower_bound} and in adopting different network structures (convolutional or recurrent NNs). 

It should be noted that the approach works also for non-linear channels where optimal codes have to be designed under an average power constraint and not for a given operating SNR which is appropriate for linear channels with additive noise.

\section{Numerical results}
\label{sec:autoencoders_results}
In this section, we present results obtained with the rate-driven AEs. They demonstrate an improvement in the decoding schemes (measured in terms of BLER) and show the achieved rates with respect to capacity in channels for which a closed form capacity formulation is known, such as the AWGN channel, and unknown, such as additive uniform noise and Rayleigh fading ones.

The following schemes consider an average power constraint at the transmitter side $\mathbb{E}[|\mathbf{x}|^2] = 1 $, implemented through a batch-normalization layer. Training of the end-to-end AE is performed w.r.t. to the loss function in \eqref{eq:AE_CE_MI}, implemented via a double minimization process since also the MINE block needs to be trained:
\begin{equation}
\label{eq:AE_CE_MI_final}
\min_{\theta_E,\theta_D} \min_{\theta_M} \mathbb{E}_{(s,y)\sim p_{\hat{S}Y}(\hat{s},y)}[-\log(p_{\hat{S}|Y}(\hat{s}|y;\theta_D))]-\beta I_{\theta_M}(X;Y).
\end{equation}

Furthermore, the training procedure was conducted with the same number of iterations for different values of the regularization parameter $\beta$, at a fixed value of $E_b/N_0 = 7$ dB. Unless otherwise specified, we included label smoothing with $\epsilon=0.2$ during the AEs training process.
We used Keras with TensorFlow \cite{TensorFlow} as backend to implement the proposed rate-driven AE. The code has been tested on a Windows-based operating system provided with Python 3.6, TensorFlow 1.13.1, Intel core i7-3820 CPU. 
To allow reproducible results and for clarity, the code is rendered publicly available \cite{CACAO_github}. 

\subsection{Coding-decoding capability}
\begin{figure}
	\centering
	\includegraphics[scale=0.34]{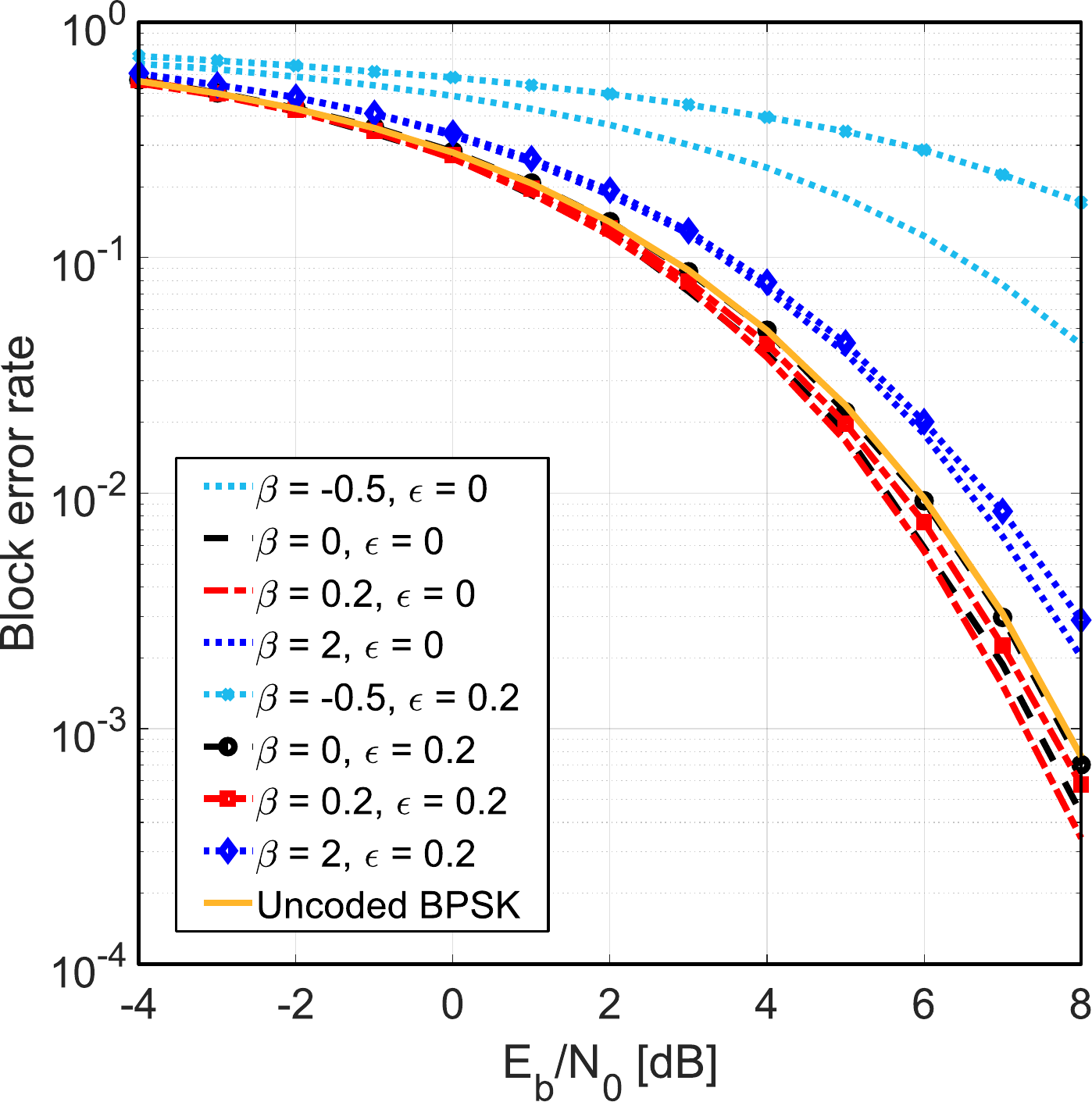}
	\caption{BLER of the rate-driven AE($4,2$) for different values of the regularization $\beta$ and label smoothing parameter $\epsilon$, for an average power constraint, $k=4$ and $n=2$.}
	\label{fig:AE_Ber_4_4}
\end{figure}

As first experiment, we consider a rate-driven AE($4,2$) with rate $R=2$. The advantage of using a MI estimator block during the end-to-end training phase is expected to be more pronounced from $n>1$. To demonstrate the effective influence on the performance of the mutual information term controlled by $\beta$ in \eqref{eq:AE_CE_MI_final}, we investigate $4$ different representative values of the regularization parameter for both cases with and without label smoothing. Fig. \ref{fig:AE_Ber_4_4} illustrates the obtained BLER after the same number of training iterations when the regularization term is added in the cost function. We notice that the lowest BLER is achieved for $\beta=0.2$, therefore as expected, the MI contributes in finding better encoding schemes. Despite the small gain, the result highlights that better BLER can be obtained using the same number of iterations. As shown in \eqref{eq:AE_CE_MI_decomp2}, negative values of $\beta$ tend to force the network to just memorize the dataset, while large positive values create an unbalance. We remark that $\beta=0$ and $\epsilon=0$ corresponds to the classic AE approach proposed in \cite{Oshea2017}. Fig. \ref{fig:AE_Ber_4_4} also illustrates a slightly worse BLER when trained with label smoothing. However, we highlight the fact that label smoothing renders the network more robust to overfitting which does not necessarily reflect into higher accuracy of the model \cite{Muller2019}.
To identify optimal values of $\beta$, a possible approach can try to find the value of $\beta$ for which the two gradients (cross-entropy and mutual information) are equal in magnitude. In the following, we assume $\beta=0.2$.

\begin{figure}
	\centering
	\includegraphics[scale=0.35]{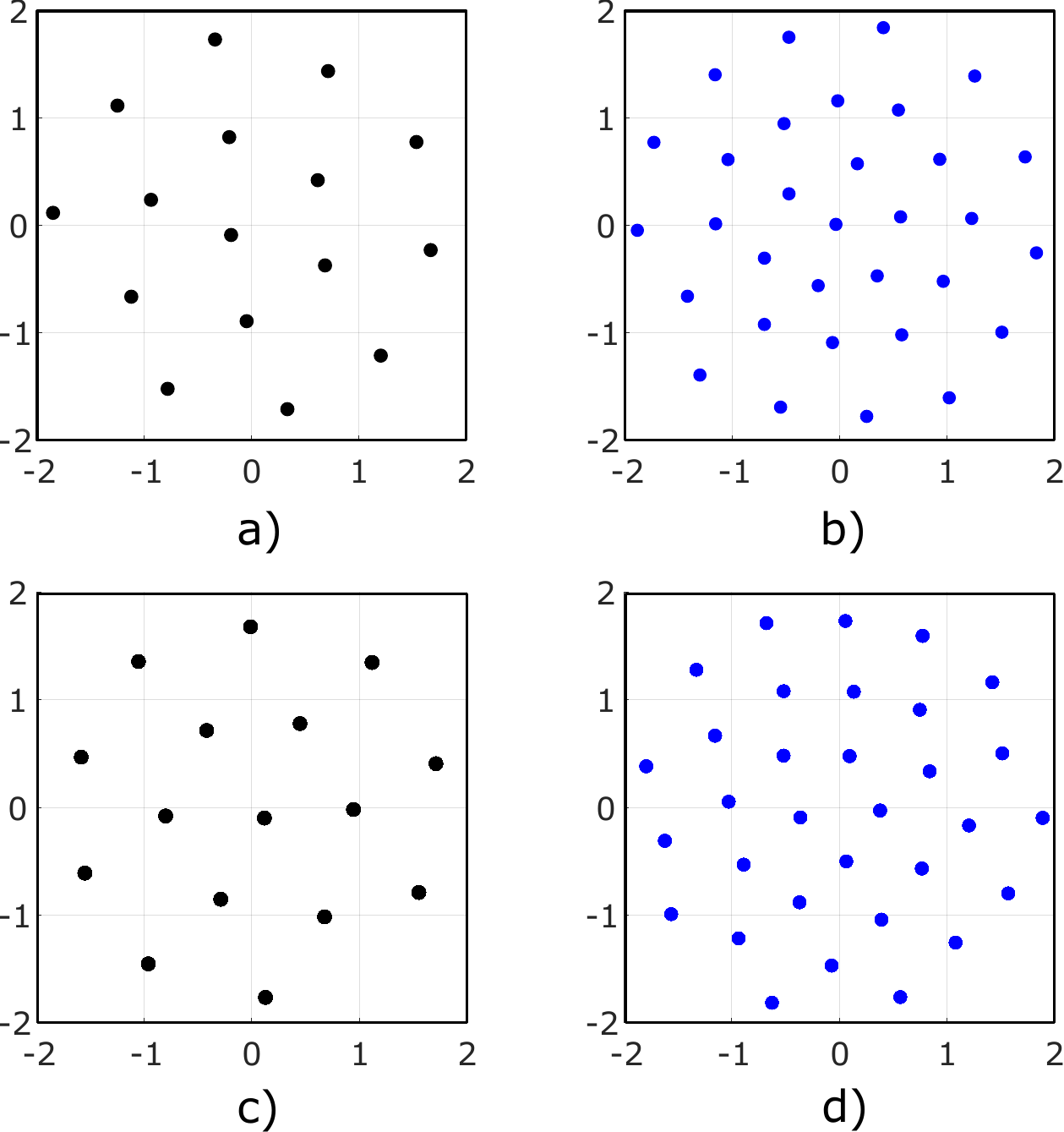}
	\caption{Constellation designed by the encoder during the end-to-end training with $\beta=0.2$ and $\epsilon=0.2$ and parameters $(k,n)$: a) ($4,1$), b) ($5,1$), c) 2 dimensional t-SNE of AE($4,2$) and d)  2 dimensional t-SNE of AE($5,2$).}
	\label{fig:AE_constellation}
\end{figure}

To assess the methodology even with higher dimension $M$ of the input alphabet, we illustrate the optimal constellation schemes when the number of possible messages is $M=16$ and $M=32$. Moreover, two cases are studied, when we transmit one complex symbol ($n=1$) and two dependent complex symbols ($n=2$) over the channel. Fig. \ref{fig:AE_constellation}a and Fig. \ref{fig:AE_constellation}b show the learned hexagonal spiral grid constellations when only one symbol is transmitted for an alphabet dimension of $M=16$ and $M=32$. Fig. \ref{fig:AE_constellation}c and Fig. \ref{fig:AE_constellation}d show, instead, an optimal projection of the coded signals in a 2D space through the learned two-dimensional t-distributed stochastic neighbor embedding
(t-SNE) \cite{vandermaaten2008visualizing}. We notice that the two pairs of constellations are similar, and therefore, even for codes of length $n=2$, the MI pushes the AE to learn the optimal signal constellation. 
As mentioned in the remarks subsection, the advantage is expected to be more pronounced the larger $n$ is. However, there is a practical limitation to obtain constructive results for higher values of $n$ since the MINE estimator is not stable. This can be solved with novel NN architectures and training algorithms, which is one of the challenges we discuss in Sec. \ref{sec:fDIME_autoencoders}.

\subsection{Entropy regularization}
\label{subsec:autoencoders_results_overfitting}
In this section, we evaluate the advantage given by the MI and label smoothing regularization techniques over the solely cross-entropy training method. We analyze the ability to generalize its applicability and the ability to mitigate the overfitting problem of the trained AE. In particular, we propose to study the distribution of the softmax output layer in the simple $4$-QAM AE($2,1$). The softmax output layer attempts to predict the output distribution $p_{\hat{S}|Y}(\hat{s}|y;\theta_D)$ and therefore provides insights about the generalization status of the network, where a smoother platikurtik output distribution is often synonym of generalization beyond the observed data \cite{Pereyra2017} while a peaky output distribution usually stands for poor generalization.

For visualization purposes, we report the analysis of the maximal output distribution $q_{Y}(y) = \max_{\hat{S}} p_{\hat{S}|Y}(\hat{s}|y)$. We consider, as an example, the behavior of the probability density function $q$ for a value of $E_b/N_0=7$ dB in three different training scenarios: only cross-entropy ($\epsilon = 0, \beta=0$), cross-entropy with label smoothing ($\epsilon = 0.2, \beta=0$), cross-entropy with mutual information and label smoothing regularizers ($\epsilon = 0.2, \beta=0.2$). From Fig. \ref{fig:AE_histograms}, we can observe that label smoothing forces a less confident prediction compared to the common cross-entropy. Interestingly, the entropic effect of the MI regularizer, which acts primarily on the encoder block, is present also at the AE output distribution even for a simple autoencoding scheme as the $4$-QAM. Indeed, for $100$ bins, the estimated entropy of $q_Y(y)$ without MI regularization is around $3.66$ Nat, while the entropy of the output with MI is around $3.72$ Nat. 

\begin{figure}
	\centering
	\includegraphics[scale=0.365]{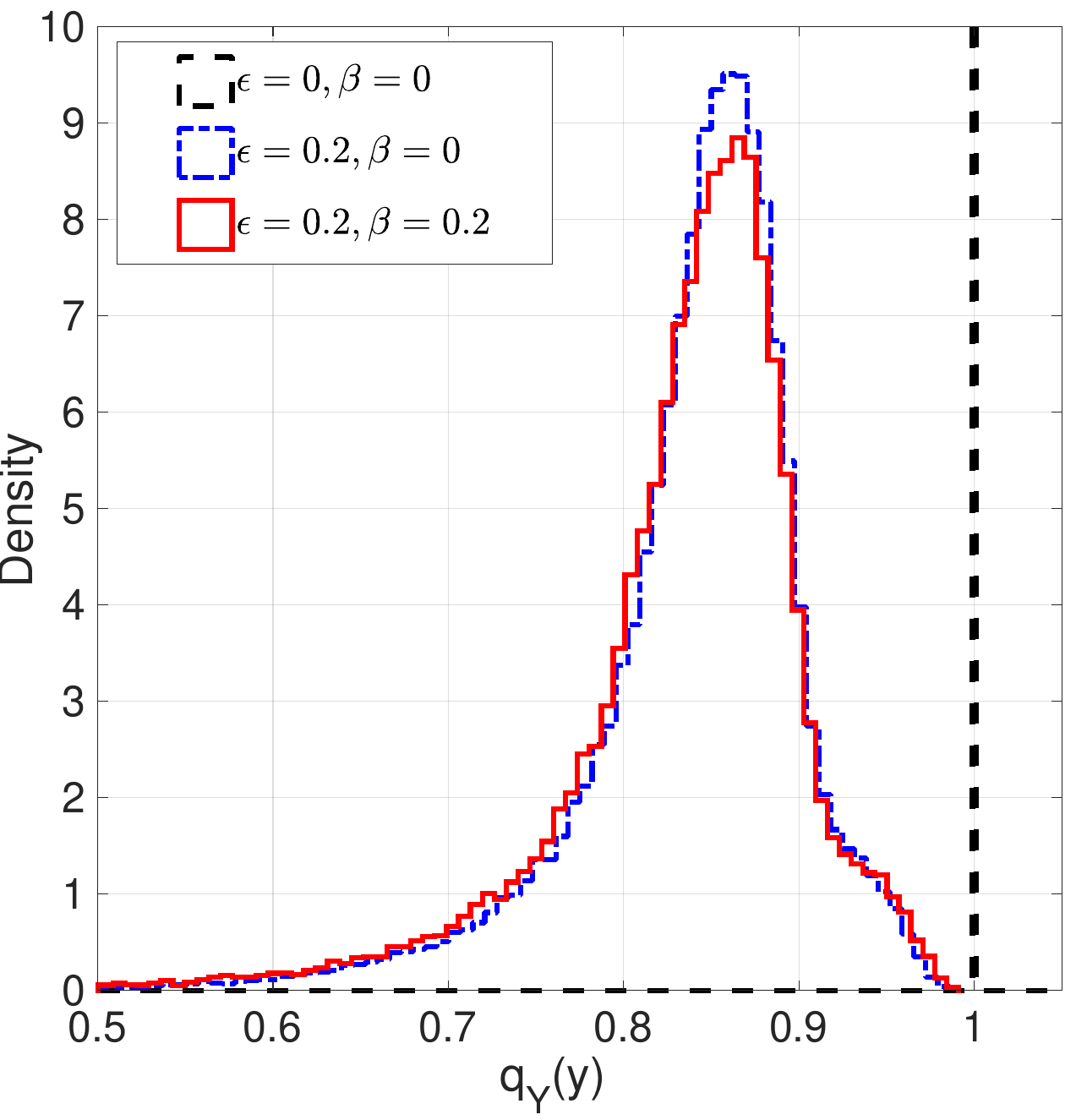}
	\caption{Distribution of the maximal value of the softmax layer of AE($2,1$) for three different scenarios: only cross-entropy training, cross-entropy with label smoothing, cross-entropy with MI and label smoothing.}
	\label{fig:AE_histograms}
\end{figure}

We also report the entropy of $q_Y(y)$ varying the energy per bit to noise ratio $E_b/N_0$ for both cases with and without MI regularization. In both cases, label smoothing for a more stable estimation is used. The network should be capable to predict correct outputs outside the training region, thus, a high entropy is desired. As depicted in Fig. \ref{fig:AE_entropy}, the entropy of the maximal output when regularized with both MI and label smoothing is greater than the entropy of the maximal output when only label smoothing is used, for positive values of $E_b/N_0$. This is coherent with the intuition provided in Sec.\ref{subsec:autoencoders_overfitting}, i.e., the MI term helps in preventing overfitting. An unexpected result comes from the evolution of the estimated entropy varying the energy per bit to noise ratio: the entropy appears to have a local maximum around $0$ dB when the noise power is comparable to the signal's one. Therein, the network is essentially uniformly guessing the predicted message, without any confidence. When the noise power decreases, the network confidently places probability masses following the density in Fig. \ref{fig:AE_histograms}. However, for values of $E_b/N_0 >10$ dB, the entropies of the predictions $q_Y(y)$ start increasing again, and the rate-driven AE maintains a better generalization ability. The point of minimal entropy can be thought as a transition point from continuous (in amplitude) received signals, towards discrete valued signals. These results stimulate further analysis of entropy trends in future research work.  
\begin{figure}
	\centering
	\includegraphics[scale=0.4]{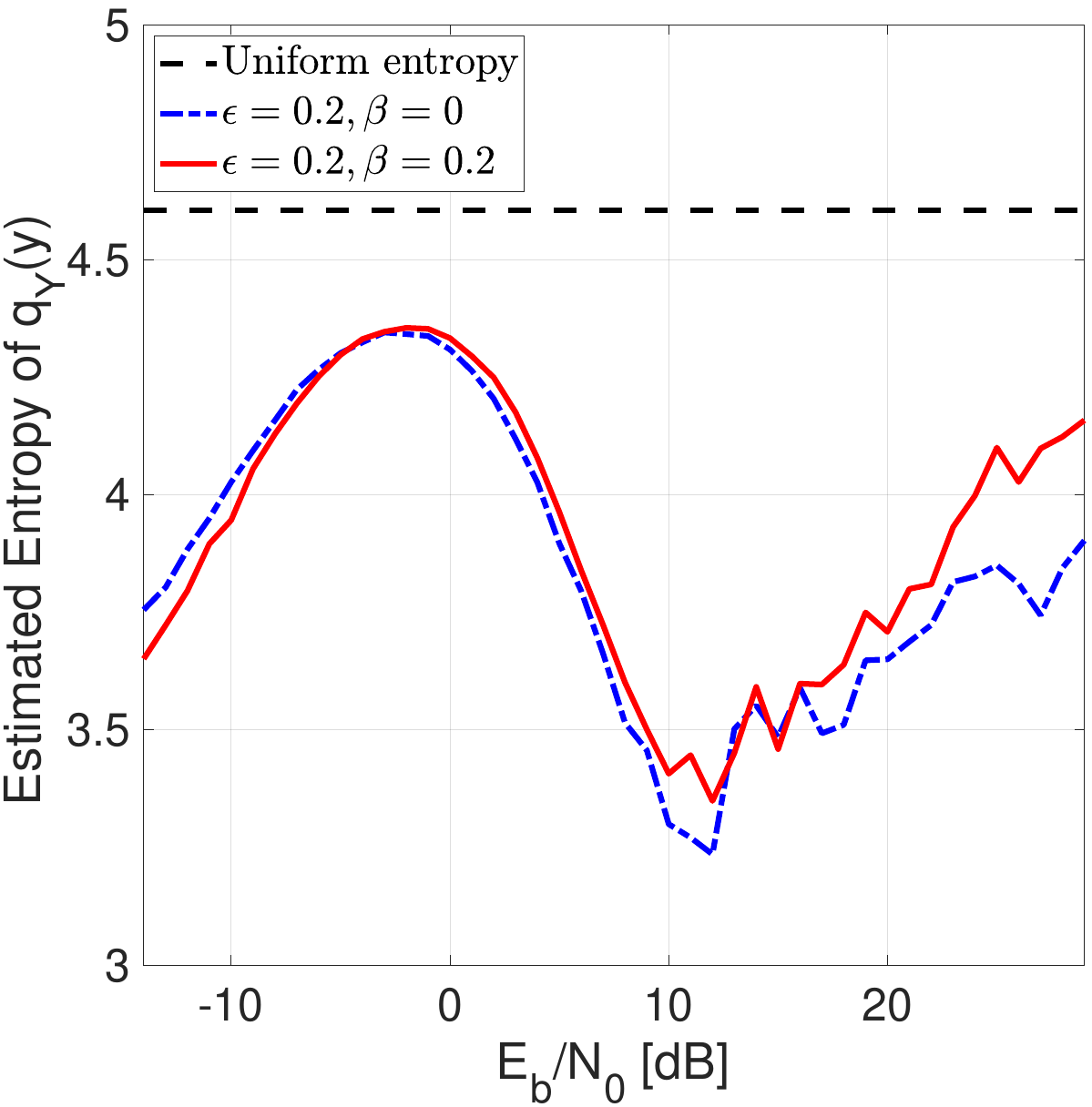}
	\caption{Estimated entropy of $q_Y(y)$ for different values of the energy per bit to noise ratio in the $4$-QAM Gaussian AE($2,1$).}
	\label{fig:AE_entropy}
\end{figure}

\subsection{Capacity-approaching codes over different channels}
The MI block inside the AE can be exploited to design capacity-approaching codes, as discussed in Sec. \ref{sec:autoencoders_capacity}. To show the potentiality of the method, we analyze the achieved rate, e.g. the MI, in three different scenarios. The first one considers the transmission over an AWGN channel, for which we know the exact closed form capacity. The second and third ones, instead, consider the transmission over an additive uniform noise channel and over a Rayleigh fading channel, for which we do not know the capacity in closed form. However, we expect the estimated MI to be a tight lower bound for the real channel capacity, especially at low SNRs.

\subsubsection{AWGN channel}
Let us consider a discrete memory-less channel with input-output relation given by (assuming complex signals)
\begin{equation}
\label{eq:AE_additive_noise}
Y_i=X_i+N_i,
\end{equation}
where the noise samples $N_i\sim \mathcal{CN}(0,\sigma^2)$ are i.i.d. and independent of $X_i$. It is well known that with a power constraint on the input signal $\mathbb{E}[|X_i|^2] \leq P$, the channel capacity is achieved by $X_i\sim \mathcal{CN}(0,P)$ and is equal to
\begin{equation}
C = \log_2(1+\text{SNR}) \; \; \text{[bits/ channel use]}.
\end{equation}

\begin{figure}
	\centering
	\includegraphics[scale=0.34]{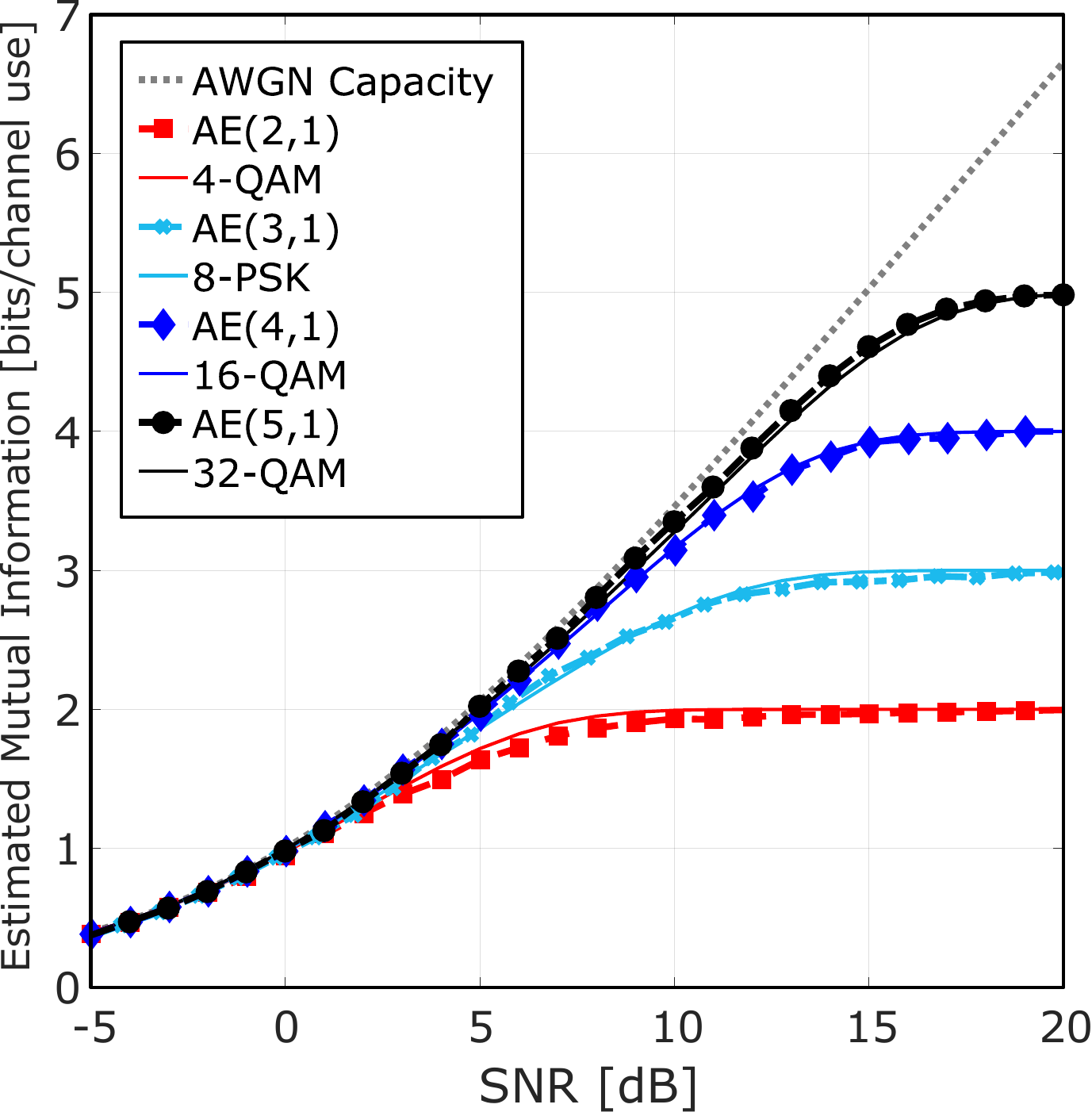}
	\caption{Estimated MI achieved with $\beta=0.2$ and $\epsilon=0.2$ for different dimension of the alphabet $M$ with code length $n=1$ over an AWGN channel.}
	\label{fig:AE_gaussian_capacity}
\end{figure}
The rate-driven AE attempts to maximize the MI during the training progress by modifying, at each iteration, the input distribution $p_X(x)$. Thus, given the input parameters, it produces optimal codes for which the estimation of the mutual information is provided by MINE. Fig. \ref{fig:AE_gaussian_capacity} illustrates the achieved and estimated MI when $\beta=0.2$ and $\epsilon=0.2$ for different values of the alphabet cardinality $M$. A comparison is finally made with established $M$-QAM schemes. We remark that for discrete-input signals with distribution $p_X(x)$, the MI is given by
\begin{equation}
I(X;Y) = \sum_x p_X(x)\cdot \mathbb{E}_{p_Y(y|x)}\biggl[\log \frac{p_Y(y|x)}{p_Y(y)}\biggr],
\end{equation}
and in particular with uniformly distributed symbols (only
geometric shaping), $p_X(x)=1/M$. It is found that the AE constructively provides a good estimate of MI. In addition, for the case $M=32$, the conventional QAM signal constellation is not optimal, since the AE($5,1$) performs geometric signal shaping and finds a constellation that can offer higher information rate as it is visible in the range $13-16$ dB. Lastly, if we code over two channel uses, i.e., $n=2$, an improvement in MI can be attained. This is shown in Fig. \ref{fig:AE_MI_all} where a comparison between channels affected by AWGN, uniform noise and Rayleigh fading is reported.

\subsubsection{Additive uniform noise channel}
\begin{figure}
	\centering
	\includegraphics[scale=0.34]{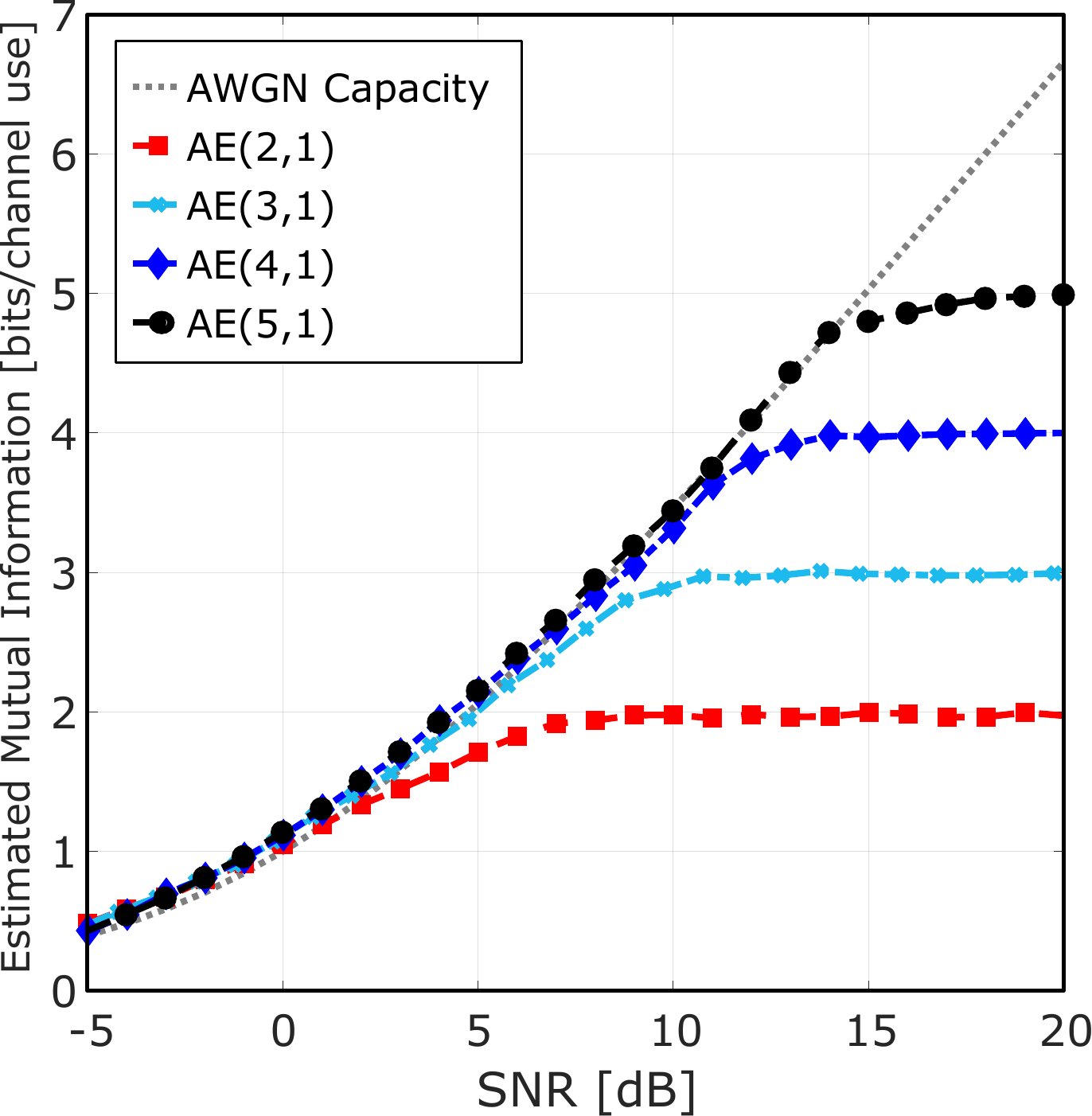}
	\caption{Estimated MI achieved with $\beta=0.2$ and $\epsilon=0.2$ for different dimension of the alphabet $M$ with code $n=1$ over an additive uniform noise channel.}
	\label{fig:AE_additive_capacity}
\end{figure}

\begin{figure}
	\centering
	\includegraphics[scale=0.20]{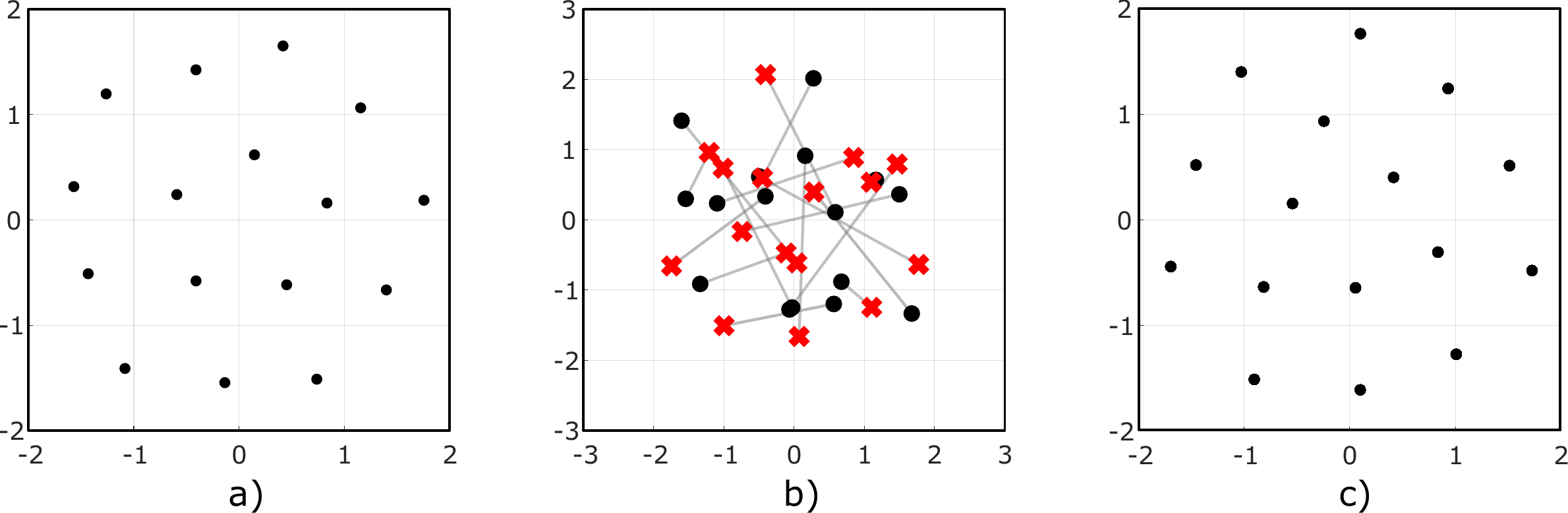}
	\caption{Constellation designed by the encoder during the end-to-end training with $\beta=0.2$ and $\epsilon=0.2$, uniform noise layer and parameters $(k,n)$: a) ($4,1$), b) pairs of transmitted coded symbols AE($4,2$), c) 2 dimensional t-SNE of AE($4,2$).}
	\label{fig:AE_constellation_unif}
\end{figure}

No closed form capacity expression is known when the noise $N$ has uniform distribution $N\sim \mathcal{U}(-\frac{\Delta}{2},\frac{\Delta}{2})$ under an average power constraint. However, Shannon proved that the AWGN capacity is the lowest among all additive noise channels of the form \eqref{eq:AE_additive_noise}. Consistently, as depicted in Fig. \ref{fig:AE_additive_capacity}, it is rather interesting to notice that the estimated MI for the uniform noise channel is higher than the AWGN capacity for low SNRs until it saturates to the coding rate $R$. 
Moreover, differently from the AWGN coding signal set, Fig. \ref{fig:AE_MI_all}b also considers the complex signals generated by the encoder over two channel uses. As expected, the MI achieved by the code produced with the AE($4,2$) is higher than with AE($2,1$), consistently with the idea that $n>1$ introduces a temporal dependence in the code allowing to improve the decoding phase. In addition, Fig. \ref{fig:AE_constellation_unif}a illustrates the constellation produced by the rate-driven AE($4,1$) in the uniform noise case. Fig. \ref{fig:AE_constellation_unif}b shows how the transmitted coded symbols (transmitted complex coefficients) vary for different channel uses while Fig. \ref{fig:AE_constellation_unif}c, instead, displays the learned two-dimensional t-SNE constellation of the code produced by the AE($4,2$). 

\subsubsection{Rayleigh channel}
\begin{figure}
	\centering
	\includegraphics[scale=0.34]{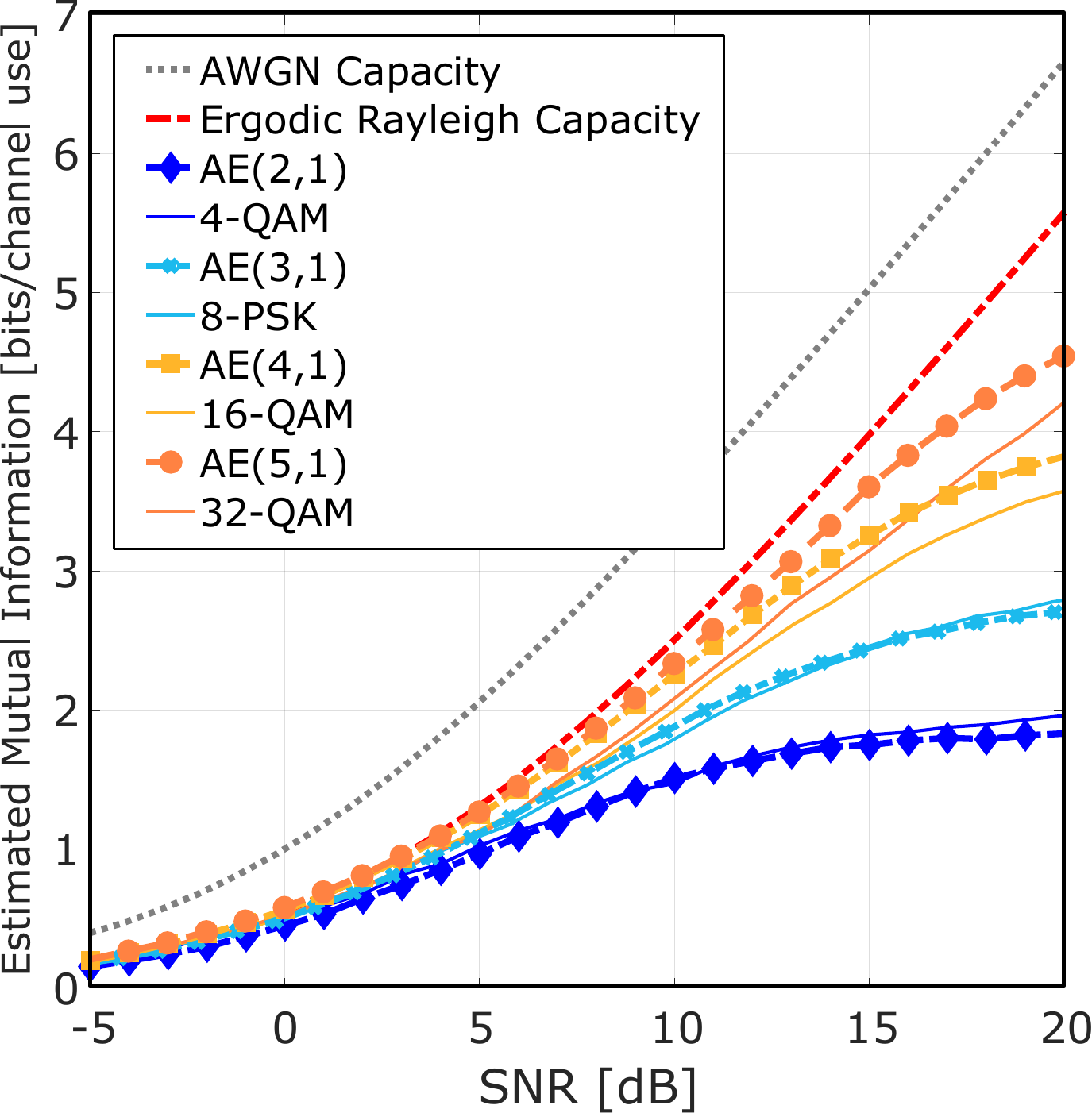}
	\caption{Estimated MI achieved with $\beta=0.2$ and $\epsilon=0.2$ for different dimension of the alphabet $M$ with code length $n=1$ over a Rayleigh channel.}
	\label{fig:AE_rayleigh_capacity}
\end{figure}
As final experiment, we introduce fading in the communication channel, in particular we consider a Rayleigh fading channel of the form
\begin{equation}
\label{eq:AE_rayleigh_noise}
Y_i=h_iX_i+N_i,
\end{equation}
where $N_i\sim \mathcal{CN}(0,\sigma^2)$ and $h_i$ is a random variable whose amplitude $\alpha$ belongs to the Rayleigh distribution $p_R(r)$ and is independent of the signal and noise. The ergodic capacity is given by
\begin{equation}
C = \mathbb{E}_{\alpha\sim p_R(r)}\biggl[\log_2(1+\alpha ^2 \cdot \text{SNR})\biggr].
\end{equation}
Fig. \ref{fig:AE_rayleigh_capacity} shows the estimated MI attained by the AE over a Rayleigh channel with several alphabet dimensions $M$ and compares it with the conventional $M$-QAM schemes. In all the cases, the achieved MI is upper bounded by the ergodic Rayleigh capacity. Similarly to the uniform case, it is curious to notice that the achieved information rate with the rate-driven AE is in some cases higher than the one obtained with the $M$-QAM schemes.
In particular, with $M=32$ the AE($5,1$) exceeds by $0.5$ bit/channel uses at SNR$=15$ dB the $32$-QAM scheme. Lastly, Fig. \ref{fig:AE_MI_all}c highlights the advantage of coding over two channel uses, especially in the range $5-15$ dB.

\begin{figure}
	\centering
	\includegraphics[scale=0.21]{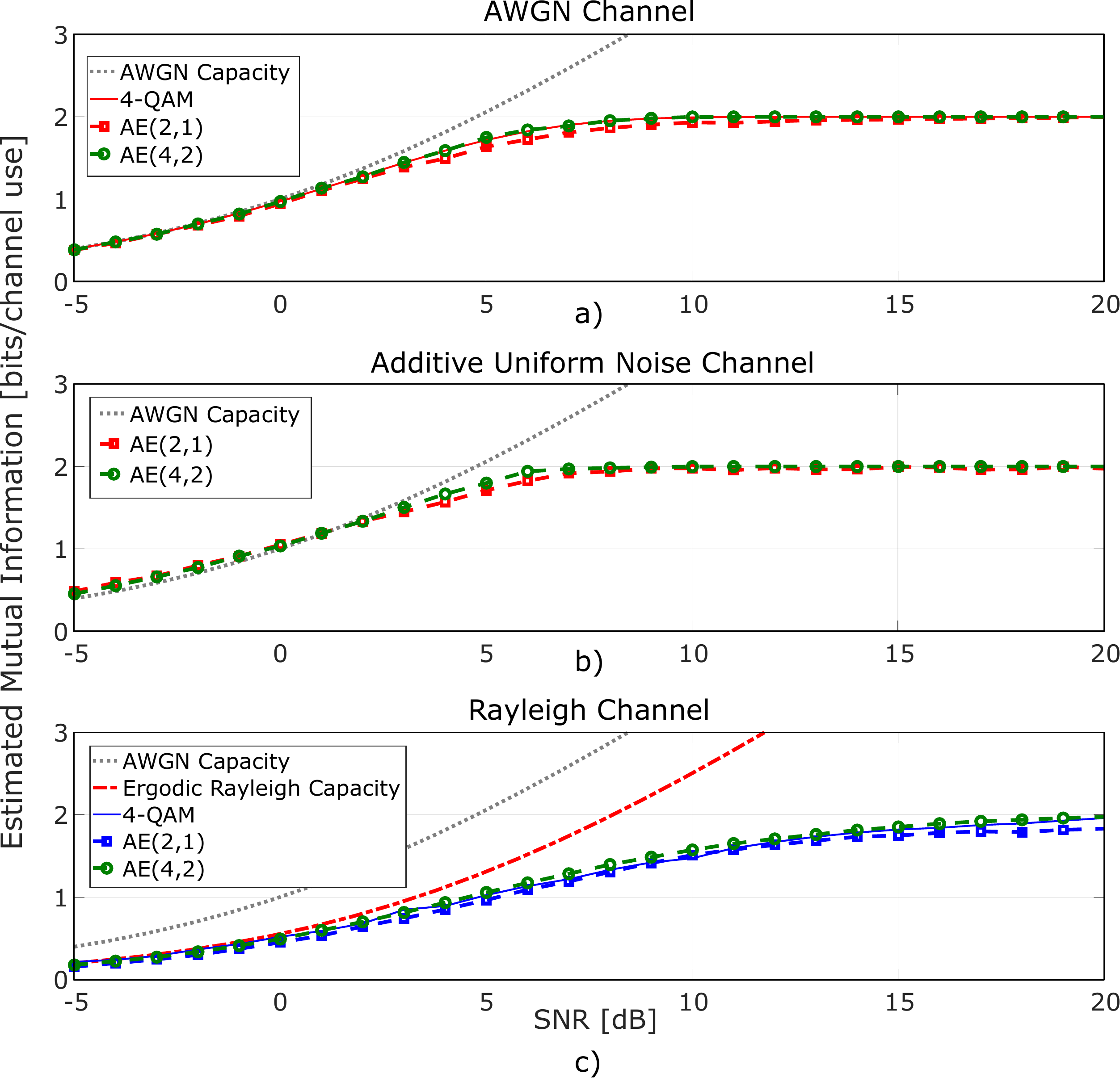}
	\caption{Estimated MI achieved with $\beta=0.2$, $\epsilon=0.2$ and rate $R=2$ with code length $n=1,2$ over a) AWGN channel, b) Additive uniform noise channel, c) Rayleigh channel.}
	\label{fig:AE_MI_all}
\end{figure}

\section{Summary}
\label{sec:autoencoders_conclusions}
This chapter has firstly discussed the AE-based communication system, highlighting the limits of the current cross-entropy loss function used for the training process. A regularization term that accounts for the MI between the transmitted and the received signals has been introduced to design optimal coding-decoding schemes for fixed rate $R$, code-length $n$ and given power constraint. The rationale behind the MI choice has been motivated exploiting the confidence penalty entropy regularization approach, the information bottleneck principle and the fundamental concept of channel capacity. In addition, an adaptation of the coding rate $R$ allowed us to build a capacity learning algorithm enabled by the novel loss function in a scheme named capacity-driven AE. Remarkably, the presented methodology does not make use of any theoretical a-priori knowledge of the communication channel and therefore opens the door to several future studies on intractable channel models, an example of which is the power line communication channel.

\part{Part 3}

\chapter{Variational Approaches for Mutual Information Estimation} 
\chaptermark{Mutual information neural estimation}
\label{sec:mi_estimators}

Estimating MI accurately is pivotal across diverse applications, from ML to communications and biology, enabling us to gain insights into the inner mechanisms of complex systems. Yet, dealing with high-dimensional data presents a formidable challenge, due to its size and the presence of intricate relationships. 

Recently proposed neural methods employing VLBs on the MI have gained prominence. However, these approaches suffer from either high bias or high variance, as the sample size and the structure of the loss function directly influence the training process. In this chapter, we present a novel class of discriminative mutual information estimators (DIME) based on the variational representation of the $f$-divergence. We investigate the impact of the permutation function used to obtain the marginal training samples and present a novel architectural solution based on derangements. The proposed estimator is flexible since it exhibits an excellent bias/variance trade-off. The comparison with state-of-the-art neural estimators, through extensive experimentation within established reference scenarios, shows that our approach offers higher accuracy and lower complexity.

The results presented in this chapter are documented in \cite{LetiziaNIPS, f-DIME, letizia2022Balkan}.

\section{Introduction}
\sectionmark{Introduction}
\label{sec:mi_related}
The MI between two multivariate random variables, $X$ and $Y$, is a fundamental quantity in statistics, representation learning, information theory, communication engineering and biology \cite{goldfeld2021sliced, tschannen2019mutual, guo2005mutual, pluim2003mutual}. 

Unfortunately, computing $I(X;Y)$ is challenging since the joint PDF $p_{XY}(\mathbf{x},\mathbf{y})$ and the marginals $p_X(\mathbf{x}),p_Y(\mathbf{y})$ are usually unknown, especially when dealing with high-dimensional data.

Traditional approaches for the MI estimation rely on binning, density and kernel estimation \cite{Moon1995} and $k$-nearest neighbors \cite{Kraskov2004}. Nevertheless, they do not scale to problems involving high-dimensional data as it is the case in modern ML applications. Hence, deep NNs have recently been leveraged to maximize VLBs on the MI \cite{Poole2019,Nguyen2010, Mine2018}. The expressive power of NNs has shown promising results in this direction although less is known about the effectiveness of such estimators \cite{Song2020}, especially since they suffer from either high bias or high variance. 

The variational approach usually exploits an energy-based variational family of functions to provide a lower bound on the KL divergence. As an example, the Donsker-Varadhan dual representation of the KL divergence \cite{Poole2019,Donsker1983} produces an estimate of the MI using the bound optimized by the mutual neural information estimator (MINE) \cite{Mine2018}
\begin{equation}
\label{eq:MI_MINE}
I_{MINE}(X;Y) = \sup_{\theta \in \Theta} \mathbb{E}_{(\mathbf{x},\mathbf{y})\sim p_{XY}(\mathbf{x},\mathbf{y})}[T_{\theta}(\mathbf{x},\mathbf{y})]  -\log(\mathbb{E}_{(\mathbf{x},\mathbf{y})\sim p_X(\mathbf{x}) p_Y(\mathbf{y})}[e^{T_{\theta}(\mathbf{x},\mathbf{y})}]),
\end{equation}
where $\theta \in \Theta$ parameterizes a family of functions $T_{\theta} : \mathcal{X}\times \mathcal{Y} \to \mathbb{R}$ through the use of a deep NN. However, the logarithm before the expectation in the second term renders MINE a biased estimator. To avoid biased gradients, the authors in \cite{Mine2018} suggested to replace the partition function $\mathbb{E}_{p_X p_Y}[e^{T_{\theta}}]$ with an exponential moving average over mini-data-batches.

Another VLB is based on the KL divergence dual representation introduced in \cite{Nguyen2010} (also referred to as $f$-MINE in \cite{Mine2018})
\begin{equation}
\label{eq:MI_NWJ}
I_{NWJ}(X;Y) = \sup_{\theta \in \Theta} \mathbb{E}_{(\mathbf{x},\mathbf{y})\sim p_{XY}(\mathbf{x},\mathbf{y})}[T_{\theta}(\mathbf{x},\mathbf{y})]  -\mathbb{E}_{(\mathbf{x},\mathbf{y})\sim p_X(\mathbf{x}) p_Y(\mathbf{y})}[e^{T_{\theta}(\mathbf{x},\mathbf{y})-1}].
\end{equation}
Although MINE provides a tighter bound $I_{MINE}\geq I_{NWJ}$, the NWJ estimator is unbiased. 

Both MINE and NWJ suffer from high-variance estimations and to combat such a limitation, the SMILE estimator was introduced in \cite{Song2020}. It is defined as
\begin{equation}
\small
\label{eq:MI_SMILE}
I_{SMILE}(X;Y) = \sup_{\theta \in \Theta} \mathbb{E}_{(\mathbf{x},\mathbf{y})\sim p_{XY}(\mathbf{x},\mathbf{y})}[T_{\theta}(\mathbf{x},\mathbf{y})] 
 -\log(\mathbb{E}_{(\mathbf{x},\mathbf{y})\sim p_X(\mathbf{x}) p_Y(\mathbf{y})}[\text{clip}(e^{T_{\theta}(\mathbf{x},\mathbf{y})},e^{-\tau},e^{\tau})]),
\end{equation}
where $\text{clip}(v,l,u) = \max(\min(v,u),l)$ and it helps to obtain smoother partition functions estimates. SMILE is equivalent to MINE in the limit $\tau \to +\infty$. 

The MI estimator based on contrastive predictive coding (CPC) \cite{NCE2018} is defined as 
\begin{equation}
\label{eq:MI_NCE}
I_{CPC}(X;Y) = \mathbb{E}_{(\mathbf{x},\mathbf{y})\sim p_{XY,N}(\mathbf{x},\mathbf{y})}\biggl[ \frac{1}{N} \sum_{i=1}^{N}{ \log\biggl( \frac{e^{T_{\theta}(\mathbf{x}_i,\mathbf{y}_i)}}{\frac{1}{N} \sum_{j=1}^{N}{e^{T_{\theta}(\mathbf{x}_i,\mathbf{y}_j)}}}\biggr)}  \biggr],
\end{equation}
where $N$ is the batch size and $p_{XY,N}$ denotes the joint distribution of $N$ i.i.d. random variables sampled from $p_{XY}$. CPC provides low variance estimates but it is upper bounded by $\log N$, resulting in a biased estimator. Such upper bound, typical of contrastive learning objectives, has been recently analyzed in the context of skew-divergence estimators \cite{RenyiCL}.

Another estimator based on a classification task is the neural joint entropy estimator (NJEE) proposed in \cite{shalev2022neural}.
Let $X_m$ be the $m$-th component of $X$, with $m\leq d$ and $N$ the batch size. $X^k$ is the vector containing the first $k$ components of $X$. Let $\hat{H}_N(X_1)$ be the estimated marginal entropy of the first components in $X$ and let $G_{\theta_m}(X_m|X^{m-1})$ be a neural network classifier, where the input is $X^{m-1}$ and the label used is $X_m$. If $\text{CE}(\cdot)$ is the cross-entropy function, then the MI estimator based on NJEE is defined as
\begin{equation}
    I_{NJEE}(X;Y) = \hat{H}_N(X_1) + \sum_{m=2}^{d} \text{CE}(G_{\theta_m}(X_m|X^{m-1})) - \sum_{m=1}^{d} \text{CE}(G_{\theta_m}(X_m|Y, X^{m-1})),
\end{equation}
where the first two terms of the RHS constitutes the NJEE estimator.

\section{Variational $f$-divergence and data derangements}
\sectionmark{f-DIME}
\label{sec:mi_f-DIME}

Inspired by the $f$-GAN training objective \cite{Nowozin2016}, in the following, we present a class of discriminative MI estimators based on the $f$-divergence measure. Conversely to what has been proposed so far in the literature, where $f$ is always constrained to be the generator of the KL divergence, we allow for any choice of $f$. Different $f$ functions will have different impact on the training and optimization sides, while on the estimation side, the partition function does not need to be computed, leading to low variance estimators.

\subsection{Contribution}
\label{subsec:mi_contribution}
Discriminative approaches \cite{raina2003classification, tonello2022mind} compare samples from both the joint and marginal distributions to directly compute the density ratio (or the log-density ratio) 
\begin{equation}
R(\mathbf{x},\mathbf{y}) = \frac{p_{XY}(\mathbf{x},\mathbf{y})}{p_X(\mathbf{x}) p_Y(\mathbf{y})}
\label{eq:MI_density_ratio_1}.
\end{equation}
Generative approaches \cite{Barber2003, song2019generative}, instead, aim at modeling the individual densities. 

We focus on discriminative MI estimation since it can in principle enjoy some of the properties of implicit generative models, which are able of directly generating data that belongs to the same distribution of the input data without any explicit density estimate. In fact, we recall that the GAN adversarial training pushes the discriminator $D(\mathbf{x})$ towards the optimum value
\begin{equation}
\label{eq:MI_gan_density_ratio}
\hat{D}(\mathbf{x}) = \frac{p_{data}(\mathbf{x})}{p_{data}(\mathbf{x})+p_{gen}(\mathbf{x})} = \frac{1}{1+\frac{p_{gen}(\mathbf{x})}{p_{data}(\mathbf{x})}}.
\end{equation}
Therefore, the output of the optimum discriminator is itself a function of the density ratio $p_{gen}/p_{data}$, where $p_{gen}$ and $p_{data}$ are the distributions of the generated and the collected data, respectively.
We generalize the observation of \eqref{eq:MI_gan_density_ratio} and we propose a family of MI estimators based on the VLB of the $f$-divergence \cite{Poole2019, sason2016f}. In particular, we argue that the maximization of any $f$-divergence VLB can lead to a MI estimator with excellent bias/variance trade-off.

Since we typically have access only to joint data points $(\mathbf{x},\mathbf{y}) \sim p_{XY}(\mathbf{x},\mathbf{y})$, another relevant practical aspect is the sampling strategy to obtain data from the product of marginals $p_{X}(\mathbf{x})p_Y(\mathbf{y})$, for instance via a shuffling mechanism along $N$ realizations of $Y$. We analyze the impact that the permutation has on the learning and training process and we propose a derangement training strategy that achieves high performance requiring $\Omega(N)$ operations.
Simulation results demonstrate that the proposed approach exhibits improved estimations in a multitude of scenarios.

In brief, we can summarize our contributions over the state-of-the-art as follows:
\begin{itemize}
    \item For any $f$-divergence, we derive a training value function whose maximization leads to a given MI estimator.
    \item We compare different $f$-divergences and comment on the resulting estimator properties and performance.
    \item We study the impact of data derangement for the learning model and propose a novel derangement training strategy that overcomes the upper bound on the MI estimation \cite{McAllester2019}, contrarily to what happens when using a random permutation strategy.
    \item We unify the main discriminative estimators into a publicly available code which can be used to reproduce all the results of this chapter.
\end{itemize}

\subsection{$f$-divergence mutual information estimation}
\label{subsec:mi_f-DIME}
The calculation of the MI via a discriminative approach requires the density ratio \eqref{eq:MI_density_ratio_1}.
From \eqref{eq:MI_gan_density_ratio}, we observe that $I(X;Y)$ can be estimated using the optimum GAN discriminator $\hat{D}$ when $p_{data}\equiv  p_{X}p_Y$ and $p_{gen} \equiv p_{XY}$.

More in general, the authors in \cite{Nowozin2016} extended the variational divergence estimation framework presented in \cite{Nguyen2010} and showed that any $f$-divergence can be used to train GANs. Inspired by such idea, we now argue that also discriminative MI estimators enjoy similar properties if the variational representation of $f$-divergence functionals $D_f(P||Q)$ is adopted.
The following theorem introduces $f$-DIME, a class of discriminative mutual information estimators (DIME) based on the variational representation of the $f$-divergence.
 
\begin{theorem}
\label{theorem:MI_theorem1}
Let $(X,Y) \sim p_{XY}(\mathbf{x},\mathbf{y})$ be a pair of multivariate random variables. Let $\sigma(\cdot)$ be a permutation function such that  $p_{\sigma(Y)}(\sigma(\mathbf{y})|\mathbf{x}) = p_{Y}(\mathbf{y})$. Let $f^*$ be the Fenchel conjugate of $f:\mathbb{R}_+ \to \mathbb{R}$, a convex lower semicontinuous function that satisfies $f(1)=0$ with derivative $f^{\prime}$.
If $\mathcal{J}_{f}(T)$ is a value function defined as 
\begin{equation}
\mathcal{J}_{f}(T) =  \mathbb{E}_{(\mathbf{x},\mathbf{y}) \sim p_{XY}(\mathbf{x},\mathbf{y})}\biggl[T\bigl(\mathbf{x},\mathbf{y}\bigr)-f^*\biggl(T\bigl(\mathbf{x},\sigma(\mathbf{y})\bigr)\biggr)\biggr],
\label{eq:MI_discriminator_function_f}
\end{equation}
then
\begin{equation}
\label{eq:MI_optimal_ratio_T}
\hat{T}(\mathbf{x},\mathbf{y}) =\arg \max_T \mathcal{J}_f(T) = f^{\prime} \biggl(\frac{p_{XY}(\mathbf{x},\mathbf{y})}{p_X(\mathbf{x})p_Y(\mathbf{y})}\biggr),
\end{equation}
and
\begin{equation}
\label{eq:MI_f-DIME}
I(X;Y) = I_{fDIME}(X;Y) =  \mathbb{E}_{(\mathbf{x},\mathbf{y}) \sim p_{XY}(\mathbf{x},\mathbf{y})}\biggl[ \log \biggl(\bigl(f^{*}\bigr)^{\prime}\bigl(\hat{T}(\mathbf{x},\mathbf{y})\bigr) \biggr) \biggr].
\end{equation}
\end{theorem}

Theorem \ref{theorem:MI_theorem1} shows that any value function $\mathcal{J}_f$ of the form in \eqref{eq:MI_discriminator_function_f}, seen as the dual representation of a given $f$-divergence $D_f$, can be maximized to estimate the MI via \eqref{eq:MI_f-DIME}. It is interesting to notice that the proposed class of estimators does not need any evaluation of the partition term.
 
We propose to parametrize $T(\mathbf{x},\mathbf{y})$ with a deep NN $T_{\theta}$ of parameters $\theta$ and solve with gradient ascent and back-propagation to obtain
\begin{equation}
\hat{\theta} = \arg \max_{\theta} \mathcal{J}_f(T_{\theta}).
\end{equation}
By doing so, it is possible to guarantee that, at every training iteration $n$, the convergence of the $f$-DIME estimator $\hat{I}_{n,fDIME}(X;Y)$ is controlled by the convergence of $T$ towards the tight bound $\hat{T}$ while maximizing $\mathcal{J}_f(T)$, as stated in the following lemma.
\begin{lemma}
\label{lemma:MI_convergence}
Let the discriminator $T(\cdot)$ be with enough capacity, i.e., in the non parametric limit. Consider the problem
\begin{equation}
\hat{T} =  \; \arg \max_T \mathcal{J}_{f}(T)
\label{eq:MI_Lemma4_problem}
\end{equation}
where $\mathcal{J}_{f}(T)$ is defined as in \eqref{eq:MI_discriminator_function_f},
and the update rule based on the gradient descent method
\begin{equation}
T^{(n+1)} = T^{(n)} + \mu \nabla \mathcal{J}_{f}(T^{(n)}).
\end{equation}
If the gradient descent method converges to the global optimum $\hat{T}$, the mutual information estimator defined in \eqref{eq:MI_f-DIME}
converges to the real value of the mutual information $I(X;Y)$.
\end{lemma}
The proof of Lemma \ref{lemma:MI_convergence}, which is described in Sec. \ref{sec:mi_appendix}, provides some theoretical grounding for the behaviour of MI estimators when the training does not converge to the optimal density ratio. Moreover, it also offers insights about the impact of different functions $f$ on the numerical bias.

It is important to remark the difference between the classical VLBs estimators that follow a discriminative approach and the DIME-like estimators. They both achieve the goal through a discriminator network that outputs a function of the density ratio. However, the former models exploit the variational representation of the MI (or the KL) and, at the equilibrium, use the discriminator output directly in one of the value functions reported in Sec. \ref{sec:mi_related}. The latter, instead, use the variational representation of \textit{any} $f$-divergence to extract the density ratio estimate directly from the discriminator output.

In the upcoming subsections, we analyze the variance of $f$-DIME and we propose a training strategy for the implementation of Theorem \ref{theorem:MI_theorem1}.
In our experiments, we consider the cases when $f$ is the generator of: a) the KL divergence; b) the GAN divergence; c) the Hellinger distance squared. We report in Sec.~\ref{subsec:mi_appendix_DIME_examples} the value functions used for training and the mathematical expressions of the resulting DIME estimators.

\subsection{Variance analysis}
\label{subsec:mi_variance}
In this subsection, we assume that the ground truth density ratio $\hat{R}(\mathbf{x},\mathbf{y})$ exists and corresponds to the density ratio in \eqref{eq:MI_density_ratio_1}. We also assume that the optimum discriminator $\hat{T}(\mathbf{x},\mathbf{y})$ is known and already obtained (e.g. via a NN parametrization). 

We define $p^M_{XY}(\mathbf{x},\mathbf{y})$ and $p^N_{X}(\mathbf{x})p^N_Y(\mathbf{y})$ as the empirical distributions corresponding to $M$ i.i.d. samples from the true joint distribution $p_{XY}$ and to $N$ i.i.d. samples from the product of marginals $p_Xp_Y$, respectively. The randomness of the sampling procedure and the batch sizes $M,N$ influence the variance of variational MI estimators. In the following, we prove that under the previous assumptions, $f$-DIME exhibits better performance in terms of variance w.r.t. some variational estimators with a discriminative approach, e.g., MINE and NWJ.

The partition function estimation $\mathbb{E}_{p^N_X p^N_Y}[\hat{R}]$ represents the major issue when dealing with variational MI estimators. Indeed, they comprise the evaluation of two terms (using the given density ratio), and the partition function is the one responsible for the variance growth. The authors in \cite{Song2020} characterized the variance of both MINE and NWJ estimators, in particular, they proved that the variance scales exponentially with the ground-truth MI $\forall M \in \mathbb{N}$
\begin{align}
\label{eq:MI_exponentially_increasing_variance}
\text{Var}_{p_{XY},p_Xp_Y}\bigl[ I^{M,N}_{NWJ}\bigr] \geq & \frac{e^{I(X;Y)}-1}{N} \nonumber \\
\lim_{N \to \infty} N\text{Var}_{p_{XY},p_Xp_Y}\bigl[ I^{M,N}_{MINE}\bigr] \geq & e^{I(X;Y)}-1,
\end{align}
where
\begin{align}
I^{M,N}_{NWJ}:= \mathbb{E}_{p^M_{XY}}[\log \hat{R} +1] - \mathbb{E}_{p^N_Xp^N_Y}[\hat{R}]   \nonumber \\
 I^{M,N}_{MINE}:= \mathbb{E}_{p^M_{XY}}[\log \hat{R}] - \log \mathbb{E}_{p^N_Xp^N_Y}[\hat{R}].   
\end{align}
To reduce the impact of the partition function on the variance, the authors of \cite{Song2020} also proposed to clip the density ratio between $e^{-\tau}$ and $e^{\tau}$ leading to an estimator (SMILE) with bounded partition variance. However, also the variance of the log-density ratio $\mathbb{E}_{p^M_{XY}}[\log \hat{R}]$ influences the variance of the variational estimators, since it is clear that
\begin{equation}
\text{Var}_{p_{XY},p_Xp_Y}\bigl[ I^{M,N}_{VLB}\bigr] \geq \text{Var}_{p_{XY}}\bigl[ \mathbb{E}_{p^M_{XY}}[\log \hat{R}] \bigr],
\label{eq:MI_lower_bound_variance}
\end{equation}
a result that holds for any type of MI estimator based on a VLB.

The great advantage of $f$-DIME is to avoid the partition function estimation step, significantly reducing the variance of the estimator. Under the same initial assumptions, from \eqref{eq:MI_lower_bound_variance} we can immediately conclude that
\begin{equation}
\text{Var}_{p_{XY}}\bigl[I^{M}_{fDIME}\bigr] \leq \text{Var}_{p_{XY},p_Xp_Y}\bigl[ I^{M,N}_{VLB}\bigr],
\end{equation}
where
\begin{equation}
\label{eq:MI_f_dime_estimator}
I^{M}_{fDIME}:= \mathbb{E}_{p^M_{XY}}[\log \hat{R}]    
\end{equation}
is the Monte Carlo implementation of $f$-DIME. Hence, the $f$-DIME class of models has lower variance than any VLB based estimator (MINE, NWJ, SMILE, etc.). 

The following Lemma provides an upper bound on the variance of the $f$-DIME estimator. Notice that such result holds for any type of value function $\mathcal{J}_{f}$, so it is not restrictive to the KL divergence.

\begin{lemma}
\label{lemma:MI_Lemma2}
Let $\hat{R} = p_{XY}(\mathbf{x},\mathbf{y})/ (p_{X}(\mathbf{x}) p_Y(\mathbf{y}))$ be the density ratio and assume $\text{Var}_{p_{XY}}[\log \hat{R}]$ exists. Let $p^M_{XY}$ be the empirical distribution of $M$ i.i.d. samples from $p_{XY}$ and let $\mathbb{E}_{p^M_{XY}}$ denote the sample average over $p^M_{XY}$. Then, it follows that

\begin{equation}
\text{Var}_{p_{XY}}\bigl[\mathbb{E}_{p^M_{XY}} [\log \hat{R}] \bigr] \leq \frac{ 4H^2(p_{XY},p_Xp_Y)\biggl|\biggl|\frac{p_{XY}}{p_Xp_Y}\biggr|\biggr|_{\infty}-I^2(X;Y)}{M}
\end{equation}
where $H^2$ is the Hellinger distance squared.
\end{lemma}

Lemma \ref{lemma:MI_finite_variance_gaus} characterizes the variance of the estimator in \eqref{eq:MI_f_dime_estimator} when $X$ and $Y$ are correlated Gaussian random variables. We found out that the variance is finite and we use this result to verify in the experiments that the variance of $f$-DIME does not diverge for high values of MI.

\begin{lemma}
\label{lemma:MI_finite_variance_gaus}
Let $\hat{R}$ be the optimal density ratio and let  $X\sim \mathcal{N}(0,\sigma_X^2)$ and $N\sim \mathcal{N}(0,\sigma_N^2)$ be uncorrelated scalar Gaussian random variables such that $Y=X+N$. Assume $\text{Var}_{p_{XY}}[\log \hat{R}]$ exists. Let $p^M_{XY}$ be the empirical distribution of $M$ i.i.d. samples from $p_{XY}$ and let $\mathbb{E}_{p^M_{XY}}$ denote the sample average over $p^M_{XY}$. Then, it holds that
\begin{equation}
\label{eq:MI_variance_f_dime}
\text{Var}_{p_{XY}}\bigl[\mathbb{E}_{p^M_{XY}} [\log \hat{R}] \bigr] = \frac{ 1-e^{-2I(X;Y)}}{M}.
\end{equation}
\end{lemma}

\subsection{Derangement strategy}
\label{subsec:mi_derangements}
The discriminative approach essentially compares expectations over both joint $(\mathbf{x},\mathbf{y}) \sim p_{XY}$ and marginal $(\mathbf{x},\mathbf{y}) \sim p_{X}p_Y$ data points. 
Practically, we have access only to $N$ realizations of the joint distribution $p_{XY}$ and to obtain $N$ marginal samples of $p_Xp_Y$ from $p_{XY}$ a shuffling mechanism for the realizations of $Y$ is typically deployed. A general result in \cite{McAllester2019} shows that failing to sample from the correct marginal distribution would lead to an upper bounded MI estimator.

We study the structure that the permutation law $\sigma(\cdot)$ in Theorem \ref{theorem:MI_theorem1} needs to have when numerically implemented. In particular, we now prove that a naive permutation over the realizations of $Y$ results in an incorrect VLB of the $f$-divergence, causing the MI estimator to be bounded by $\log(N)$, where $N$ is the batch size. To solve this issue, we propose a derangement strategy.  

Let the data points $(\mathbf{x},\mathbf{y}) \sim p_{XY}$ be $N$ pairs $(\mathbf{x}_i, \mathbf{y}_i)$, $\forall i\in \{1,\dots,N\}$. The naive permutation of $\mathbf{y}$, denoted as $\pi(\mathbf{y})$, leads to $N$ new random pairs $(\mathbf{x}_i, \mathbf{y}_j)$, $\forall i$ and $j \in \{ 1, \cdots , N\}$. The idea is that a random naive permutation may lead to at least one pair $(\mathbf{x}_k, \mathbf{y}_k)$, with $k \in \{1,\dots,N\}$, which is actually a sample from the joint distribution. 
Viceversa, the derangement of $\mathbf{y}$, denoted as $\sigma(\mathbf{y})$, leads to $N$ new random pairs $(\mathbf{x}_i, \mathbf{y}_j)$ such that $i \neq j, \forall i$ and $j \in \{ 1, \cdots , N\}$. Such pairs $(\mathbf{x}_i, \mathbf{y}_j), i \neq j$ can effectively be considered samples from $p_X(\mathbf{x})p_Y(\mathbf{y})$.
An example using these definitions is provided in Sec. \ref{subsec:mi_derangement_considerations}.

The following lemma analyzes the relationship between the Monte Carlo approximations of the VLBs of the $f$-divergence $\mathcal{J}_{f}$ in Theorem \ref{theorem:MI_theorem1} using $\pi(\cdot)$ and $\sigma(\cdot)$ as permutation laws.

\begin{lemma}
\label{lemma:MI_lemma4}
Let $(\mathbf{x}_i,\mathbf{y}_i)$, $\forall i\in \{1,\dots,N\}$, be $N$ data points. Let $\mathcal{J}_{f}(T)$ be the value function in \eqref{eq:MI_discriminator_function_f}. Let $\mathcal{J}_{f}^{\pi}(T)$ and $\mathcal{J}_{f}^{\sigma}(T)$ be numerical implementations of $\mathcal{J}_{f}(T)$ using a random permutation and a random derangement of $\mathbf{y}$, respectively. Denote with $K$ the number of points $\mathbf{y}_k$, with $k \in \{1,\dots, N\}$, in the same position after the permutation (i.e., the fixed points). Then
\begin{equation}
\mathcal{J}_{f}^{\pi}(T) \leq \frac{N-K}{N} \mathcal{J}_{f}^{\sigma}(T).
\label{eq:MI_perm_vs_derang}
\end{equation}
\end{lemma}
Lemma \ref{lemma:MI_lemma4} practically asserts that the value function $\mathcal{J}_{f}^{\pi}(T)$ evaluated via a naive permutation of the data is not a valid VLB of the $f$-divergence, and thus, there is no guarantee on the optimality of the discriminator's output. 
An interesting mathematical connection can be obtained when studying $\mathcal{J}_{f}^{\pi}(T)$ as a sort of variational skew-divergence estimator \cite{RenyiCL}.

The following theorem states that in the case of the KL divergence, the maximum of $\mathcal{J}_{f}^{\pi}(D)$ is attained for a  value of the discriminator that is not exactly the density ratio (as it should be from \eqref{eq:MI_optimal_ratio_KL}, see Sec. \ref{subsec:mi_appendix_DIME_examples}).

\begin{theorem}
\label{theorem:MI_permutationsBound}
Let the discriminator $D(\cdot)$ be with enough capacity. Let $N$ be the batch size and $f$ be the generator of the KL divergence. Let $\mathcal{J}_{KL}^{\pi}(D)$ be defined as
\begin{equation}
\mathcal{J}_{KL}^{\pi}(D) =  \mathbb{E}_{(\mathbf{x},\mathbf{y}) \sim p_{XY}(\mathbf{x},\mathbf{y})}\biggl[\log\biggl(D\bigl(\mathbf{x},\mathbf{y}\bigr)\biggr)-f^*\biggl(\log\biggl(D\bigl(\mathbf{x},\pi(\mathbf{y})\bigr)\biggr)\biggr)\biggr].
\label{eq:MI_discriminator_function_perm_KL}
\end{equation}

Denote with $K$ the number of indices in the same position after the permutation (i.e., the fixed points), and with $R(\mathbf{x},\mathbf{y})$ the density ratio in \eqref{eq:MI_density_ratio_1}.
Then,
\begin{equation}
\label{eq:MI_optimal_ratio_perm_KL}
\hat{D}(\mathbf{x},\mathbf{y}) =\arg \max_D \mathcal{J}_{KL}^{\pi}(D) = \frac{NR(\mathbf{x},\mathbf{y})}{KR(\mathbf{x},\mathbf{y})+N-K}.
\end{equation}
\end{theorem}
Although Theorem \ref{theorem:MI_permutationsBound} is stated for the KL divergence, it can be easily extended to any $f$-divergence using Theorem \ref{theorem:MI_theorem1}. Notice that if the number of indices in the same position $K$ is equal to $0$, we fall back into the derangement strategy and we retrieve the density ratio as output. 

When we parametrize $D$ with a NN, we perform multiple training iterations and so we have multiple batches of dimension $N$. This turns into an average analysis on $K$. We report in the Supplementary material section the proof that, on average, $K$ is equal to $1$.
\begin{lemma}[see \cite{alon2016probabilistic}]
\label{lemma:MI_M=1}
The average number of fixed points in a random permutation $\pi(\cdot)$ is equal to 1.
\end{lemma}

From the previous results, it follows immediately that the estimator obtained using a naive permutation strategy is biased and upper bounded by a function of the batch size $N$.  
\begin{corollary}[Permutation bound]
\label{corollary:MI_permutationsUpperBound}
Let KL-DIME be the estimator obtained via iterative optimization of $\mathcal{J}_{KL}^{\pi}(D)$, using a batch of size $N$ every training step. Then,
\begin{equation}
    I_{KL-DIME}^{\pi} := \mathbb{E}_{(\mathbf{x},\mathbf{y}) \sim p_{XY}(\mathbf{x},\mathbf{y})}\biggl[ \log \biggl(\hat{D}(\mathbf{x},\mathbf{y})\biggr) \biggr] < \log(N).
\end{equation}
\end{corollary}

We report in Fig. \ref{fig:MI_derangementVsPermutation} an example of the difference between the derangement and permutation strategies. The estimate attained by using the permutation mechanism, showed in Fig. \ref{fig:MI_permutations}, demonstrates Theorem \ref{theorem:MI_permutationsBound} and Corollary \ref{corollary:MI_permutationsUpperBound}, as the upper bound corresponding to $\log(N)$ (with $N=128$) is clearly visible. 

\begin{figure}
\centering
\begin{subfigure}{0.5\textwidth}
  \centering
  \includegraphics[scale=0.4]{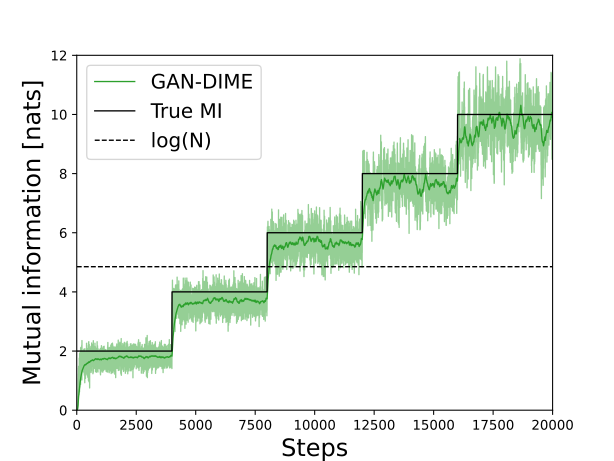}
  \caption{Derangement strategy.}
  \label{fig:MI_derangement}
\end{subfigure}%
\begin{subfigure}{0.5\textwidth}
  \centering
  \includegraphics[scale=0.4]{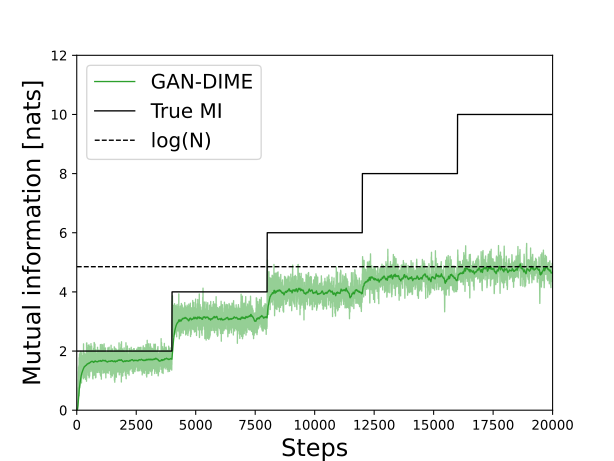}
  \caption{Permutation strategy.}
  \label{fig:MI_permutations}
\end{subfigure}
\caption{MI estimate obtained with derangement and permutation training procedures, for data dimension $d=20$ and batch size $N=128$.}
\label{fig:MI_derangementVsPermutation}
\end{figure}

\section{Experimental results}
\sectionmark{Results}
\label{sec:mi_results}
In this section, we first describe the architectures of the proposed estimators. Then, we outline the data used to estimate the MI, comment on the performance of the discussed estimators in different scenarios, also analyzing their computational complexity. Finally, we report in Sec. \ref{subsec:mi_appendix_consistency_tests} the self-consistency tests \cite{Song2020} over image datasets.

\subsection{Architectures}
\label{subsec:mi_architectures}
To demonstrate the behavior of the state-of-the-art MI estimators, we consider multiple NN \textit{architectures}. The word architecture needs to be intended in a wide-sense, meaning that it represents the NN architecture and its training strategy.
In particular, additionally to the architectures \textbf{joint} \cite{Mine2018} and \textbf{separable} \cite{oord2018representation}, we propose the architecture \textbf{deranged}.
The \textbf{joint} architecture concatenates the samples $\mathbf{x}$ and $\mathbf{y}$ as input of a single NN. Each training step requires $N$ realizations $(\mathbf{x}_i, \mathbf{y}_i)$ drawn from $p_{XY}(\mathbf{x},\mathbf{y})$, for $i \in \left\{ 1, \dotsc, N \right\}$ and $N(N-1)$ samples $(\mathbf{x}_i, \mathbf{y}_j), \forall i,j \in \left\{ 1, \dotsc, N \right\}$, with $ i \neq j$.\\
The \textbf{separable} architecture comprises two neural networks, the former fed in with $N$ realizations of $X$, the latter with $N$ realizations of $Y$. The inner product between the outputs of the two networks is exploited to obtain the MI estimate.

The proposed \textbf{deranged} architecture feeds a neural network with the concatenation of the samples $\mathbf{x}$ and $\mathbf{y}$, similarly to the \textit{joint} architecture. However, the \textit{deranged} one obtains the samples of $p_X(\mathbf{x})p_Y(\mathbf{y})$ by performing a derangement of the realizations $\mathbf{y}$ in the batch sampled from $p_{XY}(\mathbf{x},\mathbf{y})$.
Such diverse training strategy solves the main problem of the \textit{joint} architecture: the difficult scalability to large batch sizes. For large values of $N$, the complexity of the \textit{joint} architecture is $\Omega(N^2)$, while the complexity of the \textit{deranged} one is $\Omega(N)$.

NJEE utilizes a specific architecture, in the following referred to as \textbf{ad hoc}, comprising $2d-1$ neural networks, where $d$ is the dimension of $X$. $I_{NJEE}$ training procedure is supervised: the input of each neural network does not include the $\mathbf{y}$ samples.
All the implementation details are reported in Sec.  \ref{subsec:mi_appendix_experiment_details}.

\subsection{Multivariate linear and non-linear Gaussians}
\label{subsec:mi_stairs}
We benchmark the proposed class of MI estimators on two settings utilized in previous papers \cite{Poole2019, Song2020}.
In the first setting (called \textbf{Gaussian}), a $20$-dimensional Gaussian distribution is sampled to obtain $\mathbf{x}$ and $\mathbf{n}$ samples, independently. Then, $\mathbf{y}$ is obtained as linear combination of $\mathbf{x}$ and $\mathbf{n}$: $\mathbf{y} = \rho \, \mathbf{x} + \sqrt{1-\rho^2} \, \mathbf{n}$, where $\rho$ is the correlation coefficient.
In the second setting (referred to as \textbf{cubic}), the non-linear transformation $\mathbf{y} \mapsto \mathbf{y}^3$ is applied to the Gaussian samples. The true MI follows a staircase shape, where each step is a multiple of $2$ $nats$. Each NN is trained for 4k iterations for each stair step, with a batch size of $64$ samples ($N=64$). The values $d=20$ and $N=64$ are used in the literature to compare MI neural estimators. 
\begin{figure}
	\centering
	\includegraphics[scale=0.24]{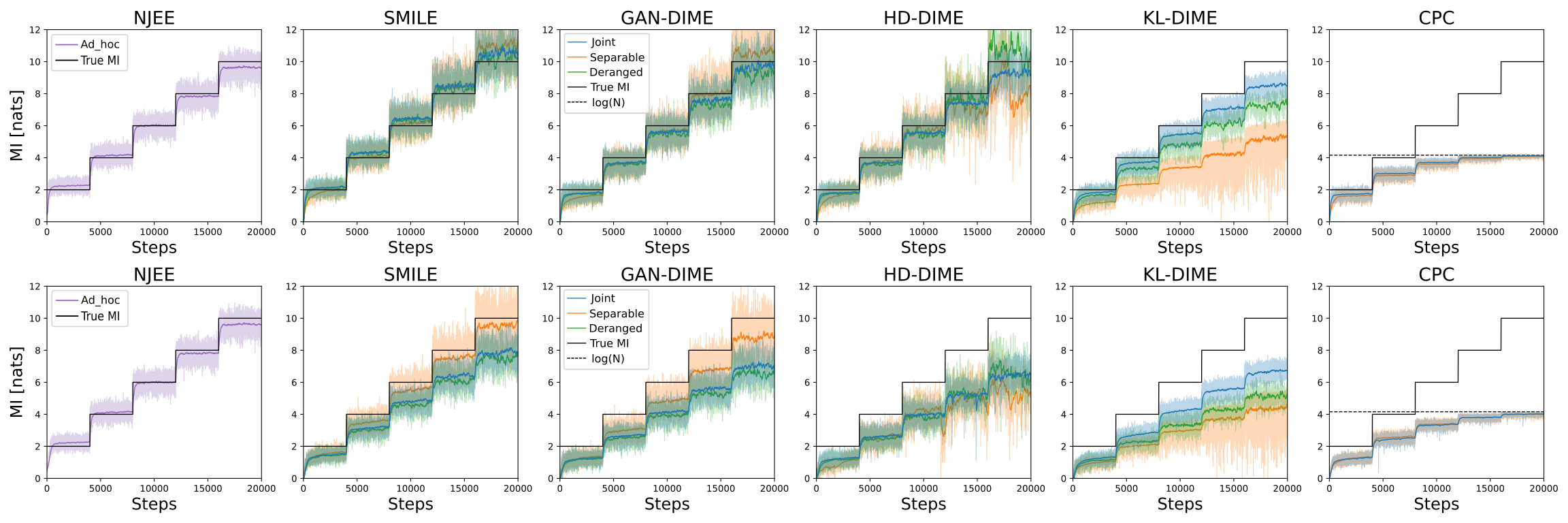}
	\caption{Staircase MI estimation comparison for $d=20$ and $N=64$. The \textit{Gaussian} case is reported in the top row, while the \textit{cubic} case is shown in the bottom row.}
	\label{fig:MI_stairs}
\end{figure} 
The tested estimators are: $I_{NJEE}$, $I_{SMILE}$ ($\tau=1$), $I_{GAN-DIME}$, $I_{HD-DIME}$, $I_{KL-DIME}$, and $I_{CPC}$, as illustrated in Fig. \ref{fig:MI_stairs}. The performance of $I_{MINE}$, $I_{NWJ}$, and $I_{SMILE} (\tau=\infty)$ is reported in Sec. \ref{subsec:mi_appendix_experiment_details}, since these algorithms exhibit lower performance compared to both SMILE and $f$-DIME.
In fact, all the $f$-DIME estimators have lower variance compared to $I_{MINE}$, $I_{NWJ}$, and $I_{SMILE} (\tau = \infty)$, which are characterized by an exponentially increasing variance (see (\ref{eq:MI_exponentially_increasing_variance})).
In particular, all the analyzed estimators belonging to the $f$-DIME class achieve significantly low bias and variance when the true MI is small. Interestingly, for high target MI, different $f$-divergences lead to dissimilar estimation properties.
For large MI, $I_{KL-DIME}$ is characterized by a low variance, at the expense of a high bias and a slow rise time.
Contrarily, $I_{HD-DIME}$ attains a lower bias at the cost of slightly higher variance w.r.t. $I_{KL-DIME}$. 
Diversely, $I_{GAN-DIME}$ achieves the lowest bias, and a variance comparable to $I_{HD-DIME}$.

The MI estimates obtained with $I_{SMILE}$ and $I_{GAN-DIME}$ appear to possess similar behavior, although the value functions of SMILE and GAN-DIME are structurally different.
The reason why $I_{SMILE}$ is almost equivalent to $I_{GAN-DIME}$ resides in their training strategy, since they both minimize the same $f$-divergence.
Looking at the implementation code of SMILE \cite{SMILE_github}, in fact, the network's training is guided by the gradient computed using the Jensen-Shannon (JS) divergence (a linear transformation of the GAN divergence). 

Given the trained network, the clipped objective function proposed in \cite{Song2020} is only used to compute the MI estimate, since when (\ref{eq:MI_SMILE}) is used to train the network, the MI estimate diverges (see Fig. \ref{fig:MI_SMILE_no_trick} in Sec. \ref{subsec:mi_appendix_experiment_details}). 
However, with the proposed class of $f$-DIME estimators we show that during the estimation phase the partition function (clipped in \cite{Song2020}) is not necessary to obtain the MI estimate.

$I_{NJEE}$ obtains an estimate for $d=20$ and $N=64$ that has slightly higher bias than $I_{GAN-DIME}$ for large MI values and slightly higher variance than $I_{KL-DIME}$.
$I_{CPC}$ is characterized by high bias and low variance. A schematic comparison between all the MI estimators is reported in Tab. \ref{tab:MI_summary_estimators} in Sec. \ref{subsec:mi_appendix_experiment_details}.

When $N$ and $d$ vary, the class of $f$-DIME estimators proves its robustness (i.e., maintains low bias and variance), as represented in Fig. \ref{fig:MI_stairs_d5_bs64} and \ref{fig:MI_stairs_d20_bs1024}. Differently, the behavior of $I_{CPC}$ strongly depends on $N$. At the same time, $I_{NJEE}$ achieves higher bias when $N$ increases and, even more severely, when $d$ decreases (see Fig. \ref{fig:MI_stairs_d5_bs64}). 

Additional results describing all estimators' behavior when $d$ and $N$ vary are reported and described in Sec.  \ref{subsec:mi_appendix_experiment_details}.

\begin{figure}
	\centering
	\includegraphics[scale=0.24]{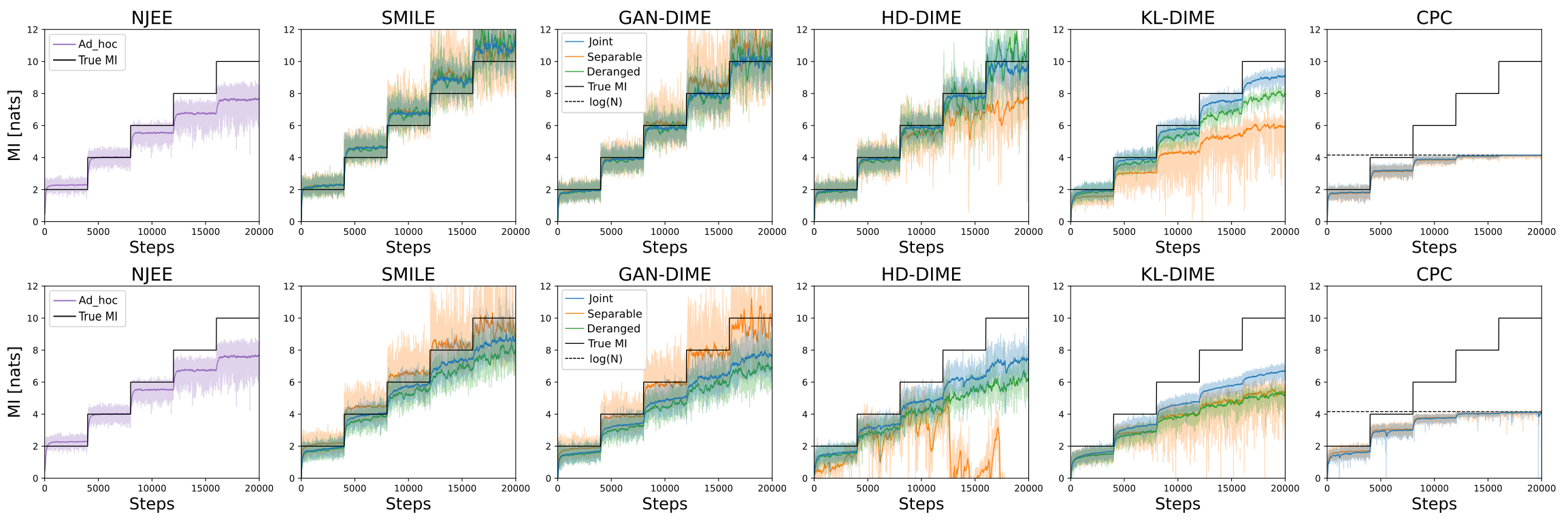}
	\caption{Staircase MI estimation comparison for $d=5$ and $N=64$. The \textit{Gaussian} case is reported in the top row, while the \textit{cubic} case is shown in the bottom row.}
	\label{fig:MI_stairs_d5_bs64}
\end{figure} 
 \begin{figure}
	\centering
	\includegraphics[scale=0.24]{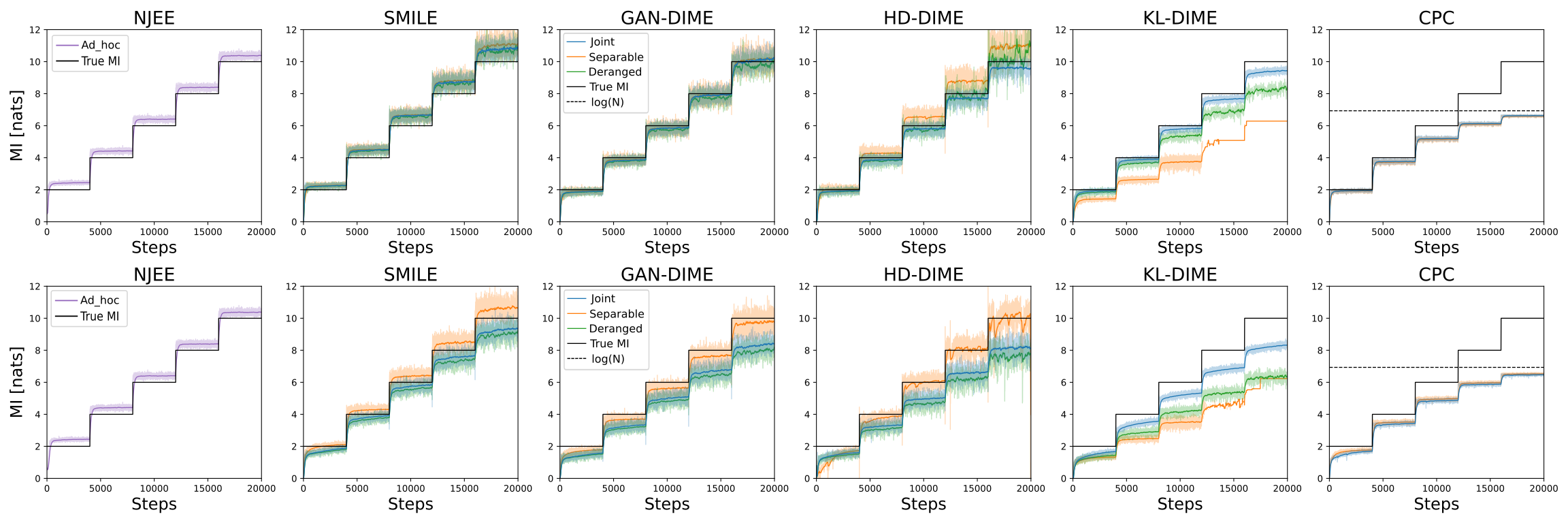}
	\caption{Staircase MI estimation comparison for $d=20$ and $N=1024$. The \textit{Gaussian} case is reported in the top row, while the \textit{cubic} case is shown in the bottom row.}
	\label{fig:MI_stairs_d20_bs1024}
\end{figure} 

\subsection{Complex Gaussian and non-Gaussian distributions}
 We test our estimators over additional complex Gaussian data transformations (half-cube, asinh, and swiss roll mappings) and non-Gaussian distributions (uniform and student distributions) as suggested in \cite{czyz2023beyond}. The half-cube mapping is used to lengthen the tails of the Gaussian distributions. The inverse hyperbolic sine (asinh) mapping shortens the tails of the Gaussian distributions. These two transformations are applied to the same scenario of the Gaussian and cubic already present in our paper. The swiss roll mapping increases the dimensionality of the data distribution (from two to three dimensions) and it is usually used to test dimensionality reduction techniques. It considers two Gaussian random variables that are transformed into uniform random variables via the probability integral transform. The swiss roll mapping is applied to the $X$ uniform random variable. The stairs plots are obtained by varying the correlation between the initial Gaussian distributions. The uniform case estimates the MI of the summation of two uniform random variables $U(0,1)$ and $U(-\epsilon, \epsilon)$, where we vary the parameter $\epsilon$, modifying the true MI. The student scenario analyzes the case of a multivariate student distribution with dispersion matrix chosen to be the identity matrix and degrees of freedom $df$. In this scenario, we vary $df$, implying a variation of the target MI. 

\begin{figure}
	\centering
	\includegraphics[scale=0.2]{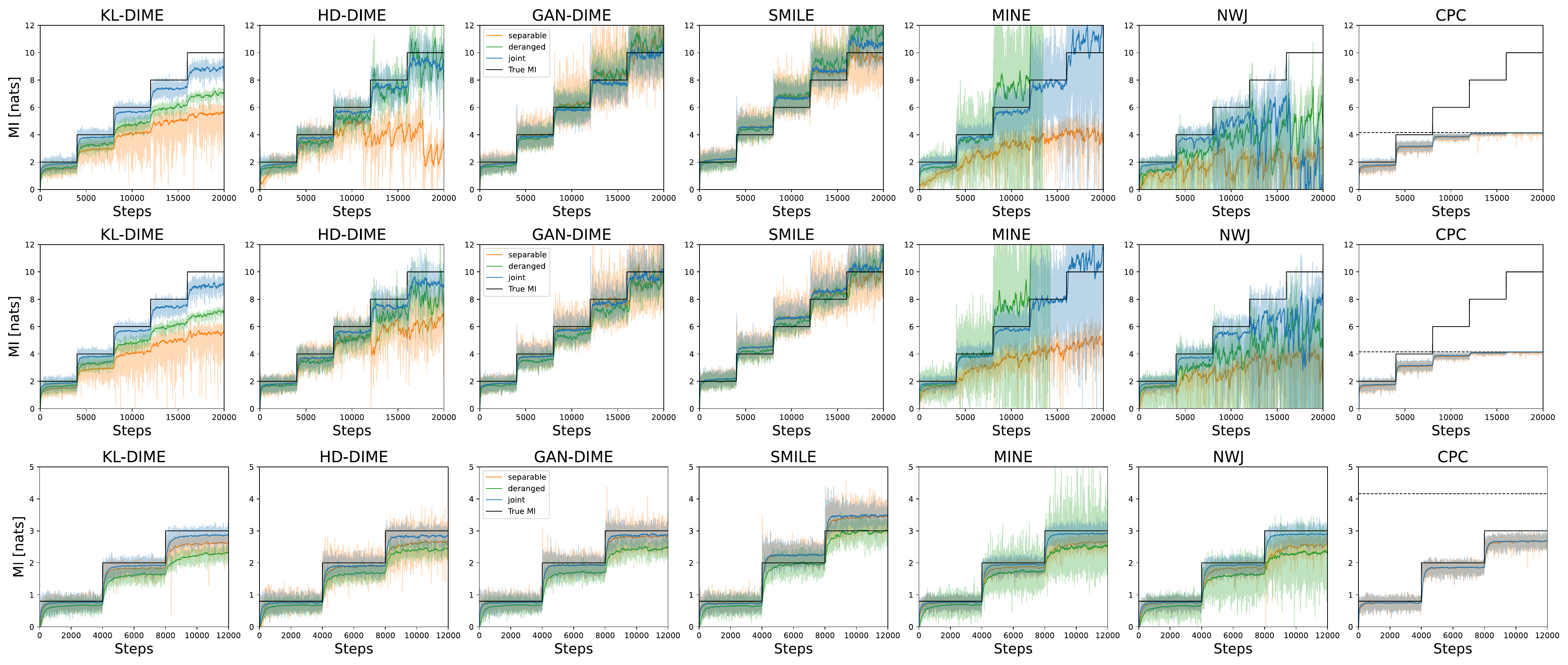}
	\caption{Staircase MI estimation comparison for $d=5$ and $N=64$. Top: Half-cube scenario. Middle: Asinh scenario. Bottom: Swiss roll scenario.}
	\label{fig:MI_stairs_d5_bs64_beyond}
\end{figure} 
 \begin{figure}
	\centering
	\includegraphics[scale=0.2]{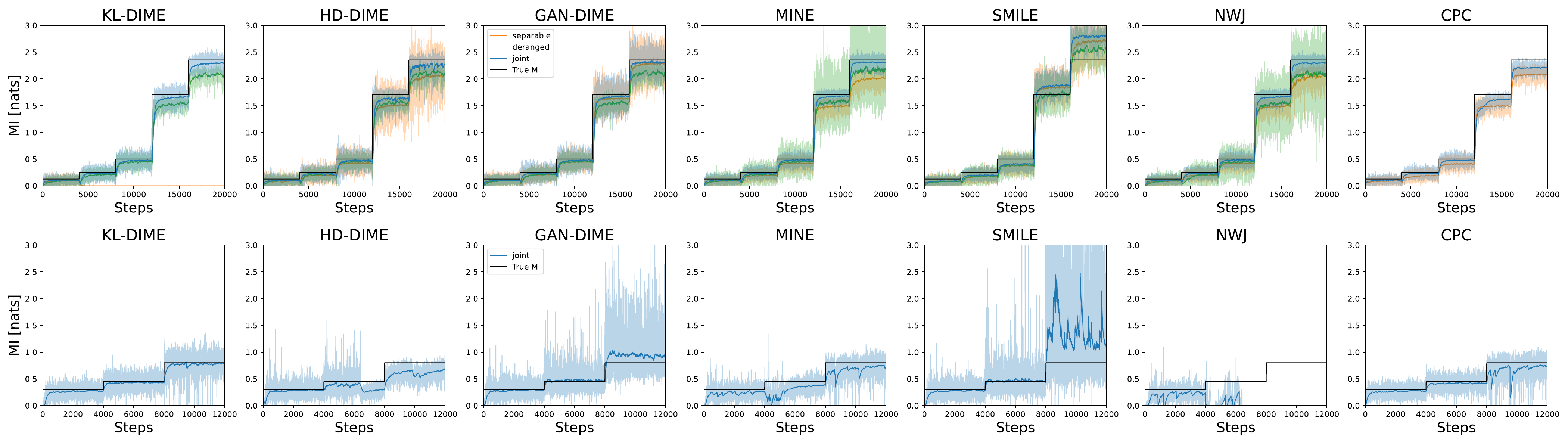}
	\caption{Staircase MI estimation comparison for $d=1$ and $N=64$. Top row: Uniform scenario. Bottom row: Student scenario}
	\label{fig:MI_stairs_d5_bs64_beyond2}
\end{figure} 

Figs. \ref{fig:MI_stairs_d5_bs64_beyond} and \ref{fig:MI_stairs_d5_bs64_beyond2} demonstrate the efficacy of $f$-DIME in all the analyzed scenarios.

\subsubsection{Computational Time Analysis}
\label{subsubsec:mi_time_analysis}
\begin{figure}
\centering
\begin{subfigure}{.5\textwidth}
  \centering
  \includegraphics[scale=0.4]{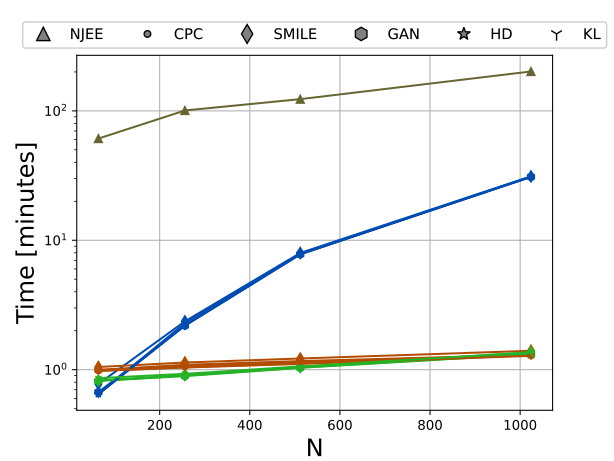}
  \caption{Multivariate Gaussian distribution size \\ fixed to $20$. Batch size varying from $64$ to $1024$.}
  \label{fig:MI_timeVaryingBS}
\end{subfigure}%
\begin{subfigure}{.5\textwidth}
  \centering
  \includegraphics[scale=0.4]{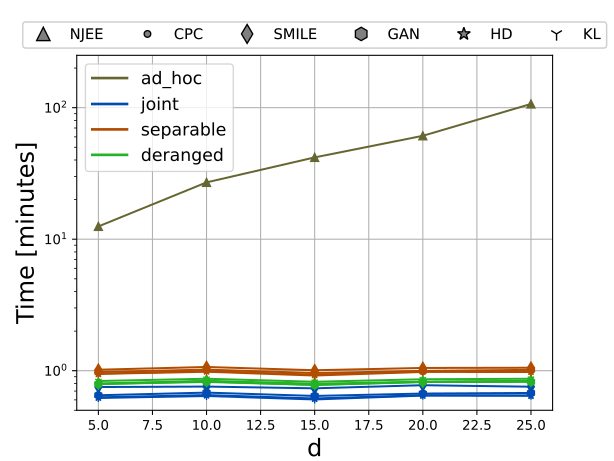}
  \caption{Multivariate Gaussian distribution size \\ varying from $5$ to $25$. Batch size fixed to $64$.}
  \label{fig:MI_timeVaryingD}
\end{subfigure}
\caption{Time requirements comparison to complete the 5-step staircase MI.}
\label{fig:MI_computationalTimeAnalysisMain}
\end{figure}

A fundamental characteristic of each algorithm is the computational time. 
The time requirements to complete the 5-step staircase MI when varying the multivariate Gaussian distribution dimension $d$ and the batch size $N$ are reported in Fig. \ref{fig:MI_computationalTimeAnalysisMain}.
The difference among the estimators computational time is not significant, when comparing the same architectures.
However, as discussed in Sec. \ref{subsec:mi_architectures}, the \textit{deranged} strategy is significantly faster than the \textit{joint} one as $N$ increases. 
The fact that the \textit{separable} architecture uses two neural networks implies that when $N$ is significantly large, the \textit{deranged} implementation is much faster than the \textit{separable} one, as well as more stable to the training and distribution parameters, as shown in Sec. \ref{subsec:mi_appendix_experiment_details}. 

$I_{NJEE}$ is evaluated with its own architecture, which is the most computationally demanding, because it trains a number of neural networks equal to $2d-1$. Thus, $I_{NJEE}$ can be utilized only in cases where the time availability is orders of magnitudes higher than the other approaches considered.
When $d$ is large, the training of $I_{NJEE}$ fails due to memory requirement problems. For example, our hardware platform (described Sec. \ref{subsec:mi_appendix_experiment_details}) does not allow the usage of $d>30$.

\section{Improved autoencoding schemes}
\sectionmark{f-DIME autoencoders}
\label{sec:fDIME_autoencoders}
In this section, we quickly discuss how to adapt the findings of $f$-DIME into the rate-driven AE architecture proposed in Sec. \ref{sec:cacao}. In particular, we show how the choice of the MI estimator impacts the training of the AE.  

\subsection{$\gamma$-DIME}
\label{subsec:mi_gammaDIME}
Only a subset of $\mathcal{J}_f$ in \eqref{eq:MI_discriminator_function_f} can be used inside the rate-driven autoencoder loss function described in Sec. \ref{sec:cacao}. Indeed, a maximization of the MI is required and only the generator $f$ of the KL divergence is suitable for the MI maximization, since in such case
\begin{equation}
\max_{p_X(x)} D_f = \max_{p_X(x)} D_{KL} = \max_{p_X(x)} I(X;Y).
\end{equation}
The generator function of the KL divergence is $f(u) = u\log u$ with conjugate $f^*(t) = e^{t-1}$, as shown when described KL-DIME estimator $I_{KL-DIME}$. Nevertheless, as shown in \cite{LSGAN}, the value function plays a fundamental role during the training process of a discriminator. Consequently, we used the approach presented for the derivation of $f$-DIME to obtain a family of lower bounds specifically on the MI that may result in more robust training processes and estimators according to specific case studies.

\begin{lemma}
\label{lemma:MI_LemmaKL}
Let $X\sim p_X(x)$ and $Y\sim p_{Y}(y|x)$ be the channel input and output, respectively. Let $D(\cdot)$ be a positive function. If $\mathcal{J}_{\gamma}(D)$, $\gamma>0$, is a value function defined as 
\begin{equation}
\label{eq:MI_value_function_f}
\mathcal{J}_{\gamma}(D) = \gamma \cdot \mathbb{E}_{(x,y) \sim p_{XY}(x,y)}\biggl[\log \biggl(D\bigl(x,y\bigr)\biggr)\biggr] + \mathbb{E}_{(x,y) \sim p_{X}(x)p_{Y}(y)}\biggl[- D^{\gamma}\bigl(x,y\bigr)\biggr],
\end{equation}
then
\begin{equation}
I(X;Y) \geq  \mathcal{J}_{\gamma}(D^*)+1,
\end{equation}
\begin{equation}
\label{eq:MI_gammaDIME}
I_{\gamma DIME} = \gamma \cdot \mathbb{E}_{(x,y) \sim p_{XY}(x,y)}\bigl[\log D^*(x,y) \bigr],
\end{equation}
where
\begin{equation}
\label{eq:MI_optimal_discriminator_gamma}
D^*(x,y) = \biggl(\frac{p_{XY}(x,y)}{p_{X}(x)\cdot p_Y(y)}\biggr)^{1/\gamma} = \arg \max_D \mathcal{J}_{\gamma}(D).
\end{equation}
\end{lemma}

\begin{proof}
Consider a scaled generator $f(u) = \frac{u}{\gamma}\log u$ and for simplicity of notation, denote $p_{XY}$ and $p_Xp_Y$ with $p$ and $q$, respectively. Then
\begin{equation}
D_{KL}(p||q) = \gamma \int_{x}{q(x) \frac{p(x)}{\gamma q(x)}\log\biggl(\frac{p(x)}{q(x)}\biggr) \diff x},
\end{equation}
with conjugate $f^*(t)$, with $t \in \mathbb{R}$ given by
\begin{equation}
f^*(t) = \frac{e^{\gamma t -1}}{\gamma}.
\end{equation}
Substituting in \eqref{eq:MI_f_bound} yields to
\begin{equation}
D_{KL}(p||q) \geq \gamma \sup_{T\in \mathbb{R}} \biggl\{ \mathbb{E}_{x \sim p(x)} \bigl[T(x)\bigr]-\frac{1}{\gamma}\mathbb{E}_{x\sim q(x)}\bigl[e^{\gamma T(x)-1} \bigr]\biggr\}.
\end{equation}
Using \eqref{eq:MI_optimal_ratio_T} it is easy to verify that the optimal value of $T$ is the log-likelihood ratio rather than the density ratio itself. Indeed,
\begin{equation}
T^*(x) = \frac{1}{\gamma}\log\biggl(\frac{p(x)}{q(x)}\biggr) + \frac{1}{\gamma}.
\end{equation}
Finally, with the change of variable $T(x)=\log(D(x))+1/{\gamma}$, the optimal discriminator has form
\begin{equation}
D^*(x) = \biggl(\frac{p(x)}{q(x)}\biggr)^{1/\gamma}
\end{equation}
and
\begin{equation}
D_{KL}(p||q) \geq \sup_{D\in \mathbb{R}_+} \biggl\{ \gamma \mathbb{E}_{x \sim p(x)} \bigl[\log \bigl( D(x) \bigr) \bigr] -\mathbb{E}_{x\sim q(x)}\bigl[D^{\gamma}(x) \bigr]\biggr\} + 1.
\end{equation}
Denoting the two terms inside the supremum with $J_{\gamma}(D)$ and replacing $p$ and $q$ with the joint and marginal distributions, concludes the proof.
\end{proof}

\begin{figure}
\centering
  	\includegraphics[scale = 0.3]{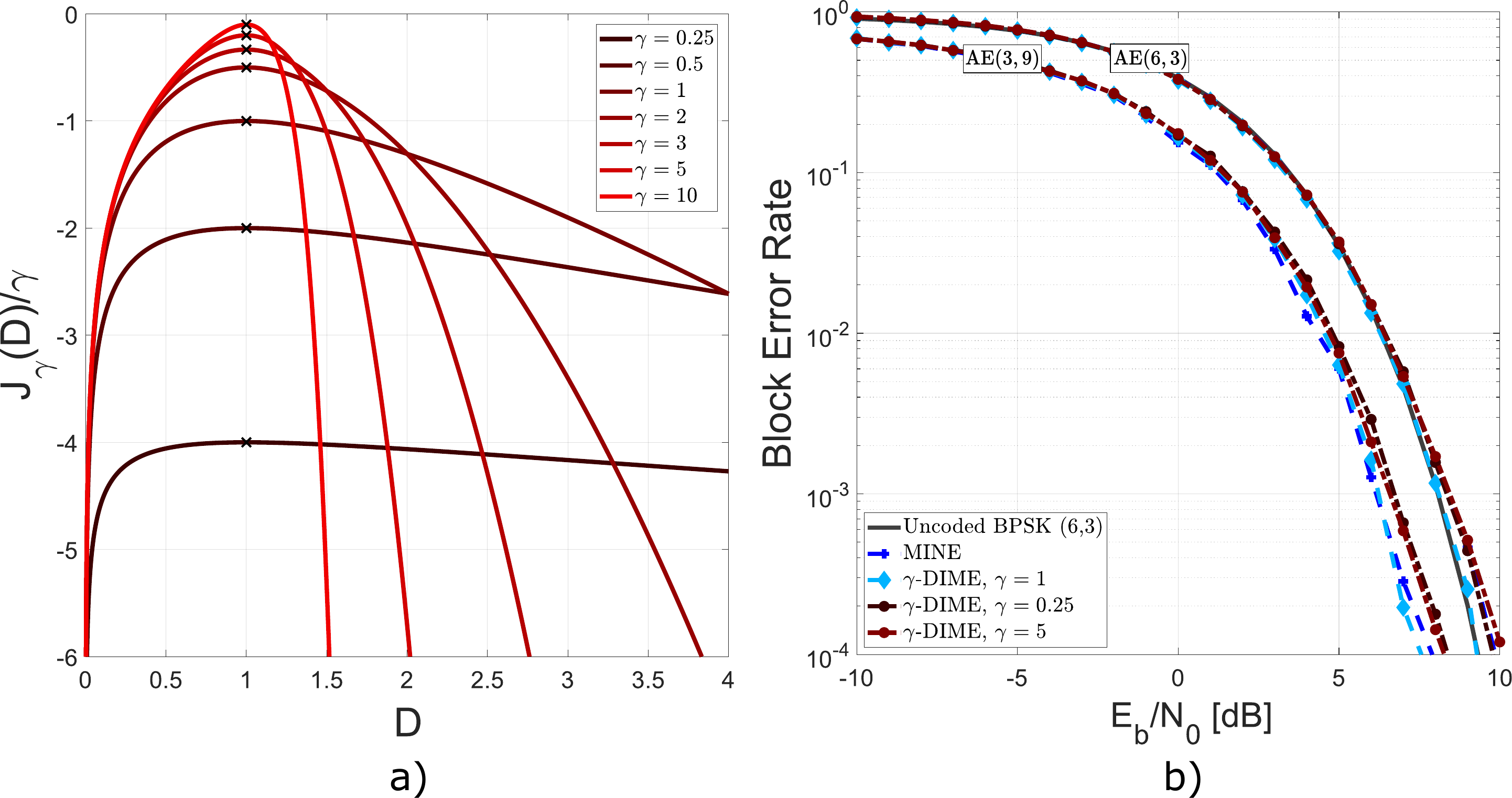}
  	\caption{a) Behaviour of the functional $J_{\gamma}$ varying the parameter $\gamma$. Concavity of the neighbourhood of the maximal value depends on $\gamma$. b) Comparison of the BLER obtained by the AE($3,9$) and AE($6,3$) with an AWGN intermediate layer, using different estimators with $\beta = 0.2$, $\epsilon = 0.2$.}
  	\label{fig:MI_gammaDIME}
\end{figure}

Fig.~\ref{fig:MI_gammaDIME}a illustrates how the value function in \eqref{eq:MI_value_function_f} (divided by $\gamma$) behaves according to the choice of the positive parameter $\gamma$. The position of the maximum depends on the density ratio as described in \eqref{eq:MI_optimal_discriminator_gamma} ($R = 1$ in the figure). The role of $\gamma$ is to modify the concavity of the value function to improve the convergence rate of the gradient ascent algorithm. Moreover, when $D$ is parameterized by a neural network, the choice of the learning rate and activation functions may be related with $\gamma$. Indeed, notice that $\gamma$-DIME can also be derived from d-DIME using the substitution $D(x,y) \rightarrow D^{\gamma}(x,y)$.

\subsection{Results}
\label{subsec:mi_gammaDIME_results}
In the experimental results, we consider the transmission of $M$ messages at rate $R$ over an AWGN channel, for which we know the channel capacity under an
average power constraint. The choice of AWGN channel shall not be seen as a limitation since the formulation transcends the channel characteristics. Moreover, it provides a reference curve for the validation of the estimators and it allows comparisons across several application domains.

We describe two scenarios: a) the transmission at a rate $R>1$ and short code-length $n$ and b) the transmission at a low rate with larger code-length. For both scenarios, we train a rate-driven AE and compare its performance for different estimators in terms of BLER and accuracy of the mutual information estimation. 

For all the experiments and as a proof of concept, we use simple encoder, decoder and neural MI estimator architectures. In particular, the encoder is a shallow neural network with $M$ neurons in the input layer, $M$ in the hidden one and $2n$ neurons at the output while the decoder is its complementary.  
The discriminative estimator possesses instead a fixed architecture, a
two layer multilayer perceptron NN of $200$
neurons in each layer and LeakyReLU with a negative slope coefficient of $0.2$ as activation function. The only
difference in the estimators resides in the final layer where we use a linear dense layer for MINE, while a softplus layer for the DIME based estimators.
The training of the AE was conducted at a fixed $E_b/N_0$ ratio of $7$ dB. We minimized the loss function alternating encoder and decoder weights update with the discriminator update. The number of iterations was set to $10$k with a learning rate of $0.01$ for both the components. For the implementation details, we refer to the GitHub code \cite{CACAO_github}.

The first scenario describes a digital transmission scheme with input alphabet size $M=64$ over $3$ channel uses. We denote such system as AE($6,3$) since $k = \log_2(M)$. The communication rate is $R=2$. Fig.~\ref{fig:MI_gammaDIME}b shows the AE performance in terms of BLER for different estimators. From the curves we can infer that the $\gamma$-DIME estimator with $\gamma=1$ provides improved performance in terms of BLER compared to both MINE and other values of $\gamma$. The result of Fig. \ref{fig:MI_MI_6_6}a is insightful: the $\gamma$-DIME estimators provide stable estimations also at high SNRs where it is expected that the MI saturates to the information rate $R=2$. Thus, the DIME based estimators are more accurate compared to the divergent MINE (see Theorem 2 of \cite{Song2020}). It is also interesting to notice that the concavity of $J_{\gamma}$ (shown in Fig.~\ref{fig:MI_gammaDIME}a) influences the MI estimation as it appears in the results. Indeed, for low values of $\gamma$, fixed learning rate and training iterations, the estimators are trained by moving on the flat curves of Fig.~\ref{fig:MI_gammaDIME}a, which is related to using a small learning rate to upgrade the gradients. For completeness and to show the advantage of the proposed estimators under fading, we report in Fig. \ref{fig:MI_MI_6_6}b the achieved estimated information rate with Rayleigh fading. Also in this case, MINE diverges for high SNR values while the DIME based estimators are much more accurate and tend to saturate to the rate $R=2$.  In addition, $\gamma =2$ seems to perform better for low and middle SNR values, which is consistent with the idea that different scenarios may require different parameters and therefore $\gamma$-DIME provides more flexibility for model learning.

\begin{figure}
\centering
  	\includegraphics[scale=0.2]{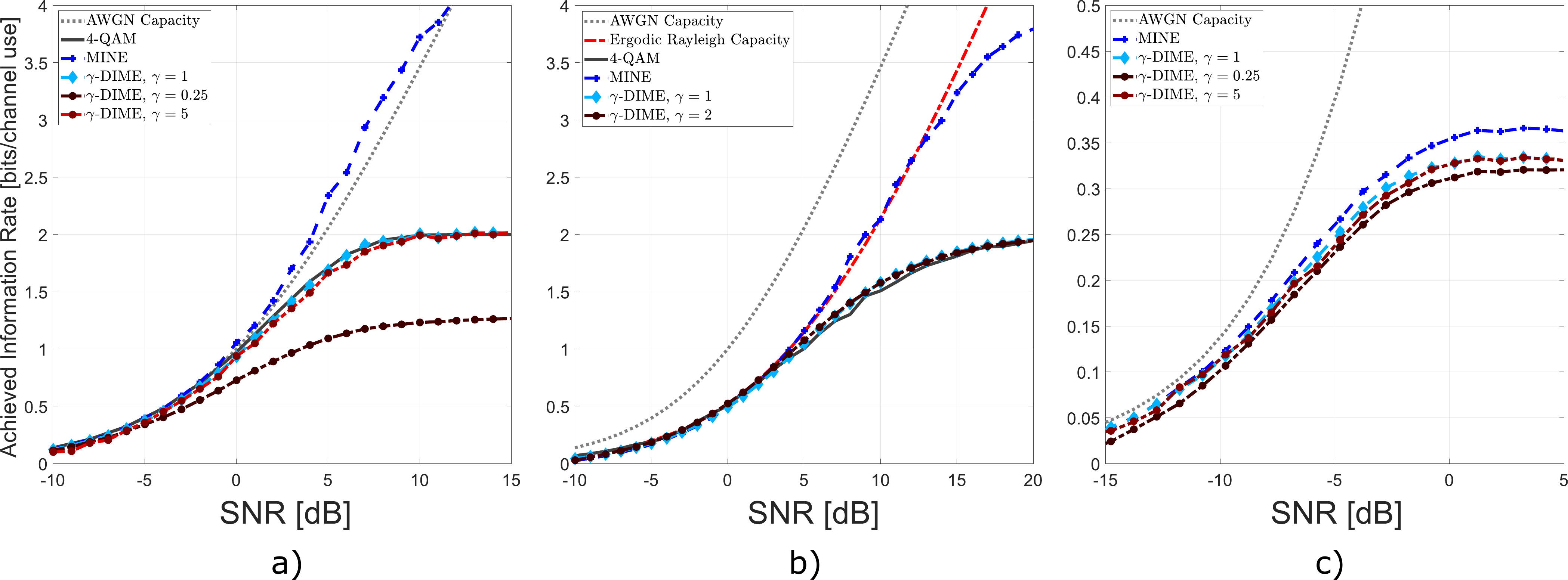}
  	\caption{Comparison of different MI estimators for the AE($6,3$) with $\beta = 0.2$, $\epsilon = 0.2$ on a) the AWGN channel; b) on the Rayleigh fading channel. c) Comparison of the achieved information rate using different MI estimators for the AE($3,9$) with $\beta = 0.2$, $\epsilon = 0.2$ on the AWGN channel.}
  	\label{fig:MI_MI_6_6}
\end{figure}

The second scenario describes a digital transmission scheme with input alphabet size $M=8$ over $9$ channel uses. We denote such system as AE($3,9$). The communication rate is $R=1/3$. Fig.~\ref{fig:MI_gammaDIME}b shows that MINE and $\gamma$-DIME with $\gamma=1$ have similar performance in terms of BLER under the same training iterations. However, Fig.~\ref{fig:MI_MI_6_6}c highlights the fact that even for a low alphabet dimension and for relatively small $n$, the DIME-based estimators are more accurate than MINE since they tend to saturate to $R=1/3$ while MINE overestimates the MI. We can conclude that despite MINE being a tighter bound on the MI compared to $\gamma$-DIME, the latter and in particular the estimator with $\gamma=1$ performs much better than MINE in terms of MI estimation. We leave comments about the architecture to use for future studies, although we remark that it is one of the main aspects in modern ML applications and therefore we expect significant improvements. 

\section{Summary}
\label{sec:mi_conclusions}
In this chapter, we presented $f$-DIME, a class of discriminative MI estimators based on the variational representation of the $f$-divergence. We proved that any valid choice of the function $f$ leads to a low-variance MI estimator which can be parametrized by a NN. We also proposed a derangement training strategy that efficiently samples from the product of marginal distributions. 
The performance of $f$-DIME is evaluated using three functions $f$, and it is compared with state-of-the-art estimators. Results demonstrate excellent bias/variance trade-off for different data dimensions and different training parameters.

When $f$ is the function generator of the KL divergence, we obtained as a special case $\gamma$-DIME, a parametric estimator that can be used both for MI estimation and design of optimal channel coding. We integrated it as the MI regularization term inside the AE loss
function. A comparison in terms of estimation accuracy and
BLER was done with MINE. Results show the advantage of adopting $\gamma$-DIME over MINE for the AE end-to-end training, especially when the aim is to estimate the achieved information rate. 

In the next chapter, we extend this idea to create a framework that builds capacity-achieving distributions and that estimates the channel capacity for any discrete-time continuous memoryless channel.

\newpage
\section{Supplementary material}
\sectionmark{Proofs}
\label{sec:mi_appendix}

\subsection{DIME estimators}
\label{subsec:mi_appendix_DIME_examples}
In this section, we provide a concrete list of DIME estimators obtained using three different $f$-divergences. In particular, we maximize the value function defined in \eqref{eq:MI_discriminator_function_f}
\begin{equation}
    \mathcal{J}_{f}(T) =  \mathbb{E}_{(\mathbf{x},\mathbf{y}) \sim p_{XY}(\mathbf{x},\mathbf{y})}\biggl[T\bigl(\mathbf{x},\mathbf{y}\bigr)-f^*\biggl(T\bigl(\mathbf{x},\sigma(\mathbf{y})\bigr)\biggr)\biggr],
\end{equation}
over $T$ or its transformation. By doing that, and using \eqref{eq:MI_f-DIME}, 
\begin{equation}
    I(X;Y) = I_{fDIME}(X;Y) =  \mathbb{E}_{(\mathbf{x},\mathbf{y}) \sim p_{XY}(\mathbf{x},\mathbf{y})}\biggl[ \log \biggl(\bigl(f^{*}\bigr)^{\prime}\bigl(\hat{T}(\mathbf{x},\mathbf{y})\bigr) \biggr) \biggr],
\end{equation}
we obtain a list of three different MI estimators.
The list is used for both commenting on the impact of the $f$ function, referred to as the generator function, and for comparing the estimators discussed in Sec. \ref{sec:mi_related}. 

We consider the cases when $f$ is the generator of: 
\begin{itemize}
    \item[{a)}] the KL divergence;
    \item[{b)}] the GAN divergence; 
    \item[{c)}] the Hellinger distance squared.
\end{itemize}
We report below the derived value functions and the mathematical expressions of the proposed estimators. 

\subsubsection{KL divergence}
The variational representation of the KL divergence \cite{Nguyen2010} leads to the NWJ estimator in \eqref{eq:MI_NWJ} when $f(u) = u\log(u)$. However, since we are interested in extracting the density ratio, we apply the transformation $T(\mathbf{x})=\log(D(\mathbf{x}))$. In this way, the lower bound introduced in \eqref{eq:MI_discriminator_function_f} reads as follows
\begin{equation}
\mathcal{J}_{KL}(D) = \mathbb{E}_{(\mathbf{x},\mathbf{y}) \sim p_{XY}(\mathbf{x},\mathbf{y})}\biggl[\log\bigl(D\bigl(\mathbf{x},\mathbf{y}\bigr)\bigr)\biggr] -\mathbb{E}_{(\mathbf{x},\mathbf{y}) \sim p_{X}(\mathbf{x})p_{Y}(\mathbf{y})}\biggl[D\bigl(\mathbf{x},\mathbf{y}\bigr)\biggr]+1,
\end{equation}
which has to be maximized over positive discriminators $D(\cdot)$.
As remarked before, we do not use $\mathcal{J}_{KL}$ during the estimation, rather we define the KL-DIME estimator as
\begin{equation}
\label{eq:MI_KL-DIME}
I_{KL-DIME}(X;Y) :=  \mathbb{E}_{(\mathbf{x},\mathbf{y}) \sim p_{XY}(\mathbf{x},\mathbf{y})}\biggl[ \log \biggl(\hat{D}(\mathbf{x},\mathbf{y})\biggr) \biggr],
\end{equation}
due to the fact that
\begin{equation}
\label{eq:MI_optimal_ratio_KL}
\hat{D}(\mathbf{x},\mathbf{y}) =\arg \max_D \mathcal{J}_{KL}(D) = \frac{p_{XY}(\mathbf{x},\mathbf{y})}{p_X(\mathbf{x})p_Y(\mathbf{y})}.
\end{equation}

\subsubsection{GAN divergence}
Following a similar approach, it is possible to define $f(u) = u\log u-(u+1)\log(u+1)+\log4$ and $T(\mathbf{x})=\log(1-D(\mathbf{x}))$. We derive from Theorem \ref{theorem:MI_theorem1} the GAN-DIME estimator obtained via maximization of
\begin{equation}
\mathcal{J}_{GAN}(D) = \mathbb{E}_{(\mathbf{x},\mathbf{y}) \sim p_{XY}(\mathbf{x},\mathbf{y})}\biggl[\log\bigl(1-D\bigl(\mathbf{x},\mathbf{y}\bigr)\bigr)\biggr] +\mathbb{E}_{(\mathbf{x},\mathbf{y}) \sim p_{X}(\mathbf{x})p_{Y}(\mathbf{y})}\biggl[\log\bigl(D\bigl(\mathbf{x},\mathbf{y}\bigr)\bigr)\biggr]+\log(4).
\end{equation}
In fact, at the equilibrium we recover \eqref{eq:MI_gan_density_ratio}, hence
\begin{equation}
\label{eq:MI_GAN-DIME}
I_{GAN-DIME}(X;Y) :=  \mathbb{E}_{(\mathbf{x},\mathbf{y}) \sim p_{XY}(\mathbf{x},\mathbf{y})}\biggl[ \log \biggl(\frac{1-\hat{D}(\mathbf{x},\mathbf{y})}{\hat{D}(\mathbf{x},\mathbf{y})}\biggr) \biggr].
\end{equation}

\subsubsection{Hellinger distance}
The last example we consider is the generator of the Hellinger distance squared $f(u)=(\sqrt{u}-1)^2$ with the change of variable $T(\mathbf{x})=1-D(\mathbf{x})$. After simple manipulations, we obtain the associated value function as
\begin{equation}
\mathcal{J}_{HD}(D) = 2-\mathbb{E}_{(\mathbf{x},\mathbf{y}) \sim p_{XY}(\mathbf{x},\mathbf{y})}\biggl[D\bigl(\mathbf{x},\mathbf{y}\bigr)\biggr] -\mathbb{E}_{(\mathbf{x},\mathbf{y}) \sim p_{X}(\mathbf{x})p_{Y}(\mathbf{y})}\biggl[\frac{1}{D(\mathbf{x},\mathbf{y})}\biggr],
\end{equation}
which is maximized for
\begin{equation}
\label{eq:MI_optimal_ratio_HD}
\hat{D}(\mathbf{x},\mathbf{y}) =\arg \max_D \mathcal{J}_{HD}(D) = \sqrt{\frac{p_X(\mathbf{x})p_Y(\mathbf{y})}{p_{XY}(\mathbf{x},\mathbf{y})}},
\end{equation}
leading to the HD-DIME estimator
\begin{equation}
\label{eq:MI_HD-DIME}
I_{HD-DIME}(X;Y) :=  \mathbb{E}_{(\mathbf{x},\mathbf{y}) \sim p_{XY}(\mathbf{x},\mathbf{y})}\biggl[ \log \biggl(\frac{1}{\hat{D}^2(\mathbf{x},\mathbf{y})}\biggr) \biggr].
\end{equation}

Given that these estimators comprise only one expectation over the joint samples, their variance has different properties compared to the variational ones requiring the partition term such as MINE and NWJ.

\subsection{Proofs of lemmas and theorems}
\label{subsec:mi_omitted_proofs}
\begin{theorem}[Related to Theorem~\ref{theorem:MI_theorem1}]
Let $(X,Y) \sim p_{XY}(\mathbf{x},\mathbf{y})$ be a pair of random variables. Let $\sigma(\cdot)$ be a permutation function such that  $p_{\sigma(Y)}(\sigma(\mathbf{y})|\mathbf{x}) = p_{Y}(\mathbf{y})$. Let $f^*$ be the Fenchel conjugate of $f:\mathbb{R}_+ \to \mathbb{R}$, a convex lower semicontinuous function that satisfies $f(1)=0$ with derivative $f^{\prime}$.
If $\mathcal{J}_{f}(T)$ is a value function defined as 
\begin{equation}
\mathcal{J}_{f}(T) =  \mathbb{E}_{(\mathbf{x},\mathbf{y}) \sim p_{XY}(\mathbf{x},\mathbf{y})}\biggl[T\bigl(\mathbf{x},\mathbf{y}\bigr)-f^*\biggl(T\bigl(\mathbf{x},\sigma(\mathbf{y})\bigr)\biggr)\biggr],
\end{equation}
then
\begin{equation}
\hat{T}(\mathbf{x},\mathbf{y}) =\arg \max_T \mathcal{J}_f(T) = f^{\prime} \biggl(\frac{p_{XY}(\mathbf{x},\mathbf{y})}{p_X(\mathbf{x})p_Y(\mathbf{y})}\biggr),
\end{equation}
and
\begin{equation}
I(X;Y) = I_{fDIME}(X;Y) =  \mathbb{E}_{(\mathbf{x},\mathbf{y}) \sim p_{XY}(\mathbf{x},\mathbf{y})}\biggl[ \log \biggl(\bigl(f^{*}\bigr)^{\prime}\bigl(\hat{T}(\mathbf{x},\mathbf{y})\bigr) \biggr) \biggr].
\end{equation}
\end{theorem}

\begin{proof} 
From the hypothesis, the value function can be rewritten as
\begin{equation}
\mathcal{J}_{f}(T) =  \mathbb{E}_{(\mathbf{x},\mathbf{y}) \sim p_{XY}(\mathbf{x},\mathbf{y})}\biggl[T\bigl(\mathbf{x},\mathbf{y}\bigr)\biggr]  -\mathbb{E}_{(\mathbf{x},\mathbf{y}) \sim p_{X}(\mathbf{x})p_{Y}(\mathbf{y})}\biggl[f^*\biggl(T\bigl(\mathbf{x},\mathbf{y}\bigr)\biggr)\biggr].
\end{equation}
The thesis follows immediately from Lemma 1 of \cite{Nguyen2010}. Indeed, the $f$-divergence $D_f$ can be expressed in terms of its lower bound via Fenchel convex duality
\begin{equation}
\label{eq:MI_f_bound}
D_f(P||Q) \geq \sup_{T\in \mathbb{R}} \biggl\{ \mathbb{E}_{x \sim p(x)} \bigl[T(x)\bigr]-\mathbb{E}_{x\sim q(x)}\bigl[f^*\bigl(T(x)\bigr)\bigr]\biggr\},
\end{equation}
where $T: \mathcal{X} \to \mathbb{R}$ and $f^*$ is the Fenchel conjugate of $f$ defined as
\begin{equation}
f^*(t) := \sup_{u\in \mathbb{R}} \{ ut -f(u)\}.
\end{equation}
Therein, it was shown that the bound in \eqref{eq:MI_f_bound} is tight for optimal values of $T(x)$ and it takes the following form
\begin{equation}
\hat{T}(x) = f^{\prime} \biggl(\frac{p(x)}{q(x)}\biggr),
\end{equation}
where $f^{\prime}$ is the derivative of $f$.

The mutual information $I(X;Y)$ admits the KL divergence representation 
\begin{equation}
I(X;Y) = D_{KL}(p_{XY}||p_X p_Y),
\end{equation}
and since the inverse of the derivative of $f$ is the derivative of the conjugate $f^*$, the density ratio can be rewritten in terms of the optimum discriminator $\hat{T}$
\begin{equation}
\bigl(f^{\prime}\bigr)^{-1}\bigl(\hat{T}(\mathbf{x},\mathbf{y})\bigr) = \bigl(f^*\bigr)^{\prime}\bigl(\hat{T}(\mathbf{x},\mathbf{y})\bigr) = \frac{p_{XY}(\mathbf{x},\mathbf{y})}{p_X(\mathbf{x})p_Y(\mathbf{y})}.
\end{equation}
$f$-DIME finally reads as follows
\begin{equation}
I_{fDIME}(X;Y) = \mathbb{E}_{(\mathbf{x},\mathbf{y}) \sim p_{XY}(\mathbf{x},\mathbf{y})}\biggl[ \log \biggl(\bigl(f^*\bigr)^{\prime}\bigl(\hat{T}(\mathbf{x},\mathbf{y})\bigr) \biggr) \biggr].
\end{equation}
\end{proof}

\begin{lemma}[Related to Lemma~\ref{lemma:MI_convergence}]
Let the discriminator $T(\cdot)$ be with enough capacity, i.e., in the non parametric limit. Consider the problem
\begin{equation}
\hat{T} =  \; \arg \max_T \mathcal{J}_{f}(T)
\end{equation}
where
\begin{equation}
\mathcal{J}_{f}(T) = \mathbb{E}_{(\mathbf{x},\mathbf{y}) \sim p_{XY}(\mathbf{x},\mathbf{y})}\biggl[T\bigl(\mathbf{x},\mathbf{y}\bigr)\biggr] -\mathbb{E}_{(\mathbf{x},\mathbf{y}) \sim p_{X}(\mathbf{x})p_{Y}(\mathbf{y})}\biggl[f^*\biggl(T\bigl(\mathbf{x},\mathbf{y}\bigr)\biggr)\biggr],
\end{equation}
and the update rule based on the gradient descent method
\begin{equation}
T^{(n+1)} = T^{(n)} + \mu \nabla \mathcal{J}_{f}(T^{(n)}).
\end{equation}
If the gradient descent method converges to the global optimum $\hat{T}$, the mutual information estimator 
\begin{equation}
I(X;Y) = I_{fDIME}(X;Y) =  \mathbb{E}_{(\mathbf{x},\mathbf{y}) \sim p_{XY}(\mathbf{x},\mathbf{y})}\biggl[ \log \biggl(\bigl(f^{*}\bigr)^{\prime}\bigl(\hat{T}(\mathbf{x},\mathbf{y})\bigr) \biggr) \biggr].
\end{equation}
converges to the real value of the mutual information $I(X;Y)$.
\end{lemma}

\begin{proof}
For convenience of notation, let the instantaneous mutual information be the random variable defined as
\begin{equation}
i(X;Y) := \log\biggl(\frac{p_{XY}(\mathbf{x},\mathbf{y})}{p_X(\mathbf{x})p_Y(\mathbf{y})}\biggr).
\end{equation}
It is straightforward to notice that the MI corresponds to the expected value of $i(X;Y)$ over the joint distribution $p_{XY}$.
The solution to \eqref{eq:MI_Lemma4_problem} is given by \eqref{eq:MI_optimal_ratio_T} of Theorem \ref{theorem:MI_theorem1}. Let $\delta^{(n)}=\hat{T}-T^{(n)}$ be the displacement between the optimum discriminator $\hat{T}$ and the obtained one $T^{(n)}$ at the iteration $n$, then
\begin{equation}
\hat{i}_{n,fDIME}(X;Y)  = \log \biggl(\bigl(f^{*}\bigr)^{\prime}\bigl(T^{(n)}(\mathbf{x},\mathbf{y}) \bigr) \biggr) = \log \biggl(R^{(n)}(\mathbf{x},\mathbf{y}) \biggr),
\end{equation}
where $R^{(n)}(\mathbf{x},\mathbf{y})$ represents the estimated density ratio at the $n$-th iteration and it is related with the optimum ratio $\hat{R}(\mathbf{x},\mathbf{y})$ as follows
\begin{align}
\hat{R} - R^{(n)} & = \bigl(f^{*}\bigr)^{\prime}\bigl(\hat{T}\bigr) - \bigl(f^{*}\bigr)^{\prime}\bigl(T^{(n)}\bigr) \nonumber \\  
& = \bigl(f^{*}\bigr)^{\prime}\bigl(\hat{T}\bigr) - \bigl(f^{*}\bigr)^{\prime}\bigl(\hat{T}-\delta^{(n)}\bigr)  \nonumber \\
& \simeq  \delta^{(n)}\cdot\biggl[\bigl(f^{*}\bigr)^{\prime \prime}\bigl(\hat{T}-\delta^{(n)}\bigr)\biggr],
\end{align}
where the last step follows from a first order Taylor expansion in $\hat{T}-\delta^{(n)}$.
Therefore,
\begin{align}
& \hat{i}_{n,fDIME}(X;Y) = \log \bigl(R^{(n)}\bigr) \nonumber \\
 & =  \log \Biggl(\bigl(\hat{R}\bigr)\biggl(1-\delta^{(n)}\cdot \frac{\bigl(f^{*}\bigr)^{\prime \prime}\bigl(\hat{T}-\delta^{(n)}\bigr)}{\bigl(f^{*}\bigr)^{\prime}\bigl(\hat{T}\bigr)}  \biggr)\Biggr) \nonumber \\ 
& = i(X;Y) + \log \Biggl(1-\delta^{(n)}\cdot \frac{\bigl(f^{*}\bigr)^{\prime \prime}\bigl(\hat{T}-\delta^{(n)}\bigr)}{\bigl(f^{*}\bigr)^{\prime}\bigl(\hat{T}\bigr)}\Biggr).
\end{align}

If the gradient descent method converges towards the optimum solution $\hat{T}$, $\delta^{(n)} \rightarrow 0$ and 
\begin{align}
& \hat{i}_{n,fDIME}(X;Y) \simeq i(X;Y) - \delta^{(n)} \cdot \Biggl[\frac{\bigl(f^{*}\bigr)^{\prime \prime}\bigl(\hat{T}-\delta^{(n)}\bigr)}{\bigl(f^{*}\bigr)^{\prime}\bigl(\hat{T}\bigr)} \Biggr] \nonumber \\ 
& \simeq i(X;Y) - \delta^{(n)} \cdot \Biggl[\frac{\bigl(f^{*}\bigr)^{\prime \prime}\bigl(\hat{T}\bigr)}{\bigl(f^{*}\bigr)^{\prime}\bigl(\hat{T}\bigr)} \Biggr] \nonumber \\
& = i(X;Y) - \delta^{(n)} \cdot \Biggl[\frac{\mathrm{d}}{\mathrm{d}T} \log \bigl( \bigl(f^{*}\bigr)^{\prime}(T)\bigr) \biggr|_{T=\hat{T}} \Biggr],
\end{align}
where the RHS is itself a first order Taylor expansion of the instantaneous mutual information in $\hat{T}$. 
In the asymptotic limit ($n\rightarrow +\infty$), it holds also for the expected values that 
\begin{equation}
|I(X;Y)-\hat{I}_{n,fDIME}(X;Y)|\rightarrow 0.
\end{equation}
\end{proof}

\begin{lemma}[Related to Lemma~\ref{lemma:MI_Lemma2}]
Let $\hat{R} = p_{XY}(\mathbf{x},\mathbf{y})/ (p_{X}(\mathbf{x}) p_Y(\mathbf{y}))$ and assume $\text{Var}_{p_{XY}}[\log \hat{R}]$ exists. Let $p^M_{XY}$ be the empirical distribution of $M$ i.i.d. samples from $p_{XY}$ and let $\mathbb{E}_{p^M_{XY}}$ denote the sample average over $p^M_{XY}$. Then, it holds that

\begin{equation}
\text{Var}_{p_{XY}}\bigl[\mathbb{E}_{p^M_{XY}} [\log \hat{R}] \bigr] \leq \frac{ 4H^2(p_{XY},p_Xp_Y)\biggl|\biggl|\frac{p_{XY}}{p_Xp_Y}\biggr|\biggr|_{\infty}-I^2(X;Y)}{M}
\end{equation}
where $H^2$ is the Hellinger distance squared.
\end{lemma}

\begin{proof}
Consider the variance of $\hat{R}(\mathbf{x},\mathbf{y})$ when $(\mathbf{x},\mathbf{y})\sim p_{XY}(\mathbf{x},\mathbf{y})$, then
\begin{equation}
\text{Var}_{p_{XY}}[\log \hat{R}] = \mathbb{E}_{p_{XY}}\biggl[\biggl(\log \frac{p_{XY}}{p_Xp_Y}\biggr)^2\biggr] - \biggl(\mathbb{E}_{p_{XY}}\biggl[\log \frac{p_{XY}}{p_Xp_Y}\biggr]\biggr)^2.
\end{equation}
The power of the log-density ratio is upper bounded as follows (see the approach of Lemma 8.3 in \cite{Ghosal2000})
\begin{equation}
\mathbb{E}_{p_{XY}}\biggl[\biggl(\log \frac{p_{XY}}{p_Xp_Y}\biggr)^2\biggr] \leq 4H^2(p_{XY},p_Xp_Y)\biggl|\biggl|\frac{p_{XY}}{p_Xp_Y}\biggr|\biggr|_{\infty},
\end{equation}
while the mean squared is the ground-truth mutual information squared, thus
\begin{equation}
\text{Var}_{p_{XY}}[\log \hat{R}] \leq 4H^2(p_{XY},p_Xp_Y)\biggl|\biggl|\frac{p_{XY}}{p_Xp_Y}\biggr|\biggr|_{\infty} - I^2(X;Y).
\end{equation}
Finally, the variance of the mean of $M$ i.i.d. random variables yields the thesis
\begin{align}
\text{Var}_{p_{XY}}\bigl[\mathbb{E}_{p^M_{XY}} [\log \hat{R}] \bigr]  = \frac{\text{Var}_{p_{XY}}[\log \hat{R}]}{M}  \leq \frac{ 4H^2(p_{XY},p_Xp_Y)\biggl|\biggl|\frac{p_{XY}}{p_Xp_Y}\biggr|\biggr|_{\infty}-I^2(X;Y)}{M}.
\end{align}
\end{proof}

\begin{lemma}[Related to Lemma~\ref{lemma:MI_lemma4}]
Let $(\mathbf{x}_i,\mathbf{y}_i)$, $\forall i\in \{1,\dots,N\}$, be $N$ data points. Let $\mathcal{J}_{f}(T)$ be the value function in \eqref{eq:MI_discriminator_function_f}. Let $\mathcal{J}_{f}^{\pi}(T)$ and $\mathcal{J}_{f}^{\sigma}(T)$ be numerical implementations of $\mathcal{J}_{f}(T)$ using a random permutation and a random derangement of $\mathbf{y}$, respectively. Denote with $K$ the number of points $\mathbf{y}_k$, with $k \in \{1,\dots, N\}$, in the same position after the permutation (i.e., the fixed points). Then
\begin{equation}
\mathcal{J}_{f}^{\pi}(T) \leq \frac{N-K}{N} \mathcal{J}_{f}^{\sigma}(T).
\end{equation}
\end{lemma}

\begin{proof}
Define $\mathcal{J}_{f}^{\pi}(T)$ as the Monte Carlo implementation of $\mathcal{J}_{f}(T)$ when using the permutation function $\pi(\cdot)$
\begin{equation}
    \mathcal{J}_{f}^{\pi}(T) = \frac{1}{N}\sum_{i=1}^{N}{T(\mathbf{x}_i,\mathbf{y}_i}) -  \frac{1}{N}\sum_{i=1}^{N}{f^{*}\bigl(T(\mathbf{x}_i,\mathbf{y}_j})\bigr),
\end{equation}
where the pair $(\mathbf{x}_i,\mathbf{y}_j)$ is obtained via a random permutation of the elements of $\mathbf{y}$ as $j=\pi(i)$, $\forall i \in \{1,\dots,N\}$.
Since $K$ is a non-negative integer representing the number of fixed points $i=\pi(i)$, the value function can be rewritten as
\begin{equation}
    \label{eq:MI_monte-carlo-1}
    \mathcal{J}_{f}^{\pi}(T) = \frac{1}{N}\sum_{i=1}^{N}{T(\mathbf{x}_i,\mathbf{y}_i)} -  \frac{1}{N}\sum_{i=1}^{K}{f^{*}\bigl(T(\mathbf{x}_i,\mathbf{y}_i)\bigr)} -
    \frac{1}{N}\sum_{i=1}^{N-K}{f^{*}\bigl(T(\mathbf{x}_i,\mathbf{y}_{j\neq i})\bigr)},
\end{equation}
which can also be expressed as
\begin{equation}
\label{eq:MI_51}
    \mathcal{J}_{f}^{\pi}(T) = \frac{1}{N}\sum_{i=1}^{K}{T(\mathbf{x}_i,\mathbf{y}_i)} + \frac{1}{N}\sum_{i=1}^{N-K}{T(\mathbf{x}_i,\mathbf{y}_i)} -  \frac{1}{N}\sum_{i=1}^{K}{f^{*}\bigl(T(\mathbf{x}_i,\mathbf{y}_i)\bigr)} -
    \frac{1}{N}\sum_{i=1}^{N-K}{f^{*}\bigl(T(\mathbf{x}_i,\mathbf{y}_{j\neq i})\bigr)}.
\end{equation}
In \eqref{eq:MI_51} it is possible to recognize that the second and last term of the RHS constitutes the numerical implementation of $\mathcal{J}_{f}(T)$ using a derangement strategy on $N-K$ elements, so that
\begin{equation}
    \mathcal{J}_{f}^{\pi}(T) = \frac{1}{N}\sum_{i=1}^{K}{T(\mathbf{x}_i,\mathbf{y}_i)} -\frac{1}{N}\sum_{i=1}^{K}{f^{*}\bigl(T(\mathbf{x}_i,\mathbf{y}_i)\bigr)} + \frac{N-K}{N} {J}_{f}^{\sigma}(T).
\end{equation}
However, by definition of Fenchel conjugate
\begin{equation}
    \frac{1}{N}\sum_{i=1}^{K}{T(\mathbf{x}_i,\mathbf{y}_i)-f^{*}\bigl(T(\mathbf{x}_i,\mathbf{y}_i)\bigr)} \leq 0,
\end{equation}
since for $t=1$
\begin{equation}
    u-f^*(u) \leq u - (ut -f(t)) = f(1) = 0.
\end{equation}
Hence, we can conclude that
\begin{equation}
    \mathcal{J}_{f}^{\pi}(T) \leq \frac{N-K}{N} {J}_{f}^{\sigma}(T).
\end{equation}
\end{proof}

\begin{lemma}[Related to Lemma~\ref{lemma:MI_finite_variance_gaus}]
Let $\hat{R}$ be the optimal density ratio and let  $X\sim \mathcal{N}(0,\sigma_X^2)$ and $N\sim \mathcal{N}(0,\sigma_N^2)$ be uncorrelated scalar Gaussian random variables such that $Y=X+N$. Assume $\text{Var}_{p_{XY}}[\log \hat{R}]$ exists. Let $p^M_{XY}$ be the empirical distribution of $M$ i.i.d. samples from $p_{XY}$ and let $\mathbb{E}_{p^M_{XY}}$ denote the sample average over $p^M_{XY}$. Then, it holds that

\begin{equation}
\text{Var}_{p_{XY}}\bigl[\mathbb{E}_{p^M_{XY}} [\log \hat{R}] \bigr] = \frac{ 1-e^{-2I(X;Y)}}{M}.
\end{equation}
\end{lemma}

\begin{proof}
From the hypothesis, the density ratio can be rewritten as $\hat{R} = p_N(y-x)/ p_Y(y)$ and the output variance is clearly equal to $\sigma_Y^2 = \sigma_X^2 + \sigma_N^2$. Notice that this is equivalent of having correlated random variables $X$ and $Y$ with correlation coefficient $\rho$, since it is enough to study the case $\sigma_X = \rho$ and $\sigma_N=\sqrt{1-\rho^2}$.

It is easy to verify via simple calculations that
\begin{align}
I(X;Y) = & \; \mathbb{E}_{p_{XY}} [\log \hat{R}] \nonumber \\
= & \log \frac{\sigma_Y}{\sigma_N} + \mathbb{E}_{p_{XY}}\biggl[ \frac{y^2}{2\sigma_Y^2} - \frac{(y-x)^2}{2\sigma_N^2}\biggr] \nonumber \\
= & \dots = \log \frac{\sigma_Y}{\sigma_N} = \frac{1}{2}\log\biggl(1+\frac{\sigma_X^2}{\sigma_N^2}\biggr)=-\frac{1}{2}\log\bigl(1-\rho^2\bigr).
\label{eq:MI_gaussians}
\end{align}
Similarly,
\begin{align}
& \text{Var}_{p_{XY}}\bigl[\log \hat{R} \bigr] = \mathbb{E}_{p_{XY}}\biggl[\biggl(\log\biggl(\frac{\sigma_Y}{\sigma_N}\biggl) +  \frac{y^2}{2\sigma_Y^2} - \frac{(y-x)^2}{2\sigma_N^2}\biggr)^2\biggr] - I^2(X;Y) \nonumber \\
& = \frac{1}{4}\mathbb{E}_{p_{XY}}\biggl[ \biggl(\frac{y-x}{\sigma_N}\biggr)^4 + \biggl(\frac{y}{\sigma_Y}\biggr)^4 -2 \biggl(\frac{y}{\sigma_Y}\biggr)^2 \biggl(\frac{y-x}{\sigma_N}\biggr)^2 \biggr] \nonumber \\
& = \dots = \text{Kurt}[Z]\biggl(\frac{1}{2}-\frac{\sigma_N^2}{2\sigma_Y^2}\biggr) - \frac{\sigma_X^2}{2\sigma_Y^2} \nonumber \\
& = \frac{\sigma_X^2}{\sigma_Y^2} = 1-\frac{\sigma_N^2}{\sigma_Y^2} = 1-e^{-2I(X;Y)}=\rho^2,
\end{align}
where the last steps use the fact that the Kurtosis of a normal distribution is $3$ and that the mutual information can be expressed as in \eqref{eq:MI_gaussians}. 
Finally, the variance of the mean of $M$ i.i.d. random variables yields the thesis
\begin{equation}
\text{Var}_{p_{XY}}\bigl[\mathbb{E}_{p^M_{XY}} [\log \hat{R}] \bigr]  =\frac{\text{Var}_{p_{XY}}[\log \hat{R}]}{M}. 
\end{equation}
If $X$ and $N$ are multivariate Gaussians with diagonal covariance matrices $\rho^2\mathbb{I}_{d\times d}$ and $(1-\rho^2) \mathbb{I}_{d\times d}$, the results for both the MI and variance in the scalar case are simply multiplied by $d$.
\end{proof}

\begin{theorem}[Related to Theorem~\ref{theorem:MI_permutationsBound}]
Let the discriminator $D(\cdot)$ be with enough capacity. Let $N$ be the batch size and $f$ be the generator of the KL divergence. Let $\mathcal{J}_{KL}^{\pi}(D)$ be defined as
\begin{equation}
\mathcal{J}_{KL}^{\pi}(D) =  \mathbb{E}_{(\mathbf{x},\mathbf{y}) \sim p_{XY}(\mathbf{x},\mathbf{y})}\biggl[\log\biggl(D\bigl(\mathbf{x},\mathbf{y}\bigr)\biggr)-f^*\biggl(\log\biggl(D\bigl(\mathbf{x},\pi(\mathbf{y})\bigr)\biggr)\biggr)\biggr].
\end{equation}
Denote with $K$ the number of indices in the same position after the permutation (i.e., the fixed points), and with $R(\mathbf{x},\mathbf{y})$ the density ratio in \eqref{eq:MI_density_ratio_1}.
Then,
\begin{equation}
\hat{D}(\mathbf{x},\mathbf{y}) =\arg \max_D \mathcal{J}_{KL}^{\pi}(D) = \frac{NR(\mathbf{x},\mathbf{y})}{KR(\mathbf{x},\mathbf{y})+N-K}.
\end{equation}
\end{theorem}

\begin{proof}
The idea of the proof is to express $\mathcal{J}_{KL}^{\pi}(D)$ via Monte Carlo approximation, in order to rearrange fixed points, and then go back to Lebesgue integration.
The value function $\mathcal{J}_{KL}(D)$ can be written as
\begin{equation}
    \mathcal{J}_{KL}(D) =  \mathbb{E}_{(\mathbf{x},\mathbf{y}) \sim p_{XY}(\mathbf{x},\mathbf{y})}\biggl[\log\bigl(D(\mathbf{x},\mathbf{y})\bigr)\biggr]  -\mathbb{E}_{(\mathbf{x},\mathbf{y}) \sim p_{X}(\mathbf{x})p_{Y}(\mathbf{y})}\biggl[D\bigl(\mathbf{x},\mathbf{y}\bigr)\biggr]+1.
\end{equation}

Similarly to \eqref{eq:MI_monte-carlo-1}, we can express $\mathcal{J}_{KL}^{\pi}(D)$ as
\begin{equation}
\label{eq:MI_monte-carlo-KL}
    \mathcal{J}_{KL}^{\pi}(D) = \frac{1}{N}\sum_{i=1}^{N}{\log\bigl(D(\mathbf{x}_i,\mathbf{y}_i)\bigr)}-  \frac{1}{N}\sum_{i=1}^{K}{D(\mathbf{x}_i,\mathbf{y}_{i})} -  \frac{1}{N}\sum_{i=1}^{N-K}{D(\mathbf{x}_i,\mathbf{y}_{j\neq i})} +1,
\end{equation}
where $K$ is the number of fixed points of the permutation $j=\pi(i), \forall i \in \{1,\dots,N\}$.
However, when $N \to \infty$, we can use Lebesgue integration and rewrite \eqref{eq:MI_monte-carlo-KL} as
\begin{align}
    \mathcal{J}_{KL}^{\pi}(D) = &\int_{\mathbf{x}}\int_{\mathbf{y}}{\biggl(p_{XY}(\mathbf{x},\mathbf{y})\log\bigl(D(\mathbf{x},\mathbf{y})\bigr) -\frac{K}{N}p_{XY}(\mathbf{x},\mathbf{y})D(\mathbf{x},\mathbf{y})\biggr)\diff \mathbf{x} \diff \mathbf{y}} \nonumber \\ 
    & -\int_{\mathbf{x}}\int_{\mathbf{y}}{\frac{N-K}{N}p_{X}(\mathbf{x})p_Y(\mathbf{y})D(\mathbf{x},\mathbf{y})\diff \mathbf{x} \diff \mathbf{y}}+1.
\end{align}
To maximize $\mathcal{J}_{KL}^{\pi}(D)$, it is enough to take the derivative of the integrand with respect to $D$ and equate it to $0$, yielding the following equation in $D$
\begin{equation}
    \frac{p_{XY}(\mathbf{x},\mathbf{y})}{D(\mathbf{x},\mathbf{y})}-\frac{K}{N}p_{XY}(\mathbf{x},\mathbf{y}) -\frac{N-K}{N}p_{X}(\mathbf{x})p_Y(\mathbf{y}) =0.
\end{equation}
Solving for $D$ leads to the thesis 
\begin{equation}
\hat{D}(\mathbf{x},\mathbf{y}) =  \frac{NR(\mathbf{x},\mathbf{y})}{KR(\mathbf{x},\mathbf{y})+N-K},
\end{equation}
since $\mathcal{J}_{KL}^{\pi}(\hat{D})$ is a maximum being the second derivative w.r.t. $D$ a non-positive function.
\end{proof}

\begin{corollary}[Related to Corollary~\ref{corollary:MI_permutationsUpperBound}]
Let KL-DIME be the estimator obtained via iterative optimization of $\mathcal{J}_{KL}^{\pi}(D)$, using a batch of size $N$ every training step. Then,
\begin{equation}
    I_{KL-DIME}^{\pi} := \mathbb{E}_{(\mathbf{x},\mathbf{y}) \sim p_{XY}(\mathbf{x},\mathbf{y})}\biggl[ \log \biggl(\hat{D}(\mathbf{x},\mathbf{y})\biggr) \biggr] < \log(N).
\end{equation}
\end{corollary}

\begin{proof}
Theorem \ref{theorem:MI_permutationsBound} implies that, when the batch size is much larger than the density ratio ($N >> R$), then the discriminator's output converges to the density ratio. Indeed,
\begin{equation}
    \lim_{N\to \infty}{\hat{D}(\mathbf{x},\mathbf{y})} = \lim_{N\to \infty}{\frac{NR(\mathbf{x},\mathbf{y})}{KR(\mathbf{x},\mathbf{y})+N-K}} = R(\mathbf{x},\mathbf{y}).
\end{equation}

Instead, when the density ratio is much larger than the batch size ($R>>N$), then the discriminator's output converges to a constant, in particular 
\begin{equation}
    \lim_{R\to \infty}{\hat{D}(\mathbf{x},\mathbf{y})} = \lim_{R\to \infty}{\frac{NR(\mathbf{x},\mathbf{y})}{KR(\mathbf{x},\mathbf{y})+N-K}} = \frac{N}{K}.
\end{equation}
However, from Lemma \ref{lemma:MI_M=1}, it is true that $K=1$ on average. Therefore, an iterative optimization algorithm leads to an upper-bounded discriminator, since
\begin{equation}
    \hat{D}(\mathbf{x},\mathbf{y}) < N,
\end{equation}
which implies the thesis.
\end{proof}

\begin{lemma}[Related to Lemma~\ref{lemma:MI_M=1}]
The average number of fixed points in a random permutation $\pi(\cdot)$ is equal to 1.
\end{lemma}
\begin{proof}
Let $\pi(\cdot)$ be a random permutation on $\left\{ 1, \dotsc, N \right\}$. Let the random variable $X$ represent the number of fixed points (i.e., the number of cycles of length $1$) of $\pi(\cdot)$. We define $X = X_1 + X_2 + \cdots + X_N$, where $X_i = 1$ when $\pi(i)=i$, and $0$ otherwise. $\mathbb{E}[X]$ is computed by exploiting the linearity property of expectation. Trivially,
\begin{equation}
    \mathbb{E}[X_i] = \mathbb{P}[\pi(i)=i] = \frac{1}{N},
\end{equation}
which implies 
\begin{equation}
    \mathbb{E}[X] = \sum_{i=1}^N \frac{1}{N} = 1.
\end{equation}
\end{proof}

\subsection{Experimental details}
\label{subsec:mi_appendix_experiment_details}

\subsubsection{Multivariate linear and non-linear Gaussians experiments}
\label{subsec:mi_appendix_staircases}

The NN architectures implemented for the linear and cubic Gaussian experiments are: \textit{joint}, \textit{separable}, \textit{deranged}, and the architecture of NJEE, referred to as \textit{ad hoc}. 

\textbf{Joint architecture}. The \textit{joint} architecture is a feed-forward fully connected NN with an input size equal to twice the dimension of the samples distribution ($2d$), one output neuron, and two hidden layers of $256$ neurons each. The activation function utilized in each layer (except from the last one) is ReLU.
The number of realizations ($\mathbf{x},\mathbf{y}$) fed as input of the neural network for each training iteration is $N^2$, obtained as all the combinations of the samples $\mathbf{x}$ and $\mathbf{y}$ drawn from $p_{XY}(\mathbf{x}, \mathbf{y})$. 

\textbf{Separable architecture}. The \textit{separable} architecture comprises two feed-forward NNs, each one with an input size equal to $d$, output layer containing $32$ neurons and $2$ hidden layers with $256$ neurons each. The ReLU activation function is used in each layer (except from the last one).
The first network is fed in with $N$ realizations of $X$, while the second one with $N$ realizations of $Y$.

\textbf{Deranged architecture}. The \textit{deranged} architecture is a feed-forward fully connected NN with an input size equal to twice the dimension of the samples distribution ($2d$), one output neuron, and two hidden layers of $256$ neurons each. The activation function utilized in each layer (except from the last one) is ReLU. 
The number of realizations ($\mathbf{x},\mathbf{y}$) the neural network is fed with is $2N$ for each training iteration: $N$ realizations drawn from $p_{XY}(\mathbf{x}, \mathbf{y})$ and $N$ realizations drawn from $p_X(\mathbf{x})p_Y(\mathbf{y})$ using the derangement procedure described in Sec. \ref{subsec:mi_derangements}.

The architecture \textit{deranged} is not used for $I_{CPC}$ because in \eqref{eq:MI_NCE} the summation at the denominator of the argument of the logarithm would require neural network predictions corresponding to the inputs $(\mathbf{x}_i, \mathbf{y}_j), \>  \forall i,j \in \{1,\dots,N\}$ with $i \neq j$.

\textbf{Ad hoc architecture}. The \textit{NJEE} MI estimator comprises $2d-1$ feed-forward NNs. Each NN is composed by an input layer with size between $1$ and $2d-1$, an output layer containing $N-k$ neurons, with $k \in \mathbb{N}$ small, and 2 hidden layers with $256$ neurons each. The ReLU activation function is used in each layer (except from the last one).
We implemented a Pytorch \cite{Pytorch2016} version of the 
code produced by the authors of \cite{shalev2022neural}, to unify NJEE with all the other MI estimators. 

Each neural estimator is trained using Adam optimizer \cite{kingma2014adam}, with learning rate $5 \times 10^{-4}$, $\beta_1 = 0.9$, $\beta_2=0.999$. The batch size is initially set to $N=64$.\\
For the \textit{Gaussian} setting, we sample a $20$-dimensional Gaussian distribution to obtain $\mathbf{x}$ and $\mathbf{n}$ samples, independently. Then, we compute $\mathbf{y}$ as linear combination of $\mathbf{x}$ and $\mathbf{n}$: $\mathbf{y} = \rho \, \mathbf{x} + \sqrt{1-\rho^2} \, \mathbf{n}$, where $\rho$ is the correlation coefficient. For the \textit{cubic} setting, the nonlinear transformation $\mathbf{y} \mapsto \mathbf{y}^3$ is applied to the Gaussian samples.
a
During the training procedure, every $4k$ iterations, the target value of the MI is increased by $2\> nats$, for $5$ times, obtaining a target staircase with $5$ steps. The change in target MI is obtained by increasing $\rho$, that affects the true MI according to
\begin{equation}
    I(X;Y) = -\frac{d}{2}\log(1 - \rho^2).
\end{equation}

\subsubsection{Supplementary analysis of the MI estimators performance}
Additional plots reporting the MI estimates obtained from MINE, NWJ, and SMILE with $\tau = \infty$, are outlined in Fig. \ref{fig:MI_stairs_d20_bs64_MINE_NW_SMILE}. The variance attained by these algorithms exponentially increases as the true MI grows, as stated in \eqref{eq:MI_exponentially_increasing_variance}. 

 \begin{figure}
	\centering
	\includegraphics[scale=0.42]{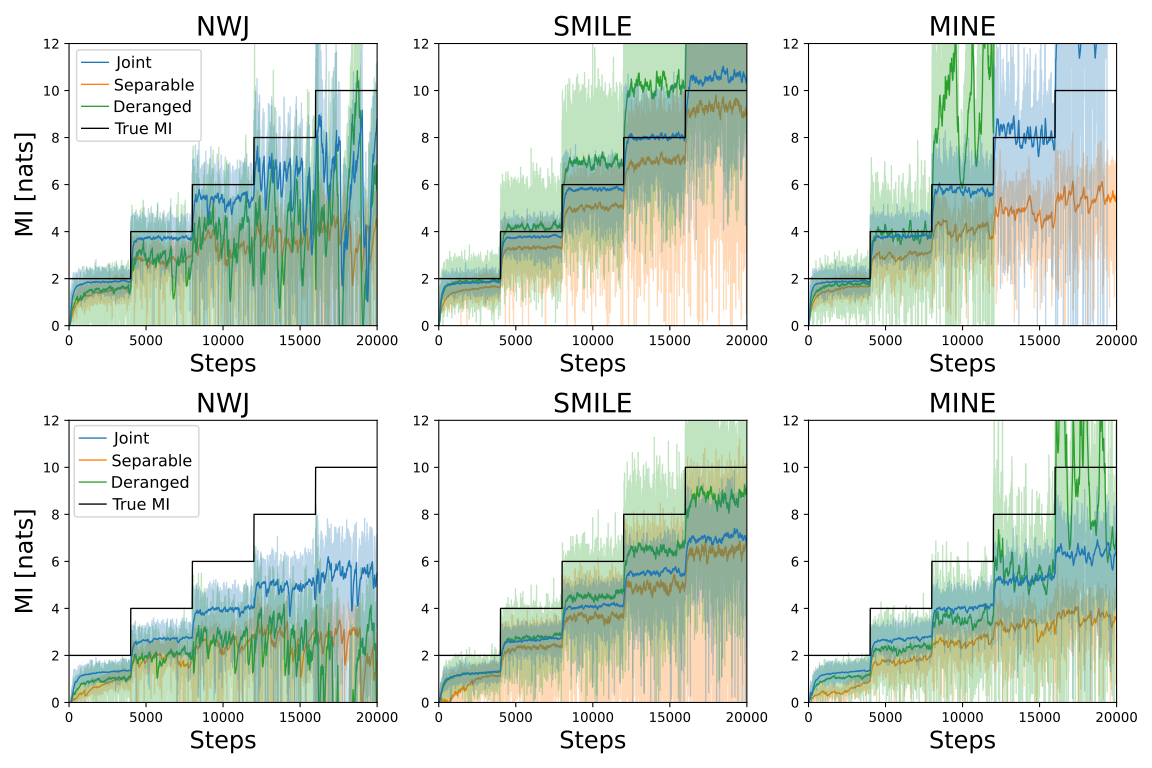}
	\caption{NWJ, SMILE ($\tau=\infty$), and MINE MI estimation comparison with $d=20$ and $N=64$. The \textit{Gaussian} setting is represented in the top row, while the \textit{cubic} setting is shown in the bottom row.}
	\label{fig:MI_stairs_d20_bs64_MINE_NW_SMILE}
\end{figure} 

We report in Fig. \ref{fig:MI_SMILE_no_trick} the behavior we obtained for $I_{SMILE}$ when the training of the neural network is performed by using the cost function in \eqref{eq:MI_SMILE}. The training diverges during the first steps when $\tau=1$ and $\tau=5$. Differently, when $\tau=\infty$, $I_{SMILE}$ corresponds to $I_{MINE}$ (without the moving average improvement), therefore the MI estimate does not diverge. Interestingly, by comparing $I_{SMILE}$ ($\tau=\infty$) trained with the JS divergence and with the MINE cost function (in Fig. \ref{fig:MI_stairs_d20_bs64_MINE_NW_SMILE} and Fig. \ref{fig:MI_SMILE_no_trick}, respectively), the variance of the latter case is significantly higher. Hence, the JS maximization trick seems to have an impact in lowering the estimator variance. 

\begin{figure}[htp]

\centering
\begin{subfigure}{.3\textwidth}
\includegraphics[width=1\textwidth]{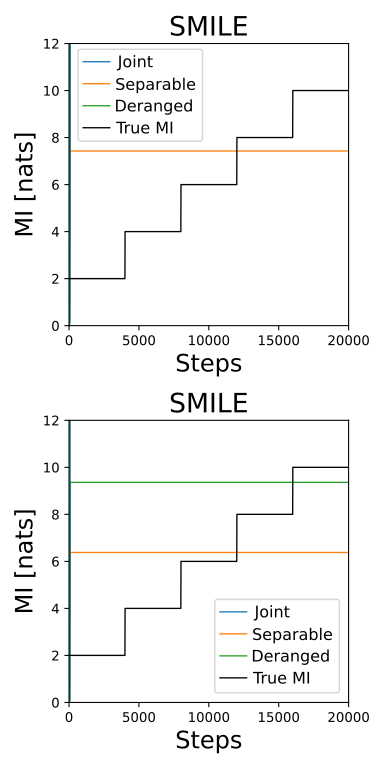}
\caption{$\tau = 1$}
\end{subfigure}
\begin{subfigure}{.3\textwidth}
\includegraphics[width=1\textwidth]{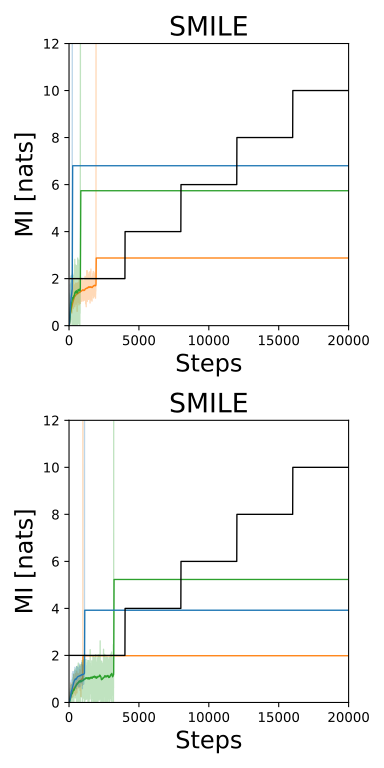}
\caption{$\tau = 5$}
\end{subfigure}
\begin{subfigure}{.3\textwidth}
\includegraphics[width=1\textwidth]{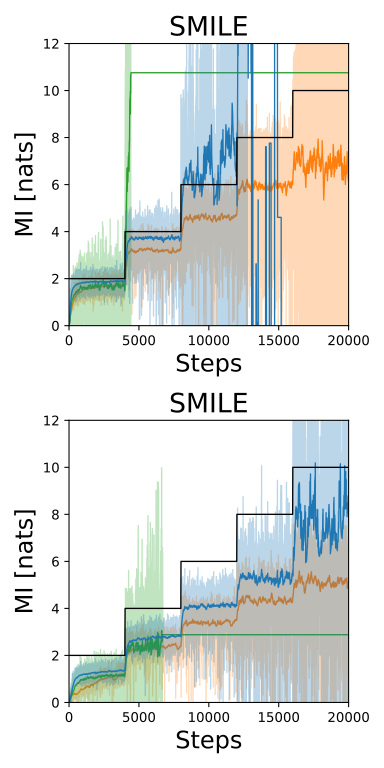}
\caption{$\tau = \infty$}
\end{subfigure}
\caption{$I_{SMILE}$ behavior for different values of $\tau$, when the JS divergence is not used to train the neural network. The \textit{Gaussian} case is reported in the top row, while the \textit{cubic} case is reported in the bottom row.}
\label{fig:MI_SMILE_no_trick}

\end{figure}

\subsubsection{Analysis for different values of $d$ and $N$}

The class of $f$-DIME estimators is robust to changes in $d$ and $N$, as the estimators' variance decreases (see \eqref{eq:MI_variance_f_dime} and Fig. \ref{fig:MI_VarianceVsBs}) when $N$ increases and their achieved bias is not significantly influenced by the choice of $d$. Differently, $I_{NJEE}$ and $I_{CPC}$ are highly affected by variations of those parameters, as described in Fig. \ref{fig:MI_stairs_d5_bs64} and Fig. \ref{fig:MI_stairs_d20_bs1024}. 
More precisely, $I_{CPC}$ is not strongly influenced by a change of $d$, but the bias significantly increases as the batch size diminishes, since the upper bound lowers.
$I_{NJEE}$ achieves a higher bias both when $d$ decreases and when $N$ increases w.r.t. the default values $d=20, N=64$. In addition, when $d$ is large, the training of $I_{NJEE}$ is not feasible, as it requires a lot of time (see Fig. \ref{fig:MI_computationalTimeAnalysisMain}) and memory (as a consequence of the large number of neural networks utilized) requirements.

We show the achieved bias, variance, and mean squared error (MSE) corresponding to the three settings reported in Fig. \ref{fig:MI_stairs}, \ref{fig:MI_stairs_d5_bs64}, and \ref{fig:MI_stairs_d20_bs1024} in Fig. \ref{fig:MI_bias_var_mse_gaussian_d20_N64}, \ref{fig:MI_bias_var_mse_gaussian_d5_N64}, and \ref{fig:MI_bias_var_mse_gaussian_d20_N1024}, respectively. The achieved variance is bounded when the estimator used is $I_{KL-DIME}$ or $I_{CPC}$. In particular, Figures \ref{fig:MI_bias_var_mse_gaussian_d20_N64}, \ref{fig:MI_bias_var_mse_gaussian_d5_N64}, \ref{fig:MI_bias_var_mse_gaussian_d20_N1024}, and \ref{fig:MI_VarianceVsBs} demonstrate that $I_{KL-DIME}$ satisfies Lemma \ref{lemma:MI_finite_variance_gaus}.

Additionally, we report the achieved bias, variance and MSE when $d=20$ and $N$ varies according to Tab. \ref{tab:MI_bias_var_mse_gauss_N_varying}. We use the notation $N = [512, 1024]$ to indicate that each cell of the table reports the values corresponding to $N=512$ and $N=1024$, with this specific order, inside the brackets.
Similarly, we show the attained bias, variance, and MSE for $d=[5, 10]$ and $N=64$ in Tab. \ref{tab:MI_bias_var_mse_gauss_d_varying}.
The achieved bias, variance and MSE shown in Tab. \ref{tab:MI_bias_var_mse_gauss_N_varying} and Tab. \ref{tab:MI_bias_var_mse_gauss_d_varying} motivate that the class of $f$-DIME estimators attains the best values for bias and MSE. Similarly, $I_{KL-DIME}$ obtains the lowest variance, when excluding $I_{CPC}$ from the estimators comparison ($I_{CPC}$ should not be desirable as it is upper bounded).
The illustrated results are obtained with the \textit{joint} architecture (except for NJEE) because, when the batch size is small, such an architecture achieves slightly better results than the \textit{deranged} one, as it approximates the expectation over the product of marginals with more samples. 

The variance of the $f$-DIME estimators achieved in the Gaussian setting when $N$ ranges from $64$ to $1024$ is reported in Fig. \ref{fig:MI_VarianceVsBs}. The behavior shown in such a figure demonstrates what is stated in Lemma \ref{lemma:MI_Lemma2}, i.e., the variance of the $f$-DIME estimators varies as $\frac{1}{N}$. 

\begin{figure}
	\centering
	\includegraphics[scale=0.4]{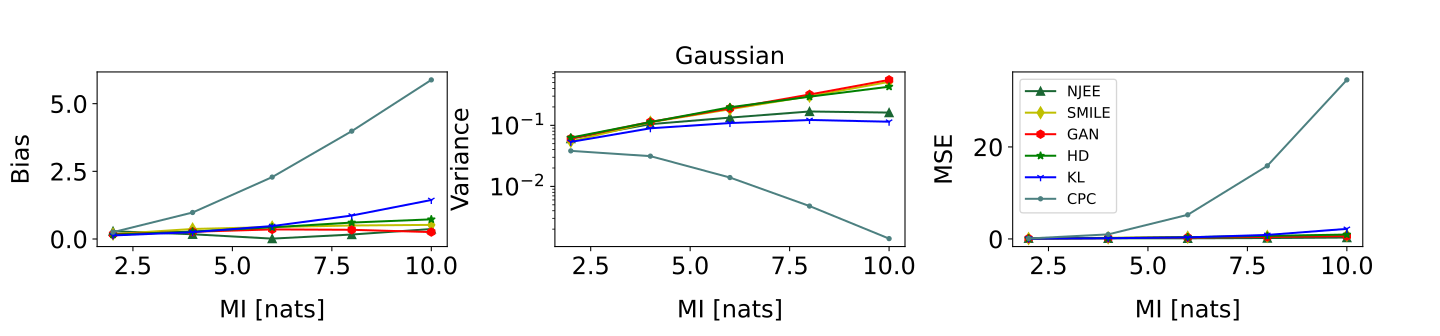}
	\caption{Bias, variance, and MSE comparison between estimators, using the joint architecture for the \textit{Gaussian} case with $d=20$ and $N=64$.}
	\label{fig:MI_bias_var_mse_gaussian_d20_N64}
\end{figure} 
\begin{figure}
	\centering
	\includegraphics[scale=0.4]{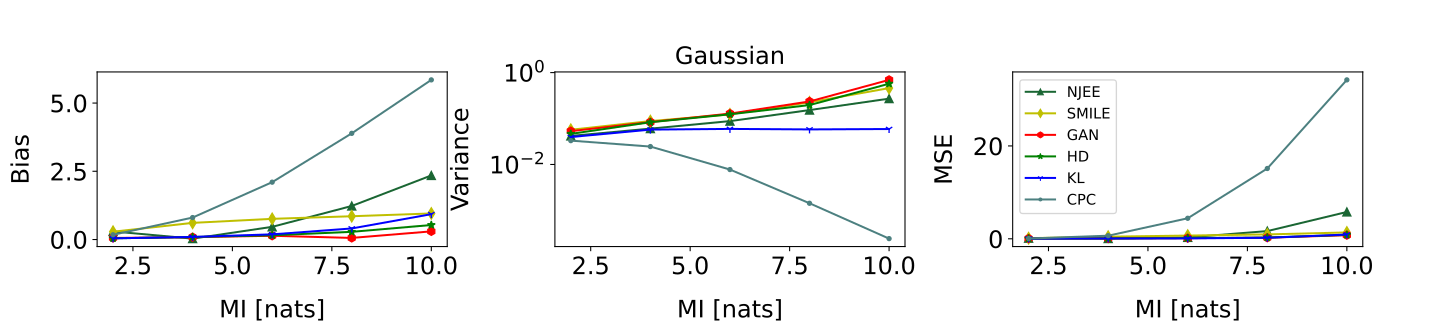}
	\caption{Bias, variance, and MSE comparison between estimators, using the joint architecture for the \textit{Gaussian} case with $d=5$ and $N=64$.}
	\label{fig:MI_bias_var_mse_gaussian_d5_N64}
\end{figure} 
\begin{figure}
	\centering
	\includegraphics[scale=0.4]{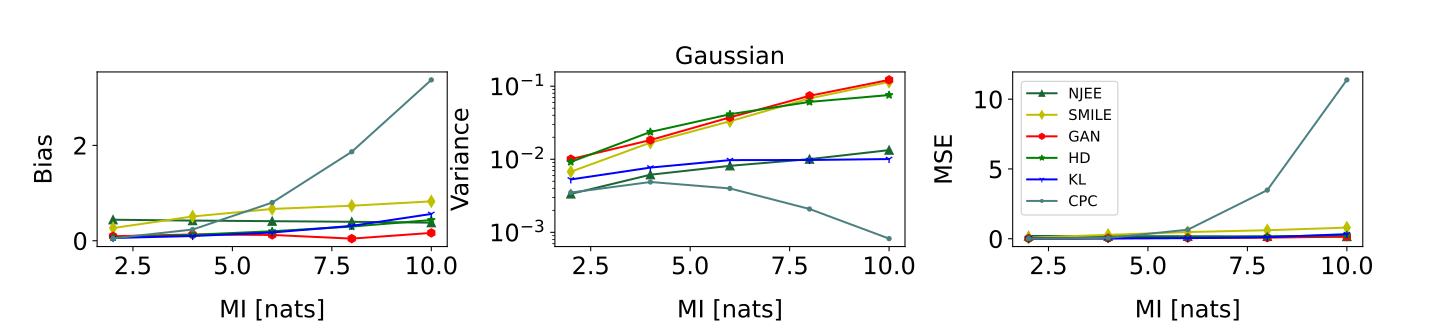}
	\caption{Bias, variance, and MSE comparison between estimators, using the joint architecture for the \textit{Gaussian} case with $d=20$ and $N=1024$.}
	\label{fig:MI_bias_var_mse_gaussian_d20_N1024}
\end{figure} 

\begin{figure}
	\centering
	\includegraphics[scale=0.29]{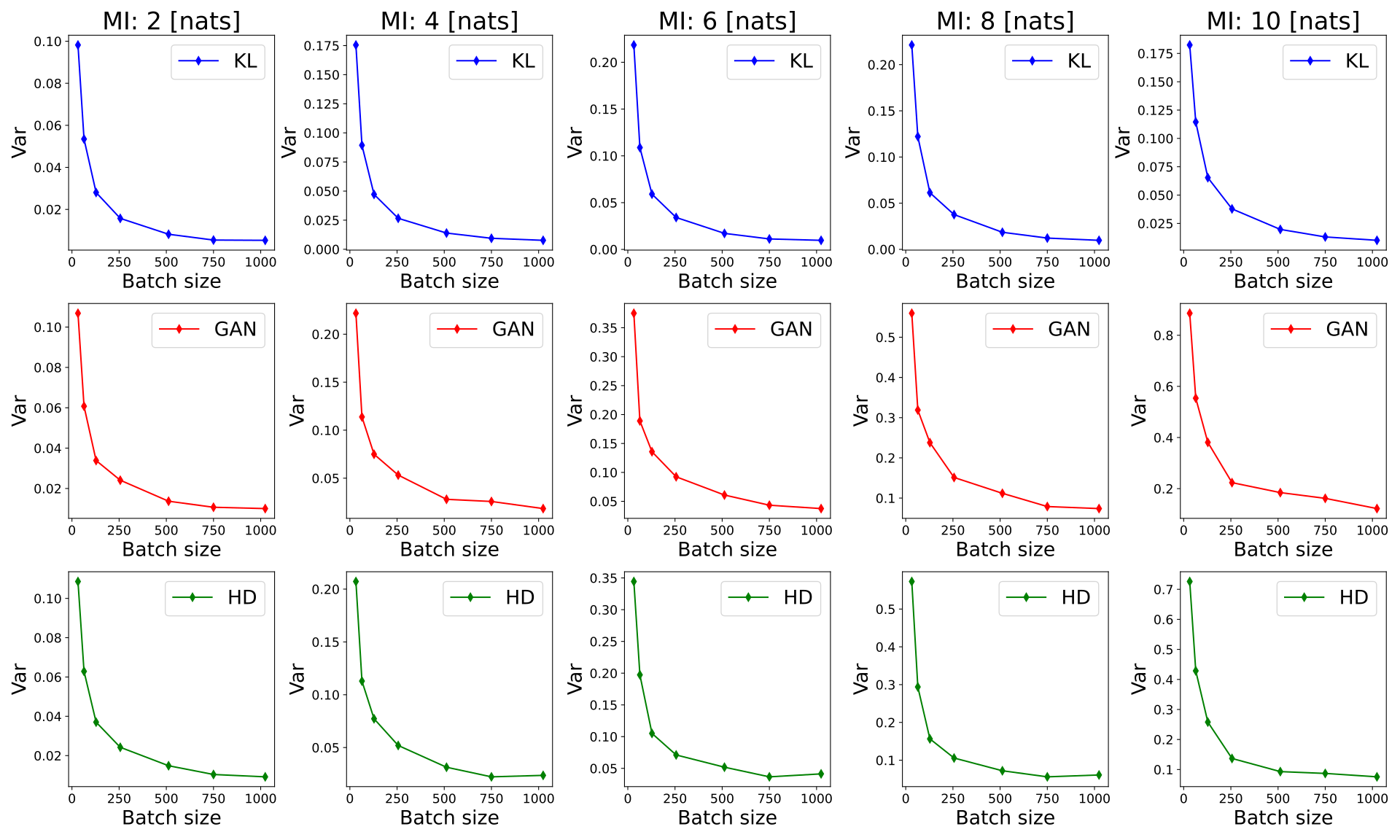}
	\caption{Variance of the $f$-DIME estimators corresponding to different values of batch size.}
	\label{fig:MI_VarianceVsBs}
\end{figure} 

\begin{table}
\caption{Bias (B), variance (V), and MSE (M) of the MI estimators using the joint architecture, when $d=20$ and $N=[512, 1024]$, for the Gaussian setting. Each $f$-DIME estimator is abbreviated to $f$-D.}
\setlength{\arrayrulewidth}{0.5mm}
\centering
    \begin{tabular}{ |c|c|c c c c c|c c c c c| } 
     \hline
     & & \multicolumn{5}{|c|}{Gaussian} \\
     \hline
     & MI & 2 & 4 & 6 & 8 & 10 \\
     \hline
      & NJEE & [0.42, 0.44] & [0.40, 0.42] & [0.37, 0.41] & [0.34, 0.40] & [0.32, 0.38] \\
      & SMILE & [0.25, 0.27] & [0.48, 0.51] & [0.64, 0.67] & [0.74, 0.73] & [0.86, 0.83]\\
      B & GAN-D & [0.11, 0.09] & [0.15, 0.13] & [\textbf{0.16}, \textbf{0.12}] & [\textbf{0.14}, \textbf{0.04}] & [\textbf{0.01}, \textbf{0.16}] \\
      & HD-D & [0.08, 0.07] & [0.15, 0.12] & [0.24, 0.20] & [0.37, 0.30] & [0.47, 0.43]  \\
      & KL-D & [\textbf{0.07}, 0.06] & [\textbf{0.12}, \textbf{0.10}] & [0.21, 0.17] & [0.38, 0.31] & [0.69, 0.56] \\
      & CPC & [0.08, \textbf{0.05}] & [0.34, 0.23] & [1.07, 0.80] & [2.32, 1.87] & [3.96, 3.37] \\
     \hline
     & NJEE & [0.01, 0.00] & [0.01, 0.01] & [0.02, 0.01] & [0.02, 0.01] & [0.02, 0.01] \\
      & SMILE & [0.01, 0.01] & [0.03, 0.02] & [0.06, 0.03] & [0.11, 0.07] & [0.17, 0.11]  \\
      V & GAN-D & [0.01, 0.01] & [0.03, 0.02] & [0.06, 0.04] & [0.11, 0.07] & [0.17, 0.12] \\
      & HD-D & [0.01, 0.01] & [0.03, 0.02] & [0.05, 0.04] & [0.07, 0.06] & [0.09, 0.08]  \\
      & KL-D& [0.01, 0.01] & [0.01, 0.01] & [0.02, 0.01] & [0.02, 0.01] & [0.02, 0.01] \\
      & CPC & [\textbf{0.01}, \textbf{0.00}] & [\textbf{0.01}, \textbf{0.00}] & [\textbf{0.01}, \textbf{0.00}] & [\textbf{0.00}, \textbf{0.00}] & [\textbf{0.00}, \textbf{0.00}]  \\
     \hline
     & NJEE & [0.18, 0.20] & [0.18, 0.18] & [0.16, 0.18] & [0.14, 0.17] & [\textbf{0.12}, 0.16]  \\
      & SMILE & [0.08, 0.08] & [0.26, 0.28] & [0.47, 0.48] & [0.66, 0.61] & [0.90, 0.80] \\
      M & GAN-D & [0.03, 0.02] & [0.05, 0.04] & [0.09, 0.05] & [\textbf{0.13}, \textbf{0.08}] & [0.18, \textbf{0.15}]  \\
      & HD-D & [0.02, 0.01] & [0.05, 0.04] & [0.11, 0.08] & [0.21, 0.15] & [0.31, 0.26] \\
      & KL-D & [\textbf{0.01}, \textbf{0.01}] & [\textbf{0.03}, \textbf{0.02}] & [\textbf{0.06}, \textbf{0.04}] & [0.17, 0.11] & [0.49, 0.33]  \\
      & CPC & [0.01, 0.01] & [0.13, 0.06] & [1.16, 0.64] & [5.38, 3.48] & [15.67, 11.38]  \\
     \hline
    \end{tabular}
    \label{tab:MI_bias_var_mse_gauss_N_varying}
\end{table}

\begin{table}
\caption{Bias (B), variance (V), and MSE (M) of the MI estimators using the joint architecture, when $d=[5, 10]$ and $N=[64]$, for the Gaussian setting. Each $f$-DIME estimator is abbreviated to $f$-D.} 
\setlength{\arrayrulewidth}{0.5mm}
\centering
    \begin{tabular}{ |c|c|c c c c c|c c c c c| } 
     \hline
     & & \multicolumn{5}{|c|}{Gaussian} \\
     \hline
     & MI & 2 & 4 & 6 & 8 & 10 \\
     \hline
      & NJEE & [0.30, 0.29] & [\textbf{0.03}, 0.13] & [0.46, \textbf{0.06}] & [1.23, 0.38] & [2.35, 0.80] \\
      & SMILE & [0.29, 0.24] & [0.61, 0.52] & [0.76, 0.68] & [0.85, 0.71] & [0.96, 0.68]\\
      B & GAN-D & [0.06, 0.12] & [0.09, 0.17] & [\textbf{0.14}, 0.17] & [\textbf{0.06}, \textbf{0.20}] & [\textbf{0.30}, \textbf{0.18}] \\
      & HD-D & [0.04, 0.09] & [0.09, 0.14] & [0.15, 0.22] & [0.28, 0.39] & [0.53, 0.40]  \\
      & KL-D & [\textbf{0.04}, \textbf{0.07}] & [0.09, \textbf{0.13}] & [0.19, 0.30] & [0.40, 0.58] & [0.93, 1.05]  \\
      & CPC & [0.17, 0.20] & [0.80, 0.89] & [2.10, 2.20] & [3.89, 3.93] & [5.85, 5.86]  \\
     \hline
     & NJEE & [0.04, 0.05] & [0.06, 0.08] & [0.09, 0.10] & [0.15, 0.13] & [0.27, 0.13] \\
      & SMILE & [0.06, 0.06] & [0.09, 0.13] & [0.12, 0.20] & [0.23, 0.32] & [0.46, 0.46]  \\
      V & GAN-D & [0.05, 0.06] & [0.08, 0.12] & [0.13, 0.19] & [0.24, 0.30] & [0.69, 0.52] \\
      & HD-D & [0.05, 0.06] & [0.08, 0.11] & [0.12, 0.16] & [0.20, 0.24] & [0.57, 0.49]  \\
      & KL-D & [0.04, 0.05] & [0.06, 0.08] & [0.06, 0.10] & [0.06, 0.10] & [0.06, 0.10] \\
      & CPC & [\textbf{0.03}, \textbf{0.04}] & [\textbf{0.02}, \textbf{0.03}] & [\textbf{0.01}, \textbf{0.01}] & [\textbf{0.00}, \textbf{0.00}] & [\textbf{0.00}, \textbf{0.00}]  \\
     \hline
     & NJEE & [0.13, 0.13] & [\textbf{0.06}, \textbf{0.09}] & [0.30, \textbf{0.10}] & [1.66, \textbf{0.28}] & [5.78, 0.76]  \\
      & SMILE & [0.14, 0.11] & [0.46, 0.40] & [0.70, 0.66] & [0.95, 0.83] & [1.37, 0.93] \\
      M & GAN-D & [0.06, 0.08] & [0.09, 0.15] & [0.15, 0.22] & [0.24, 0.34] & [\textbf{0.78}, \textbf{0.55}]  \\
      & HD-D & [0.05, 0.07] & [0.09, 0.13] & [0.15, 0.21] & [0.28, 0.40] & [0.86, 0.65]  \\
      & KL-D & [\textbf{0.04}, \textbf{0.06}] & [0.07, 0.10] & [\textbf{0.10}, 0.19] & [\textbf{0.22}, 0.44] & [0.92, 1.20]  \\
      & CPC & [0.06, 0.08] & [0.67, 0.83] & [4.42, 4.84] & [15.14, 15.45] & [34.22, 34.32]  \\
     \hline
    \end{tabular}
    \label{tab:MI_bias_var_mse_gauss_d_varying}
\end{table}

\begin{table}
\caption{Bias (B), variance (V), and MSE (M) of the MI estimators using the joint architecture, when $d=20$ and $N=64$, for the Gaussian setting. Each $f$-DIME estimator is abbreviated to $f$-D.} 
\setlength{\arrayrulewidth}{0.5mm}
\centering
    \begin{tabular}{ |c|c|c c c c c|c c c c c| } 
     \hline
     & & \multicolumn{5}{|c|}{Gaussian} \\
     \hline
     & MI & 2 & 4 & 6 & 8 & 10 \\
     \hline
      & NJEE & 0.29 & \textbf{0.18} & \textbf{0.01} & \textbf{0.17} & 0.37 \\
      & SMILE & 0.18 & 0.37 & 0.44 & 0.50 & 0.52\\
      B & GAN-D & 0.17 & 0.27 & 0.35 & 0.34 & \textbf{0.26} \\
      & HD-D & 0.16 & 0.28 & 0.43 & 0.61 & 0.73  \\
      & KL-D & \textbf{0.13} & 0.25 & 0.48 & 0.87 & 1.44  \\
      & CPC & 0.25 & 0.98 & 2.29 & 3.99 & 5.88  \\
     \hline
     & NJEE & 0.06 & 0.10 & 0.13 & 0.17 & 0.16 \\
      & SMILE & 0.05 & 0.11 & 0.18 & 0.30 & 0.51  \\
      V & GAN-D & 0.06 & 0.11 & 0.19 & 0.32 & 0.55 \\
      & HD-D & 0.06 & 0.11 & 0.20 & 0.29 & 0.43  \\
      & KL-D & 0.05 & 0.09 & 0.11 & 0.12 & 0.11 \\
      & CPC & \textbf{0.04} & \textbf{0.03} & \textbf{0.01} & \textbf{0.00} & \textbf{0.00}  \\
     \hline
     & NJEE & 0.14 & \textbf{0.14} & \textbf{0.13} & \textbf{0.20} & \textbf{0.30}  \\
      & SMILE & 0.09 & 0.25 & 0.38 & 0.55 & 0.79 \\
      M & GAN-D & 0.09 & 0.19 & 0.31 & 0.43 & 0.62  \\
      & HD-D & 0.09 & 0.19 & 0.39 & 0.66 & 0.96  \\
      & KL-D & \textbf{0.07} & 0.15 & 0.34 & 0.87 & 2.19  \\
      & CPC & 0.10 & 0.99& 5.25 & 15.89 & 34.57  \\
     \hline
    \end{tabular}
    \label{tab:MI_bias_var_mse_gauss_d_varying2}
\end{table}

The class $f$-DIME is able to estimate the MI for high-dimensional distributions, as shown in Fig. \ref{fig:MI_allStairs_d100_bs64}, where $d=100$. In that figure, the estimates behavior is obtained by using the simple architectures described in Sec. \ref{subsec:mi_appendix_staircases}. Thus, the input size of these NNs ($200$) is comparable with the number of neurons in the hidden layers ($256$).
Therefore, the estimates could be improved by increasing the number of hidden layers and neurons per layer. The graphs in Fig. \ref{fig:MI_allStairs_d100_bs64} illustrate the advantage of the architecture \textit{deranged} over the \textit{separable} one.

\subsubsection{Considerations on derangements}
\label{subsec:mi_derangement_considerations}
To facilitate the understanding of the role of derangements during training, we provide a practical example in the following.
 
Suppose for simplicity that $N=3$. Then, a random permutation of $\mathbf{y} = [y_1, y_2, y_3]$ can be $[y_2, y_3, y_1]$, where the number of fixed points is $K=0$ as no elements remain in the same position after the permutation.
However, another permutation of $\mathbf{y}$ is $[y_1, y_3, y_2]$. In this case, it is evident that $y_1$ remains in the same initial position, and the number of fixed points is $K=1$.
A random derangement of $\mathbf{y} = [y_1, y_2, y_3]$, instead, ensures by definition that no element of $\mathbf{y}$ ends up in the same initial position, contrarily from a naive random permutation. This idea is essential to avoid having shuffled marginal samples that actually are realizations of the joint distribution.
In fact, we proved that a random permutation strategy would lead to a biased estimator (see the permutation bound in Corollary \ref{corollary:MI_permutationsUpperBound}).
It is extremely important to remark that the derangement sampling strategy it is not only applicable to $f$-divergence based estimators, rather, any discriminative variational estimator should use it to avoid upper bound MI estimates, as it can be observed from Fig. \ref{fig:MI_derangementVsPermutation_beyond}
\begin{figure}
	\centering
	\includegraphics[scale=0.24]{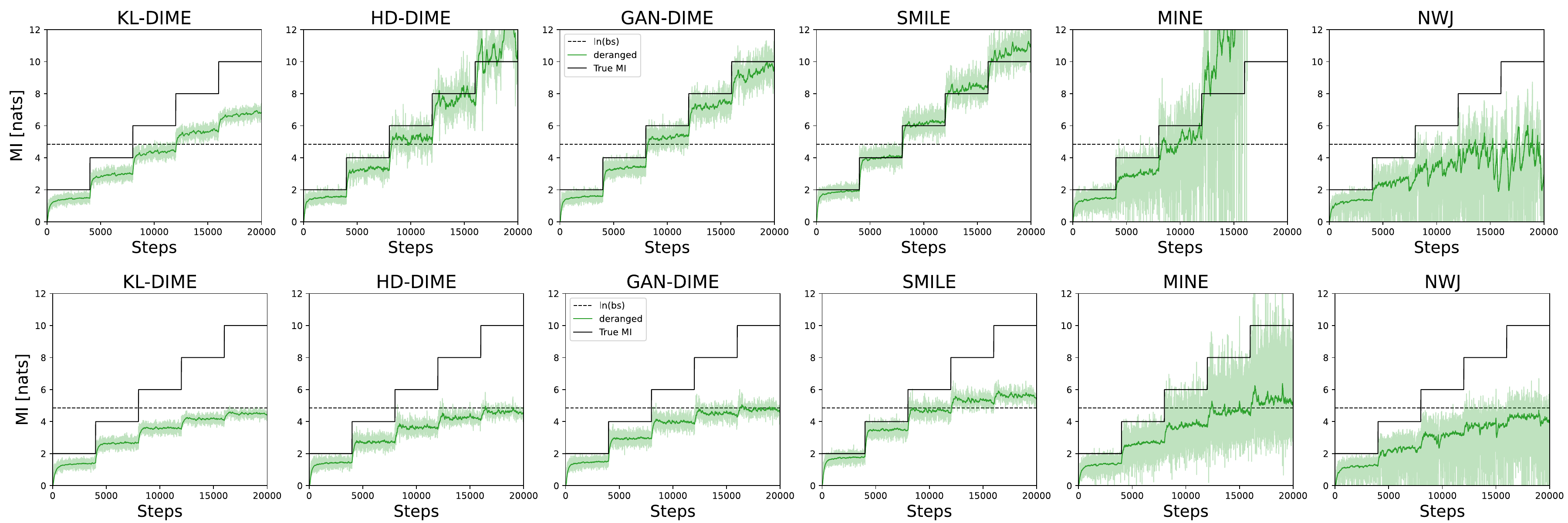}
	\caption{MI estimates when $d=20$ and $N=128$, top row: derangement strategy; bottom row: permutation strategy.}
	\label{fig:MI_derangementVsPermutation_beyond}
\end{figure} 
 
\begin{figure}
	\centering
	\includegraphics[scale=0.28]{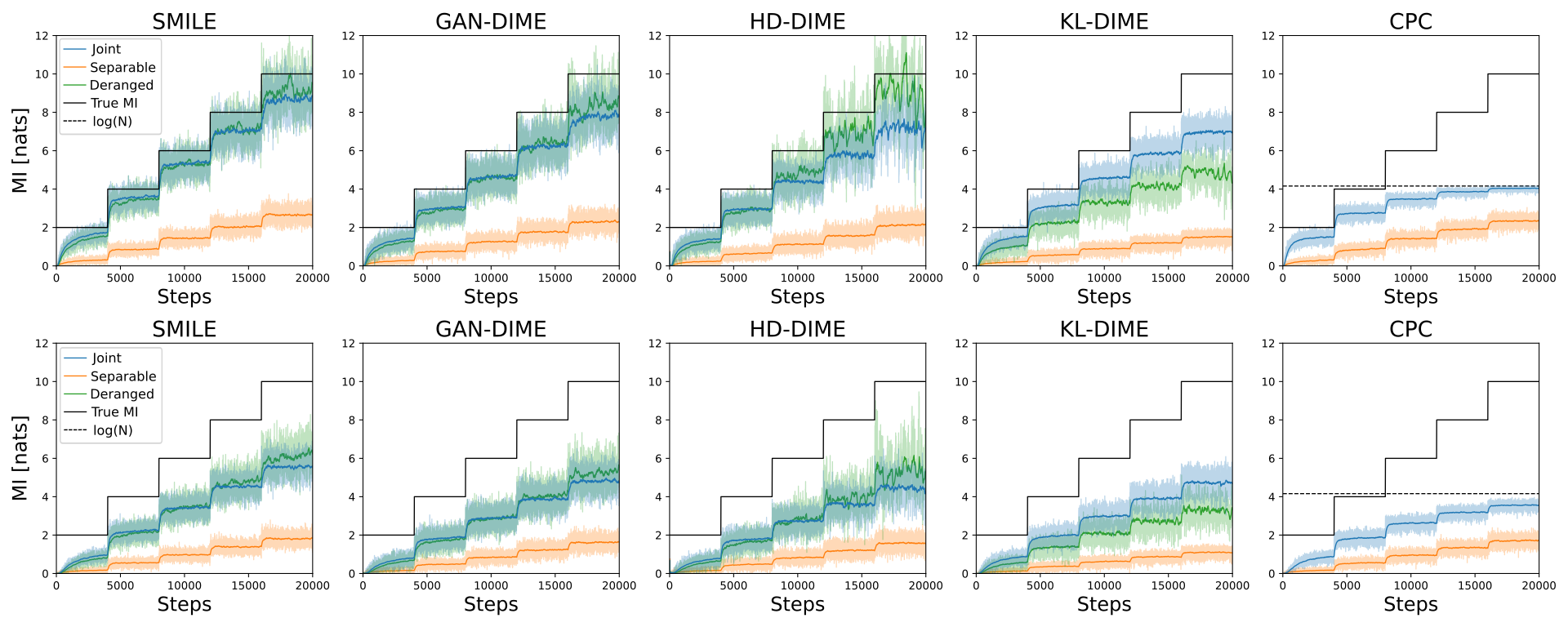}
	\caption{MI estimates when $d=100$ and $N=64$. The \textit{Gaussian} setting is represented in the top row, while the \textit{cubic} setting is shown in the bottom row.}
	\label{fig:MI_allStairs_d100_bs64}
\end{figure} 

\subsubsection{Time complexity analysis}
\label{subsubsec:mi_Appendix_time_analysis}

\begin{figure}
\centering
\begin{subfigure}{.5\textwidth}
 \centering
	\includegraphics[scale=0.4]{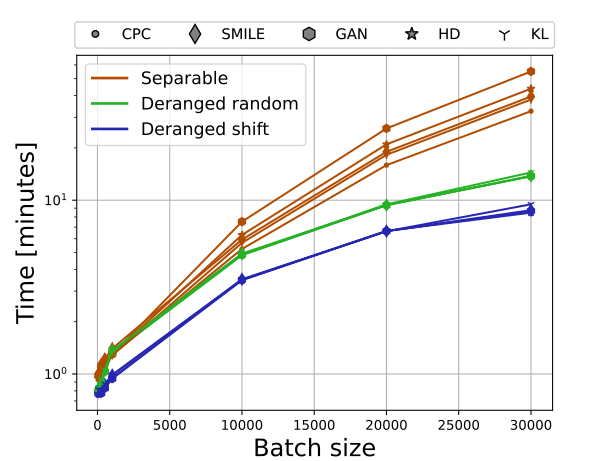}
	\caption{5-step staircases for the \textit{separable}, random-based \textit{deranged}, and shift-based \textit{deranged}.}
	\label{fig:MI_SeparableVsCombined_time_analysis}
\end{subfigure}%
\begin{subfigure}{.5\textwidth}
\centering
	\includegraphics[scale=0.4]{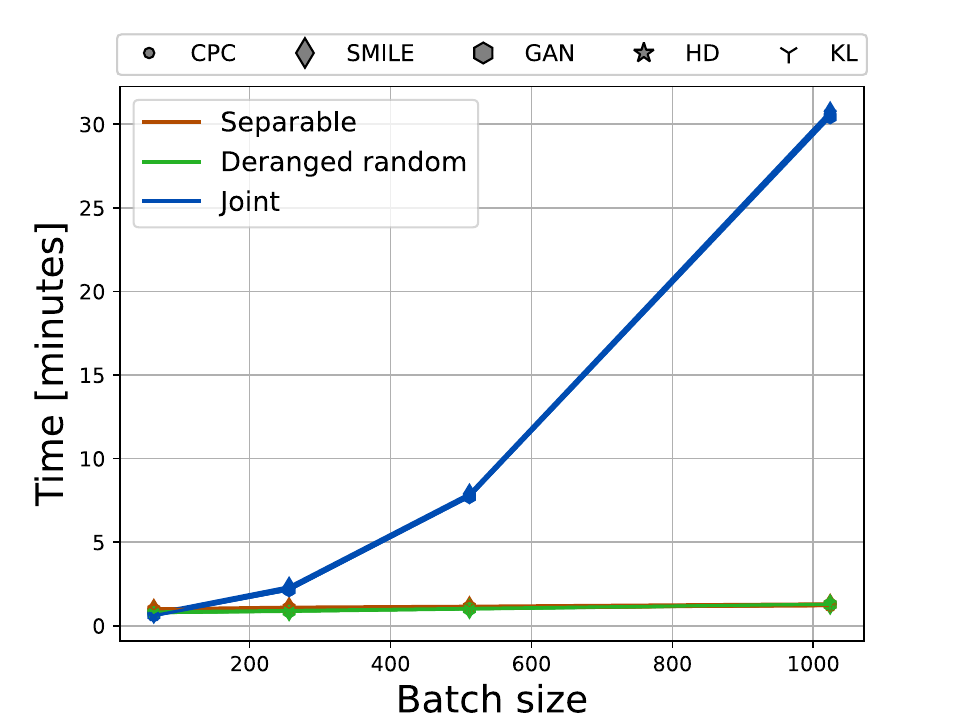}
	\caption{\textit{deranged}, \textit{separable} and \textit{joint} architectures complexity in linear scale, varying the batch size.}
	\label{fig:MI_computationalTimeAnalysisMain_beyond}
\end{subfigure}
\caption{Time comparison between sampling strategies.}
\label{fig:MI_extra_time_appendix}
\end{figure}

The computational time analysis is developed on a server with CPU "AMD Ryzen Threadripper 3960X 24-Core Processor" and GPU "MSI GeForce RTX 3090 Gaming X Trio 24G, 24GB GDDR6X".

Before analyzing the time requirements to complete the $5$-step MI staircases, we specify two different ways to implement the derangement of the $\mathbf{y}$ realizations in each batch:
\begin{itemize}
    \item \textbf{Random-based}. The trivial way to achieve the derangement is to randomly shuffle the elements of the batch until there are no fixed points (i.e., all the $\mathbf{y}$ realizations in the batch are assigned to a different position w.r.t. the starting location).
    \item \textbf{Shift-based}. Given $N$ realizations $(\mathbf{x}_i, \mathbf{y}_i)$ drawn from $p_{XY}(\mathbf{x},\mathbf{y})$, for $i \in \left\{ 1, \dotsc, N \right\}$, we obtain the deranged samples as $(\mathbf{x}_i, \mathbf{y}_{(i+1)\%N})$, where "\%" is the modulo operator. 
\end{itemize}
Although the MI estimates obtained by the two derangement methods are almost indistinguishable, all the results shown in the section are achieved by using the random-based method. We additionally demonstrate the time efficiency of the shift-based approach. 

We show in Fig. \ref{fig:MI_computationalTimeAnalysisMain} that the architectures \textit{deranged} and \textit{separable} are significantly faster w.r.t. \textit{joint} and NJEE ones, for a given batch size $N$ and input distribution size $d$.
However, Fig. \ref{fig:MI_computationalTimeAnalysisMain} exhibits no difference between the \textit{deranged} and \textit{separable} architectures. 
Fig. \ref{fig:MI_SeparableVsCombined_time_analysis} illustrates a detailed representation of the time requirements of these two architectures to complete the $5$-step stairs presented in Sec. \ref{sec:mi_results}.
As $N$ increases, the gap between the time needed by the architectures \textit{deranged} and \textit{separable} grows, demonstrating that the former is the fastest. For example, when $d=20$ and $N=30k$, $I_{GAN-DIME}$ needs about $55$ minutes when using the architecture \textit{separable}, but only $15$ minutes when using the \textit{deranged} one and less than $9$ minutes for the shift-based \textit{deranged} architecture.
In fact, if we plot the data in a linear scale (see Fig. \ref{fig:MI_computationalTimeAnalysisMain_beyond}), it is immediate to see that the computational requirements are $\Omega(N)$ for the deranged and separable architectures, and $\Omega(N^2)$ for the joint architecture. 

\subsubsection{Summary of the estimators}
\label{subsubsec:mi_appendix_final_considerations}

We give an insight on how to choose the best estimator in Tab. \ref{tab:MI_summary_estimators}, depending on the desired specifics. We assign qualitative grades to each estimator over different performance indicators.
All the indicators names are self-explanatory, except from \textit{scalability}, which describes the capability of the estimator to obtain precise estimates when $d$ and $N$ vary from the default values ($d=20$ and $N=64$). The grades ranking is, from highest to lowest: \cmark\cmark, \cmark, $\sim$, \xmark. When more than one architecture is available for a specific estimator, the grade is assigned by considering the best architecture within that particular case.
Even though the estimator choice could depend on the specific problem, we consider $I_{GAN-DIME}$ to be the best one. The rationale behind this decision is that $I_{GAN-DIME}$ achieves the best performance for almost all the indicators and lacks weaknesses. Differently, $I_{CPC}$ estimate is upper bounded, $I_{SMILE}$ achieves slightly higher bias, and $I_{NJEE}$ is strongly $d$ and $N$ dependent. However, if the considered problem requires the estimation of a low-valued MI, $I_{KL-DIME}$ is slightly more accurate than $I_{GAN-DIME}$.  

One limitation of this contribution is that the set of $f$-divergences analyzed is restricted to three elements. Thus, probably there exists a more effective $f$-divergence which is not analyzed here. 

\begin{table}
	\centering
	\caption{Summary of the MI estimators.}
	\begin{tabular}{c||c|c|c|c|c} 
        \toprule
		\multirow{2}{*}{\textbf{Estimator}} & \multicolumn{2}{c|}{\textbf{Low MI}} & \multicolumn{2}{c|}{\textbf{High MI}} & \multirow{2}{*}{\textbf{Scalability}} \\
		\cmidrule{2-5} 
		 & \textbf{Bias} 	& \textbf{Variance} & \textbf{Bias}	& \textbf{Variance} \\
		\midrule
		$I_{KL-DIME}$ & \cmark\cmark & \cmark\cmark & $\sim$ & \cmark\cmark & \cmark\cmark \\ 
        $I_{HD-DIME}$ & \cmark\cmark & \cmark\cmark & \cmark & \cmark & \cmark\cmark \\ 
        $I_{GAN-DIME}$ & \cmark\cmark & \cmark\cmark & \cmark\cmark & \cmark & \cmark\cmark \\ 
        \midrule
        $I_{SMILE} (\tau=1)$ & \cmark & \cmark\cmark & \cmark & \cmark & \cmark\cmark \\ 
        $I_{NJEE}$ & \cmark & \cmark\cmark & \cmark & \cmark\cmark & \xmark \\ 
        $I_{CPC}$ & $\sim$ & \cmark\cmark & \xmark& \cmark\cmark & \xmark \\ 
        \midrule
        $I_{SMILE} (\tau=\infty)$ & \cmark & $\sim$ & \cmark & \xmark & \cmark\cmark \\ 
        $I_{MINE}$ & \cmark & \xmark & \xmark& \xmark & \cmark\cmark \\ 
        $I_{NWJ}$ & \cmark & \xmark & \xmark& \xmark & \cmark\cmark \\ 
		\midrule
	\end{tabular}
	\label{tab:MI_summary_estimators}
\end{table}

\subsubsection{Self-consistency tests}
\label{subsec:mi_appendix_consistency_tests}
The benchmark considered for the self-consistency tests is similar to the one applied in prior work \cite{Song2020}. We use the images collected in MNIST \cite{deng2012mnist} and FashionMNIST \cite{xiao2017fashion} data sets. Here, we test three properties of MI estimators over images distributions, where the MI is not known, but the estimators consistency can be tested:
\begin{enumerate}
    \item \textbf{Baseline}. $X$ is an image, $Y$ is the same image masked in such a way to show only the top $t$ rows. The value of $\hat{I}(X;Y)$ should be non-decreasing in $t$, and for $t=0$ the estimate should be equal to 0, since $X$ and $Y$ would be independent. Thus, the ratio $\hat{I}(X;Y)/\hat{I}(X;X)$ should be monotonically increasing, starting from $0$ and converging to $1$.  
    \item \textbf{Data-processing}. $\bar{X}$ is a pair of identical images, $\bar{Y}$ is a pair containing the same images masked with two different values of $t$. We set $h(Y)$ to be an additional masking of $Y$ of $3$ rows. The estimated MI should satisfy $\hat{I}([X,X];[Y, h(Y)])/\hat{I}(X;Y) \approx 1$, since including further processing should not add information.  
    \item \textbf{Additivity}. $\bar{X}$ is a pair of two independent images, $\bar{Y}$ is a pair containing the masked versions (with equal $t$ values) of those images. The estimated MI should satisfy $\hat{I}([X_1,X_2];[Y_1, Y_2])/\hat{I}(X;Y) \approx 2$, since the realizations of the $X$ and $Y$ random variables are drawn independently. 
\end{enumerate}
These tests are developed for $I_{fDIME}$, $I_{CPC}$, and $I_{SMILE}$. Differently, $I_{NJEE}$ training is not feasible, since by construction $2d-1$ models should be created, with $d=784$ (the gray-scale image shape is $28 \times 28$ pixels).
The neural network architecture used for these tests is referred to as \textbf{conv}.

\textbf{Conv}. It is composed by two convolutional layers and one fully connected. The first convolutional layer has $64$ output channels and convolves the input images with $(5 \times 5)$ kernels, stride $2 \> px$ and padding $2 \> px$. The second convolutional layer has $128$ output channels, kernels of shape $(5 \times 5)$, stride $2 \> px$ and padding $2 \> px$. The fully connected layer contains $1024$ neurons. ReLU activation functions are used in each layer (except from the last one). The input data are concatenated along the channel dimension.  We set the batch size equal to $256$.

The comparison between the MI estimators for varying values of $t$ is reported in Fig. \ref{fig:MI_baselines}, \ref{fig:MI_data processings}, and \ref{fig:MI_additivities}. The behavior of all the estimators is evaluated for various random seeds. These results highlight that almost all the analyzed estimators satisfy the first two tests ($I_{HD-DIME}$ is slightly unstable), while none of them is capable of fulfilling the additivity criterion. Nevertheless, this does not exclude the existence of an $f$-divergence capable to satisfy all the tests. 
\begin{figure}
\centering
\begin{subfigure}{.5\textwidth}
  \centering
  \includegraphics[scale=0.4]{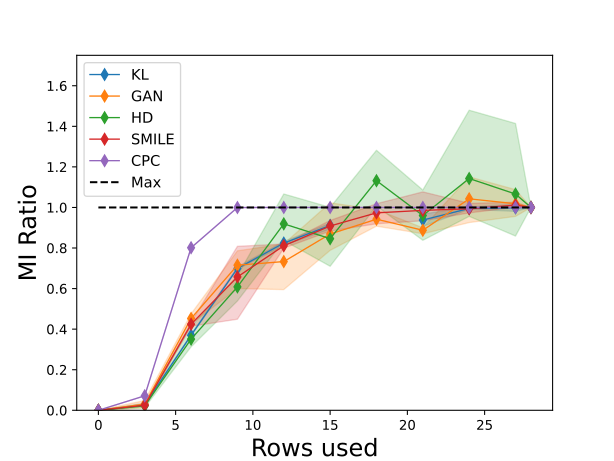}
  \caption{Baseline property, MNIST digits data set.}
  \label{fig:MI_baseline_digits}
\end{subfigure}%
\begin{subfigure}{.5\textwidth}
  \centering
  \includegraphics[scale=0.4]{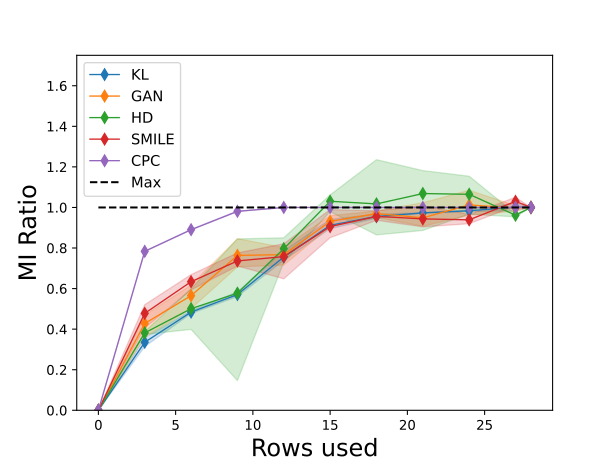}
  \caption{Baseline property, FashionMNIST data set.}
  \label{fig:MI_baseline_fashion}
\end{subfigure}
\caption{Comparison between different estimators for the baseline property, using MNIST data set on the left and FashionMNIST on the right.}
\label{fig:MI_baselines}
\end{figure}

\begin{figure}
\centering
\begin{subfigure}{.5\textwidth}
  \centering
  \includegraphics[scale=0.4]{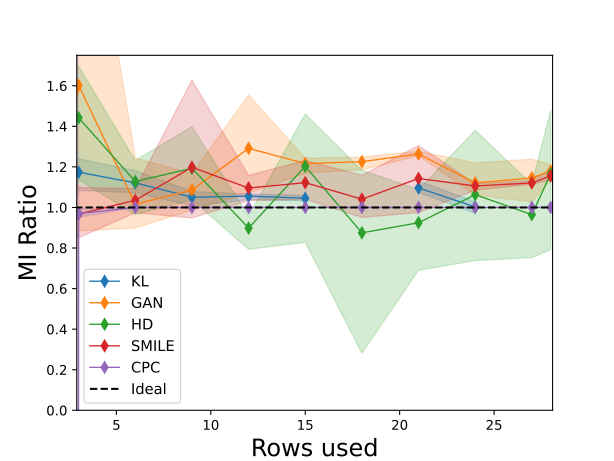}
  \caption{Data processing property, MNIST digits data set.}
  \label{fig:MI_data_processing_digits}
\end{subfigure}%
\begin{subfigure}{.5\textwidth}
  \centering
  \includegraphics[scale=0.4]{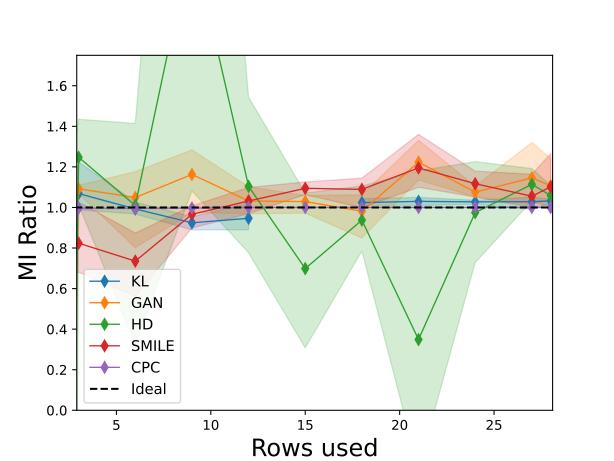}
  \caption{Data processing property, FashionMNIST data set.}
  \label{fig:MI_data_processing_fashion}
\end{subfigure}
\caption{Comparison between different estimators for the data processing property, using MNIST data set on the left and FashionMNIST on the right.}
\label{fig:MI_data processings}
\end{figure}

\begin{figure}
\centering
\begin{subfigure}{.5\textwidth}
  \centering
  \includegraphics[scale=0.4]{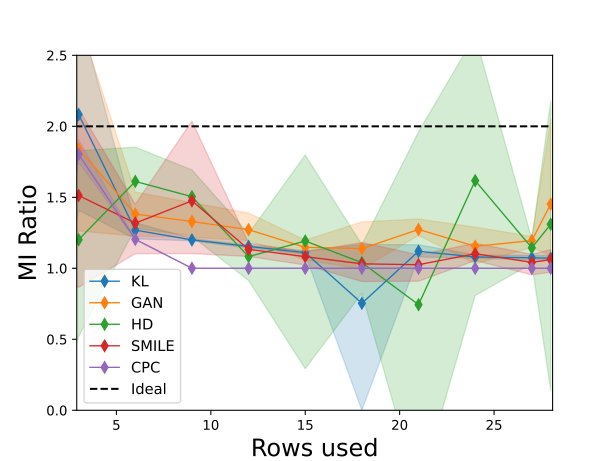}
  \caption{Additivity property, MNIST digits data set.}
  \label{fig:MI_additivity_digits}
\end{subfigure}%
\begin{subfigure}{.5\textwidth}
  \centering
  \includegraphics[scale=0.4]{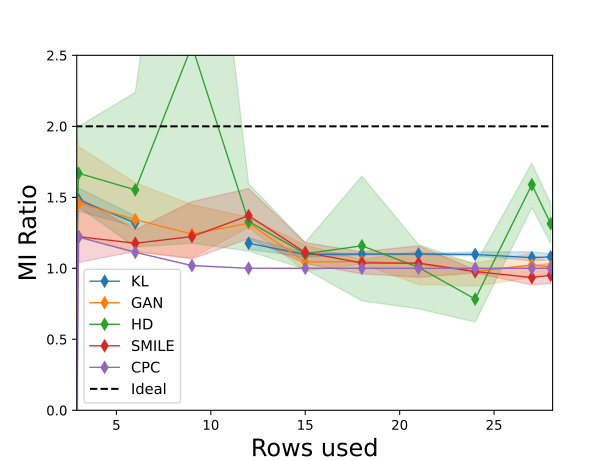}
  \caption{Additivity property, FashionMNIST data set.}
  \label{fig:MI_additivity_fashion}
\end{subfigure}
\caption{Comparison between different estimators for the additivity property, using MNIST data set on the left and FashionMNIST on the right.}
\label{fig:MI_additivities}
\end{figure}
\chapter{Cooperative Channel Capacity Learning} 
\chaptermark{Cooperative Channel Capacity Learning}
\label{sec:cortical}

In this chapter, the problem of determining the capacity of a communication channel is formulated as a cooperative game, between a generator and a discriminator, that is solved via deep learning techniques. The task of the generator is to produce channel input samples for which the discriminator ideally distinguishes conditional from unconditional channel output samples.

The learning approach, referred to as cooperative channel capacity learning (CORTICAL), provides both the optimal input signal distribution and the channel capacity estimate. Numerical results demonstrate that the proposed framework learns the capacity-achieving input distribution under challenging non-Shannon settings.

The results presented in this chapter are documented in \cite{LetiziaNIPS, CORTICAL}.

\section{Introduction}
\sectionmark{Introduction}
\label{sec:cortical_intro}
While the benefit of a DL approach \cite{Oshea2017} is evident for tasks like signal detection \cite{Ye2017}, decoding \cite{tonello2022mind} and channel estimation \cite{Ye2018,Soltani2019}, the fundamental problem of estimating the capacity of a channel remains elusive. Despite recent attempts via neural MI estimation \cite{Aharoni2020, Letizia2021,Farhad2022}, it is not clear yet whether DL can provide novel insights.

For a discrete-time continuous memoryless vector channel, the capacity is defined as
\begin{equation}
C = \max_{p_X(\mathbf{x})} I(X;Y),
\end{equation}
where $p_X(\mathbf{x})$ is the input signal PDF, $X$ and $Y$ are the channel input and output random vectors, respectively, and $I(X; Y)$ is the MI between $X$ and $Y$. 
The channel capacity problem involves both determining the capacity-achieving distribution and evaluating the maximum achievable rate. Only a few special cases, e.g., additive noise channels with specific noise distributions under input power constraints, have been solved so far. When the channel is not an additive noise channel, analytical approaches become mostly intractable leading to numerical solutions, relaxations \cite{CuttingPlane}, capacity lower and upper bounds \cite{McKellips2004}, and considerations on the support of the capacity-achieving distribution \cite{Smith1971, Dytso2020}. 
It is known that the capacity of a discrete memoryless channel can be computed using the Blahut-Arimoto (BA) algorithm \cite{BlahutArimoto}, whereas a particle-based BA method was proposed in \cite{Dauwels2005} to tackle the continuous case although it fails to scale to high-dimension vector channels.

In the following section, we show that a data-driven approach can be pursued to obtain the channel capacity. In particular, we propose to learn the capacity and the capacity-achieving distribution of any discrete-time continuous memoryless vector channel via a cooperative framework referred to as CORTICAL. The framework is inspired by GANs \cite{Goodfellow2014} but it can be interpreted as a dual version using an appropriately defined value function. In fact, CORTICAL comprises two blocks cooperating with each other: a generator that learns to sample from the capacity-achieving distribution, and a discriminator that learns to differentiate paired channel input-output samples from unpaired ones, i.e. it distinguishes the joint PDF $p_{XY}(\mathbf{x},\mathbf{y})$ from the product of marginal $p_{X}(\mathbf{x})p_{Y}(\mathbf{y})$, so as to estimate the MI.

\section{Cooperative principle for capacity learning}
\sectionmark{Cortical}
\label{sec:cortical_theory}
To understand the working principle of CORTICAL and why it can be interpreted as a cooperative approach, it is useful to briefly review how GANs operate. 

In the GAN framework, the adversarial training procedure for the generator $G$ consists of maximizing the probability of the discriminator $D$ making a mistake. If $\mathbf{x}\sim p_{\text{data}}(\mathbf{x})$ are the data samples and $\mathbf{z}\sim p_{Z}(\mathbf{z})$ are the latent samples, the Nash equilibrium is reached when the minimization over $G$ and the maximization over $D$ of the value function
\begin{equation}
\mathcal{V}(G,D) = \mathbb{E}_{\mathbf{x} \sim p_{\text{data}}(\mathbf{x})}\biggl[\log \bigl(D(\mathbf{x})\bigr)\biggr]  +\mathbb{E}_{\mathbf{z} \sim p_{Z}(\mathbf{z})}\biggl[\log\bigl(1-D\bigl(G(\mathbf{z})\bigr)\bigr)\biggr],
\label{eq:CORTICAL_GAN_value_function}
\end{equation}
is attained so that $G(\mathbf{z})\sim p_{\text{data}}(\mathbf{x})$. Concisely, the minimization over $G$ forces the generator to implicitly learn the distribution that minimizes the given statistical distance. 

\begin{figure}
	\centering
	\includegraphics[scale=0.45]{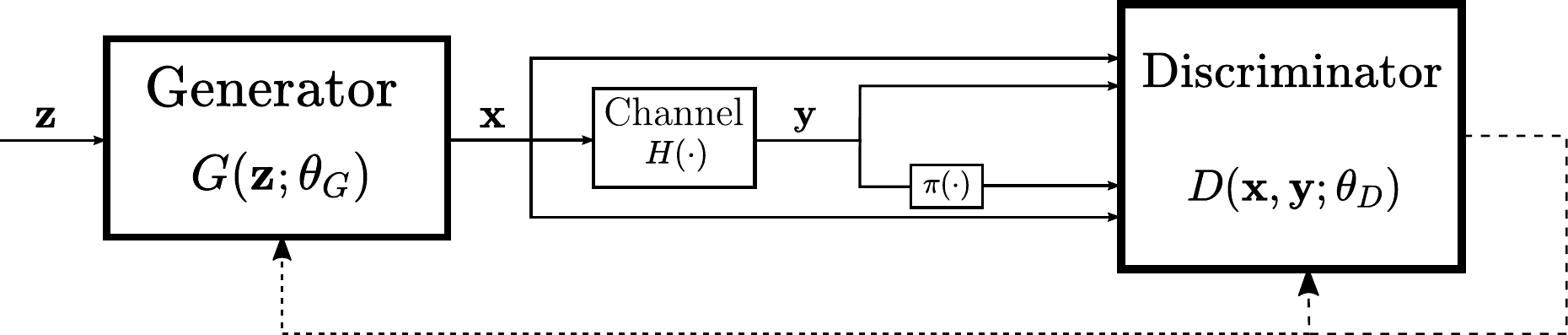}
	\caption{CORTICAL, Cooperative framework for capacity learning: a generator produces input samples with distribution $p_X(\mathbf{x})$ and a discriminator attempts to distinguish between paired and unpaired channel input-output samples.}
	\label{fig:CORTICAL_Cooperative_networks}
\end{figure}

Conversely to GANs, the channel capacity estimation problem requires the generator to learn the distribution maximizing the MI measure. 
Therefore, the generator and discriminator need to play a cooperative max-max game with the objective for $G$ to produce channel input samples for which $D$ exhibits the best performance in distinguishing (in the KL sense) paired and unpaired channel input-output samples. Thus, the discriminator of CORTICAL is fed with both the generator and channel's output samples, $\mathbf{x}$ and $\mathbf{y}$, respectively.
Fig. \ref{fig:CORTICAL_Cooperative_networks} illustrates the proposed cooperative framework that learns both for the cooperative training, as discussed next.

\begin{theorem}
\label{theorem:CORTICAL_Theorem1}
 Let $X\sim p_X(\mathbf{x})$ and $Y|X\sim p_{Y|X}(\mathbf{y}|\mathbf{x})$ be the vector channel input and conditional output, respectively. Let $Y = H(X)$ with $H(\cdot)$ being the stochastic channel model and let $\pi(\cdot)$ be a permutation function \footnote{The permutation takes place over the temporal realizations of the vector $Y$ so that $\pi(Y)$ and $X$ can be considered statistically independent vectors, for instance with the derangement strategy.} over the realizations of $Y$ such that $p_{\pi(Y)}(\pi(\mathbf{y})|\mathbf{x}) = p_{Y}(\mathbf{y})$. In addition, let $G$ and $D$ be two functions in the non-parametric limit such that $X = G(Z)$ with $Z\sim p_Z(\mathbf{z})$ being a latent random vector. If $\mathcal{J}_{\alpha}(G,D)$, $\alpha>0$, is the value function defined as 
\begin{equation}
\mathcal{J}_{\alpha}(G,D) = \alpha \cdot \mathbb{E}_{\mathbf{z} \sim p_{Z}(\mathbf{z})}\biggl[\log \biggl(D\biggl(G(\mathbf{z}),H(G(\mathbf{z}))\biggr)\biggr)\biggr]  -\mathbb{E}_{\mathbf{z} \sim p_{Z}(\mathbf{z})}\biggl[D\biggl(G(\mathbf{z}),\pi(H(G(\mathbf{z})))\biggr)\biggr],
\label{eq:CORTICAL_value_function}
\end{equation}
then the channel capacity $C$ is the solution of
\begin{equation}
C = \max_{G} \max_{D} \frac{\mathcal{J}_{\alpha}(G,D)}{\alpha} + 1- \log(\alpha),
\end{equation}
and $\mathbf{x}=G^*(\mathbf{z})$ are samples from the capacity-achieving distribution, where 
\begin{equation}
G^* = \argmax_G \max_D \mathcal{J}_{\alpha}(G,D).
\end{equation}
\end{theorem}

\begin{proof}
The value function in (\ref{eq:CORTICAL_GAN_value_function}) can be written as
\begin{equation}
\mathcal{J}_{\alpha}(G,D) = \alpha \cdot \mathbb{E}_{(\mathbf{x},\mathbf{y}) \sim p_{XY}(\mathbf{x},\mathbf{y})}\biggl[\log \biggl(D(\mathbf{x},\mathbf{y})\biggr)\biggr] -\mathbb{E}_{(\mathbf{x},\mathbf{y}) \sim p_{X}(\mathbf{x}) p_Y(\mathbf{y})}\biggl[D(\mathbf{x},\mathbf{y})\biggr].
\label{eq:CORTICAL_value_function_p}
\end{equation}
Given a generator $G$ that maps the latent (noise) vector $Z$ into $X$, we need firstly to prove that $\mathcal{J}_{\alpha}(G,D)$ is maximized for 
\begin{equation}
D(\mathbf{x},\mathbf{y})=D^*(\mathbf{x},\mathbf{y})=\alpha  \frac{p_{XY}(\mathbf{x},\mathbf{y})}{p_{X}(\mathbf{x}) p_Y(\mathbf{y})}.
\end{equation}
Using the Lebesgue integral to compute the expectation
\begin{equation}
\mathcal{J}_{\alpha}(G, D) = \alpha  \int_{\mathbf{y}} \int_{\mathbf{x}}\biggl[p_{XY}(\mathbf{x},\mathbf{y}) \log \biggl(D(\mathbf{x},\mathbf{y})\biggr) - p_{X}(\mathbf{x}) p_Y(\mathbf{y}) \biggl(D(\mathbf{x},\mathbf{y})\biggr)\biggr] \diff \mathbf{x} \diff \mathbf{y},
\label{eq:CORTICAL_value_function_q}
\end{equation}
taking the derivative of the integrand with respect to $D$ and setting it to $0$, yields the following equation in $D$:
\begin{equation}
\alpha  \frac{p_{XY}(\mathbf{x},\mathbf{y})}{D(\mathbf{x},\mathbf{y})} - p_{X}(\mathbf{x}) p_Y(\mathbf{y})=0,
\end{equation}
whose solution is the optimum discriminator
\begin{equation}
D^*(\mathbf{x},\mathbf{y}) = \alpha  \frac{p_{XY}(\mathbf{x},\mathbf{y})}{p_{X}(\mathbf{x}) p_Y(\mathbf{y})}.
\end{equation}
In particular, $\mathcal{J}_{\alpha}(D^*)$ is a maximum since the second derivative of the integrand $-\alpha  \frac{p_{XY}(\mathbf{x},\mathbf{y})}{D^2(\mathbf{x},\mathbf{y})}$ is a non-positive function.
Therefore, substituting $D^*(\mathbf{x},\mathbf{y})$ in \eqref{eq:CORTICAL_value_function_q} yields
\begin{align}
\mathcal{J}_{\alpha}(G,D^*) = \; & \alpha \int_{\mathbf{y}} \int_{\mathbf{x}}{\biggl[p_{XY}(\mathbf{x},\mathbf{y}) \log \biggl(\alpha  \frac{p_{XY}(\mathbf{x},\mathbf{y})}{p_{X}(\mathbf{x}) p_Y(\mathbf{y})}} \biggr) \nonumber \\
& - p_{X}(\mathbf{x}) p_Y(\mathbf{y}) \biggl(\alpha \frac{p_{XY}(\mathbf{x},\mathbf{y})}{p_{X}(\mathbf{x}) p_Y(\mathbf{y})}\biggr)\biggr] \diff \mathbf{x} \diff \mathbf{y}.
\end{align}
Now, it is simple to recognize that the second term on the right hand side of the above equation is equal to $-\alpha$. Thus, 
\begin{align}
\mathcal{J}_{\alpha}(G,D^*) & = \; \alpha \cdot \mathbb{E}_{(\mathbf{x},\mathbf{y}) \sim p_{XY}(\mathbf{x},\mathbf{y})}\biggl[\log \biggl(\frac{p_{XY}(\mathbf{x},\mathbf{y})}{p_{X}(\mathbf{x}) p_Y(\mathbf{y})} \biggr)\biggr]  + \alpha \log(\alpha) - \alpha \nonumber \\
& = \alpha \bigl(I(X;Y)+\log(\alpha)-1\bigr).
\end{align}
Finally, the maximization over the generator $G$ results in the mutual information maximization since $\alpha$ is a positive constant. We therefore obtain,  
\begin{equation}
\max_{G} \mathcal{J}_{\alpha}(G,D^*)  + \alpha - \alpha \log(\alpha) = \alpha \max_{p_X(\mathbf{x})} {I}(X;Y).
\end{equation}
\end{proof}
Theorem \ref{theorem:CORTICAL_Theorem1} states that at the equilibrium, the generator of CORTICAL has implicitly learned the capacity-achieving distribution with the mapping $\mathbf{x}=G^*(\mathbf{z})$. 
In contrast to the BA algorithm, here the generator samples directly from the optimal input distribution $p_X(\mathbf{x})$ rather than explicitly modeling it. No assumptions on the input distribution's nature are made. Moreover, we have access to the channel capacity directly from the value function used for training as follows:
\begin{equation}
    \label{eq:CORTICAL_capacity}
    C = \frac{\mathcal{J}_{\alpha}(G^*,D^*)}{\alpha} + 1- \log(\alpha).
\end{equation}

In the following, we propose to parametrize $G$ and $D$ with NNs and we explain how to train CORTICAL.

\subsection{Parametric implementation}
\label{subsec:cortical_implementation}
It can be shown (see Sec.4.2 in \cite{Goodfellow2014}) that the alternating training strategy described in Alg. \ref{alg:CORTICAL_1}
converges to the optimal $G$ and $D$ under the assumption of having enough capacity and training time. Practically, instead of optimizing over the space of functions $G$ and $D$, it is reasonable to model both the generator and the discriminator with NNs $(G,D)=(G_{\theta_G},D_{\theta_D})$ and optimize over their parameters $\theta_G$ and $\theta_D$ (see Alg. \ref{alg:CORTICAL_1}).

We consider the distribution of the source $p_Z(\mathbf{z})$ to be a multivariate normal distribution with independent components. The function $\pi (\cdot)$ implements a random derangement of the batch $\mathbf{y}$ and it is used to obtain unpaired samples. We use simple NN architectures and we execute $K=10$ discriminator training steps every generator training step. Details of the architecture and implementation are reported in \cite{CORTICAL_github}, together with animated graphs.

It should be noted that in Theorem \ref{theorem:CORTICAL_Theorem1}, $p_X(\mathbf{x})$ can be subject to certain constraints, e.g., peak and/or average power. Such constraints are met by the design of $G$, for instance using a batch normalization layer. 
Alternatively, we can impose peak and/or average power constraints in the estimation of capacity by adding regularization terms (hinge loss) in the generator part of the value function \eqref{eq:CORTICAL_value_function}, as in constrained optimization problems, e.g., Sec. III of \cite{Faycal2001}. Specifically, the value function becomes
\begin{align}
\mathcal{J}_{\alpha}(G,D) = \; & \alpha \cdot \mathbb{E}_{\mathbf{z} \sim p_{Z}(\mathbf{z})}\biggl[\log \biggl(D\biggl(G(\mathbf{z}),H(G(\mathbf{z}))\biggr)\biggr)\biggr] -\mathbb{E}_{\mathbf{z} \sim p_{Z}(\mathbf{z})}\biggl[D\biggl(G(\mathbf{z}),\pi(H(G(\mathbf{z})))\biggr)\biggr] \nonumber \\ 
& - \lambda_A \max(||G(\mathbf{z})||^2_2-A^2,0) - \lambda_P \max(\mathbb{E}[||G(\mathbf{z})||^2_2]-P,0),
\label{eq:CORTICAL_value_function_lambda}
\end{align}
with $\lambda_A$ and $\lambda_P$ equal to $0$ or $1$.

\begin{algorithm}[t]
\caption{Cooperative Channel Capacity Learning}
\label{alg:CORTICAL_1}
\begin{algorithmic}[1]
\Inputs{$N$ training steps, $K$ discriminator steps, $\alpha$.}
\For{$n=1$ to $N$}
    \For{$k=1$ to $K$}
	\State \multiline{Sample batch of $m$ noise samples $\{\mathbf{z}^{(1)},\dots,\mathbf{z}^{(m)}\}$ from $p_Z(\mathbf{z})$;}
	\State \multiline{Produce batch of $m$ channel input/output paired samples $\{(\mathbf{x}^{(1)},\mathbf{y}^{(1)}),\dots,(\mathbf{x}^{(m)},\mathbf{y}^{(m)})\}$ using the generator $G_{\theta_G}$ and the channel model $H$;}
        \State \multiline{Shuffle (derangement) $\mathbf{y}$ and get input/output unpaired samples $\{(\mathbf{x}^{(1)},\mathbf{\tilde{y}}^{(1)}),\dots,(\mathbf{x}^{(m)},\mathbf{\tilde{y}}^{(m)})\}$;}
        \State \multiline{Update the discriminator by ascending its stochastic gradient:}
        \begin{equation*}
            \nabla_{\theta_D} \frac{1}{m} \sum_{i=1}^{m}{\alpha \log\bigl(D_{\theta_D}\bigl(\mathbf{x}^{(i)},\mathbf{y}^{(i)}\bigr)\bigr)-D_{\theta_D}}\bigl(\mathbf{x}^{(i)},\mathbf{\tilde{y}}^{(i)}\bigr).
        \end{equation*}
\EndFor
    \State \multiline{Sample batch of $m$ noise samples $\{\mathbf{z}^{(1)},\dots,\mathbf{z}^{(m)}\}$ from $p_Z(\mathbf{z})$;}
    \State \multiline{Update the generator by ascending its stochastic gradient:}
        \begin{align*}
            \nabla_{\theta_G} \frac{1}{m} & \sum_{i=1}^{m}{\alpha \log\biggl( D_{\theta_D}\biggl(G_{\theta_G}\bigl(\mathbf{z}^{(i)}\bigr),H\bigl(G_{\theta_G}\bigl(\mathbf{z}^{(i)}\bigr)\bigr)\biggr)\biggr)} \\
            & -D_{\theta_D}\biggl(G_{\theta_G}\bigl(\mathbf{z}^{(i)}\bigr),\pi \bigl(H\bigl(G_{\theta_G}\bigl(\mathbf{z}^{(i)}\bigr)\bigr)\bigr)\biggr).
        \end{align*}
\EndFor
\end{algorithmic}
\end{algorithm}

For an AWGN channel under an average power constraint the exact form of the capacity and the capacity-achieving input distribution are well known. Using CORTICAL to retrieve the optimal Gaussian (of variance $P$) input is not a sufficient proof of concept since one may argue that the neural generator implements a sort of central limit theorem approximation. Therefore, we decide to apply the cooperative framework in three, more representative, non-Shannon's setups.

\section{Applications to non-Shannon scenarios}
\sectionmark{non-Shannon scenarios}
\label{sec:cortical_results}
To demonstrate the ability of CORTICAL to learn the optimal channel input distribution, we evaluate its performance in three non-standard scenarios: 1) the AWGN channel subject to a peak-power constrained input; 2) an additive non-Gaussian noise channel subject to two different input power constraints; and 3) the Rayleigh fading channel known at both the transmitter and the receiver subject to an average power constraint. These are among the few scenarios for which analysis has been carried out in the literature and thus offer a baseline to benchmark CORTICAL. For both the first and third scenarios, it is known that the capacity-achieving distribution is discrete with a finite set of mass points \cite{Smith1971,Tchamkerten2004,Dytso2020}. For the second scenario the nature of the input distribution depends on the type of input power constraint \cite{Fahs2014}. Additional results including the study of the classical AWGN channel with average power constraint are reported in GitHub \cite{CORTICAL_github}.

\subsection{Peak power-limited Gaussian channels}
The capacity of the discrete-time memoryless vector Gaussian noise channel with unit-variance under peak-power constraints on the input is defined as \cite{Rassouli2016}
\begin{equation}
C(A) = \sup_{p_X(\mathbf{x}): ||X||_2\leq A} I(X;Y),
\end{equation}
where $p_X(\mathbf{x})$ is the channel input PDF and $A^2$ is the upper bound for the input signal peak power.
The AWGN Shannon capacity constitutes a trivial upper bound for $C(A)$,
\begin{equation}
C(A) \leq \frac{d}{2}\log_2\biggl(1+\frac{A^2}{d}\biggr),
\end{equation}
while a tighter upper bound for the scalar channel is provided in \cite{McKellips2004}
\begin{equation}
C(A) \leq \min\biggl\{\log_2\biggl(1+\frac{2A}{\sqrt{2\pi e}}\biggr), \frac{1}{2}\log_2\bigl(1+A^2\bigr) \biggr\}.
\end{equation}
\begin{figure}
	\centering
	\includegraphics[scale=0.35]{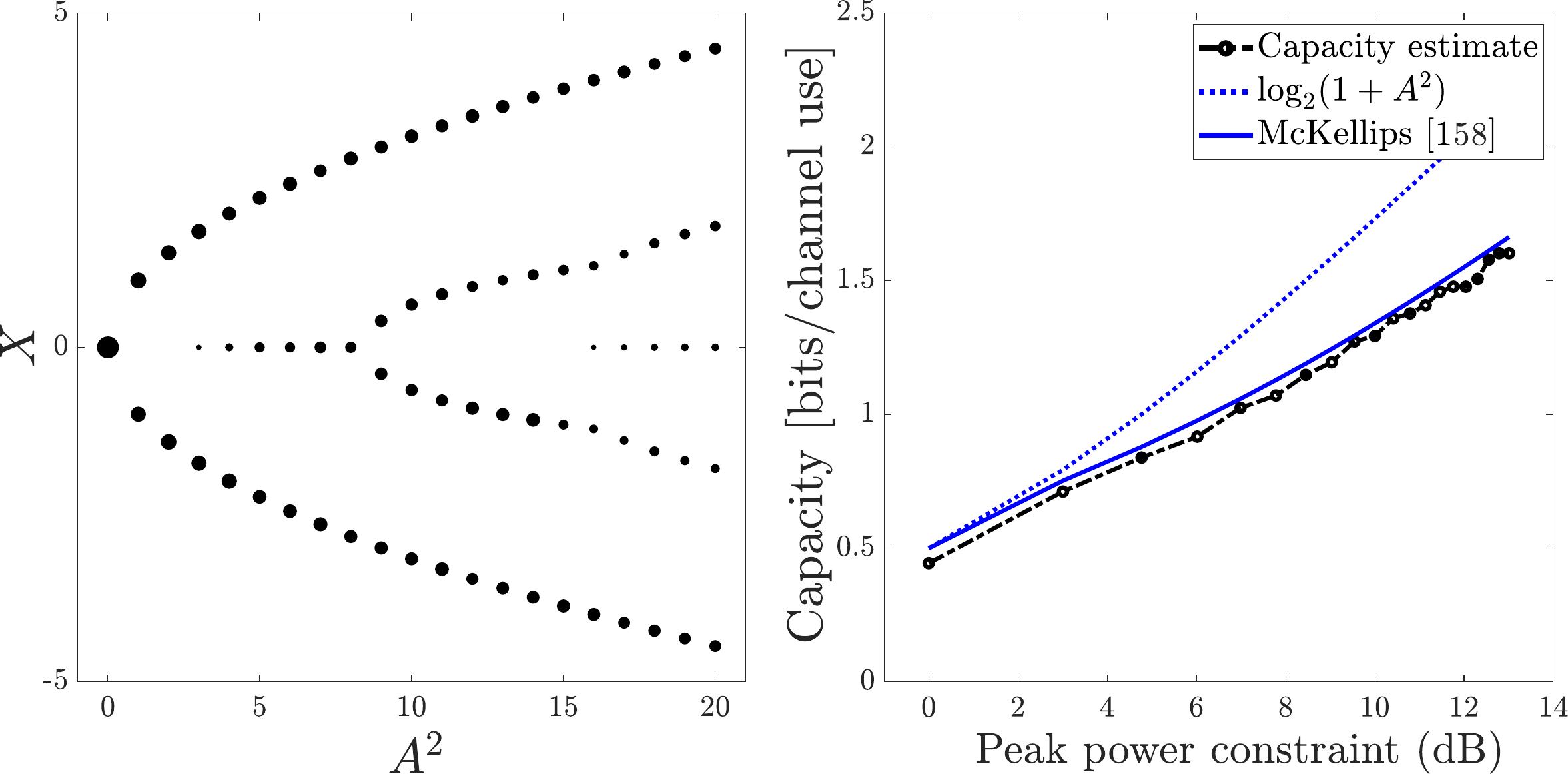}
	\caption{AWGN scalar peak-power constrained channel: a) capacity-achieving distribution learned by CORTICAL, the marker's radius is proportional to the PMF; b) capacity estimate and comparisons.}
	\label{fig:CORTICAL_bifurcation}
\end{figure} 
For convenience of comparison and as a proof of concept, we focus on the scalar $d=1$ and $d=2$ channels. 
The first findings on the capacity-achieving discrete distributions were reported in \cite{Smith1971}. For a scalar channel, it was shown in \cite{Sharma2010} that the input has alphabet $\{-A,A\}$ with equiprobable values if $0<A \lessapprox 1.6$, while it has ternary alphabet $\{-A,0,A\}$ if $1.6 \lessapprox A \lessapprox 2.8$.
Those results are confirmed by CORTICAL which is capable of both understanding the discrete nature of the input and learning the support and probability mass function (PMF) for any value of $A$. Notice that no hypothesis on the input distribution is provided during training. Fig. \ref{fig:CORTICAL_bifurcation}a reports the capacity-achieving input distribution as a function of the peak-power constraint $A^2$. The figure illustrates the typical bifurcation structure of the distribution. Fig. \ref{fig:CORTICAL_bifurcation}b shows the channel capacity estimated by CORTICAL with \eqref{eq:CORTICAL_capacity} and compares it with the upper bounds known in the literature.  

\begin{figure}
	\centering
	\includegraphics[scale=0.4]{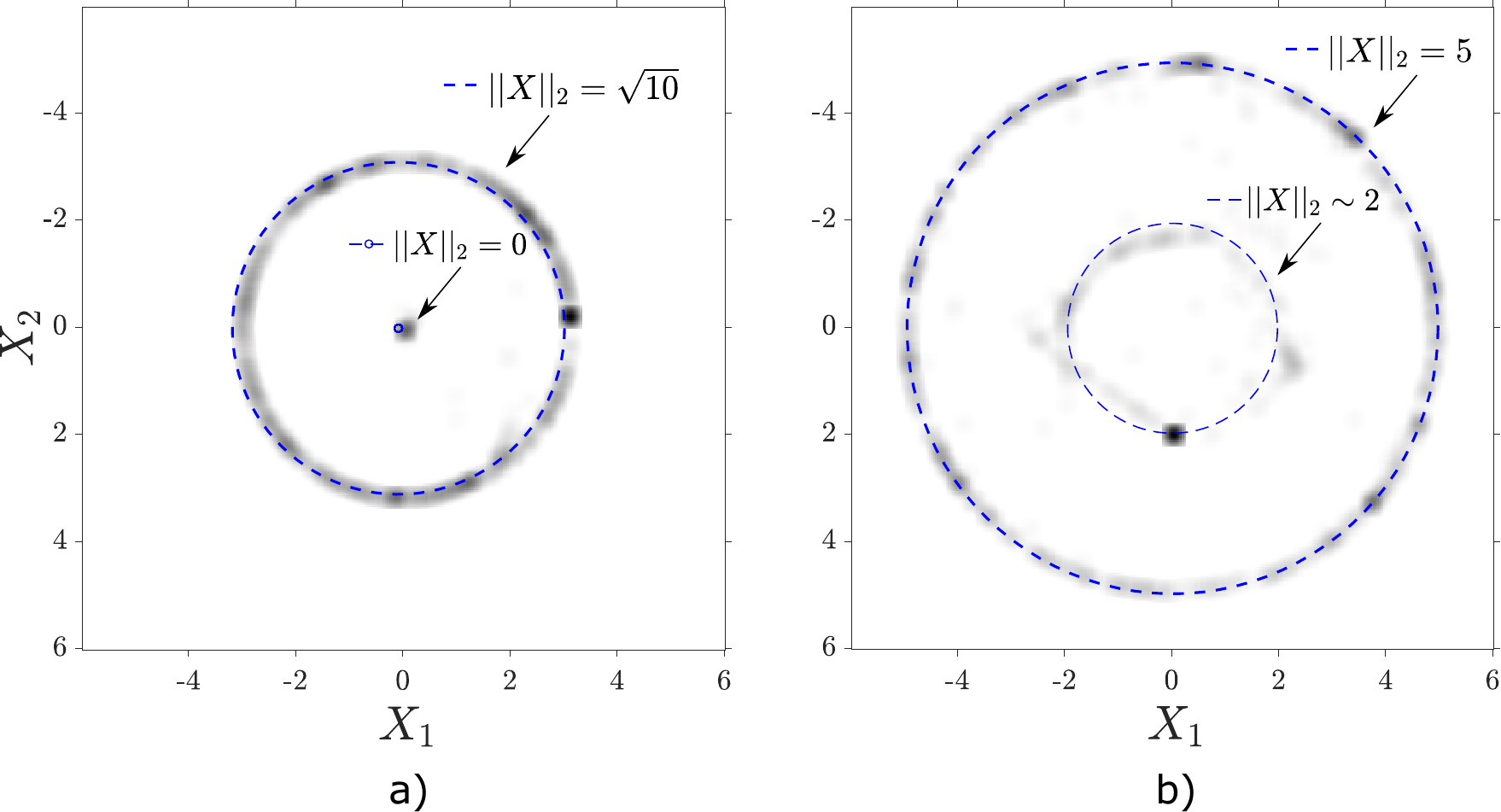}
	\caption{AWGN $d=2$ channel input distributions learned by CORTICAL under a peak-power constraint: a) $A=\sqrt{10}$; b) $A=5$.}
	\label{fig:CORTICAL_heatmaps}
\end{figure} 

\begin{figure}
	\centering
	\includegraphics[scale=0.4]{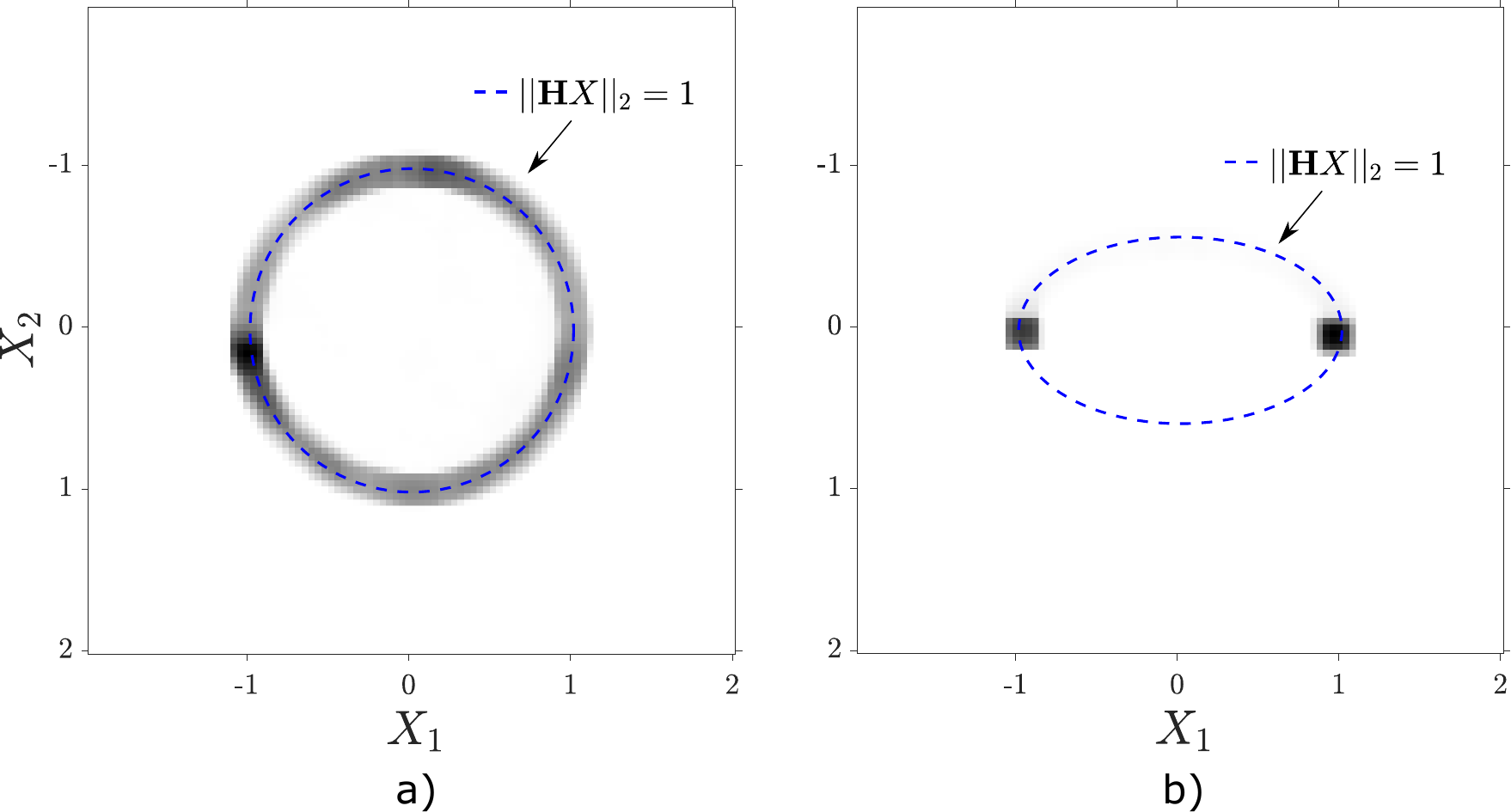}
	\caption{MIMO channel input distributions learned by CORTICAL for different channels $\mathbf{H}$: a) $r_2=1$; b) $r_2=3$. Locus of points satisfying $||\mathbf{H}X||_2= 1$ is also shown.}
	\label{fig:CORTICAL_MIMO}
\end{figure} 
Similar considerations can be made for the $d=2$ channel under peak-power constraints \cite{Rassouli2016}, and in general for the vector Gaussian channel of dimension $d$ \cite{Dytso2019}. 
The case $d=2$ is considered in Fig. \ref{fig:CORTICAL_heatmaps}a ($A=\sqrt{10}$) and Fig. \ref{fig:CORTICAL_heatmaps}b ($A=5$) where CORTICAL learns optimal bi-dimensional distributions, matching the results in \cite{Rassouli2016}. In this case the amplitude is discrete. 

We now consider the MIMO channel where the analytical characterization of the capacity-achieving distribution under a peak-power constraint remains mostly an open problem \cite{Dytso2019-MIMO}. We analyze the particular case of $d=2$: 
\begin{equation}
C(\mathbf{H},r) = \sup_{p_{X}(\mathbf{x}): ||\mathbf{H}X||_2\leq r} I(X;\mathbf{H}X+N),
\end{equation}
where $N \sim \mathcal{N}(0,\mathbf{I}_2)$ and $\mathbf{H} \in \mathbb{R}^{2\times 2}$ is a given MIMO channel matrix known to both transmitter and receiver. We impose $r=1$, and we also impose a diagonal structure on $\mathbf{H}=\text{diag}(1,r_2)$ without loss of generality since diagonalization of the system can be performed. We study two cases: $r_2 = 1$, which is equivalent to a $d=2$ Gaussian channel with unitary peak-power constraint; and $r_2=3$, which forces an elliptical peak-power constraint. The former set-up produces a discrete input distribution in the magnitude and a continuous uniform phase, as shown in Fig. \ref{fig:CORTICAL_MIMO}a. The latter case produces a binary distribution, as shown in Fig. \ref{fig:CORTICAL_MIMO}b. To the best of our knowledge, no analytical results are available for $1<r_2<2$ and thus CORTICAL offers a guiding tool for the identification of capacity-achieving distributions. 

\subsection{Additive non-Gaussian channels}
The nature of the capacity-achieving input distribution remains an open challenge for channels affected by additive non-Gaussian noise. In fact, it is known that under an average power constraint, the capacity-achieving input distribution is discrete and the AWGN channel is the only exception \cite{Fahs2012}. Results concerning the number of mass points of the discrete optimal input have been obtained in \cite{Tchamkerten2004,Das2000}. If the transmitter is subject to an average power constraint, the support is bounded or unbounded depending on the decay rate (slower or faster, respectively) of the noise PDF tail.

\begin{figure}
	\centering
	\includegraphics[scale=0.4]{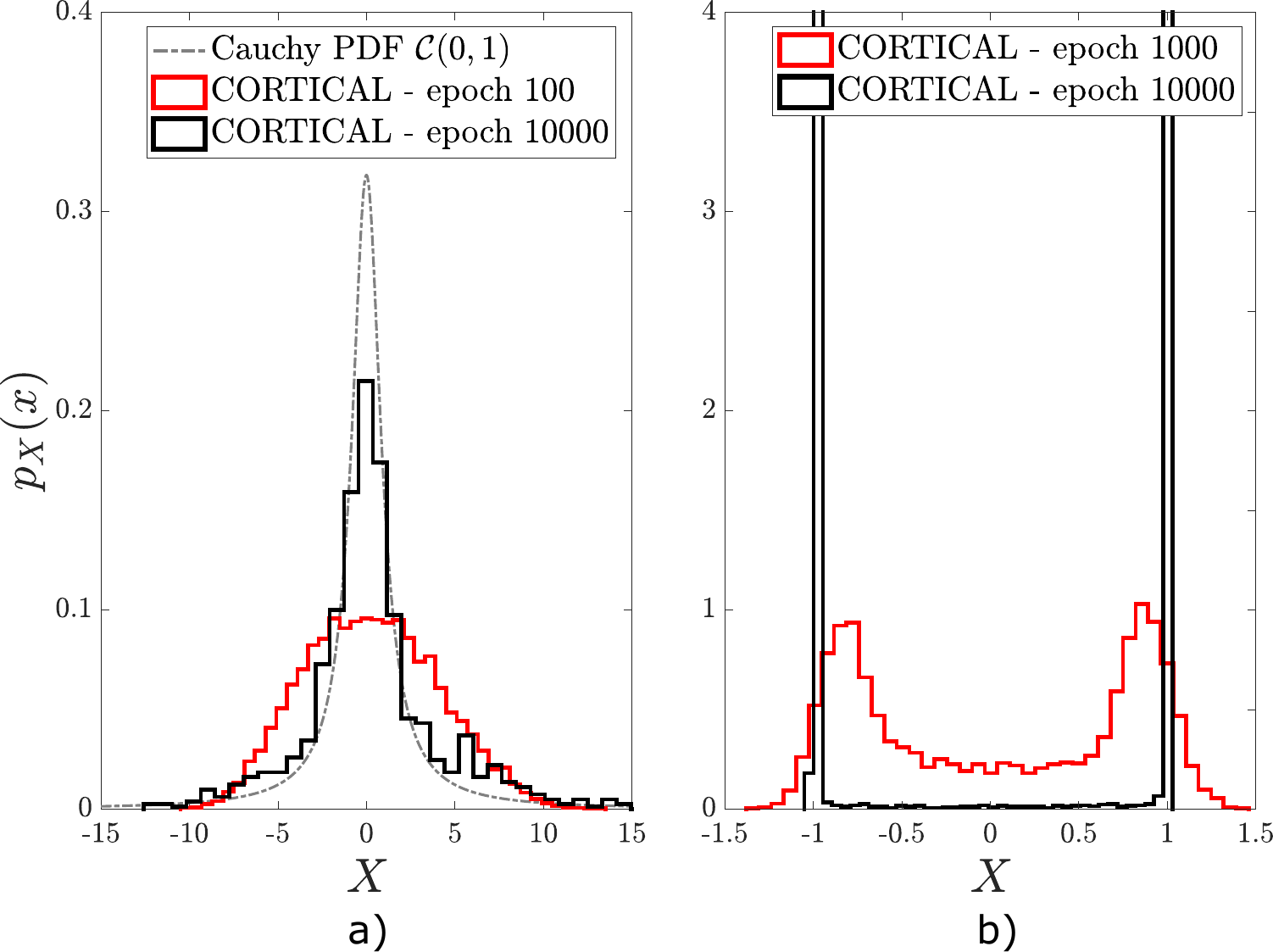}
	\caption{AICN channel input distributions learned by CORTICAL at different training steps under: a) logarithmic power constraint; b) peak-power constraint.}
	\label{fig:CORTICAL_cauchy_subplot}
\end{figure} 
In this section, we consider a scalar additive independent Cauchy noise (AICN) channel with scale factor $\gamma$, under a specific type of logarithmic power constraint \cite{Fahs2014}. In particular, given the Cauchy noise PDF $\mathcal{C}(0,\gamma)$
\begin{equation}
    p_N(n) = \frac{1}{\pi \gamma}\frac{1}{1+\bigl(\frac{n}{\gamma}\bigr)^2},
\end{equation}
we are interested in showing that CORTICAL learns the capacity-achieving distribution that solves
\begin{equation}
C(A,\gamma) = \sup_{p_X(\mathbf{x}): \mathbb{E}\bigl[\log\bigl(\bigl(\frac{A+\gamma}{A}\bigr)^2+\bigl(\frac{X}{A}\bigr)^2\bigr)\bigr]\leq \log(4)} I(X;Y),
\end{equation}
for a given $A\geq \gamma$. From \cite{Fahs2014}, it is known that under such power constraint the channel capacity is $C(A,\gamma)=\log(A/\gamma)$ and the optimal input distribution is continuous with Cauchy PDF $\mathcal{C}(0,A-\gamma)$. For illustration purposes, we study the case of $\gamma=1$ and $A=2$ and report in Fig. \ref{fig:CORTICAL_cauchy_subplot}a the input distribution obtained by CORTICAL after $100$ and $10000$ training steps. For the same channel, we also investigate the capacity-achieving distribution under a unitary peak-power constraint. Fig. \ref{fig:CORTICAL_cauchy_subplot}b shows that the capacity achieving distribution is binary.

\subsection{Fading channels}
As the last representative scenario, we study the Rayleigh fading channel subject to an average power constraint $P$ with input-output relation given by
\begin{equation}
    Y = \alpha X + N,
\end{equation}
where $\alpha$ and $N$ are independent circular complex Gaussian random variables, so that the amplitude of $\alpha$ is Rayleigh-distributed.
\begin{figure}
	\centering
	\includegraphics[scale=0.25]{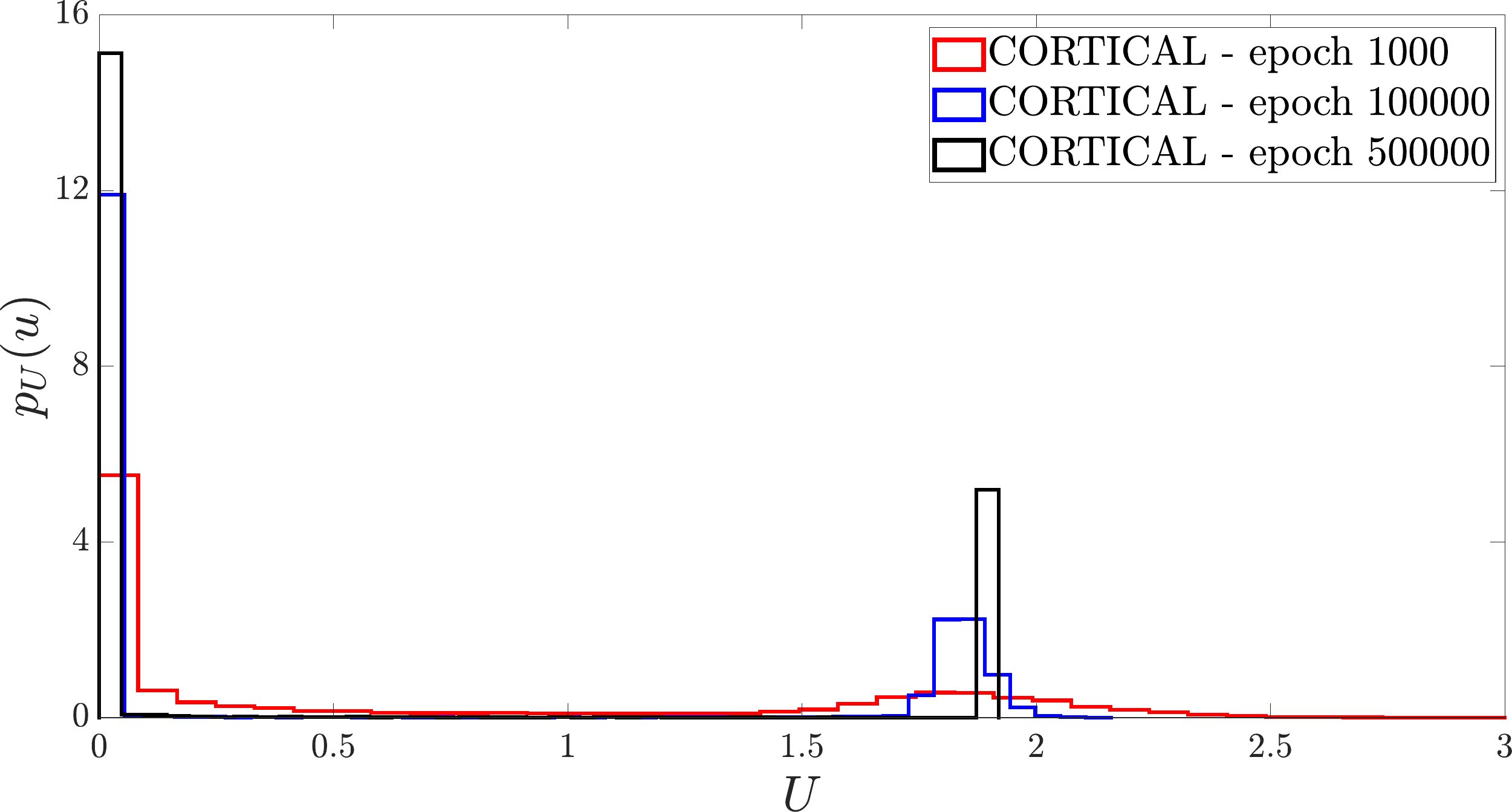}
	\caption{Optimal input $U$ learned by CORTICAL at different training steps.}
	\label{fig:CORTICAL_rayleigh}
\end{figure} 
It is easy to prove that an equivalent model for deriving the distribution of the amplitude of the signal achieving-capacity is obtained by defining $U=|X|\sigma_{\alpha}/\sigma_N$ and $V=|Y|^2/\sigma_N^2$, which are non-negative random variables such that $\mathbb{E}[U^2]\leq P\frac{\sigma_{\alpha}^2}{\sigma_N^2}\triangleq a$ and whose conditional distribution is
\begin{equation}
    p_{V|U}(v|u) = \frac{1}{1+u^2} \exp \biggl(-\frac{v}{1+u^2}\biggr).
\end{equation}
A simplified expression for the conditional PDF is 
\begin{equation}
\label{eq:CORTICAL_rayleigh_channel}
    p_{V|S}(v|s) = s \cdot \exp (-sv)
\end{equation}
where $S = 1/(1+U^2) \in (0,1]$ and $\mathbb{E}[1/S-1]\leq a$.
From \cite{Faycal2001}, it is known that the optimal input signal amplitude $S$ (and thus $U$) is discrete and possesses an accumulation point at $S=1$ ($U=0)$. We use CORTICAL to learn the capacity-achieving distribution and verify that it matches \eqref{eq:CORTICAL_rayleigh_channel} for the case $a=1$ reported in \cite{Faycal2001}. 

Fig. \ref{fig:CORTICAL_rayleigh} illustrates the evolution of the input distribution for different CORTICAL training steps. Also in this scenario, the generator learns the discrete nature of the input distribution and the values of the PMF. 

\section{Summary}
\label{sec:cortical_conclusions}
In this chapter, we have presented CORTICAL, a novel deep learning-based framework that learns the capacity-achieving distribution for any discrete-time continuous memoryless channel. CORTICAL's framework consists of a cooperative max-max game between two NNs. Non-Shannon channel settings have validated the algorithm. The proposed approach offers a novel tool for studying the channel capacity in non-trivial channels and paves the way for the holistic and data-driven analysis/design of communication systems. 

\part{Part 4}

\chapter{Applications to Power Line Communications} 
\chaptermark{Applications to PLC}
\label{sec:plc}

In this chapter, we discuss in detail the application of
ML and DL in the context of power line communication (PLC), a technology which exploits the existing power delivery infrastructure to convey information signals \cite{LampeTonelloSwart}.
We show that PLC represents an excellent case study due to the complex and intricate nature of the power line network. 
In fact, such network is not designed to carry
high frequency signals and multiple effects reduce the communication performances. High noise, electromagnetic interference, channel characterized by strong attenuation and frequency selectivity are the main challenges in PLC. Therefore, it is widely accepted in the community that to develop new efficient and reliable PLC technology it is fundamental to model the PLC physical layer \cite{LampeTonelloSwart}. 
Nevertheless, the complexity of the environment and the multi-parameter dependence render the development of mathematical models extremely challenging. For this reason, a phenomenological approach such as DL constitutes then a great opportunity. 
In the following sections we report several experiments to validate, in the context of PLC, the theoretical findings presented in this thesis. 

The results presented in this chapter are documented in \cite{ML_PLC, RighiniLetizia2019, Letizia2019a, TonelloImpedance, LetiziaIsplc2021, Letizia2019c}.

\section{PLC channel and noise generation}
\sectionmark{PLCs medium modeling}
\label{sec:plc_gan}

\subsection{Channel generation}
\label{subsec:plc_channel}

\subsubsection{Channel modeling methods}
The PLC CTF is usually modeled in time or frequency domain.
Time-domain models are generally described by multi-path effects, whereas frequency-domain models are characterized
by frequency notches. The latter model has been analyzed with three main approaches: bottom-up,
top-down and synthetic [2].

\begin{itemize}
    \item Bottom-up channel modeling refers to an approach where the channel impulse/frequency response is obtained via the application of transmission line theory to a specified network topology, cables and loads characteristics. Conventionally, this approach is applied to obtain a specific response and it is also referred to as deterministic model. Further, for complex networks the existing calculation methods are rather elaborated. 
    \item Top-down channel modeling refers to an approach where the channel impulse/frequency response comes from a parametric model with parameters obtained by fitting real data from measurements. The statistical top-down model allows generating channels with statistics close to those exhibited by real channels [8]. In particular, the model is simple, flexible, and it uses a small set of parameters. The parameters can be adjusted to generate channels according to a certain statistical class.
    \item The synthetic PLC channel model is a phenomenological method to emulate the statistics of the PLC channel. This method does not include the medium physical knowledge to define the model but only the CTFs distribution [3]. Hence, it is possible to develop a model purely based on the statistical description of the observed data, totally abstracting from the physical interpretation of the medium. The synthetic model can be obtained via generative models such as GANs, as described in Sec. \ref{sec:gan_ch_synthesis}.
\end{itemize}

\subsubsection{Metrics for model validation}
The dataset used for the implementation consists of $1312$
measured single-input-single-output CTFs with bandwidth of
$100$ MHz obtained from measurements in the in-home environment
\cite{tonello2014inhome}. The frequency resolution was set to $62.5$
KHz which leads to $1601$ frequency samples. To overcome
lack of channel realizations and complexity problems we
decided to downsample by a factor of $32$. To work with
complex numbers, a vector of $100$ samples has been defined,
whose first $50$ entries correspond to the real part of $H(f)$ and the last flipped $50$ to the imaginary one. This set-up constitutes our input data $\mathbf{x}$: $1312$ realizations of a random variable of dimension $100$.

We use the following metrics to assess the quality of the generated CTFs:
\begin{itemize}
    \item The average channel gain (ACG) corresponds to the mean value of the magnitude of CTF in dB scale
\begin{align}
ACG=\frac{1}{N} \sum\limits_{k=0}^{N-1} |H_k|^2 = T_s^2 \sum\limits_{k=0}^{N-1} |h_k|^2 
\label{eq:plc_acg}
\end{align}
where $|H_k|$ is the CTF sampled with frequency resolution $1/(NT_s)$ and obtained as the $N$-points DFT of the impulse response $h(n
T_s)$. The ACG in dB scale is defined as $ACG|_{dB}=10\log_{10}(ACG)$;

\begin{figure}
\centering
\begin{subfigure}{0.5\textwidth}
  \centering
 \includegraphics[scale=0.4]{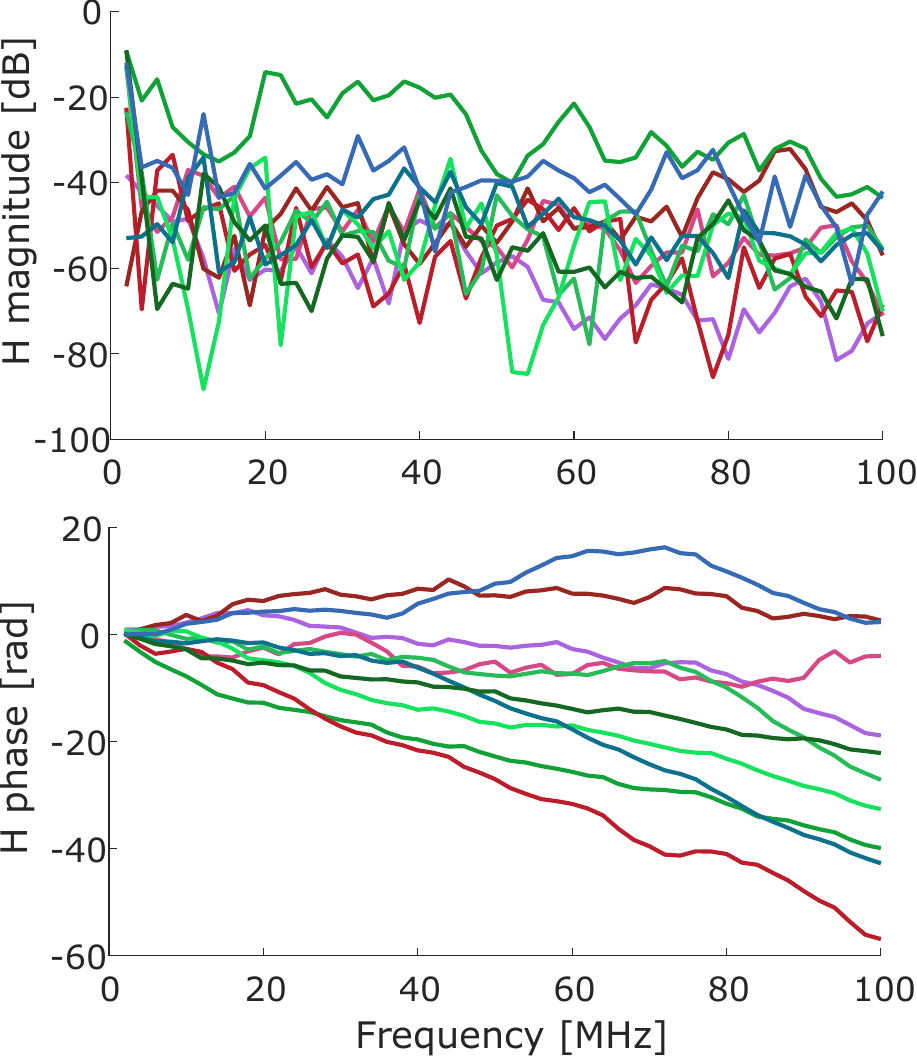}
	\caption{database of measurements.}
	\label{fig:plc_H_measured}
\end{subfigure}%
\begin{subfigure}{0.5\textwidth}
  \centering
 \includegraphics[scale=0.4]{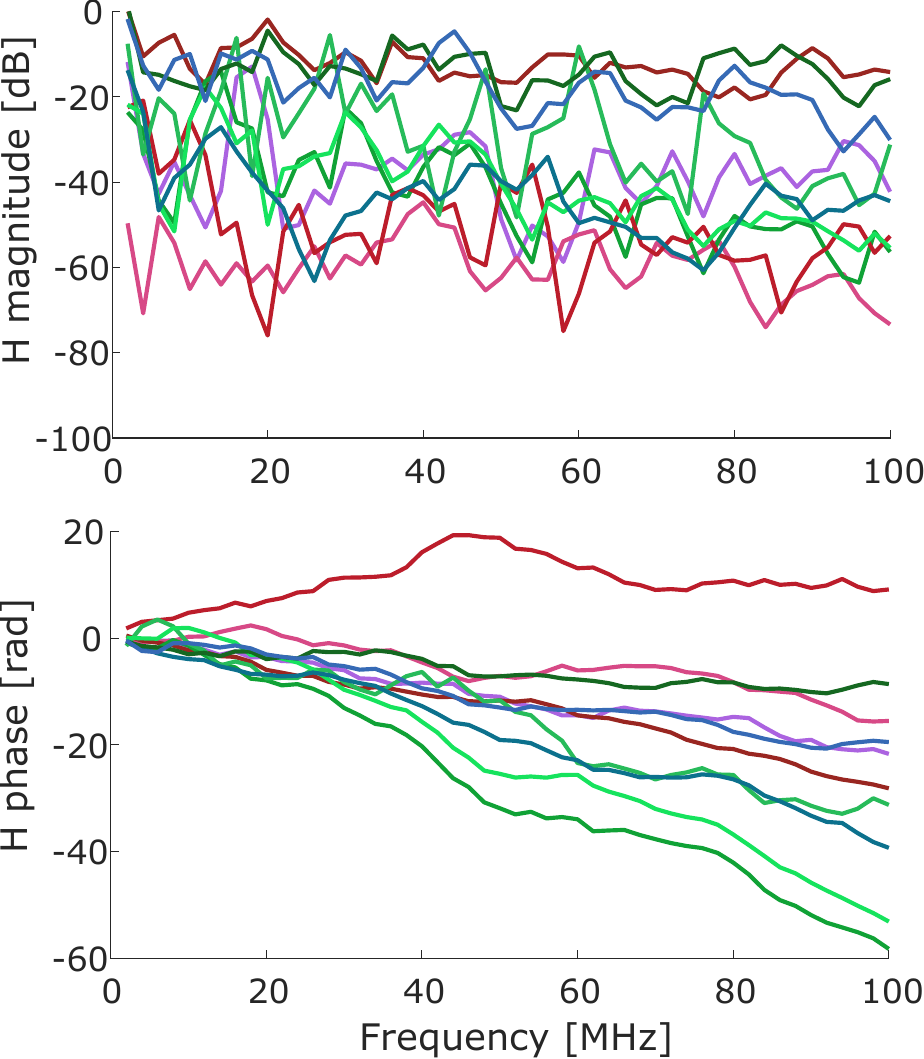}
	\caption{generated with AE \& GAN.}
	\label{fig:plc_H_generated}
\end{subfigure}
\caption{10 randomly picked CTFs. Magnitude and phase.}
\label{fig:plc_CTFs}
\end{figure}

\item The root mean square delay Spread (RMS-DS) is the square root of the second central moment of the power-delay profile, defined as
\begin{align}
\sigma_\tau = \sqrt{\mu_\tau^{''} - \mu_\tau^2}
\label{eq:plc_mu0}
\end{align}
with 
\begin{align}
\mu_\tau &= \frac{ \sum\limits_{n=0}^{N-1} nT_s|h(nT_s)|^2}{\sum\limits_{n=0}^{N-1} |h(nT_s)|^2}  \\
\quad
\mu_\tau^{''}
&= \frac{ \sum\limits_{n=0}^{N-1} (nT_s)^2 |h(nT_s)|^2}{\sum\limits_{n=0}^{N-1} |h(nT_s)|^2},
\label{eq:plc_mu1}
\end{align}
it is used to measure the multi-path spread and characterize the effect of channel dispersion on a transceiver. 

\item The coherence bandwidth (CB) measures the frequency selective behavior of the channel. This can be analyzed in terms of the autocorrelation function of the frequency response, defined as 
\begin{align}
R(\Delta f)=\int_{B}^{} H(f)H^*(f+\Delta f)df
\label{eq:plc_Rdf}
\end{align}
where $^*$ denotes the complex conjugate, $\Delta f$ is the frequency shift and $B$ is the channel band.
\end{itemize}

The analysis is enhanced with the mean and the statistical moments (standard deviation (STD), skewness (SKW) and kurtosis (KURT)) computed on the CTF magnitude.

\subsubsection{Results}
This paragraph presents the list of results achieved with the generative system described in Sec. \ref{sec:gan_ch_synthesis}. 
The hidden representation has been extracted according to the parameters listed in Tab. \ref{tab:autoencoder_nn}. 

Figures \ref{fig:plc_H_measured} and \ref{fig:plc_H_generated} show the measured and generated CTFs, respectively. The generated trends look very consistent with the measured ones. The magnitude spans the large dynamic range of attenuation from $-10$ to $-100$ dB and the characteristic frequency notches are clearly visible. The unwrapped phase exhibits similar trends along frequency for both generated and measured CTFs. 

\begin{figure}[t]
	\centering
	\includegraphics[width = 8cm, height = 5cm]{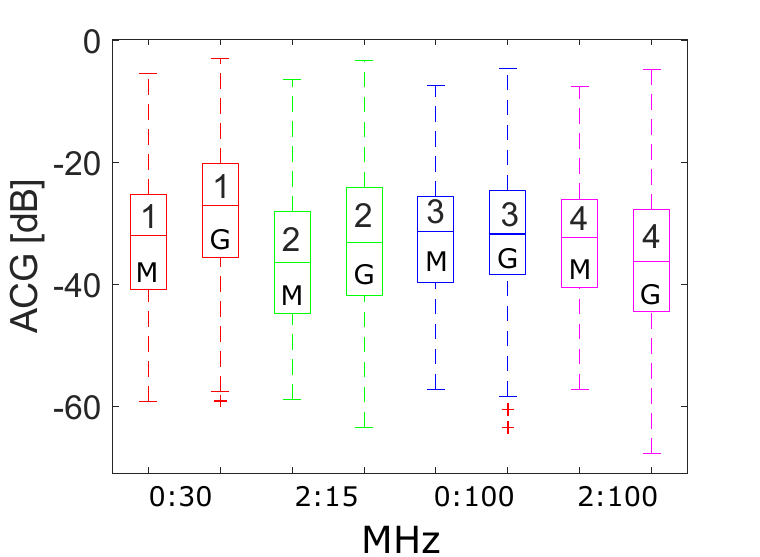}
	\caption{ACG boxplot comparison.}
	\label{fig:plc_acg}
\end{figure}

The boxplot in Fig. \ref{fig:plc_acg} shows the comparison of ACG for four different frequency ranges. Respectively 1, 2, 3, 4 refer to 0:30 MHz, 2:15 MHz, 0:100 MHz and 2:100 MHz. It allows for an agile representation of the minimum, maximum, 25th and 75th percentiles, mean and outliers. The graph shows that the metrics are almost matched for all frequency bands. The highest discrepancies are visible at the lowest frequencies.
Tab. \ref{tab:plc_metrics_comparison} provides the full list of metrics considered in the assessment of the model. The 25th and 75th percentiles and mean are shown for both measured and generated data in the PLC broad band spectrum (up to 100 MHz). The values share similar ranges for both measured and generated data. The computed CB on the data is very high due to the low resolution ($2$ MHz) which is the only drawback of adopting this methodology. Thus, in this case, it should not be considered a useful metric to characterize the CTF. Nevertheless, both generated and measured CB values are similar. 

\renewcommand{\arraystretch}{1.7} 
\begin{table}[t]
	\scriptsize 
	\centering
	\caption{Metrics comparison between measured and generated data. The table shows the metrics in range 0:100 MHz.}
	\begin{tabular}{c|*3c|*3c} 
		\toprule
		Metrics
		          & \multicolumn{3}{c}{Measured}  & \multicolumn{3}{c}{Generated} \\
                        & 25p   & mean & 75p        & 25p   & mean & 75p      \\
		\midrule
		ACG [dB]        & -39,6 &	-31,9&	-25,6   & -38,3&	-31,7&	-24,5 \\ \hline
		RMS-DS [$\mu$s] & 0.126 &	0.136&	0.152   & 0.082&	0.105&	0.129 \\ \hline
        CB [MHz]        & --    &   4.0  & --       &  --  &  4.2    &    --  \\ \hline
        STD             & 0.006 & 0.029  &	0.032   & 0.007 &	0.037 &	0.041 \\ \hline
        SKW             & -1.79 &	0.083&	1.80    & -0.030&	1.988 &	4.45  \\ \hline
        KURT            & 7.04  &	18.68&	27.84   & 7.15  &	20.33 &	31.54 \\ 
		\bottomrule	
	\end{tabular}	
	\label{tab:plc_metrics_comparison}
\end{table}

\subsection{Noise generation}
\label{subsec:plc_noise}

\subsubsection{Traditional noise modeling methods}
The PLC noise signal is generated by different sources and
possesses peculiar properties. Power line networks are
usually made by unshielded cables that tend to couple with
both radiated and conducted noise signals that are traveling
in the surroundings of the network. Moreover, the appliances
connected to the power line generate and inject noise components that
may have a cyclostationary behavior with period related to the
mains frequency of $50$/$60$ Hz.
In literature, the PLC noise is described as a sum of
five components: colored background noise (CBGN), narrow-band noise (NBN), periodic impulsive noise synchronous
to mains (PINS), periodic impulsive noise asynchronous
to mains (PINAS) and aperiodic impulsive noise (APIN) \cite{Zimmermann2002}:
\begin{itemize}

\item The CBGN class \cite{6288568} represents the noise as a process obtainable as follows: $n_{CBGN}(t) = h_{CBGN}* (n_W(t) \cdot \sigma)$;
where $h_{CBGN}$ is a LTI filter, $n_W$(t) is a white Gaussian noise, $\sigma$ is the constant standard deviation. In the literature, it is reported that this type of noise exhibits a stationary behavior \cite{1650328}. Two types of PDFs are associated to the CBGN noise: Normal and Alpha stable. 

\item The NBN class models the noise as a sum of different independent harmonic interferences \cite{5570980,7248409}. 
\begin{align}
n_{NBN}(t) = \sum_{j=1}^{N_{NBN}} A_j(t) \sin(2\pi f_it+\phi_i(t)),  
\label{eq:plc_nbn}
\end{align}
where $N_{NBN}$ is the number of independent interferences, $A_i(t)$, $f_i$ and $\phi_i$ are respectively the time envelope, the central frequency and the phase of the $j^{th}$ interference. 
Typically, it is reported that this type of noise exhibits two different time behaviors: stationary and cyclostationary. 

\item The PINS and the PINAS are described by a summation of different independent Gaussian components with a time varying power \cite{1650328,HanStoicaKaiser2016_1000056745}. Usually, the inter-arrival time between pulses, amplitude and width are the main parameters to characterize this type of noise. Generalized extreme value PDF are associated to PINS noise.

\item The APIN exhibits no deterministic behavior and it is usually described with a Middleton class A model \cite{4091283,6547827}. 
\begin{align}
p_{APIN}(n) &= \sum_{j=0}^{\infty} \frac{e^{-A}A^j}{j!} \frac{1}{\sqrt{2\pi \sigma^2_j}}\exp \left(-\frac{n^2}{2\sigma^2_j}\right),
\label{eq:plc_apin}
\end{align}
where $p_{APIN}(n)$ is the PDF of the noise amplitude. $A$ is the impulsive index, $\sigma_j$ are the standard deviations of each Gaussian PDF.
\end{itemize}

The noise in multi-conductor power lines keeps the aforementioned classes and assumes interesting joint characteristics between different conductor pairs. An example is the high spatial correlation between the cyclostationary noise signals on different phases. This is usually caused by network-based switched power supplies \cite{8245785}. Moreover, measurements demonstrate that multi-conductor noise is not only correlated, but sometimes reveals deterministic behaviors \cite{8360239}. 

We use the GAN-based methodology described in Sec. \ref{sec:medium_channelsynthesis} to model the noise.

\subsubsection{Model validation and results}
The dataset consists of PLC noise measurements in the narrow-band spectrum (from $3$ kHz to $500$ kHz) \cite{8360239}. The dataset contains two traces of multi-conductor
noise sampled at $1$ MSa/s. The multi-conductor noise was
acquired between live and protective-earth and between neutral
and protective earth pairs. The noise traces were cut in time
windows of $1024$ samples, to obtain a new dataset of $12540$
noise frames.

\begin{table}
	\scriptsize 
	\centering
	\caption{Features statistics.}
	\begin{tabular}{c|cc|cc|cc} 
		\toprule
		\textbf{features}  & \multicolumn{2}{c|}{\textbf{mean}} & \multicolumn{2}{c}{\textbf{std}} & \multicolumn{2}{|c}{\textbf{median}} \\
		& generated &	real&	generated&	real&	generated&	real \\
		\midrule
		max [V]& 0.06&	0.08&	0.06&	0.09&	0.04&	0.05 \\
		mean [mV]& -6.2&	0.46&	10.01&	24.58&	-4.39&	0.280\\
		energy [$\mu J$]& 0.11&	0.26&	0.54&	1.10&	0.027&	0.039\\
		std [V] & 0.03&	0.03&	0.03&	0.05&	0.02&	0.02\\
		skewness& -0.18&	-0.01&	0.42&	0.44&	-0.19&	-0.01\\
		kurtosis& 2.29&	2.36&	0.55&	0.65&	2.2&	2.24\\
		entropy& 1.48&	1.52&	0.49&	0.56&	1.45&	1.52\\
		peaks $>$ 0.05 V & 2.56&	5.25&	5.62&	9.61&	0&	0\\
		skew. auto.& 0.1&	0.13&	0.52&	0.5&	0.07&	0.13\\
		kurt. auto.& 3.03&	2.66&	0.65&	0.76&	2.95&	2.62\\		
	\end{tabular}
	\label{tab:plc_features_stat}
\end{table}

\begin{figure}
	\centering
	\includegraphics[scale=0.6]{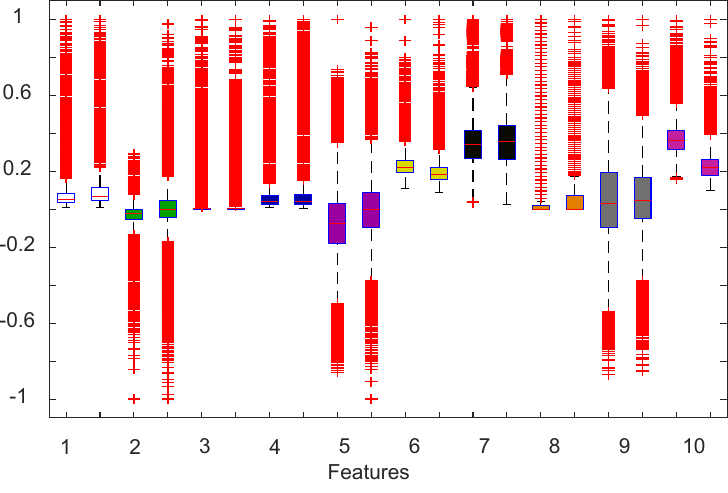}
	\caption{Normalized box-plot of several features for both the generated and real noise traces.}
	\label{fig:plc_Boxplot}
\end{figure}

\begin{figure}
	\centering
	\includegraphics[scale=0.6]{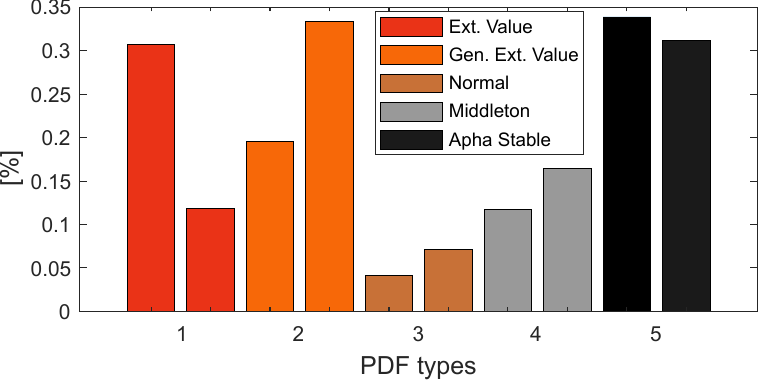}
	\caption{Bar plot of PDF classification for the generated and real noise traces.}
	\label{fig:plc_Barplot}
\end{figure}

Considering the physical nature of the noise measurements,
we rely on statistical physical properties of the noise
to assess the goodness of the generation method. However,
qualitative metrics based on human visual perception are
useful during the training process in order to understand the
status of the network. Furthermore, images provide a quick
visual understanding of the features similarities which in
general is hard to formulate as a mathematical problem. As
an example, the top row of Fig. \ref{fig:plc_Spectrogram} shows two spectrograms,
a real and a generated one, close to each other. The structure
similarity in the spectrograms is depicted also in the noise
traces. The same type of consideration can be made for multiconductor noise generation in Fig. \ref{fig:medium_Block-Spectrogram}. The synthetic block-spectrogram has not only a similar structure compared to the real one, but it also provides a double internal structure which is reflected into a sort of negative correlation between the two noise traces, as in the real scenario.

\begin{figure}[t]
	\centering
	\includegraphics[scale=0.45]{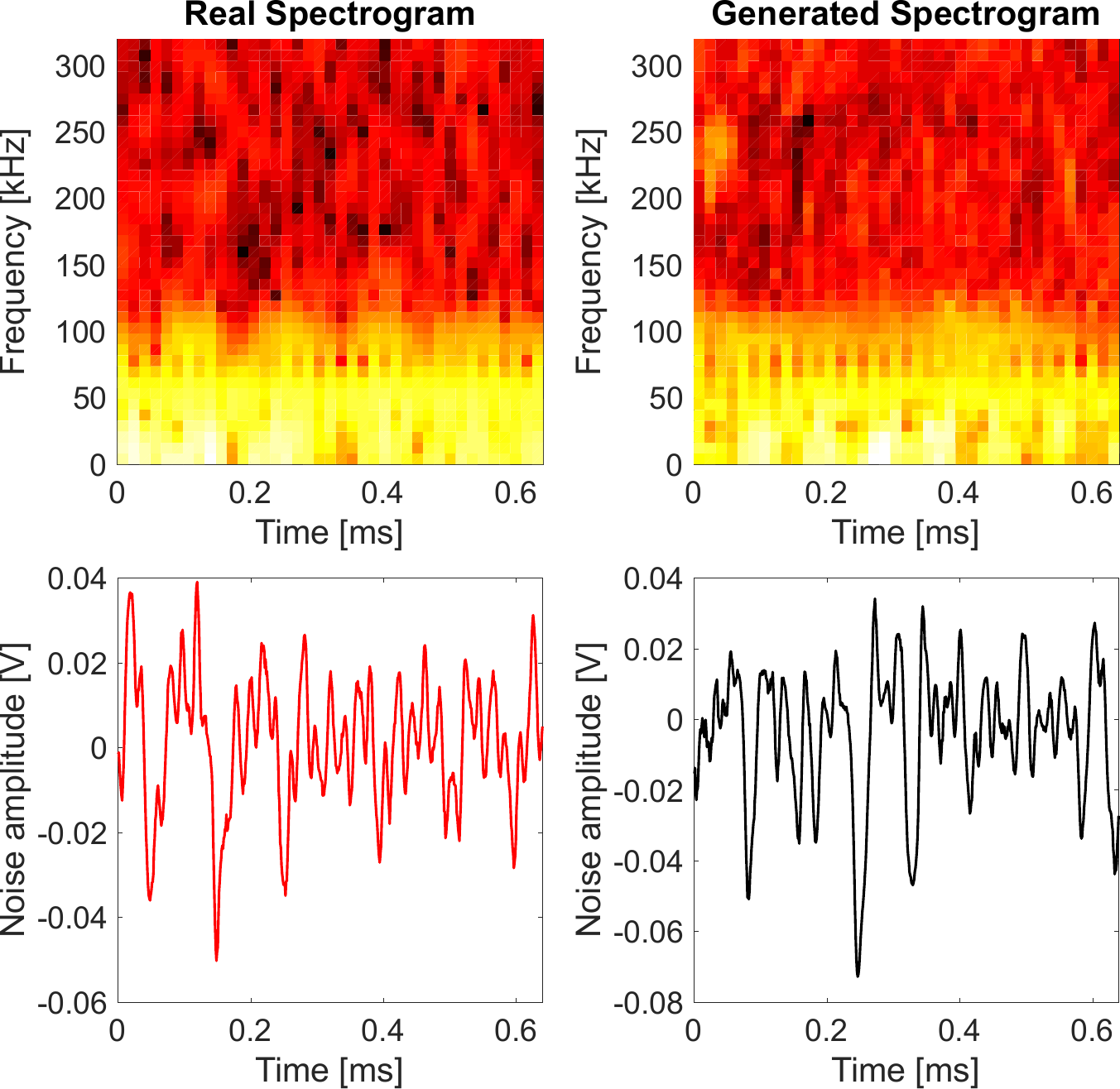}
	\caption{Spectrogram representations of randomly picked real noise measurement and generated noise trace.}
	\label{fig:plc_Spectrogram}
\end{figure}

We consider several quantitative parameters to statistically compare the generated and real noise traces. During such analysis, we applied the method explained in \cite{righiniAutomaticClustering} to both the generated and real noise traces. Practically, the quantitative evaluation procedure is split into three phases: a first step to extract a dataset of physical features calculated for the noise traces, a second part in which the self-organizing map (SOM) clustering algorithm is applied to isolate different noise classes, and a third last phase where a labeling system associates PDFs of a certain type to the identified clusters. The considered PDFs are (1) = Extreme value, (2) = Generalized extreme value, (3) = Normal, (4) = Middleton class A, (5) = Alpha stable. 

A first glance of the normalized features is shown in Fig. \ref{fig:plc_Boxplot}. In details, the shown features are: (1) maximum of the magnitude, (2) mean, (3) energy, (4) standard deviation, (5) skewness, (6) kurtosis, (7) estimated entropy, (8) peaks over 0.05 V, (9) skewness of the autocorrelation and (10) kurtosis of the autocorrelation. The boxplot illustrates a correspondence between the fake and real noise statistical features. A quantitative analysis is also reported in Tab. \ref{tab:plc_features_stat}. In particular, the values corresponding to the generated traces are consistent with the statistical analysis conducted on the real noise measurements.
As an example, features such max, energy, standard deviation, kurtosis, entropy and skewness autocorrelation are matched in mean, standard deviation and median. Other features instead, such as the mean and peaks over 0.05 V, show different statistical moments, sign of possible improvements during the generation training process. Another important physical properties of the noise come from its spectrum analysis, specifically from the analysis of either the Fourier transform or the power spectral density (PSD). Fig.~\ref{fig:plc_FFT} illustrates the average value of the Fourier Transform computed for each fake and real noise trace. The higher discrepancies between the two types of noise are indeed detectable in the frequency domain where the fake noise has not completely learned yet how to replicate the frequency information. 

The last clustering phase deals with a comparison of PDFs. A certain PDF is associated for each noise cluster, and the difference between generated and real data can be formulated as the probability to identify a certain PDF. Fig.~\ref{fig:plc_Barplot} shows that similar percentages are obtained for normal, Middleton class A and Alpha stable PDFs. On the other hand, the percentages of extreme value and generalized extreme values PDFs are slightly different. An immediate consequence of this consideration is the fact that the generation method is not accurate in fully replicating the impulsive noise, leaving space to future improvements, for instance via StyleGANs.

\begin{figure}[t]
	\centering
	\includegraphics[scale=0.5]{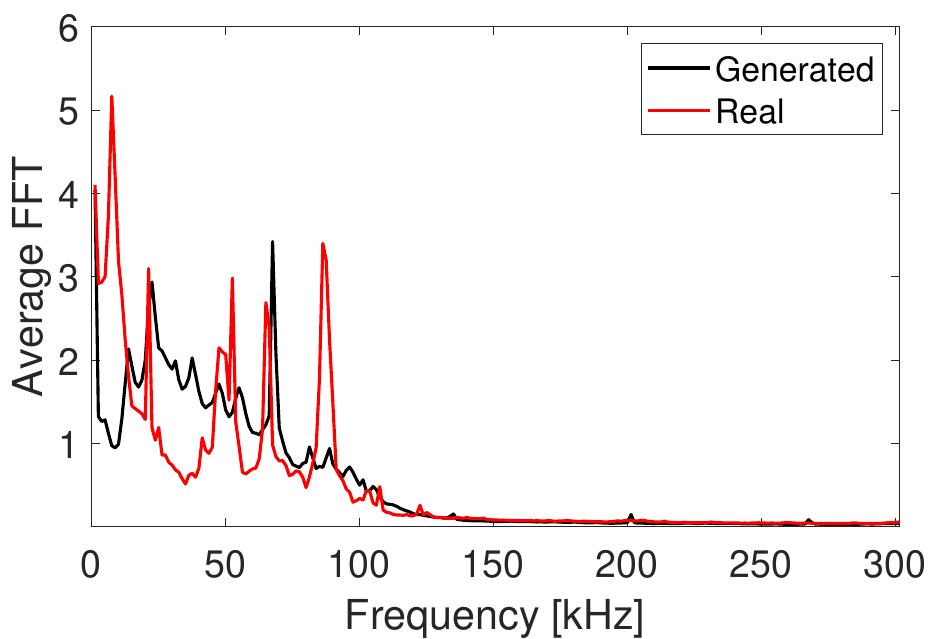}
	\caption{Average Fourier transform magnitude for both the real and generated noise traces.}
	\label{fig:plc_FFT}
\end{figure}

\section{Learning impedance entanglements}
\sectionmark{PLC impedance entanglement}
\label{sec:plc_entanglement}
In wireline communication networks, a line impedance entanglement (IE) exists since changes of the line impedance at one network port cause a change of the line impedance at the other port. This physical phenomenon can be constructively exploited to realize a form of digital modulation that is referred to as impedance modulation (IM). IM is an alternative method to more conventional voltage modulation (VM). The IE can be studied in ML terms, enabling the implementation of a ML based receiver. Numerical results are obtained in a dataset of measured power line communication channels, which is among the most challenging environments for such a modulation approach. The resulting system can have practical implementation, for instance in a smart building automation network where monitoring-control of sensors and devices enables the efficient energy management.

Comparisons with the optimal MaxL receiver that perfectly knows the IE transfer function are made. It is found that the ML based receiver performs close to the optimal genie receiver.

\subsection{The impedance entanglement}
In recent work \cite{TonelloIsplc2020}, a fundamental question has been posed: \emph{is it possible to encode information and physically transmit it over a wireline medium in a different way other than a VM?} The answer has been found by the analysis of the PLC medium, where it has been shown that a change of the impedance at one channel port induces a change of the impedance on the other port \cite{tcas}, \cite{tee}. This can be proved with transmission line theory arguments and it can be referred to as IE. It is important to observe that the intensity of the IE depends on the physical medium characteristics (electrical characteristics of the wires, network topology, and loads). In many applications and scenarios of practical relevance the exploitation of the IE enables another mode of transmission. In addition, it can be jointly used with VM, which generates two parallel (although coupled) communication channels.

We assume a wireline system, e.g., a PLC system, where physical wires offer a medium for data transmission. Impedance modulation is exploited to encode and transmit information between two network ports \cite{TonelloIsplc2020}. The transmitter is solely composed by a set of impedances, which are selected through mapping of a sequence of data symbols. Each impedance encodes a symbol to be transmitted. The receiver goal is to decode the impedance coded at the transmitter side exploiting the IE. This is accomplished by using a shunt based receiver \cite{TonelloIsplc2020}. It consists of a signal generator and a measurement of the currents and voltages at the receiver port.
Assuming to represent voltages and currents with phasors, the wired channel acts as a complex and non-linear function of the transmitter impedance. In fact, from the microwave circuit model where the ABCD matrix \cite{pozar}
is deployed, the relationships among the injected currents and the applied voltages are described by
\begin{equation}
	\begin{cases}
		V_{r}(f) = A(f)V_{tx}(f) + B(f)I_{tx}(f) \\
		I_{r}(f) = C(f)V_{tx}(f) + D(f)I_{tx}(f)
		\end{cases}.
	\label{eq:plc_VIinVIout}
\end{equation}
$V_{tx}$ and $I_{tx}$ are the voltage and the current at the transmitter side, respectively. It is easy to understand that the ratio between these two quantities gives the transmitter impedance $Z_{tx} = \frac{V_{tx}}{I_{tx}}$. The system model and related quantities is reported in Fig. \ref{fig:plc_abcd}.
\begin{figure}
	\includegraphics[width=\columnwidth]{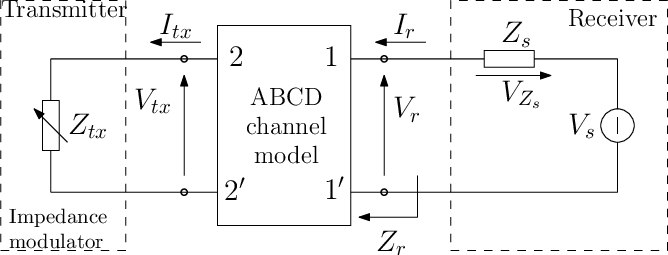}
	\caption{Impedance modulation scheme with shunt based receiver.}
	\label{fig:plc_abcd}
\end{figure}
The impedance seen at the receiver port $Z_{r}$ (i.e., the line impedance at the port $1-1'$) depends on the physical wireline network and on the impedance of the transmitter $Z_{tx}$. The relation can be obtained from (\ref{eq:plc_VIinVIout}), and it reads
\begin{equation}
	Z_{r} = \frac{AV_{tx}+BI_{tx}}{CV_{tx}+DI_{tx}} =  \frac{AZ_{tx}+B}{CZ_{tx}+D}.
	\label{eq:plc_Zr}
\end{equation}
As it becomes clear from (\ref{eq:plc_Zr}), the impedance $Z_{r}$ is a function $\Omega({\bold{T},Z_{tx},f})$ of the ABCD channel matrix $\bold{T}$, the transmitter impedance $Z_{tx}$ and the frequency $f$. The receiver is designed to measure the impedance $Z_{r}$ with the deployment of a known shunt impedance $Z_{s}$ and a voltage generator $V_{s}$. In fact, $Z_{r}$ can be retrieved from the measurement of the shunt voltage $V_{Z_s}$ since
\begin{equation}
	Z_{r} = \left(\frac{V_s}{V_{Z_s}}-1\right)Z_s.
	\label{eq:plc_ZrV}
\end{equation}
The performance of the modulation scheme depends on the IE and it is additionally influenced by the presence of noise which adds onto the measured shunt voltage drop.  In accordance with (\ref{eq:plc_Zr}), the function $\Omega(\bold{T},Z_{tx})$ exhibits a non-linear behaviour strongly influenced by the parameter $C$. The parameter $C$ of the ABCD matrix $\bold{T}$ represents the transadmittance of the modeled channel. The higher the transadmittance is, the higher the fraction of the injected current $I_r$ lost in the channel is. In turn, this is reflected into smaller sensitivity to the variations of the measured voltage $V_{Z_s}$ induced by the variations of $Z_{tx}$, i.e., lower IE.

The effect of the non-linearity of $\Omega$ is exemplified in Fig. \ref{fig:plc_nonlineare} which shows the distortion of the transmitted constellation (comprising four different impedance values) introduced by the channel. Higher values of $C$ reduce the distance among the constellation points challenging further the symbol detection.
  \begin{figure}
  	\includegraphics[width=\columnwidth]{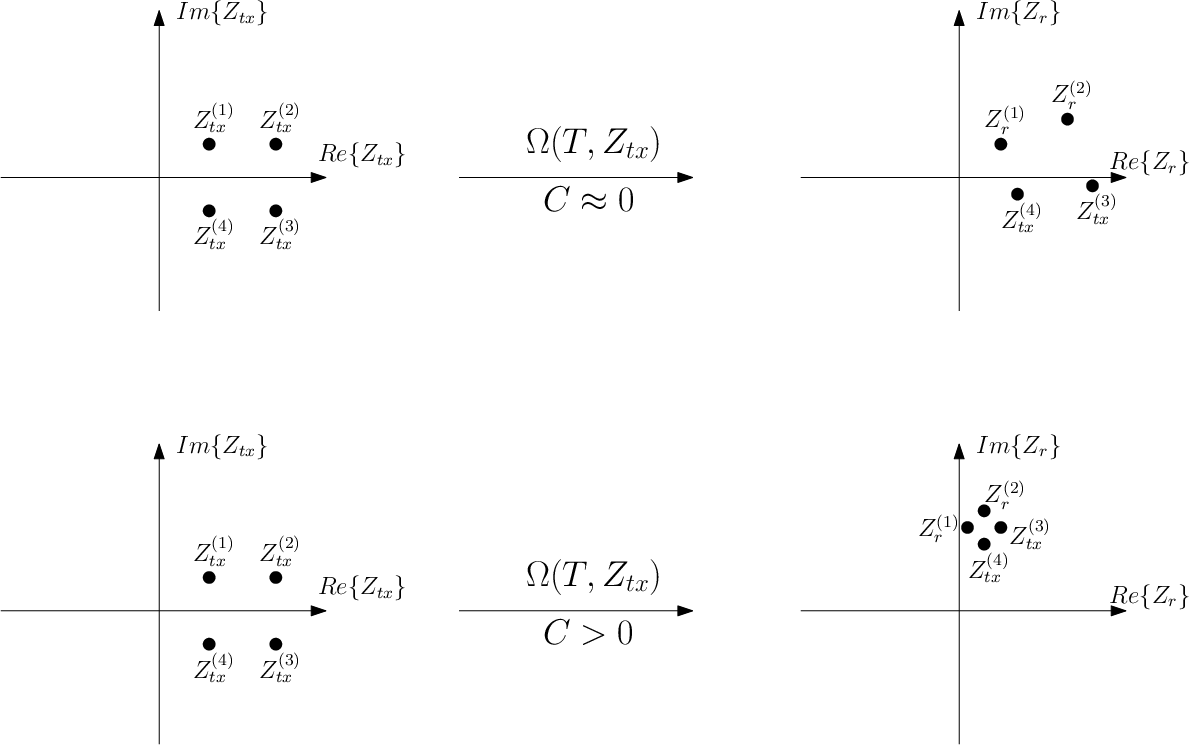}
  	\caption{Impedance entanglement dependence of channel transadmittance $C$.}
  	\label{fig:plc_nonlineare}
  \end{figure}

In this section, a different and simpler approach from the one and more general followed in \cite{TonelloIsplc2020} is presented.
It is assumed to start by having two possible resistive states at the transmitter side, i.e., a short circuit resistance $Z_{tx}^{0}=0 \Omega$ and an open circuit resistance $Z_{tx}^{\infty}$ (with a resistance set to 1 G$\Omega$). These two extreme impedance values induce at the receiver side two distinct impedance profiles. For instance, considering (\ref{eq:plc_Zr}), and substituting the aforementioned transmitter impedance states, two different values of impedance $Z_r$ are obtainable, i.e. $Z_r^{0} = \frac{B}{D}$ and $Z_r^{\infty} = \frac{A}{C}$.
Qualitatively, it is easy to understand that the greater the difference between the two impedance profiles in a certain frequency band, the higher the probability to correctly discriminate the correct value at the receiver side is.

Since what is directly measurable is the shunt voltage drop, it is useful to look at it. In fact, the impedance state at the transmitter induces a certain shunt voltage response. In particular, such a response is informative if evaluated in a large spectrum, e.g., in the bandwidth $BW={1-100} MHz$ spectrum. Such a voltage response can be used to define a measure of the IE. The discriminative measure that we use herein is based on the average distance of the two voltage signals entangled at the receiver. Taking into account (\ref{eq:plc_ZrV}), and assuming a constant value for both the shunt impedance $Z_s$ and the shunt voltage generator $V_s$, the entanglement measure $E_V$ is defined as
\begin{equation}
	E_V = \frac{|Z_s|}{|V_s|} \int_{BW} |V_d(f)|\,df,
	\label{eq:plc_wV}
\end{equation}
where $V_d(f)$ is the voltage defined as the difference $V_s^{0}(f)-V_s^{\infty}(f)$, in which $V_s^{0}(f)$ and $V_s^{\infty}(f)$ are the two measured voltages referred to the two impedances $Z_r^{0}(f)$ and $Z_r^{\infty}(f)$, respectively.
High values of $E_V$ correspond to have noticeable differences between the two measured voltages, thus to more detectable impedances and higher entanglement.

\subsubsection{Representative channels: the in-building PLC case}
A set of measured broad-band PLC channels has been exploited for practically assessing the theoretical discussion of previous sections. In particular, about 1200 PLC channel measurements are taken into account, and for each of them the entanglement measure (\ref{eq:plc_wV}) is calculated. A subset of three representative channels are chosen considering the 0.1, 0.5 and 0.9 quantiles of (\ref{eq:plc_wV}). Fig. \ref{fig:plc_wV} reports the voltage difference $|V_s^{0}(f)-V_s^{\infty}(f)|$ expressed in dB$\mu$V, evaluated on a 50 $\Omega$ shunt. The test voltage $V_s$ has been set to an amplitude of 0.0224 V at each frequency, which corresponds to a dissipated power of -50 dBm/Hz.
  \begin{figure}
  \centering
	\includegraphics[scale=0.75]{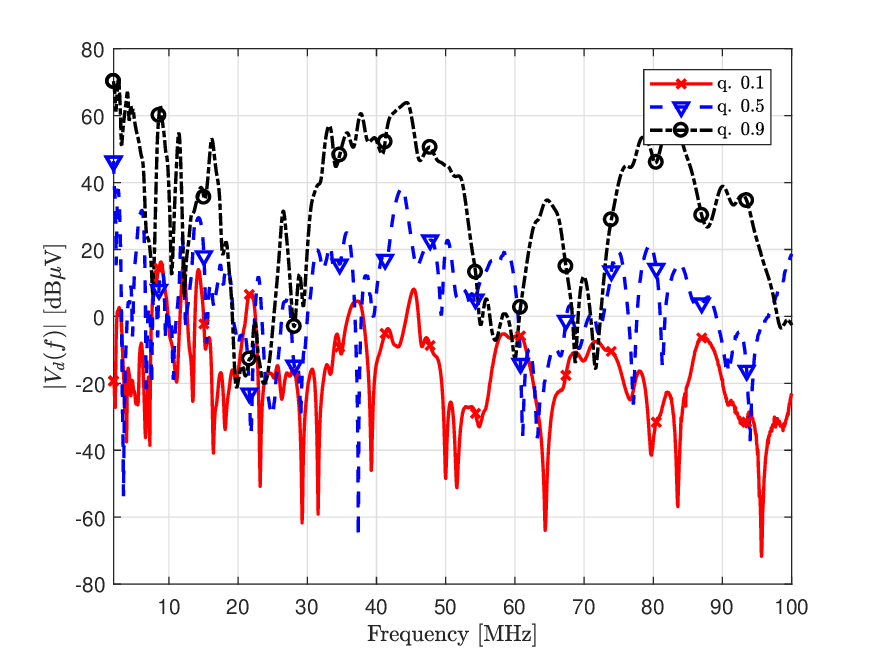}
	\caption{Voltage difference magnitude measured with the shunt-based receiver when the transmitter impedance swaps from $Z_{tx}^{0}$ to $Z_{tx}^{\infty}$. Different colours refer to the different quantile values.}
	\label{fig:plc_wV}
\end{figure}

We now discuss the MaxL based receiver and then we present a ML based approach to learn the IE and realize a practical receiver.

\subsection{Maximum-likelihood based receiver}
To measure the impedance $Z_{r}$ at the receiver, multiple circuit architectures can be used \cite{Passerini2017, Hallak2018}. We consider a shunt based impedance meter as shown in Fig. \ref{fig:plc_abcd} that is capable of measuring $V_s$. The measurement is done on a wide spectrum. To derive an optimal receiver metric, and not to neglect the presence of noise (denoted with $V_n$), we can rearrange the above relations to obtain
\begin{equation}
V_{Z_s,n}= V_s Z_s\frac{A + CZ_{tx}}{(A+CZ_{tx})Z_s +B+DZ_{tx}} + V_n.
\label{eq:plc_Vzs}
\end{equation}
If we assume to map data symbols into a discrete set of values of $Z_{tx}$, and to consider the relation (\ref{eq:plc_Vzs}), then a MaxL estimate is obtained as follows

\begin{equation}
Z_{tx}=\arg\min_{Z}\left\{\left|V_{Z_s} - V_sZ_s\frac{A+CZ}{(A+CZ)Z_s + B + DZ}\right|\right\}.
\label{eq:plc_zgarg}
\end{equation}

It should be noted that the MaxL receiver requires the knowledge of the entanglement transfer function, therefore of the ABCD parameters. In the numerical results, we assume a genie receiver that perfectly knows them. This is used as a baseline situation to make a comparison with the practical ML receiver proposed. It should also be noted that IM does not require any voltage source at the transmitter side, therefore it can be considered a sort of passive modulator since it requires only switching among a finite set of impedances at the transmitter node. It can be realized with a bank of impedances and controlled radio frequency switches. 

\subsection{Machine learning-based receiver}
\subsubsection{Training set and parameters}
The shunt voltage including noise is described as
\begin{equation}
V_{Z_s,n} = V_{Z_s}+V_n,
\end{equation}
where $V_n \sim \mathcal{CN}(0, \sigma_n^2)$. The measured voltage is a non-linear function of the impedance state $Z_{tx}$ at the transmitter side as shown in (\ref{eq:plc_Vzs}). Hence, if we assume that $Z_{tx}$ possesses only two states, it is in principle possible to detect these states by training a binary classifier with the pair of samples $(V_{Z_s,n},Z_{tx})$.

Each channel differently distorts the transmitted impedance as in (\ref{eq:plc_Zr}). Therefore, as a proof of concept, we consider two scenarios to train the symbol detector/classifier: a first scenario where only one channel realization correspondent to the $0.9$ quantile is used to obtain the shunt voltage $V_{Z_s,n}$ over a wide spectrum; a second scenario where $15$ channel realizations in the neighbourhood of the $0.9$ quantile are used to obtain $V_{Z_s,n}$. With the former, we want to prove that a machine learning-based detector does exist and achieves performance comparable to the optimal MaxL receiver. The latter, instead, extends the classification under a variety of channels with a common discriminative feature, that is, high value of $E_V$. For both scenarios, the training phase was conducted using $50\%$ of the data, whereas the testing phase was performed using the other half. For each channel realization, we consider $1000$ noisy voltages $V_{Z_s,n}$ obtained by transmitting the short and open resistor states, denoted with the class $0$ and $1$, respectively. The voltage $V_{Z_s,n}$ is split in real and imaginary parts, and they both possess a total of $1569$ frequency points from $1$ to $100$ MHz which are concatenated in a unique input vector. As common practice, the input realization has been normalized to values inside the interval $[-1,1]$, to avoid ill-conditioned networks, as follows
\begin{equation}
\label{eq:plc_normalization}
\hat{x}_i = 2\frac{x_i-\min(x_i)}{\max(x_i)-\min(x_i)}-1.
\end{equation}
We have implemented a simple shallow NN with one hidden layer of only $100$ neurons and sigmoidal activation function. The training set has been further divided into two subsets, a pure training set used to update the parameters of the artificial neural network and a validation set used to evaluate the performance of the network at each iteration. To avoid overfitting, we early stopped the training process after $10$ consecutive iterations where the error in the validation set increases. We have repeated the same experiment $10$ times in order to reduce the bias coming from the parameters initialization.

\subsubsection{Individual channel realization training}
As a first analysis, we consider a given channel realization, referred to as individual channel (IC), and train a classifier via cross-entropy to estimate the impedance state $Z_{tx}$ from the measured noisy sample $V_{Z_s,n}$, with different values of the noise variance $\sigma_n^2$. Once the training process is concluded, the testing phase kicks in and we evaluate the network performance by measuring the BER varying the SNR ratio. Given the network architecture described before, we follow two distinct training approaches. In the first approach, referred to as multi-network, we train several different NNs for every possible SNR value of interest and for each of them we compute the BER using the test set. In the second naive approach instead, referred to as single-network, we train only one classifier assuming an SNR value of $20$ dB and generalize it to unseen SNRs. Surprisingly, the BER obtained with the network trained with only one value of SNR is always lower than the BER experienced from individual networks trained for each specific value of SNR. This counter-intuitive result is a direct consequence of the overfitting issue. Indeed, when training networks with low SNR values, the network attempts to fit the noisy measurements rather than understanding the noise distribution. Fig. \ref{fig:plc_ber1} compares the single and multi-network machine learning-based approaches with the optimal MaxL receiver. In particular, it illustrates the accuracy of the impedance state prediction in terms of BER for different values of the SNR.
\begin{figure}
\centering
  	\includegraphics[scale=0.35]{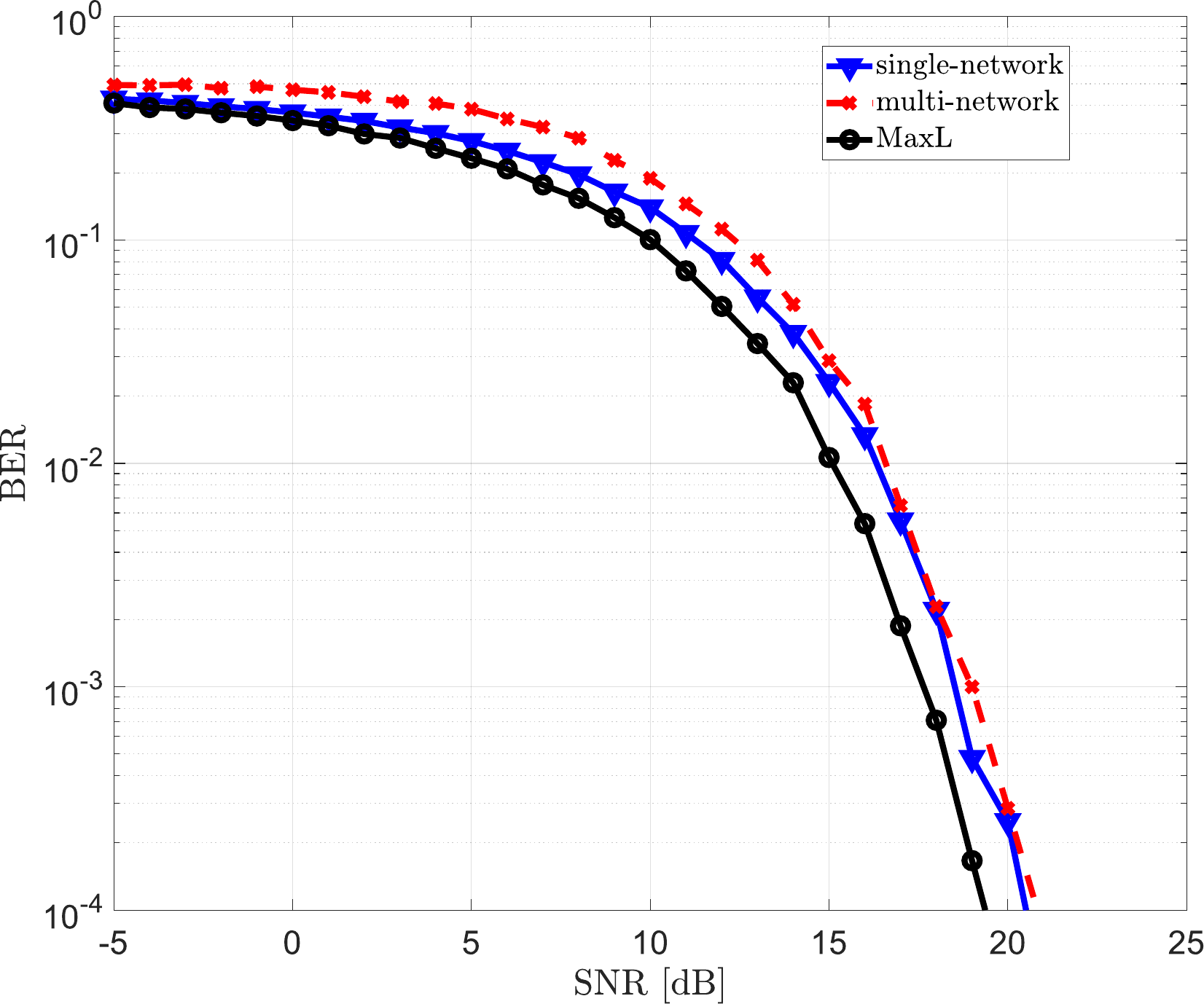}
  	\caption{BER of single and multi-network approaches compared to the optimal MaxL receiver.}
  	\label{fig:plc_ber1}
\end{figure}

It is interesting to notice also the following. The ML approach provides worse performance of less than $1$ dB compared to the genie MaxL receiver with the AWGN model. The approach can be generalized to non-Gaussian noise for which no optimal receivers are known. In addition, the single-network approach only requires few data and few iterations to achieve competitive performance. Conversely, the genie MaxL receiver requires the knowledge of the IE transfer function which implies that estimation algorithms have to be realized.

\begin{figure}
\centering
  	\includegraphics[scale=0.35]{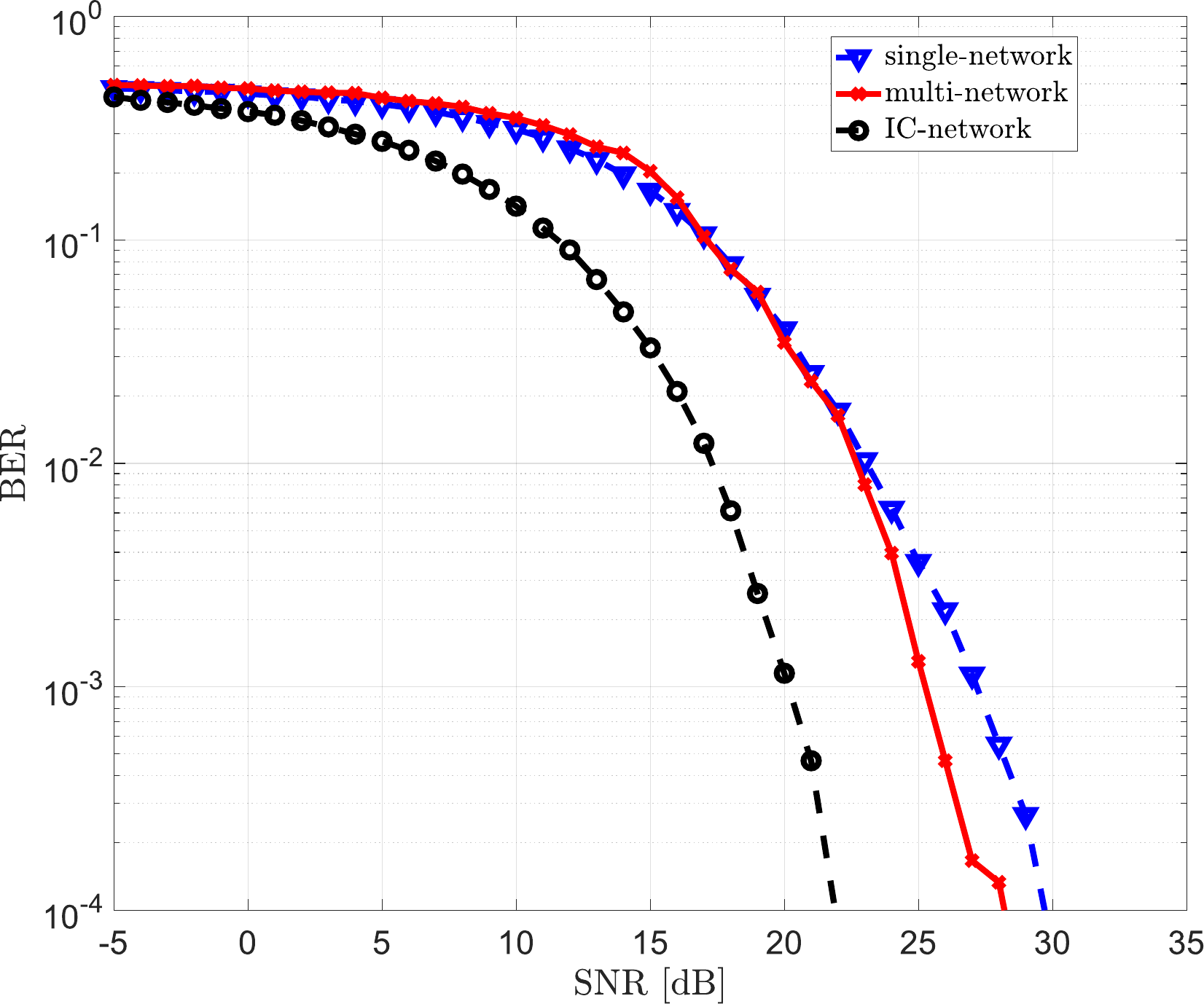}
  	\caption{BER of single and multi-network approaches obtained with channel diversity compared to the average BER obtained with single-network training on individual channels.}
  	\label{fig:plc_ber2}
\end{figure}

\subsubsection{Multiple channel realizations training}
Perhaps more fascinating is the idea of a general classifier able to account for the channel diversity, thus, trained on multiple channel realizations. It is clear, however, that the multiple channel realizations shall share a common discriminative feature, for example the discrimination quantity $E_V$ defined in (\ref{eq:plc_wV}). To prove that it is indeed possible to build a unique NN that detects the input resistive state in the presence of multiple channel realizations, we consider a set of $15$ channel realizations as described above. We repeat the training approach used for the individual training, in particular, we train a single-network classifier at $20$ dB and multi-networks classifiers each for every SNR value. Similarly to the single channel realization case, the best performance in terms of BER is achieved for the single-network approach where the network learns to generalize for unseen low values of SNR (lower than the $20$ dB of training). Conversely to the individual channel case, the performance of the multi-network training is better compared to the single-network one for higher SNR values.
Fig.\ref{fig:plc_ber2} reports the average BER obtained with the single and multi-network approaches in the presence of multiple channel realizations (unique network for all channel realizations at a given SNR, and unique network for all channel realizations and SNRs). The average BER obtained over the $15$ channel realizations when training an individual network for each given channel realization and SNR (IC-network) is also shown. The performance loss of using a unique neural network for all channel realizations w.r.t. to an individual neural network for each channel realization is up to $7$ dB.

\subsection{Summary}
A ML approach can be followed where joint IE learning and impedance detection are obtained by training a NN classifier. Numerical results have been obtained using a set of measured power line channels in the 1-100 MHz spectrum.
The results show the great potential that ML and DL algorithms have when applied to detection tasks. Various training approaches have been followed and they enable performance close to the optimal genie MaxL receiver. They stimulate further studies to include higher order impedance modulation and different neural architectures for improved performance, as well as the possibility to improve the bandwidth efficiency by considering narrower frequency intervals where the IM is implemented. In addition, an intriguing idea consists of an end-to-end learning scheme where the transmitter maps data symbols in optimal impedance constellation points so that the receiver exhibits the best detection performance and automatically learns the decoding strategy, for instance using CORTICAL (see Sec. \ref{sec:cortical_theory}).
  
\section{Capacity learning for additive channels under Nagakami-$m$ noise}
\sectionmark{PLC capacity learning}
\label{sec:plc_nakagami}
The development of PLC systems and algorithms is significantly challenged by the presence of unconventional noise. The analytic description of the PLC noise has always represented a formidable task and less or nothing is known about optimal channel coding/decoding schemes for systems affected by such type of noise. 
In this section, we present a statistical learning framework to estimate the capacity of additive noise channels, for which no closed form or numerical expressions are available. In particular, we study the capacity of a PLC medium under Nakagami-$m$ noise and determine the optimal symbol distribution that approaches it. We provide insights on how to extend the framework to any real PLC system for which a noise measurement campaign has been conducted. Numerical results demonstrate the potentiality of the proposed methods.

\subsection{Capacity learning under Nakagami-$m$ noise}
\label{subsec:plc_capacity_nakagami}
The amplitude of the PLC background noise in the broadband frequency range $1-30$ MHz is statistically modeled by the Nakagami-$m$ distribution with $m<1$ \cite{Meng2005}. Hence, in a quadrature modulation scheme, noise $\mathbf{n}$ can be represented as
\begin{equation}
\mathbf{n} = w \exp(j\theta) = n_r + jn_i,
\end{equation}
where $j=\sqrt{-1}$ and the amplitude $w$ is distributed as
\begin{equation}
p_W(w) = 2\frac{(m/\sigma_n^2)^m}{\Gamma(m)}w^{2m-1}\exp\biggl(-\frac{m w^2}{\sigma_n^2}\biggr).
\end{equation}
However, the in-phase and quadrature components of the noise have not equal variances as show in \cite{Mallik2010}. Therein, it has also been observed that in the interval $1/2\leq m <1$, the characteristic function of the real part $n_r$ of the noise resembles the one of a zero-mean Gaussian distribution of variance $\sigma_n^2(1+b)/2$. Similarly, the characteristic function of the imaginary part $n_i$ approximates the one of a zero-mean Gaussian distribution of variance $\sigma_n^2(1-b)/2$, with real and imaginary part being independent and
\begin{equation}
b = \sqrt{\frac{1}{m} -1}.
\end{equation}
Notice that the average SNR per symbol of the constellation modeled by the generator of the cooperative capacity framework is equal to $P/\sigma_n^2$. Moreover, if $b=0$, the additive noise is white Gaussian.

The design of optimal receivers and the PLC system error analysis affected by this type of noise have been discussed with binary phase-shift keying (BPSK) signaling in \cite{Dash2016} and \cite{Mathur2014}. Recently, an optimal Quadrature Phase Shift Keying (QPSK) constellation for this noise model has been presented in \cite{Reddy2020}, albeit a capacity characterization was not  treated.

We exploit the aforementioned DL approach to present results concerning the optimal quadrature amplitude modulation (QAM) constellation design under an average power constraint. We also report MI boundaries. Optimal neural network-based receivers are left for future studies, although a trivial extension using the autoencoders for communications concept is expected to follow some of the ideas in \cite{Oshea2017,Letizia2021}. 

\subsection{Results}
\label{subsec:plc_results}
In this section, we present the results for the two proposed scenarios: additive Nakagami-$m$ noise, which typically describes the background noise in the broadband spectrum, and additive noise traces obtained via a measurement campaign in the narrowband spectrum. For the former case, we exploit the mutual information estimator (DIME) and the cooperative capacity learning framework (CORTICAL), presented in Sec. \ref{sec:mi_f-DIME} and \ref{sec:cortical_theory}, to estimate the channel capacity and to build the optimal channel input distribution, under an average power constraint. In the latter scenario, we estimate the MI between the Gaussian distributed channel input and the channel output, affected by additive measured narrowband noise. For both cases, we show that the achieved information rate is higher than the AWGN channel capacity.

\begin{figure}
\begin{subfigure}{0.5\textwidth}
	\centering
	\includegraphics[scale=0.25]{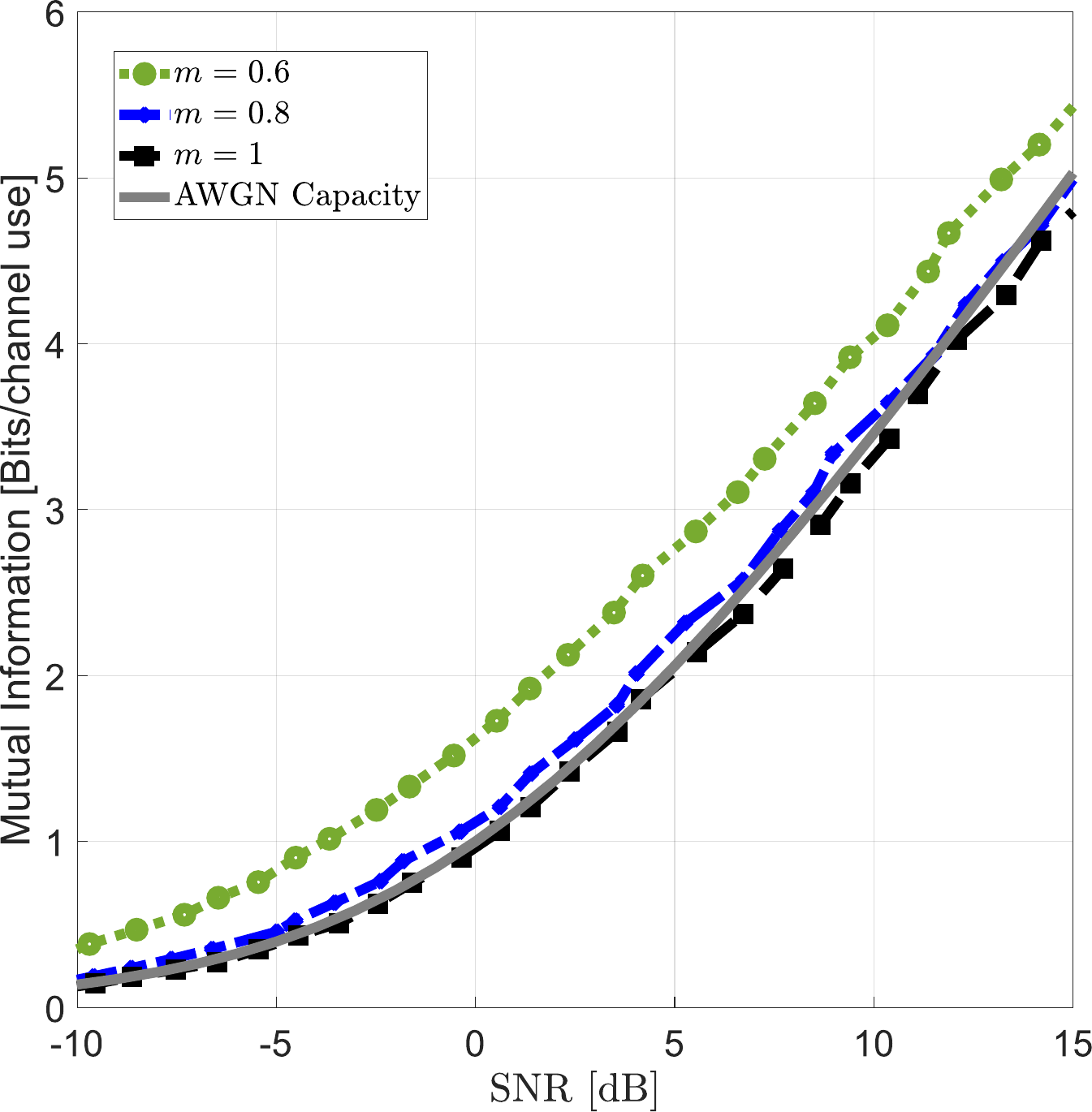}
	\caption{Continuous channel input distribution.}
	\label{fig:plc_capacity_nakagami}
\end{subfigure}
\begin{subfigure}{0.5\textwidth}
	\centering
	\includegraphics[scale=0.25]{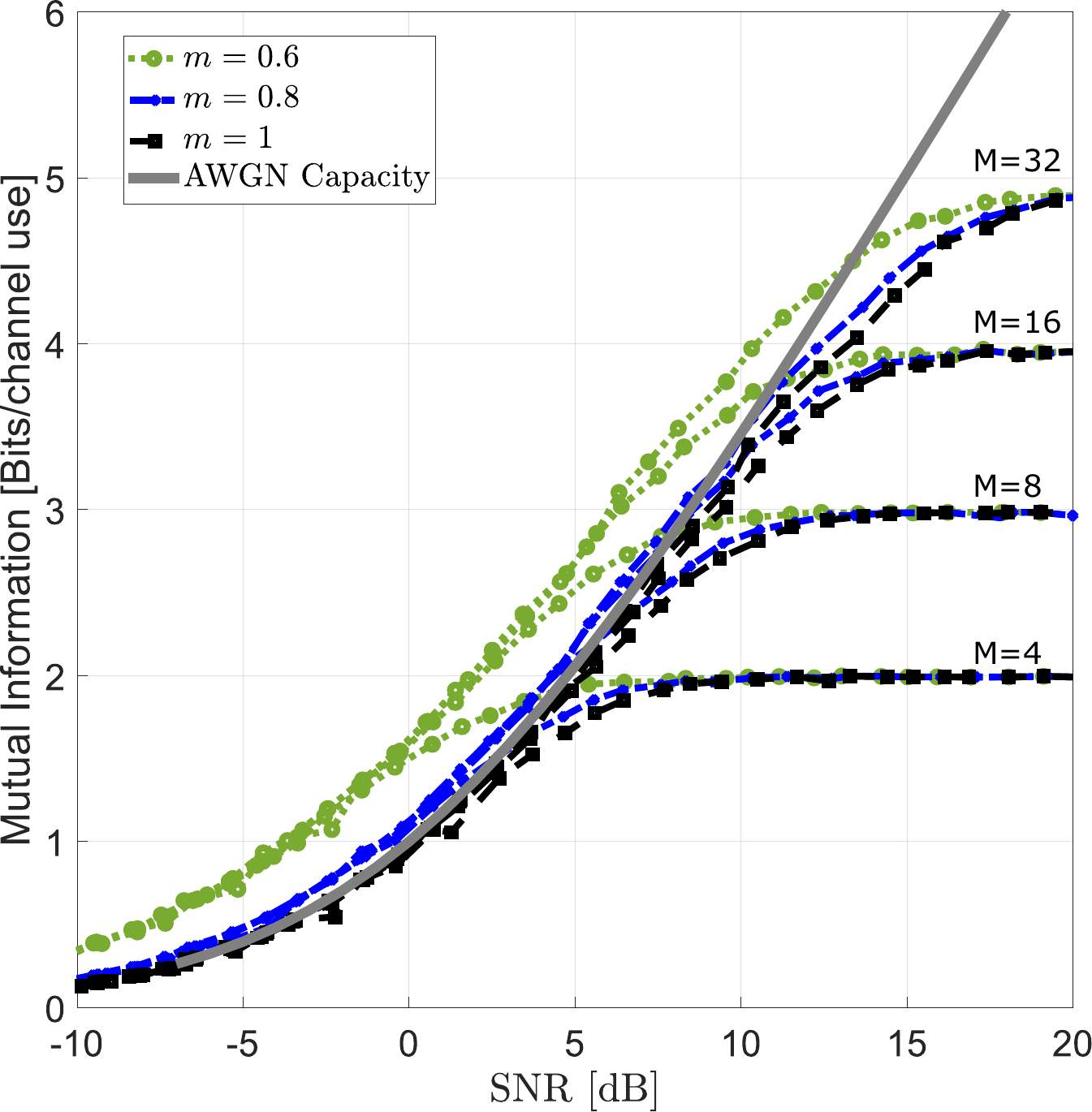}
	\caption{Discrete channel input distribution.}
	\label{fig:plc_capacity_nakagami_discrete}
\end{subfigure}
\caption{Maximal mutual information estimation between channel input and output using DIME in a channel corrupted by additive Nakagami-$m$ noise under an average power constraint.}
\end{figure} 

\subsubsection{Optimal constellation under Nakagami-$m$ noise}
\label{subsec:plc_optimal_constellation}
We parametrize both the generator $G$ and the discriminator $D$ of the CORTICAL framework with NNs and alternately update their parameters $\theta_G$ and $\theta_D$. Given the parametric limit, we study the achievable rate. Details on the networks architecture and on the training parameters are reported in Table~\ref{tab:plc_parameters}.

We consider an analog and digital transmission scheme. The former is realized by taking as input of the generator $G$ a vector $\mathbf{z}$ sampled from a $30$-dimensional Gaussian distribution and by mapping it to a complex symbol $\mathbf{x}$ with continuous distribution $p_{X}(\mathbf{x})$. The latter, instead, considers an input vector $\mathbf{z}$ sampled from a multivariate Bernoulli $k$-dimensional distribution with probability $p=0.5$, where $k=\log_2(M)$ and $M$ is the dimension of the alphabet. For the discrete case, the coding rate is defined as $R = 2k / d$, where $d$ is the dimension of the channel input vector $\mathbf{x}$. For the purposes of this section, we consider only the case $d=2$, corresponding to a geometric constellation shaping. However, an extension to probabilistic shaping is straightforward.

Fig.~\ref{fig:plc_capacity_nakagami} illustrates the estimated maximal MI over a channel affected by the Nakagami-$m$ noise under an average power constraint $P=1$, for three values of the parameter $m$, $m=0.6$, $m=0.8$ and $m=1$. It is rather interesting to notice that the estimated MI, for the case of different noise power on the two components but mostly for $m=0.6$, is higher than the AWGN capacity. Indeed, this is consistent with the fact that the AWGN capacity is the lowest among all additive noise channels. Moreover, the case $m=1$ that corresponds to the classic AWGN channel is accurately approximated by the proposed capacity estimator in the SNR range $-10$ to $15$ dB.

Similarly, when we limit the input distribution to a discrete one, the generator automatically learns how to mitigate the impact of the noise on the MI performance by optimally placing the $M$ messages into a bi-dimensional constellation. In particular, we compare the maximal MI when transmitting $4,8,16,32$ messages for the three values of the parameter $m$. Fig.~\ref{fig:plc_capacity_nakagami_discrete} shows the achieved information rate in all the studied cases and again highlights the possibility to achieve higher rates than the Shannon's AWGN limit for low SNR values. For high SNR values, the MI saturates to the coding rate $R$. An example of optimal constellation for a SNR of $7$ dB, $M=16$ and $m=1$ is depicted in Fig.~\ref{fig:plc_features}a where it is curious to comment on the geometric structure of the constellation. Indeed, the constellation is more densely packed along the $Y_2$ component, the one less affected by the noise, and more widely spaced along the $Y_1$ component, corrupted by the real part of the Nakagami distribution. The extreme case $m=0.5$ provides a deterministic imaginary part. As a consequence, the network learns a sort of amplitude modulation along that components, leading to a MI saturation to the coding rate $R$ for every SNR. 

\begin{figure}
	\centering
	\includegraphics[scale=0.25]{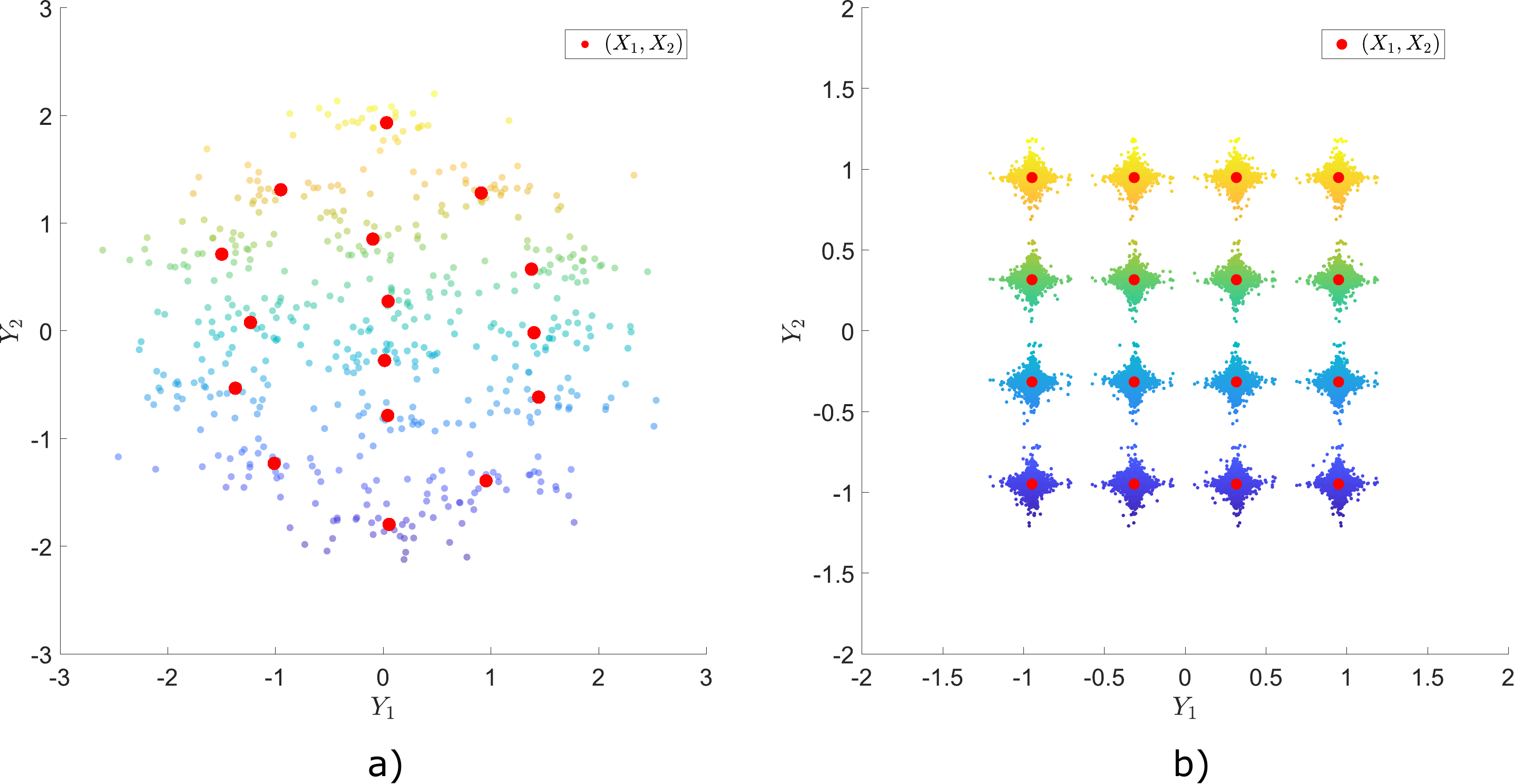}
	\caption{a) Optimal constellation scheme for a channel corrupted by additive Nakagami-$m$ noise under an average power constraint with an SNR of $7$ dB and $m=0.6$. b) $16$-QAM constellation scheme for a channel corrupted by additive measured noise with an SNR of $30$ dB.}
	\label{fig:plc_features}
\end{figure}

\subsubsection{Mutual information estimation under narrowband measured noise}
\label{subsec:plc_real_constellation}
We consider the transmission of Gaussian distributed channel input $\mathbf{x}$ to provide an upper bound of the achievable rate when transmitting $M$ messages. In particular, we set the symbol period to be $T_{\text{sym}} = 50\mu$s, typical of an OFDM modulation. Noise is sampled at $1$ MSa/s, hence, the noise traces are cut in time windows of $2000$ samples each and averaged. Lastly, we randomly pick  samples from the noise empirical distribution for both the real and imaginary part, rendering them independent. It is important to remark that such choice does not represent a performance limitation since the framework is capable to implicitly model any joint distribution, therefore any sort of dependence between real and imaginary part. Moreover, when extending to longer codes ($d>2$), it is also capable to model time dependence.
An example of a received constellation when using a $16$-QAM over the measured additive noise channel with a SNR of $30$ dB is illustrated in Fig.~\ref{fig:plc_features}b.

\begin{table}
	\centering
	\caption{CORTICAL generator architecture and training parameters.}
	\begin{tabular}{ p{4cm}|p{3cm}|p{3cm}} 
		\toprule
		\textbf{Layer} & \textbf{Output dimension } 		& \textbf{Activation function} \\
		\midrule
		\textbf{CORTICAL generator} & &\\
		Input $\mathbf{z}$ & 30 (continuous) / $k$ (discrete) & \\ 
		Fully connected & 100 & ReLU \\ 
		Fully connected &100& ReLU  \\ 
		Fully connected &100& ReLU  \\ 
		Fully connected & 2 &   \\  		\midrule
		
		\textbf{CORTICAL discriminator (DIME)} & &\\
		Input $[\mathbf{x},\mathbf{y}]$ & 4 & \\ 
		Fully connected & 100 & ReLU \\ 
		Dropout  & 0.3 &   \\ 
		Fully connected &100& ReLU  \\ 
		Fully connected & 1 & Softplus  \\ 
		Batch normalization  & & \\  \midrule

		Batch size &  \multicolumn{2}{c}{512}  \\ 
		Training iterations &  \multicolumn{2}{c}{500 generator / 5000 discriminator}  \\ 
		Learning rate &  \multicolumn{2}{c}{0.0002}   \\ 
		Optimizer &  \multicolumn{2}{c}{Adam ($\beta_1$ = 0.5, $\beta_2$ = 0.999)}  \\ 		\midrule

	\end{tabular}
	\label{tab:plc_parameters}
\end{table}

Fig.~\ref{fig:plc_capacity_real} shows the estimated MI between $\mathbf{x}$ and $\mathbf{y}$ using DIME. As expected also in this case, the value of the MI at low SNRs is greater than the AWGN capacity. Indeed, the narrowband power line statistical noise model is a mixture of different components. An interesting future experiment concerning the measured noise scenario shall consider modeling the channel input distribution with the generator of the CORTICAL framework and modeling the channel itself with the generator of a GAN.

\begin{figure}
	\centering
	\includegraphics[scale=0.3]{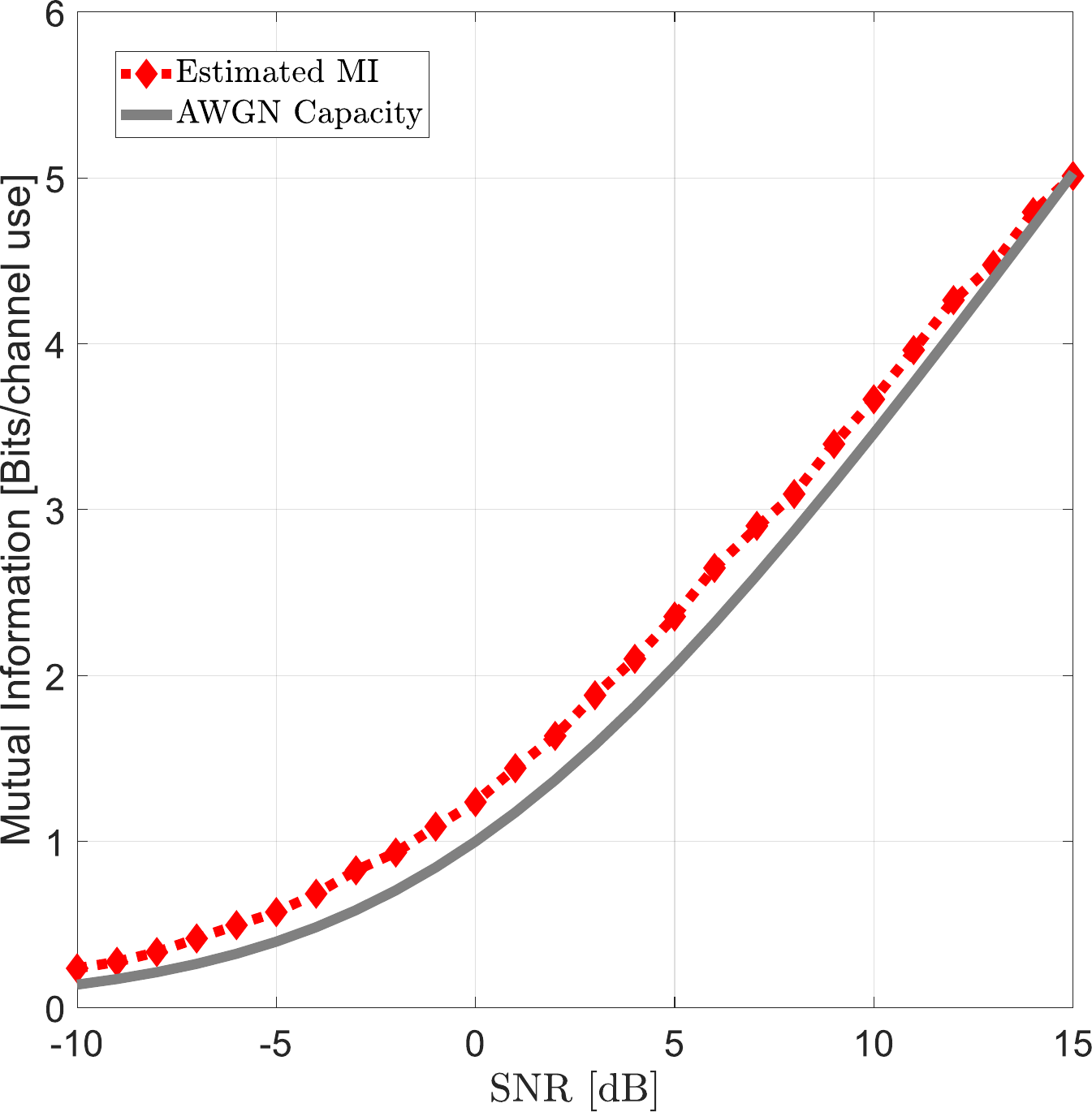}
	\caption{Mutual information estimation between channel input and output using DIME in a channel corrupted by measured narrowband noise.}
	\label{fig:plc_capacity_real}
\end{figure}

\subsection{Summary}
\label{subsec:isplc_conclusions}
In this section, we exploited CORTICAL to estimate the capacity of channels corrupted by additive noise. In particular, the PLC noise represents an important and challenging application scenario, thus, two noise classes have been analyzed. In the former case, we designed the optimal signal constellations for an additive channel in presence of Nakagami-$m$ noise and we provided an estimate of its channel capacity. In the latter case, we exploited noise traces obtained via a measurement campaign to estimate the MI between the channel input and the additive channel output, being the input Gaussian distributed. Numerical results illustrate the capability of the learning approach and how the estimated MI is greater than the AWGN channel capacity in a wide range of SNRs. The presented data-driven approach offers a first important step towards optimal channel coding design and capacity estimation in real scenarios, typically formidable unsolved tasks when the noise is not AWGN.

\section{Supervised fault detection}
\sectionmark{PLC fault detection}
\label{sec:plc_detection}
Power line modems (PLMs) act as communication devices inside a PLN. However, they can be exploited also as active sensors to monitor the status of the electric power distribution grid. Indeed, PLC signals carry information about the topological structure of the network, internal electrical phenomena, the surrounding environment and possible anomalies in the grid. An accurate and efficient identification of the types of anomaly through direct sensing measurements can enable grid operators to both prevent malfunctions and effectively intervene when faults occur. In this section, we discuss how to use supervised ML techniques to extract anomalies information from high frequency measurement of electrical quantities, namely the line impedance, the reflection coefficient and the CTF, in the PLC signal band. Simulation results confirm the potentiality of the NN method, outperforming existing model-based approaches in the field without any hyperparameter tuning. 

\subsection{Anomaly detection}
\label{subsec:plc_anomaly_detection}
Topology reconstruction \cite{lampe2013tomography,7797477} and anomalies detection \cite{8641473,7897106} are some of the applications enabled by PLC for grid sensing. The former allows the grid operators to extend the knowledge of the grid configuration, switches and feeders status. Furthermore, the knowledge of the network topology is per se relevant for networking purposes, i.e., geo-routing algorithms. The latter, instead, helps grid operators to better monitor the network status, malfunctions etc. which in turn enables its predictive maintenance granting a more efficient energy delivery service.

A number of contributions on topology reconstruction have been made \cite{6525848,6507589}, while less work has been carried on anomaly detection using PLMs. One of the main issues resides on the difficulty to formulate analytic representative models. Some recent results have shown how to relate the state of the grid, cables and loads to several electrical quantities that can be measured to provide anomalies information \cite{8653266}.
Nonetheless, diagnosis of malfunctions ends up into a classification problem \cite{8641473}. 

Faults are generally classified as low impedance faults (LIFs) and high impedance faults (HIFs). LIFs are usually bolted faults, thus, short circuits which an upstream switch or fuse can sense. HIF is, instead, a consequence of an unwanted electrical contact with a low conductive object, resulting in a low fault current that cannot be detected by conventional protection systems. HIF may cause service failures in the long run or even public hazards in the worst case. 

Researchers have documented several HIF detection techniques, spanning from classical approaches involving the analysis of the magnitude and phase of the fault current, its Fourier or wavelet transform, to heuristical approaches that exploit fuzzy logic and NNs \cite{HIF_review}.
These methodologies rely on the transmission line theory, carry-back equation and time-frequency metrics, making complicated the development of a full-representative model. NNs avoid the bottom-up approach and make use of \textit{a-posteriori} observations of natural phenomena, the experience, to create an implicit empirical top-down model \cite{ML_PLC}. 

\subsection{Results}
\label{subsec:plc_anomaly_results}
We now discuss the structure of the dataset, the training parameters and the architecture of the NN. Finally, we present the results obtained using the supervised methodology and we compare them with the model-based approach in \cite{8641473}.

\subsubsection{Materials}
\label{subsec:plc_materials}
We made use of a multi-conductor transmission line PLN simulator with the anomalies models discussed in \cite{8653266}. The simulator randomly displaces $20$ nodes on a single topology realization with average node distance of $700$ m, which mimics the average displacement in a low-voltage distribution network.
As for the cable parameters, we use the same as those presented in \cite{versolatto2011an}.
Furthermore, both the case of constant (at $2$ k$\Omega$) and randomly variable load impedances are considered.
Anomalies occurring on the network have been artificially added
to obtain $10000$ realizations of a perturbed grid \cite{8653266}. 

The training process was conducted using $50\%$ of the data, while testing was done using the other unbiased half part. 
Input signals are the ratios between the measurement of the electrical parameters $\mathbf{Y_{\text{in}}}, \mathbf{\rho_{\text{in}}}$ and $\mathbf{H}$ for a given realization, and the reference value of the respective signal in the band $4.3-500$ kHz with $4.3$ kHz sampling. The reference values were obtained when there was no anomaly. To avoid ill-conditioned networks and have a stable convergence of its weights and biases, we normalized each input realization in the interval $[-1,1]$ by using the linear transformation in \eqref{eq:plc_normalization}.

The output signals from the classifier are $4$ different class indicators according to the type of anomaly. In particular, $4$ different classes of anomalies in the network can be detected:
\begin{enumerate}
\item Unperturbed;
\item Load impedances changes;
\item Concentrated faults;
\item Distributed faults.
\end{enumerate}
To work with the cross-entropy cost function, we converted each categorical target output using a classical one-hot encoding.

\begin{figure}
	\centering
	\includegraphics[scale=0.5]{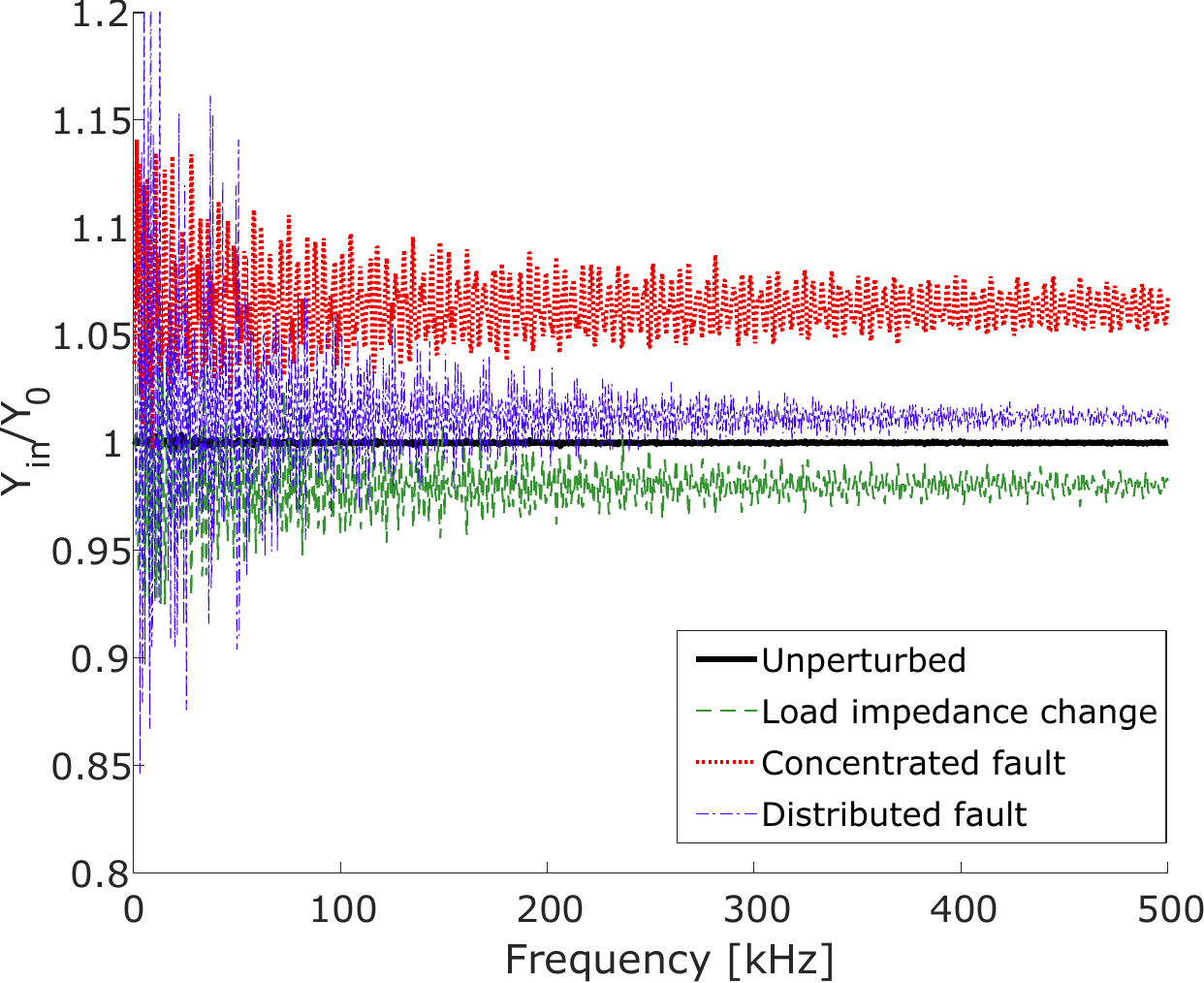}
	\caption{4 randomly picked realizations of admittance variation over the frequency for each type of anomaly.}
	\label{fig:plc_Yin}
\end{figure}

When training a neural network, overfitting is the main issue to prevent. For this reason, a standard way to proceed is to split the initial training dataset into two subsets, a training and a validation set. The former is  responsible for the update of the network parameters (weights and biases). The latter is, instead, used to check the current performance of the network and understand how it performs with previously unseen data. For the purpose of this section, we decided to split the collected dataset in $40\%$ training set, $10\%$ validation set and $50\%$ test set.

To show the potentiality and applicability of the ML approach, we chose to implement a shallow neural network, thus, a neural network with only one hidden layer consisting of $25$ neurons whose activation function $\sigma(\cdot)$ is the sigmoid.

We stopped the training process when the error on the validation set started increasing for more than $20$ consecutive iterations, sign of an overfitted model. 

\subsubsection{Performance evaluation}
\label{subsec:plc_performance}
To assess the performance of the neural network classifier, we performed a set of testing experiments judging the ability to detect and classify the type of anomaly compared to the framework in \cite{8641473}. We evaluated the models by comparing the accuracy obtained for the training and testing set. The accuracy quantitatively expresses the probability of correctly identifying the belonging class. 

We distinguished $3$ different classifiers: a binary model that is able to detect a perturbed situation from an unperturbed one; a ternary model that discriminates the unperturbed, the impedances changes and the fault cases; a quaternary model that detects and identifies all the $4$ classes previously discussed.

The first set of experiments considers the scenario of constant load impedances for each realization, while the second set takes into account the random variability of the load impedances for each realization.

$3$ classifiers for $3$ types of input signal ($\mathbf{Y_{\text{in}}}, \mathbf{\rho_{\text{in}}}$ and $\mathbf{H}$) in $2$ loads setup yield to $18$ different neural network models, whose performances are reported in Tab.~\ref{tab:plc_performances1}, Tab.~\ref{tab:plc_performances2} and Tab.~\ref{tab:plc_performances3}.

\begin{table}[b]
\begin{center}
\resizebox{\columnwidth}{!}{%
{\renewcommand{\arraystretch}{1.2}
\begin{tabular}{l|c|c|l|l}
\toprule
\multicolumn{2}{c|}{\pbox{15cm}{\textbf{Accuracy} \\ Training/\textbf{Testing}}} & \multicolumn{3}{c}{\textbf{Models for $\mathbf{Y_{\text{in}}}$}}                                            \\ \cline{3-5}
\multicolumn{2}{c|}{}                                   & \multicolumn{1}{c|}{Binary} & \multicolumn{1}{c|}{Ternary} & \multicolumn{1}{c}{Quaternary} \\ 			
\midrule

\multirow{2}{*}{\textbf{\rotatebox[origin=c]{90}{Load} \textbf{\rotatebox[origin=c]{90}{Imp.}}}}    & Constant      & 88.2\%                          /    \textbf{86.1}\%                               & 100\%                          /    \textbf{100\%}                    & 88.8\%                          /    \textbf{87.4\%}       \\ \cline{2-5}
                                            & Variable   & 99.9\%                          /    \textbf{99.7\%}  & 99.4\%                          /    \textbf{96.9\%                  }  & 97.6\%                          /    \textbf{93.0\%}       \\ \hline
\end{tabular}}%
}
\caption{Detection accuracy for $\mathbf{Y_{\text{in}}}$ with constant and variable load impedances.}
\label{tab:plc_performances1}
\end{center}
\end{table}

\begin{table}[h]
\begin{center}
\resizebox{\columnwidth}{!}{%
{\renewcommand{\arraystretch}{1.2}
\begin{tabular}{l|c|c|l|l}
\toprule
\multicolumn{2}{c|}{\pbox{15cm}{\textbf{Accuracy} \\ Training/\textbf{Testing}}} & \multicolumn{3}{c}{\textbf{Models for $\mathbf{\rho_{\text{in}}}$}}                                            \\ \cline{3-5}
\multicolumn{2}{c|}{}                                   & \multicolumn{1}{c|}{Binary} & \multicolumn{1}{c|}{Ternary} & \multicolumn{1}{c}{Quaternary} \\ 			
\midrule

\multirow{2}{*}{\textbf{\rotatebox[origin=c]{90}{Load} \textbf{\rotatebox[origin=c]{90}{Imp.}}}}    & Constant      & 86.3\%                          /    \textbf{84.3}\%                               & 100\%                          /    \textbf{100\%}                    & 88.0\%                          /    \textbf{86.8\%}       \\ \cline{2-5}
                                            & Variable   & 97.9\%                          /    \textbf{96.0\%}  & 90.2\%                          /    \textbf{86.4\%                  }  & 84.3\%                          /    \textbf{80.3\%}       \\ \hline
\end{tabular}}%
}
\caption{Detection accuracy for $\mathbf{\rho_{\text{in}}}$ with constant and variable load impedances.}
\label{tab:plc_performances2}
\end{center}
\end{table}

\begin{table}[h]
\begin{center}
\resizebox{\columnwidth}{!}{%
{\renewcommand{\arraystretch}{1.2}
\begin{tabular}{l|c|c|l|l}
\toprule
\multicolumn{2}{c|}{\pbox{15cm}{\textbf{Accuracy} \\ Training/\textbf{Testing}}} & \multicolumn{3}{c}{\textbf{Models for $\mathbf{H}$}}                                            \\ \cline{3-5}
\multicolumn{2}{c|}{}                                   & \multicolumn{1}{c|}{Binary} & \multicolumn{1}{c|}{Ternary} & \multicolumn{1}{c}{Quaternary} \\ 			
\midrule

\multirow{2}{*}{\textbf{\rotatebox[origin=c]{90}{Load} \textbf{\rotatebox[origin=c]{90}{Imp.}}}}    & Constant      & 87.3\%                          /    \textbf{85.1}\%                               & 94.8\%                          /    \textbf{94.6\%}                    & 86.8\%                          /    \textbf{85.7\%}       \\ \cline{2-5}
                                            & Variable   & 99.7\%                          /    \textbf{98.9\%}  & 99.2\%                          /    \textbf{98.0\%                  }  & 95.4\%                          /    \textbf{93.8\%}       \\ \hline
\end{tabular}}%
}
\caption{Detection accuracy for $\mathbf{H}$ with constant and variable load impedances.}
\label{tab:plc_performances3}
\end{center}
\end{table}

\subsubsection{Comments}
\label{subsec:plc_comments}
Several considerations can be made by analyzing the accuracy tables. As the intuition suggests, it is immediate to notice that the accuracy values for the training set are always higher than the one for the testing set. However, the relevance resides in the small difference between the pair of values which highlights the ability of the network to generalize the model to unseen samples.
Another remark follows from noticing that the models that accept as input the admittance $\mathbf{Y_{\text{in}}}$ perform globally better in terms of accuracy. A logic reason that motivates such result comes from the fact that the anomalies are artificially generated by imposing an abnormal impedance value in a certain node. Therefore, the primary physical quantity affected by such irregularity is the impedance (admittance) itself. Fig.~\ref{fig:plc_Yin} presents $4$ randomly picked realizations of admittance in frequency, one for each anomaly class.

\begin{figure}
\centering
\begin{subfigure}{0.5\textwidth}
  \centering
 \includegraphics[scale=0.5]{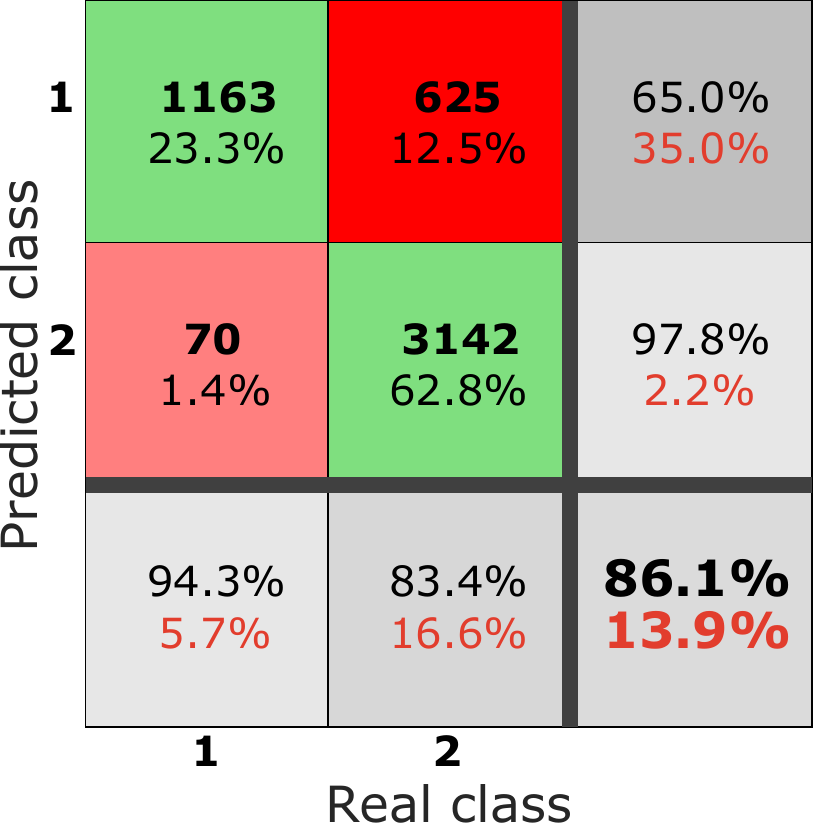}
	\caption{Binary classifier}
	\label{fig:plc_Confusion2}
\end{subfigure}%
\begin{subfigure}{0.5\textwidth}
 \centering
	\includegraphics[scale=0.31]{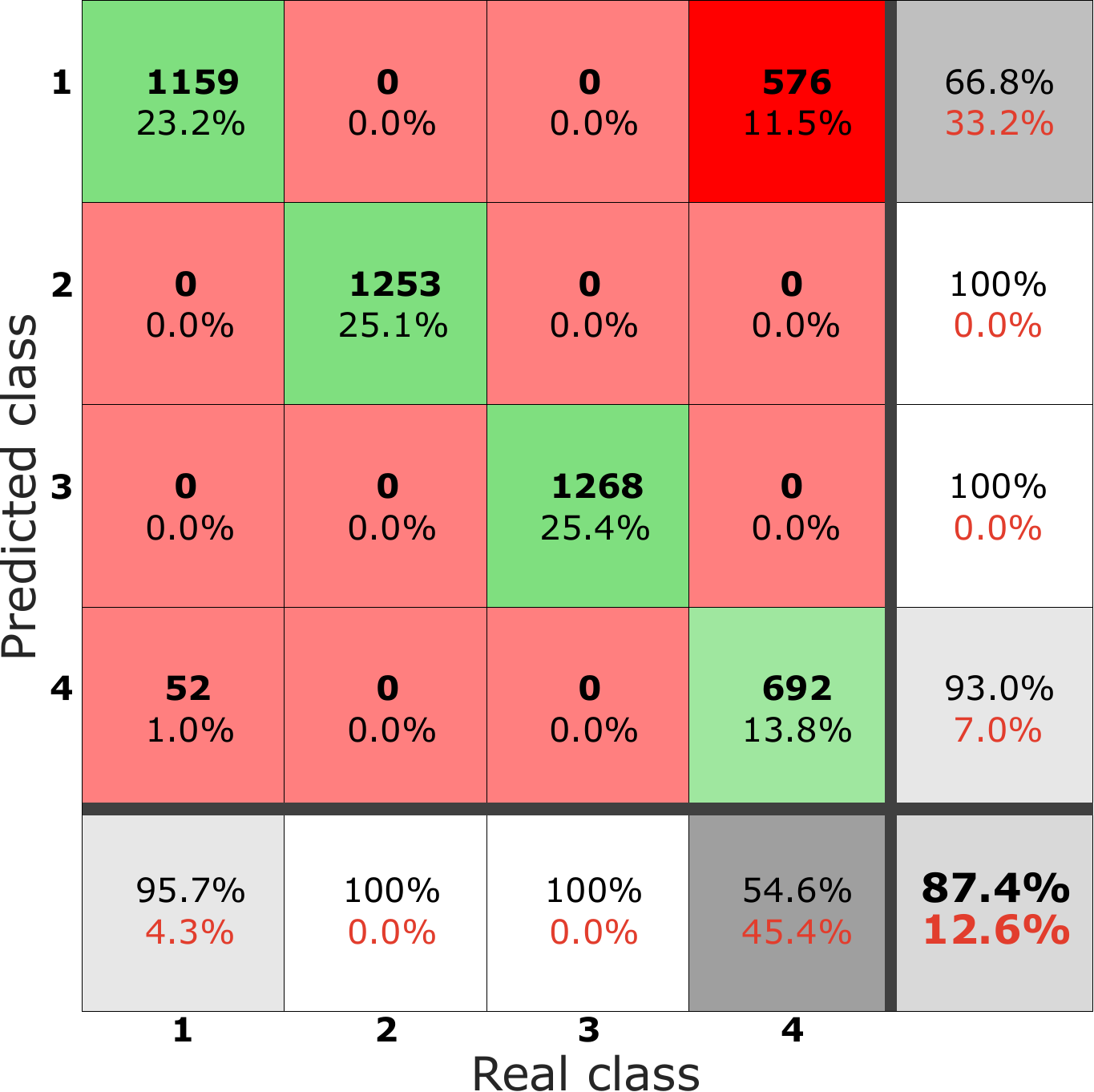}
	\caption{4-classes classifier}
	\label{fig:plc_Confusion4}
 \end{subfigure}%
 \caption{Confusion matrix (accuracy of the prediction) for anomaly detection using the admittance as the main input
signal.}
\label{fig:plc_Confusion}
\end{figure}

By looking at the binary models, the variability introduced in the load impedances improves the accuracy. A random source in the network and in the parameter configuration often brings a desired unbiasing effect that reduces overfitting issues, as the dropout layer does. Such consideration can, partially, be extended to the ternary and quaternary models.
The binary models detect the presence of a generic anomaly or not. Fig.~\ref{fig:plc_Confusion2} shows the confusion matrix for the binary model with constant load impedance when the input is the the admittance $\mathbf{Y_{\text{in}}}$. The work in \cite{8641473} proposed a bottom-up methodology with a probability of correctly detecting an anomaly equal to $0.8$, in a network of $20$ nodes. Our supervised approach reached a probability of $0.83$, improving the previous result with a minimal hyper-parameter tuning. However, the greater improvement comes from the analysis of the admittance quaternary model and its comparison with the results in \cite{8641473}. Indeed, from Fig.~\ref{fig:plc_Confusion4}, it is clear that the model perfectly classifies the load impedance change and concentrated fault anomalies, contrary to the physical model. The same confusion matrix shows that the network struggles in distinguishing between the unperturbed situation (class 1) and the distributed
fault event (class 4), as the benchmark model does in identifying distributed faults. This suggests to include the last type of anomaly in the third one, and study the ability of the neural network to classify the remaining 3 classes, in a ternary model. Such division motivates the high accuracy values obtained by the ternary model, as presented in all the tables. 

In conclusion, the developed supervised models that rely on a simple neural network structure provide excellent classification results. In particular, the ternary models have shown the best performance, discriminating an unperturbed network from a network affected by a sudden load impedance change and by a fault. As expected, it is possible to identify such events analyzing either the variation of the input admittance $\mathbf{Y_{\text{in}}}$, the variation of the reflection coefficient $\mathbf{\rho_{\text{in}}}$, or the variation of the channel transfer function $\mathbf{H}$. The detection methodology does not require an a priori knowledge of the physical model, but only a simple neural network structure that learns from the measurements. Moreover, the single hidden layer and the reduced number of neurons allow an easy and fast implementation in a non-complex PLM device. 

\subsection{Summary}
\label{subsec:plc_anomaly_conclusions}
In this section, we discussed a supervised NN approach to detect anomalies inside a PLN, exploiting the high frequency PLC signals sensed by power line modems.
To serve the scope, we proposed a shallow NN, consisting of only one hidden layer with $25$ neurons. Accuracy results proved the efficacy of the learned models, which are able to both detect and classify the type of anomaly with $3$ different input signal: $\mathbf{Y_{\text{in}}}, \mathbf{\rho_{\text{in}}}$ and $\mathbf{H}$. The performance of the presented ML approach together with the simplicity of the methodology encourage further endeavors in the area of anomaly detection with power line technology. 

\part{Part 5}

\chapter{Recursive Data Interpolation Techniques} 
\chaptermark{Interpolation}
\label{sec:data interpolation}

Ch.~\ref{sec:copulas} described how, from $N$ data points $(x_1,\dots, x_N)$, it is possible to leverage NNs to interpolate a function representing the underlying PDF $p_X(x)$ that generated the observed data. In all the previous contexts, interpolation effectively meant estimating the distribution of the given samples in order to understand the \textit{stochastic} nature of the data. We typically exploit $p_X(x)$ to produce new realizations $(\hat{x}_1,\dots, \hat{x}_N)$.

In this chapter, instead, we focus on capturing the \textit{deterministic} relationship between inputs and outputs, determining a function that passes through specific points. 
When we approach interpolation from the perspective of finding a fitting function, we aim to create a mathematical model that accurately represents the relationship between the given data points, abstracting from how these points have been obtained. 
This fitting function allows us to predict intermediate values within the range of the data, enabling smooth transitions between observed points.

Fitting functions are extremely useful for trajectory generation in various fields such as physics, robotics and aerospace engineering. When designing trajectories for moving objects or systems, it is essential to create smooth and continuous paths that meet specific constraints and objectives.
In the following, we study the trajectory generation problem and propose a novel iterative interpolation technique.

The results presented in this chapter are documented in \cite{Letizia2020_rst, LetiziaRobotics, 9525383}.

\label{sec:rst}
\section{Introduction}
\label{subsec:rst_introduction}
Recent advances in transportation have resulted in an increased adoption of unmanned vehicles, i.e., underwater unmanned vehicles, unmanned ground vehicles and unmanned aerial vehicles (UAVs), for a wide range of applications. The autopilot that controls the vehicle is typically guided by the inertial navigation system. In the particular case of aerial navigation, current guidance and control systems deployed in UAVs provide real-time control laws to track desired trajectories set by the embedded flight management system (E-FMS) \cite{spitzer2000digital}. The E-FMS intrinsically has the advantage that it can be segmented into four blocks and each of them individually studied and optimized. Its synoptic realization, that is embedded in the UAV platform, is shown in Fig.~\ref{fig:rst_E-FMS}. The policy planner is used to define the waypoints and the situation of a mission. The path planner is responsible for constructing the feasible trajectory according to tasks specified by the policy planner and it is a core block for physical task realizations. The duty of this block becomes extremely challenging because a feasible trajectory must consider and satisfy: 1) the non-linearity of the system model; 2) the kinematic constraints; and 3) smoothness. The agent block takes complexity into account and intelligently computes the control laws to track the desired trajectories. Clearly, the control laws affect the UAV flight dynamics. \par
In this chapter, we focus on aerial vehicles and on the the path planning function of the E-FMS.
The path planner is an upstream block, thus, it allows generating desired trajectories for the considered aircraft, enabling the selection of a flight control law by the guidance function.
\begin{figure}[t]
\includegraphics[scale=0.46]{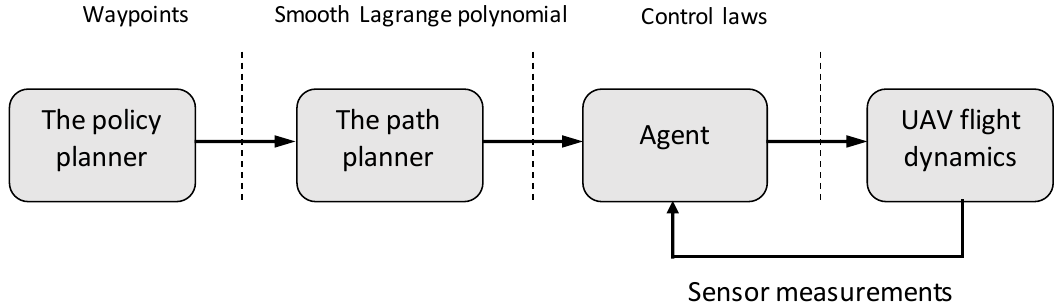}
\centering
\caption{The synoptic realization of an E-FMS.}
\label{fig:rst_E-FMS}
\end{figure}

A large effort has been dedicated to implement the path planner block with acceptable computational complexity and several solutions have been developed. A possible categorization of the methodologies is given as follows. The first class of algorithms designs trajectories based on statistical learning. Some examples are the $A^{\ast}$ algorithm \cite{DeFilippis2012}, the graph methodology \cite{1206461}, evolutionary algorithms \cite{richter2016polynomial}, the rapidly-exploring random tree (RRT) \cite{5175292}, and mixed-integer linear programming \cite{5446292}. $A^{\ast}$ and RRT generate an optimal collision-free piecewise linear path but they do not consider the dynamics constraints of the system. In this sense, evolutionary algorithms are able to generate an optimal feasible trajectory that takes into account the dynamics constraints at the expense of a high computational cost.

The second class of algorithms can be considered as a parametric function of the time, which provides at each instant the corresponding desired position. Path primitives such as lines \cite{HoffmannQuadrotor}, polynomials \cite{7068316}, and splines \cite{4602025}, have been deployed to build such algorithms. This class relies on polynomial interpolation techniques to obtain a trajectory that passes through multiple waypoints. In fact, the differentiability of polynomials makes them a suitable interpolation choice for considering the vehicle dynamics. However, it is necessary to solve an inverse problem to find a unique polynomial trajectory that fulfills the constraints. Therefore, the interpolation problem is often split into splines, i.e., piecewise polynomial trajectories. Splines are easy to be constructed and provide bounded trajectories although the continuity of the derivatives (smoothness) at the waypoints is only guaranteed up to a certain order.

Finally, the third class of algorithms generates trajectories exploiting the differential flatness property of the system. The flatness property provides an analytical mapping from a path and its derivatives to the states and control inputs of the UAV flight dynamics required to track that path accurately. Furthermore, it is particularly advantageous for solving 3D trajectory planning and waypoint following control. In \cite{richter2016polynomial} the flatness-based property of a UAV is exploited to convert predefined waypoints into polynomial trajectories using quadratic programming. In \cite{4602025} the trajectory planning is extended to a constrained optimization problem and the parametrization of the trajectory is modeled as a composition of a parametric function $P(\lambda)$ defining the path, and a monotonically increasing function $\lambda(t)$ specifying the motion on this path. $P(\lambda)$ and $\lambda(t)$ are modeled using B-spline functions. The tuning parameters of $P(\lambda)$ and $\lambda(t)$ are obtained using the sequential quadratic programming technique. The B\'ezier polynomial function has been used to solve the constrained optimization problem in  \cite{6324664}. However, all the previous works focused on a piecewise parametrization of the flat outputs for the system, which may cause discontinuities in the control inputs as a direct consequence of non sufficiently smooth generated trajectories. 
We solve the former issue by proposing a recursive paradigm to generate smooth trajectories. The main ideas and the preliminary numerical results showing the advantage of the recursive formulation of the trajectory generation algorithm were presented in \cite{Letizia2020_rst}. This section offers the mathematical foundation and derivation of the method. In addition, this chapter considers the statistical perturbation of the constraints, it considers the trajectory generation in an optimization framework, and it reports the complexity analysis results of the recursive algorithm, which were missing aspects in \cite{Letizia2020_rst}.

In the following, we provide theoretical foundations for the recursive smooth trajectory generation algorithm (RST) that allows the design of a smooth ($\mathcal{C}^{\infty}$) polynomial path. It uses the concept of waypoints to find a closed form trajectory that satisfies any arbitrary dynamic limitations mapped into kinematic constraints. It is fundamentally different from existing approaches \cite{6376099} in the following aspects:
\begin{itemize}
\item The trajectory is defined by a unique polynomial which fulfills all the kinematic constraints;
\item The trajectory is recursively built exploiting the concept of $m$-partial trajectory (See Def. \ref{def:rst_def1}), thus, the generation method avoids matrix inversion;
\item The RST algorithm enables a new piecewise (PRST) and blockwise (BRST) interpolation approach to tackle numerical and oscillation problems;
\item Uncertainties in kinematic constraints are simply translated into uncertainties in the polynomial coefficients. In this way, all the feasible trajectories are obtained from a deterministic polynomial plus a stochastic perturbation described by a multivariate Gaussian distribution;
\item Lastly, RST provides an immediate formulation and solution to trajectory optimization problems.
\end{itemize}

\subsection*{Notation}
\label{subsec:rst_notation}
Vectors and matrices are denoted with bold letters. Unless stated otherwise, all vectors in this section are column vectors.
$t_j$ denotes a point in time while $x(t_j)$ and $\mathbf{x}(t_j)$ denote a point in the $1D$ space and a vector with components in the multi-dimensional space, respectively. $\frac{d^i\mathbf{x}}{dt^i}\bigr|_{t=t_j}$ denotes the $i$-th derivative of $\mathbf{x}(t)$ evaluated in $t_j$. The factorial of $n$ is denoted by $n!$. A stochastic perturbation of the deterministic variable $y$ will be described by $\tilde{y}$. $||.||^2$ denotes the norm squared.
The state variables of the dynamical system are a function of time, e.g., $\bm{x} = \bm{x}(t)$. The first, second and $n$-th derivative with respect to time of a state space variable $\bm{x}$ are denoted respectively with $\bm{\dot{x}}$, $\bm{\ddot{x}}$ and $\bm{x}^{(n)}$. 
\section{Problem statement}
\label{subsec:rst_problem}
Consider the problem of navigating a moving body, e.g., an unmanned vehicle, through $N+1$ waypoints at specific time stamps. Assume that any arbitrary dynamic limitation of the  body translates into kinematic constraints (e.g. position, velocity, acceleration, etc.). Such kinematic constraints are useful to impose, for instance, that the body starts from the rest at the beginning of the trajectory and reaches each waypoint with specific kinematics (e.g. aggressive maneuvering).

The task of path planning consists of finding a trajectory $\mathbf{x}(t)$ to move a given body from an initial point $\mathbf{x}(t_0)$ to a final point $\mathbf{x}(t_N)$, minimizing a certain cost function $J(\cdot)$ (time, energy, jerk, etc.), under given geometrical and kinematic constraints $\sigma(\cdot)$.

Let $t_0<t_1<\dots<t_N$ be $N+1$ time stamps for which all the corresponding positions $\mathbf{x}(t_0),\mathbf{x}(t_1), \dots, \mathbf{x}(t_N)$ and the respective $k$ derivatives $\frac{d^i\mathbf{x}}{dt^i}\bigr|_{t=t_0}, \frac{d^i\mathbf{x}}{dt^i}\bigr|_{t=t_1}, \dots, \frac{d^i\mathbf{x}}{dt^i}\bigr|_{t=t_N}$ are a-priori assigned, for $i = 1,\dots, k$. For notation convenience, the trajectory for which the first $k$ derivatives are defined, is denoted with $\mathbf{x}_k(t)$.

The objective is to design a feasible trajectory $\mathbf{x}_k(t)$ which satisfies the equality kinematic constraints $\sigma_{i,j}$ expressed by the first $k$ derivatives $\frac{d^i\mathbf{x}}{dt^i}\bigr|_{t=t_j} = \sigma_{i,j} $, for $i=0,1,\dots,k$ and $j=0,1,\dots,N$. 

To do so we propose to use polynomials as the set of our feasible trajectories. In particular, we build the polynomial trajectory as a linear combination of $k+1$ polynomial basis
\begin{equation}
\label{eq:rst_trajectory}
\mathbf{x}_k(t) = \sum_{i=0}^{k}{\mathbf{p}_i(t)},
\end{equation}
where $\mathbf{p}_i(t)$ is the polynomial (for each component of the $3D$ space) \textit{responsible} for the fulfillment of the $i$-th derivative constraint and $k$ is the number of considered derivatives. In other words, each polynomial component $\mathbf{p}_i(t)$ is designed as an additive term which iteratively fulfills the $i$-th kinematic constraint.
Furthermore, we exploit Lagrange interpolating polynomials to compute the polynomial basis $\mathbf{p}_i(t)$.

To the best of authors' knowledge, this is the first work that recursively builds a set of polynomials for trajectory generation under kinematic constraints.

\section{Theoretical formulation of RST}
\label{subsec:rst_results}
This section presents the mathematical foundations behind the choice of the polynomial representation in \eqref{eq:rst_trajectory}. We show that such choice produces a closed form continuous polynomial trajectory and enables an iterative algorithm for its generation. In particular, Lagrange polynomials are exploited. For ease of notation, in the following we consider the trajectory in \eqref{eq:rst_trajectory} as unidimensional, assuming that the $3D$ extension is obtained by working component-wise.%
\subsection{Preliminaries on Lagrange polynomials}
\label{subsec:rst_Lagrange}
Lagrange polynomials offer a technique to interpolation problems. In particular, given a set of $N+1$ $1D$ control points (waypoints) $(t_0,x(t_0)), (t_1,x(t_1)), \dots, (t_N,x(t_N))$, the interpolation polynomial in the Lagrange form is defined as
\begin{equation}
\label{eq:rst_lagrange}
L(t) := \sum_{j=0}^{N}{x(t_j) \ell_j(t)}
\end{equation}
where $\ell_j(t)$ is the Lagrange polynomial basis whose form reads as follows
\begin{equation}
\ell_j(t):= \prod_{\substack{m=0 \\ m\neq j}}^{N}{\frac{t-t_m}{t_j-t_m}},
\end{equation}
with $0\leq j\leq N$. The main idea behind this definition is that, by construction, the $N+1$ basis functions are such that $\ell_j(t_i)\equiv 0$ in $i=0,\dots,N \; \wedge \; i\neq j$. So for each waypoint, only one single basis function contributes to the sum in \eqref{eq:rst_lagrange}.

\subsection{The recursive formula}
To discover the recursive property behind the formulation in \eqref{eq:rst_trajectory}, we start the mathematical derivation introducing two lemmas regarding a particular polynomial choice and its derivatives. Such type of choice enables a remarkable recursive property which is proved in Theorem \ref{theorem:rst_theorem1}, the main result of this chapter. The corollary \ref{corollary:rst_corollary1}, instead, establishes an upper bound for the minimum degree of the polynomial trajectory generated via the recursive formulation.

\begin{lemma}
\label{lemma:rst_Lemma1}
Let
\begin{equation}
a(t) = \prod_{n=0}^{N}{(t-t_n)},
\label{eq:rst_a_t}
\end{equation}
then
\begin{equation}
\frac{d}{dt}a(t) = a(t)\cdot \sum_{n=0}^{N}{\frac{1}{t-t_n}}
\label{eq:rst_da_t}
\end{equation}
and
\begin{equation}
\frac{d}{dt}a(t)\biggr|_{t=t_j} = \prod_{\substack{n=0\\ n\neq j}}^{N}{(t_j-t_n)}.
\label{eq:rst_da_tj}
\end{equation}
\end{lemma}
\begin{proof}
Using product rule
\begin{equation}
\frac{d}{dt}a(t) = \sum_{j=0}^{N}{\prod_{\substack{n=0\\ n\neq j}}^{N}{(t-t_n)}}.
\label{eq:rst_Lemma1}
\end{equation}
Dividing and multiplying by $a(t)$ gives \eqref{eq:rst_da_t}. For $t=t_j$ only one term of the summation in \eqref{eq:rst_Lemma1} contributes leading to \eqref{eq:rst_da_tj}. \qedhere  
\end{proof}

\begin{lemma}
\label{lemma:rst_Lemma2}
Let
\begin{equation*}
a(t) = \prod_{n=0}^{N}{(t-t_n)},
\end{equation*}
and let $s(t)$ be a polynomial. Then
\begin{equation}
\frac{d^h}{dt^h}\biggl[\frac{a^k(t)}{k!}\cdot s(t)\biggr]_{t=t_j} \equiv 0 \; \forall k>h
\label{eq:rst_diff_ak}
\end{equation}
and
\begin{equation}
\frac{d^h}{dt^h}\biggl[\frac{a^h(t)}{h!}\cdot s(t)\biggr]_{t=t_j} = \biggl(\frac{d}{dt}a(t)\biggr)^h\ \biggr|_{t=t_j} \cdot s(t_j).
\label{eq:rst_diff_ah}
\end{equation}
\end{lemma}
\begin{proof}
Since $t=t_j$ is a zero of $a^k(t)$ with multiplicity $k>h$, the factor $(t-t_j)$ appears in every term of the $h$-th derivative. If $k=h$, then
\begin{equation}
\frac{d^h}{dt^h}\biggl[\frac{a^h(t)}{h!}\cdot s(t)\biggr]_{t=t_j} = \frac{d^{h-1}}{dt^{h-1}}\biggl[\frac{d}{dt} \biggl(\frac{a^h(t)}{h!}\cdot s(t)\biggr)\biggr]_{t=t_j}.
\end{equation}
By the linearity of the differential operator and product rule, RHS can be rewritten as
\begin{equation}
\frac{d^{h-1}}{dt^{h-1}}\biggl[\frac{a^{h-1}(t)}{{h-1}!}\cdot \frac{d a(t)}{dt}\cdot s(t)\biggr]_{t=t_j} + \cancel{\frac{d^{h-1}}{dt^{h-1}}\biggl[\frac{a^h(t)}{h!}\cdot \frac{d}{dt}s(t)}\biggr]_{t=t_j}
\end{equation}
where the second term in the above expression vanishes because of the first part of the lemma. Proceeding in the same way until the $h$-th derivative leads to
\begin{equation}
\frac{d^h}{dt^h}\biggl[\frac{a^h(t)}{h!}\cdot s(t)\biggr]_{t=t_j} = \biggl(\frac{d}{dt}a(t)\biggr)^h\ \biggr|_{t=t_j} \cdot s(t_j)
\end{equation}
which concludes the proof. \qedhere  
\end{proof}

To understand the terms involved while building the final trajectory $x_k(t)$, the concept of $m$-partial trajectory is introduced.
\begin{defn}
\label{def:rst_def1}
Let $\frac{d^i}{dt^i}x_k(t)\bigr|_{t=t_j}$ be given kinematic constraints, for $i=0,1,\dots, k$. An $m$-partial trajectory is a feasible trajectory which only fulfills the first $m$ kinematic constraints, i.e. 
\begin{equation}
x_m(t) = \sum_{i=0}^{m}{p_i(t)}, \; \text{with }\; m\leq k.
\end{equation} 
\end{defn}
As an example, if the kinematic constraints are set up to the acceleration ($k=2$), a $1$-partial trajectory is a polynomial trajectory that passes through the $N+1$ waypoints with the given velocities.

\begin{theorem}
\label{theorem:rst_theorem1}
Let $t_j$ be a point in time, for $j=0,1,\dots, N$, such that $\frac{d^i}{dt^i}x_k(t)\bigr|_{t=t_j}=\sigma_{i,j}$ is the associated given kinematic constraint, for $i=0,1,\dots, k$. Let $x_k(t)$ be a feasible polynomial trajectory defined as
\begin{equation}
x_k(t) = \sum_{i=0}^{k}{p_i(t)}.
\end{equation}
If
\begin{equation}
p_i(t)=\dfrac{1}{i!}\biggl(\prod_{n=0}^{N}{(t-t_n)}\biggr)^i\cdot s_i(t)
\end{equation}
then the $i$-partial trajectory $x_i(t)$ depends recursively on $x_{i-1}(t)$. In particular,
\begin{equation}
s_i(t_j)=\dfrac{\dfrac{d^i}{dt^i}x_k(t)\biggr|_{t=t_j} -\dfrac{d^i}{dt^i}x_{i-1}(t)\biggr|_{t=t_j}}{\displaystyle \Biggl(\prod_{\substack{n=0\\ n\neq j}}^{N}{(t_j-t_n)}\Biggr)^i}.
\label{eq:rst_recursive}
\end{equation}
\end{theorem}

\begin{proof}
For $i=0$, $x_{-1}(t):=0$ and the kinematic constraint is $x_k(t_j)$ which brings to the first polynomial $s_0(t)=p_0(t)=x_0(t)$ obtained using an interpolation technique, e.g. the Lagrange polynomials as explained in Sec. \ref{subsec:rst_Lagrange}. No other contributions $p_i(t)$ are taken into account since $p_i(t_j)=0 \; \forall i>0$. For brevity of notation it is now convenient to define a polynomial $a(t)$ as $a(t) = \prod_{n=0}^{N}{(t-t_n)}$.

For $i=1$ and constraint $\frac{d}{dt}x_k(t)\bigr|_{t=t_j}$, 
\begin{equation}
x_k(t) = x_0(t) + a(t)\cdot s_1(t) + \sum_{i=2}^{k}{\dfrac{a^i(t)}{i!}s_i(t)},
\label{eq:rst_i1}
\end{equation}
taking the derivative in $t_j$ in both sides of \eqref{eq:rst_i1} yields to
\begin{equation}
\frac{d}{dt}x_k(t)\biggr|_{t=t_j} = \frac{d}{dt}x_0(t)\biggr|_{t=t_j} + \frac{d}{dt}a(t)\biggr|_{t=t_j}\cdot s_1(t_j)
\end{equation}
where, using Lemma \ref{lemma:rst_Lemma2}, no contribution in $t_j$ comes from $i>1$. Rearranging,
\begin{equation}
s_1(t_j)=\dfrac{\dfrac{d}{dt}x_k(t)\biggr|_{t=t_j} -\dfrac{d}{dt}x_0(t)\biggr|_{t=t_j}}{\displaystyle \prod_{\substack{n=0\\ n\neq j}}^{N}{(t_j-t_n)}}
\end{equation}
where the denominator is a consequence of Lemma \ref{lemma:rst_Lemma1}.
To compute $s_1(t)$, it is convenient to interpolate the points $s_1(t_j)$ again with Lagrange polynomials. The $1$-partial trajectory has expression $x_1(t)=x_0(t)+a(t)\cdot s_1(t)$.

In general, for $i=h$ the constraint to be fulfilled is $\frac{d^h}{dt^h}x_k(t)\bigr|_{t=t_j}$. Thus
\begin{equation}
x_k(t) = x_{h-1}(t) + \dfrac{a^h(t)}{h!}\cdot s_h(t) + \sum_{i=h+1}^{k}{\dfrac{a^i(t)}{i!}\cdot s_i(t)},
\label{eq:rst_ih}
\end{equation}
where $x_{h-1}(t)$ is the $(h-1)$-partial trajectory. Taking the $h$-th derivative in $t_j$ in both sides of \eqref{eq:rst_ih} and using again Lemma \ref{lemma:rst_Lemma2} for the second and third term, yields to
\begin{equation}
\frac{d^h}{dt^h}x_k(t)\biggr|_{t=t_j} = \frac{d^h}{dt^h}x_{h-1}(t)\biggr|_{t=t_j} + \biggl(\frac{d}{dt}a(t)\biggr)^h\ \biggr|_{t=t_j} \cdot s_h(t_j).
\end{equation}
Finally, rearranging
\begin{equation}
s_h(t_j)=\dfrac{\dfrac{d^h}{dt^h}x_k(t)\biggr|_{t=t_j} -\dfrac{d^h}{dt^h}x_{h-1}(t)\biggr|_{t=t_j}}{\displaystyle \Biggl(\prod_{\substack{n=0\\ n\neq j}}^{N}{(t_j-t_n)}\Biggr)^h}
\end{equation}
concludes the proof. \\ \qedhere  
\end{proof}

Theorem \ref{theorem:rst_theorem1} provides a recursive formula to evaluate the points $s_i(t_j)$. Hence, it is sufficient to interpolate them, for example with Lagrange polynomials, in order to get $s_i(t)$ at each iteration. As a remark, the interpolation of $s_i(t_j)$ can be carried out also by adopting other polynomial basis functions, e.g., Newton polynomials, B-Spline, etc. Nevertheless, the polynomial assumption comes from Lemma \ref{lemma:rst_Lemma2}. Indeed, polynomials, when differentiated, do not introduce extra poles which could cancel with $a(t)$. On the contrary, other basis functions do not guarantee this property.

The following corollary provides an information on the minimum degree of the polynomial trajectory $x_k(t)$ defined in \eqref{eq:rst_trajectory}.

\begin{corollary}
\label{corollary:rst_corollary1}
Let $t_j$ be a point in time, for $j=0,1,\dots, N$, such that $\frac{d^i}{dt^i}x_k(t)\bigr|_{t=t_j}$ is the associated given kinematic constraint, for $i=0,1,\dots, k$. If $x_k(t)$ is a feasible polynomial trajectory defined as
\begin{equation}
x_k(t) = \sum_{i=0}^{k}{\dfrac{1}{i!}\biggl(\prod_{n=0}^{N}{(t-t_n)}\biggr)^i\cdot s_i(t)},
\end{equation}
then the minimum degree of $x_k(t)$ is less or equal to $(k+1)(N+1)-1$.
\end{corollary}
\begin{proof}
Each $s_i(t)$ is a Lagrange polynomial that passes through $N+1$ points, therefore its minimum degree is $N$. The highest contribution in terms of degree to $x_k(t)$ comes when $i=k$, so that $\prod_{n=0}^{N}{(t-t_n)^k}$ is a polynomial of degree equal to $(N+1)k$. Thus, the product of $\prod_{n=0}^{N}{(t-t_n)^k}$ and $s_k(t)$ gives a polynomial whose minimum degree is at most $(N+1)k+N = (k+1)(N+1)-1$.
This is somehow consistent with the idea that imposing $k+1$ constraints for each of the points in time $t_j$ gives $(k+1)(N+1)$ constraints and the minimum degree of a polynomial that satisfies them is $(k+1)(N+1)-1$. \qedhere
\end{proof}

\subsection{Recursive smooth trajectory generation}
\label{subsec:rst_code}

\begin{figure}[t]
\includegraphics[scale=0.40]{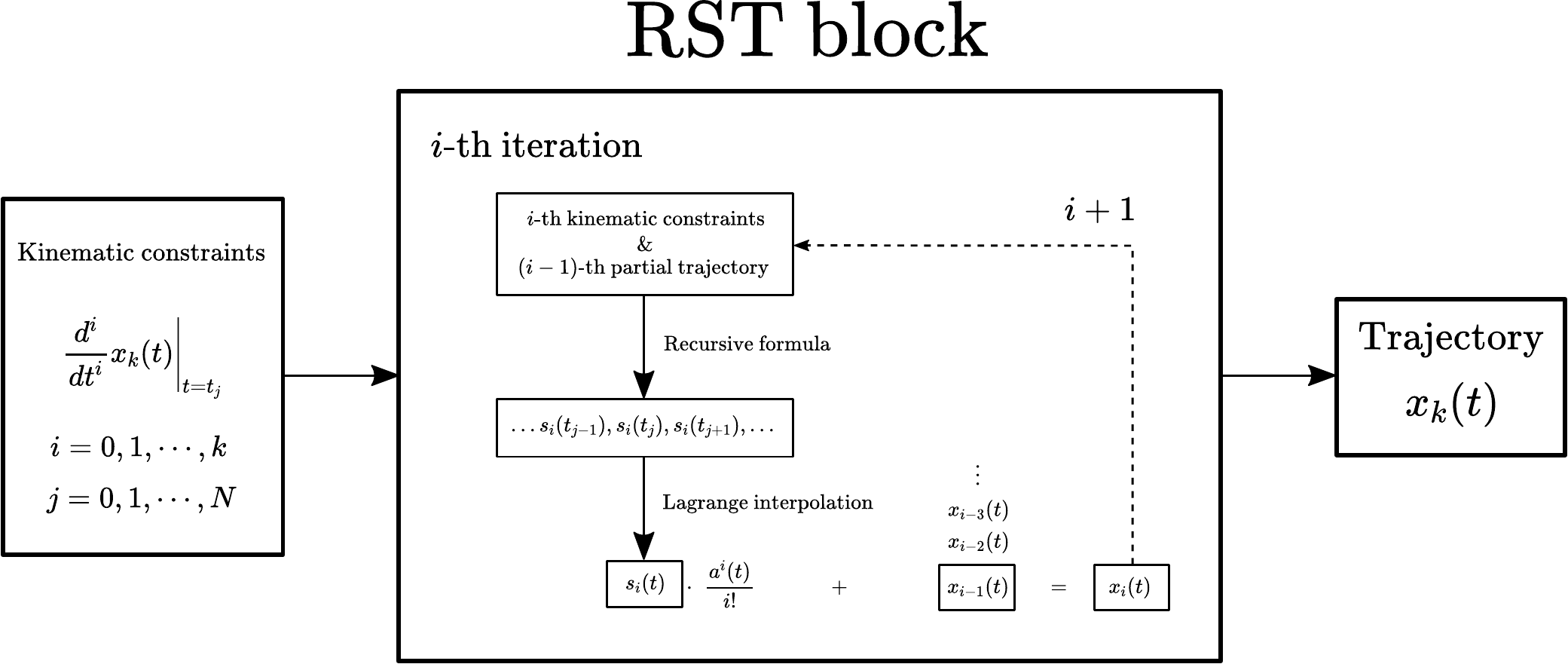}
\centering
\caption{RST block diagram.}
\label{fig:rst_RST}
\end{figure}
In the following, we will denote our approach as recursive smooth trajectory (RST) generation. The idea behind it, is that each component $p_i(t)$ in \eqref{eq:rst_trajectory} guarantees that $\dfrac{d^i}{dt^i}x_k(t)\bigr|_{t=t_j}=\sigma_{i,j}$ is fulfilled. Starting from the waypoints constraint which can be easily calculated through Lagrange polynomials, all the following higher-order differential kinematic constraints depend recursively on previous ones, according to \eqref{eq:rst_recursive}. 

The pseudocode in Alg.~\ref{alg:RST} provides a practical idea on how to iteratively design the trajectory $x_k(t)$ under the conditions aforementioned. The scheme in Fig.~\ref{fig:rst_RST} illustrates how RST operates.

We now discuss the influence of the distribution of the time instants $t_j$ for $j=0,\dots,N$ over the trajectory on oscillations and numerical limitations by proposing a known solution (Chebyshev nodes) and a novel hybrid solution, denoted as blockwise recursive smooth trajectory (BRST) that tackle said practical issues.

\begin{algorithm}
\caption{Recursive smooth trajectory (RST) generation}
\label{alg:RST}
\begin{algorithmic}[1]
\Inputs{$N+1$ points in time $t_0<t_1<\dots<t_N$; \\ Number of derivatives $k$ to fulfill; \\ Kin. constr. $\frac{d^i}{dt^i}x_k(t)\bigr|_{t=t_0}, \dots, \frac{d^i}{dt^i}x_k(t)\bigr|_{t=t_N}$.}
\Initialize{$a(t)=(t-t_0)\cdots (t-t_N)$ polynomial; \\ $x_{-1}(t) \equiv 0$.}
\For{$i=0$ to $k$}
	\For{$j=0$ to $N$}
		\State $s_i(t_j)=\dfrac{\dfrac{d^i}{dt^i}x_k(t)\biggr|_{t=t_j} -\dfrac{d^i}{dt^i}x_{i-1}(t)\biggr|_{t=t_j}}{\displaystyle \Biggl(\dfrac{d}{dt}a(t)\biggr|_{t=t_j}\Biggr)^i}$;
	\EndFor
	\State Interpolate $s_i(t_j)$ with Lagrange polynomial $s_i(t)$;
	\State $x_i(t) = x_{i-1}(t)+\dfrac{a^i(t)}{i!}\cdot s_i(t)$;
\EndFor
\end{algorithmic}
\end{algorithm}

\subsection{Remarks and insights on blockwise approach}
In the previous sections, we presented the formal generation approach (RST) of $x_k(t)$ for points in time $t_0<t_1< \dots <t_N$ with no constraints on $N$ and on the length of the interval $I_j = t_{j+1}-t_j$ for $j = 0,\dots, N-1$. 
Unfortunately, when the number of points $N+1$ is large and in particular when points in time $t_j$ are equally spaced ($I_j$ is constant), the Runge's phenomenon may occur \cite{Runge}. The Runge's phenomenon is an oscillation problem near the endpoints of the polynomial interpolation function, as illustrated in Fig. \ref{fig:rst_Runge}. To overcome it, one could either move to spline interpolation as mentioned in Sec. \ref{subsec:rst_introduction} or change the distribution of the nodes $t_j$ more densely towards the edges of the interval $[t_0, t_N]$ as proposed in \cite{Berrut}.
In the latter case, a standard choice considers the set of points in time as the set of Chebyshev nodes. In particular, for $N+1$ points in the interval $[t_0, t_N]$, nodes are transformed into
\begin{equation}
\hat{t}_j = \frac{1}{2}(t_0+t_N) + \frac{1}{2}(t_N-t_0)\cos\biggl[\dfrac{2j+1}{2(N+1)}\pi\biggr], \; j = 0,\dots, N.
\label{Cheby}
\end{equation}

Moreover, when the number of points $N+1$ increases, we encounter computational limitations in evaluating the powers $t^{(k+1)(N+1)-1}$. In such a case, a possible new blockwise approach that we name blockwise RST (BRST) concatenates intervals 
\begin{equation}
    [t_{0,1},t_{N,1}], [t_{0,2},t_{N,2}], \dots, [t_{0,M},t_{N,M}]
\end{equation} 
in $M$ blocks. For each block, the associated trajectory is separately calculated as described in Sec. \ref{subsec:rst_code}. Since kinematic constraints are intrinsically considered in the formulation of the trajectory, interfaces are already jointly matched (up to $k$-th derivative) without any need of optimization steps. Finally, when $N=1$, we bring blockwise back to piecewise polynomial trajectories under the recursive framework and we will refer to it as piecewise RST (PRST) algorithm. It is important to notice that the PRST approach provides a piecewise trajectory that is exactly the same as the one generated using the classic spline interpolation method. However, the two methods are intrinsically different: PRST finds the piecewise trajectory recursively by building $k+1$ partial trajectories, while the spline interpolation technique solves a system of linear equations, thus it needs a matrix inversion. A quick comparison of the computational cost is presented in Sec. \ref{subsec:rst_complexity}, while a more detailed study is left for future research, since the most efficient implementation of the RST needs to be studied.
\section{Trajectory perturbation}
\label{subsec:rst_perturbation}
In this section, we present the formal procedure in order to deal with uncertainties in the kinematic constraints. For notation convenience, we will denote $\frac{d^i}{dt^i}\tilde{x}_k(t)\bigr|_{t=t_j}$ as the perturbed constraint. 

\subsection{Uncertainty model}
We model the uncertainty in the constraints as an additive contribution $\varepsilon_i$ to the fixed deterministic part, in particular
\begin{equation}
\label{eq:rst_perturbation}
\frac{d^i}{dt^i}\tilde{x}_k(t)\bigr|_{t=t_j} = \underbrace{\frac{d^i}{dt^i}x_k(t)\bigr|_{t=t_j}}_{\text{deterministic}}+\underbrace{\varepsilon_i(t_j)}_{\text{stochastic}},
\end{equation}
where $\varepsilon_i\sim \mathcal{N}(0,\sigma^2_{i}(t_j))$ is a Gaussian random variable with zero mean and variance $\sigma^2_{i}(t_j)$. The uncertainty models the deviation from the expected value of the kinematic constraint.
As an example, \eqref{eq:rst_perturbation} can be used to analyze how the noise in the E-FMS system affects the kinematic constraints, thus the polynomial trajectory.
Due to the intrinsic linearity of RST and the perturbation model, the following theorem proves that it is possible to translate the uncertainties in the constraints into uncertainties in the polynomial coefficients.

\begin{theorem}
\label{lemma:rst_Lemma6}
Let $\frac{d^i}{dt^i}\tilde{x}_k(t)\bigr|_{t=t_j}$ be the perturbed kinematic constraints for $i=0,1,\dots, k$ and let $x_k(t)$ be a feasible polynomial trajectory computed with RST. Then $\tilde{x}_k(t)$ represents the perturbed polynomial trajectory and it can be written as
\begin{equation}
\tilde{x}_k(t) = x_k(t)+r_k(t),
\end{equation} 
where $r_k(t)$ is a random polynomial whose coefficients belong to a multivariate Gaussian distribution $\mathcal{N(\mathbf{\mu,\Sigma})}$.
\end{theorem}

\begin{proof}
To prove the theorem we will proceed by induction on the ($i-1$)-partial trajectory. Consider the base case when $i=0$, i.e., $\tilde{x}_0(t)$. From Sec. \ref{subsec:rst_Lagrange}, 
\begin{equation}
\label{eq:rst_lagrange_perturbated}
\tilde{x}_0(t) = \sum_{j=0}^{N}{\tilde{x}(t_j) \ell_j(t)}
\end{equation}
where $\ell_j(t)$ is the Lagrange basis. By substituting the perturbation expression for the constraints \eqref{eq:rst_perturbation} in \eqref{eq:rst_lagrange_perturbated}, it follows that 
\begin{align}
\tilde{x}_0(t) & = \sum_{j=0}^{N}{x(t_j) \ell_j(t)}+\sum_{j=0}^{N}{\varepsilon_0(t_j) \ell_j(t)} \nonumber \\ 
& = x_0(t)+r_0(t).
\end{align}

Suppose that the statement of the theorem is true for the $(i-1)$-partial trajectory, which means that
\begin{equation}
\tilde{x}_{i-1}(t)= x_{i-1}(t)+r_{i-1}(t).
\end{equation}
Then, it is true also for the $i$-partial trajectory. Indeed, from the RST algorithm derived in Theorem \ref{theorem:rst_theorem1},
\begin{equation}
\tilde{s}_i(t_j)=\dfrac{\dfrac{d^i}{dt^i}\tilde{x}_k(t)\biggr|_{t=t_j} -\dfrac{d^i}{dt^i}\tilde{x}_{i-1}(t)\biggr|_{t=t_j}}{\displaystyle \Biggl(\prod_{\substack{n=0\\ n\neq j}}^{N}{(t_j-t_n)}\Biggr)^i},
\label{eq:rst_perturbed_recursive}
\end{equation}
and using the induction hypothesis and the linearity of the differential operator,

\begin{align}
\tilde{s}_i(t_j)& = \dfrac{\Biggl(\dfrac{d^i}{dt^i}x_k(t)\biggr|_{t=t_j}+\varepsilon_i(t_j)\Biggr) -\Biggl(\dfrac{d^i}{dt^i}x_{i-1}(t)\biggr|_{t=t_j}+\dfrac{d^i}{dt^i}r_{i-1}(t)\biggr|_{t=t_j}\Biggr)}{\displaystyle \Biggl(\prod_{\substack{n=0\\ n\neq j}}^{N}{(t_j-t_n)}\Biggr)^i} \nonumber \\
&= \dfrac{\dfrac{d^i}{dt^i}x_k(t)\biggr|_{t=t_j}-\dfrac{d^i}{dt^i}x_{i-1}(t)\biggr|_{t=t_j}}{\displaystyle \Biggl(\prod_{\substack{n=0\\ n\neq j}}^{N}{(t_j-t_n)}\Biggr)^i}+\dfrac{\varepsilon_i(t_j)-\dfrac{d^i}{dt^i}r_{i-1}(t)\biggr|_{t=t_j}}{\displaystyle \Biggl(\prod_{\substack{n=0\\ n\neq j}}^{N}{(t_j-t_n)}\Biggr)^i} \nonumber \\
&= s_i(t_j)+s_i^{\varepsilon}(t_j).
\end{align}

This result is important because it separates $\tilde{s}_i(t_j)$ in two components. Using again Lagrange interpolation as done for $i=0$ yields to $\tilde{s}_i(t) = s_i(t)+s_i^{\varepsilon}(t)$. By the RST properties and definition of $i$-partial trajectory 
\begin{align}
\tilde{x}_i(t) &= \tilde{x}_{i-1}(t)+\frac{a^i(t)}{i!}\cdot \tilde{s}_i(t) \nonumber \\
&= x_{i-1}(t)+r_{i-1}(t)+ \frac{a^i(t)}{i!}\cdot (s_i(t)+s_i^{\varepsilon}(t)) \nonumber \\
&= x_{i-1}(t)+\frac{a^i(t)}{i!}\cdot s_i(t) + r_{i-1}(t)+ \frac{a^i(t)}{i!}\cdot s_i^{\varepsilon}(t) \nonumber \\
&= x_i(t)+r_i(t).
\label{eq:rst_derivation_perturbation}
\end{align}
Calculating \eqref{eq:rst_derivation_perturbation} in $i=k$ concludes the proof because $r_k(t)$ is a polynomial whose coefficients are a weighted sum of the uncertainties $\varepsilon_i$ in the constraints, for $i=0,1,\dots,k$. The way in which the random polynomial coefficients depend one to each other is described by the covariance matrix $\mathbf{\Sigma}$ of a multivariate Gaussian distribution, which has to be estimated.
\end{proof}
To characterize and generate new perturbed trajectories, an estimation of the multivariate Gaussian distribution described by the random coefficients in $r_k(t)$ has to be carried out.

\subsection{Coefficients estimation}
The coefficients of the random polynomial $r_k(t)$ incorporate the stochastic information of the constraints. 
Under the Gaussian hypothesis, it is straightforward to state that the coefficients are themselves Gaussian random variables since they are the outcome of linear combinations of $\varepsilon_i(t_j)$. However, the way in which the random variables $\varepsilon_i(t_j)$ interact and correlate one to each other strongly depends on the points in time $t_j$, in particular on the differences $t_j-t_n$, for $j=0,\dots,N$ and $n=0,\dots,N$ with $j\neq n$. Therefore, no closed form expression is available to describe the covariance matrix $\mathbf{\Sigma}$. Another approach consists of estimating $\mathbf{\Sigma}$, and we will refer to $\mathbf{\hat{\Sigma}}$ as the estimated version.

To compute $\mathbf{\hat{\Sigma}}$, the procedure involves the following steps:
\begin{itemize}
\item take $P$ realizations of the uncertainty in the constraints $\varepsilon_i(t_j)\sim \mathcal{N}(0,\sigma^2_{i}(t_j))$;
\item for each realization, evaluate the coefficients of the trajectory $r_{k}(t)$ generated via RST;
\item evaluate the sample covariance matrix $\mathbf{\hat{\Sigma}}$ of the stored coefficients.
\end{itemize}
The sample covariance matrix $\mathbf{\hat{\Sigma}}$ is the unbiased estimator of the covariance matrix $\mathbf{\Sigma}$.

\subsection{Random trajectory generation}
The multivariate Gaussian distribution is characterized by the mean (in this case $\mathbf{\mu} = \mathbf{0}$) and the covariance matrix $\mathbf{\Sigma}$. To generate new polynomial coefficients, thus new feasible trajectories, it is enough to sample from the multivariate distribution. In particular, consider a vector $\mathbf{z}$ of uncorrelated normal random variables. If the matrix $\mathbf{C}$ is a square root of $\mathbf{\Sigma}$, such as $\mathbf{C}\cdot \mathbf{C}^T = \Sigma$ (for example using the Cholesky decomposition), it follows that $\mathbf{y}=\mathbf{\mu} + \mathbf{C}\cdot \mathbf{z}$ is a vector of Gaussian random variables representing the coefficients of the random polynomial $r_k(t)$.

\begin{algorithm}
\caption{Recursive smooth random trajectory (RSRT) generation}
\label{alg:RSRT}
\begin{algorithmic}[1]
\Inputs{$N+1$ points in time $t_0<t_1<\dots<t_N$; \\ Kin. constr. $\frac{d^i}{dt^i}x_k(t)\bigr|_{t=t_0}, \dots, \frac{d^i}{dt^i}x_k(t)\bigr|_{t=t_N}$; \\ Uncertainties $\varepsilon_i(t_j)$.}
\Initialize{$a(t)=(t-t_0)\cdots (t-t_N)$ polynomial; \\ $P$ number of realizations; \\ $x_{-1}(t) \equiv 0$ \\ $r_{-1}(t) \equiv 0$.}
\State Compute $x_k(t)$ with RST;
\For{$p=1$ to $P$}
		\State Get realizations of $\varepsilon_i(t_j)$;
		\State Compute $r_{k,p}(t)$ with RST and constraints $\varepsilon_i(t_j)$;
	\EndFor
	\State Estimate $\mathbf{\hat{\Sigma}}$ of the coefficients of $r_k(t)$;
	\State Generate a set of coefficients of $r_k(t)$;
		\State Generate a perturbed trajectory $\tilde{x}_k(t)=x_k(t)+r_k(t)$;
\end{algorithmic}
\end{algorithm}

This process offers a fast methodology for generating perturbed trajectories $\tilde{x}_k(t)=x_k(t)+r_k(t)$ from the estimated coefficients in $r_k(t)$. We will refer to it as recursive smooth random trajectory (RSRT) generation. The algorithm is described in Alg.~\ref{alg:RSRT}

So far we have only considered a multivariate Gaussian distribution for $\varepsilon_i(t_j)$. Nevertheless as a consequence of the central limit theorem, whenever the uncertainties have different distribution, correlated Gaussian random variables can approximate the statistics of the polynomial coefficients of $r_k(t)$.
\section{RST in an optimization framework}
\label{subsec:rst_optimization}
Most of trajectory generation and path planning research concentrates in finding an optimal trajectory that minimizes a cost function $J(\cdot)$ under given constraints. Nevertheless, trajectory generation does not necessarily require optimality in the solution. In this section, we present an example of optimization framework built around the RST algorithm and we prove that when the number of waypoints $N+1$ is equal to $2$, RST (or PRST) directly provides the optimal solution in terms of minimum integral of the $p$-th derivative of the position squared, matching the trajectory generated by minimum-snap algorithm \cite{5980409} without any use of quadratic programming.

Sec.~\ref{subsec:rst_results} illustrated the RST algorithm, which is able to generate a polynomial trajectory $x_k(t)$ with minimum degree that satisfies the constraints. However, it is easy to notice that the general set of feasible trajectories is induced by $q(t)$ as follows
\begin{equation}
x_{\text{ext}}(t) = x(t)+\frac{a^{k+1}(t)}{(k+1)!}\cdot q(t),
\end{equation}
where $q(t)$ is a polynomial which introduces extra degrees of freedom needed for an optimization phase and $x(t)=x_k(t)$ is the trajectory generated via RST. 

As an example of an optimization problem, let $p = k+1$ be the order of the derivative of $x_{\text{ext}}(t)$ whose energy has to be minimized. A possible approach finds the solution to
\begin{equation}
\min_{q(t)}{\int_{t_0}^{t_N}{\biggl|\biggl|\frac{d^p}{dt^p}\biggl(x(t)+\frac{a^p(t)}{p!}\cdot q(t)\biggr)\biggr|\biggr|^2 dt}}
\end{equation}
with $q(t)$ polynomial, providing the optimal trajectory as
\begin{equation}
x_{\text{opt}}(t)=x(t)+\frac{a^p(t)}{p!}\cdot q_{\text{opt}}(t).
\end{equation}
Since all the functions inside the functional are polynomials, the coefficients of $q(t)$ can be in principle expressed analytically by integrating polynomials and by solving a system of linear equations. The convexity of the norm squared function guarantees a global minimum. 

When the number of waypoints is equal to $2$, that is $N=1$, the following Lemma asserts the optimality (in terms of energy) of the trajectory $x_k(t)$ generated with RST.

\begin{lemma}
\label{lemma:rst_Lemma5}
Let $x_k(t)$ be the trajectory generated with RST which satisfies the given kinematic constraints $\frac{d^i}{dt^i}x_k(t)\bigr|_{t=t_j}$ for $i=0,1,\dots, k$ and $j=0,1,\dots, N$. If $N=1$ and $p=k+1$, then the solution to
\begin{equation}
\label{prob:rst_functional}
\min_{q(t)}{\int_{t_0}^{t_1}{\biggl|\biggl|\frac{d^p}{dt^p}\biggl(x_{p-1}(t)+\frac{a^p(t)}{p!}\cdot q(t)\biggr)\biggr|\biggr|^2 dt}}
\end{equation}
is $q_{\text{opt}}(t)=0$, therefore the trajectory generated with RST is already the optimal one.
\end{lemma}

\begin{proof}
The proof uses some concepts of calculus of variations. In particular, let $\mathcal{L}$ be a Lagrangian function defined as 
\begin{equation}
\label{eq:rst_Lagrangian}
\mathcal{L} = \biggl(\frac{d^p x_{p-1}(t)}{dt^p}+\frac{d^p f(t)}{dt^p}\biggr)^2,
\end{equation}
with 
\begin{equation}
f(t) = \frac{a^p(t)}{p!}\cdot q(t).
\end{equation}
From calculus of variations theory, solving problem \eqref{prob:rst_functional} is equal to solving the Euler-Lagrange equation
\begin{equation}
\small
\label{eq:rst_EL}
\frac{\partial \mathcal{L}}{\partial f} - \frac{d}{dt}\biggl(\frac{\partial \mathcal{L}}{\partial \dot{f}}\biggr)+ \frac{d^2}{dt^2}\biggl(\frac{\partial \mathcal{L}}{\partial \ddot{f}}\biggr)-\dots+(-1)^{p}\frac{d^{p}}{dt^{p}}\biggl(\frac{\partial \mathcal{L}}{\partial f^{(p)}}\biggr)=0
\end{equation}
and by substituting the Lagrangian defined in~\eqref{eq:rst_Lagrangian} into~\eqref{eq:rst_EL}
it follows that 
\begin{equation}
\frac{d^{2p}}{dt^{2p}}\biggl(x_{p-1}(t)+f(t)\biggr)=0.
\end{equation}
From the considerations in Corollary \ref{corollary:rst_corollary1}, 
\begin{align}
\text{deg}(x_{p-1}(t))&=2p-1, \nonumber \\
\text{deg}(f(t))&= \text{deg}(a(t))+\text{deg}(q(t)) = 2p+Q,
\end{align}
with $Q\geq 0$. But, since $x$ and $f$ are polynomials, each differentiation reduces the degree by one and
\begin{equation}
\text{deg}\Bigg(\frac{d^{2p}}{dt^{2p}}\biggl(x_{p-1}(t)+f(t)\biggr)\Biggr)=Q=0,
\end{equation}
therefore $\text{deg}(q(t))=Q=0$ and in particular $q(t)\equiv 0$. \qedhere
\end{proof}
When the number of blocks $M$ increases, the overall optimal trajectory is obtained by optimizing the trajectories in each block.
When the number of waypoints in a single block is greater than $2$, the intrinsic optimality of the trajectory generated with RST is not guaranteed anymore and the optimization process provides $q(t)\neq 0$.  Next section illustrates trajectories generated via RST and via optimization of the integral of the $p$-th derivative of the position squared, denoted with RST$_{\text{opt}}$. 

\section{Applications}
\label{subsec:rst_examples}
In this section, we firstly discuss two different case-studies based on the number of waypoints $N+1$. For both of them we present the trajectory generated with RST, we numerically describe the influence of uncertainties in the kinematic constraints and we lastly compare the RST with the optimized extension introduced in Sec.~\ref{subsec:rst_optimization} (RST$_{\text{opt}}$), and with the minimum-snap piecewise polynomial trajectory obtained as proposed in \cite{5980409}.
In the last part, we consider a 2D planner quadrotor with highly input coupling   to assess the tracking performance capability and highlight the advantage of the RST algorithm over a piecewise polynomial method such as the spline interpolation technique or the minimum-snap approach.

\subsection{Scenario 1}
As a first example, we consider a rest-to-rest maneuver. Such trajectory satisfies the following deterministic kinematic constraints in a time-span of $10$ seconds
\begin{equation*}
  \dfrac{d^i}{d t^i}x_k(t)=
  \begin{blockarray}{*{3}{c} l}
    \begin{block}{*{3}{>{$\footnotesize}c<{$}} l}
      $t_0$ & $t_1$ & $t_2$ & \\
    \end{block}
    \begin{block}{[*{3}{c}]>{$\footnotesize}l<{$}}
      0 & 2 & 0 \bigstrut[t]& space \\
       0 & 0 & 0 & velocity \\
       0 & 0 & 0 & acceleration \\
       0 & 0 & 0 & jerk \\
    \end{block}
  \end{blockarray}
\end{equation*}
and can be determined using the RST algorithm, which generates a polynomial trajectory $x_k(t)$ of degree $(k+1)(N+1)-1=4\cdot 3-1=11$. 
A perturbation is later applied in some of the constraints, in particular
\begin{equation*}
  \varepsilon_i(t)=
  \begin{blockarray}{*{3}{c} l}
    \begin{block}{*{3}{>{$\footnotesize}c<{$}} l}
      $t_0$ & $t_1$ & $t_2$ & \\
    \end{block}
    \begin{block}{[*{3}{c}]>{$\footnotesize}l<{$}}
      0 & \varepsilon_0 & 0 \bigstrut[t]& space \\
       0 & \varepsilon_1 & 0 & velocity \\
       0 & 0 & 0 & acceleration \\
       0 & 0 & 0 & jerk \\
    \end{block}
  \end{blockarray}
\end{equation*}

\begin{figure}
\includegraphics[scale=0.82]{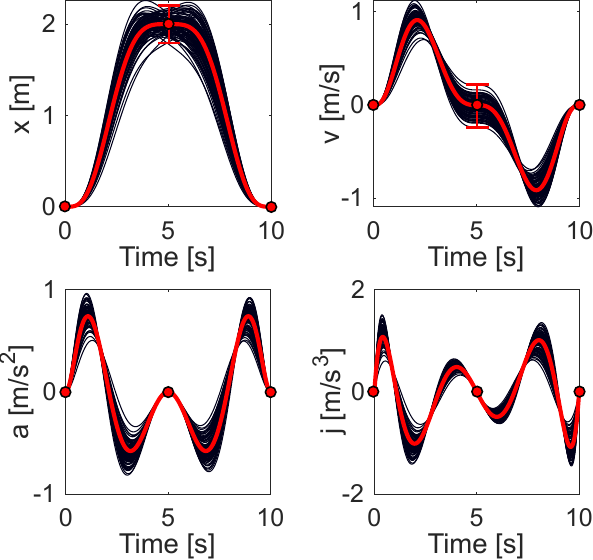}
\centering
\caption{Deterministic trajectory (in red) and $100$ realizations of perturbed trajectories (in black) for the case of $3$ waypoints equally spaced in time.}
\label{fig:rst_RSRT_1}
\end{figure}

\begin{figure}
\includegraphics[scale=0.58]{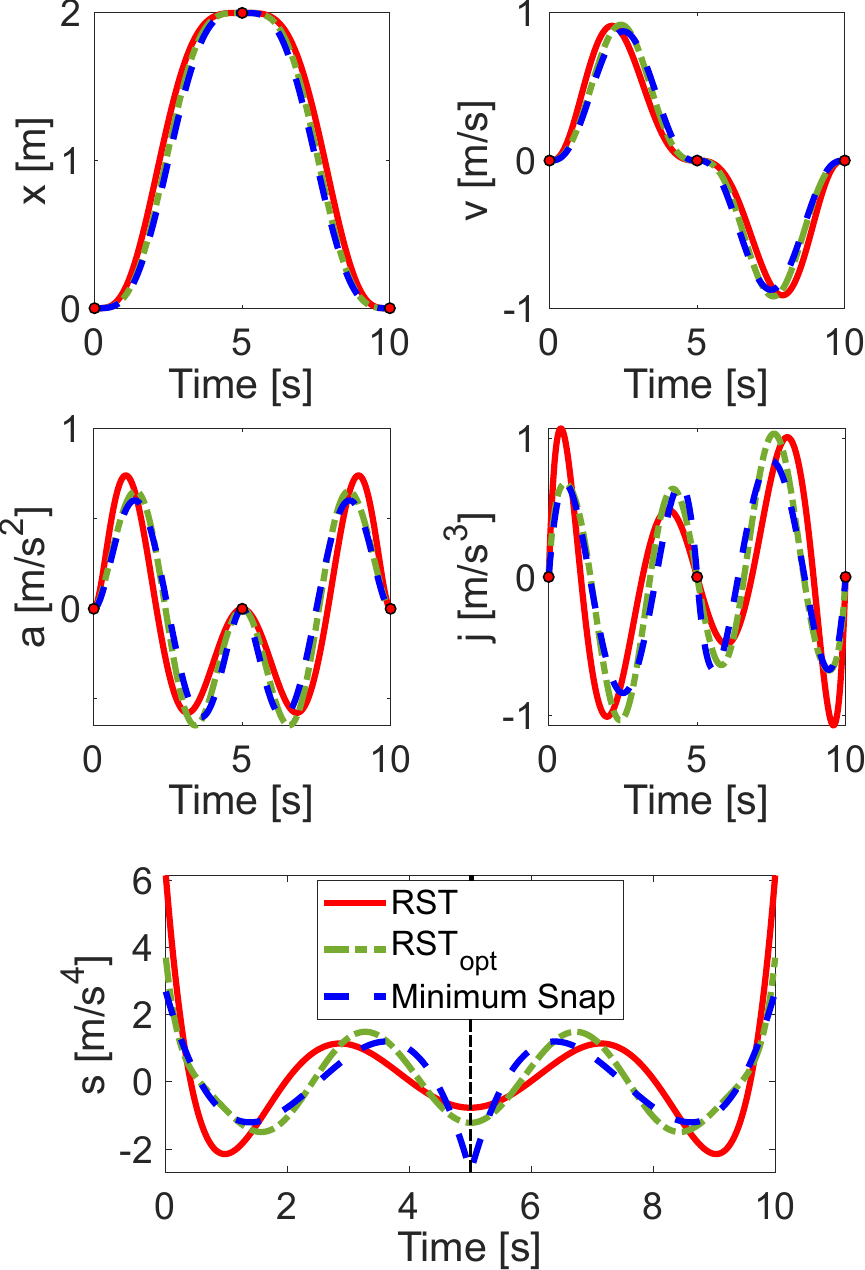}
\centering
\caption{Comparison of trajectories and derivatives generated with RST, RST optimized and minimum-snap algorithm for the case of $3$ waypoints equally spaced in time.}
\label{fig:rst_RST_1}
\end{figure}

with $\varepsilon_0 \sim \mathcal{N}(0,0.1^2)$ and $\varepsilon_1 \sim \mathcal{N}(0,0.1^2)$. From Alg.~\ref{alg:RSRT}, the estimation and generation of the covariance matrix $\hat{\Sigma}$ describing the joint statistics of the coefficients of $r_k(t)$ provides several perturbed trajectories $\tilde{x}_k(t)$ that are reported in Fig.~\ref{fig:rst_RSRT_1}, top left corner. In particular, the red line represents the fixed trajectory $x_k(t)$ obtained via RST, while the thin black lines are $100$ stochastic realizations of the estimated perturbed version $\tilde{x}_k(t)$. To show the consistency of the RSRT method, velocity, acceleration and jerk trends are also illustrated.

When the number of waypoints is greater than $2$, RST does not provide directly the optimal solution in terms of minimum integral of the snap squared. Therefore, an optimization step looks for the best polynomial $q(t)$ (in this case of degree $Q = 2$) which minimizes the snap, without interfering with the kinematic constraints thanks to the pre-multiplicative factor $a^p(t)$. 

Fig.~\ref{fig:rst_RST_1} compares the trajectories and the respective derivatives trends obtained using three different approaches, RST, RST$_{\text{opt}}$ and minimum-snap piecewise. Despite the fact that up to the jerk the three curves look similar, the evolution of the snap over time provides insightful observations. The snap obtained through RST is continuous by construction but has elongations in the edges which are more pronounced when the number of waypoints increases, as the consequence of Runge's phenomenon. The snap obtained using piecewise polynomials is the minimum in terms of integral but reveals a discontinuity jump which can produce undesired vibrations causing aging and damages in the UAVs structure \cite{Letizia2020_rst}. However, the snap that comes from the optimization of RST is continuous, optimal and in most of cases can be interpreted as a continuous approximation of the piecewise minimum snap, resulting in a good trade-off between complexity and continuity.

\subsection{Scenario 2}
As a second case-study, we consider a trajectory passing through $6$ waypoints. In order to avoid Runge's phenomenon, we choose the points in time (from $0$ to $10$ seconds) according to \eqref{Cheby} and in each of them we impose the following deterministic kinematic constraints

\begin{equation*}
  \dfrac{d^i}{d t^i}x_k(t)=
  \begin{blockarray}{*{6}{c} l}
    \begin{block}{*{6}{>{$\footnotesize}c<{$}} l}
      $t_0$ & $t_1$ & $t_2$ & $t_3$ & $t_4$ & $t_5$ \\
    \end{block}
    \begin{block}{[*{6}{c}]>{$\footnotesize}l<{$}}
      0 & 1 & -1 & 1 & -1 & 0 \bigstrut[t]& space \\
       0 & 0 & 0 & 0 & 0 & 0 & velocity \\
       0 & 0 & 0 & 0 & 0 & 0 & acceleration \\
       0 & 0 & 0 & 0 & 0 & 0 & jerk \\
    \end{block}
  \end{blockarray}
\end{equation*}

A perturbation is later applied in some of the constraints as follows
\begin{equation*}
  \varepsilon_i(t)=
  \begin{blockarray}{*{6}{c} l}
    \begin{block}{*{6}{>{$\footnotesize}c<{$}} l}
      $t_0$ & $t_1$ & $t_2$ & $t_3$ & $t_4$ & $t_5$ \\
    \end{block}
    \begin{block}{[*{6}{c}]>{$\footnotesize}l<{$}}
      0 & \varepsilon_0 & \varepsilon_0 & \varepsilon_0 & \varepsilon_0 & 0 \bigstrut[t]& space \\
       0 & \varepsilon_1 & \varepsilon_1 & \varepsilon_1 & \varepsilon_1 & 0 & velocity \\
       0 & 0 & 0 & 0 & 0 & 0 & acceleration \\
       0 & 0 & 0 & 0 & 0 & 0 & jerk \\
    \end{block}
  \end{blockarray}
\end{equation*}
with $\varepsilon_0 \sim \mathcal{N}(0,0.1^2)$ and $\varepsilon_1 \sim \mathcal{N}(0,0.1^2)$.
\begin{figure}
\includegraphics[scale=0.82]{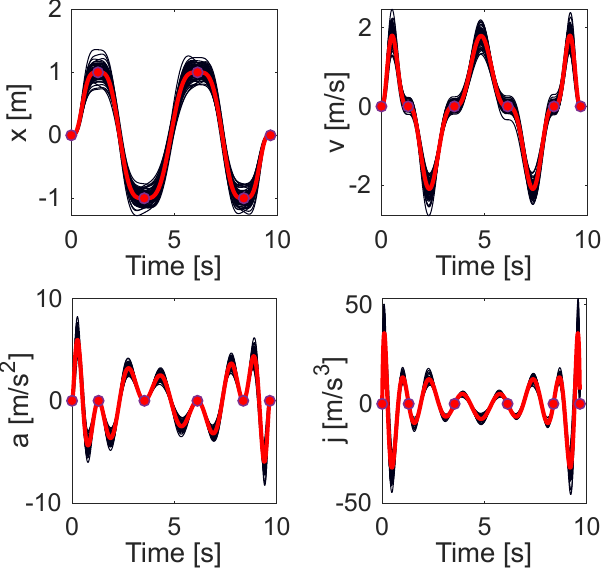}
\centering
\caption{Deterministic trajectory (in red) and $100$ realizations of perturbed trajectories (in black) with $6$ waypoints and Chebyshev nodes in time.}
\label{fig:rst_RSRT_2}
\end{figure}
\begin{figure}
\includegraphics[scale=0.58]{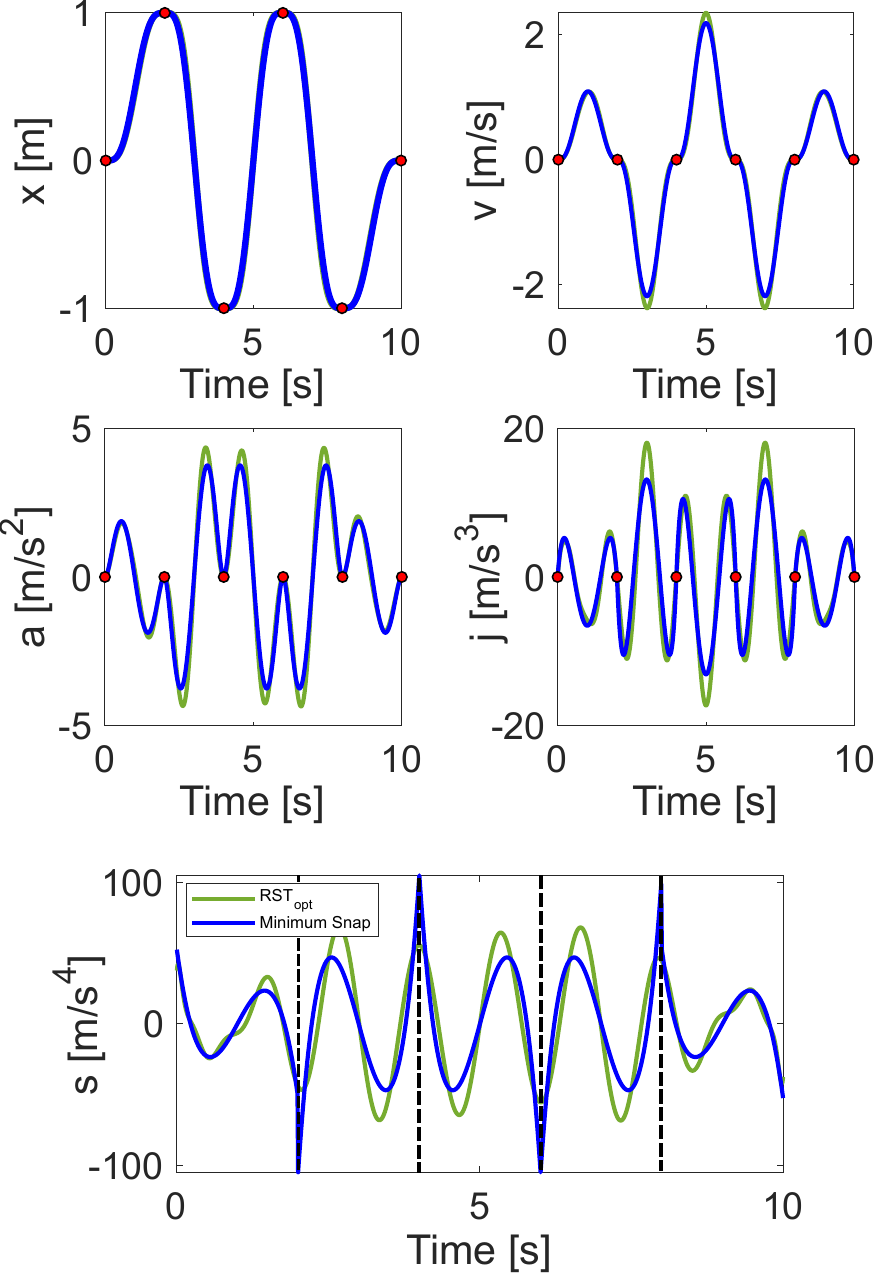}
\centering
\caption{Comparison of trajectories and derivatives generated with RST$_{\text{opt}}$ and minimum-snap algorithm for the case of $6$ waypoints equally spaced in time.}
\label{fig:rst_RST_2}
\end{figure}
Proceeding in the same way aforementioned for the first scenario, brings to a set of feasible trajectories described in Fig.~\ref{fig:rst_RSRT_2}. It is evident that the trajectories accept uncertainties only in the kinematic constraints involving space and velocity, as the initial hypothesis.

To show the generality of the optimization step for the snap's analysis, we decided to consider points in time equally spaced. In such situation, the RST gives an overshoot trajectory which the optimization will try to compensate, so we only compare RST$_{\text{opt}}$ with minimum-snap piecewise. 

As in the previous example, the trajectories and the derivatives up to jerk appear to be similar. Nevertheless, when looking to the snap trends, the snap generated via optimized RST results smooth while the snap generated via minimum-snap piecewise results discontinuous. Smoothness in the trajectories is often a wanted property, especially from a control perspective since real-time control laws to track desired trajectories need to be given by the E-FMS. 

\subsection{RST tracking performance}
In the following, we extend the 1D recursive trajectory generation method to the 3D case. To achieve that, we exploit the flatness property \cite{6324664} of UAVs.
In particular, we conduct a close-to-real simulation of a highly non-linear dynamical system model of a 3D quadrotor \cite{1570447} to assess the capability and effectiveness of the RST. 
A natural choice of the flat outputs is $\gamma(t) = [x(t)\; y(t)\; z(t)\; \psi(t)]^T$. The components $x,y,z$ represent the position of the center of mass of the quadrotor in the inertial reference frame, while $\psi$ is the yaw angle. The trajectory is defined as a smooth curve in the space of flat outputs $\gamma(t): [t_0,t_N] \to \mathbb{R}^3 \times SO(2)$. 

\subsection{Control inputs smoothness}
A simple rest to rest maneuvering with the following kinematic constraints on $x,y$ and $z$ is considered as the flying scenario
\begin{equation*}
  \dfrac{d^i}{d t^i}x_k(t)=
  \begin{blockarray}{*{3}{c} l}
    \begin{block}{*{3}{>{$\footnotesize}c<{$}} l}
      $t_0$ & $t_1$ & $t_2$ & \\
    \end{block}
    \begin{block}{[*{3}{c}]>{$\footnotesize}l<{$}}
      0 & 2 & 4 \bigstrut[t]& space \\
       0 & 0.5 & 0 & velocity \\
       0 & 0 & 0 & acceleration \\
       0 & 0 & 0 & jerk \\
    \end{block}
  \end{blockarray}
\end{equation*}
\begin{equation*}
  \dfrac{d^i}{d t^i}y_k(t)=
  \begin{blockarray}{*{3}{c} l}
    \begin{block}{*{3}{>{$\footnotesize}c<{$}} l}
      $t_0$ & $t_1$ & $t_2$ & \\
    \end{block}
    \begin{block}{[*{3}{c}]>{$\footnotesize}l<{$}}
      0 & 2 & 0 \bigstrut[t]& space \\
       0 & -0.5 & 0 & velocity \\
       0 & 0 & 0 & acceleration \\
       0 & 0 & 0 & jerk \\
    \end{block}
  \end{blockarray}
\end{equation*}
and
\begin{equation*}
  \dfrac{d^i}{d t^i}z_k(t)=
  \begin{blockarray}{*{3}{c} l}
    \begin{block}{*{3}{>{$\footnotesize}c<{$}} l}
      $t_0$ & $t_1$ & $t_2$ & \\
    \end{block}
    \begin{block}{[*{3}{c}]>{$\footnotesize}l<{$}}
      0 & 5 & 0 \bigstrut[t]& space \\
       0 & 0 & 0 & velocity \\
       0 & 0 & 0 & acceleration \\
       0 & 0 & 0 & jerk \\
    \end{block}
  \end{blockarray}
\end{equation*}

\begin{figure}
\includegraphics[scale=0.32]{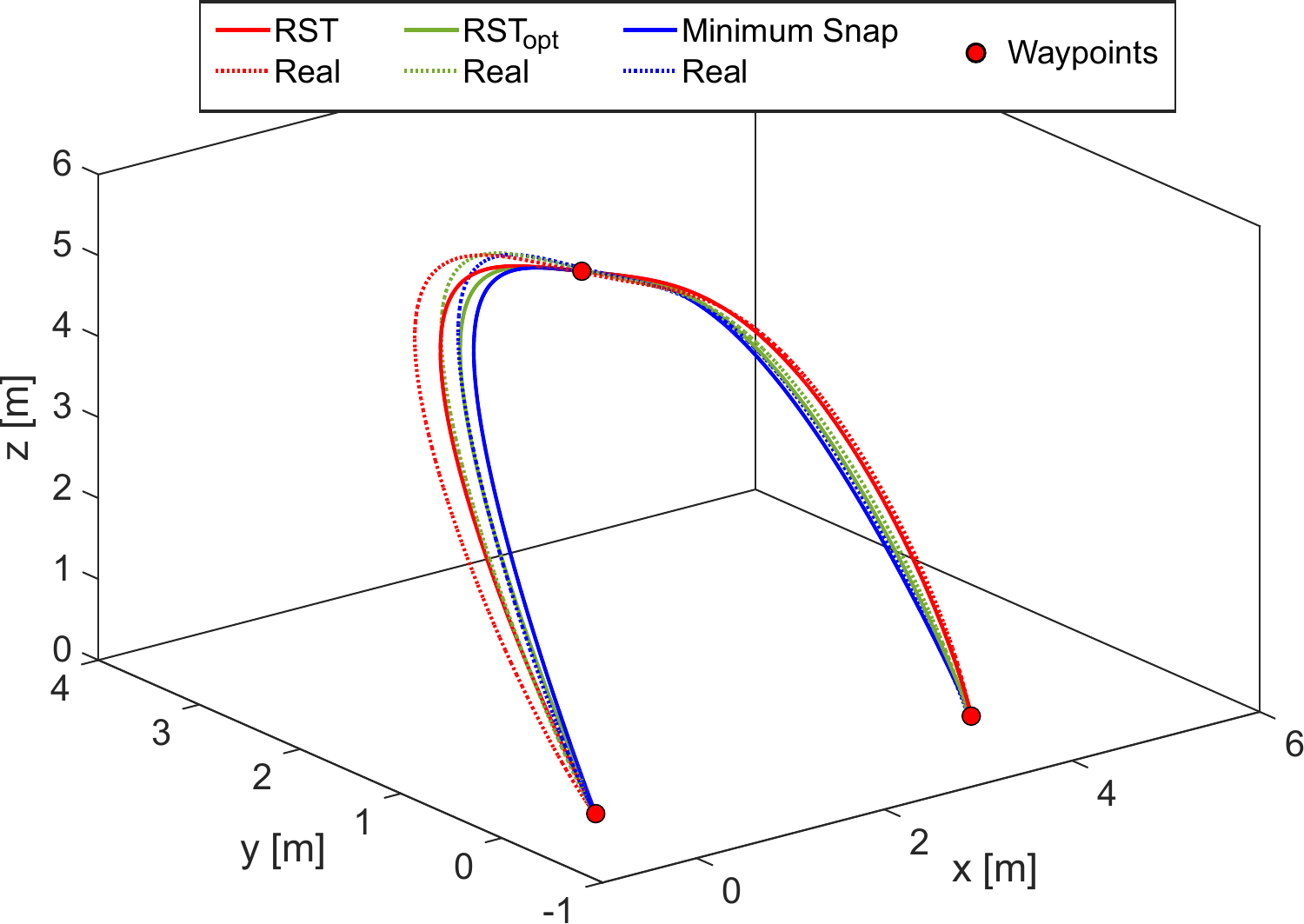}
\centering
\caption{Quadrotor tracking (dashed line) of 3D trajectories (solid line) generated by RST (red), RST$_{\text{opt}}$ (green) and minimum-snap (blue) approaches.}
\label{fig:rst_quadrator_curve}
\end{figure} 

with $\psi(t)\equiv 0$ in $[t_0,t_N]$. 
The sampling time is set to $0.01$ s and the physical parameters of the quadrotor helicopter are taken according to \cite{1570447}. It should be noted that the non-linear dynamical system model of the quadrotor allows studying and testing the robustness of the methodologies applied. In particular, Fig.~\ref{fig:rst_quadrator_curve} shows the tracking of the 3D desired trajectories using the proposed approaches: RST, RST$_{\text{opt}}$ and minimum-snap.

\begin{figure}
\includegraphics[scale=0.35]{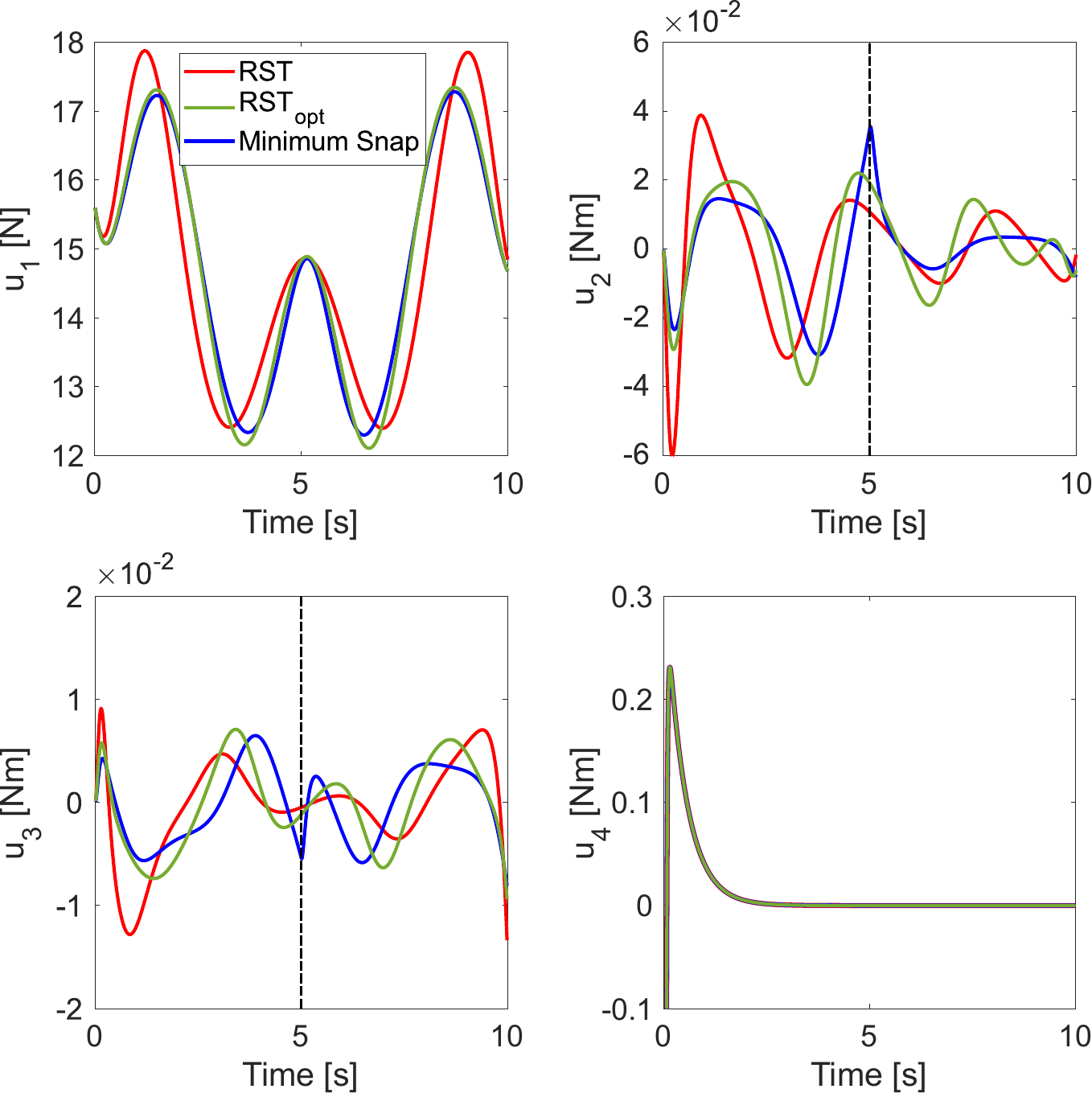}
\centering
\caption{Evolution of control inputs of the system $u_1, u_2, u_3$ and $u_4$ over time for three different approach: RST, RST$_{\text{opt}}$ and minimum-snap.}
\label{fig:rst_control_inputs}
\end{figure} 

In all the cases studied, the controller succeeds in following the desired path. However, the main advantage of adopting a smooth path as RST is depicted in Fig.~\ref{fig:rst_control_inputs}. The RST and RST$_{\text{opt}}$ approaches perform better than the minimum-snap one in terms of smoothness of the control signals. This is a direct consequence of the fact that the control inputs ($u_2$ and $u_3$) responsible for the roll and pitch angle, depend on the $4$-th derivative of the position, namely the snap. Therefore, peaks and discontinuities in high order derivatives can have an impact on the continuity of the control inputs, causing aging and damaging of the quadrotor mechanical structure. 

The following section discusses this problem in detail and provides a possible solution as a trade-off between complexity and discontinuities. 
\section{Comments and further discussion}
\label{subsec:rst_comment}
\subsection{Accuracy vs discontinuities}
\begin{figure}
\includegraphics[scale=0.50]{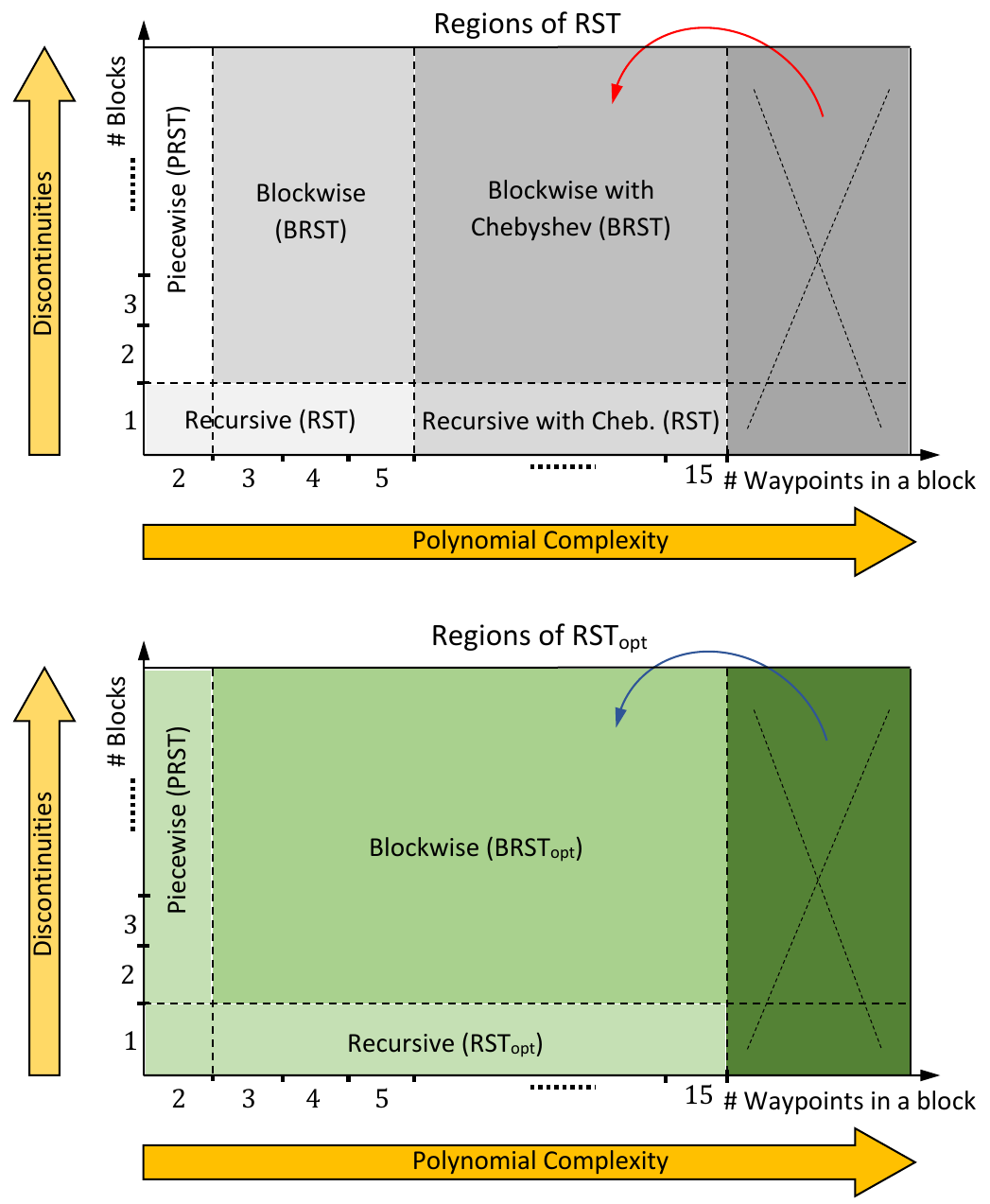}
\centering
\vspace{-0.1em}
\caption{Complexity and smoothness trade-off regions of RST algorithm depending on number of blocks and waypoints in each block.}
\vspace{-1.0em}
\label{fig:rst_regions}
\end{figure}
In this section, we present a useful rule of thumb for applying the RST and the RST$_{\text{opt}}$ algorithm according to the number of blocks the user considers and the number of waypoints inside each of these blocks. 
Indeed, Fig.~\ref{fig:rst_regions} and Fig.~\ref{fig:rst_example1} identify regions, suggesting which algorithm to apply according to the number of waypoints and blocks.
So far, most of path planning research has been focused on piecewise polynomial approaches, sticking to the first left strip in Fig.~\ref{fig:rst_regions}. There are, indeed, cases where numerical polynomial complexity is the main concern and in such situations a standard piecewise polynomial approach is sufficient. Nevertheless, we propose to use the PRST algorithm to compute the piecewise polynomial trajectory (e.g. RST with $2$ points in each block). The reason for this choice comes from the intrinsic optimality of RST when the number of waypoints in each block is $2$, as proved in Lemma \ref{lemma:rst_Lemma5}.

An unwanted side effect of the piecewise choice is the presence of discontinuities in the $p$-th derivative's interface. To reduce and eventually remove them, a good compromise is the BRST algorithm which balances polynomial complexity and discontinuity issues. For instance, instead of using $15$ piecewise polynomials for a trajectory consisting of $16$ points, one could choose to split the trajectory generation in $5$ blocks of $4$ points each, leading to a reduction of discontinuities, from $14$ to $4$ matching interfaces, while at the same time keeping a low level of complexity in polynomials. See Fig.~\ref{fig:rst_example1} for an example of such blockwise trajectory.
Whenever the complexity is not the issue to consider at first, one could build a smooth trajectory which passes through all the waypoints. In such a case RST is one possible way to proceed if no optimal condition is required, otherwise RST$_{\text{opt}}$ provides the optimality at expenses of a higher computational cost. 
Moreover, we empirically found that smooth trajectories with more than $15$ waypoints in the same block are unstable, therefore we suggest to always split them into at least $2$ consecutive blocks.
Fig.~\ref{fig:rst_regions} and \ref{fig:rst_example1} illustrate these concepts and guide the user to select the proper algorithm according to given specifics.

\begin{figure}
\includegraphics[scale=0.37]{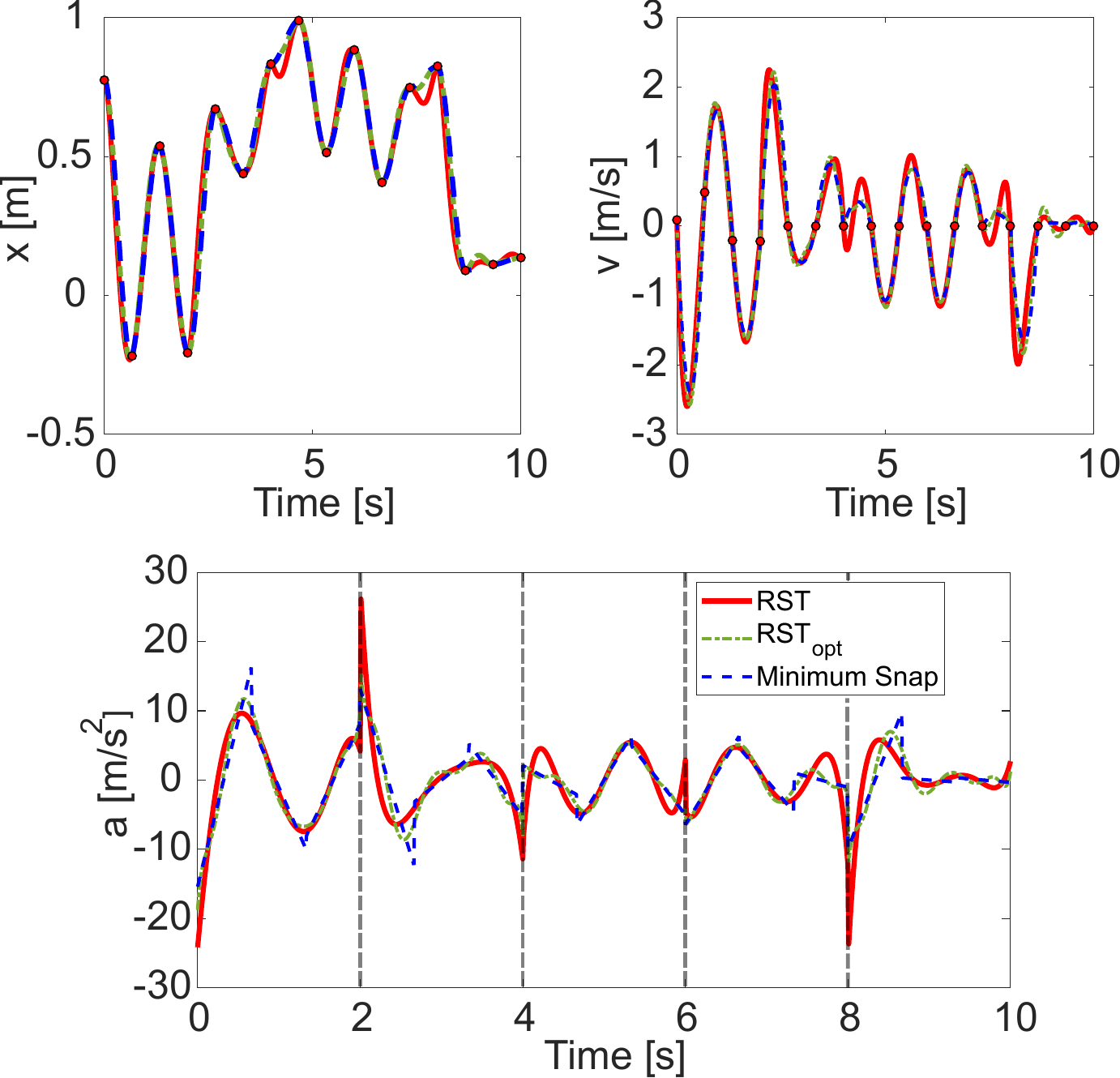}
\centering
\vspace{-0.1em}
\caption{Example of blockwise trajectory (BRST, BRST$_{\text{opt}}$ and minimum-snap) that minimizes the integral of the acceleration squared.}
\vspace{-1.0em}
\label{fig:rst_example1}
\end{figure}

\subsection{Memory requirements}
\label{subsec:rst_time}
The number of coefficients of the RST polynomial trajectory to store is $(k+1)(N+1)$ (See Corollary \ref{corollary:rst_corollary1}) where $N+1$ is the number of waypoints and $k$ is the last kinematic constraint. Extending the results also to BRST leads to $M$ polynomials of degree $(k+1)(\frac{N+M}{M})-1$. Thus, the number of coefficients to store is equal to $(k+1)(N+M)$. The ratio between the number of coefficients to store for BRST and RST is equal to $1+\frac{M-1}{N+1}$ which means that BRST requires $100\cdot \frac{M-1}{N+1}$ percent more memory than RST. As an extreme case, PRST is the technique which requires the highest memory requirements since it stores $2N(k+1)$ coefficients, almost twice the memory required by RST.

\subsection{Computational complexity}
\label{subsec:rst_complexity}
To evaluate the computational complexity, it is convenient to segment the RST algorithm (See Alg. \ref{alg:RST}) in $3$ parts: the recursive formula, which computes the control points $s_i(t_j)$ as in \eqref{eq:rst_recursive} (line 5 of Alg. \ref{alg:RST}), the interpolation phase with Lagrange polynomials (line 7 of Alg. \ref{alg:RST}), and the generation of the $i$-partial trajectory (line 8 of Alg. \ref{alg:RST}).
The recursive formula provides $N+1$ control points $s_i(t_j)$ at the $i$-th iteration, with $i=0,1,\dots,k$ and $j=0,1,\dots,N$. In particular, to compute a single value $s_i(t_j)$ it needs a number of operations that goes as $\mathcal{O}(i\cdot \text{deg}(x_{i-1}(t))) \sim \mathcal{O}(i^2 \cdot N)$, where the first $i$ contribution comes from the $i$ derivatives of the $(i-1)$-partial trajectory. Hence, for $N+1$ points the complexity of the recursive formula is $N\mathcal{O}(i^2 \cdot N)$. The complexity of the Lagrange interpolation technique is $\mathcal{O}((N+1)^2) \sim \mathcal{O}(N^2)$. Lastly, the generation of the $i$-partial trajectory involves the product between $a^i(t)\cdot s_i(t)$, which requires a number of operations that grow as $\mathcal{O}((N+1)\cdot i \cdot (N+1)) \sim \mathcal{O}(i \cdot N^2)$. Since the number of iterations are $k+1$, the overall time complexity $T(k,N)$ reads as follows
\begin{align}
T(k,N) &= \sum_{i=0}^{k}{N\mathcal{O}(i^2 \cdot N)+\mathcal{O}(N^2)+\mathcal{O}(i \cdot N^2)} \nonumber \\
& \sim \mathcal{O}(k^3 \cdot N^2)+\mathcal{O}(k\cdot N^2)+\mathcal{O}(k^2 \cdot N^2) \nonumber \\
& \sim \mathcal{O}(k^3 \cdot N^2).
\end{align}
Although the estimated complexity is a rough approximation, it is interesting to highlight the following fact: the minimum degree polynomial trajectory $x_k(t)$ could have been derived in the classical approach just by evaluating the polynomial and its derivatives in the time stamps and by solving a system of linear equations. Alternative in a matrix form, $b = \mathbf{X}\cdot a$ where $\mathbf{X}$ is a $(k+1)(N+1)\times (k+1)(N+1)$ square matrix, badly conditioned from a numerical point of view, $a$ is the unknown vector of the polynomial coefficients and $b$ the vector of the kinematic constraints. To find the polynomial trajectory, thus, the coefficients, the matrix $\mathbf{X}$ needs to be inverted (when numerically possible). However, the matrix inversion operation involves a complexity of order $\mathcal{O}(k^3 \cdot N^3)$, that is higher than the RST complexity. Therefore, RST is not only numerically stable since no matrix inversion is required, but it is also faster than the classical interpolation approach (INV). Fig. \ref{fig:rst_time_complexity} illustrates the computational complexity advantage of RST over the classic interpolation method.

\begin{figure}
\includegraphics[scale=0.205]{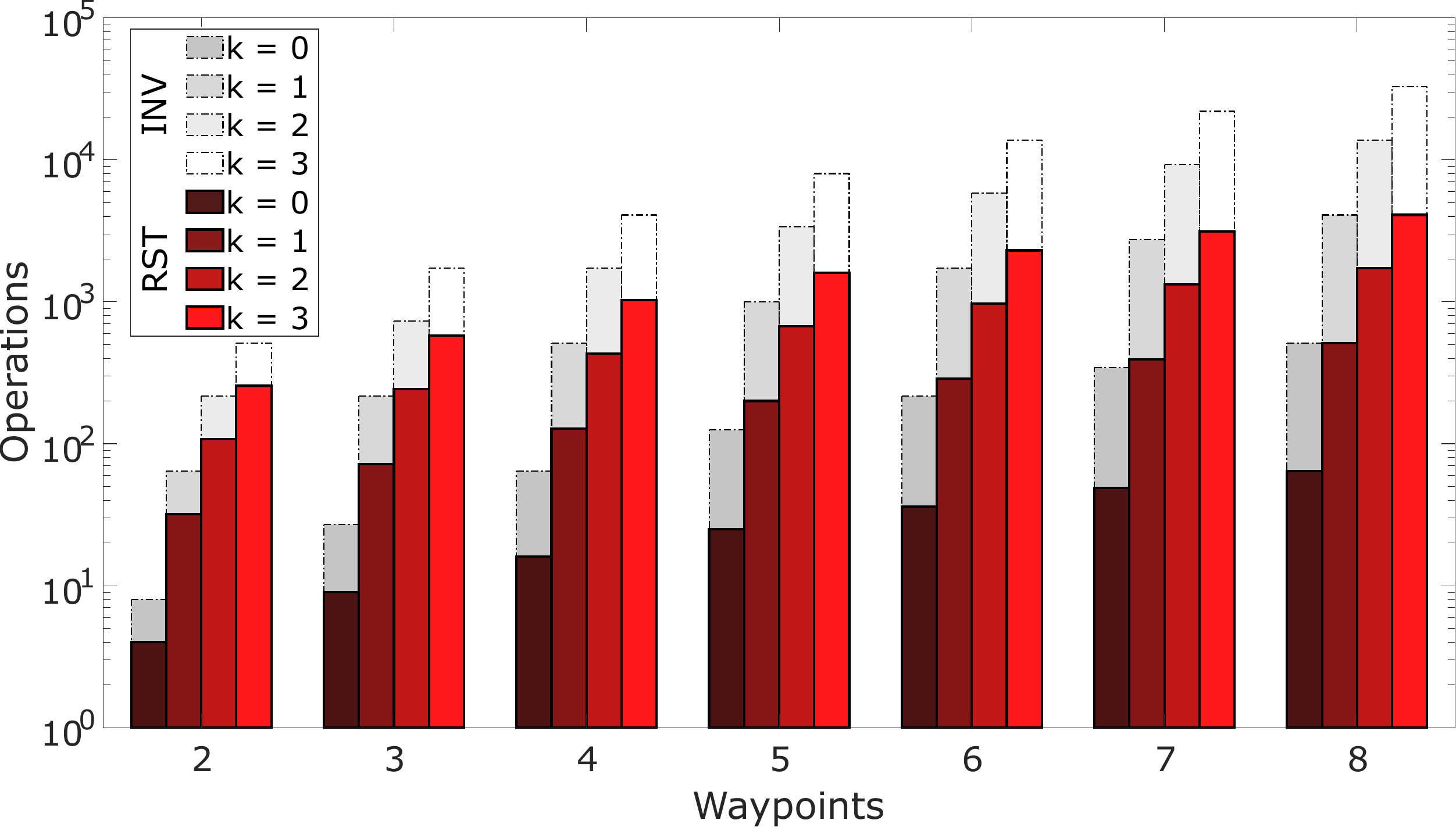}
\centering
\vspace{-0.1em}
\caption{Computational complexity comparison between RST and the classic interpolation approach through matrix inversion (INV).}
\vspace{-1.0em}
\label{fig:rst_time_complexity}
\end{figure}

\subsection{Extension of the proposed framework}
We presented the RST algorithm and extensions to block (BRST) and piecewise (PRST) approaches. Initial assumptions always considered time intervals with the same length or points in time following \eqref{Cheby}. The case that considers random initially located points in time can been studied under the optimization framework RST$_{\text{opt}}$. To tackle the oscillation problem (see Fig. \ref{fig:rst_Runge}), typical of high-order polynomial interpolation, a possible solution without involving the optimization step could either pass through spline interpolation or different interpolating polynomials such as barycentric Lagrange polynomials \cite{Berrut} or Newton ones.

Lastly and perhaps more fascinating, is the idea of mixing and eventually replacing polynomial trajectories with other basis functions. All the mathematical formulation and most of the derivation actually transcend the polynomial assumption. The only point in which this hypothesis plays a role is in the $h$-th derivative step (See Lemma \ref{lemma:rst_Lemma2}). The outcome of a further investigation is discussed in Sec. \ref{sec:rst_rrst}

\begin{figure}
\includegraphics[scale=0.50]{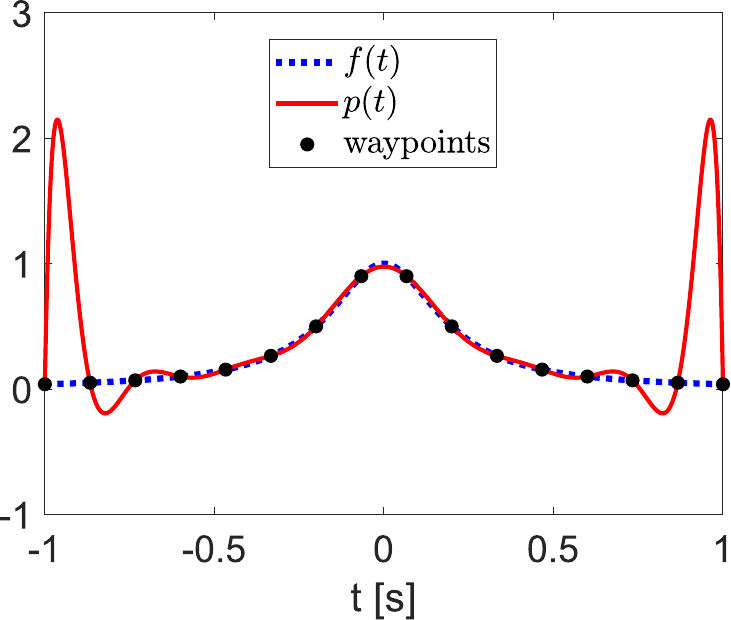}
\centering
\caption{Illustration of Runge's phenomenon: the Runge function (blue dashed line) is approximated with a $15$-th order polynomial (red solid line) that interpolates $16$ equally spaced nodes.}
\label{fig:rst_Runge}
\end{figure}

\section{Rational interpolation} 
\sectionmark{RRST}
\label{sec:rst_rrst}
Polynomial interpolation is in general a simple and fast process to implement. Nevertheless, when the degree of the interpolant function is high, oscillation at the edges may occur as mentioned before. For this reason, we consider a different basis function which may take advantage of the simplicity of polynomials but also provide more flexibility and degrees of freedom to tackle Runge's phenomenon.

\subsection{Rational recursive smooth trajectory}
We identify and propose a new basis as the rational basis function
\begin{equation}
R_{n,d}(t) = \frac{N(t)}{D(t)},
\end{equation}
where $N(t)$ is the numerator, a polynomial of degree $n$, and $D(t)$ is the denominator, a polynomial of degree $d$.
Such choice allows us to exploit some of the polynomial properties for both numerator and denominator but most importantly, enables the development of a new algorithm, referred to as rational recursive smooth trajectory (RRST). To find the coefficients of both numerator and denominator, the idea is to pick the denominator $D(t)$ and use RST to find the coefficients of the numerator $N(t)$. Intuitively, the new kinematic constraints for building $N(t)$ are a weighted sum of the kinematic constraints $\frac{d^i}{dt^i}f_k(t)\biggr|_{t=t_j}$ (given) and the kinematic constraints $\frac{d^i}{dt^i}D(t)\biggr|_{t=t_j}$ (designed as input). The following Lemma provides the mathematical formulation for the RRST.

\begin{lemma}
\label{lemma:rrst_Lemma1}
Let $t_j$ be a point in time, for $j=0,1,\dots, N$, such that $\frac{d^i}{dt^i}f_k(t)\bigr|_{t=t_j}$ is the associated given kinematic constraint, for $i=0,1,\dots, k$. Let $N(t)$ and $D(t)$ be polynomials with $D$ given of degree $d$. If $f_k(t)$ is a rational function defined as
\begin{equation}
f_k(t) = R_{n,d}(t) = \frac{N(t)}{D(t)},
\end{equation}
with $n=(k+1)(N+1)-1$, then the coefficients of $N(t)$ can be obtained with RST, in particular its associated kinematic constraint has expression
\begin{equation}
\frac{d^i}{dt^i}N(t)\biggr|_{t=t_j} = \sum_{l=0}^{i}{\binom{i}{l}\biggr(\frac{d^l}{dt^l}f_k(t)\bigr|_{t=t_j}\biggr) \cdot \biggr(\frac{d^{i-l}}{dt^{i-l}}D(t)\bigr|_{t=t_j}}\biggr).
\end{equation}
\end{lemma}

\begin{proof}
For simplicity of notation, the rational function $R_{n,d}(t)$ will be denoted with $R(t)$.
We proceed by induction on the kinematic constraint. Consider the case when $i=0$, then
\begin{equation}
N(t_j) = R(t_j)\cdot D(t_j)
\end{equation}
represents the value that $N(t)$ needs to assume at the time $t_j$. For the case $i=1$
\begin{equation}
\frac{d}{dt}N(t)\biggr|_{t=t_j} = \frac{d}{dt} \biggl(R(t)\cdot D(t)\biggr)\biggr|_{t=t_j}
\end{equation}
which is equal to
\begin{equation}
\frac{d}{dt}N(t)\biggr|_{t=t_j} = \binom{1}{0}R(t_j) \cdot \biggr( \frac{d}{dt}D(t)\bigr|_{t=t_j}\biggr) + \binom{1}{1} \biggr( \frac{d}{dt}R(t)\bigr|_{t=t_j}\biggr) \cdot D(t_j) .
\end{equation}
Suppose that the statement of the lemma is true for the case $i$, which means that
\begin{equation}
\frac{d^i}{dt^i}N(t)\biggr|_{t=t_j} = \sum_{l=0}^{i}{\binom{i}{l}\biggr(\frac{d^l}{dt^l}R(t)\bigr|_{t=t_j}\biggr) \cdot \biggr(\frac{d^{i-l}}{dt^{i-l}}D(t)\bigr|_{t=t_j}}\biggr).
\end{equation}
Then, it is true also for the case $i+1$. Indeed
\begin{align}
\frac{d^{i+1}}{dt^{i+1}}N(t)\biggr|_{t=t_j} = &\; 
\frac{d}{dt}\sum_{l=0}^{i}{ \binom{i}{l}\biggr(\frac{d^l}{dt^l}R(t)\bigr|_{t=t_j}\biggr) \cdot \biggr(\frac{d^{i-l}}{dt^{i-1}}D(t)\bigr|_{t=t_j}\biggr)} \nonumber \\
= &\; \sum_{l=0}^{i}{\binom{i}{l} \frac{d}{dt}\Biggl[ \biggr(\frac{d^l}{dt^l}R(t)\bigr|_{t=t_j}\biggr) \cdot \biggr(\frac{d^{i-l}}{dt^{i-1}}D(t)\bigr|_{t=t_j}\biggr)\Biggr] }  \nonumber \\ 
= &\; \sum_{l=0}^{i}{\binom{i}{l} \biggr(\frac{d^{l+1}}{dt^{l+1}}R(t)\bigr|_{t=t_j}\biggr) \cdot \biggr(\frac{d^{i-l}}{dt^{i-1}}D(t)\bigr|_{t=t_j}}\biggr) \nonumber \\
+ &\; \sum_{l=0}^{i}{\binom{i}{l} \biggr(\frac{d^{l}}{dt^{l}}R(t)\bigr|_{t=t_j}\biggr) \cdot \biggr(\frac{d^{i+1-l}}{dt^{i+1-l}}D(t)\bigr|_{t=t_j}\biggr)}
\end{align}
where we used the linearity of the differential operator and the product rule. With a change of variable in the first term of the RHS, $h=l+1$, it follows that
\begin{align}
\frac{d^{i+1}}{dt^{i+1}}N(t)\biggr|_{t=t_j} = &\; 
   \sum_{h=1}^{i+1}{\binom{i}{h-1} \biggr(\frac{d^{h}}{dt^{h}}R(t)\bigr|_{t=t_j}\biggr) \cdot \biggr(\frac{d^{i+1-h}}{dt^{i+1-h}}D(t)\bigr|_{t=t_j}}\biggr) \nonumber \\
+ &\; \sum_{l=0}^{i}{\binom{i}{l} \biggr(\frac{d^{l}}{dt^{l}}R(t)\bigr|_{t=t_j}\biggr) \cdot \biggr(\frac{d^{i+1-l}}{dt^{i+1-l}}D(t)\bigr|_{t=t_j}}\biggr) \nonumber \\ 
= &\; \sum_{l=0}^{i+1}{\binom{i+1}{l}\biggr(\frac{d^l}{dt^l}R(t)\bigr|_{t=t_j}\biggr) \cdot \biggr(\frac{d^{i+1-l}}{dt^{i+1-l}}D(t)\bigr|_{t=t_j}}\biggr)
\end{align}
where we used the Pascal's identity
\begin{equation}
\binom{i+1}{l} = \binom{i}{l-1} + \binom{i}{l}.
\end{equation}
Hence the result is true for $i+1$ and by induction is true for all positive integers. From Corollary \ref{corollary:rst_corollary1}, the minimum degree $n$ of $N(t)$ is $(k+1)(N+1)-1$.
\qedhere
\end{proof}

\begin{figure}[b]
\includegraphics[scale=0.25]{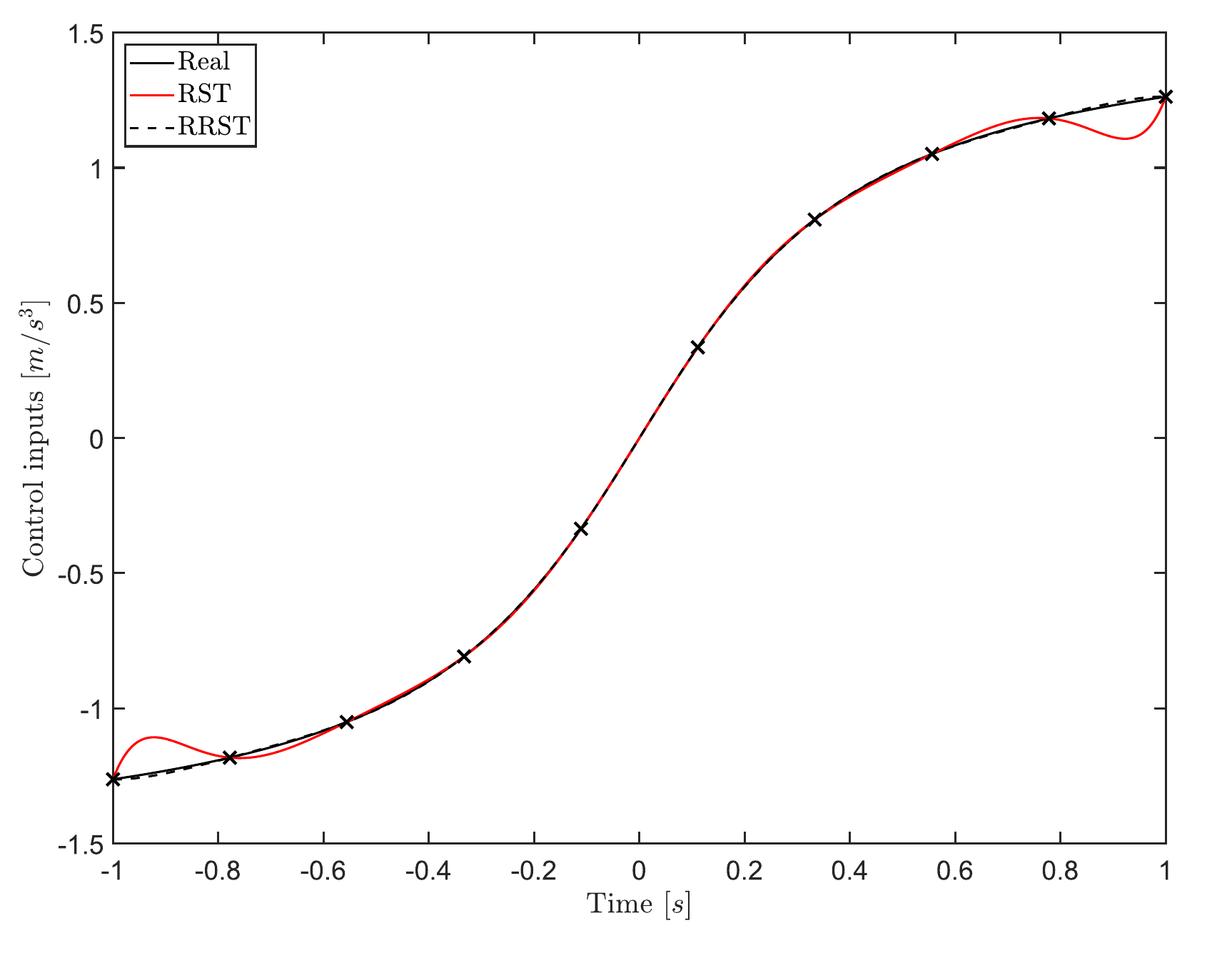}
      \centering
      \caption{Comparison between polynomial (RST) and rational (RRST) interpolation of $10$ waypoints, obtained as samples of the analytic control input $\text{arctan}(\pi t)$.}
      \label{fig:rst_arctan}
\end{figure}

\begin{algorithm}
\caption{Rational recursive smooth trajectory (RRST)}
\label{alg:rst_RRST}
\begin{algorithmic}[1]
\Inputs{$N+1$ points in time $t_0<t_1<\dots<t_N$; \\ Number of derivatives $k$ to fulfill; \\ Kin. constr.
 $\frac{d^i}{dt^i}f_k(t)\bigr|_{t=t_0}, \dots, \frac{d^i}{dt^i}f_k(t)\bigr|_{t=t_N}$; \\ Denominator $D(t)$ of degree $d$. \\}
\Initialize{Kin. constr. $\frac{d^i}{dt^i}D(t)\bigr|_{t=t_0}, \dots, \frac{d^i}{dt^i}D(t)\bigr|_{t=t_N}$;}
\For{$i=0$ to $k$}
	\For{$j=0$ to $N$}
		\State $\frac{d^i}{dt^i}N(t)\biggr|_{t=t_j} =$
          \State $\sum_{l=0}^{i}{\binom{i}{l}\biggr(\frac{d^l}{dt^l}f_k(t)\bigr|_{t=t_j}\biggr) \cdot \biggr(\frac{d^{i-l}}{dt^{i-l}}D(t)\bigr|_{t=t_j}}\biggr)$;
	\EndFor
\EndFor
\State Get $N(t)$ with RST given the kinematic constraints $\frac{d^i}{dt^i}N(t)\biggr|_{t=t_j}$ as input;
\State $f_k(t)=\frac{N(t)}{D(t)}$.
\end{algorithmic}
\end{algorithm}

Lemma \ref{lemma:rrst_Lemma1} provides the general expression of the kinematic constraints associated to $N(t)$, however it assumes that the denominator $D(t)$ is given. The choice of the denominator remains an open question although some considerations can be made. The denominator represents a whole set of degree of freedoms and therefore the choice of the coefficients should in principle consider some strategies. For example, a fundamental aspect is the position of the roots inside the interval $[t_0, \hspace{0.2em} t_N]$. Indeed, if one real pole (denominator root) falls inside the desired interval, it may cause discontinuities in the rational interpolant. To avoid this, a possible strategy relies on the selection of multiple complex conjugate roots. Further studies have to be made in the roots locus analysis for such rational function but they go out of the scope of this section therefore we postpone these questions to future work. Finally, it is interesting to notice that if the denominator $D(t)$ is constant, we lead back to the classical polynomial interpolation via RST, therefore we can tract RRST as a rational basis extension of the RST algorithm. The implementation of the RRST algorithm is we reported in the pseudo code of Alg. \ref{alg:rst_RRST}.

To show how the RRST tackles the oscillation problem, we report in Fig.~\ref{fig:rst_arctan} an example of function approximation with polynomials (RST) and rational functions (RRST). In particular, we select as function to interpolate $f_k(t)= \text{arctan}(\pi t)$, with $t\in [-1,1]$. Fig.~\ref{fig:rst_arctan} illustrates the resulting interpolants when the number of waypoints is set to $10$ and no kinematic constraints (from velocity on) are imposed. The denominator of the rational function is set to $D(t)=t^2+0.1$ and for such choice, RRST shows to perform better than the polynomial interpolant at the edges.

In conclusion, we extended the RST algorithm to rational functions. The algorithm can effectively generate an analytic expression that approximates control inputs, for which no closed-form solutions are in general attainable. More details are offered in \cite{9525383}.

\section{Summary}
\label{subsec:rst_con}
This chapter investigated the problem of trajectory generation for unmanned vehicles, and in particular for UAVs. The RST algorithm has been mathematically derived and implemented with the objective to find a path satisfying arbitrary conditions, specified in terms of position, velocity, acceleration, etc. An analysis of the perturbation in the kinematic constraints has been carried out with the result that the uncertainty in the constraint can be translated into uncertainty in the polynomial coefficients. An optimization framework has been developed in order to improve the trajectory generated with RST, which in general is not optimal by itself. 

Two examples have been presented with the RST methodology and compared to a common benchmark in the field such as the minimum-snap piecewise polynomial trajectory algorithm. It should be noted that in the latter approach, a joint optimization process (quadratic programming) is needed to fulfill the kinematic constraints. Our RST approach eliminates the optimization step and directly provides an analytic smooth trajectory. Furthermore, whenever the number of waypoints increases, we extended our methodology to a blockwise framework (BRST) where interfaces are intrinsically jointly matched. Finally, when the number of waypoints in a single block is equal to $2$, BRST converges to the piecewise one (PRST) which generates optimal trajectories in terms of minimum-snap.  

Results about the application to a UAV structure proved the capability and effectiveness of RST (and BRST, PRST extensions) to outperform the piecewise approach in terms of smoothness. The results show that RST offers a new direction in the domain of trajectory generation and path planning.

\chapter{Conclusion}
\label{sec:conclusion}
This dissertation proposes novel ML and DL techniques to advance the field of physical layer communications. The research contributions encompass generative models, coding and decoding neural architectures, mutual information estimation and capacity learning frameworks.
In the following we examine the key conclusions and future implications of this work.

A fundamental theme centers on the segmentation of complex statistical problems into more tractable, understandable components using mathematical and information theoretical tools.
Ch.~\ref{sec:copulas} proposed a solution with a segmented GAN architecture, leveraging copula theory to separate dependence modeling from marginal distribution estimation. This fosters explainable ML and effective generation of data with complex statistical relationships. It also offers a simple framework to estimate any data distribution with NNs.

Ch.~\ref{sec:medium} presented a powerful DL methodology to model and produce synthetic channel and noise traces for any communication medium. This is achieved via pre-processing steps (with latent space representations or spectrograms), GAN-based generation, and suitable post-processing.

To effectively rethink building blocks of a communication chain as DL components, MI emerges as a cornerstone for several contributions in this thesis:
\begin{itemize}

\item Ch.~\ref{sec:decoder} introduced MIND, a MI-based decoder that attains performance of the genie MAP decoder, even without precise knowledge of the channel. The envisioned neural decoder also returns precise estimates of the a-posteriori PDFs;

\item Ch.~\ref{sec:autoencoders} discussed the limitations in traditional cross-entropy loss functions used with AE-based communication systems. We leveraged MI as a fundamental regularization term for training AEs, leading to optimal coding schemes and to the introduction of algorithms and architectures for capacity-achieving codes;

\item Ch.~\ref{sec:mi_estimators} introduced $f$-DIME, a novel class of discriminative MI estimators based on the variational representation of $f$-divergence. We proved that the variance of the estimators is the lowest in literature, and we proposed a derangement training strategy for efficient sampling. We finally evaluated $f$-DIME performance and demonstrated its excellent bias/variance trade-off. 

\item Ch.~\ref{sec:cortical} unveiled CORTICAL, a novel framework for learning capacity-achieving distributions in arbitrary discrete-time continuous memoryless channels. It employs a two-network cooperative game, validating its efficacy in non-Shannon channel settings. CORTICAL represents a powerful new DL tool for intricate channel capacity analysis. 
\end{itemize}

Ch.~\ref{sec:plc} demonstrated the value of DL in PLC systems:
Sec.~\ref{sec:plc_entanglement} and ~\ref{sec:plc_nakagami} showcase the potential for intelligent impedance detection and capacity estimation in PLC channels with non-AWGN noise (leveraging the CORTICAL framework).
Sec.~\ref{sec:plc_detection} highlights the successful application of NNs for accurate anomaly detection within PLNs.

Ch.~\ref{sec:rst} presented data interpolation techniques under given constraints exploiting polynomials and rational functions as basis. It reinforces the notion that a robust understanding of both theoretical principles and practical tools is essential for driving innovation.

\section*{Limitations}
DL contributions can be commonly grouped into three macro-areas such as: a) data collection and preparation; b) loss function design; c) neural architecture design.
Although the contribution of this thesis mostly centered around b), we discussed and provided insights for the other two areas as well. For instance, in Sec.~\ref{sec:medium_channelsynthesis} we discussed pre-processing and data transformation techniques to prepare generative models and in Sec.~\ref{subsec:mi_derangements} we showed how  training improves if we use a derangement sampling strategy. We also designed the deployed NN architectures and made several observations on the type of activation functions to be used based on the desired output. However, several limitations remain:
\begin{itemize}
    \item availability and quality of data, as our generative models are built upon datasets collected during measurement campaigns where we assumed time-invariance and stationarity as hypothesis; 
    \item during our research, we neither focused on studying architectural improvements nor on adapting state-of-the-art models for our purpose. Thus, we expect that many of our results can be improved over the next couple of years just by utilizing more advances neural architectures;
    \item for tasks such as MI estimation and capacity learning, it is difficult to investigate training converge properties, especially at high SNRs. Therefore, our theoretical derivations always assume the existence of an optimal solution, which may be different from the one found during training;
    \item we often neglected time-dependence on the input-output channel pairs. Nevertheless, the proposed formulation could be adapted and extended to such case with different architectures;
    \item as a more general limitation, while DL holds promise for the future of communications engineering, its widespread integration is not imminent. To achieve implementation in industry standards, DL models must consistently outperform existing solutions in terms of accuracy, efficiency, and robustness.
\end{itemize}

\section*{Future directions}
This thesis opens multiple avenues for future research, such as: 
\begin{itemize}
    \item an holistic system design, since the developed techniques could lead to a completely data-driven approach for communication system design, where channel modeling, coding/decoding, and capacity learning all synergistically utilize novel DL methods;
    \item data modeling and generation pipelines where probability estimation and distribution sampling tools are integrated with theoretical information theory principles to derive new bounds, performance limits, and insights into the fundamental nature of communication systems;
    \item exploration of model-based strategies to mitigate the computational expense of training, particularly focusing on methods for efficient offline training and transfer learning. Additionally, investigations into model compression and quantization techniques could enable the deployment of complex NNs on resource-constrained online devices;
    \item apply CORTICAL to tackle open problems such as the two-user Gaussian interference channel in order to get the achievable rates and the achievable distribution, which could inspire and offer theoretical insights on new coding schemes;
    \item beyond physical layer communications, since, for instance, the proposed capacity learning framework can be applied in biology, to analyze the genes’ activation, in genetics, to study the capacity of DNA storage systems, for privacy-preserving data release mechanism, secrecy capacity, or more in general for representation learning. 
\end{itemize}

In summary, this dissertation establishes innovative DL methodologies that significantly contribute to the understanding and optimization of communication systems. It lays a foundation for a new paradigm where data-driven approaches, explainability and the ability to bridge theoretical insights with practical applications pave the way towards the next generation of intelligent and efficient communication infrastructure.

\newpage
\bibliographystyle{unsrt}
\addcontentsline{toc}{chapter}{Bibliography} 
\bibliography{IEEEabrv,biblio}
\end{document}